\documentclass{article}

\usepackage{microtype}
\usepackage{graphicx}
\usepackage{subfigure}
\usepackage{booktabs} 

\usepackage{hyperref}

\usepackage[accepted]{icml2024}

\usepackage{amsmath}
\usepackage{amssymb}
\usepackage{mathtools}
\usepackage{amsthm}
\usepackage{multirow}
\usepackage{makecell}
\usepackage{tabu}
\usepackage{color}
\usepackage{colortbl}
\usepackage{bm}
\usepackage{algorithm}
\usepackage{algorithmic}
\usepackage{subfigure}

\usepackage[capitalize,noabbrev]{cleveref}

\theoremstyle{plain}
\newtheorem{theorem}{Theorem}[section]

\newtheorem{lemma}[theorem]{Lemma}

\theoremstyle{definition}
\newtheorem{definition}[theorem]{Definition}

\theoremstyle{remark}

\newcommand{\header}[1]{\noindent\textbf{#1}.}

\definecolor{grey}{rgb}{0.1,0.1,0.1}

\usepackage[textsize=tiny]{todonotes}

\icmltitlerunning{Improving Interpretation Faithfulness for Vision Transformers}

\begin{document}

\twocolumn[
\icmltitle{Improving Interpretation Faithfulness for Vision Transformers}



\icmlsetsymbol{equal}{*}

\begin{icmlauthorlist}
\icmlauthor{Lijie Hu}{equal,1,3,4}
\icmlauthor{Yixin Liu}{equal,2}
\icmlauthor{Ninghao Liu}{5}
\icmlauthor{Mengdi Huai}{6}
\icmlauthor{Lichao Sun}{2}
\icmlauthor{Di Wang}{1,3,4}
\end{icmlauthorlist}

\icmlaffiliation{1}{King Abdullah University of Science and Technology (KAUST)}
\icmlaffiliation{2}{Lehigh University}
\icmlaffiliation{3}{Provable Responsible AI and Data Analytics (PRADA) Lab}
\icmlaffiliation{4}{SDAIA-KAUST AI}
\icmlaffiliation{5}{University of Georgia}
\icmlaffiliation{6}{Iowa State University}

\icmlcorrespondingauthor{Di Wang}{di.wang@kaust.edu.sa}

\icmlkeywords{Machine Learning, ICML}

\vskip 0.3in
]



\printAffiliationsAndNotice{\icmlEqualContribution} 

\begin{abstract}
Vision Transformers (ViTs) have achieved state-of-the-art performance for various vision tasks. One reason behind the success lies in their ability to provide plausible innate explanations for the behavior of neural architectures. However, ViTs suffer from issues with explanation faithfulness, as their focal points are fragile to adversarial attacks and can be easily changed with even slight perturbations on the input image.
In this paper, we propose a rigorous approach to mitigate these issues by introducing Faithful ViTs (FViTs). Briefly speaking, an FViT should have the following two properties: (1) The top-$k$ indices of its self-attention vector should remain mostly unchanged under input perturbation, indicating stable explanations; (2) The prediction distribution should be robust to perturbations.
To achieve this, we propose a new method called Denoised Diffusion Smoothing (DDS), which adopts randomized smoothing and diffusion-based denoising. We theoretically prove that processing ViTs directly with DDS can turn them into FViTs. We also show that Gaussian noise is nearly optimal for both $\ell_2$ and $\ell_\infty$-norm cases. Finally, we demonstrate the effectiveness of our approach through comprehensive experiments and evaluations. 
Results show that FViTs are more robust against adversarial attacks while maintaining the explainability of attention, indicating higher faithfulness.
\end{abstract}

\section{Introduction}
\label{sec:intro}
Transformers and attention-based frameworks have been widely adopted as benchmarks for natural language processing tasks \citep{kenton2019bert,radford2019language}. Recently,  their ideas have also been borrowed in many computer vision tasks such as image recognition \citep{dosovitskiy2021an}, objective detection \citep{zhu2021deformable}, image processing \citep{chen2021pre} and semantic segmentation \citep{zheng2021rethinking}. Among them, the most successful variant is the vision transformer (ViT) \citep{dosovitskiy2021an}, which uses self-attention modules. Similar to tokens in the text domain, ViTs divide each image into a sequence of patches (visual tokens), and then feed them into self-attention layers to produce representations of correlations between visual tokens. The success of these attention-based modules is not only because of their good performance but also due to their ``self-explanation'' characteristics. Unlike post-hoc interpretation methods \citep{du2019techniques}, attention weights can intrinsically provide the ``inner-workings'' of models \citep{meng2019interpretable}, i.e., the entries in attention vector could point us to the most relevant features of the input image for its prediction, and can also provide visualization for ``where'' and ``what'' the attention focuses on \citep{xu2015show}.

As a crucial characteristic for explanation methods, faithfulness requires that the explanation method reflects its reasoning process \citep{jacovi2020towards}. Therefore, for ViTs, their attention feature vectors should reveal what is essential to their prediction.
Furthermore, faithfulness encompasses two properties: completeness and stability. Completeness means that the explanation should cover all relevant factors or patches related to its corresponding prediction \citep{sundararajan2017axiomatic}, while
stability ensures that the explanation is consistent with humans understanding and robust to slight perturbations. Based on these properties, we can conclude that ``faithful ViTs" (FViTs) should have the {\bf good robustness performance}, and certified {\bf explainability} of attention maps which are {\bf robust against perturbations}.

\begin{figure*}[t] 
    \setlength{\tabcolsep}{1pt} 
    \renewcommand{\arraystretch}{1} 
    \begin{center}
    \begin{tabular*}{\linewidth}{@{\extracolsep{\fill}}ccccccc}
     Corrupted Input & Raw Attention & Rollout 
     & GradCAM &LRP& VTA  & Ours \\
    \includegraphics[width=0.12\linewidth]{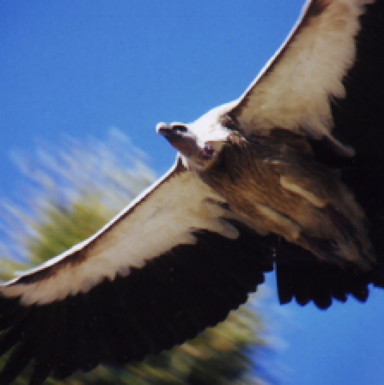}
&\includegraphics[width=0.12\linewidth]{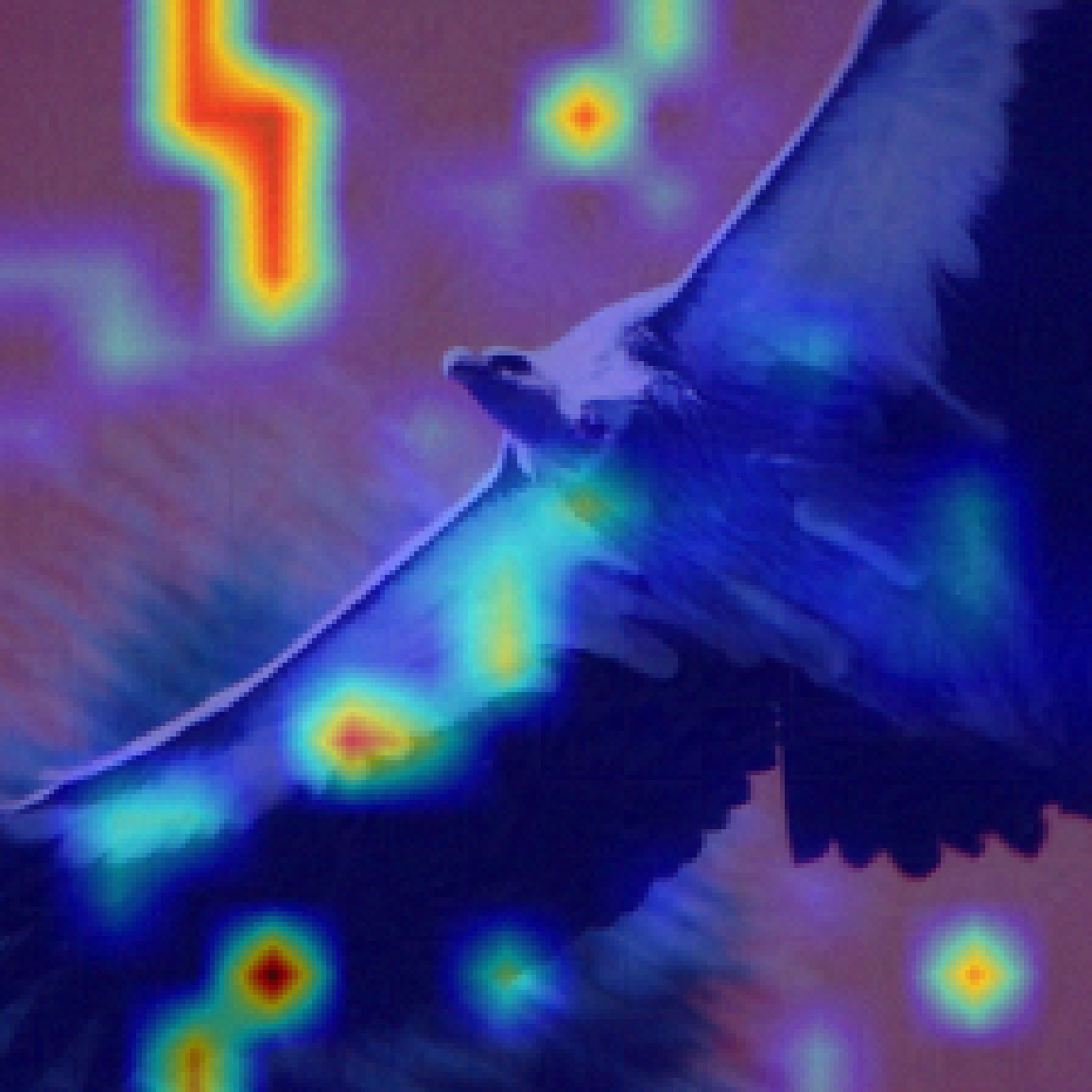} &
    \includegraphics[width=0.12\linewidth]{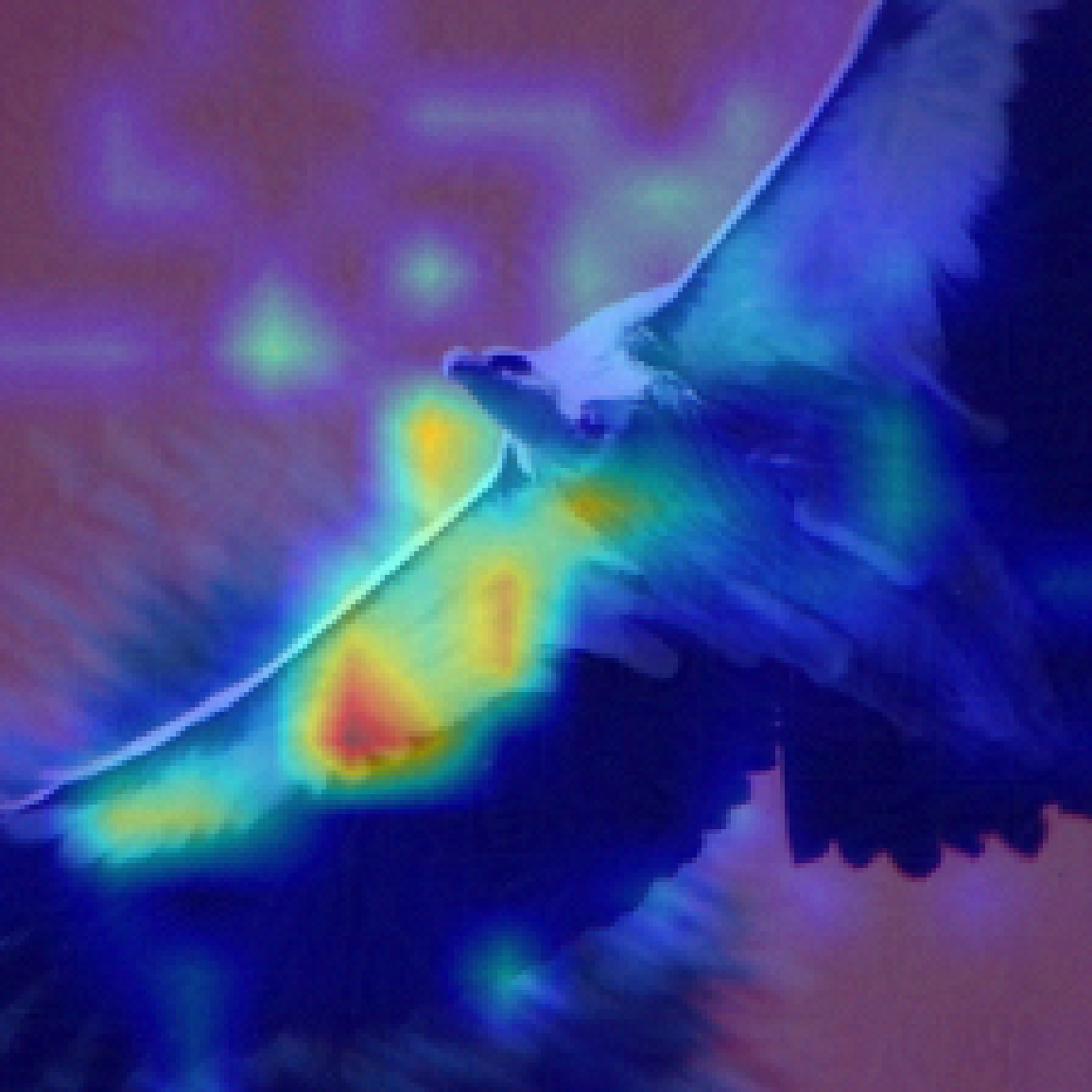} &
    \includegraphics[width=0.12\linewidth]{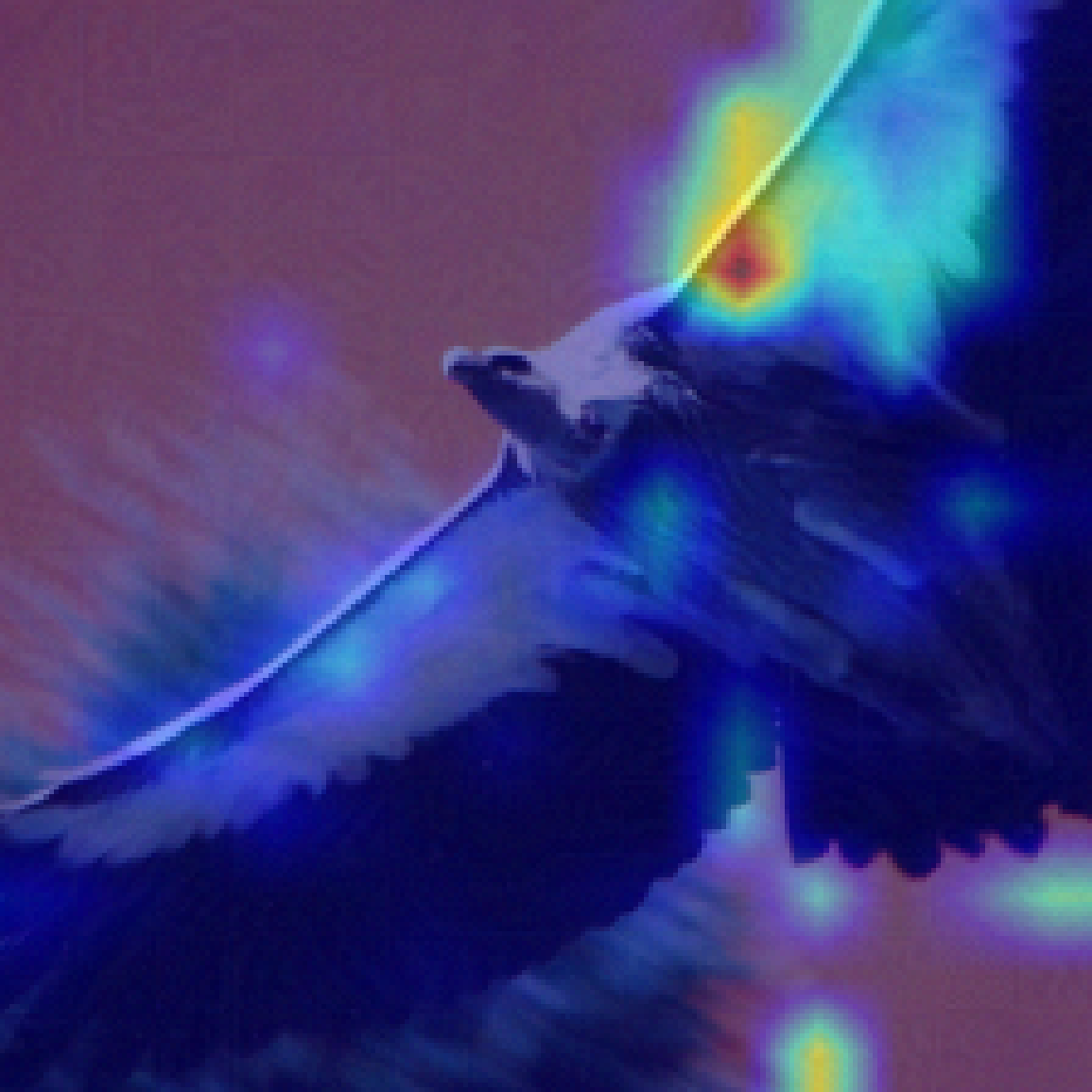} & \includegraphics[width=0.12\linewidth]{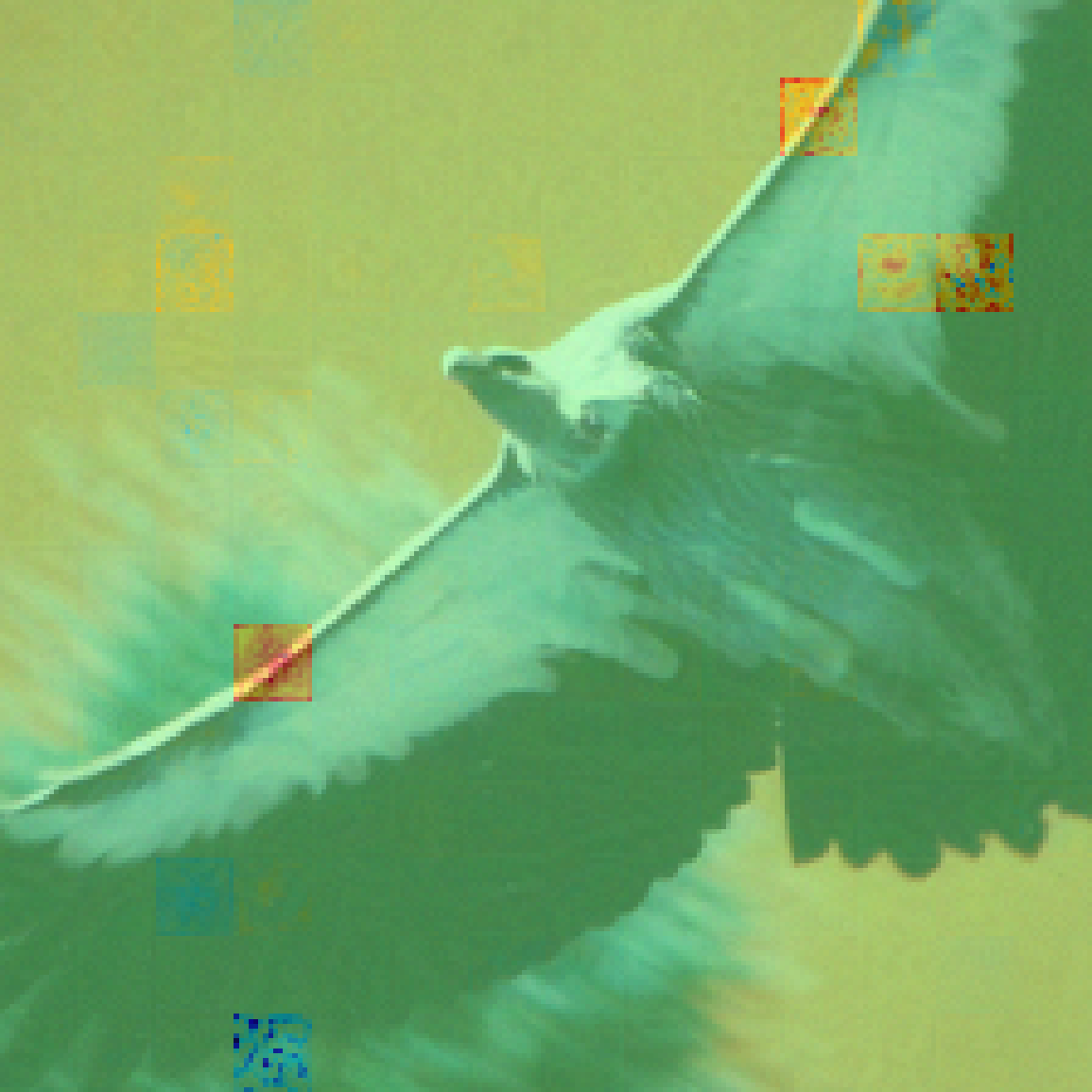} & \includegraphics[width=0.12\linewidth]{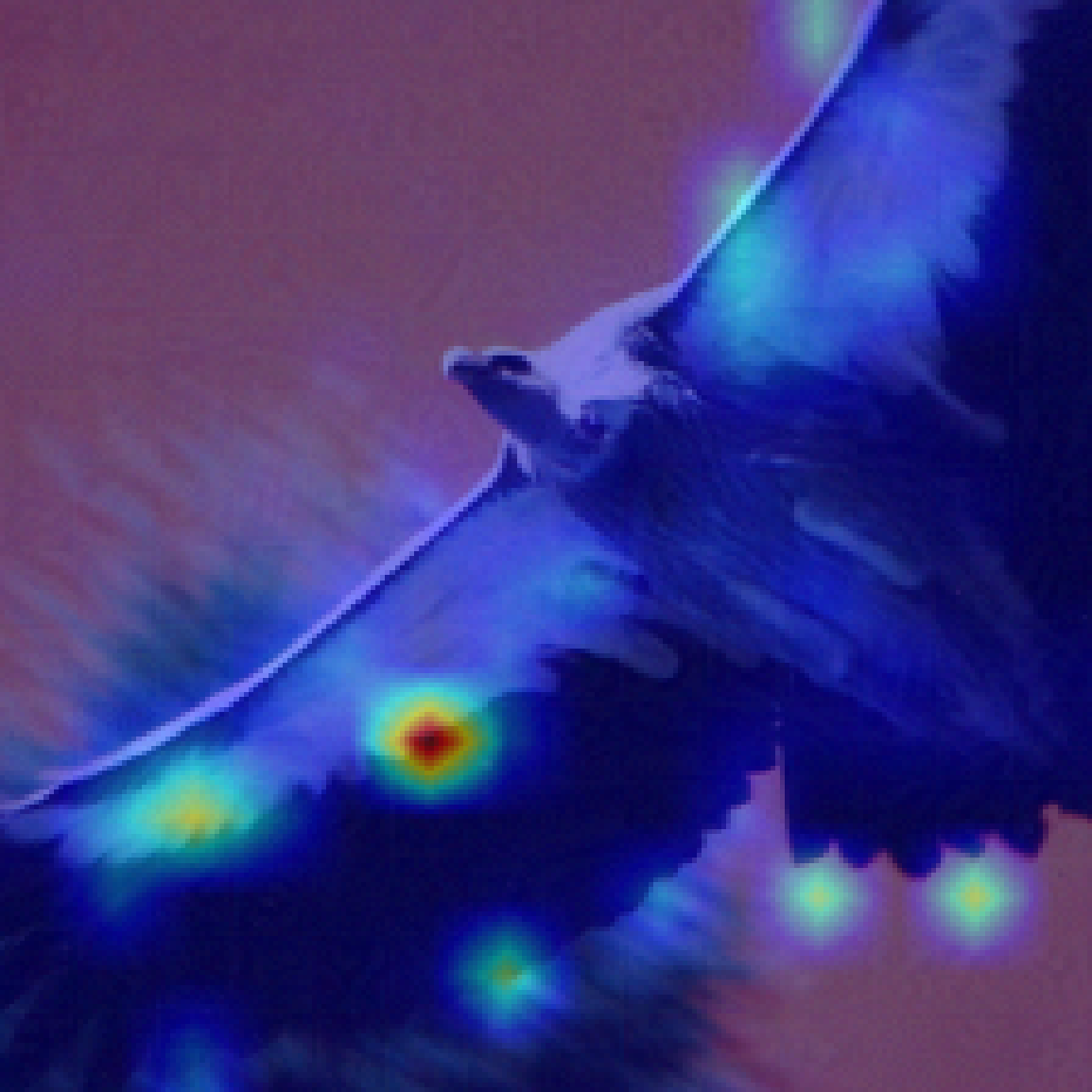} 
    & \includegraphics[width=0.12\linewidth]{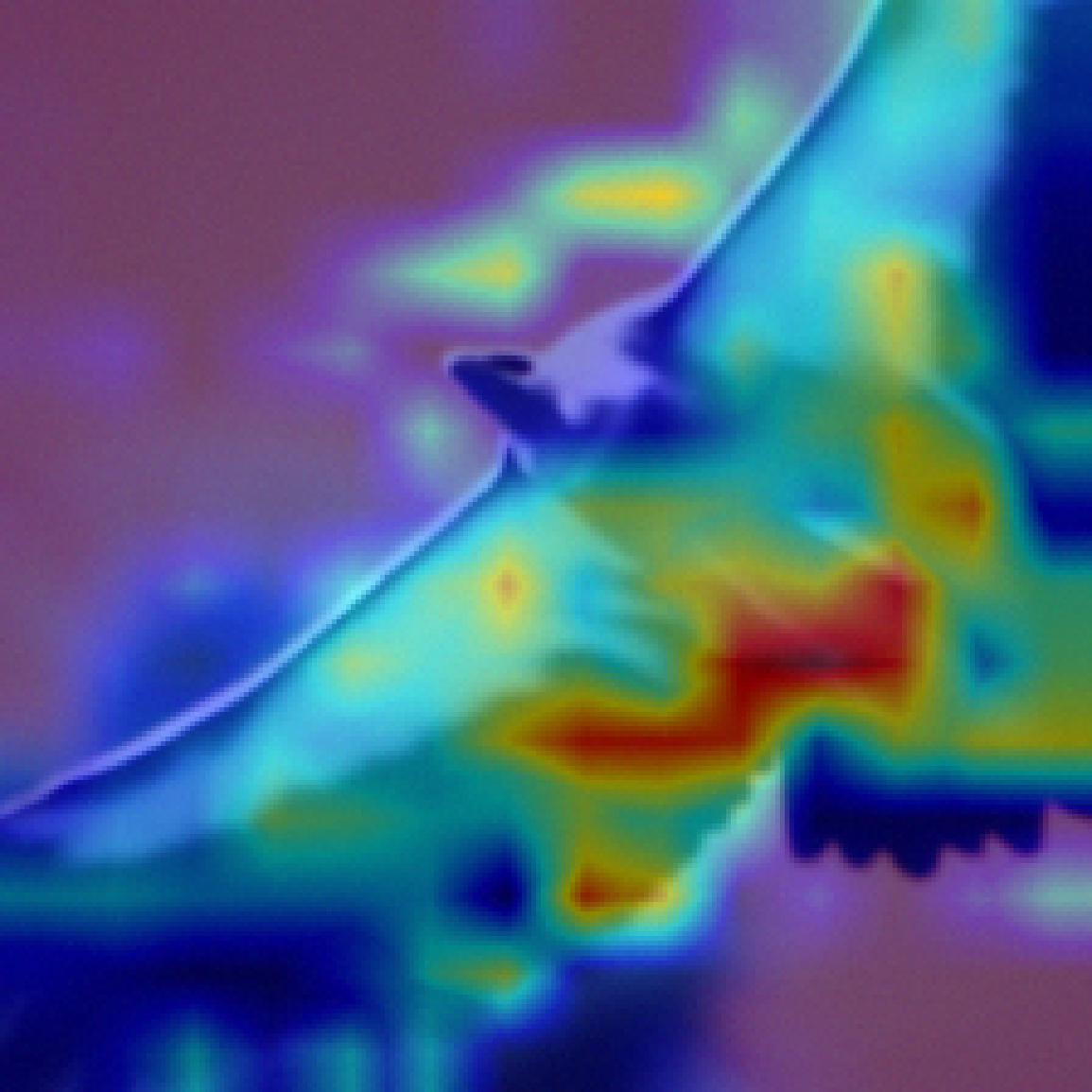} \\
   \includegraphics[width=0.12\linewidth]{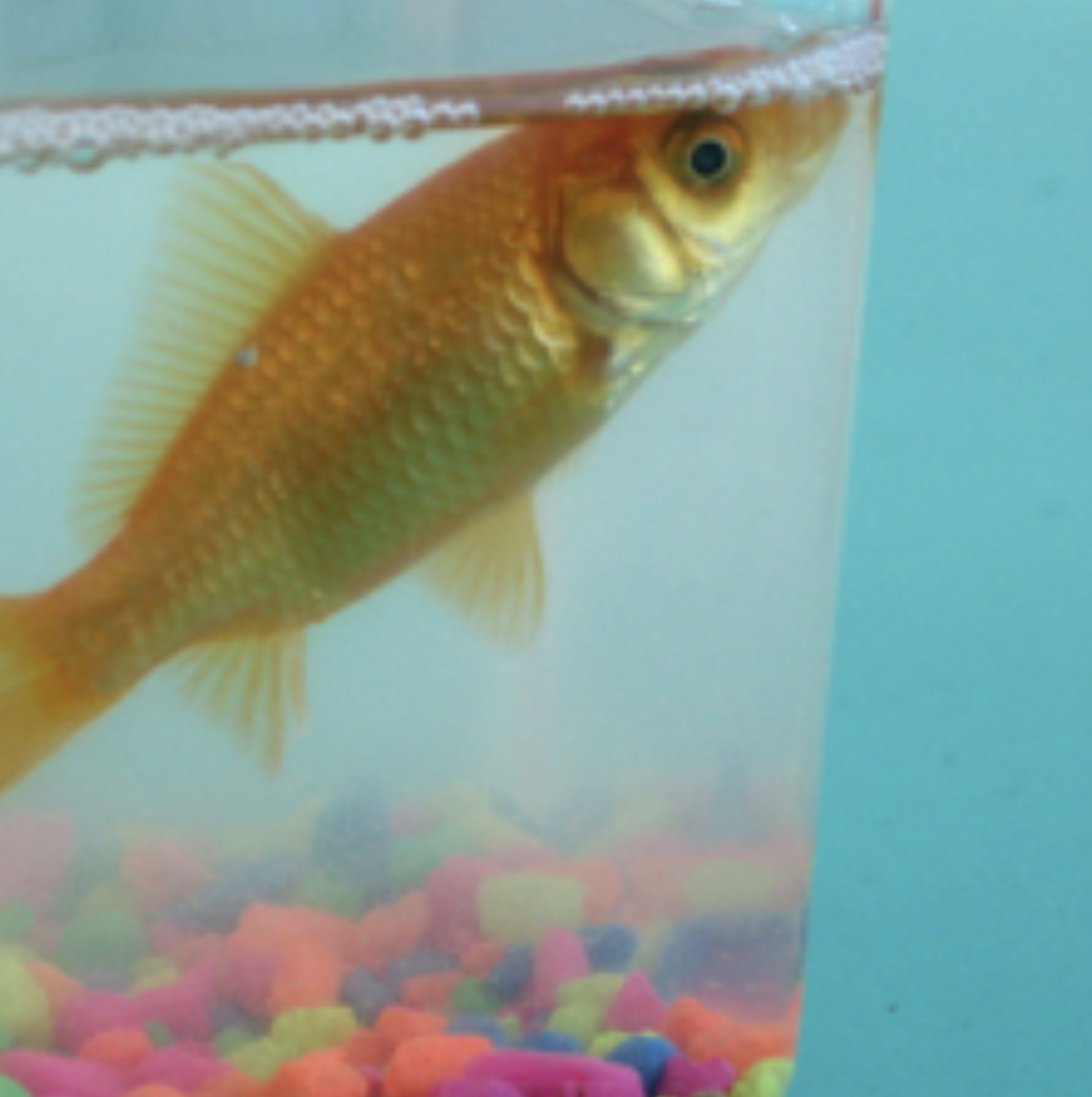}
    &
    \includegraphics[width=0.12\linewidth]{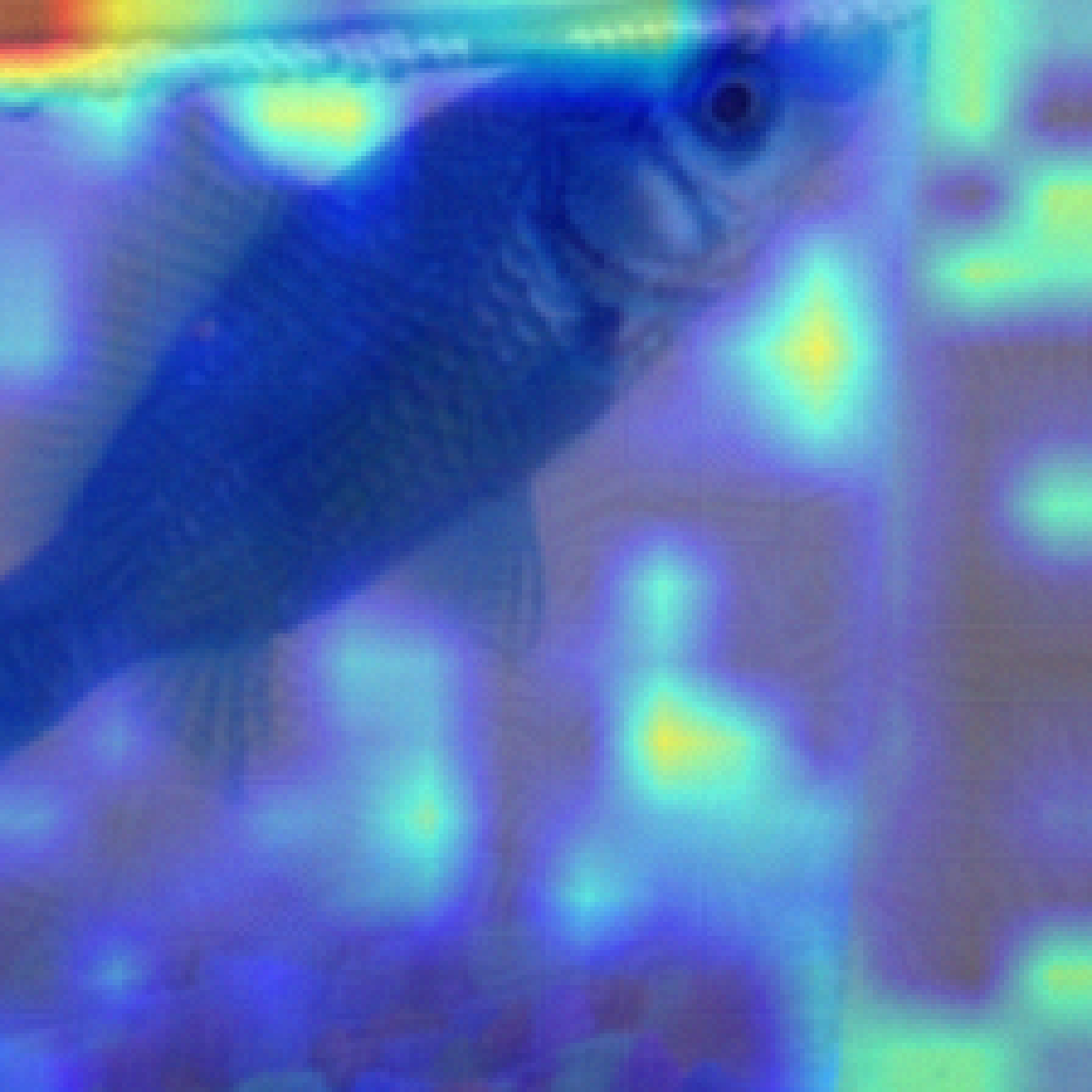} &
    \includegraphics[width=0.12\linewidth]{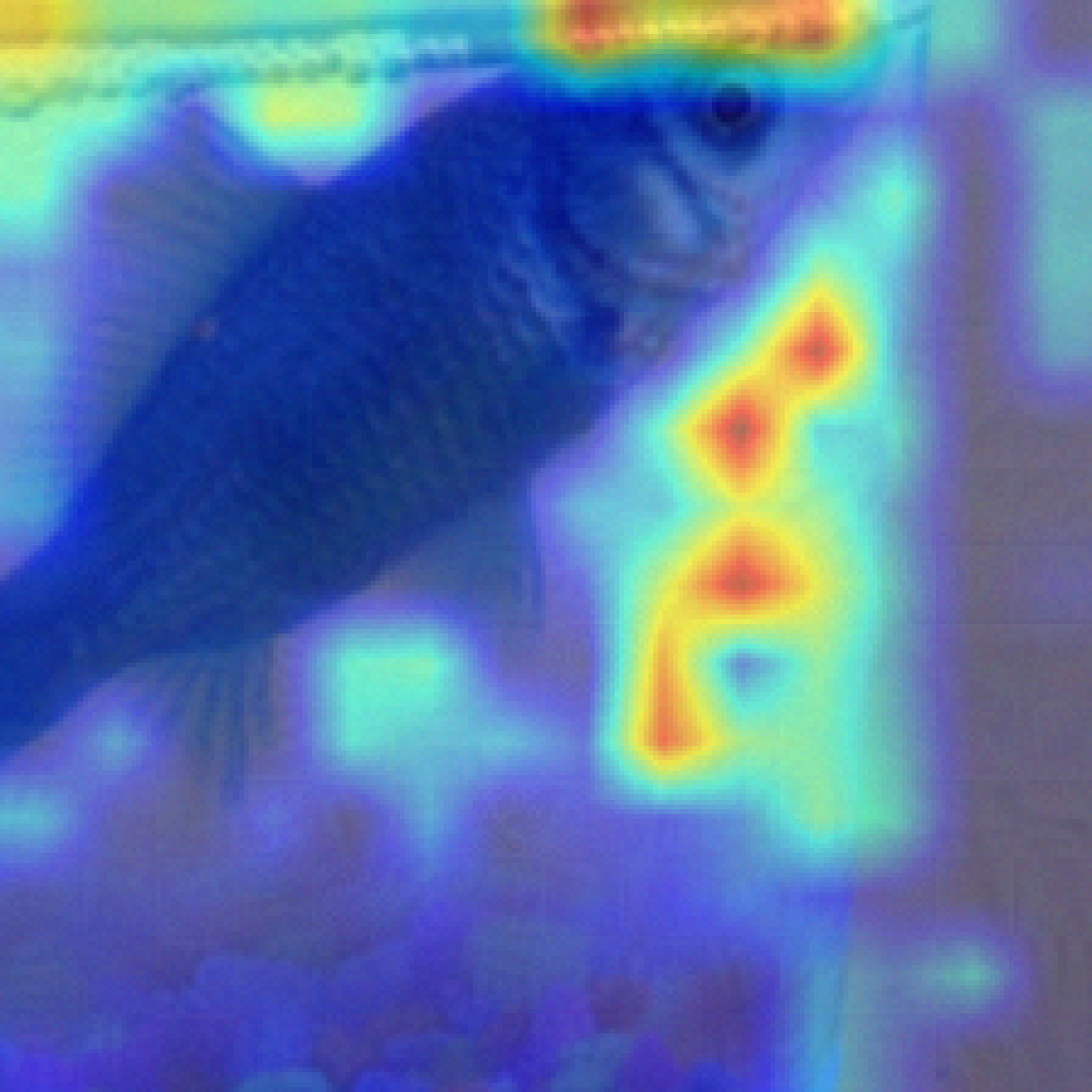} &
    \includegraphics[width=0.12\linewidth]{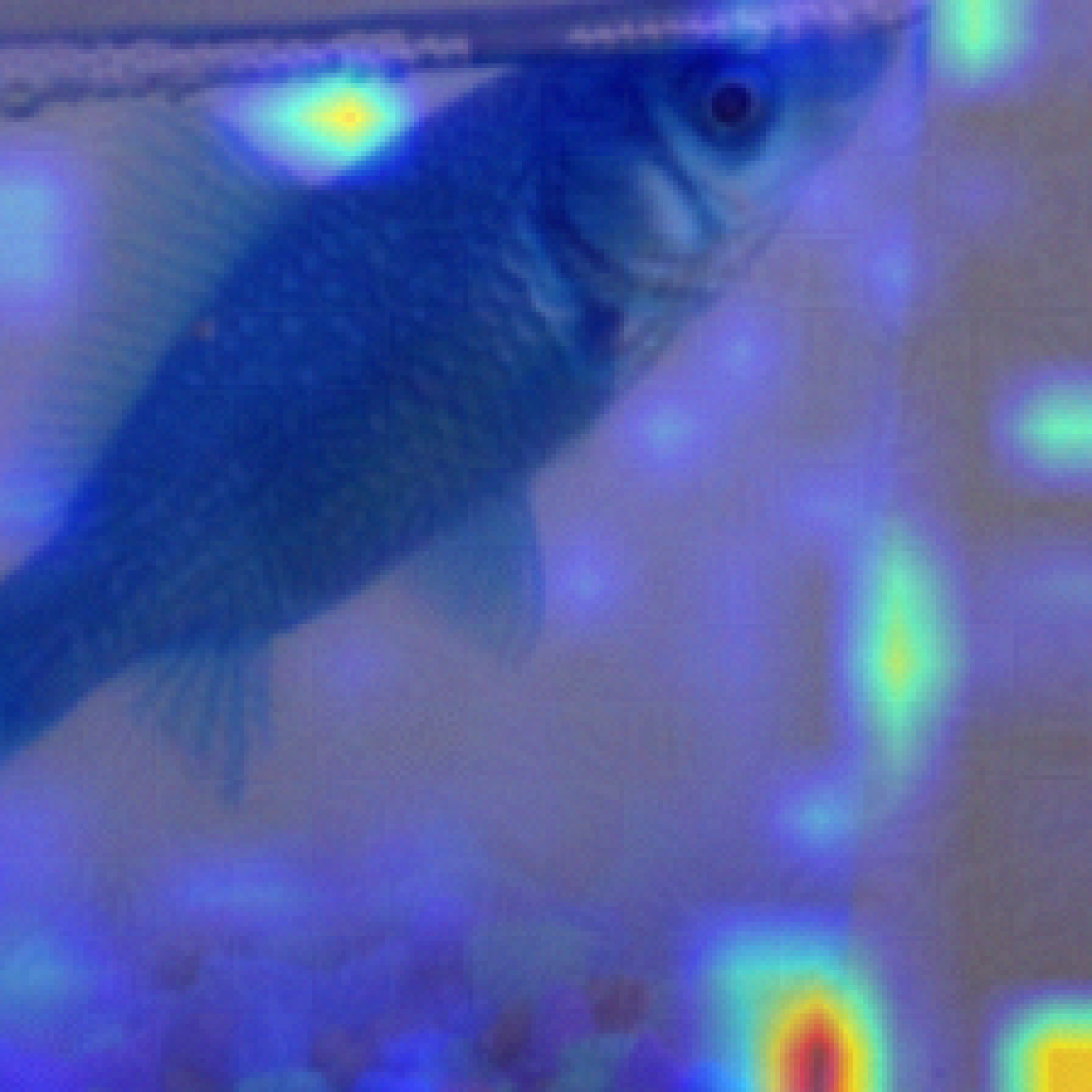} &
    \includegraphics[width=0.12\linewidth]{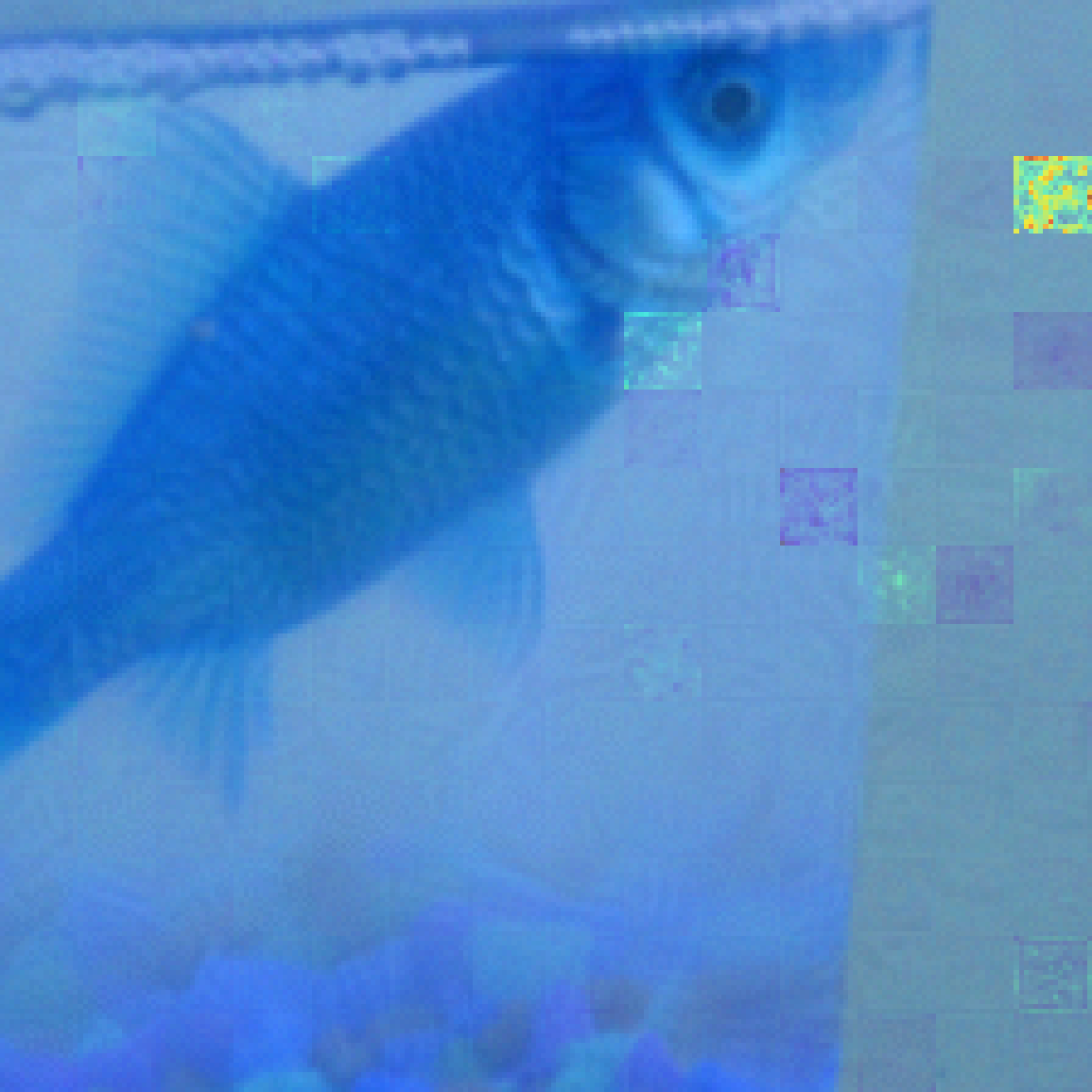} &    \includegraphics[width=0.12\linewidth]{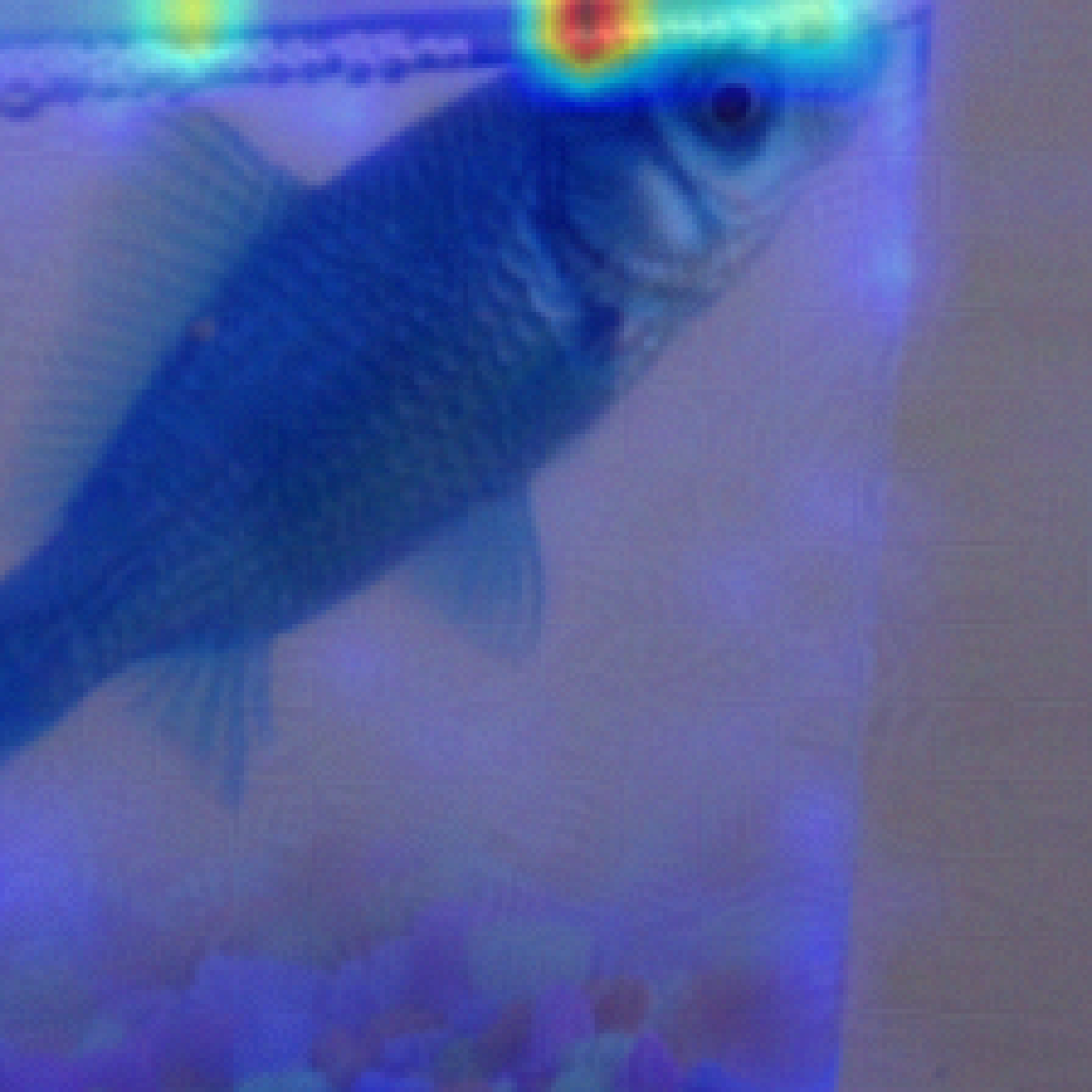} &\includegraphics[width=0.12\linewidth]{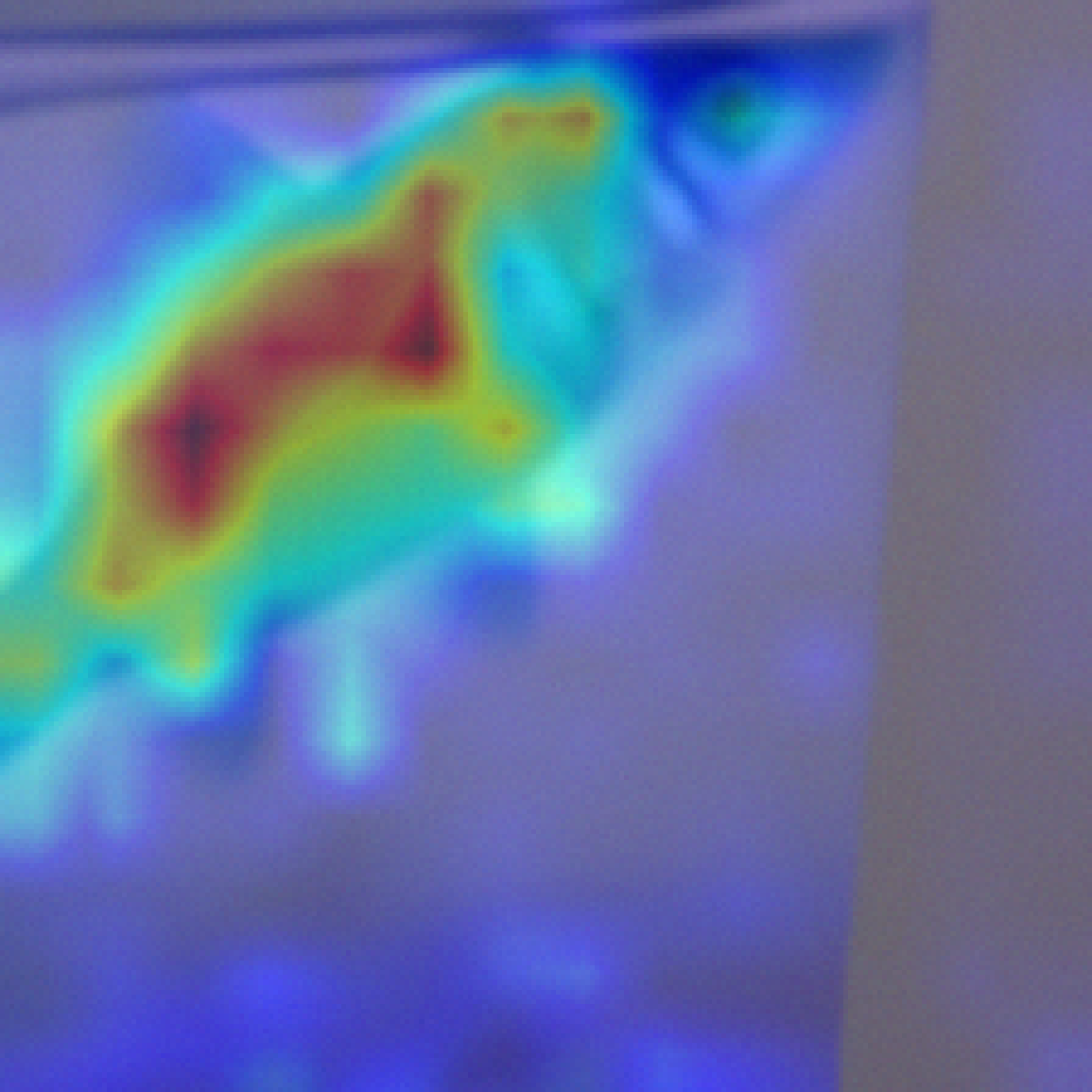} 
    \end{tabular*}
    \caption{Visualization results of the attention map on corrupted input for different methods.}
    \label{fig: exp}
    \end{center}
\vspace{-0.2in}
\end{figure*}

While the self-attention module in ViTs already possesses the property of completeness, it often suffers from stability issues. For example, in language tasks, \citet{wiegreffe2019attention,hu2022seat} have shown that attention is not robust, as a slight perturbation on the embedding vector can cause the attention vector distribution to change drastically. Similarly, for vision tasks, we find that such a phenomenon is not uncommon for ViTs, where attention modules are sensitive to perturbations on input images. For instance, in Fig. \ref{fig: exp}, we can see that a slight perturbation on the input image can cause the attention vector to focus on the wrong region for the image class, leading to an incorrect interpretation heat map and, consequently, confusing predictions. Actually, interpretation instability has also been identified as a pervasive issue in deep learning models \citep{ghorbani2019interpretation}, where carefully crafted small perturbations on the input can significantly change the interpretation result. Thus, stability has become an important factor for faithful interpretations. First, an unstable interpretation is prone to noise interference, hindering users' understanding of the underlying reasoning behind model predictions. Second, instability undermines the reliability of interpretation as a diagnostic tool for models \citep{ghorbani2019interpretation, dombrowski2019explanations, yeh2019fidelity, hu2022seat}. Therefore, it is crucial to mitigate the issue of unstable explanations in ViTs.

Although ViTs have shown to be more robust in predictions than convolutional neural networks (CNNs) \citep{bai2021transformers,paul2022vision,naseer2021intriguing}, whether they can provide {\bf faithful interpretations} remains uncertain. 
Compared to adversarial machine learning, which aims to enhance the robustness of classification accuracy, mitigating the unstable explanation issue in ViTs is more challenging. Here, we not only aim to improve the robustness of prediction performance but also, more importantly, {\bf make the attention heat map robust to perturbations}.  
Therefore, a natural question arises: {\bf can we develop provable and faithful variants of ViTs whose attention vectors and predictions are robust to perturbations?} In this paper, we provide an affirmative answer to this question. Specifically, our contributions can be summarized as follows. 

1. We propose a formal and mathematical definition for FViTs. Specifically, an FViT must satisfy the following two properties for any input image: (1) It must ensure that the top-$k$ indices of its attention feature vector remain relatively stable with perturbations, indicating interpretability stability. (2) However, attention vector stability alone cannot guarantee prediction robustness for FViTs. Therefore, an FViT must also ensure that its prediction distribution remains relatively stable under perturbations. 
    
2. We propose a method called Denoised Diffusion Smoothing (DDS) to obtain FViTs. Surprisingly, we demonstrate that our DDS can directly transform ViTs into FViTs. Briefly speaking, DDS involves two main components: (1) the standard randomized smoothing with Gaussian noise and (2) a denoising diffusion probabilistic model. It is worth noting that prior work on randomized smoothing has focused on enhancing prediction robustness. In contrast, we demonstrate that randomized smoothing can also enhance the faithfulness of attention vectors in ViTs. Additionally, we demonstrate that Gaussian noise is near-optimal for our method in both the $\ell_2$-norm and $\ell_\infty$-norm cases.

3. We conducted extensive experiments on three benchmark datasets using ViT, DeiT, and Swin models to verify the above two properties of the FViTs as claimed by our theories. Firstly, we demonstrate that our FViTs are more robust to different types of perturbations than other baselines. Secondly, we show the interpretability faithfulness of FViTs through visualization. Lastly, we verify our certified faithful region as claimed by our theories. The results reveal that our FViTs can provide provable faithful interpretations. 

Due to space limitations, we have included details on theoretical proofs and additional experimental results in the Appendix section.

\section{Related Work}
\noindent{\bf Interpretability for computer vision tasks.} Interpretation approaches in computer vision can be broadly categorized into two classes according to the part of models they participated in: post-hoc interpretation and self-explaining interpretation. Post-hoc interpretation methods require post-processing the model after training to explain the behaviors of black-box models. Typically, these methods either use surrogate models to explain local predictions \citep{ribeiro2016should}, or adopt gradient perturbation methods \citep{zeiler2014visualizing, lundberg2017unified} or feature importance methods \citep{ross2017right,selvaraju2017grad}. While post-hoc approaches require additional post-processing steps after the training process, self-explaining interpretation methods could be considered as integral parts of models, and they can generate explanations and make predictions simultaneously \citep{li2018tell,abnar2020quantifying}. From this view, ViTs can be considered as self-explaining interpretation approaches, as they use attention feature vectors to quantify the contributions of different tokens to their predictions.

\noindent{\bf Faithfulness in explainable methods.}
Faithfulness is a crucial property that explanation models should satisfy, ensuring that the explanation accurately reflects the true reasoning process of the model \citep{wiegreffe2019attention,herman2017promise,jacovi2020towards,lyu2022towards}. 
Faithfulness is also related to other principles such as sensitivity, implementation invariance, input invariance, and completeness \citep{yeh2019fidelity}. Completeness refers to the explanation comprehensively covering all relevant factors to the prediction \citep{sundararajan2017axiomatic}, while the other three terms are all related to the stability of different kinds of perturbations. The explanation should change if heavily perturbing the important features that influence the prediction \citep{adebayo2018sanity}, but be stable to small perturbations \citep{yin2022sensitivity}. Thus, stability is crucial to explanation faithfulness. Some preliminary work has been proposed to obtain stable interpretations. For example, \citet{yeh2019fidelity} theoretically analyzes the stability of post-hoc interpretation and proposes the use of smoothing to improve interpretation stability. \citet{yin2022sensitivity} designs an iterative gradient descent algorithm to get a counterfactual interpretation, which shows desirable stability. However, these techniques are designed for post-hoc interpretation and cannot be directly applied to attention-based mechanisms like ViTs. 

\noindent{\bf Robustness for ViTs.} There is also a substantial body of work on achieving robustness for ViTs, including studies such as \citep{mahmood2021robustness,salman2022certified,aldahdooh2021reveal,naseer2021intriguing,paul2022vision,mao2022towards}. However, these studies exclusively focus on improving the model's robustness in terms of its prediction, without considering the stability of its interpretation (i.e., attention feature vector distribution). While we do employ the randomized smoothing approach commonly used in adversarial machine learning, our primary objective is to maintain the top-$k$ indices unchanged under perturbations. And we introduce DDS, which leverages a smoothing-diffusion process to obtain faithful ViTs while also enhancing prediction performance.

\section{Vanilla Vision Transformers}
In this paper, we adopt the notation introduced in \citep{zhou2022understanding} to describe ViTs. ViTs process an input image $x$ by first dividing it into $n$ uniform patches. Each patch is then represented as a token embedding, denoted as $x_i \in \mathbb{R}^q$ for $i = 1,\cdots,n$. The token embeddings are then fed into a stack of transformer blocks, which use self-attention for token mixing and MLPs for channel-wise feature transformation.

\noindent \textbf{Token mixing.} 
A ViT makes use of the self-attention mechanism to aggregate global information. Given an input token embedding tensor $X = [x_1,\cdots,x_n] \in \mathbb{R}^{q\times n}$, self-attention applies linear transformations with parameters $W_K$, $W_Q$, 
to embed $X$ into a key $K = W_K X \in \mathbb{R}^{q \times n}$ and a query $Q = W_Q X \in \mathbb{R}^{q \times n}$ respectively. The self-attention module then calculates the attention matrix and aggregates the token features as follows:
\begin{equation}\label{eq:1}
Z^{\top} = \text{self-attention}(X) = \text{softmax}(\frac{Q^{\top}K}{\sqrt{q}})V^{\top}W_L,
\end{equation}
where $Z = [z_1, \cdots, z_n]$ is the aggregated token feature, $W_L \in \mathbb{R}^{q \times q}$ is a linear transformation, and $\sqrt{q}$ is a scaling factor. The output of the self-attention is normalized with Layer-norm and fed into an MLP to generate the input for the next block. At the final layer, the ViT outputs a prediction vector. It is worth highlighting that the self-attention mechanism can be viewed as a function that maps each image $X \in \mathbb{R}^{q\times n}$ to an attention feature vector $Z(X)\in \mathbb{R}^n$.

\section{Towards Faithful Vision Transformers}
As mentioned in the introduction, improving the stability and robustness of the self-attention modules in ViTs is crucial for making them more faithful. However, when it comes to explanation methods, it is not only important to consider the robustness of the model's prediction under perturbations but also the sensitivity and stability of its explanation modules. Specifically, {\bf the explanation modules should be sensitive enough to important token perturbations while remaining stable under noise perturbations}. As the attention mechanism in ViTs outputs a vector indicating the importance of each visual token, it is necessary to reconsider the robustness of both the attention module and ViTs. In general, a faithful ViT should satisfy the following two properties:

1. The magnitude of each entry in the attention vector indicates the importance of its associated patch. To ensure interpretability robustness, it is sufficient to maintain the order of leading entries. We measure interpretability robustness by computing the overlap of the top-$k$ indices between the attention vector of the original input and the perturbed input, where $k$ is a hyperparameter. 
 
2. The attention vector is fed to an MLP for prediction. In addition to the robustness for top-$k$ indices, a robust attention vector should also preserve the final model prediction. Specifically, we aim for the prediction distribution based on the robust attention under perturbations is almost the same as the distribution without perturbation. We measure the similarity or closeness between these distributions using different divergences.

\vspace{-0.1in}
\subsection{Motivation and Challenges}
\noindent \textbf{Motivation.} Although some literature has addressed methods to improve the robustness of ViTs, to the best of our knowledge, this is the first paper to propose a solution that enhances the faithfulness of ViTs while providing provable FViTs. Our work fills a gap in addressing both the robustness and interpretability of ViTs, as demonstrated through both theoretical analysis and empirical experiments.

\noindent \textbf{Why stability of attention vectors' top-$k$ indices cannot imply the robustness of prediction?} While ensuring the stability of the top-$k$ indices of the attention vectors is crucial for interpretability, it does not necessarily guarantee the robustness of the final prediction. This is because the stability of the prediction is also dependent on the magnitude of the entries associated with the top-$k$ indices. For instance, consider the vectors $v_1=(0.1, 0.2, 0.5, 0.7)$ and $v_2 = (0.2, 0.8, 0.9, 2)$, which have the same top indices. However, the difference in their magnitudes can significantly affect the final prediction. Therefore, in addition to ensuring the stability of the top-$k$ indices, an FViT should also meet the requirement that its prediction distribution remains relatively unchanged under perturbations to achieve robustness.

\noindent \textbf{Technical Challenges.} The technical challenges of this paper are twofold. First, we need to give a definition of FViTs, which contains the conditions that can quantify the stability of both the attention vectors and the model prediction under perturbations. This is challenging because we need to balance the sensitivity and stability of the explanation modules, and also consider the trade-off between interpretability and utility. Addressing these technical challenges is critical to achieving the main objective of this paper, which is to provide faithful  ViTs. Second, we need to design an efficient and effective algorithm to generate noise to preserve the robustness and interpretability of ViTs. This is challenging because standard noise methods may cause significant changes to the attention maps, which could lead to inaccurate and misleading explanations. To tackle this challenge, we introduce a mathematically proven approach for noise generation and leverage denoised diffusion to balance the utility and interpretability trade-off, which is non-trivial and provable. Also, this study is the first to demonstrate its effectiveness in enhancing explanation faithfulness, providing rigorous proof, and certifying the faithfulness of ViTs.

\subsection{Definition of FViTs}
In the following, we will mathematically formulate our above intuitions. Before that, we first give the definition of the top-$k$ overlap ratio for two vectors.

\begin{definition}
For vector $x\in \mathbb{R}^n$, we define the set of top-$k$ component $T_k(\cdot)$ as 
\begin{equation*}
    T_k(x)=\{i: i\in [d] \text{ and } \{|\{x_j\geq x_i: j\in [n]\}|\leq k\} \}.
\end{equation*}
And for two vectors $x$, $x'$, their top-$k$ overlap ratio $V_k(x, x')$  is defined as $V_k(x, x')=\frac{1}{k} | T_k(x) \cap T_k(x')|. $ 
\end{definition}

\begin{definition}[Faithful ViTs]
We call a function $f: \mathbb{R}^{q \times n}\mapsto \mathbb{R}^n$ is an $(R, D, \gamma,  \beta, k,  \|\cdot\|)$-\textit{ faithful attention module} for ViTs if  for any given  
input data $\bm{x}$  and for all $x'\in \mathbb{R}^{q\times n}$ such that $\|x-x'\|\leq R$, $f(x')$ satisfies 

\noindent 1. (Top-$k$ Robustness)  $V_k(f(x'), f(x))\geq \beta$. 

\noindent 2. (Prediction Robustness)  $D(\bar{y}(x), \bar{y}(x'))\leq \gamma$, where $\bar{y}(x), \bar{y}(x')$ are the prediction distribution of ViTs based on $f(x), f(x')$ respectively.

We also call the vector $f(x)$ as an $(R, D, \gamma,  \beta, k,  \|\cdot\|)$-\textit{ faithful attention} for $x$, and the models of ViTs based on $f$ as \textit{ faithful ViTs} (FViTs). 
\end{definition}
We can see there are several terms in the above definition. Specifically, $R$ represents the faithful radius, which measures the faithful region;  $D$ is a metric of the similarity between two distributions, which could be a distance or a divergence; 
$\gamma$ measures the closeness of the two prediction distributions; $0<\beta<1$ is the robustness of top-$k$ indices; $\|\cdot\|$ is some norm.  When $\gamma$ is smaller or $\beta$ is larger, then the attention module will be more robust and thus will be more faithful. 
In this paper, we will focus on the case where divergence $D$ is the R\'{e}nyi divergence and $\|\cdot\|$ is either the $\ell_2$-norm  or the $\ell_\infty$-norm (if we consider $x$ as a $d=q\times n$ dimensional vector), as we can show if the prediction distribution is robust under R\'{e}nyi divergence, then the prediction will be unchanged with perturbations on input \citep{li2019certified}. 
\begin{definition}
Given two probability distributions $P$ and $Q$, and  $\alpha\in (1,\infty)$, the $\alpha$-R\'{e}nyi divergence  $D_\alpha(P||Q)$ is defined as $D_\alpha(P||Q) = \frac{1}{\alpha-1} \log \mathbb{E}_{x\sim Q}(\frac{P(x)}{Q(x)})^\alpha$.
\end{definition}

\begin{theorem}\label{thm:0}
If a function is a $(R, D_\alpha, \gamma,  \beta, k,  \|\cdot\|)$-{faithful attention module} for ViTs, then if 
\begin{equation*}
    \gamma \leq - \log(1 - p_{(1)} - p_{(2)} + 2(\frac{1}{2}(p_{(1)}^{1 - \alpha} + p_{(2)}^{1 - \alpha}))^{\frac{1}{1- \alpha}}),
\end{equation*}
we have for all $x'$ such that where $\|x - x'\| \leq R$, 
\begin{equation*}
\arg\max_{g \in \mathcal{G}} \mathbb{P}(\bar{y}(x)=g) = \arg\max_{g \in \mathcal{G}} \mathbb{P}(\bar{y}(x')=g),   
\end{equation*}
where $\mathcal{G}$ is the set of classes,
$p_{(1)}$ and $p_{(2)}$ refer to the largest and the second largest probabilities in $\{ p_i \}$, where $p_i$ is the probability that $\bar{y}(x)$ returns the $i$-th class.
\end{theorem}

\section{Finding Faithful Vision Transformers}
\subsection{\texorpdfstring{$\ell_2$}{Lg}-norm Case}
In the previous section, we introduced faithful attention and FViTs, now we want to design algorithms to find such a faithful attention module. We notice that in faithful attention, the condition of prediction robustness is quite close to adversarial machine training which aims to design classifiers that are robust against perturbations on inputs. Thus, a natural idea is to borrow the approaches in adversarial machine training to see whether they can get faithful attention modules. Surprisingly, we find that using randomized smoothing to the vanilla ViT, which is a standard method for certified robustness \citep{cohen2019certified}, and then applying a denoised diffusion probabilistic model \citep{ho2020denoising} to the perturbed input can adjust it to an FViT. And its corresponding attention module becomes a faithful attention module. Specifically, for a given input image $x$, we preprocess it by adding some randomized Gaussian noise, i.e., $\tilde{x} = x + z$ with $z \sim \mathcal{N}(0, \sigma^2 I_{q \times n})$. Then we will denoise $\tilde{x}$ via some denoised diffusion model to get $\hat{x}$, and feed the perturbed-then-denoised $\hat{x}$ to the self-attention module $Z$ in (\ref{eq:1}) and process to later parts of the original ViT to get the prediction. Thus, in total, we can represent the attention module as $\tilde{w}(x) = Z(T(x+z))$, where $T$ represents the denoised diffusion method. Here, we mainly adopt the 
denoising diffusion probabilistic model in \citep{nichol2021improved,ho2020denoising,carlini2023certified}, which leverages off-the-shelf diffusion models as image denoiser. Specifically, it has the following steps after we add Gaussian noise to $x$ and get $\tilde{x}$. 

In the first step, we establish a connection between the noise models utilized in randomized smoothing and diffusion models. Specifically, while randomized smoothing augments data points with additive Gaussian noise i.e., $x_{\mathrm{rs}} \sim \mathcal{N}(x, \sigma^2 \mathbf{I})$, diffusion models rely on a noise model of the form $x_t \sim \mathcal{N}(\sqrt{\alpha_t} x, (1-\alpha_t) \mathbf{I})$, 
where the factor $\alpha_t$ is a constant derived from the timestamp $t$ (i.e., $\alpha_t:=\prod_{s=1}^t 1-\beta_s$). To employ a diffusion model for randomized smoothing, DDS scales $x_{\mathrm{rs}}$ by $\sqrt{\alpha_t}$ and adjusts the variances to obtain the relationship $\sigma^2 = \frac{1 - \alpha_t}{\alpha_t}$. The formula for this equation may vary depending on the schedule of the $\alpha_t$ terms employed by the diffusion model, but it can be calculated in closed form.

Using this calculated timestep, we can then compute $x_{t^\star} = \sqrt{\alpha_{t^\star}} (x + \delta)$, where $\delta \sim \mathcal{N}(0,\sigma^2 \mathbf{I})$, and apply a diffusion denoiser on $x_{t^\star}$ to obtain an estimate of the denoised sample, $\hat{x} = \texttt{denoise}(x_{t^\star};t^\star)$. To further enhance the robustness, we repeat this denoising process multiple times (e.g., 100{,}000). The details of our method, namely Denoised Diffusion Smoothing (DDS), are shown in Algorithm~\ref{alg:1} in the Appendix.

In the following, we will show $\tilde{w}$ is a faithful attention module. Before showing the results, we first provide some notations. For input image $x$, we denote $\tilde{w}_{i^*}$ as the $i$-th largest component in $\tilde{w}(x)$. Let $k_0=\lfloor (1-\beta)k \rfloor +1$ as the minimum number of changes on $\tilde{w}(x)$ to make it violet the  $\beta$-top-$k$ overlapping ratio with $\tilde{w}(x)$. Let $\mathcal{S}
$ denote the set of last $k_0$ components in top-$k$ indices and the top $k_0$ components out of top-$k$ indices.

\begin{theorem}\label{thm:5.1}
Consider the function $\tilde{w}$ where $\tilde{w}(x)=Z(T(x+z))$ with $Z$ in (\ref{eq:1}), $T$ as the denoised diffusion model and $z\sim \mathcal{N}(0, \sigma^2 I_{q\times n})$. Then, it is an
$(R, D_\alpha, \gamma,  \beta, k,  \|\cdot\|_2)$-{faithful attention module} for ViTs for any $\alpha> 1$ if for any input image $x$ we have
\begin{align*}\small
\sigma^2 & \geq \max\{{\alpha R^2} / 2(\frac{\alpha}{\alpha-1}\ln(2k_0(\sum_{i\in \mathcal{S}}\tilde{w}^\alpha_{i^*})^\frac{1}{\alpha} \\
    & + (2k_0)^\frac{1}{\alpha}\sum_{i\not\in \mathcal{S}}\tilde{w}_{i^*})  -\frac{1}{\alpha-1}\ln (2k_0)),  {\alpha R^2} /2\gamma\}.
\end{align*}
\end{theorem}
Theorem \ref{thm:5.1} indicates that $\tilde{w}(x)$ will be faithful attention for input $x$ when $\sigma^2$ is large enough. Equivalently, based on Theorem \ref{thm:5.1}  and \ref{thm:0}, we can also find a faithful region given $\beta$ and $k$. Note that in practice, it is hard to determine the specific $\alpha$ in R\'{e}nyi divergence. Thus, we can take the supreme w.r.t all $\alpha>1$ in finding the faithful region. See Algorithm \ref{alg:2} in Appendix for details.

We have shown that adding some Gaussian noise to the original data could get faithful attention through the original attention module. A natural question is whether adding Gaussian noise can be further improved by using other kinds of noise. Below, we show that \textbf{Gaussian noise is already near optimal} for certifying a faithful attention module via randomized smoothing.

\begin{theorem}\label{thm:l2lower}
Consider any function $\tilde{w}: \mathbb{R}^{q\times n}\mapsto \mathbb{R}^n$ where $\tilde{w}(x)=Z(T(x+z))$ with some random noise $z$, $T$ as the denoised diffusion model and $Z$ in (\ref{eq:1}). Then if  it is an $(R, D_\alpha, \gamma, \beta, k, \|\cdot\|_2)$-faithful attention module for ViTs with sufficiently large $\alpha$ and  $\mathbb{E}[\|z\|_{\max}]\leq \tau$ holds
for sufficiently small $\tau = O(1)$. Then it must be true that  
\begin{equation*}\small
\tau \geq \Omega(\frac{\sqrt{\alpha} R}{\sqrt{\gamma}}).    
\end{equation*}
Here for an matrix $z\in \mathbb{R}^{q\times n}$, $\|z\|_{\max}$ is defined as $\max_{i \in [q], j \in [n] } \| z_{i,j} \|$ is the maximal magnitude among all the entries in $z$. 
\end{theorem}
Note that in Theorem \ref{thm:5.1} we can see when $\gamma$ is small enough then  $\tilde{w}$ will be a $(R, D_\alpha, \gamma, \beta, k, \|\cdot\|_2)$-faithful attention module if $z\sim \mathcal{N}(0, \sigma^2I_{q\times n})$ with  $\sigma=\frac{\alpha R^2}{2\gamma}$. 
In this case we can see  $\mathbb{E}\|z\|_{\max}=O(\frac{\log {(q\cdot n)} \sqrt{\alpha} R}{\sqrt{\gamma}})$. Thus, the Gaussian noise is optimal up to some logarithmic factors.

\subsection{\texorpdfstring{$\ell_\infty$}{Lg}-norm Case}

\begin{figure*}[htbp]
    \setlength{\tabcolsep}{1pt} 
    \renewcommand{\arraystretch}{1} 
    \begin{center}
    \begin{tabular*}{\linewidth}
{@{\extracolsep{\fill}}cccccccc}
    Input & Raw Attention & Rollout & GradCAM & LRP & VTA  & Ours \\

    \raisebox{13mm}{\multirow{2}{*}{\makecell*[c]{Dog: clean$\rightarrow$\\\includegraphics[width=0.15\linewidth]{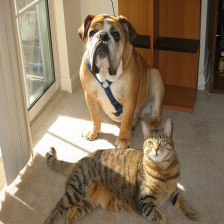}\\ 
    Dog: poisoned$\rightarrow$\\{\scriptsize $7/255$}
    }}
    }
       &
    \includegraphics[width=0.13\linewidth]{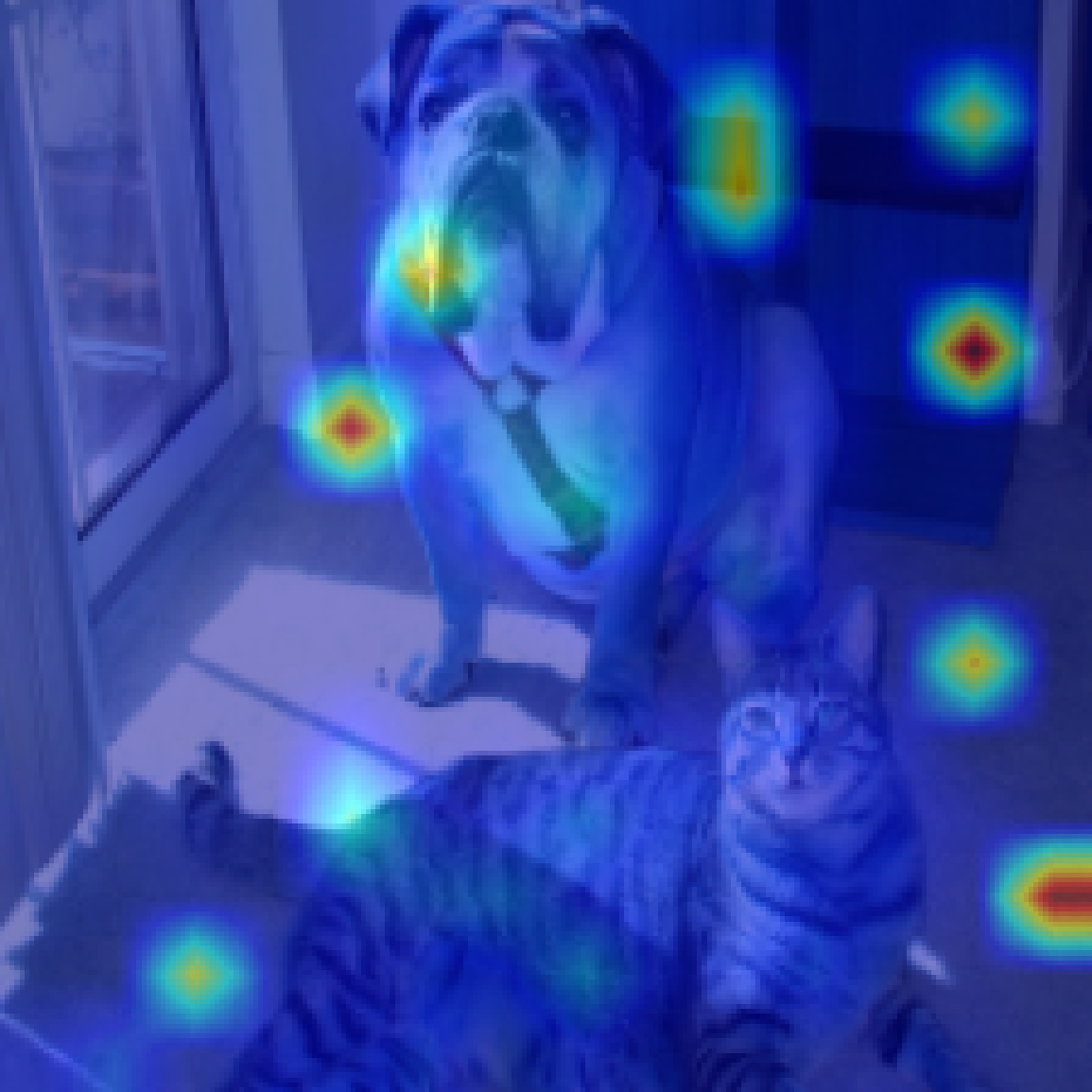} &
    \includegraphics[width=0.13\linewidth]{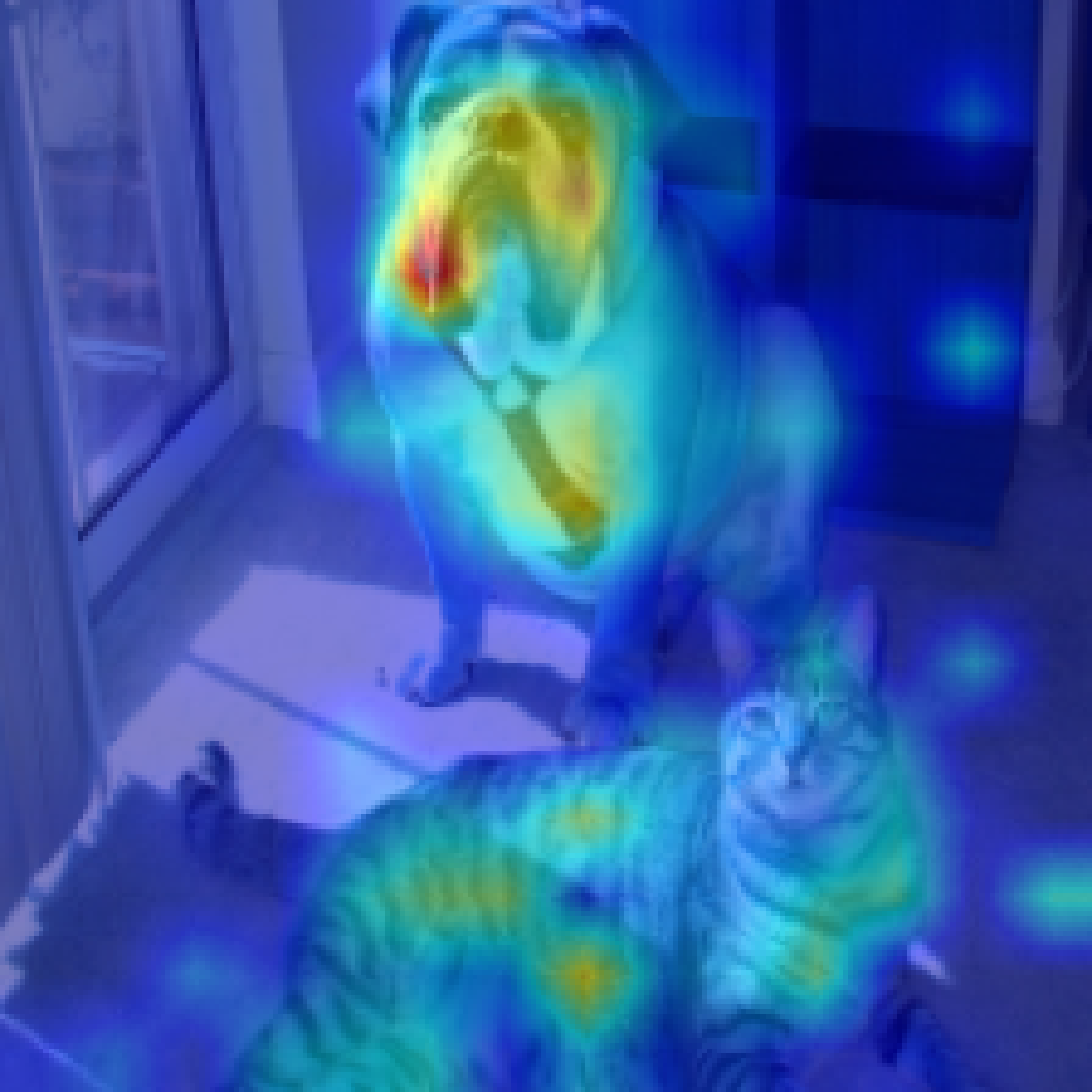} &
    \includegraphics[width=0.13\linewidth]{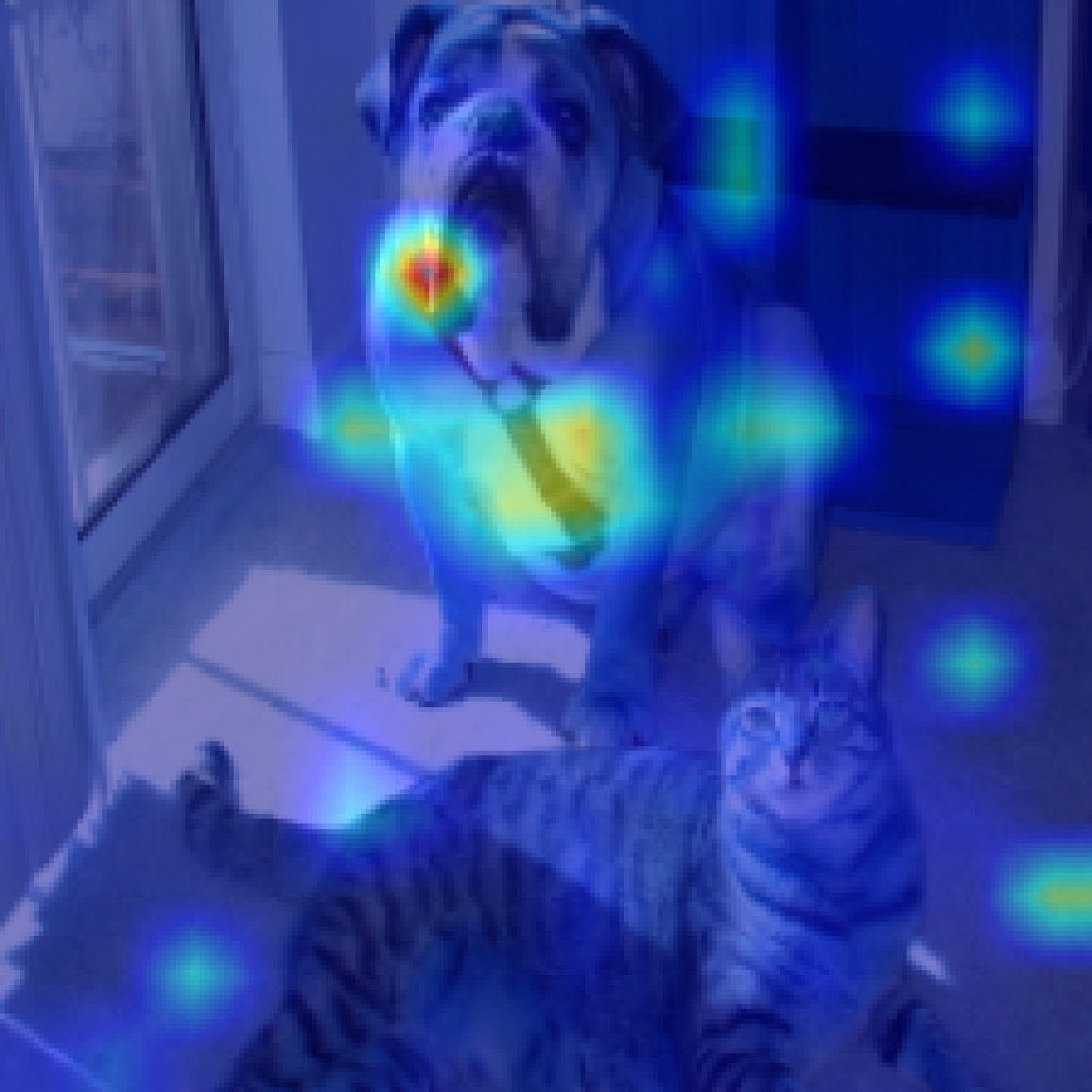} &
    \includegraphics[width=0.13\linewidth]{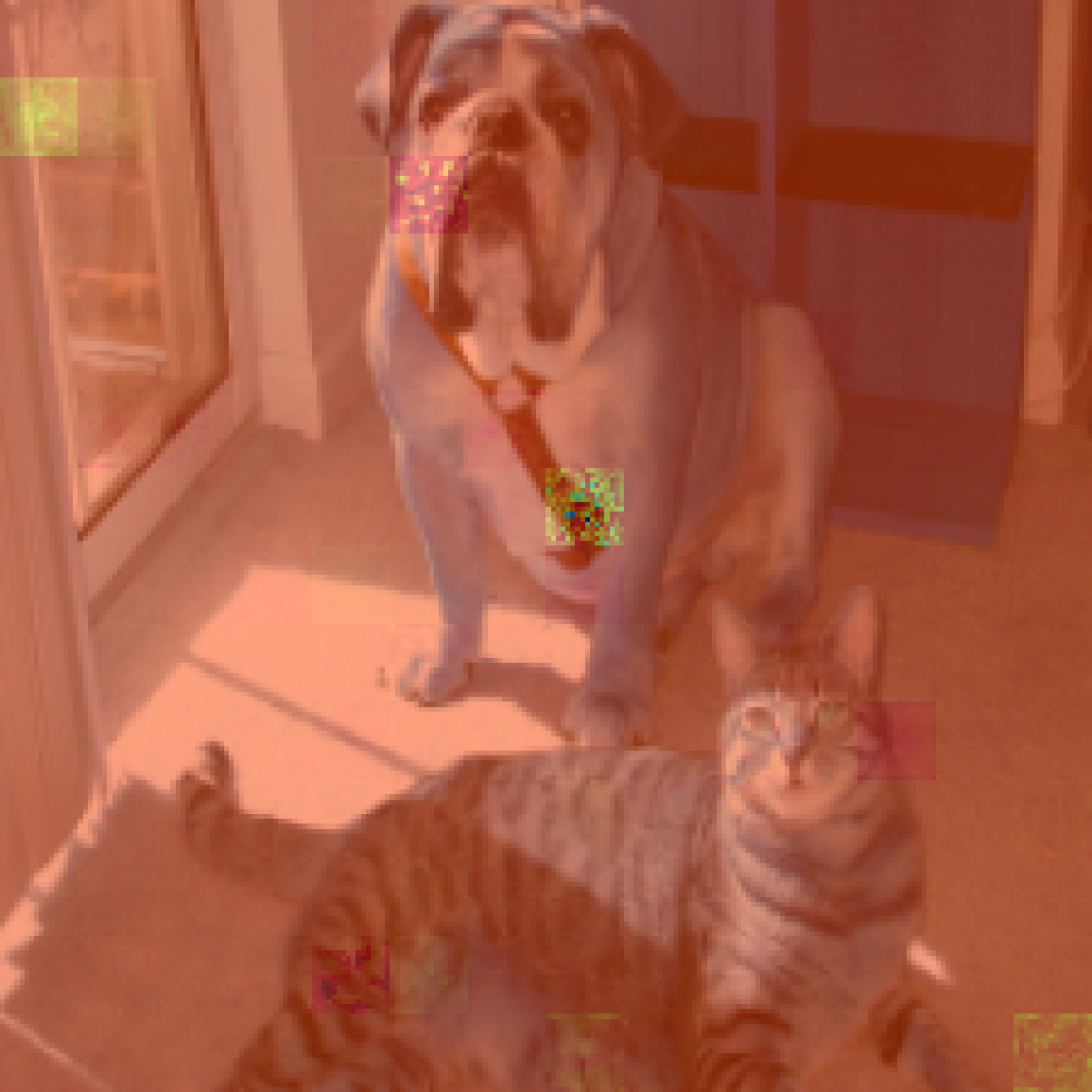} &
    \includegraphics[width=0.13\linewidth]{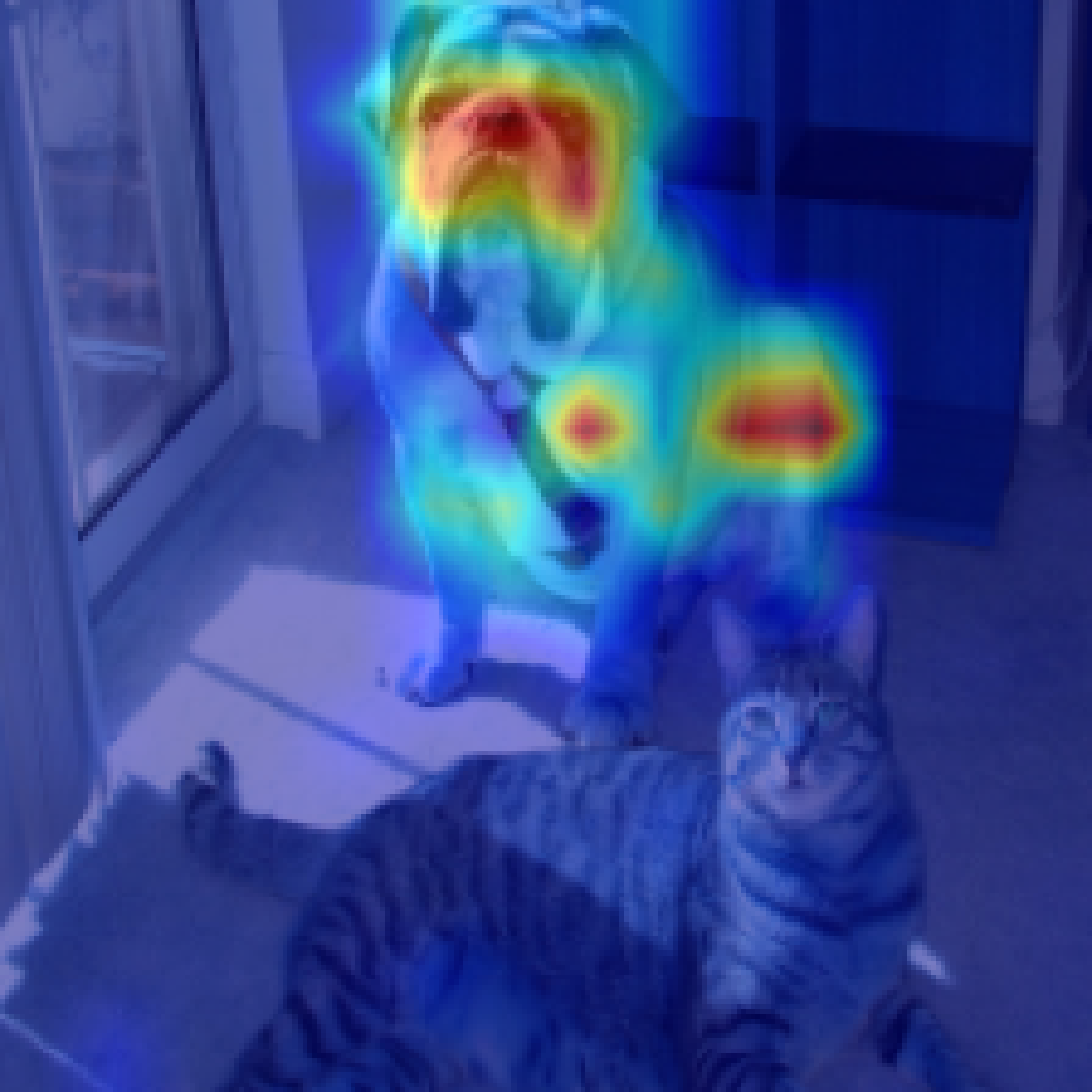} &
    \includegraphics[width=0.13\linewidth]{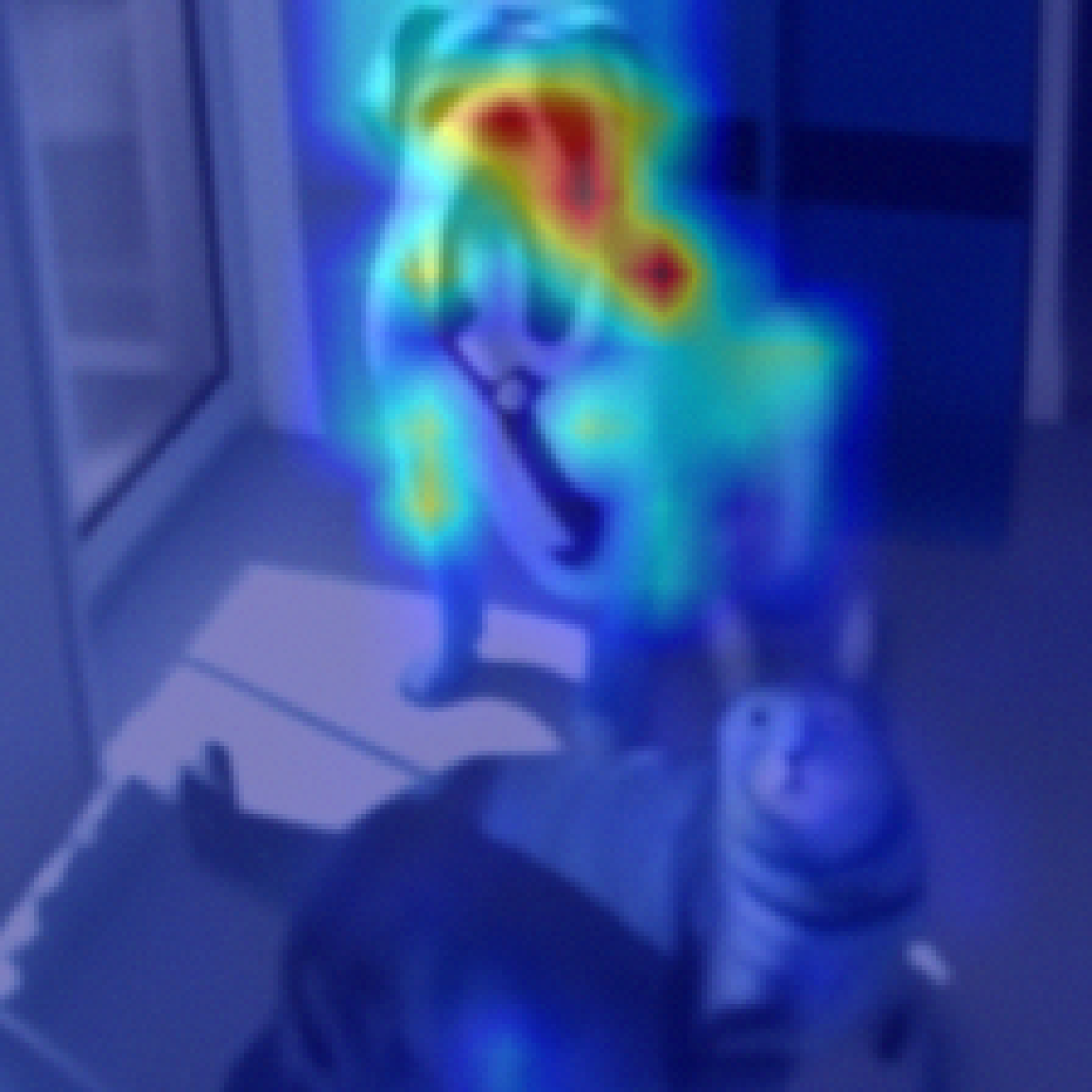}
    \\
    &\includegraphics[width=0.13\linewidth]{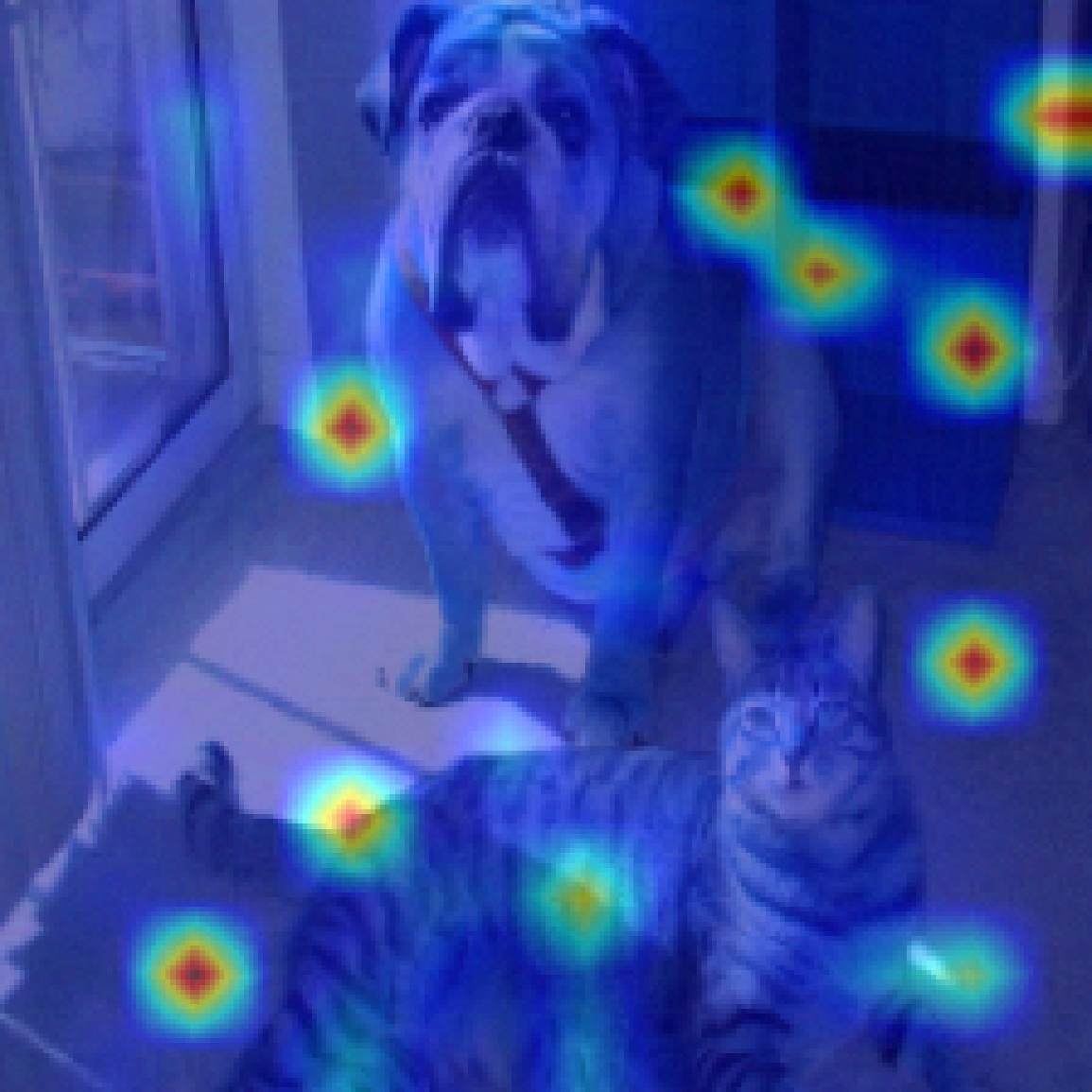} &
    \includegraphics[width=0.13\linewidth]{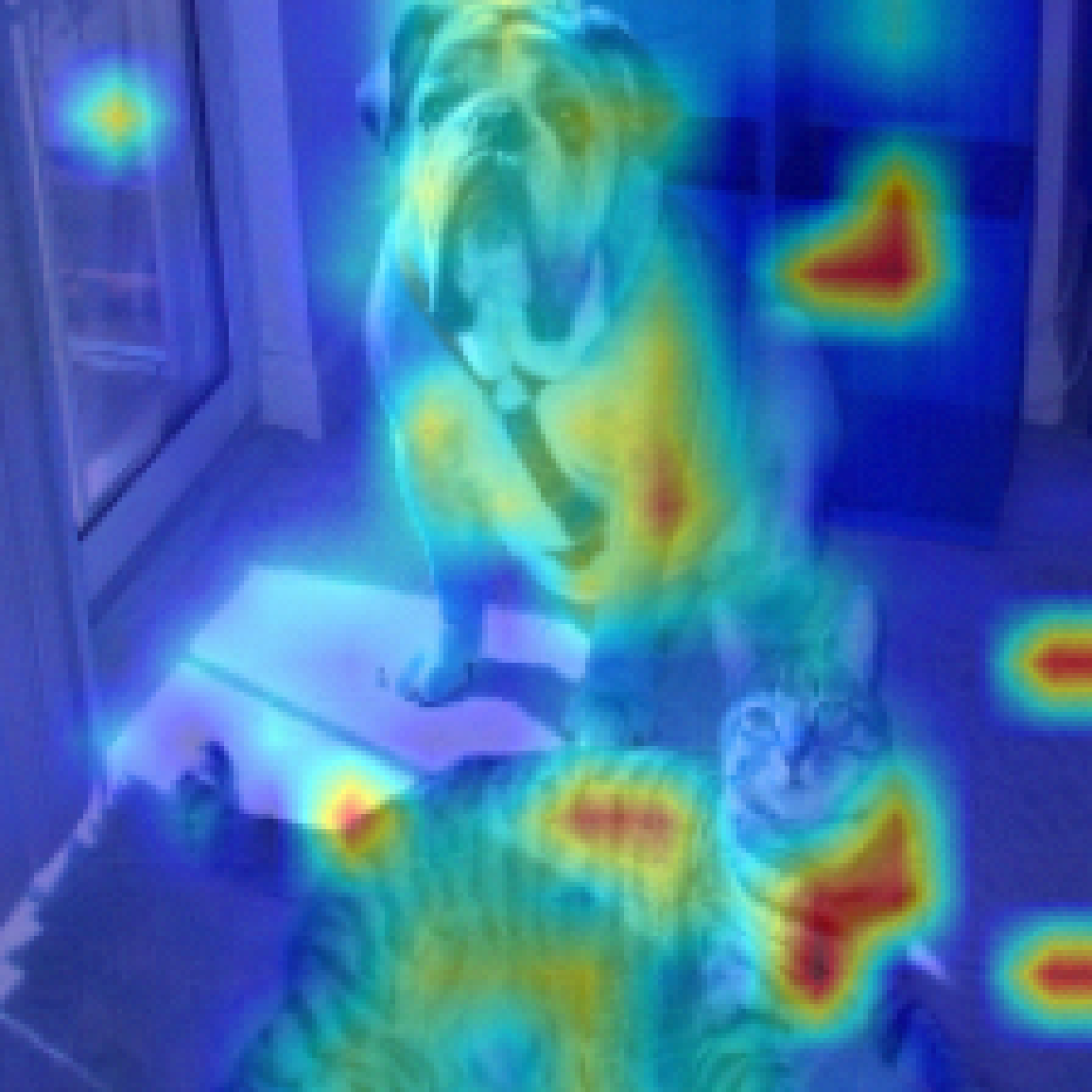} &
    \includegraphics[width=0.13\linewidth]{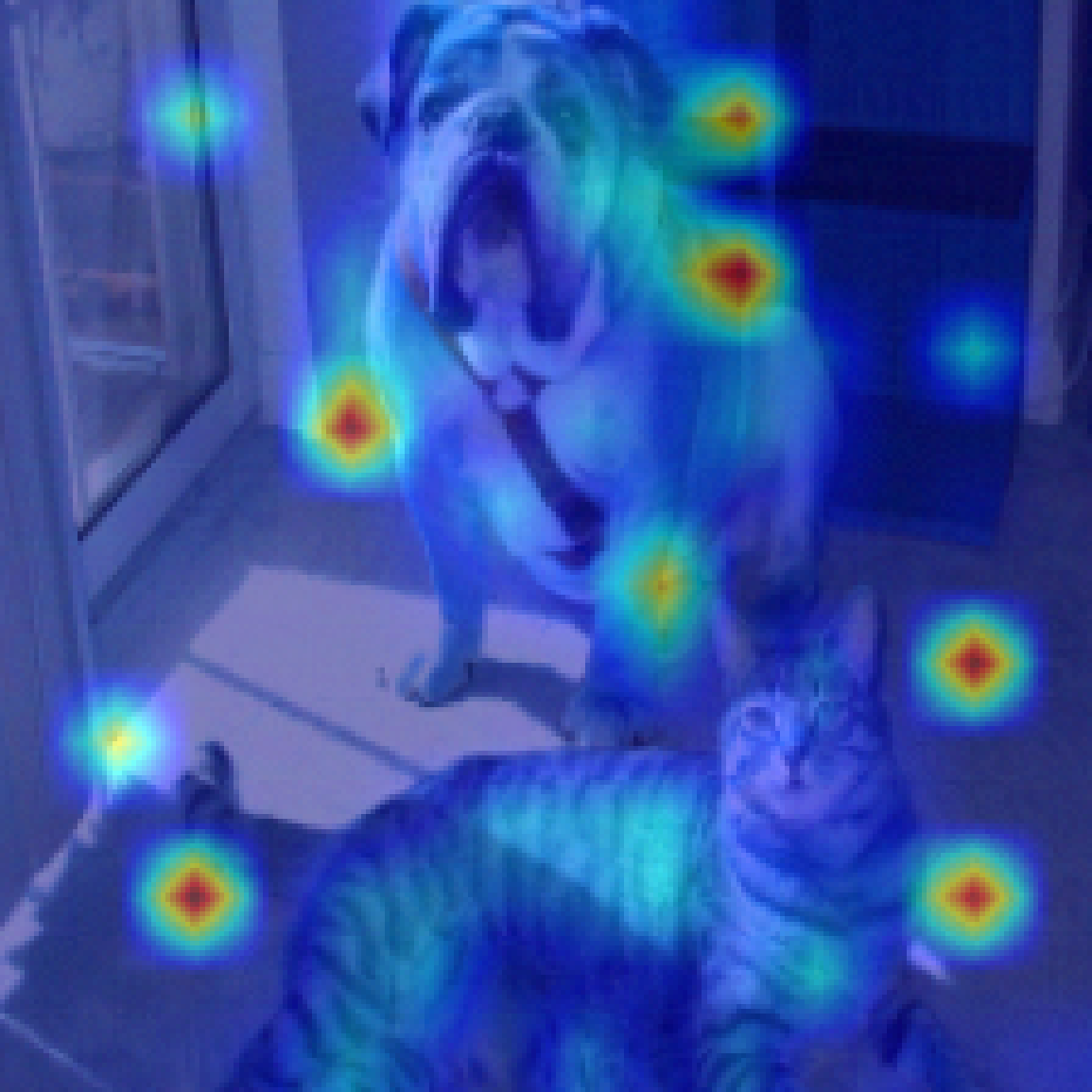} &
    \includegraphics[width=0.13\linewidth]{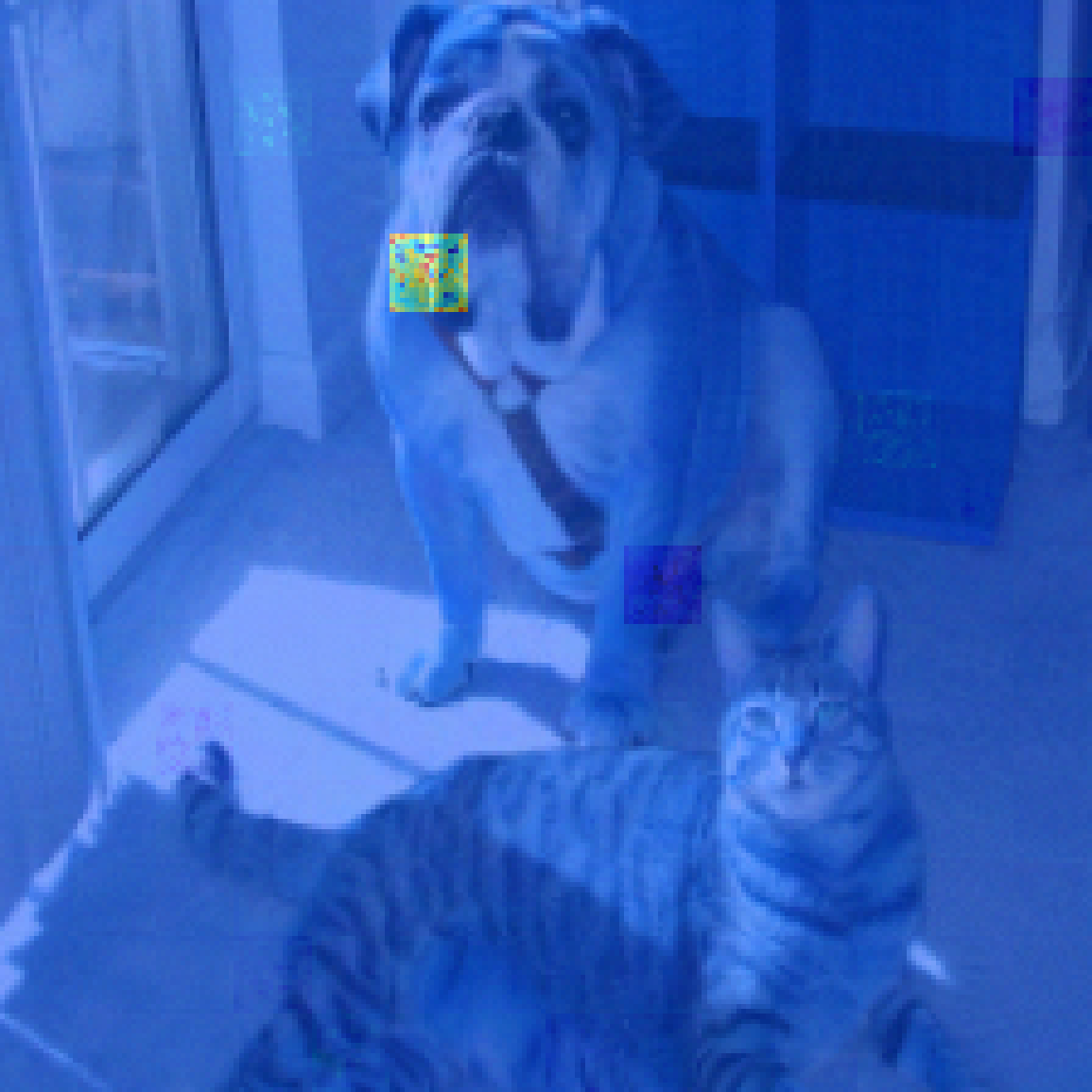} &
    \includegraphics[width=0.13\linewidth]{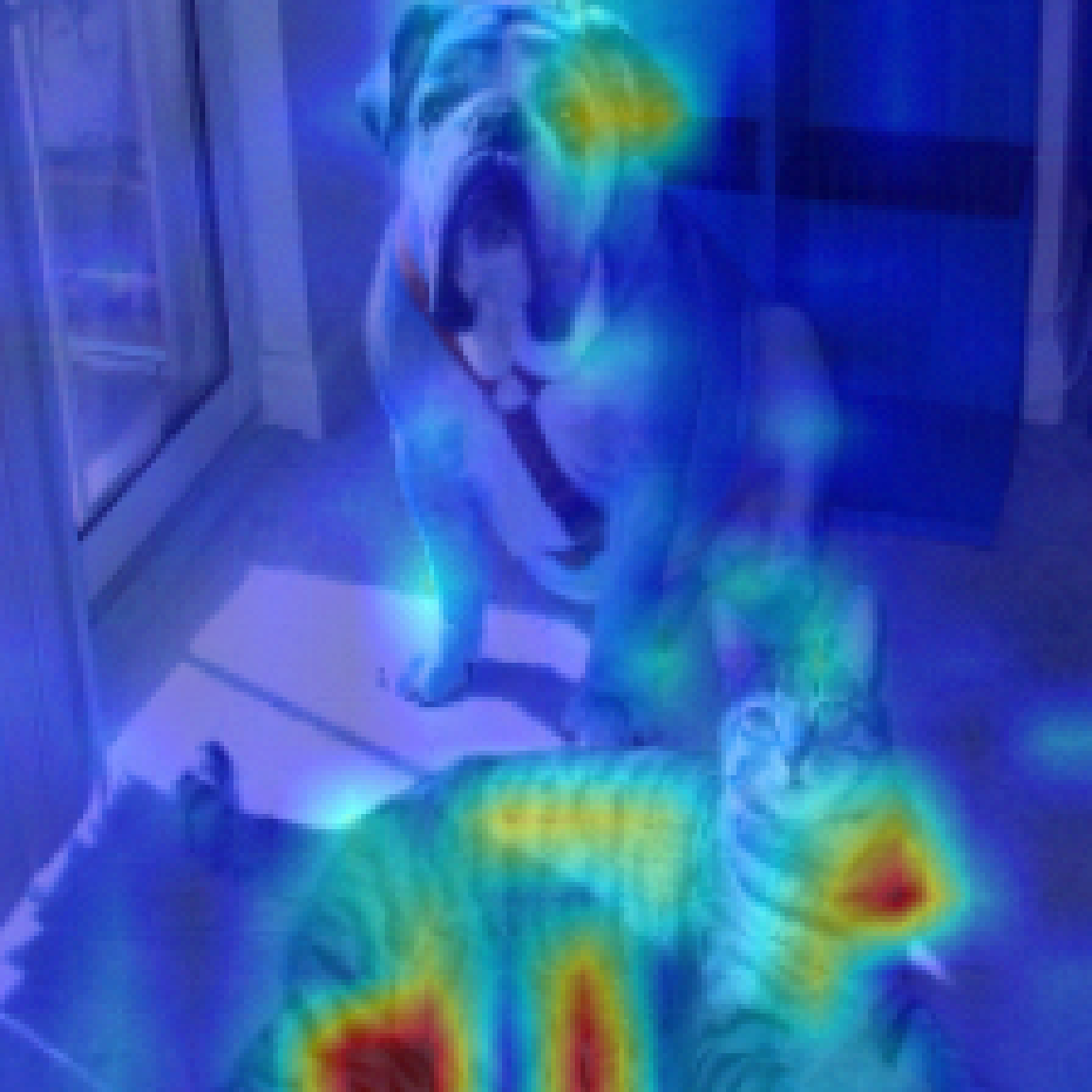} &
    \includegraphics[width=0.13\linewidth]{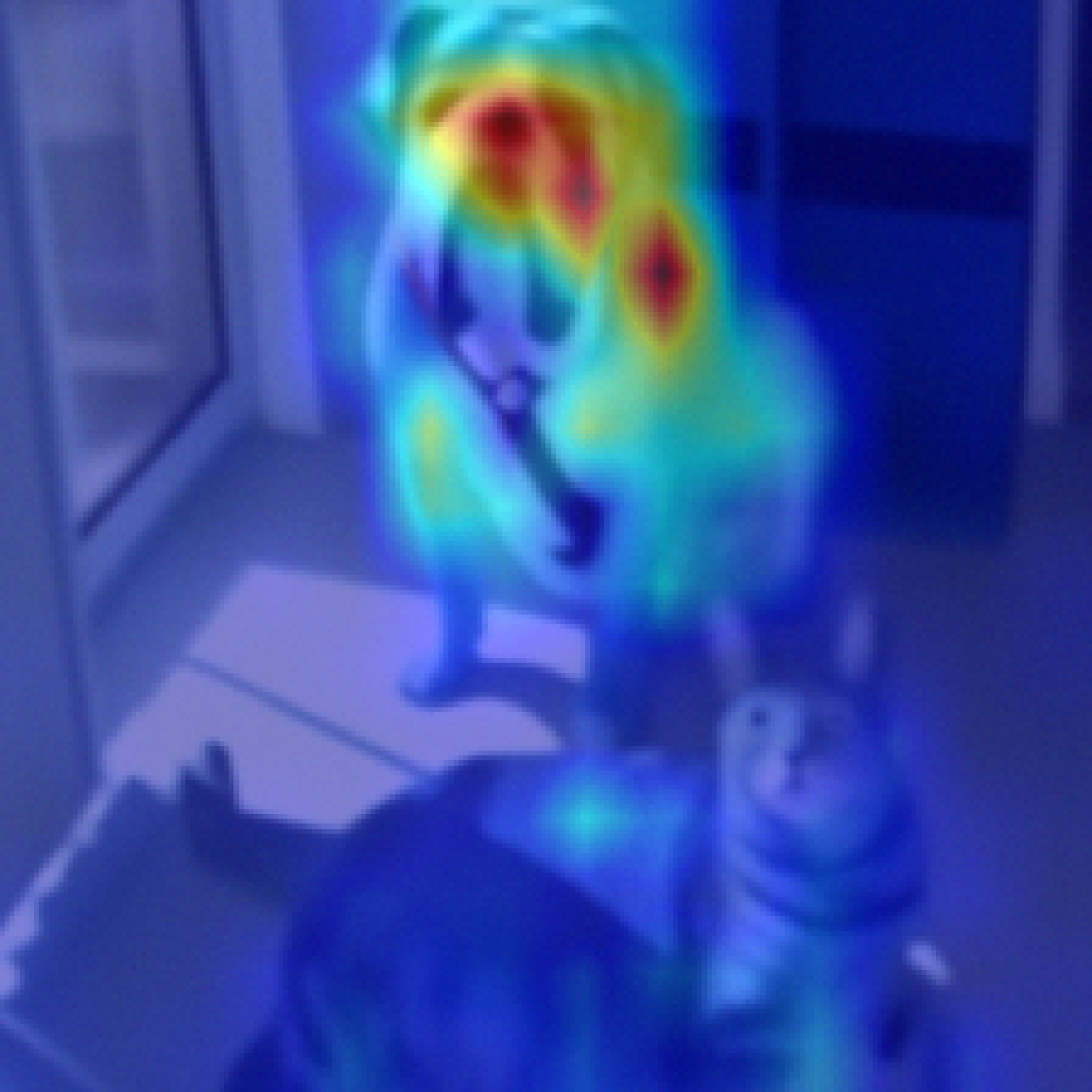}
    \end{tabular*}
    \caption{
    Class-specific explanation heat map visualizations under adversarial corruption. For each image, we present results for two different classes. Other baselines either give inconsistent interpretations or show wrong focus class regions under adversarial perturbations. While our method gives a consistent interpretation map and is robust against adversarial attacks.
    }
    \label{fig: example}
    \end{center}
\vspace{-0.2in}
\end{figure*}

\noindent In this section we consider the $\ell_\infty$-norm instead of the $\ell_2$-norm. Surprisingly, we show that using the same method as above, we can still get a faithful attention module. Moreover, the Gaussian noise is still near-optimal. 
\begin{theorem}\label{thm:5.2}
Consider the function $\tilde{w}$ where $\tilde{w}(x)=Z(T(x+z))$ with $Z$ in (\ref{eq:1}), $T$ as the denoised diffusion model and $z\sim \mathcal{N}(0, \sigma^2 I_{q\times n})$. Then it is  an
$(R, D_\alpha, \gamma,  \beta, k,  \|\cdot\|_\infty)$-{faithful attention module} for ViTs for $\alpha> 1$ if for any input image $x$ we have the following, where $d=q\cdot n$. 
\begin{align*}\small
\sigma^2 & \geq \max\{{d\alpha R^2} / 2(\frac{\alpha}{\alpha-1}\ln(2k_0(\sum_{i\in \mathcal{S}}\tilde{w}^\alpha_{i^*})^\frac{1}{\alpha} \\
& +(2k_0)^\frac{1}{\alpha}\sum_{i\not\in \mathcal{S}}\tilde{w}_{i^*})-\frac{1}{\alpha-1}\ln (2k_0)),  {d\alpha R^2} /2\gamma\}. 
\end{align*}
\end{theorem}
Compared with the result in Theorem \ref{thm:5.1} for the $\ell_2$-norm case, we can see there is the additional factor of $d$ in the bound of the noise. This means if we aim to achieve the same faithful level as in the $\ell_2$-norm case, then in the $\ell_\infty$-norm case, we need to enlarge the noise by a factor of $d$. Equivalently, if we add the same scale of noise, then the faithful region for $\ell_\infty$-norm will be shrunk by a factor of $d$ of the region for $\ell_2$-norm. See Algorithm \ref{alg:2} for details. 

\begin{theorem}\label{thm:linflower}
Consider any function $\tilde{w}: \mathbb{R}^{q\times n}\mapsto \mathbb{R}^n$ where $\tilde{w}(x)=Z(T(x+z))$ with some random noise $z$ and $Z$ in (\ref{eq:1}). Then if  it is an $(R, D_\alpha, \gamma, \beta, k, \|\cdot\|_\infty)$-faithful attention module for ViTs with sufficiently large $\alpha$ and  $\mathbb{E}[\|z\|_{\max}]\leq \tau$ holds
for sufficiently small $\tau \leq O(1)$. Then it must be true that
\begin{equation*}\small
\tau \geq \Omega(\frac{\sqrt{\alpha} R}{\sqrt{d\gamma}}).    
\end{equation*}
\end{theorem}
Theorem \ref{thm:5.1} we can see when $\gamma$ is small enough then  $\tilde{w}$ will an  $(R, D_\alpha, \gamma, \beta, k, \|\cdot\|_\infty)$-faithful attention module if $z\sim \mathcal{N}(0, \sigma^2I_{q\times n})$ with  $\sigma=\frac{d\alpha R^2}{2\gamma}$. 
In this case we can see  $\mathbb{E}\|z\|_{\max}=O(\frac{\log {(q\cdot n)} \sqrt{d\alpha} R}{\sqrt{\gamma}})$. Thus, the Gaussian noise is optimal up to some logarithmic factors. 

\begin{table*}[htbp]
\centering
\resizebox{0.75\linewidth}{!}{
\begin{tabular}{cccccccccc} 
\toprule
\multirow{2}{*}{Model} & \multirow{2}{*}{Method} & \multicolumn{4}{c}{ImageNet}                                            & \multicolumn{2}{c}{Cityscape}                      & \multicolumn{2}{c}{COCO}                            \\ 
\cmidrule(l){3-10}
                       &                         & \textbf{Cla. Acc.} & \textbf{Pix. Acc.} & \textbf{mIoU} & \textbf{mAP}  & \textbf{Pix. Acc.} & \textbf{mIoU} & \textbf{Pix. Acc.} & \textbf{mIoU}  \\ 
\midrule
\multirow{6}{*}{ViT}   & Raw Attention           & 0.78               & 0.65               & 0.54          & 0.82          & 0.72               & 0.62                    & 0.8                & 0.7                      \\
                       & Rollout                 & 0.79               & 0.67               & 0.56          & 0.84          & 0.74         & 0.64                     & 0.82               & 0.72                     \\
                       & GradCAM                 & 0.8                & 0.69               & 0.58          & 0.86          & 0.76               & 0.66                  & 0.84               & 0.74                    \\
                       & LRP                     & 0.81               & 0.71               & 0.6           & 0.88          & 0.78               & 0.68                & 0.86               & 0.76                 \\
                       & VTA                     & 0.82               & 0.73               & 0.62          & 0.9           & 0.8                & 0.7            & 0.88               & 0.78                   \\
\rowcolor{grey!20}
                       & Ours                    & \textbf{0.85}      & \textbf{0.76}      & \textbf{0.65} & \textbf{0.93} & \textbf{0.83}      & \textbf{0.73} &  \textbf{0.91}      & \textbf{0.81}  \\ 
\midrule
\multirow{6}{*}{DeiT}  & Raw Attention           & 0.79               & 0.66               & 0.55          & 0.83          & 0.73               & 0.63             & 0.81               & 0.71             \\
                       & Rollout                 & 0.8                & 0.68               & 0.57          & 0.85          & 0.75               & 0.65            & 0.83               & 0.73               \\
                       & GradCAM                 & 0.81               & 0.7                & 0.59          & 0.87          & 0.77               & 0.67           & 0.85               & 0.75                 \\
                       & LRP                     & 0.82               & 0.72               & 0.61          & 0.89          & 0.79               & 0.69             & 0.87               & 0.77         \\
                       & VTA                     & 0.83               & 0.74               & 0.63          & 0.91          & 0.81               & 0.71          & 0.89               & 0.79           \\
\rowcolor{grey!20}
                       & Ours                    & \textbf{0.86}      & \textbf{0.77}      & \textbf{0.66} & \textbf{0.94} & \textbf{0.84}      & \textbf{0.74} &  \textbf{0.89}      & \textbf{0.79}   \\ 
\midrule
\multirow{6}{*}{Swin}  & Raw Attention           & 0.8                & 0.67               & 0.56          & 0.84          & 0.74               & 0.64                 & 0.82               & 0.72                \\
                       & Rollout                 & 0.81               & 0.69               & 0.58          & 0.86          & 0.76               & 0.66               & 0.84               & 0.74               \\
                       & GradCAM                 & 0.82               & 0.71               & 0.6           & 0.88          & 0.78               & 0.68             & 0.86               & 0.76               \\
                       & LRP                     & 0.83               & 0.73               & 0.62          & 0.9           & 0.8                & 0.7            & 0.88               & 0.78               \\
                       & VTA                     & 0.84               & 0.75               & 0.64          & 0.92          & 0.82               & 0.72             & 0.9                & 0.8                \\
\rowcolor{grey!20}
& Ours                    & \textbf{0.87}      & \textbf{0.78}      & \textbf{0.67} & \textbf{0.95} & \textbf{0.85}      & \textbf{0.75} & \textbf{0.93}      & \textbf{0.83} \\
\bottomrule
\end{tabular}
}
\caption{ Performance comparison of different methods on ImageNet, Cityscape, and COCO under the default attack. }
\label{tab: main}
\vspace{-10pt}
\end{table*}

\vspace{-7pt}
\section{Experiments}
In this section, we present experimental results on evaluating the interpretability and utility of our FViTs on various datasets and tasks. More details are in the Appendix.

\vspace{-10pt}
\subsection{Experimental Setup}
\noindent {\bf Datasets, tasks, and network architectures.}
We consider two different tasks: classification and segmentation. For the classification task, we use ILSVRC-2012 ImageNet. And for segmentation, we use ImageNet-segmentation subset \citep{guillaumin2014imagenet}, COCO \citep{lin2014microsoft}, and Cityscape \citep{cordts2016cityscapes}. To demonstrate our method is architectures-agnostic, we use three different ViT-based models, including Vanilla ViT \citep{dosovitskiy2021an}, DeiT \cite{touvron2021training}, and Swin ViT \citep{liu2021swin}.

\noindent {\bf Threat model.}
We focus on $l_2$-norm bounded and $l_\infty$-norm bounded noises under a white-box threat model assumption for adversarial perturbations. Mathematically, with the same noise level, $l_\infty$-norm ball $B_\infty$ is a superset of the $l_2$-norm ball $B_2$. Thus, we show the performance under $l_\infty$-norm threat model, and we report the $l_2$-norm case in the appendix \ref{l2:case}. The radius of adversarial noise $\rho _u$ was set as $8/255$ by default. We employ the PGD \citep{madry2017towards} algorithm to craft adversarial examples with a step size of $2/255$ and a total of $10$ steps.

\noindent {\bf Baselines and attention map backbone.}
Since our DDS method can be used as a plugin to provide certified faithfulness for interpretability under adversarial attacks, regardless of the method used to generate attention maps. We set the standard deviation $\delta=8/255$ for the Gaussian noise in our method as default. In this paper, we leverage Trans. Att. \citep{chefer2021transformer} as our explanation tool, which is a state-of-the-art method for generating class-aware interpretable attention maps. 
We include five baselines for comparison, including Raw Attention \citep{vaswani2017attention}, Rollout \citep{abnar2020quantifying}, GradCAM \citep{selvaraju2017grad}, LRP \citep{binder2016layer}, and Vanilla Trans. Att. (VTA) \citep{chefer2021transformer}. 

\noindent {\bf Evaluation metrics.}
To show the utility of our approach, we report the classification accuracy on test data for classification tasks. As for the interpretability of our approach, we seek to evaluate the explanation map by leveraging the label of segmentation task as `ground truth' following \citep{chefer2021transformer}. To be specific, we compare the explanation map with the ground truth segmentation map. We measure the interpretability using pixel accuracy, mean intersection over union (mIoU) \citep{varghese2020unsupervised}, and mean average precision (mAP) \citep{henderson2017end}. Note that pixel accuracy is calculated by thresholding the visualization by the mean value, while mAP uses the soft-segmentation to generate a score that is not affected by the threshold. Following conventional practices, we also report the results for negative and positive perturbation (pixels-erasing) tests \citep{chefer2021transformer}. The area-under-the-curve (AUC) measured by erasing between $10\%-90\%$ of the pixels is used to indicate the performance of explanation methods for both perturbations. For negative perturbation, a higher AUC indicates a more faithful interpretation since a good explanation should maintain accuracy after removing unrelated pixels (also referred to as input invariance \citep{kindermans2019reliability}). On the other hand, for positive perturbations, we expect to see a steep decrease in performance after removing the pixels that are identified as important, where a lower AUC indicates the interpretation is better. We term the AUC of such perturbation tests as P-AUC and plot the P-AUC-radius curve under adversarial perturbations. 

\begin{figure*}[htbp]
\centering
	\subfigure[Pos. Perturbation] 
	{
            \label{fig:per-pos}
		\includegraphics[scale=0.33]{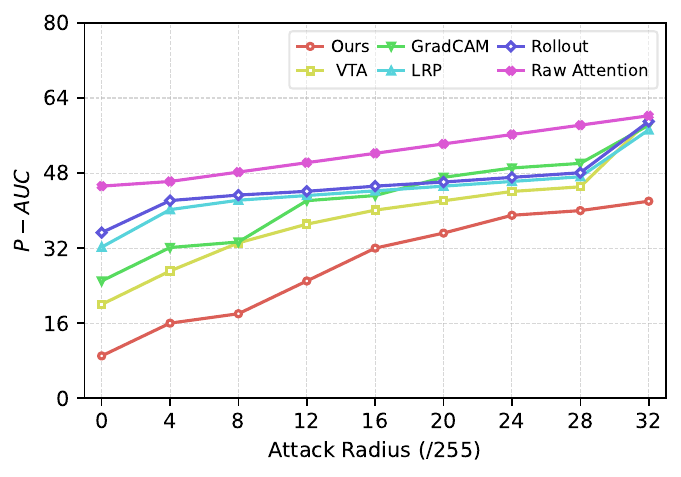} 
	}
	\subfigure[Neg. Perturbation] 
	{
            \label{fig:per-neg}
		\includegraphics[scale=0.33]{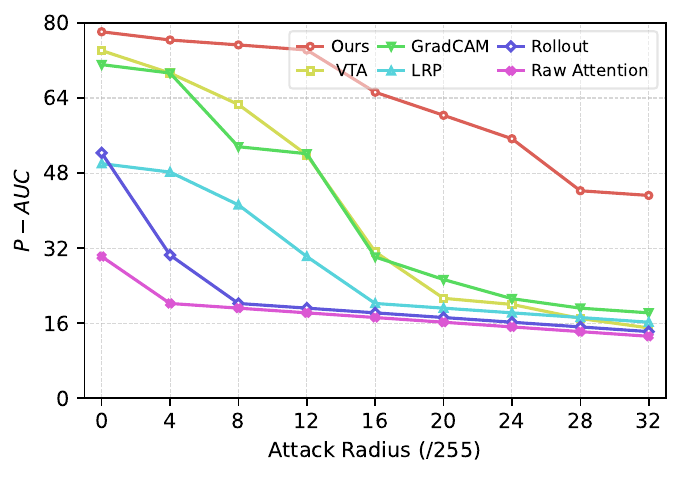}
	}
     \subfigure[Sens. Diff. Radius] 
	{
		\label{fig:ver-r}
            \includegraphics[scale=0.33]{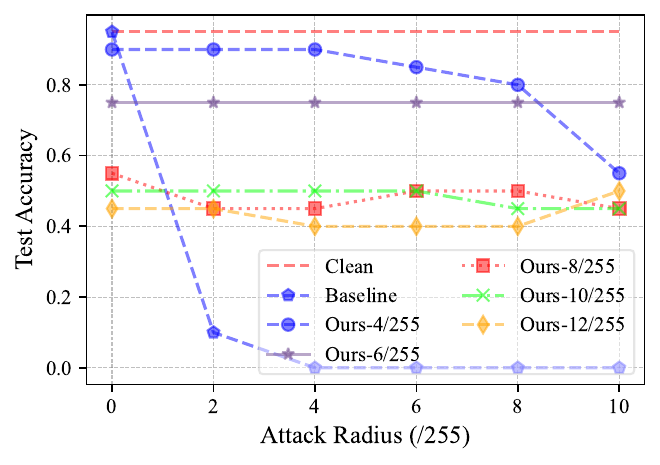}
	}
    \caption{(a) and (b) are results of the perturbation test. (c) is the sensitivity analysis results.
    }
    \label{fig: per}
\vspace{-0.2in}
\end{figure*}

\vspace{-10pt}
\subsection{Evaluating Interpretability and Utility}
\vspace{-5pt}
\noindent {\bf Classification and segmentation results.}
Based on the results shown in Table \ref{tab: main}, it is obvious that our method is more robust and effective for all three datasets. Moreover, across all metrics,  %
our method consistently outperforms other methods for all model architectures. For example, on ImageNet, our method achieves the highest classification accuracy (0.85) and pixel accuracy (0.76) for the ViT model and the highest mean IoU (0.66) and mean AP (0.94) for the DeiT model. These results suggest that our method could even outperform the previous methods on accuracy, and it is more faithful in identifying the most relevant features for image classification under malicious attacks compared to other methods. 
Please refer to Table \ref{tab: cleanu} and \ref{tab: full-result} in the Appendix for complete results across different levels of perturbation. 


Additionally, Figure \ref{fig: poison},\ref{fig: diff-posin2},\ref{fig: diff-posin3} in Appendix demonstrate visual comparisons of our method with baselines under adversarial attacks. It is clear that the baseline methods produce inconsistent results, while our method produces more consistent and clear visualizations even under data corruption. Moreover, as shown in Figure \ref{fig: example} (more results are in Appendix Figure \ref{fig: full_example},\ref{fig: class2},\ref{fig: class3}), when analyzing images with two objects from different classes under adversarial perturbations, all previous methods produce similar but worse visualizations for each class. Surprisingly, our method is able to provide accurate and distinct visualization maps for each class despite adversarial perturbations. This indicates that our method is more faithful which is robust class-aware under attacks. 

\noindent {\bf Perturbation tests.}
The results on Pos. perturbation in Figure \ref{fig:per-pos} show that the P-AUC of our method consistently achieves the lowest value when we perform attacks with a radius ranging from $0/255$ to $32/255$, which suggests that our method are more faithful and interpretable. Similarly, as for Neg. perturbation, the results in Figure \ref{fig:per-neg} also suggest that our method is more robust than other baselines when removing unrelated pixels, and indicate that our method can identify important pixels under corruptions.

\begin{table}[htbp]
\centering
\resizebox{\linewidth}{!}{
\begin{tabular}{lccccc} 
\toprule
\multirow{2}{*}{} & Classification & \multicolumn{2}{c}{Segementation} & \multicolumn{2}{c}{Perturbation Tests}    \\ 
\cmidrule(lr){2-2}\cmidrule(l){3-4}\cmidrule(l){5-6}
                  & \textbf{Rob.Acc} & \textbf{Rob.Acc}  & \textbf{mIOU}            & \textbf{Pos.} & \textbf{Neg.}  \\ 
\midrule
\textbf{Ours}              & \textbf{99.5 } & \textbf{ 97.8 } & \textbf{0.985}  & \textbf{ 15.29}     & \textbf{63.23}       \\
\textbf{$-$smoothing}      & 98.2           & 96.3            & 0.977           & 18.51               & 54.65                \\
\textbf{$-$denosing}       & 96.4           & 94.7            & 0.965           & 21.36               & 50.53                \\
\textbf{$-$both}           & 92.1           & 90.5            & 0.947           & 38.13               & 48.58                \\
\bottomrule
\end{tabular}
}
\caption{Results of the ablation study on classification, segmentation, and perturbation tests.}
\label{tab: abl}
\vspace{-12pt}
\end{table}

\noindent {\bf Ablation study.}
The results are shown in Table \ref{tab: abl}, highlighting the crucial role that the denoising diffusion model and randomized smoothing play in the effectiveness of DDS. As we can see from the table, removing either of the components leads to a significant decrease in performance (under adversarial attacks) across all three evaluation metrics: classification accuracy, segmentation accuracy, and P-AUC. In particular, the classification and segmentation accuracy will decrease by 3.1\% when the denoising step is removed, and by 1.3\% and 1.5\%, respectively when the randomized smoothing is eliminated. Moreover, we visualized the ablated version of our method in Figure \ref{fig: abl-vis}. 
It is noteworthy that the performance degradation becomes more pronounced when both components are removed, compared to when only a single component is removed. This suggests that these two components are highly effective in improving the faithfulness of model prediction and explanation. 

\noindent {\bf Sensitivity analysis.}
To evaluate the sensitivity of standard deviation $\delta$ of the added Gaussian noise, we conduct adversarial attacks on the ImageNet dataset with different $\delta$ for a certain number of data samples. 
We conduct testing under $\delta \in \left\{ 4/255,6/255,8/255,10/255,12/255 \right\} $ and attack radius $\rho_a\in \left\{ 0,2/255,4/255,6/255,8/255,10/255 \right\} $. The results in Figure \ref{fig:ver-r} suggest that, for the cases of $\delta=4/255$ and $\delta=6/255$, compared to the vanilla baseline, i.e., without any processing of images, our method is able to prevent the testing accuracy from dropping significantly as the attack radius increase. However, we find that larger $\delta$ does not significantly decrease test accuracy when $\delta$ exceeds some threshold ($\delta=8/255-12/255$). These results suggest that our method is sensitive to the selection of $\delta$ when $\delta$ is small, and it becomes insensitive when  $\delta$ is larger. Nevertheless, across different $\delta$, our method outperforms the baseline in terms of utility.

\noindent {\bf Verifying faithful region.}
To verify the proposed faithful region estimation in Algorithm \ref{alg:2}, we conduct an adversarial attack using projected gradient descent on our denoised smoothing classifier following \cite{cohen2019certified}. Given the faithful region radius $R(\delta)=\min\{P(\delta), Q(\delta)\}$ obtained in Algorithm \ref{alg:2}, 
we attempt to find an adversarial example for our denoised smoothing classifier within radii of $1.5R$ or $2R$, under the condition that the example has been correctly classified within faithful region $R$. We succeed in finding such adversarial examples 23\% of the time at a radius $1.5R$ and 64\% of the time in a radius $2R$ on the ImageNet. These results empirically demonstrate the tightness of our proposed faithful bound.

\vspace{-10pt}
\subsection{Computational Cost}
\vspace{-5pt}
Our denoising algorithm is quite fast. For example, under the noise level of $\frac{8}{255}$, in each denoising trail, it only requires one forward step in adding a random Gaussian noise and $t^{*} = 45$ backward steps for denoising, which empirically takes about 0.32 seconds per images (256x256) on ImageNet \ref{tab:time}. This shows that our methods are efficient and promising for real-world applications with large-scale data. 

\begin{table}[thbp]
\centering
\resizebox{0.9\linewidth}{!}{
\begin{tabular}{cccccc} 
\toprule
$||\delta||_{\infty} $ & 2/255 & 4/255 & 8/255 & 12/255 & 16/255    \\
\midrule 
$t^{*}$ & 0     & 8     & 45    & 107    & 193    \\
Total Time(s)   & 0.60  & 3.18  & 3.20  & 3.25   & 3.33    \\
Per Sample(s) & 0.060 & 0.318 & 0.320 & 0.325  & 0.333 \\
\bottomrule
\end{tabular}
}
\caption{The time cost of denoising under different noise levels with a total sample size of 10.}
\label{tab:time}
\vspace{-15pt}
\end{table}

\section{Conclusion}
\vspace{-5pt}
We proposed FViTs to improve faithfulness in vanilla ViTs. We first gave a rigorous definition for FViTs and then proposed a method with theoretical proof to achieve robustness for both explainability and prediction, and finally, we conducted comprehensive experiments to prove our claim. 

\section{Impact Statements}
Our research enhances both the interpretation faithfulness and prediction robustness of vision transformers. Given that vision transformers constitute a major component of recent Large Vision-Language Models (LVLMs), our method holds general applicability, potentially fostering alignment and safety within these models. Employing our denoising and smoothing techniques can bolster decision-making robustness in LVLMs and enhance their resilience against malicious manipulation, thereby contributing to the trustworthiness and superalignment of emergent superintelligences. We believe this work does not present significant ethical concerns.

\section*{Acknowledgements}
Di Wang and Lijie Hu are supported in part by the baseline funding BAS/1/1689-01-01, funding from the CRG grand URF/1/4663-01-01, FCC/1/1976-49-01 from CBRC, and funding from the AI Initiative REI/1/4811-10-01 of King Abdullah University of Science and Technology (KAUST). Di Wang and Lijie Hu are also supported by the funding of the SDAIA-KAUST Center of Excellence in Data Science and Artificial Intelligence (SDAIA-KAUST AI). Yixin Liu and Lichao Sun are supported by the National Science Foundation Grants CRII-2246067 and partially supported by Lehigh Grant FRGS00011497.

\bibliography{reference}
\bibliographystyle{icml2024}

\newpage
\appendix
\onecolumn
\section{Algorithms}

\begin{algorithm}
    \caption{FViTs via Denoised Diffusion Smoothing}
    \label{alg:1}
    \begin{algorithmic}[1]
        \STATE {\bfseries Input:} $x$; A standard deviation $\sigma > 0$. 
        \STATE $t^{*}$, find $t$ s.t. $\frac{1-\alpha_t}{\alpha_t} = \sigma^2$. 
        \STATE $x_{t^{*}} = \sqrt{\alpha_{t^{*}}} (\tilde{x} + \mathcal{N}(0, \sigma^2 \textbf{I})) $.
        \STATE $\hat{x} = \text{denoise}(x_{t^{*}}; t^{*})$.
        \STATE $w = \text{self-attention}(\hat{x})$.
        \STATE {\bfseries Return:} attention weight $w$.
    \end{algorithmic}
\end{algorithm}

\begin{algorithm}
    \caption{Finding the Faithfulness Region in FViTs}
    \label{alg:2}
    \begin{algorithmic}[1]
        \STATE {\bfseries Input:} Original self-attention module  $Z$; the standard deviation $\sigma > 0$; classifier of the original ViT $\bar{y}$; Number of repetitions  $m$. Input image $x$. 
            \FOR {$i \in [m]$}
            \STATE Sample a Gaussian noise $z_i\sim \mathcal{N}({0}, \sigma^2I_{q\times n})$ and add it to the input image $x$. Then get an attention vector $\tilde{w}_i=Z(T(x+z_i))$ via Algorithm \ref{alg:1} and feed it to the original ViT and get the prediction $c_i=\arg\max_{g\in \mathcal{G}}\bar{y}(x+z_i).$
        \ENDFOR
        \STATE Estimate the distribution of the output as $p_j = \frac{\# \{  c_i=j;i=1,...,m\}}{m}$. Compute the average of $\tilde{w}_i$: $\tilde{w}=\frac{1}{m}\sum_{i=1}^m \tilde{w}_m$. 
        \STATE {\bf For the $\ell_2$-norm case:} Calculate the upper bound $P$ as the following: 
        \begin{equation*}
        \sup_{\alpha > 1} [-\frac{2\sigma^2}{\alpha}\ln (1-p_{(1)}-p_{(2)}+2 (\frac{1}{2}(p_{(1)}^{1-\alpha}+p_{(2)}^{1-\alpha} ) )^{\frac{1}{1-\alpha}} ) ]^{1/2},
        \end{equation*}
        where $p_{(1)}$ and $p_{(2)}$ are the first and the second largest values in $\{p_i\}$. Then calculate the upper bound $Q$ as 
        \begin{equation*}
       \sup_{\alpha>1}[ \frac{2\sigma^2}{\alpha}(\frac{\alpha}{\alpha-1}\ln(2k_0(\sum_{i\in \mathcal{S}}\tilde{w}^\alpha_{i^*})^\frac{1}{\alpha}+(2k_0)^\frac{1}{\alpha}\sum_{i\not\in \mathcal{S}}\tilde{w}_{i^*}) -\frac{1}{\alpha-1}\ln (2k_0)) ]^{1/2}, 
        \end{equation*}
   where    $\tilde{w}_{i^*}$ is the $i$-th largest component in $\tilde{w}$,  $k_0=\lfloor (1-\beta)k \rfloor +1$, and $\mathcal{S}
$ denotes the set of last $k_0$ components
in top-$k$ indices and the top $k_0$ components out of top-$k$ indices.
\STATE {\bf For the $\ell_\infty$-norm case:} Calculate $P=\frac{\tilde{P}}{d}$ and $Q=\frac{\tilde{Q}}{d}$, where $\tilde{P}$  and $\tilde{Q}$  is equivalent to $P$ and $Q$ in the $\ell_2$-norm case respectively, $d=q\times n$. 
\STATE {\bfseries Return:} The tolerable size of the attack $\min\{P, Q\}$.
    \end{algorithmic}
\end{algorithm}

\section{Proof of Theorem \ref{thm:5.1}}
\begin{proof}
Firstly, we know that the $\alpha$-R\'{e}nyi divergence between two Gaussian distributions $\mathcal{N}(0, \sigma^2 I_d)$ and $\mathcal{N}(\mu, \sigma^2 I_d)$  is bounded by $\frac{\alpha\|\mu\|_2^2}{2\sigma^2}$. Thus by the postprocessing property of R\'{e}nyi divergence, we have 
\begin{equation*}
\begin{aligned}
     & D_\alpha (\tilde{w}(x), \tilde{w}(x'))=D_\alpha (Z(T(x+z)), Z(T(x'))) 
     \leq D_\alpha(x+z, x'+z) \\
     & \leq \frac{\alpha\|x-x'\|_F^2}{2\sigma^2}\leq \frac{\alpha R^2}{2\sigma^2}. 
\end{aligned}
\end{equation*}
Thus, when $\frac{\alpha R^2}{2\sigma^2}\leq \gamma$ it satisfies the utility robustness. 

Second, we show it satisfies the prediction robustness. We first recall the following lemma which shows a lower bound between the R\'{e}nyi divergence of two discrete distributions:

\begin{lemma}[R\'{e}nyi Divergence Lemma \cite{li2019certified}]\label{lemma:renyi}
Let $P=(p_1, p_2, ..., p_k)$ and $Q=(q_1, q_2, ..., q_k)$ be two multinomial distributions. If the indices of the largest probabilities \textbf{do not} match on $P$ and $Q$, then the R\'{e}nyi divergence between $P$ and $Q$, {\em i.e.,} $D_\alpha(P||Q)$\footnote{For $\alpha\in (1,\infty)$, $D_\alpha(P||Q)$ is defined as $D_\alpha(P||Q) = \frac{1}{\alpha-1} \log \mathbb{E}_{x\sim Q}(\frac{P(x)}{Q(x)})^\alpha$.}, satisfies
\begin{align}
        D_\alpha(P||Q) 
       \geq -\log(1 - p_{(1)} - p_{(2)} + 2(\frac{1}{2}(p_{(1)}^{1-\alpha} + p_{(2)}^{1-\alpha}))^{\frac{1}{1-\alpha}}).
    \nonumber 
\end{align}
where $p_{(1)}$ and $p_{(2)}$ refer to the largest and the second largest probabilities in $\{p_i\}$, respectively.
\end{lemma}
By Lemma \ref{lemma:renyi} we can see that as long as $  D_\alpha(\tilde{w}(x), \tilde{w}(x')) \leq -\log(1 - p_{(1)} - p_{(2)} + 2(\frac{1}{2}(p_{(1)}^{1-\alpha} + p_{(2)}^{1-\alpha}))^{\frac{1}{1-\alpha}})$ we must have  the prediction robustness. Thus, if $\frac{\alpha R^2}{2\sigma^2}\leq -\log(1 - p_{(1)} - p_{(2)} + 2(\frac{1}{2}(p_{(1)}^{1-\alpha} + p_{(2)}^{1-\alpha}))^{\frac{1}{1-\alpha}})$  we have the condition. 

Finally we proof the Top-$K$ robustness. The idea of the proof follows \cite{liu2021certifiably}. We proof the following lemma first 
\begin{lemma}\label{lemma:4.3}
Consider the set of all vectors with unit $\ell_1$-norm in $\mathbb{R}^T$, $\mathcal{Q}$. Then we have 
\begin{equation*}
\begin{aligned}
      \min_{q\in \mathcal{Q}, V_k(\hat{w}, q)\geq \beta} 
      D_\alpha(\hat{w}, q) =\frac{\alpha}{\alpha-1}\ln(2k_0(\sum_{i\in \mathcal{S}}\tilde{w}^\alpha_i)^\frac{1}{\alpha} 
       +(2k_0)^\frac{1}{\alpha}\sum_{i\not\in \mathcal{S}}\tilde{w}_i)-\frac{1}{\alpha-1}\ln (2k_0), 
\end{aligned}
\end{equation*}
where $D_\alpha(\hat{w}, q)$ is the $\alpha$-divergence of the distributions whose probability vectors are $\hat{w}$ and $q$. 
\end{lemma}
Now we back to the proof, we know that $D_\alpha(x+z, x'+z)\leq \frac{\alpha R^2}{2\sigma^2}$. And $D_\alpha(Z(T(x+z)), Z(T((x'+z)))\leq D_\alpha(x+z, x'+z)$. Thus, if $\frac{\alpha R^2}{2\sigma^2}\leq \frac{\alpha}{\alpha-1}\ln(2k_0(\sum_{i\in \mathcal{S}}\tilde{w}^\alpha_i)^\frac{1}{\alpha}+(2k_0)^\frac{1}{\alpha}\sum_{i\not\in \mathcal{S}}\tilde{w}_i)-\frac{1}{\alpha-1}\ln (2k_0)$, we must have $V_k(g(x+z), g(x'+z))\geq \beta$. 
\end{proof}

\begin{proof}[Proof of Lemma \ref{lemma:4.3}]
We denote $m^T=(m_1, m_2, \cdots, m_T)$ and $q^T=(q_1, \cdots, q_T)$. W.l.o.g we assume that $m_1\geq \cdots\geq m_T$.  Then,  to reach the minimum of R\'{e}nyi
divergence we show that the minimizer $q$ must satisfies $q_1\geq \cdots \geq q_{k-k_0-1}\geq q_{k-k_0}=\cdots=q_{k+k_0+1}\geq q_{k+k_0+2}\geq q_T$.  We need the following statements for the proof. 
\begin{lemma} \label{lemma:4.4}
We have the following statements: 
\begin{enumerate}
    \item 
To reach the minimum, there are exactly $k_0$ different components in the top-k of $\tilde{w}$ and $q$. 
\item To reach the minimum, $q_{k-k_0+1}, \cdots, q_k$ are not in the top-k of $q$. 
\item  To reach the minimum, $q_{k+1}, \cdots, q_{k+k_0}$  must appear in the top-k
of $q$. 
\item \cite{li2019certified} To reach the minimum, we must have $q_i\geq q_j$ for all $i\leq j$. 
\end{enumerate}
\end{lemma}
Thus, based on Lemma \ref{lemma:4.4}, we only need to solve the following optimization problem to find a minimizer $q$: 
\begin{align*}
   & \min_{q_1, \cdots, q_T}=\sum_{i=1}^T q_i (\frac{\tilde{w}_i }{q_i})^\alpha \\ 
   &\textbf{s.t. } \sum_{i=1}^T q_i=1 \\
   & \textbf{s.t. } q_i\leq q_j, i\geq j \\
   & \textbf{s.t. } q_i\geq 0\\
   & \textbf{s.t. } q_i-q_j=0, \forall i, j\in \mathcal{S}=\{k-k_0+1, \cdots, k+k_0\} 
\end{align*}
Solve the above optimization by using the Lagrangian method, we can get 
\begin{equation}
    q_i=\frac{s}{2k_0 s+(2k_0)^\frac{1}{\alpha}\sum_{i\not\in \mathcal{S}}\tilde{w}_i}, \forall i\in \mathcal{S}, 
\end{equation}
\begin{equation}
     q_i=\frac{(2k_0)^\frac{1}{\alpha}\tilde{w}_i}{2k_0 s+(2k_0)^\frac{1}{\alpha}\sum_{i\not\in \mathcal{S}}\tilde{w}_i}, \forall i\not\in \mathcal{S}
\end{equation}
where $s=(\sum_{i\in \mathcal{S}}\tilde{w}^\alpha_i)^\frac{1}{\alpha}$. We can get in this case $D_\alpha(\tilde{w}, q)=\frac{\alpha}{\alpha-1}\ln(2k_0s+(2k_0)^\frac{1}{\alpha}\sum_{i\not\in \mathcal{S}}\tilde{w}_i)-\frac{1}{\alpha-1}\ln (2k_0) $. 

\end{proof}
\begin{proof}[Proof of Lemma \ref{lemma:4.4}]
We first proof the first item: 

Assume that $i_1, \cdots, i_{k_0+j}$ are the $j$ components in the top-k of $\tilde{w}$ but not 
in the top-k of $q$, and  $i'_1, \cdots, i'_{k_0+j}$ 
are the components  in the top-k of q but
not in the top-k of  $\tilde{w}$. Consider we have another vector $q^1$ with
the same value with $q$ while replace $q_{i_{k_0+j}}$ with $q_{i'_{k_0+j}}$. Thus we have 
\begin{align*}
    &e^{(\alpha-1) D_\alpha(\tilde{w}, q^1)}-  e^{(\alpha-1)D_\alpha(\tilde{w}, q)} \\
    & = (\frac{\tilde{w}^\alpha_{i_{k_0+j}} }{q^{\alpha-1
    }_{i'_{k_0+j}}}+\frac{\tilde{w}^\alpha_{i'_{k_0+j}} }{q^{\alpha-1
    }_{i_{k_0+j}}})-(\frac{\tilde{w}^\alpha_{i_{k_0+j}} }{q^{\alpha-1
    }_{i_{k_0+j}}}+\frac{\tilde{w}^\alpha_{i'_{k_0+j}} }{q^{\alpha-1
    }_{i'_{k_0+j}}})\\
    &=(\tilde{w}^\alpha_{i_{k_0+j}} -\tilde{w}^\alpha_{i'_{k_0+j}})(\frac{1}{q^{\alpha-1
    }_{i'_{k_0+j}}}-\frac{1}{q^{\alpha-1
    }_{i_{k_0+j}}})<0,
\end{align*}
since $\tilde{w}_{i_{k_0+j}}\geq \tilde{w}_{i'_{k_0+j}}$ and $q_{i'_{k_0+j}}\geq q _{i_{k_0+j}}$.  Thus, we know reducing the number of misplacement in
top-k can reduce the value $D_\alpha(\tilde{w}, q)$ which contradict to $q$ achieves the minimal. Thus we must have $j=0$. 

We then proof the second statement. 

Assume that $i_1, \cdots, i_{k_0}$ are the $k_0$ components in the top-k of $\tilde{w}$ but not 
in the top-k of $q$, and  $i'_1, \cdots, i'_{k_0}$ 
are the components  in the top-k of q but
not in the top-k of  $\tilde{w}$.  Consider we have another unit $\ell_1$-norm vector $q^2$  with the
same value with $q$ while $q_{i_j}$ is replaced by $q_{j'}$ where $\tilde{w}_{j'}\geq \tilde{w}_{i_j}$ and $j'$ is in the top-k component of $q$ (there must exists such index $j'$). Now we can see that $q^2_{j'}$ is no longer a top-k component of $q^2$ and $q^2_{i_j}$ is a top-k component.  Thus we have 
\begin{align*}
   &e^{(\alpha-1) D_\alpha(\tilde{w}, q^2)}-  e^{(\alpha-1) D_\alpha(\tilde{w}, q)} \\
   & = (\frac{\tilde{w}^\alpha_{i_{j}} }{q^{\alpha-1
    }_{j'}}+\frac{\tilde{w}^\alpha_{j'} }{q^{\alpha-1
    }_{i_{j}}})-(\frac{\tilde{w}^\alpha_{i_{j}} }{q^{\alpha-1
    }_{i_{j}}}+\frac{\tilde{w}^\alpha_{j'} }{q^{\alpha-1
    }_{j'}})\\
    &=(\tilde{w}^\alpha_{i_{j}} -\tilde{w}^\alpha_{j'})(\frac{1}{q^{\alpha-1
    }_{j'}}-\frac{1}{q^{\alpha-1
    }_{i_{j}}})\geq 0. 
\end{align*}
Now we back to the proof of the statement. We first proof $q_k$ is not in the top-k of $q$. If not, that is $k\not\in \{i_1, \cdots, i_{k_0}\}$ and all $i_j<k$. Then we can always find an $i_j<k$ such that $\tilde{w}_k\leq \tilde{w}_{i_j}$, we can find a vector $\tilde{q}$ by replacing $q_{i_j}$ with $q_k$. And we can see that $   D_\alpha(\tilde{w}, \tilde{q})-  D_\alpha(\tilde{w}, q)\leq 0$, which contradict to that $q$ is the minimizer. 

We then proof $q_{k-1}$ is not in the top-k of $q$. If not we can construct $\tilde{q}$ by replacing $q_{k}$ with $q_{k-1}$. Since $q_k$ is not in top-k and $\tilde{w}_{k}\leq \tilde{w}_{k-1}$. By the previous statement we have $   D_\alpha(\tilde{w}, \tilde{q})-  D_\alpha(\tilde{w}, q)\leq 0$, which contradict to that $q$ is the minimizer. Thus, $q_{k-1}$ is not in the top-k of $q$. We can thus use induction to proof statement 2. 

Finally we proof statement 3. We can easily show that $q_{i}\geq q_{k+1}$ for $i\leq k$, and $q_{i}\leq q_{k+1}$ for $i\geq k+2$. Thus, $q_1, \cdots, q_k$ are greater than the left entries. Since by Statement 2 we have $q_{k-k_0}, \cdots q_k$ are not top $k$. Thus we must have $q_{k+1}, \cdots q_{k+k_0} $ must be top-k of $q$. 

\end{proof}

\section{Proof of Theorem \ref{thm:l2lower}}
\begin{proof}
For simplicity in the following we think the data $x$ as a $d$-dimensional vector and thus the Frobenious norm now becomes to the $\ell_2$-norm of the vector and the max norm becomes to the $\ell_\infty$-norm of the vector.  

We first show that, in order to prove Theorem~\ref{thm:l2lower}, we only need to prove Theorem~\ref{thm:formall2lower1}. Then we show that, to prove Theorem~\ref{thm:formall2lower1}, we only need to prove Theorem~\ref{thm:formall2lower2}. Finally, we give a formal proof of Theorem~\ref{thm:formall2lower2}.

\begin{theorem} \label{thm:formall2lower1}

For any $\gamma \leq O(1)$, if  a randomized (smoothing) mechanism $\mathcal{M}(x)=x+z: \{0, \frac{R}{2\sqrt{d}}\}^d \mapsto \mathbb{R}^d$ that $D_\alpha(\mathcal{M}(x), \mathcal{M}(x'))\leq \gamma$ for all $\|x-x'\|_2\leq R$. Moreover, if we have 
 for any $x \in \{0, \frac{R}{2\sqrt{d}}\}^d$,
\begin{equation*}
    \mathbb{E}[\|z\|_\infty]= \mathbb{E}[\|\mathcal{M}(x)-x\|_\infty]\leq \tau
\end{equation*}
for some $\tau \leq O(1)$. Then it must be true that $\tau \geq \Omega(\frac{\sqrt{\alpha}R}{\sqrt{\gamma}})$. 
\end{theorem}
For any $\mathcal{M}(x) = x + z: \mathbb{R}^d\mapsto \mathbb{R}^d$, in Theorem~\ref{thm:formall2lower1}, we only consider the expected $\ell_\infty$-norm of the noise added by $\mathcal{M}(x)$ on $x \in \{0, \frac{R}{2\sqrt{d}}\}^d$. Thus, the $\tau$ in Theorem~\ref{thm:formall2lower1} should be less than or equal to the $\tau$ in Theorem~\ref{thm:l2lower} (on $x \in \mathbb{R}^d$). Therefore, the lower bound for the $\tau$ in Theorem~\ref{thm:formall2lower1} ({\em i.e.,} $\Omega(\frac{\sqrt{\alpha}R}{\sqrt{\gamma}})$) is also a lower bound for the $\tau$ in Theorem~\ref{thm:l2lower}. That is to say, if Theorem~\ref{thm:formall2lower1} holds, then Theorem~\ref{thm:l2lower} also holds true. 

Next, we show that if Theorem~\ref{thm:formall2lower2} holds, then Theorem~\ref{thm:formall2lower1} also holds.

\begin{theorem} \label{thm:formall2lower2}
For any $\gamma \leq O(1)$, if  a randomized (smoothing) mechanism  $\mathcal{M}(x)\footnote{This mechanism might not be simply $x + z$ since it must involve operations to clip the output into $[0, \frac{R}{2\sqrt{d}}]^d$}: \{0, \frac{R}{2\sqrt{d}}\}^d\mapsto [0, \frac{R}{2\sqrt{d}}]^d$ that $D_\alpha(\mathcal{M}(x), \mathcal{M}(x'))\leq \gamma$ for all $\|x-x'\|_2\leq R$. Moreover, if  for any $x\in \{0, \frac{R}{2\sqrt{d}}\}^d$ 
\begin{equation*}
    \mathbb{E}[\|z\|_\infty]= \mathbb{E}[\|\mathcal{M}(x)-x\|_\infty]\leq \tau
\end{equation*}
for some $\tau \leq O(1)$. Then it must be true that $\tau  \geq \Omega(\frac{\sqrt{\alpha}R}{\sqrt{\gamma}})$. 
\end{theorem}
\paragraph{Proof of Theorem~\ref{thm:formall2lower1}} For any $\mathcal{M}(x)=x + z: \{0, \frac{R}{2\sqrt{d}}\}^d\mapsto \mathbb{R}^d$ considered in Theorem~\ref{thm:formall2lower1} that $D_\alpha(\mathcal{M}(x), \mathcal{M}(x'))\leq \gamma$ for all $\|x-x'\|_2\leq R$, there  randomized mechanism $\mathcal{M}''(x): \{0, \frac{R}{2\sqrt{d}}\}^d\mapsto [0, \frac{R}{2\sqrt{d}}]^d$ considered in Theorem~\ref{thm:formall2lower2} such that for all $x \in \{0, \frac{R}{2\sqrt{d}}\}^d$
\begin{equation*}
   \mathbb{E}[\|\mathcal{M}''(x)-x\|_\infty]\leq  \mathbb{E}[\|\mathcal{M}(x)-x\|_\infty].
\end{equation*}
To prove the above statement, we first let $a=\frac{R}{2\sqrt{d}}$ and $\mathcal{M}'(x)=\min\{\mathcal{M}(x), a\}$, where $\min$ is a coordinate-wise operator. Now we fix the randomness of $\mathcal{M}(x)$ (that is $\mathcal{M}(x)$ is deterministic), and we assume that $\|\mathcal{M}(x)-x\|_\infty= |\mathcal{M}_j(x)-x_j|$, $\|\mathcal{M}'(x)-x\|_\infty=|\mathcal{M}_i'(x)-x_i|$. If $\mathcal{M}_i(x)< a$, then by the definitions, we have $\|\mathcal{M}'(x)-x\|_\infty = |\mathcal{M}_i'(x)-x_i| = |\mathcal{M}_i(x)-x_i| \leq \|\mathcal{M}(x)-x\|_\infty$. If $\mathcal{M}_i(x)\geq  a$, then we have $|\mathcal{M}_i'(x)-x_i|=|a-x_i|$. Since $x_i \in \{0, a\}$ and $\mathcal{M}_i(x)\geq a$, $|\mathcal{M}_i(x)-x_i|\geq|a-x_i|$. $\|\mathcal{M}(x)-x\|_\infty \geq |\mathcal{M}_i(x)-x_i|\geq |a-x_i|$. Thus, $\mathbb{E}[\|\mathcal{M}'(x)-x\|_\infty]\leq  \mathbb{E}[\|\mathcal{M}(x)-x\|_\infty]$.

Then, we let $\mathcal{M}''(x)=\max\{\mathcal{M}'(x), 0\}$ where $\max$ is also a coordinate-wise operator. We can use a similar method to prove that $\mathbb{E}[\|\mathcal{M}''(x)-x\|_\infty]\leq \mathbb{E}[\|\mathcal{M}'(x)-x\|_\infty] \leq \mathbb{E}[\|\mathcal{M}(x)-x\|_\infty]$. Also, we can see that $\mathcal{M}''(x)=\max\{0, \min\{\mathcal{M}(x), a\}\}$, which means $\mathcal{M}''$ satisfies  $D_\alpha(\mathcal{M}''(x), \mathcal{M}''(x'))\leq \gamma$ for all $\|x-x'\|_2\leq R$ due to the postprocessing property. 

Since $\mathbb{E}[\|\mathcal{M}''(x)-x\|_\infty]\leq \mathbb{E}[\|\mathcal{M}(x)-x\|_\infty]$, and $\mathcal{M}''(x)$ is a randomized mechanism satisfying the conditions in Theorem~\ref{thm:formall2lower2}, the $\tau$ in Theorem~\ref{thm:formall2lower2} should be less than or equal to the $\tau$ in Theorem~\ref{thm:formall2lower1}.  
Therefore, the lower bound for the $\tau$ in Theorem~\ref{thm:formall2lower2} is also a lower bound for the $\tau$ in Theorem~\ref{thm:formall2lower1}. That is to say, if Theorem~\ref{thm:formall2lower2} holds, then Theorem~\ref{thm:formall2lower1} also holds.

{\em Finally, we give a proof of Theorem \ref{thm:formall2lower2}.} Before that we need to review some definitions of Differnetial Privacy \cite{dwork2006calibrating}.

\begin{definition} \label{def:DP}
Given a data universe $X$, we say
that two datasets $D, D' \subset X$ are neighbors if they differ by only one entry, which is denoted by
$D\sim D'$. A randomized algorithm $\mathcal{M}$ is $(\epsilon,\delta)$-differentially private (DP) if for all neighboring datasets $D, D'$ and all events $S$
the following holds
\begin{equation*}
    P(\mathcal{M}(D)\in S)\leq e^\epsilon P(\mathcal{M}(D')\in S)+\delta.
\end{equation*}
\end{definition}
\begin{definition}
 A randomized algorithm $\mathcal{M}$ is $(\alpha , \epsilon)$-R\'{e}nyi differentially private (DP) if for all neighboring datasets $D, D'$ 
the following holds
\begin{equation*}
    D_\alpha(\mathcal{M}(D)\|\mathcal{M}(D'))\leq \epsilon. 
\end{equation*}
\end{definition}
\begin{lemma}[From RDP to DP \cite{mironov2017renyi}]\label{rdp_dp}
If a mechanism is $(\alpha , \epsilon)$-RDP, then it also satifies $(\epsilon+\frac{\log \frac{1}{\delta}}{\alpha-1}, \delta)$-DP. 
\end{lemma}

\paragraph{Proof of Theorem \ref{thm:formall2lower2}}	Since $\mathcal{M}$ satisfies $D_\alpha(\mathcal{M}(x), \mathcal{M}(x'))\leq \gamma$ for all $\|x-x'\|_2\leq R$ on $\{0, \frac{r}{2\sqrt{d}}\}^d$, and for any $x_i, x_j \in \{0, \frac{R}{2\sqrt{d}}\}^d$, 
	$\|x_i-x_j\|_2 \leq R$ ({\em i.e.,} $x_j \in \mathbb{B}_2(x_i, R)$), we can see $\mathcal{M}$ is $(\alpha, \gamma)$-RDP on $\{0, \frac{r}{2\sqrt{d}}\}^d$. Thus by Lemma \ref{rdp_dp} we can see  
$\mathcal{M}$ is $(\gamma + 2\frac{\log \frac{1}{\delta}}{\alpha-1}, \delta)$-DP on $\{0, \frac{r}{2\sqrt{d}}\}^d$. 
	
	Then let us take use of the above condition by connecting the lower bound of the sample complexity to estimate one-way marginals ({\em i.e.,} mean estimation) for DP mechanisms with the lower bound studied in Theorem \ref{thm:formall2lower2}. Suppose an $n$-size dataset $X\in \mathbb{R}^{n\times d}$, the one-way marginal is $h(D)=\frac{1}{n}\sum_{i=1}^n{X_i}$, where $X_i$ is the $i$-th row of $X$. In particular, when $n=1$, one-way marginal is just the data point itself, and thus, the condition in Theorem \ref{thm:formall2lower2} can be rewritten as 
	\begin{equation}
		\mathbb{E}[\|\mathcal{M}(D)-h(D)\|_\infty] \leq \alpha.
	\end{equation}
	
	Based on this connection, we first prove the case where $r=2\sqrt{d}$, and then generalize it to any $r$. For $r=2\sqrt{d}$, the conclusion reduces to $\tau \geq \Omega(\sqrt{\frac{d}{\epsilon}})$. To prove this, we employ the following lemma, which provides a one-way margin estimation for all DP mechanisms.
		\begin{lemma}[Theorem 1.1 in \cite{steinke2016between}]\label{lemma:one_way}
			For any $\epsilon\leq O(1)$, every $2^{-\Omega(n)}\leq \delta \leq \frac{1}{n^{1+\Omega(1)}}$ and every $\alpha\leq \frac{1}{10}$, if $\mathcal{M}: (\{0, 1\}^d)^n \mapsto [0,1]^d$ is $(\epsilon, \delta)$-DP and $\mathbb{E}[\|\mathcal{M}(D)-h(D)\|_\infty]\leq \tau$, then  we have 
			$
				n\geq \Omega(\frac{\sqrt{d\log \frac{1}{\delta}}}{\epsilon\tau }). 
			$
		\end{lemma}
		Setting $n=1, \epsilon=\gamma + 2\frac{\log \frac{1}{\delta}}{\alpha-1}$ in Lemma~\ref{lemma:one_way}, we can see that if $\mathbb{E}[\|\mathcal{M}(x)-x\|_\infty]\leq \tau$, then we must have 
		$$1\geq \Omega(\frac{\sqrt{d\log \frac{1}{\delta}}}{(\gamma + 2\frac{\log \frac{1}{\delta}}{\alpha-1})\tau})\geq \Omega(\frac{\sqrt{\alpha}\sqrt{d}}{\sqrt{\gamma}\tau }),$$ where the last inequality holds if $\alpha$ is sufficiently large and $\gamma$ is sufficiently small. 
    Therefore, we have the following theorem, 
    \begin{theorem}\label{thm:simple_l2bound}
For all $\mathcal{M}$ satisfies $D_\alpha(\mathcal{M}(x), \mathcal{M}(x'))\leq \gamma$ for all $\|x-x'\|_2\leq 2\sqrt{d}$ on $\{0, 1\}^d$  such that 
		\begin{equation}
			\mathbb{E}[\|\mathcal{M}(x)- x\|_\infty] \leq \tau,
		\end{equation}
		for some $\tau\leq O(1)$. Then $\tau \geq \Omega(\frac{\sqrt{\alpha d}}{\sqrt{\gamma}})$.
	\end{theorem}
	Apparently, Theorem~\ref{thm:simple_l2bound} is special case of Theorem~\ref{thm:formall2lower2} where $R = 2\sqrt{d}$.
	Now we come back to the proof for any $\mathcal{M}(x): \{0, \frac{r}{2\sqrt{d}}\}^d\mapsto [0, \frac{r}{2\sqrt{d}}]^d$ satisfies $D_\alpha(\mathcal{M}(x), \mathcal{M}(x'))\leq \gamma$ for all $\|x-x'\|_2\leq R$. We substitute $\frac{2\sqrt{d}}{R}x$ with $\tilde{x} \in \{0, 1\}^d$ and construct $\tilde{\mathcal{M}}$ as $\tilde{\mathcal{M}}(\tilde{x}) = \frac{2\sqrt{d}}{R}\mathcal{M}(x) \in [0, 1]^d$. Since $\mathcal{M}(x)$ satisfies
	$$\mathbb{E}[\|\mathcal{M}(x)-x\|_\infty] \leq \tau,$$
	then we have $$\mathbb{E}[\|\tilde{\mathcal{M}}(\tilde{x})-\tilde{x}\|_\infty] = \mathbb{E}[\|\frac{2\sqrt{d}}{R}\mathcal{M}(x)-\frac{2\sqrt{d}}{R}x\|_\infty] \leq \frac{2\sqrt{d}}{R}\alpha.$$ 
By the postprocessing property of R\'{e}nyi divergence we can see $D_\alpha(\tilde{\mathcal{M}}(x), \tilde{\mathcal{M}}(x'))\leq \gamma$ for all $\|x-x'\|_2\leq 2\sqrt{d}$.  
	
    Considering $\tilde{\mathcal{M}}: \{0, 1\}^d\mapsto [0, 1]^d$ in Theorem~\ref{thm:simple_l2bound} with $\tau=\frac{2\sqrt{d}}{r}\tau\leq O(1)$ (because $\mathbb{E}[\|\tilde{\mathcal{M}}(\tilde{x})-\tilde{x}\|_\infty] \leq \frac{2\sqrt{d}}{R}\tau$), we have 
	\begin{equation}
	\frac{2\sqrt{d}}{R}\tau\geq \Omega(\frac{\sqrt{\alpha d}}{\sqrt{\gamma}}). 
	\end{equation}
	Therefore, Theorem~\ref{thm:formall2lower2} holds true, thus, Theorem~\ref{thm:formall2lower1} also holds true, and Theorem~\ref{thm:l2lower} is proved.
\end{proof}

\section{Proof of Theorem \ref{thm:5.2}}
\begin{proof}
We can see the dataset as a $d$-dimensional vector by unfolding it. Thus, now the max norm of a matrix becomes to the $\ell_\infty$-norm of a vector. 
Firstly, we know that the $\alpha$-R\'{e}nyi divergence between two Gaussian distributions $\mathcal{N}(0, \sigma^2 I_d)$ and $\mathcal{N}(\mu, \sigma^2 I_d)$  is bounded by $\frac{\alpha\|\mu\|_2^2}{2\sigma^2}\leq \frac{\alpha d R^2}{2\sigma^2}$. Thus by the postprocessing property of R\'{e}nyi divergence we have 
\begin{equation*}
    \begin{aligned}
   & D_\alpha (\tilde{w}(x), \tilde{w}(x'))=D_\alpha (Z(T(x+z)), Z(T(x')))\leq D_\alpha(x+z, x'+z) \\
   & \leq \frac{\alpha\|x-x'\|_F^2}{2\sigma^2}\leq \frac{\alpha d R^2}{2\sigma^2}. 
   \end{aligned}
\end{equation*}
Thus, when $\frac{d\alpha R^2}{2\sigma^2}\leq \gamma$ it satisfies the utility robustness. For the prediction and top-$k$ robustness we can use the similar proof as in Theorem \ref{thm:5.1}. We omit it here for simplicity.

    \end{proof}

\section{Proof of Theorem \ref{thm:linflower}}
\begin{proof}
Similar to the proof of Theorem \ref{thm:l2lower}, in order to prove Theorem~\ref{thm:linflower}, we only need to prove the following theorem:
\begin{theorem}\label{thm:formallinf}
 If there is a randomized (smoothing) mechanism $\mathcal{M}(x): \{0, \frac{r}{2}\}^d\mapsto [0, \frac{r}{2}]^d$ such that $D_\alpha(\mathcal{M}(x), \mathcal{M}(x'))\leq \gamma$ for all $\|x-x'\|_\infty \leq R$ for any $x\in \{0,\frac{r}{2}\}^d$, the following holds 
\begin{equation*}
    \mathbb{E}[\|z\|_\infty]= \mathbb{E}[\|\mathcal{M}(x)-x\|_\infty] \leq \gamma
\end{equation*}
for some $\gamma \leq O(1)$. Then it must be true that $\gamma \geq \Omega (\frac{\sqrt{\alpha d}}{\sqrt{\gamma}})$. 
\end{theorem}
	Since $\mathcal{M}$ satisfies $D_\alpha(\mathcal{M}(x), \mathcal{M}(x'))\leq \gamma$ for all $\|x-x'\|_\infty \leq R$ on $\{0, \frac{r}{2}\}^d$, and for any $x_i, x_j \in \{0, \frac{R}{2}\}^d$, 
	$\|x_i-x_j\|_2 \leq R$ ({\em i.e.,} $x_j \in \mathbb{B}_2(x_i, R)$), we can see $\mathcal{M}$ is $(\alpha, \gamma)$-RDP on $\{0, \frac{r}{2}\}^d$. Thus by Lemma \ref{rdp_dp} we can see  
$\mathcal{M}$ is $(\gamma + 2\frac{\log \frac{1}{\delta}}{\alpha-1}, \delta)$-DP on $\{0, \frac{r}{2}\}^d$. 
	
     We first consider the case where $r=2$. By setting $n=1$ and $\gamma=\gamma+2\frac{\log \frac{1}{\delta}}{\alpha-1}$ in Lemma~\ref{lemma:one_way}, we have a similar result as in Theorem \ref{thm:simple_l2bound}: 
	    \begin{theorem}\label{thm:simple_linfbound}
	    For all $\mathcal{M}$ satisfies $D_\alpha(\mathcal{M}(x), \mathcal{M}(x'))\leq \gamma$ for all $\|x-x'\|_\infty \leq 2$ on $\{0, 1\}^d$  such that 
		\begin{equation}
			\mathbb{E}[\|\mathcal{M}(x)- x\|_\infty] \leq \tau,
		\end{equation}
		for some $\tau\leq O(1)$. Then $\tau \geq \Omega(\frac{\sqrt{\alpha d}}{\sqrt{\gamma}})$.
	    
	\end{theorem}
	
	For general $R$, similar to the proof of Theorem~\ref{thm:formall2lower2}, we substitute $\frac{2}{R}x$ with $\tilde{x} \in \{0, 1\}^d$ and construct $\tilde{\mathcal{M}}$ as $\tilde{\mathcal{M}}(\tilde{x}) = \frac{2}{R}\mathcal{M}(x) \in [0, 1]^d$. 
	Since $\mathcal{M}(x)$ satisfies
	$$\mathbb{E}[\|\mathcal{M}(x)-x\|_\infty] \leq \alpha,$$
	then we have $$\mathbb{E}[\|\tilde{\mathcal{M}}(\tilde{x})-\tilde{x}\|_\infty] = \mathbb{E}[\|\frac{2}{R}\mathcal{M}(x)-\frac{2}{R}x\|_\infty] \leq \frac{2}{R}\alpha.$$ 
	Also, $\tilde{\mathcal{M}}: \{0, 1\}^d\mapsto [0, 1]^d$  satisfies $D_\alpha(\tilde{\mathcal{M}}(x), \tilde{\mathcal{M}}(x'))\leq \gamma$ for all $\|x-x'\|_\infty \leq R$ on $\{0, \frac{R}{2}\}^d$. 
    Thus by Theorem \ref{thm:simple_linfbound} with $\alpha=\frac{2}{R}\alpha$ we have 
	\begin{equation*}
	    \alpha \geq \Omega(\frac{R\sqrt{\alpha d}}{\sqrt{\gamma}}),
	\end{equation*}
	thus we have Theorem~\ref{thm:formallinf}. 
\end{proof}

\section{Baseline Methods}
Five baseline methods are considered in this paper. The implementation and parameter setting of each method is based on the corresponding official code. Note that in this paper, we did not involve compression with Shapely-value methods \cite{lundberg2017unified} due to the large computational complexity and sub-optimal performance \cite{chen2018shapley}. 

\begin{enumerate}
    \item \header{Raw Attention \cite{vaswani2017attention}} It is common practice to consider the raw attention value as a relevancy score for a single attention layer in both visual and language domains \cite{xu2015show}. However, for the case of multiple layers, the attention score in deeper layers may be unreliable for explaining the importance of each token due to the token mixing property \cite{lee2021fnet} of the self-attention mechanism. Based on observations in \cite{chefer2021transformer}, we consider the raw attention in the first layer since they are more faithful in explanation compared to deeper layers.
    \item \header{Rollout \cite{abnar2020quantifying}} To compute the attention weights from positions in layer $l_i$ to positions in layer $l_j$ in a Transformer with $L$ layers, we multiply the attention weights matrices in all the layers below the layer $l_i$. If $i > j$, we multiply by the attention weights matrix in the layer $l_{i-1}$, and if $i = j$, we do not multiply by any additional matrices. This process can be represented by the following equation:

$$ \tilde{A}(l_i)=\left\{  \begin{array}{@{}ll@{}}    A(l_i)\tilde{A}(l_{i-1}) & \mbox{if}~ i>j \\    A(l_i) & \mbox{if}~ i=j  \end{array}\right.  $$

    \item \header{GradCAM \cite{selvaraju2017grad}} The main motivation of the Class Activation Mapping (CAM) approach is to obtain a weighted map based on the feature channels in one layer. The derived map can explain the importance of each pixel of the input image based on the intuition that non-related features are filtered in the channels of the deep layer. GradCAM \cite{selvaraju2017grad} proposes to leverage the gradient information by globally averaging the gradient score as the weight. To be more specific, the weight of channel $k$ with respect to class $c$ is calculated using 
    
    $$\alpha_k^c=\frac{1}{Z} \sum_i \sum_j \quad \frac{\partial y^c}{\partial A_{i j}^k},$$
    
    where $\alpha_k^c$ is the attention weight for feature map $k$ of the final convolutional layer, for class $c$. $y^c$ is the output class score for class $c$, and $A_{i j}^k$ is the activation at location $(i,j)$ in feature map $k$. The summation over $i$ and $j$ indicates that the gradient is computed over all locations in the feature map. The normalization factor $Z$ is the sum of all the attention weights for the feature maps in the final convolutional layer.
    \item \header{LRP \cite{bach2015pixel}} Layer-wise Relevance Propagation (LRP) method propagates relevance from the predicated score in the final layer to the input image based on the Deep Taylor Decomposition (DTD) \cite{montavon2017explaining} framework. Specifically, to compute the relevance of each neuron in the network, we iteratively perform backward propagation using the following equation: $$R_j=\sum_k \frac{u_j w_{j k}}{\sum_{j} a_j w_{j k}+\epsilon\cdot \text{sign}(a_j w_{j k})} R_k,$$where $R_j$ and $R_k$ are the relevance scores of neurons $j$ and $k$, respectively, in consecutive layers of the network. $a$ represents the activation of neuron $j$, $w_{j,k}$ is the weight between neurons $j$ and $k$, and $\epsilon$ is a small constant used as a stabilizer to prevent numerical degeneration. For more information on this technique, please see the original paper.
    \item \header{VTA \cite{chefer2021transformer}} Vanilla Trans. Att. (VTA) uses a LRP-based relevance score to evaluate the importance of each attention head in each layer. These scores are then integrated throughout the attention graph by combining relevancy and gradient information in an iterative process that eliminates negative values.
\end{enumerate}


\section{Implementation details}
\header{Diffusion Denoiser Implementation}
For COCO and CitySpace, we trained the diffusion models from scratch following \cite{ho2020denoising}. For ImageNet, we leverage the pre-trained diffusion model released in the \texttt{guided-diffusion}. Spcificly, the \texttt{256x256\_diffusion\_uncond} is used as a denoiser. To resolve the size mismatch, we resize the images each time of their inputs and outputs from the diffusion model. The diffusion model we adopted uses a linear noise schedule with $\beta_1=0.0001$ and $\beta_N=0.02$. The sampling steps $N$ are set to 1000. We clip the optimal $t^*$ when it falls outside the range of $[0, N-1]$. 

\header{Max-fuse with lowest pixels drop}
After obtaining explanation maps with a number of sampled noisy images, instead of fusing these maps with \textit{mean} operation, we leverage the approach of max fusing with the lowest pixels drop following \cite{jacobgil}. Specifically, we drop the lowest $10\%$ unimportant pixels for each map and apply element-wise maximum on the set of modified maps. The element of the final map is re-scaled to $[0, 1]$ using min-max normalization. 


\header{Model Training Details}
We use the pre-trained backbones in the \texttt{timm} library for feature extractor of classification and segmentation. As for different ViT, we both leverage the base version with a patch size of 16 and an image size of 224. For the downstream dataset, we then fine-tuned these models using the Adam optimizer with a learning rate of 0.001 for a total of 50 epochs, with a batch size of 128. To prevent overfitting, we implemented early stopping with a criterion of 20 epochs. For data augmentation, we follow the common practice: {Resize(256) $\rightarrow$ CenterCrop(224) $\rightarrow$ ToTensor $\rightarrow$ Normalization}. And the mean and stand deviation of normalization are both $[0.5,0.5,0.5]$.

\noindent {\bf Adversarial Perturbation.}
The perturbation radius is denoted by $\rho_u$ and is set to $8/255$ unless otherwise stated. For CIFAR-10, ImageNet, and COCO, the step size is set to $\rho_u/5$, and the total number of steps is set to $10$. For the Cityscape dataset, the step size was set to $\rho_u/125$, and the total number of steps was set to $250$. 

\section{More Results}
In this section, we provide more results to demonstrate the performance of our methods in terms of both model prediction and explainability. First, we aim to evaluate whether our method will affect utility when no attack is presented. We evaluate this on the ImageNet classification dataset using three different kinds of model architectures. Then we ablate the component proposed in our method to study their individual contributions. 

\subsection{Clean Utility}
As we can see from Table \ref{tab: cleanu}, our method outperforms the Vanilla approach under a relatively small smoothing radius, $\delta=2/255$. This result suggests that our method is able to enhance the classification utility with appropriate $\delta$. However, we also find that, as $\delta$ increases to $\delta=5/255$, there is a slight performance drop. And as the $\delta$ increases to $\delta=10/255$, the testing accuracy drops more, which indicates the necessity of choosing the right $\delta$. Since too large $\delta$ might lead to the smoothed images being overwhelmed with noise, which will lead to lower classification confidence. In a word, the results show that our method can even improve clean utility with appropriate small $\delta$. 

\begin{table}[htbp]
\centering
\begin{tabular}{ccccc} 
\toprule
\multirow{2}{*}{Method} & \multirow{2}{*}{$\delta$} & \multicolumn{3}{c}{Model}                                                                                                 \\ 
\cmidrule(l){3-5}
                        &                           & ViT                                                & DeiT                                               & Swin            \\ 
\midrule
Vanilla                 & -                         & \textcolor[rgb]{0.129,0.145,0.161}{85.22}          & 85.80                                              & 86.40           \\ 
\midrule
\multirow{3}{*}{Ours}   & $2/255$                   & \textcolor[rgb]{0.129,0.145,0.161}{\textbf{86.35}} & \textcolor[rgb]{0.129,0.145,0.161}{\textbf{86.50}} & \textbf{87.20}  \\
                        & $5/255$                   & \textcolor[rgb]{0.129,0.145,0.161}{84.83}          & 84.51                                              & 85.84           \\
                        & $10/255$                  & 79.59                                              & 80.89                                              & 81.25           \\
\bottomrule
\end{tabular}
\caption{The testing accuracy of our method and vanilla approach on ImageNet using three different ViT-based models under no attack. }
\label{tab: cleanu}
\end{table}



\begin{table}
	\centering 
	\resizebox{\linewidth}{!}{
 \begin{tabular}{ccccc}
		\toprule 
	Input & Clean & PGD & FGSM & AutoAttack \\
	\includegraphics[width=0.15\linewidth]{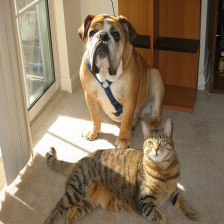}&
	\includegraphics[width=0.15\linewidth]{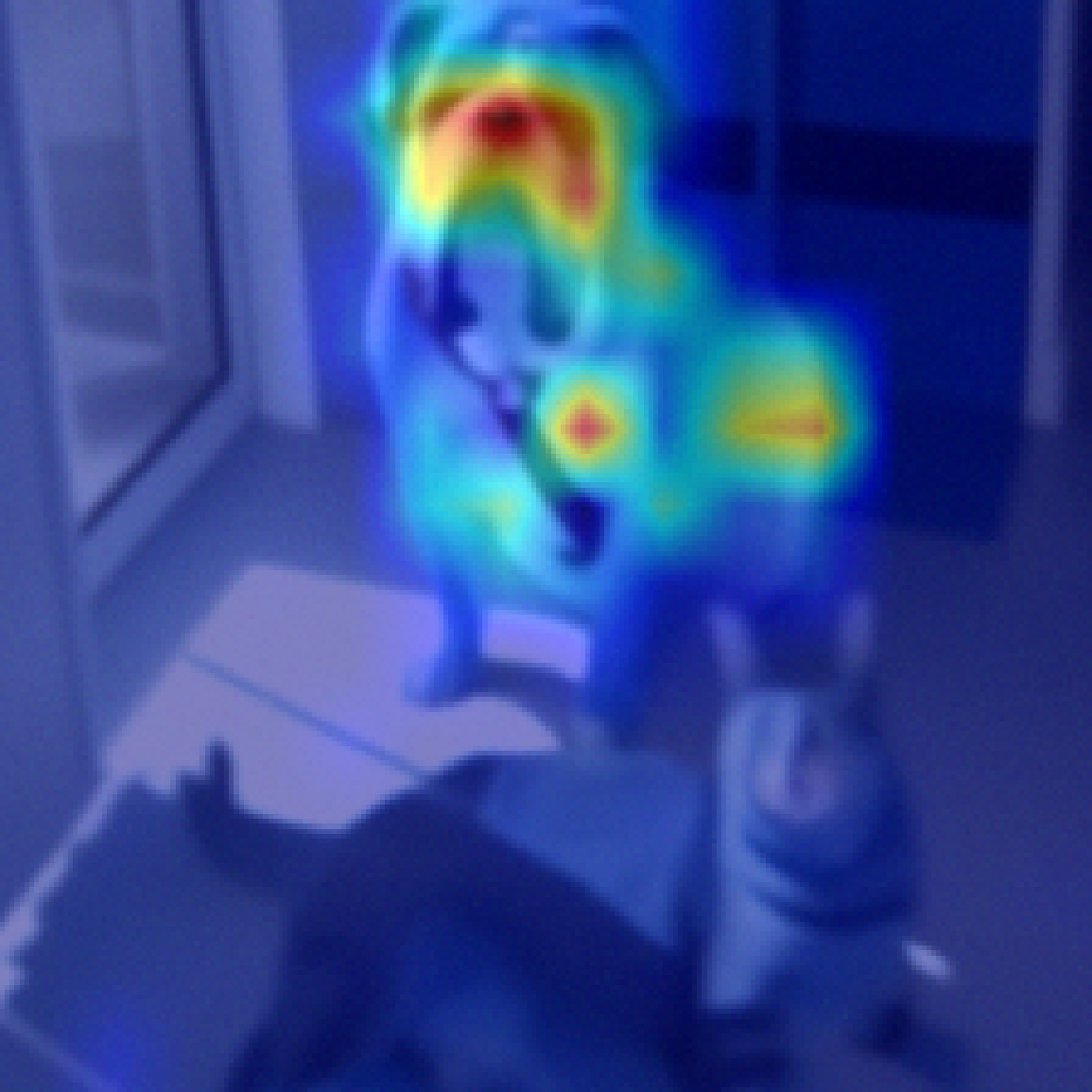}&
	\includegraphics[width=0.15\linewidth]{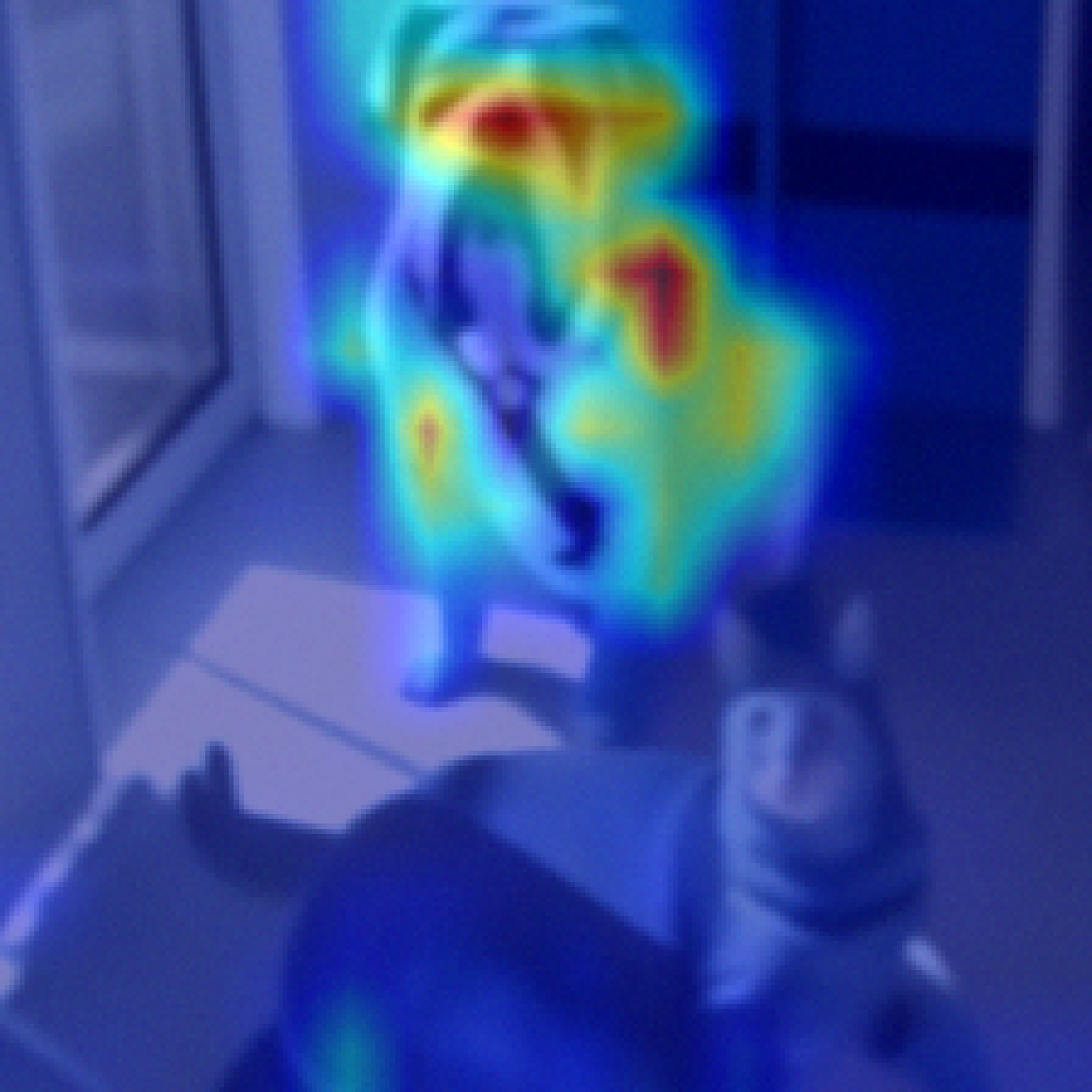}&
	\includegraphics[width=0.15\linewidth]{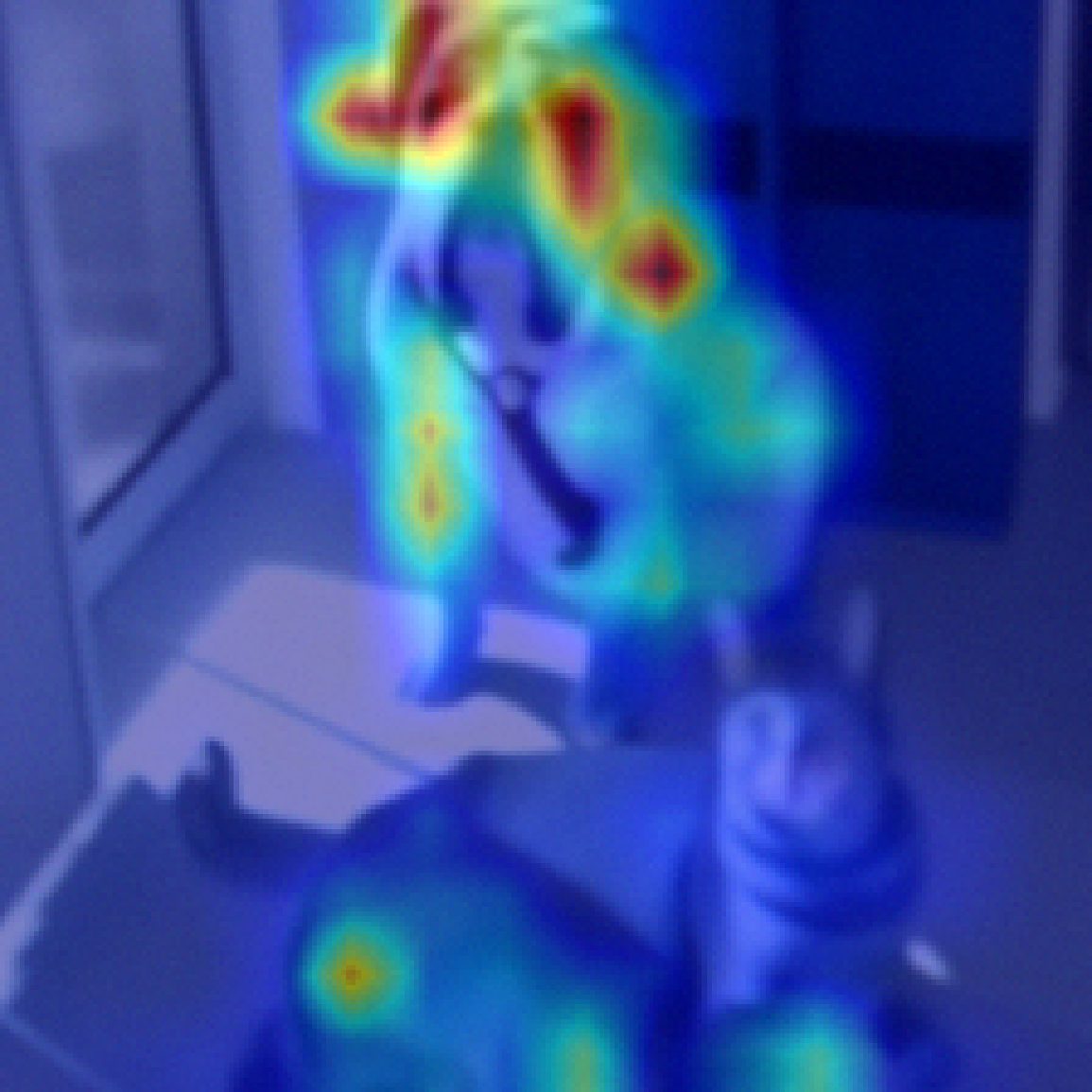}&
	\includegraphics[width=0.15\linewidth]{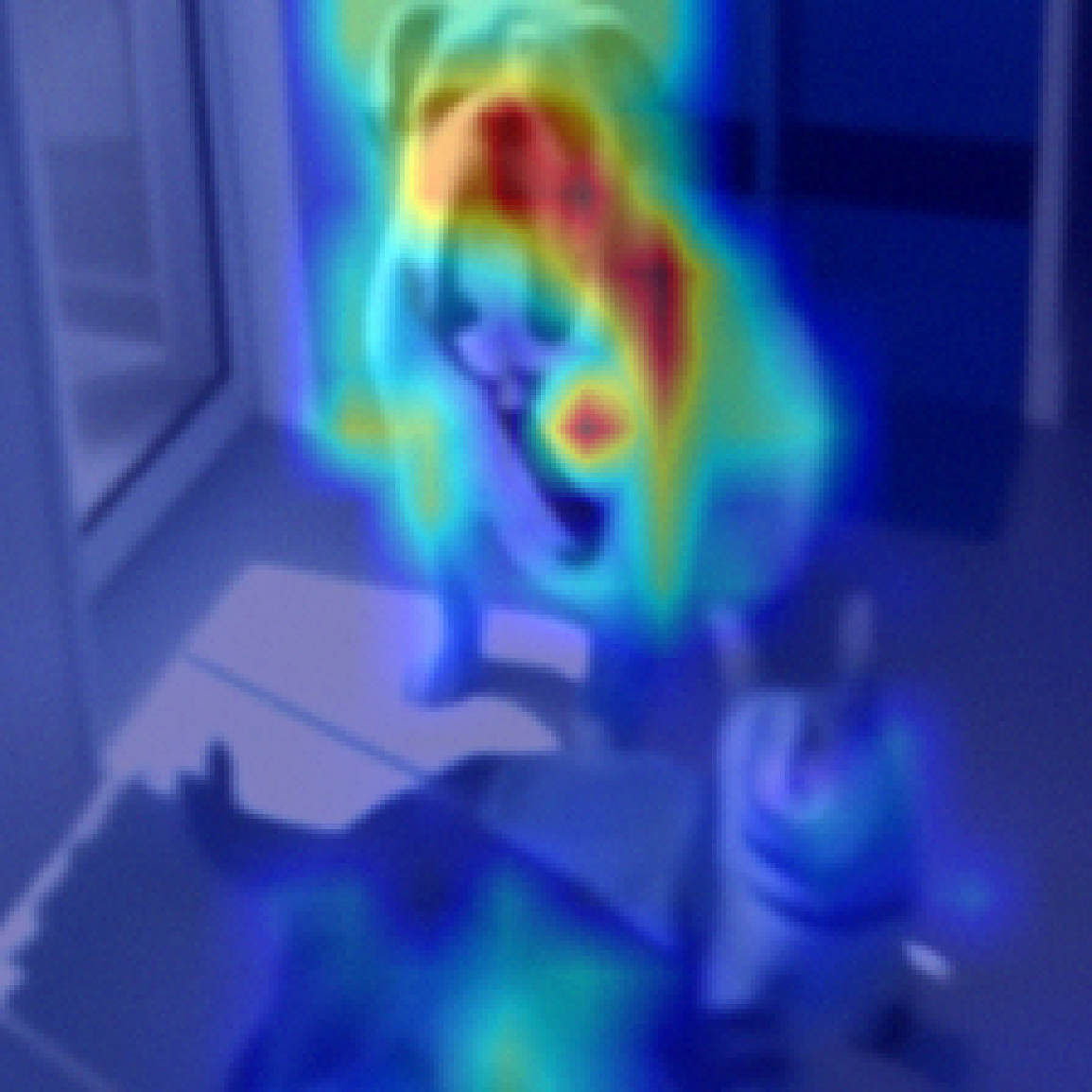}\\
	$S_{Faith}$ & - & 4.60 & 4.45 & 4.22 \\
	\bottomrule
	\end{tabular}
    }
	\caption{Faithfulness score $S_{Faith}$ and visualization results of FViT under additional adversarial attack methods with $\ell_\infty$-norm threat model and $\rho_u = \frac{8}{255}$. $S_{Faith}$ is reciprocal to the average of the absolute difference between the ground-truth heat map and the predicted one. The higher the score, the more faithful the explanation. }
	\label{tab:diff-att}
  \vspace{-10pt} 
\end{table}

\subsection{Results of $l_2$-norm}
\label{l2:case}
Mathematically, with the same noise level, $l_\infty$-norm ball $B_\infty$ is a superset of the $l_2$-norm ball $B_2$. Thus, it will lead to more powerful attacks. Thus, by showing the effectiveness under $l_\infty$-norm threat model, we can also bound the performance of FViT under $l_2$-norm threat model. We additionally present more results in Tab. \ref{tab:diff-norm}, which demonstrates this statement since the faithfulness score yield from $l_2$-norm threat model is consistently higher than that from $l_\infty$-norm threat model. 

\begin{table}
\small
	\centering
	\begin{tabular}{cccccc} 
	\toprule
	Norm type   & $2/255$ & $4/255$ & $6/255$                                  & $8/255$ & $10/255$                                 \\
	\midrule
	$\ell_2$    & 29.16   & 12.63   & 10.56 & 5.62    & 6.30                                     \\
	$\ell_\infty$ & 3.27    & 11.90   & 8.79                                     & 4.74    & 4.70  \\
	\bottomrule
	\end{tabular}
	\caption{
		Faithfulness score of FViT under $\ell_{\infty}$ and $\ell_2$-norm. 
	}
	\label{tab:diff-norm}
\end{table}

\subsection{Ablation Study} As shown in Table \ref{tab: full-result}, it suggests that our method outperforms all other baselines on all three datasets under adversarial attacks with different budgets. In particular, On the ImageNet dataset, the ViT model with our method has the highest pixel accuracy at 64\%, while the DeiT model with our method had the highest mIoU at 46\%. On the Cityscape dataset, the ViT model with Ours had the highest mIoU at 59\%. On the COCO dataset, the ViT model with Ours had the highest pixel accuracy at 74\% and the highest mIoU at 76\%. Moreover, we visualized the ablated version of our method in Figure \ref{fig: abl-vis}. Overall, our method consistently outperforms all other methods, indicating its superiority in accuracy and robustness.



\begin{figure*}[htbp]
    \setlength{\tabcolsep}{1pt} 
    \renewcommand{\arraystretch}{1} 
    \begin{center}
    \begin{tabular*}{\linewidth}
{@{\extracolsep{\fill}}cccccccc}
    Input & Raw Attention & Rollout & GradCAM & LRP & VTA  & Ours \\
    \raisebox{13mm}{\multirow{2}{*}{\makecell*[c]{Cat: clean$\rightarrow$\\\includegraphics[width=0.15\linewidth]{exp/catdog.png}\\ 
    Cat: poisoned$\rightarrow$\\{\scriptsize $7/255$}
    }}
    }
     &
    \includegraphics[width=0.13\linewidth]{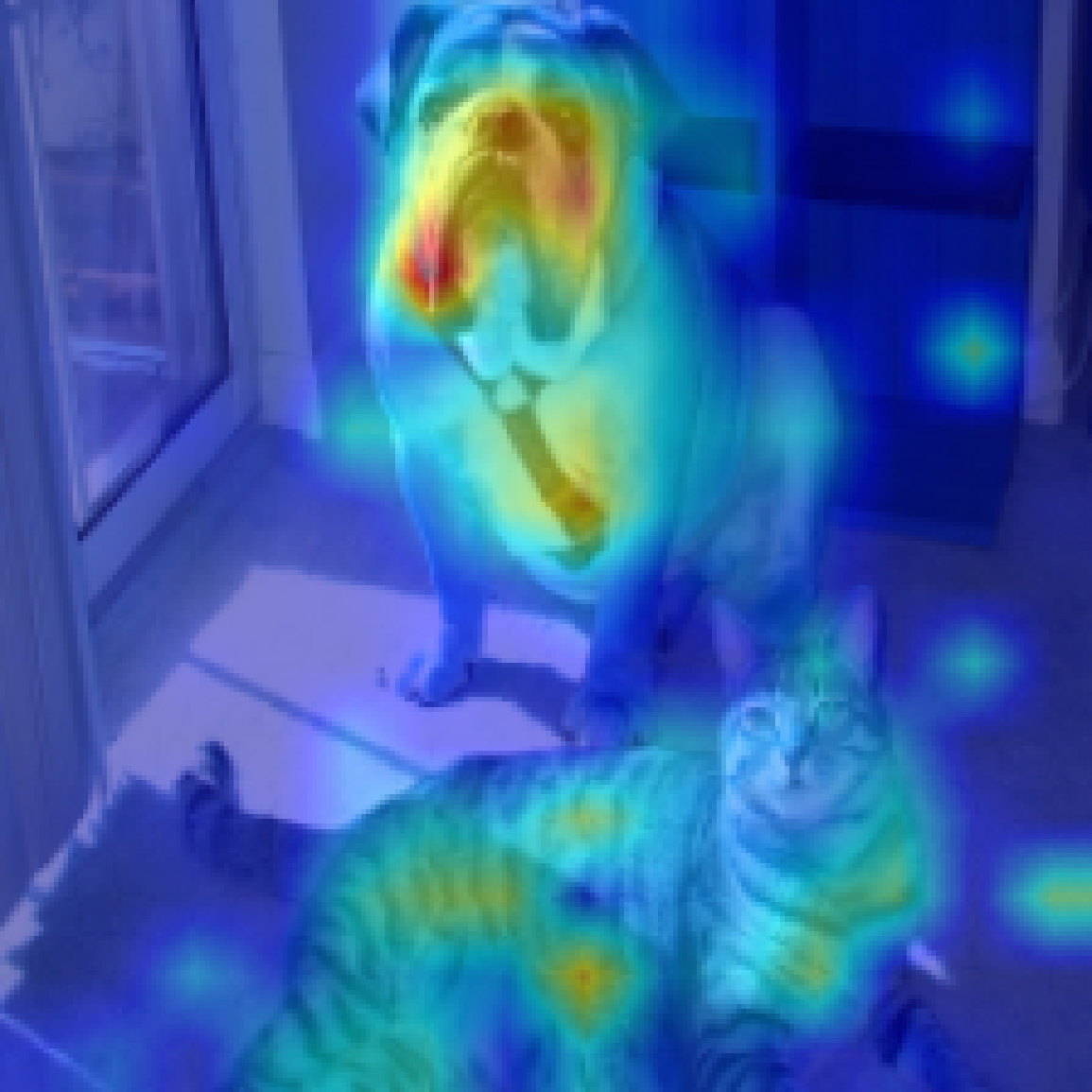} &
    \includegraphics[width=0.13\linewidth]{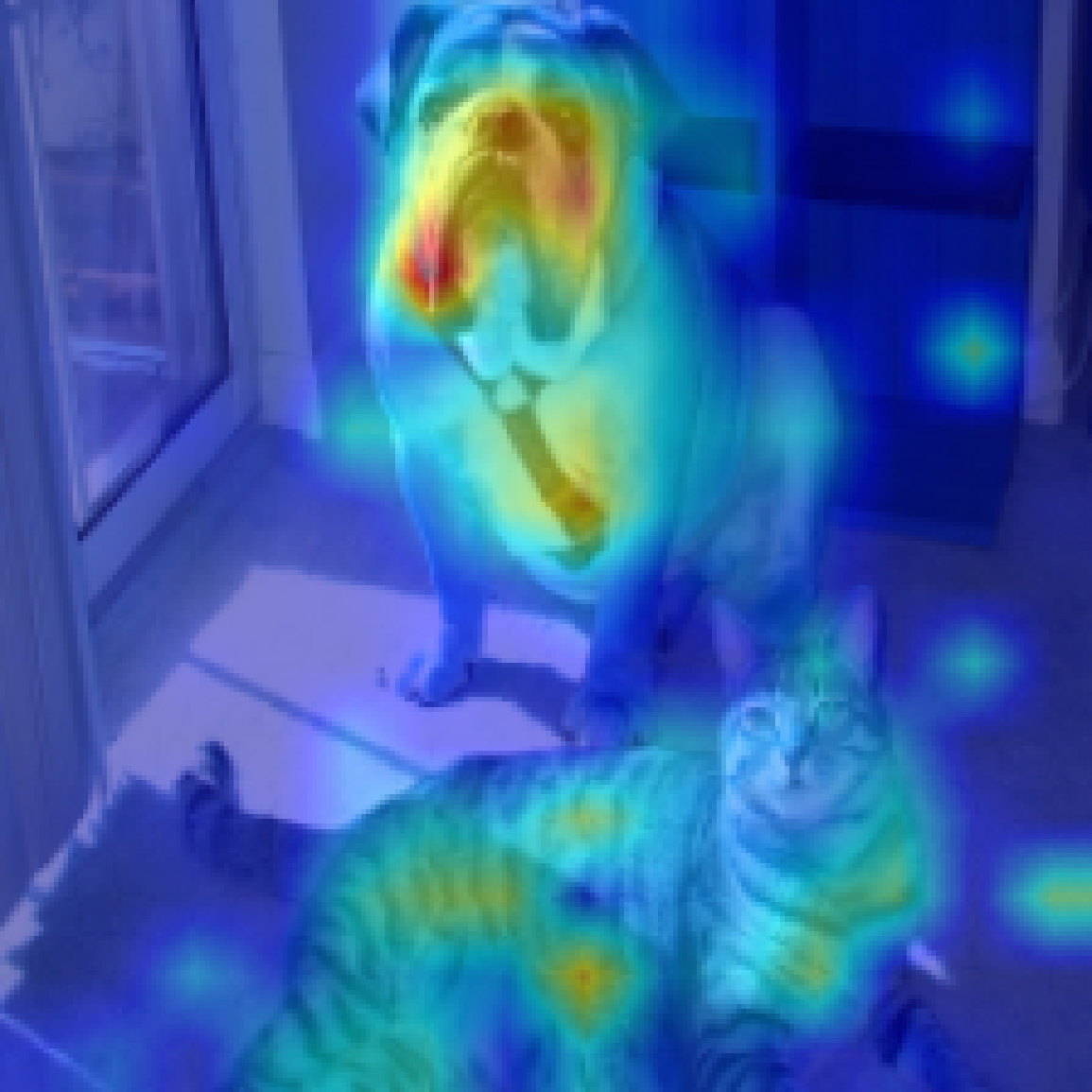} &
    \includegraphics[width=0.13\linewidth]{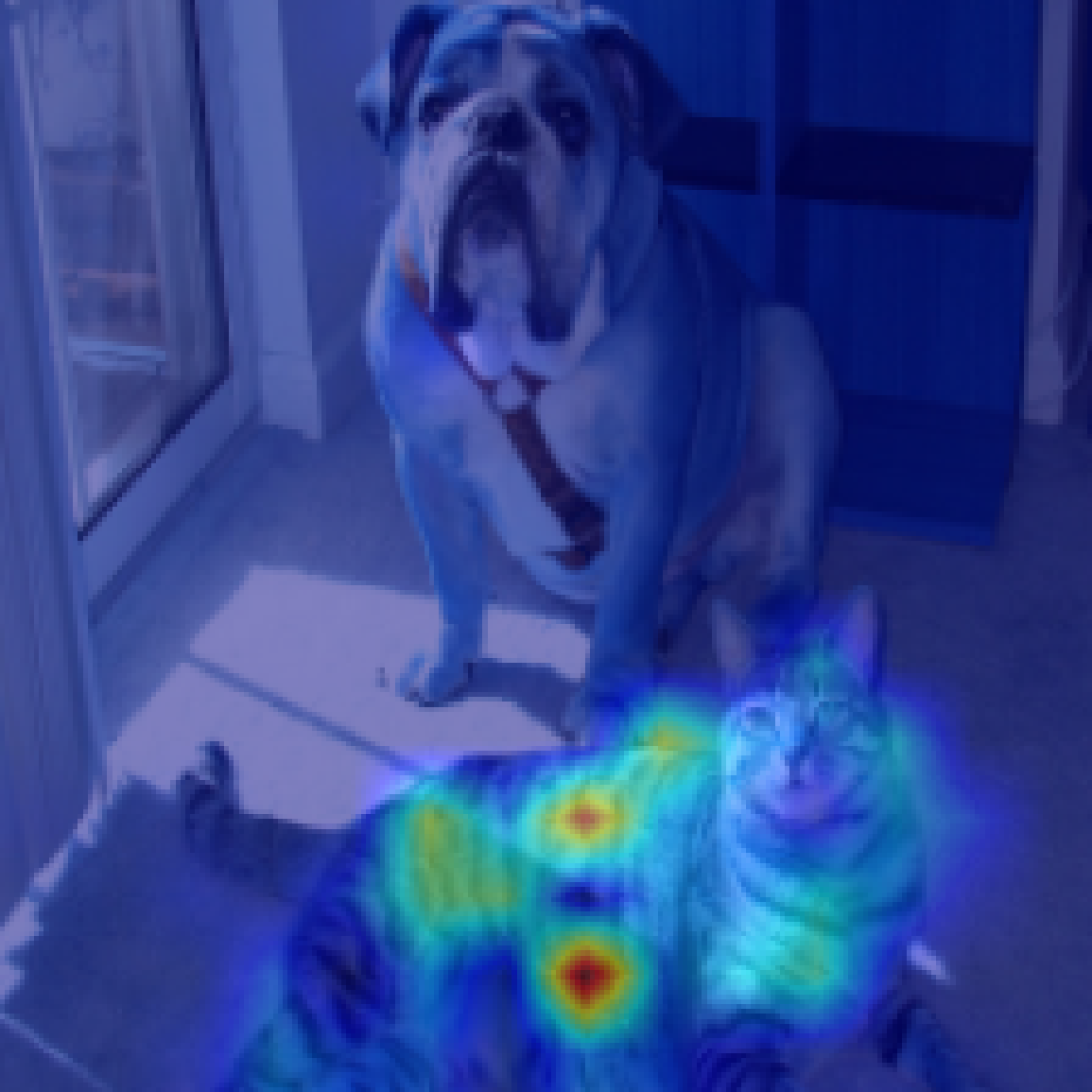} &
    \includegraphics[width=0.13\linewidth]{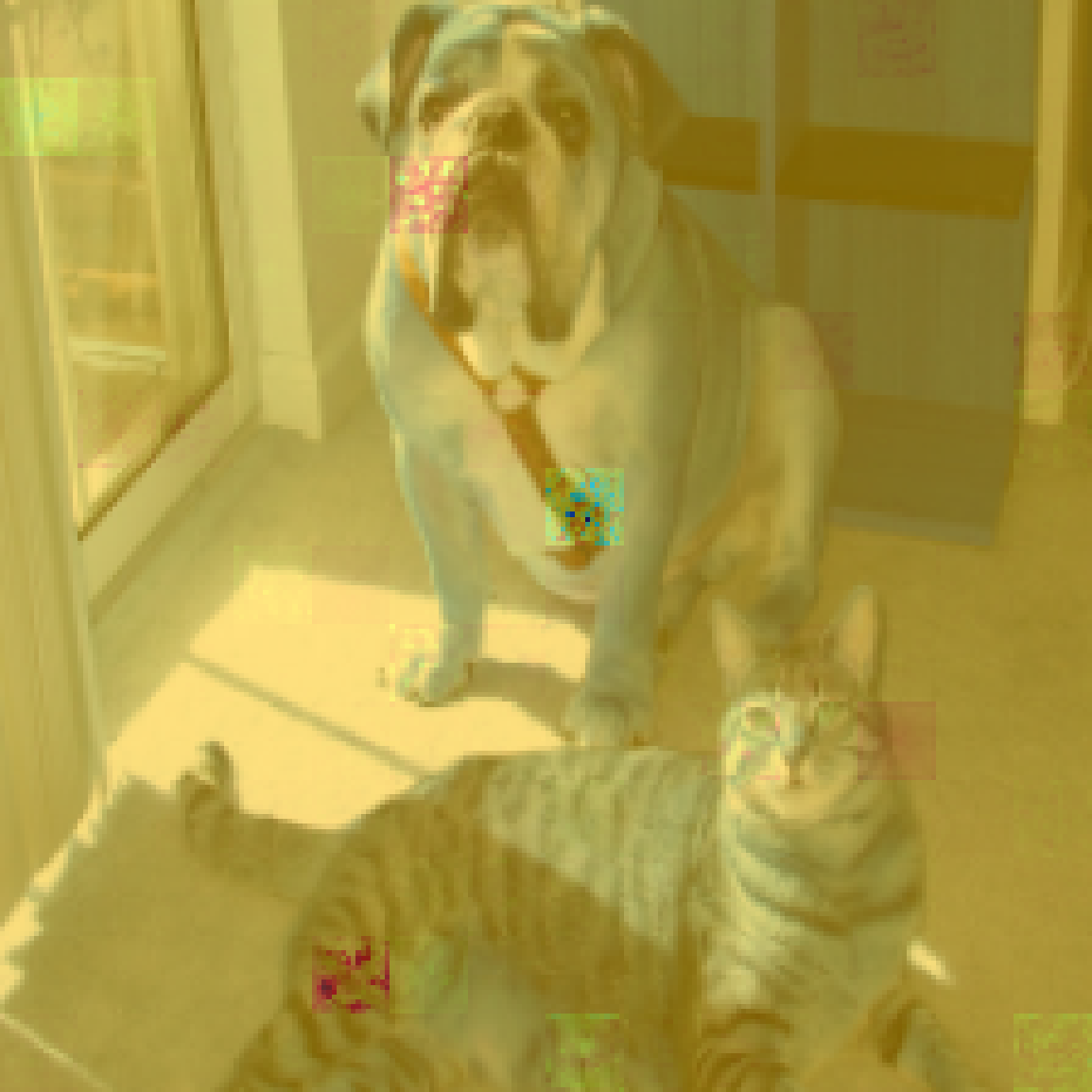} &
    \includegraphics[width=0.13\linewidth]{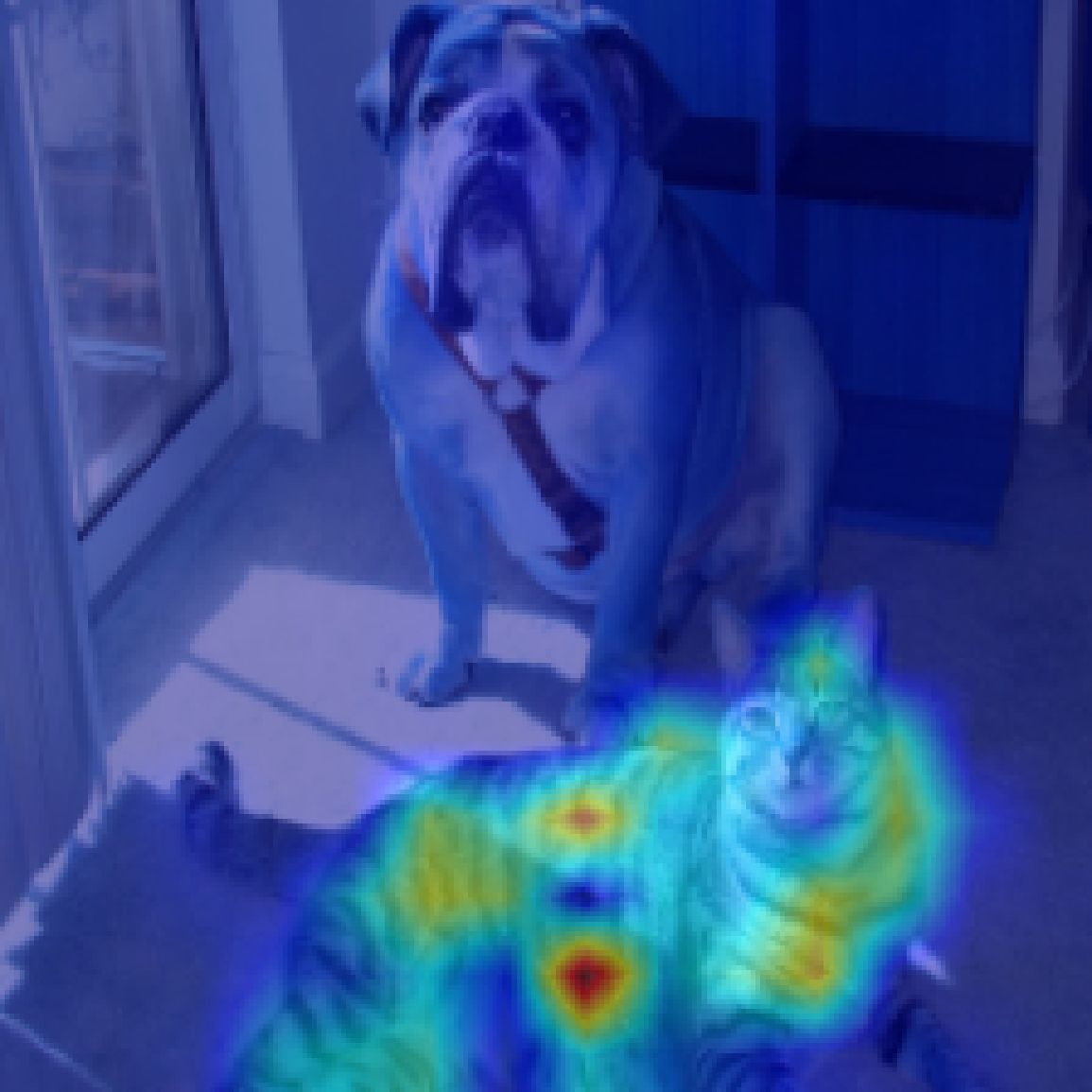} &
    \includegraphics[width=0.13\linewidth]{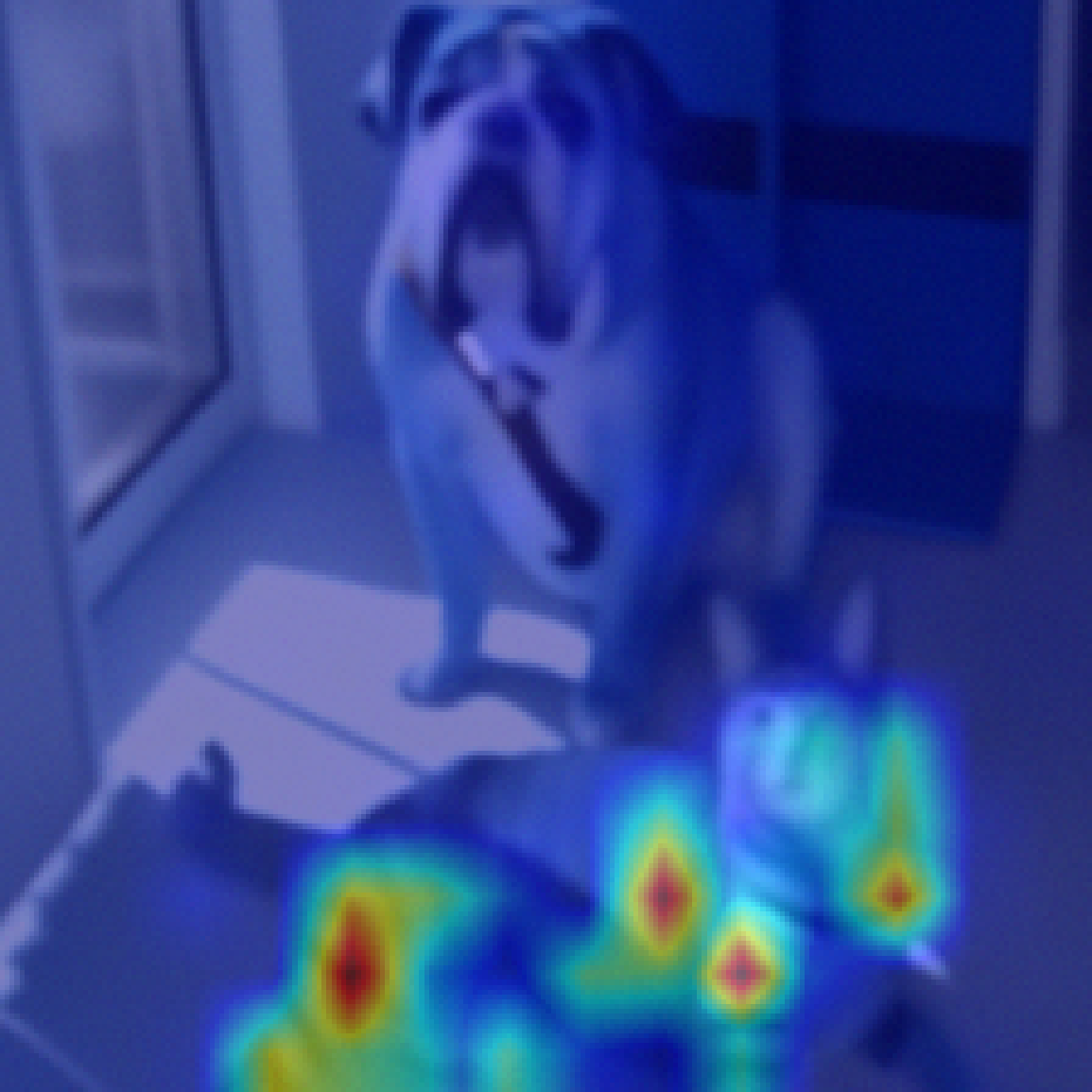}
    \\
    &\includegraphics[width=0.13\linewidth]{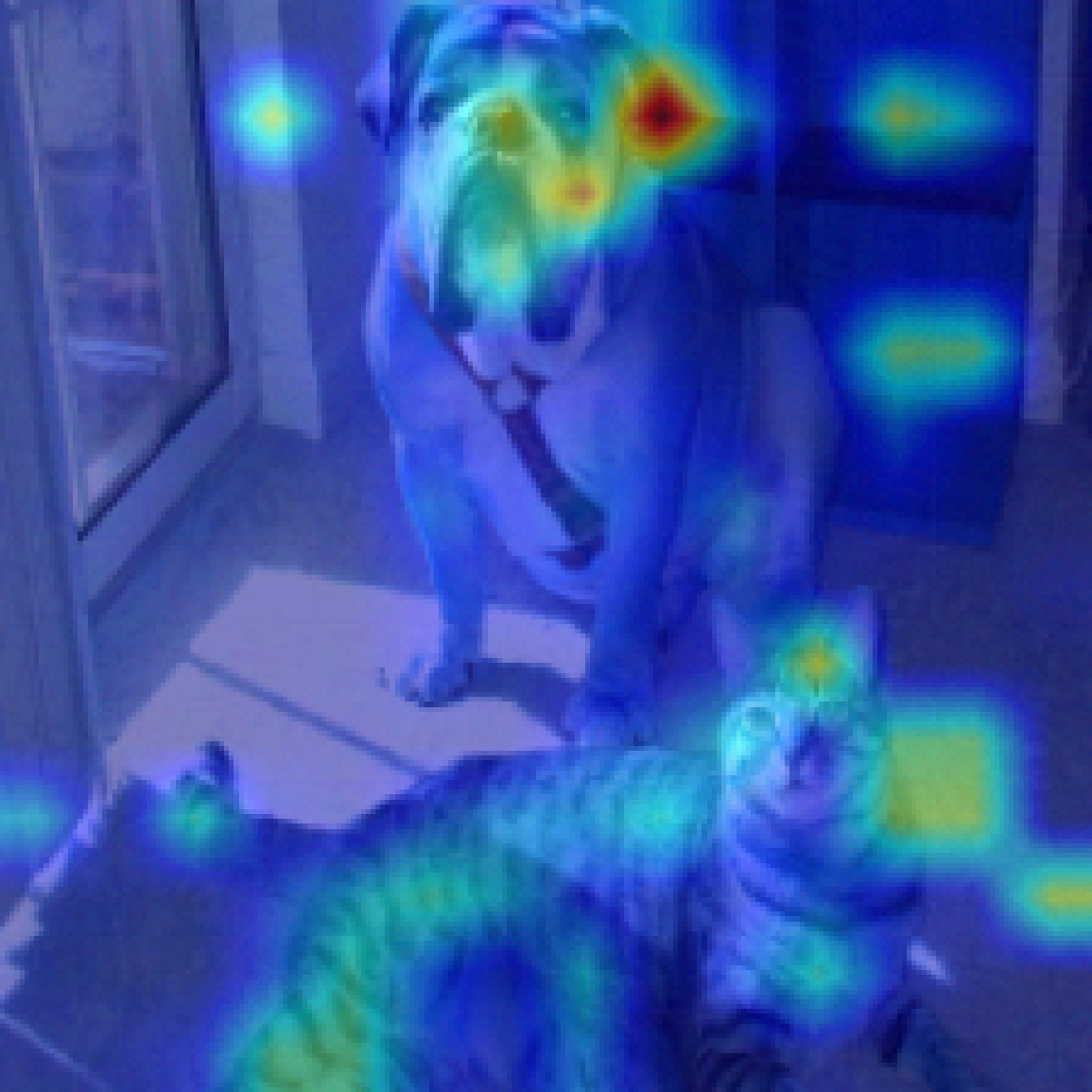} &
    \includegraphics[width=0.13\linewidth]{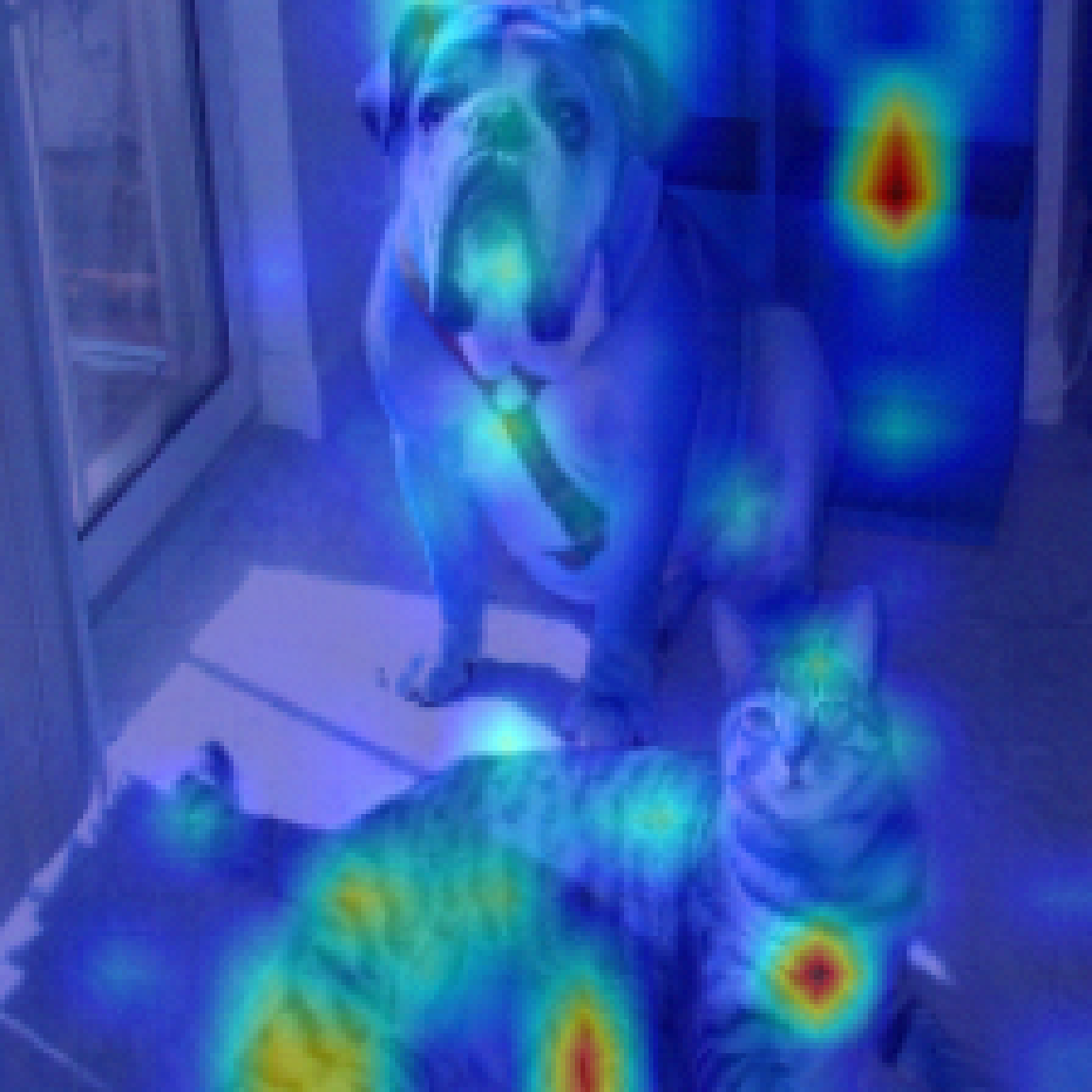} &
    \includegraphics[width=0.13\linewidth]{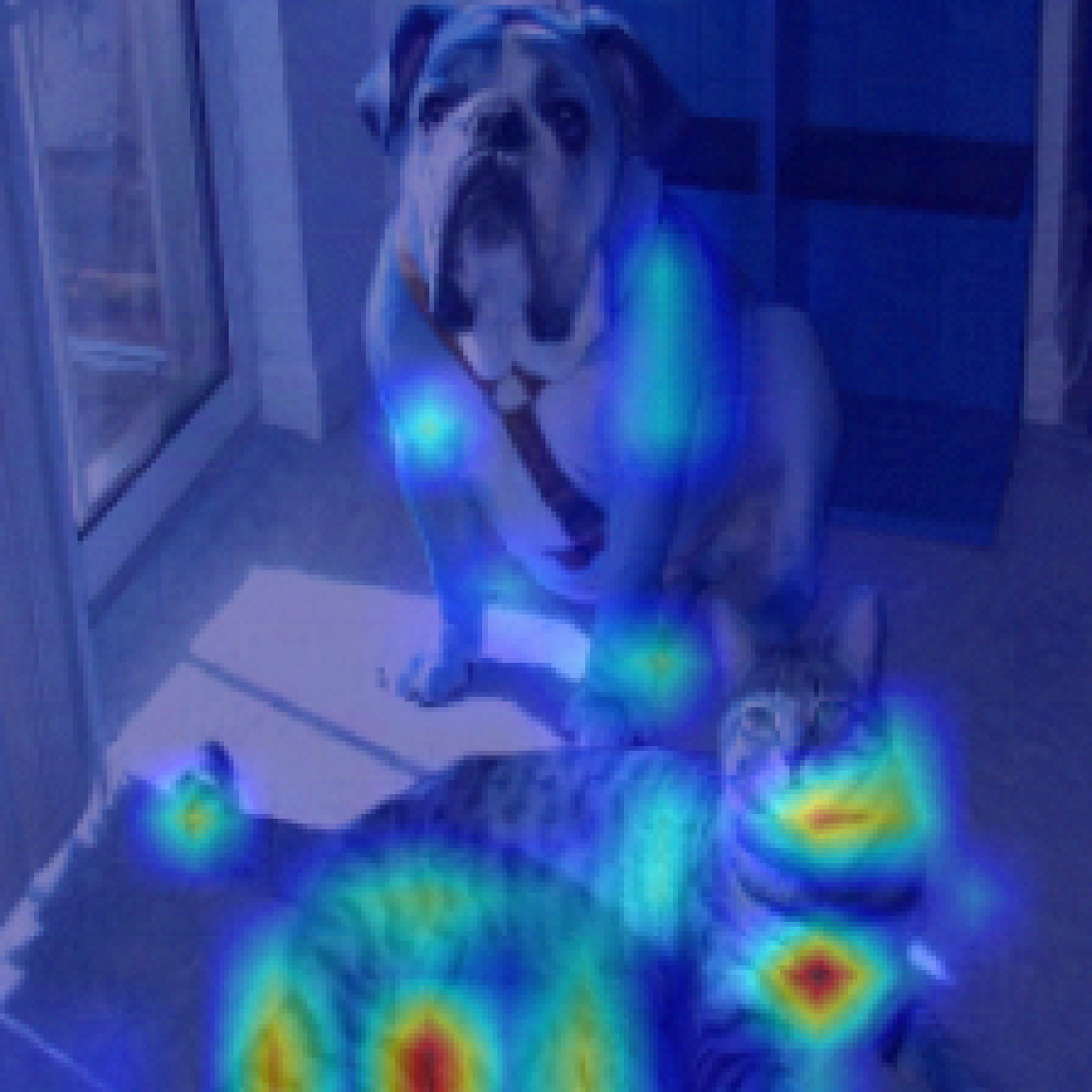} &
    \includegraphics[width=0.13\linewidth]{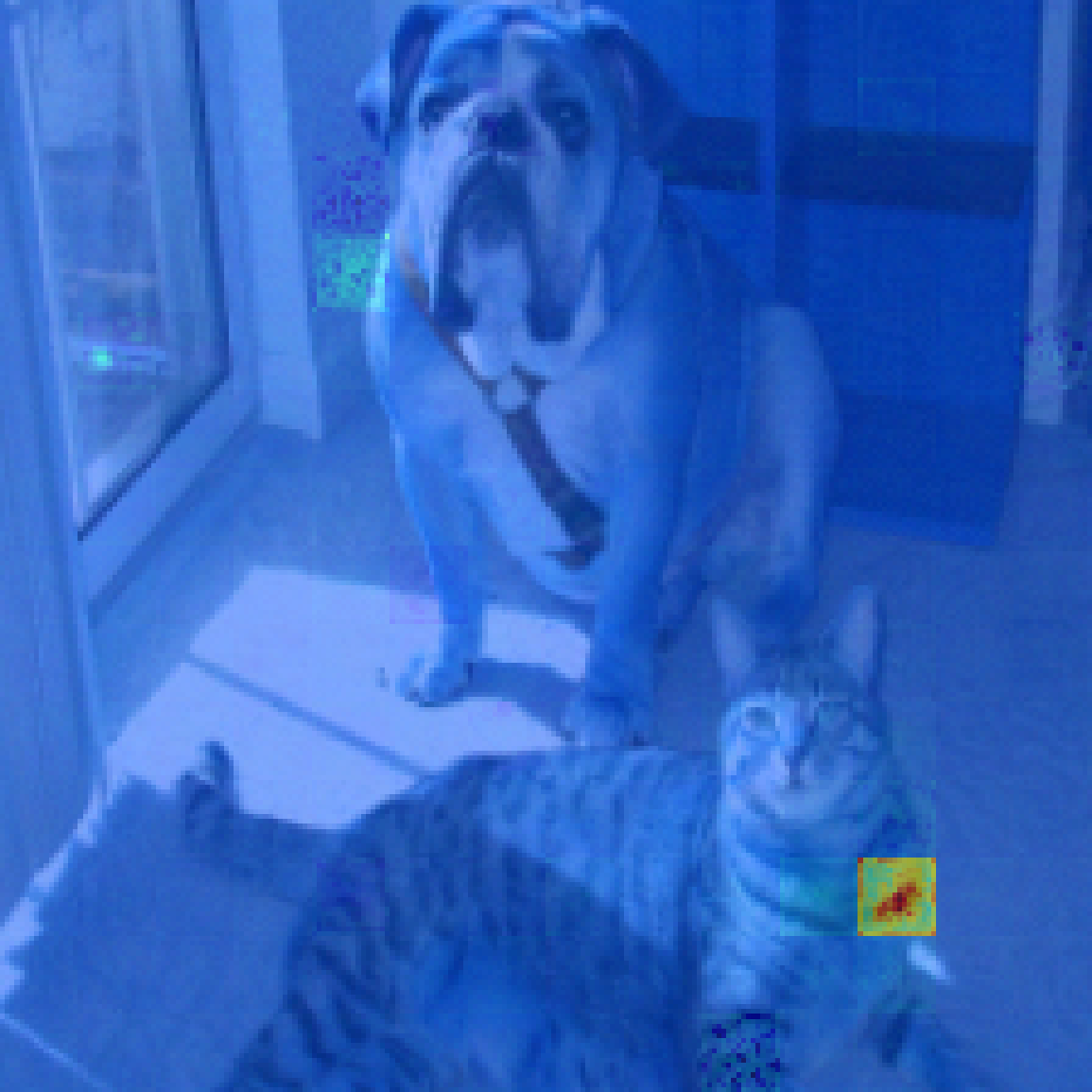} &
    \includegraphics[width=0.13\linewidth]{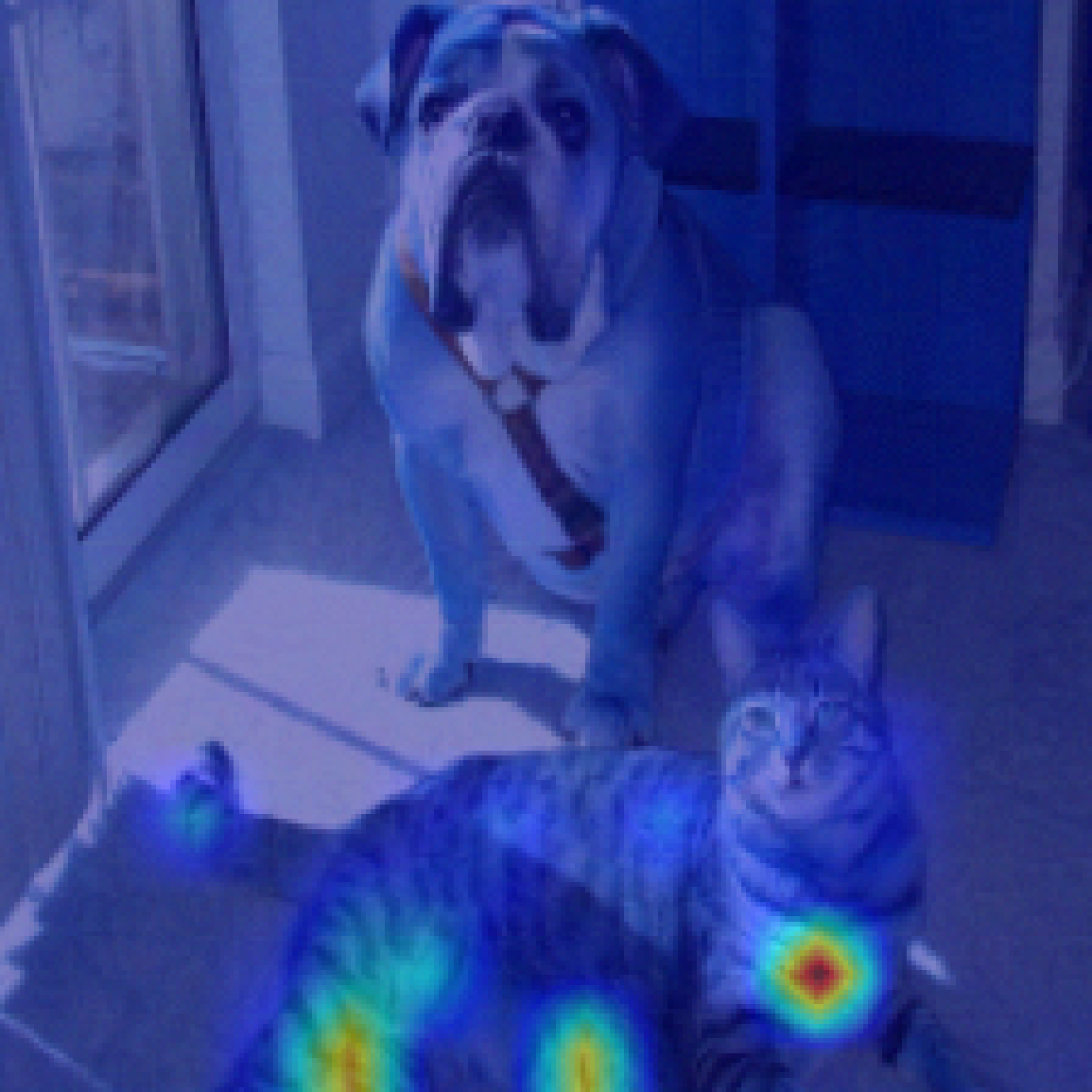} &
    \includegraphics[width=0.13\linewidth]{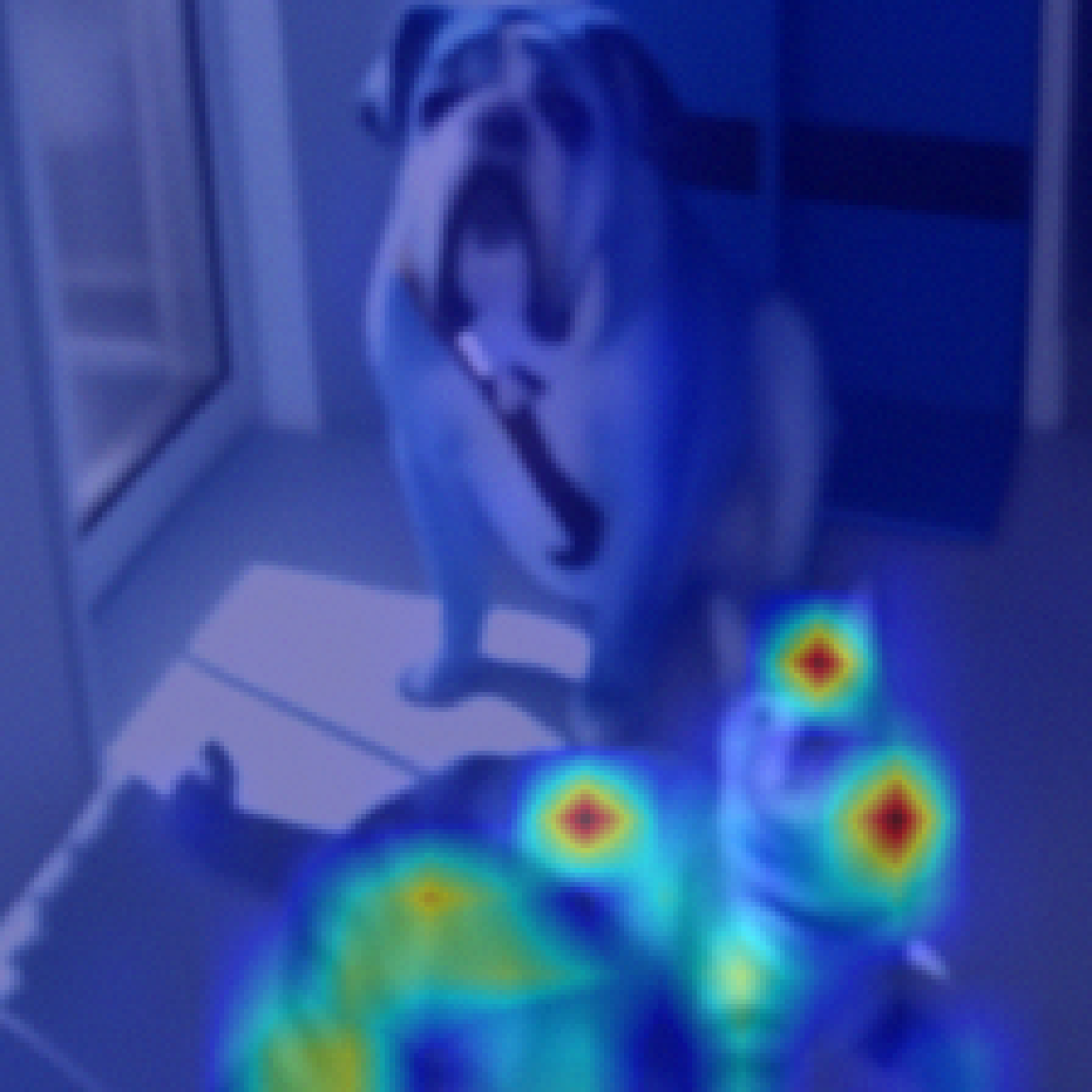}
    \\
        \raisebox{13mm}{\multirow{2}{*}{\makecell*[c]{Dog: clean$\rightarrow$\\\includegraphics[width=0.15\linewidth]{exp/catdog.png}\\ 
    Dog: poisoned$\rightarrow$\\{\scriptsize $7/255$}
    }}
    }
       &
    \includegraphics[width=0.13\linewidth]{exp/CatDog/CatDog-raw_attn-dog-clean.png} &
    \includegraphics[width=0.13\linewidth]{exp/CatDog/CatDog-rollout-dog-clean.png} &
    \includegraphics[width=0.13\linewidth]{exp/CatDog/CatDog-attn_gradcam-dog-clean.png} &
    \includegraphics[width=0.13\linewidth]{exp/CatDog/CatDog-full-LRP-dog-clean.png} &
    \includegraphics[width=0.13\linewidth]{exp/CatDog/CatDog-vta-dog-clean.png} &
    \includegraphics[width=0.13\linewidth]{exp/CatDog/CatDog-ours-dog-clean.png}
    \\
    &\includegraphics[width=0.13\linewidth]{exp/CatDog/CatDog-raw_attn-dog-perturbed-7.png} &
    \includegraphics[width=0.13\linewidth]{exp/CatDog/CatDog-rollout-dog-perturbed-7.png} &
    \includegraphics[width=0.13\linewidth]{exp/CatDog/CatDog-attn_gradcam-dog-perturbed-7.png} &
    \includegraphics[width=0.13\linewidth]{exp/CatDog/CatDog-full-LRP-dog-perturbed-7.png} &
    \includegraphics[width=0.13\linewidth]{exp/CatDog/CatDog-vta-dog-perturbed-7.png} &
    \includegraphics[width=0.13\linewidth]{exp/CatDog/CatDog-ours-dog-perturbed-7.png}
    \end{tabular*}
    \caption{
    Class-specific explanation heat map visualizations under adversarial corruption. For each image, we present results for two different classes. Other baselines either give inconsistent interpretations or show wrong focus class regions under adversarial perturbations. While our method gives a consistent interpretation map and is robust against adversarial attacks.
    }
    \label{fig: full_example}
    \end{center}
\vspace{-0.2in}
\end{figure*}
\begin{table*}[bthp]
\centering
\resizebox{\linewidth}{!}{
\begin{tabular}{ccccccccccccccc} 
\toprule
\multirow{2}{*}{\begin{tabular}[c]{@{}c@{}}Noise\\Radius\end{tabular}} & \multirow{2}{*}{Model} & \multirow{2}{*}{Method} & \multicolumn{4}{c}{ImageNet}                  & \multicolumn{4}{c}{Cityscape}                 & \multicolumn{4}{c}{COCO}                       \\ 
\cmidrule(lr){4-7}\cmidrule(lr){8-11}\cmidrule(lr){12-15}
&                        &                         & Cla. Acc. (\%) & Pix. Acc. (\%) & mIoU & mAP  & Cla. Acc. (\%) & Pix. Acc. (\%) & mIoU & mAP  & Cla. Acc. (\%) & Pix. Acc. (\%) & mIoU & mAP   \\ 
\midrule 
\multirow{18}{*}{$\rho_u=9/255$}                            & \multirow{6}{*}{VIT}   & Raw Attention           & 0.6            & 0.5            & 0.42 & 0.77 & 0.66           & 0.62           & 0.46 & 0.75 & 0.78           & 0.65           & 0.63 & 0.81  \\
&                        & Rollout                 & 0.72           & 0.56           & 0.42 & 0.76 & 0.79           & 0.55           & 0.51 & 0.76 & 0.82           & 0.64           & 0.53 & 0.85  \\
&                        & GradCAM                 & 0.64           & 0.49           & 0.5  & 0.78 & 0.79           & 0.64           & 0.5  & 0.75 & 0.74           & 0.78           & 0.67 & 0.91  \\
 &                        & LRP                     & 0.7            & 0.54           & 0.43 & 0.78 & 0.68           & 0.68           & 0.52 & 0.86 & 0.82           & 0.69           & 0.7  & 0.81  \\
 &                        & VTA                     & 0.64           & 0.56           & 0.43 & 0.77 & 0.7            & 0.74           & 0.58 & 0.89 & 0.8            & 0.82           & 0.67 & 0.88  \\
 \rowcolor{pink!20}
 &                        & Ours                    & 0.69           & 0.64           & 0.48 & 0.73 & 0.74           & 0.71           & 0.59 & 0.9  & 0.88           & 0.74           & 0.76 & 0.97  \\ 
\cmidrule{2-15} 
& \multirow{6}{*}{DeiT}  & Raw Attention           & 0.63           & 0.6            & 0.42 & 0.68 & 0.75           & 0.56           & 0.49 & 0.71 & 0.83           & 0.64           & 0.57 & 0.79  \\
&                        & Rollout                 & 0.7            & 0.57           & 0.38 & 0.66 & 0.69           & 0.64           & 0.53 & 0.76 & 0.82           & 0.76           & 0.57 & 0.76  \\
&                        & GradCAM                 & 0.67           & 0.52           & 0.46 & 0.81 & 0.81           & 0.66           & 0.55 & 0.83 & 0.76           & 0.71           & 0.57 & 0.91  \\
&                        & LRP                     & 0.63           & 0.63           & 0.46 & 0.78 & 0.81           & 0.66           & 0.63 & 0.79 & 0.78           & 0.73           & 0.66 & 0.86  \\
 &                        & VTA                     & 0.77           & 0.68           & 0.54 & 0.72 & 0.71           & 0.71           & 0.58 & 0.79 & 0.82           & 0.71           & 0.61 & 0.86  \\
 \rowcolor{pink!20}
 &                        & Ours                    & 0.8            & 0.64           & 0.5  & 0.81 & 0.83           & 0.72           & 0.67 & 0.83 & 0.81           & 0.72           & 0.72 & 0.84  \\ 
\cmidrule{2-15} 
 & \multirow{6}{*}{Swin}  & Raw Attention           & 0.65           & 0.49           & 0.41 & 0.76 & 0.69           & 0.67           & 0.44 & 0.79 & 0.85           & 0.72           & 0.58 & 0.82  \\
 &                        & Rollout                 & 0.62           & 0.56           & 0.4  & 0.81 & 0.72           & 0.56           & 0.52 & 0.84 & 0.76           & 0.69           & 0.62 & 0.81  \\
 &                        & GradCAM                 & 0.77           & 0.58           & 0.4  & 0.8  & 0.68           & 0.68           & 0.51 & 0.88 & 0.87           & 0.76           & 0.65 & 0.86  \\
 &                        & LRP                     & 0.67           & 0.58           & 0.48 & 0.74 & 0.72           & 0.63           & 0.52 & 0.84 & 0.84           & 0.79           & 0.66 & 0.82  \\
 &                        & VTA                     & 0.7            & 0.59           & 0.56 & 0.83 & 0.85           & 0.65           & 0.62 & 0.88 & 0.87           & 0.74           & 0.72 & 0.83  \\
 \rowcolor{pink!20}
 &                        & Ours                    & 0.78           & 0.63           & 0.49 & 0.9  & 0.74           & 0.78           & 0.56 & 0.81 & 0.86           & 0.74           & 0.75 & 0.93  \\ 
\midrule 
\\
\midrule 
\multirow{18}{*}{$\rho_u=10/255$}                            & \multirow{6}{*}{VIT}   & Raw Attention           & 0.57           & 0.38           & 0.34 & 0.53 & 0.57           & 0.46           & 0.38 & 0.67 & 0.74           & 0.62           & 0.45 & 0.66  \\
 &                        & Rollout                 & 0.58           & 0.43           & 0.38 & 0.67 & 0.6            & 0.55           & 0.36 & 0.7  & 0.66           & 0.6            & 0.48 & 0.69  \\
 &                        & GradCAM                 & 0.52           & 0.39           & 0.3  & 0.6  & 0.57           & 0.57           & 0.46 & 0.63 & 0.65           & 0.63           & 0.45 & 0.8   \\
  &                        & LRP                     & 0.62           & 0.5            & 0.38 & 0.6  & 0.66           & 0.57           & 0.51 & 0.77 & 0.65           & 0.57           & 0.55 & 0.71  \\
 &                        & VTA                     & 0.52           & 0.5            & 0.34 & 0.68 & 0.63           & 0.6            & 0.46 & 0.75 & 0.7            & 0.69           & 0.57 & 0.75  \\
 \rowcolor{pink!20}
&                        & Ours                    & 0.59           & 0.54           & 0.36 & 0.78 & 0.74           & 0.59           & 0.45 & 0.72 & 0.76           & 0.73           & 0.63 & 0.74  \\ 
\cmidrule{2-15} 
 & \multirow{6}{*}{DeiT}  & Raw Attention           & 0.53           & 0.46           & 0.25 & 0.62 & 0.55           & 0.5            & 0.39 & 0.6  & 0.68           & 0.63           & 0.44 & 0.75  \\
 &                        & Rollout                 & 0.6            & 0.42           & 0.33 & 0.66 & 0.67           & 0.5            & 0.44 & 0.63 & 0.73           & 0.55           & 0.47 & 0.72  \\
 &                        & GradCAM                 & 0.64           & 0.52           & 0.37 & 0.69 & 0.63           & 0.54           & 0.46 & 0.75 & 0.67           & 0.58           & 0.54 & 0.72  \\
 &                        & LRP                     & 0.65           & 0.55           & 0.4  & 0.73 & 0.58           & 0.51           & 0.53 & 0.66 & 0.74           & 0.65           & 0.52 & 0.74  \\
 &                        & VTA                     & 0.58           & 0.54           & 0.42 & 0.66 & 0.64           & 0.61           & 0.44 & 0.68 & 0.69           & 0.59           & 0.5  & 0.75  \\
\rowcolor{pink!20}
&                       & Ours                    & 0.65           & 0.48           & 0.47 & 0.73 & 0.69           & 0.66           & 0.49 & 0.85 & 0.76           & 0.72           & 0.51 & 0.85  \\ 
\cmidrule{2-15} 
 & \multirow{6}{*}{Swin}  & Raw Attention           & 0.6            & 0.41           & 0.3  & 0.55 & 0.7            & 0.54           & 0.37 & 0.64 & 0.63           & 0.65           & 0.47 & 0.72  \\
  &                        & Rollout                 & 0.57           & 0.43           & 0.36 & 0.66 & 0.63           & 0.51           & 0.47 & 0.64 & 0.64           & 0.66           & 0.49 & 0.7   \\
 &                        & GradCAM                 & 0.6            & 0.47           & 0.31 & 0.58 & 0.58           & 0.54           & 0.5  & 0.67 & 0.7            & 0.68           & 0.48 & 0.76  \\
 &                        & LRP                     & 0.53           & 0.53           & 0.32 & 0.63 & 0.67           & 0.6            & 0.52 & 0.69 & 0.72           & 0.64           & 0.59 & 0.84  \\
 &                        & VTA                     & 0.59           & 0.6            & 0.48 & 0.68 & 0.75           & 0.54           & 0.51 & 0.73 & 0.71           & 0.61           & 0.63 & 0.74  \\
\rowcolor{pink!20}
&                        & Ours                    & 0.63           & 0.54           & 0.41 & 0.77 & 0.78           & 0.63           & 0.49 & 0.79 & 0.7            & 0.65           & 0.62 & 0.83  \\
\bottomrule
\end{tabular}
}
\caption{Comparison of different explanation methods on multiple datasets using VIT, DeiT, and Swin models}
\label{tab: full-result}
\end{table*}

\begin{figure*}[htbp]
    \setlength{\tabcolsep}{1pt} 
    \renewcommand{\arraystretch}{1} 
    \begin{center}
    \begin{tabular*}{\linewidth}{@{\extracolsep{\fill}}ccccccc}
     Corrupted Input & Raw Attention & Rollout 
     & GradCAM &LRP& VTA  & Ours \\
    \includegraphics[width=0.12\linewidth]{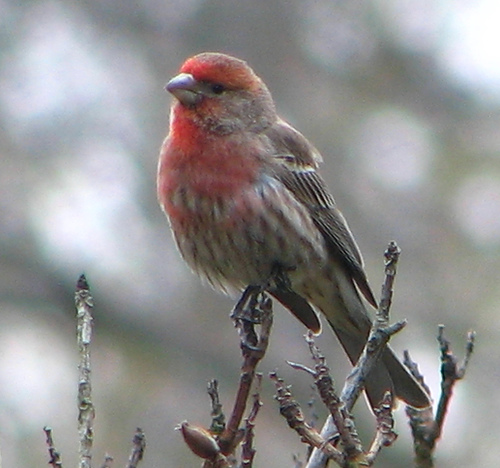} &
    \includegraphics[width=0.12\linewidth]{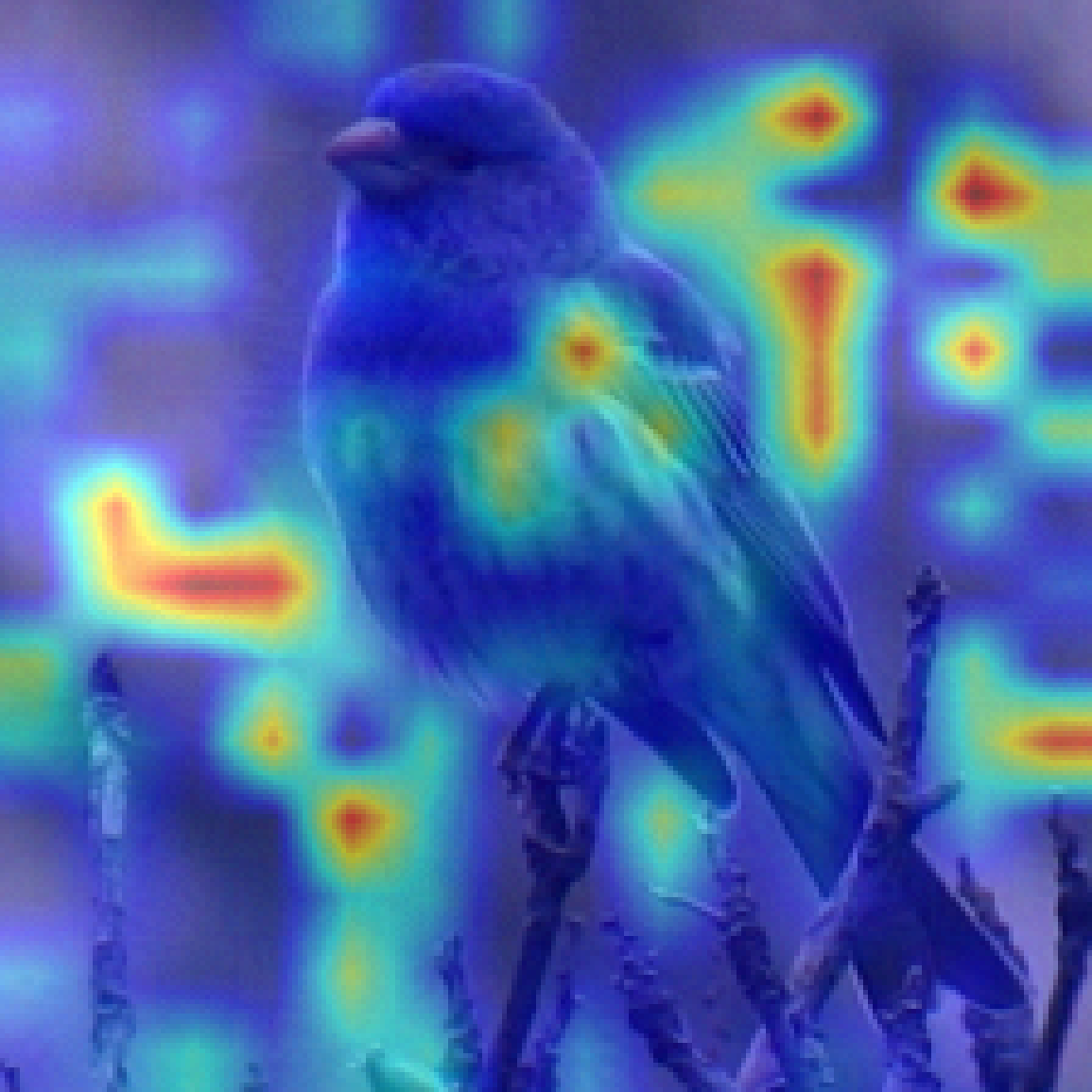} &
    \includegraphics[width=0.12\linewidth]{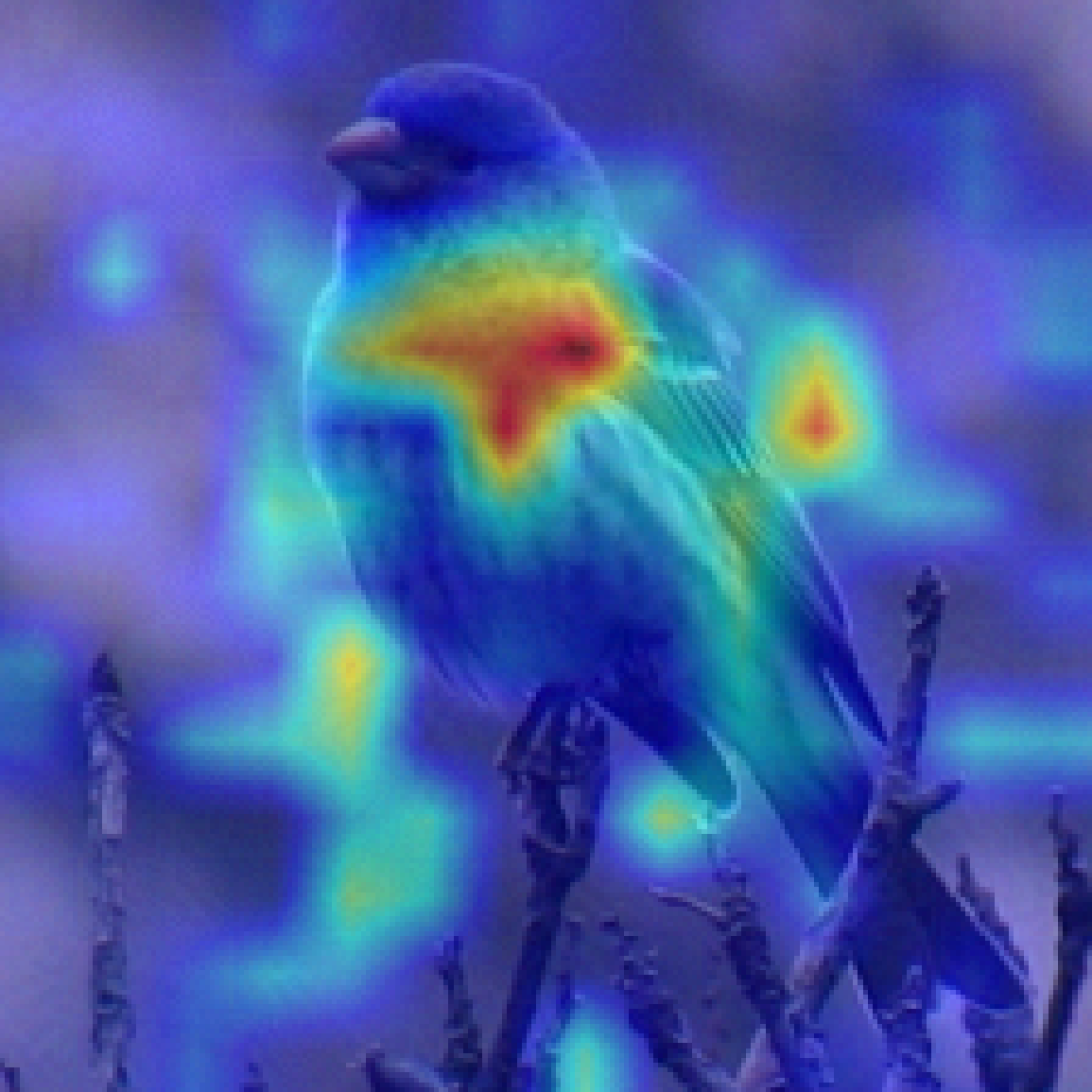} &
    \includegraphics[width=0.12\linewidth]{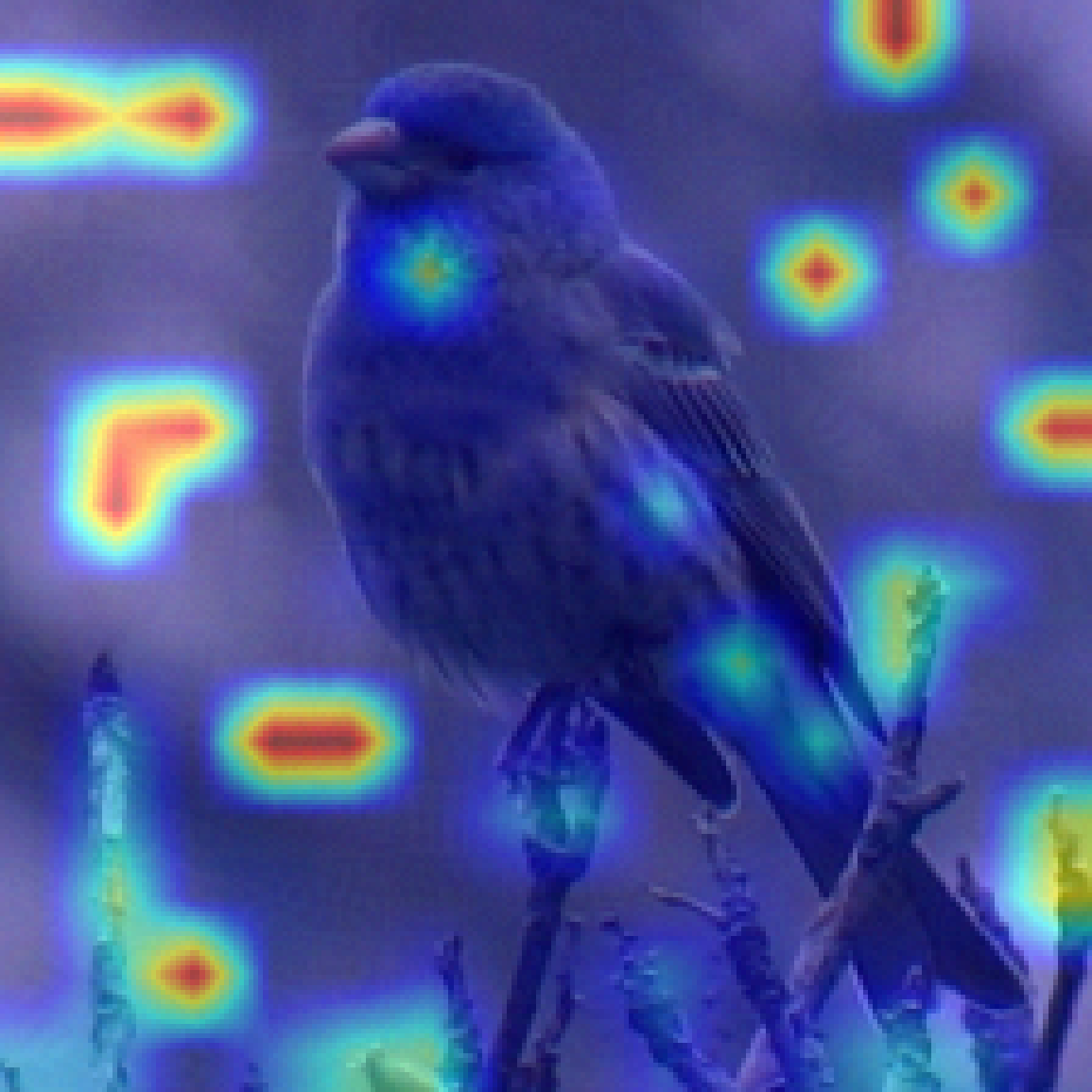} & \includegraphics[width=0.12\linewidth]{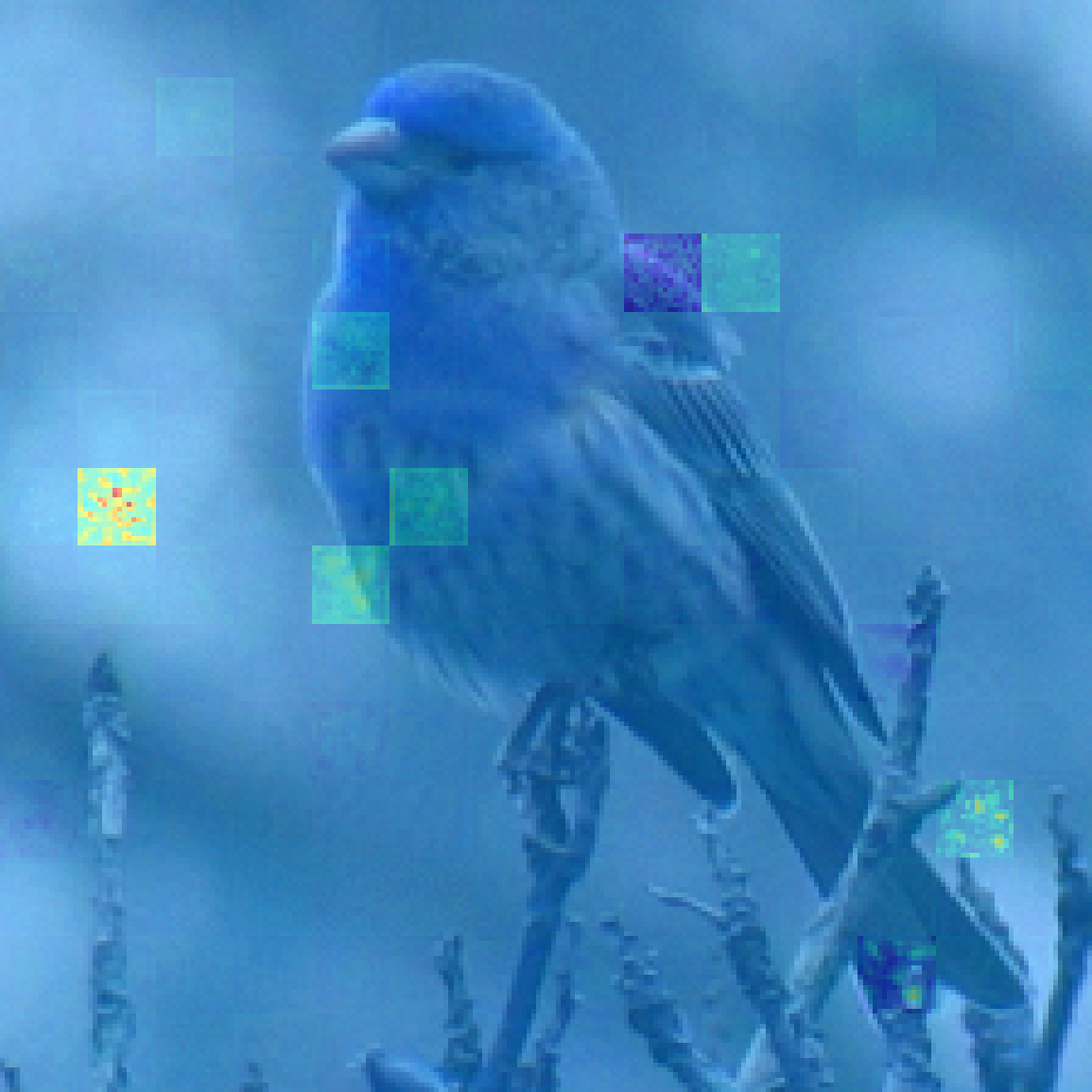} &\includegraphics[width=0.12\linewidth]{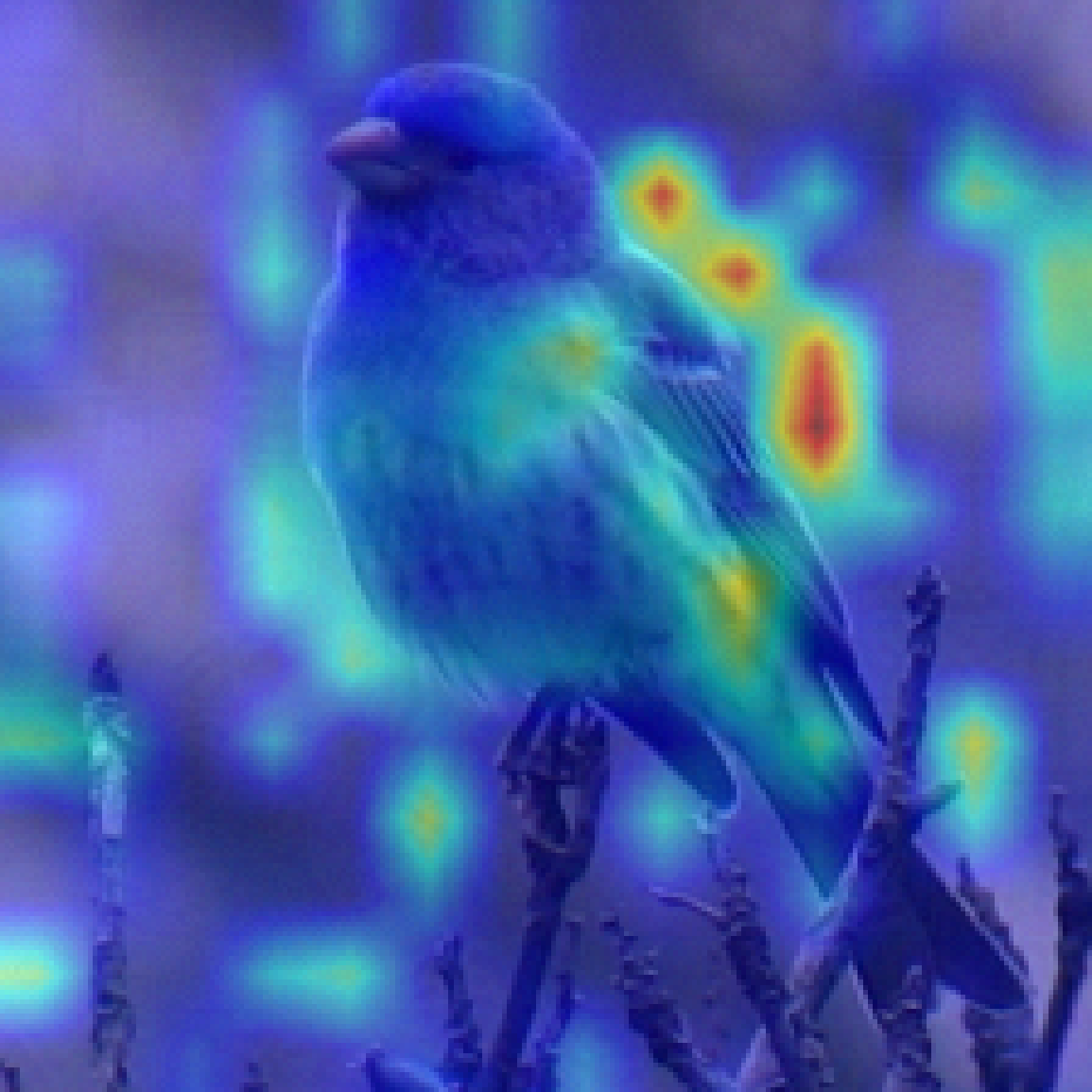}&\includegraphics[width=0.12\linewidth]{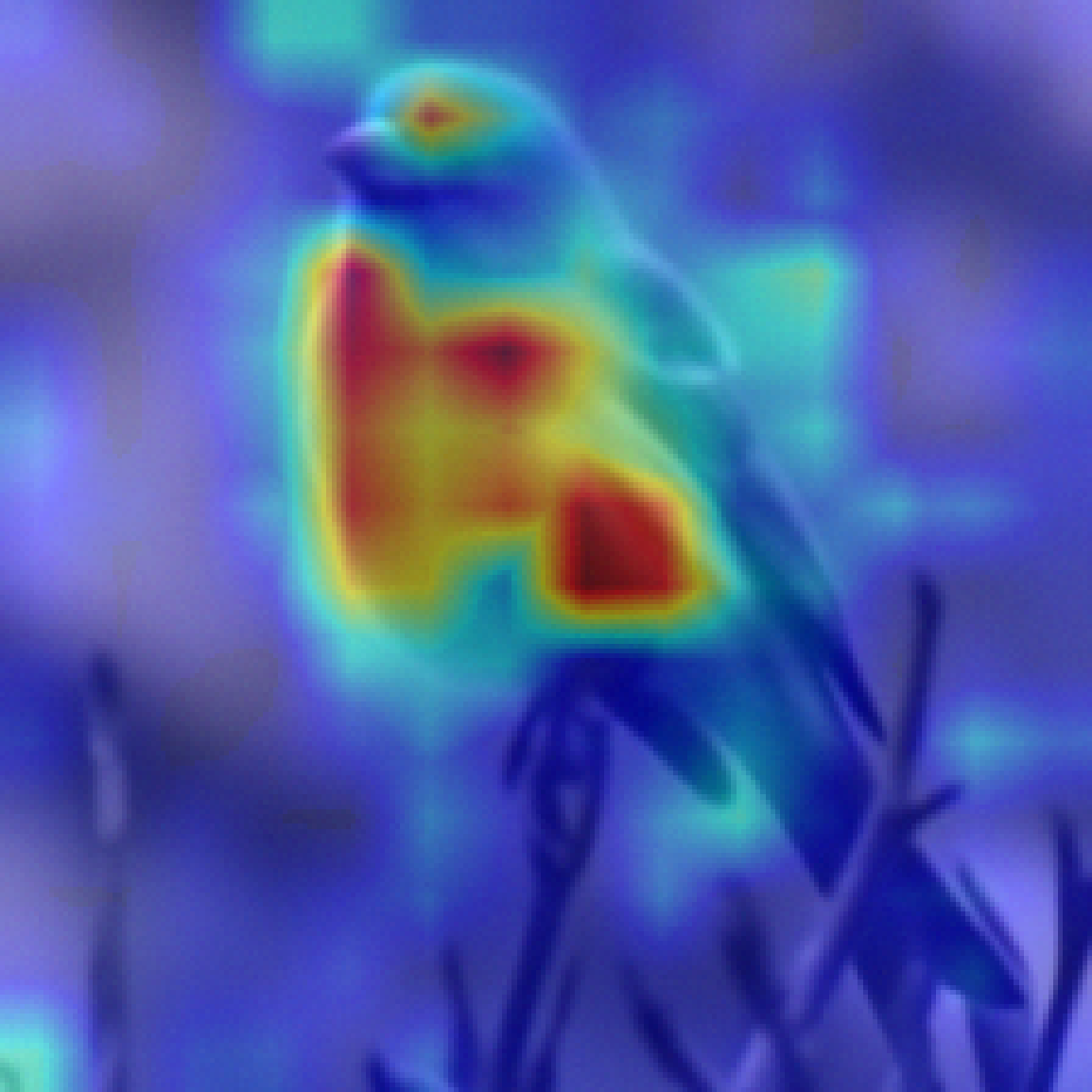}
    \\
    \includegraphics[width=0.12\linewidth]{exp/eagle1/eagle1.png}
&\includegraphics[width=0.12\linewidth]{exp/eagle1/eagle1-raw_attn-1-perturbed-7.png} &
    \includegraphics[width=0.12\linewidth]{exp/eagle1/eagle1-rollout-1-perturbed-7.png} &
    \includegraphics[width=0.12\linewidth]{exp/eagle1/eagle1-attn_gradcam-1-perturbed-7.png} & \includegraphics[width=0.12\linewidth]{exp/eagle1/eagle1-full-LRP-1-perturbed-7.png} & \includegraphics[width=0.12\linewidth]{exp/eagle1/eagle1-vta-1-perturbed-7.png} 
    & \includegraphics[width=0.12\linewidth]{exp/eagle1/eagle1-ours-1-perturbed-7.png} \\
   \includegraphics[width=0.12\linewidth]{exp/fish/fish.png}
    &
    \includegraphics[width=0.12\linewidth]{exp/fish/fish-raw_attn-1-perturbed-7.png} &
    \includegraphics[width=0.12\linewidth]{exp/fish/fish-rollout-1-perturbed-7.png} &
    \includegraphics[width=0.12\linewidth]{exp/fish/fish-attn_gradcam-1-perturbed-7.png} &
    \includegraphics[width=0.12\linewidth]{exp/fish/fish-full-LRP-1-perturbed-7.png} &    \includegraphics[width=0.12\linewidth]{exp/fish/fish-vta-1-perturbed-7.png} &\includegraphics[width=0.12\linewidth]{exp/fish/fish-ours-1-perturbed-7.png} 
    \\
\includegraphics[width=0.12\linewidth]{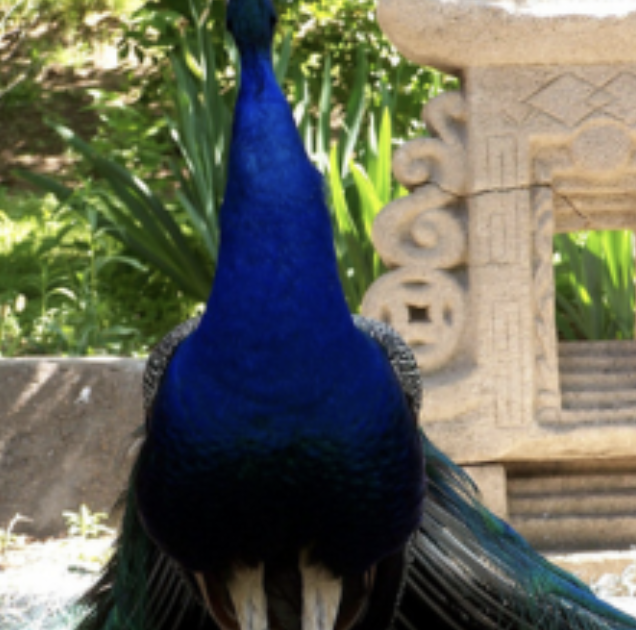} 
    &
    \includegraphics[width=0.12\linewidth]{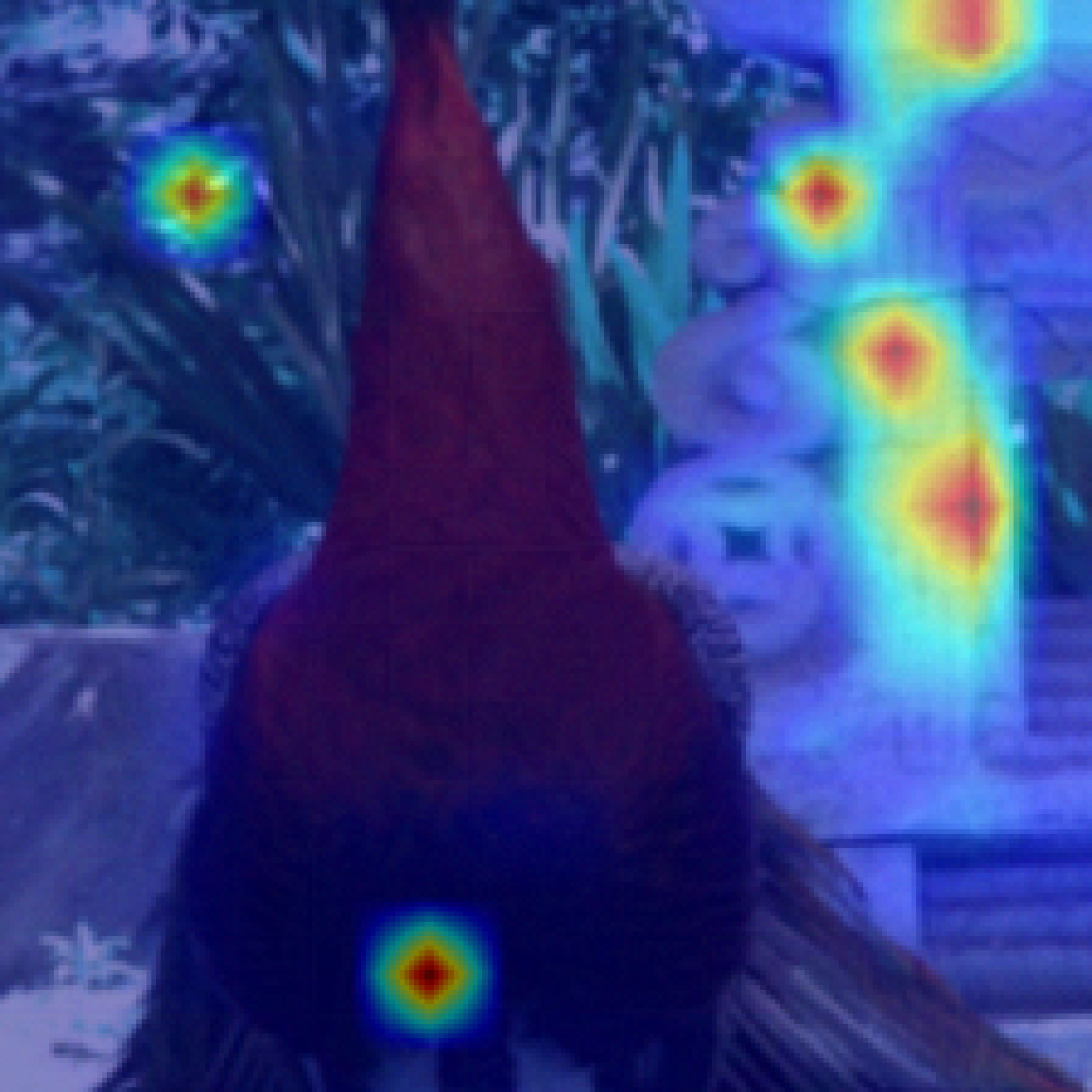} &
    \includegraphics[width=0.12\linewidth]{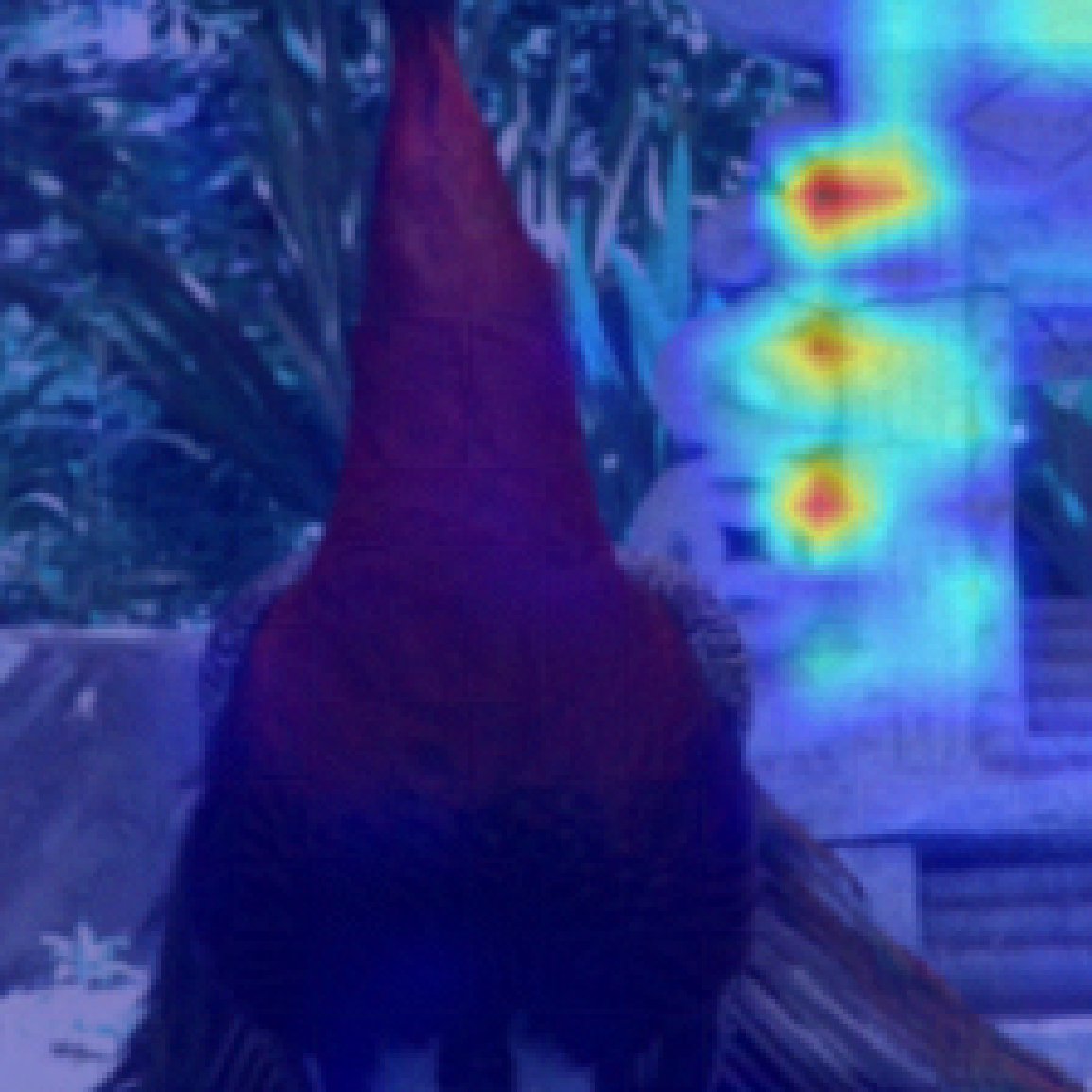} &\includegraphics[width=0.12\linewidth]{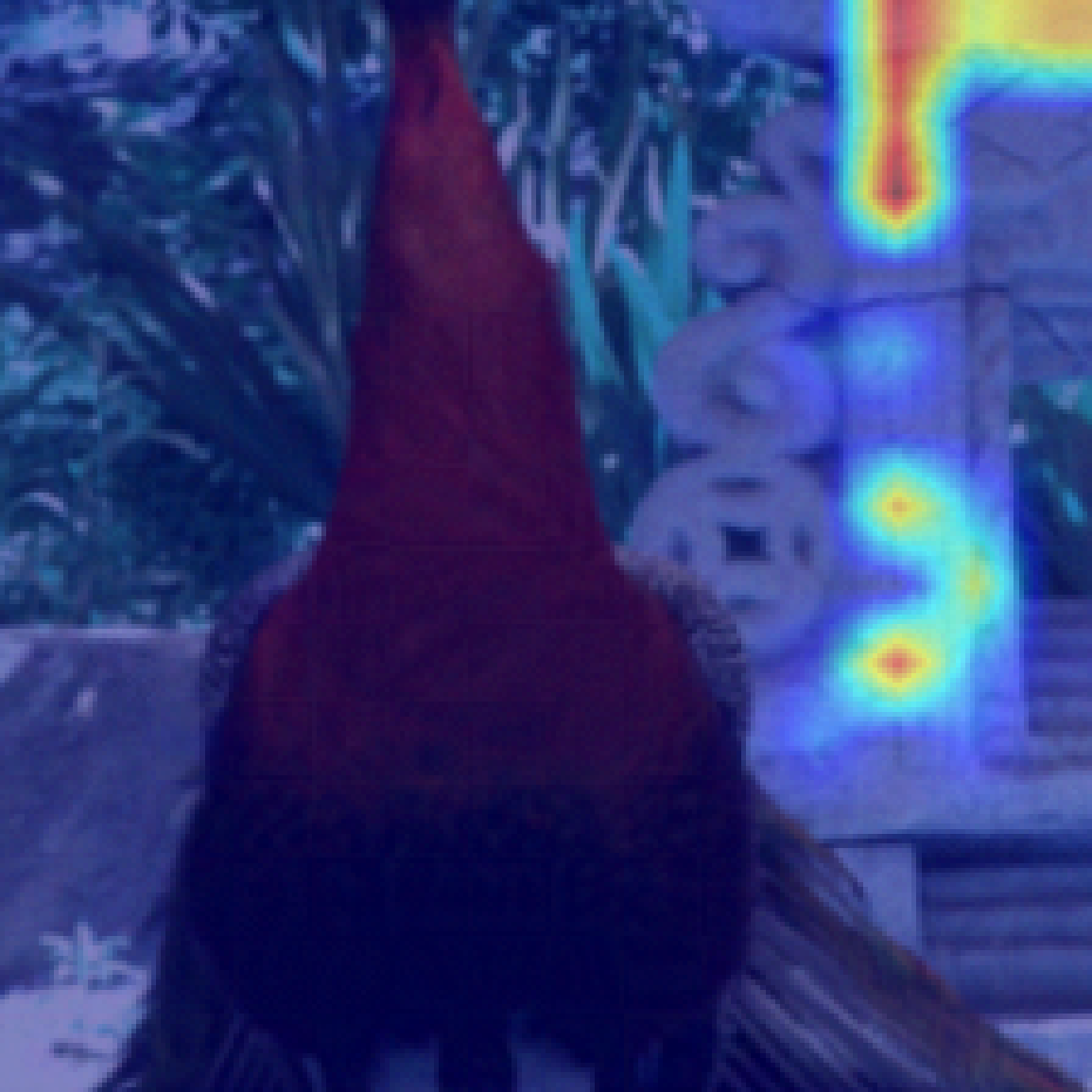} &\includegraphics[width=0.12\linewidth]{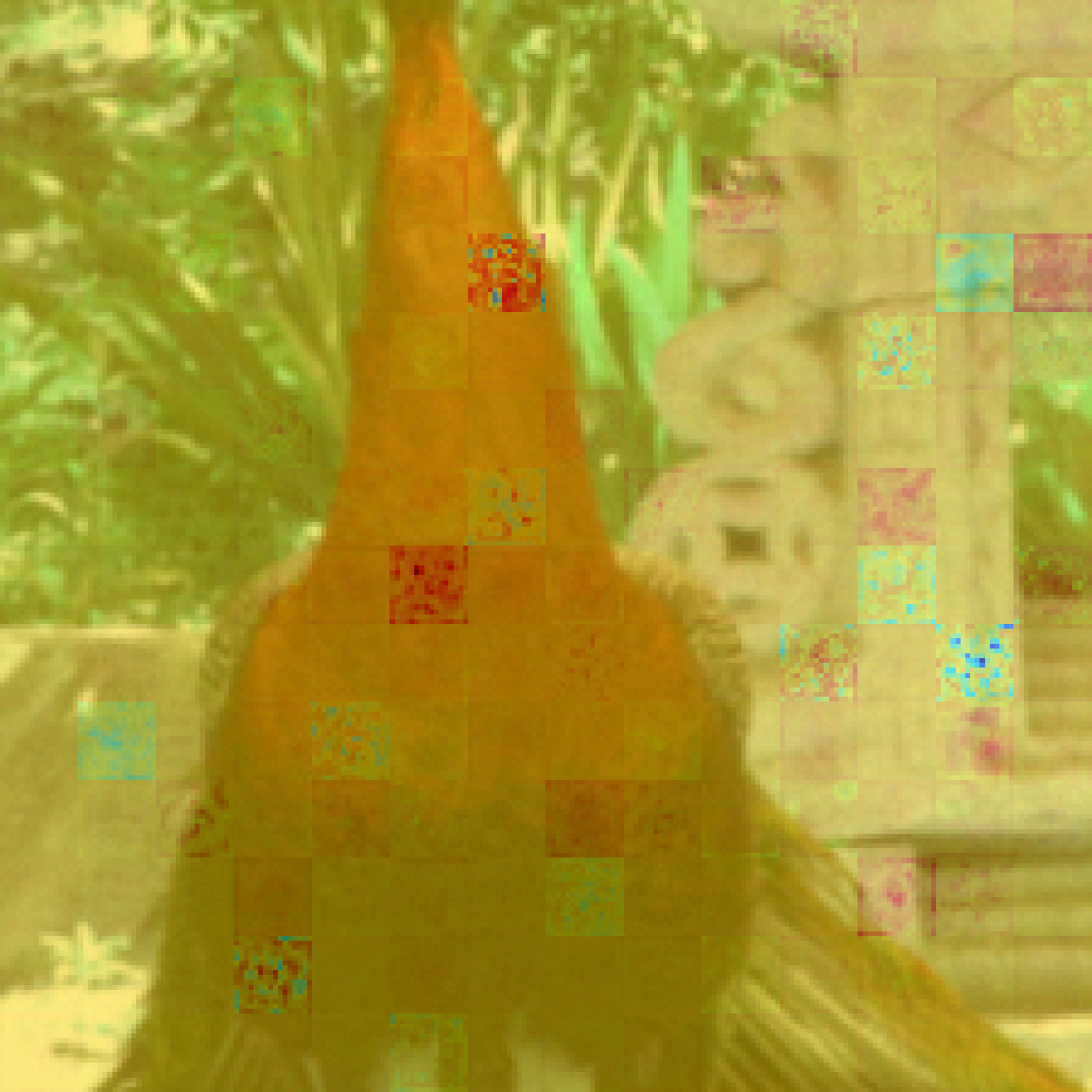} &\includegraphics[width=0.12\linewidth]{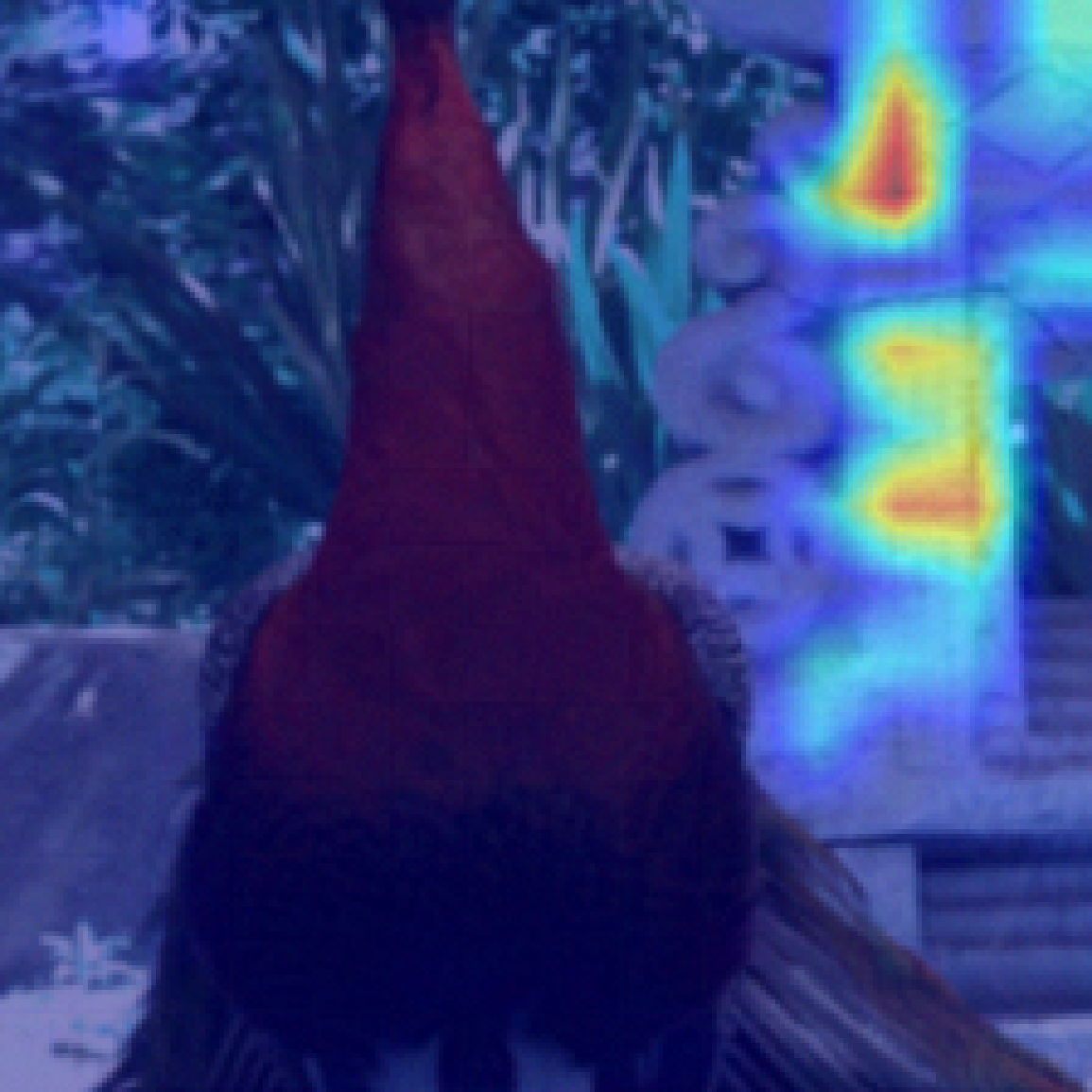} & \includegraphics[width=0.12\linewidth]{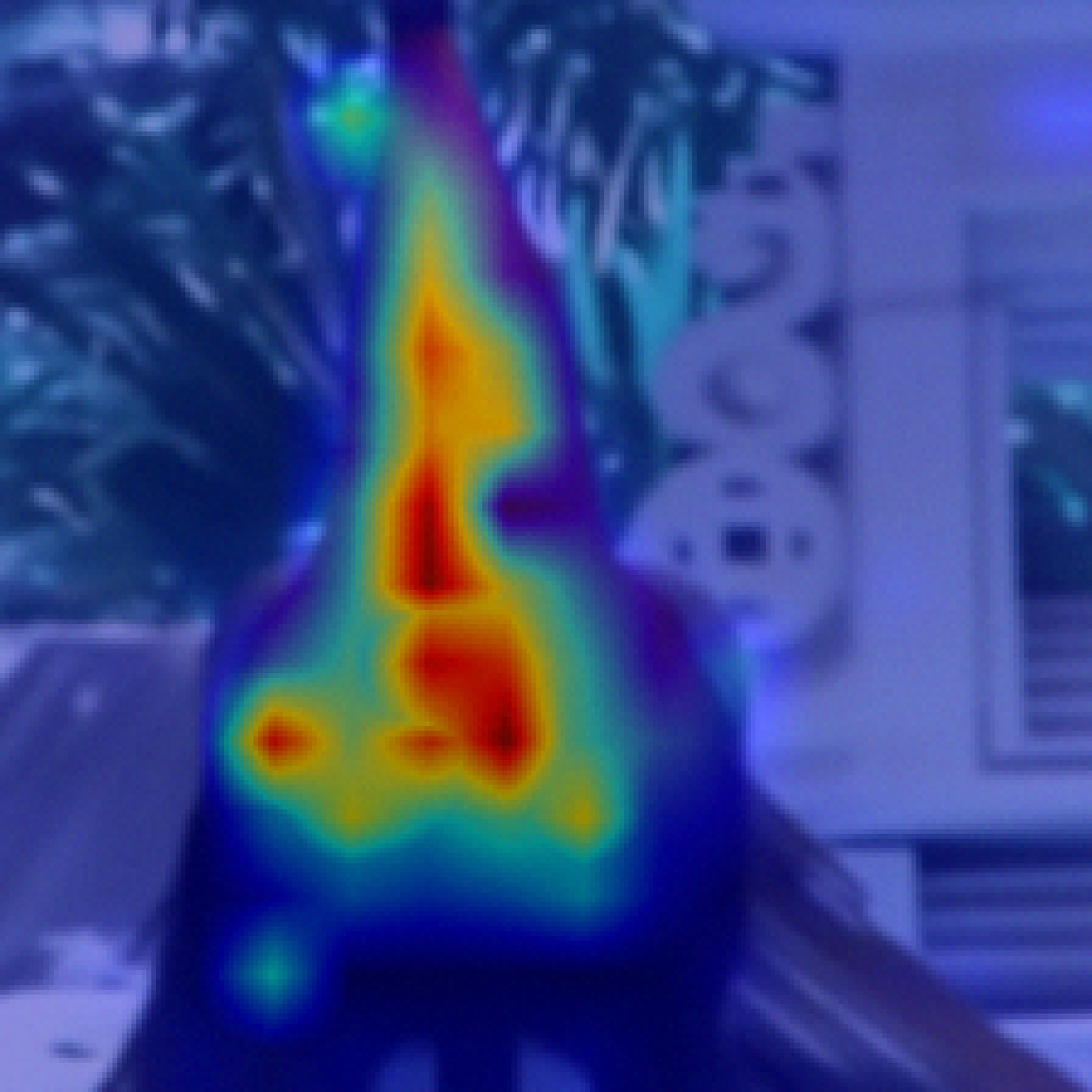}
    \end{tabular*}
    \caption{Visualization results of the attention map on corrupted input for different methods. }
    \label{fig: poison}
    \end{center}
\end{figure*}

\begin{figure*}
    \setlength{\tabcolsep}{1pt} 
    \renewcommand{\arraystretch}{1} 
    \begin{center}
    \begin{tabular*}{\linewidth}
{@{\extracolsep{\fill}}cccccccc}
    Input & Raw Attention & Rollout & GradCAM & LRP & VTA  & Ours \\
    \raisebox{13mm}{\multirow{2}{*}{\makecell*[c]{Dog: clean$\rightarrow$\\\includegraphics[width=0.15\linewidth]{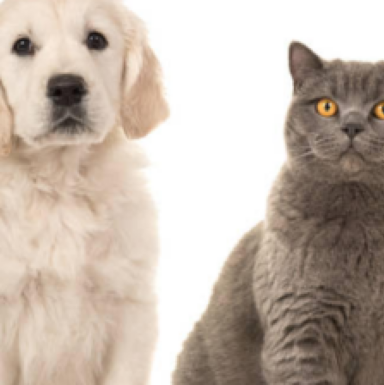}\\ 
    Dog: poisoned$\rightarrow$\\{\scriptsize $7/255$}
    }}
    }
     &
    \includegraphics[width=0.13\linewidth]{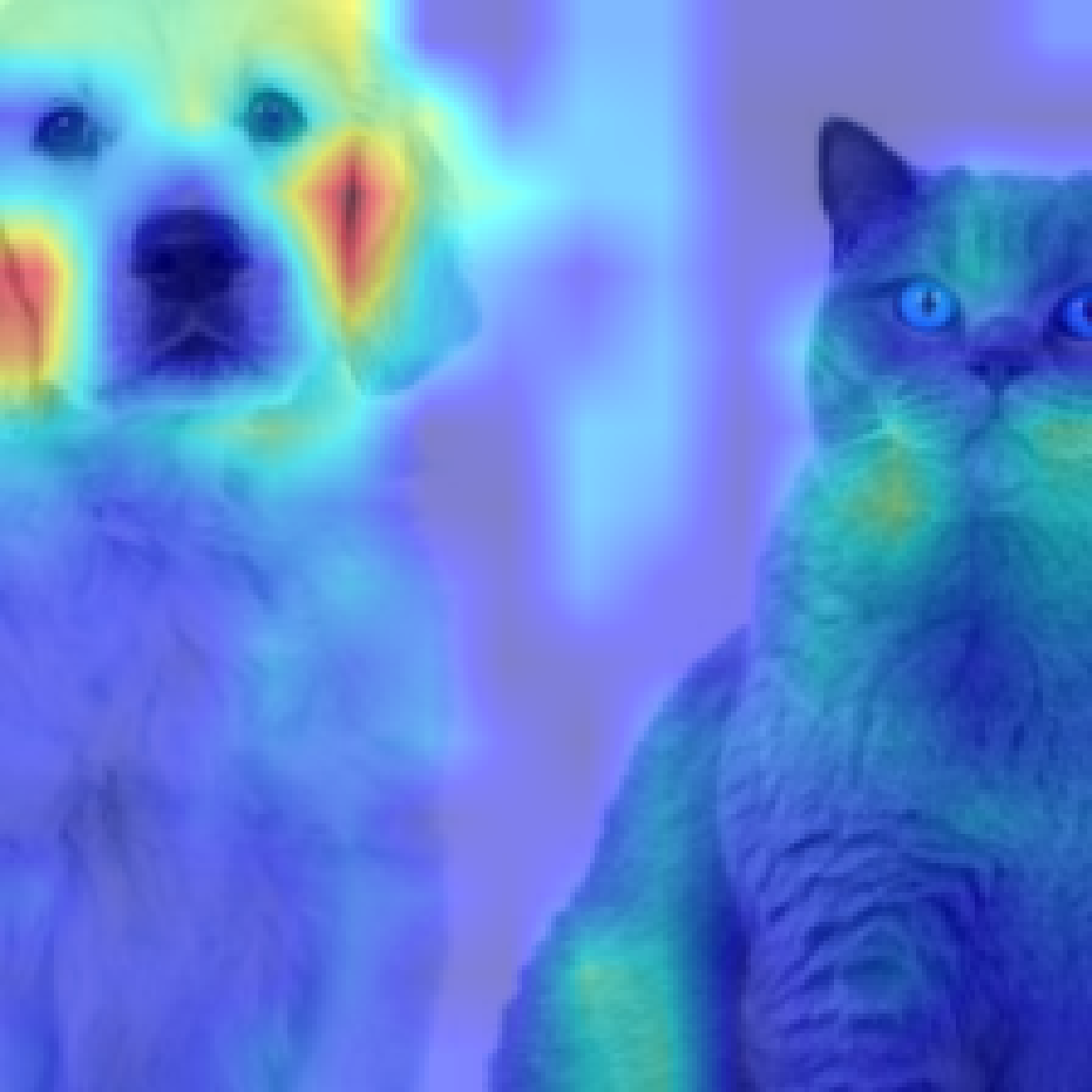} &
    \includegraphics[width=0.13\linewidth]{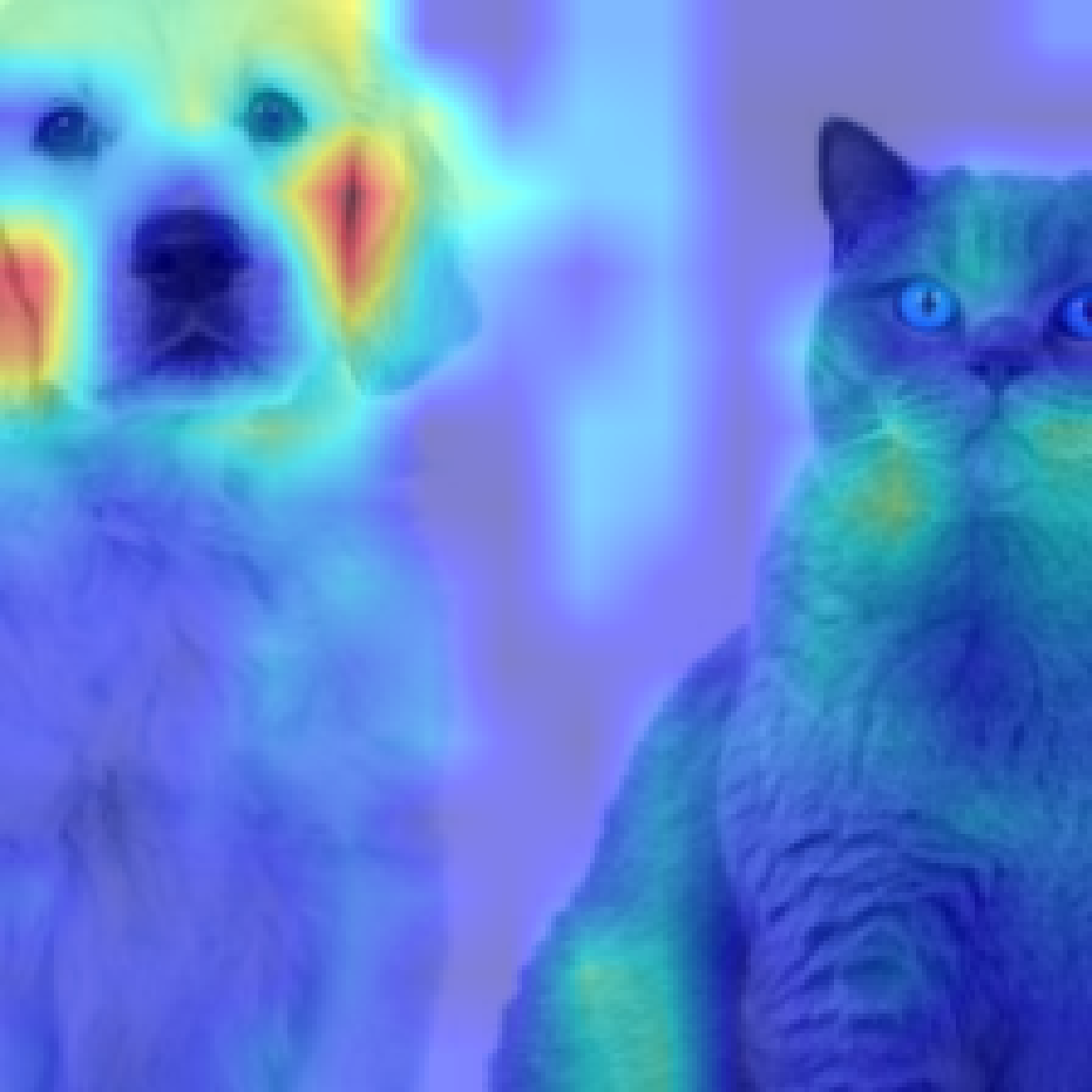} &
    \includegraphics[width=0.13\linewidth]{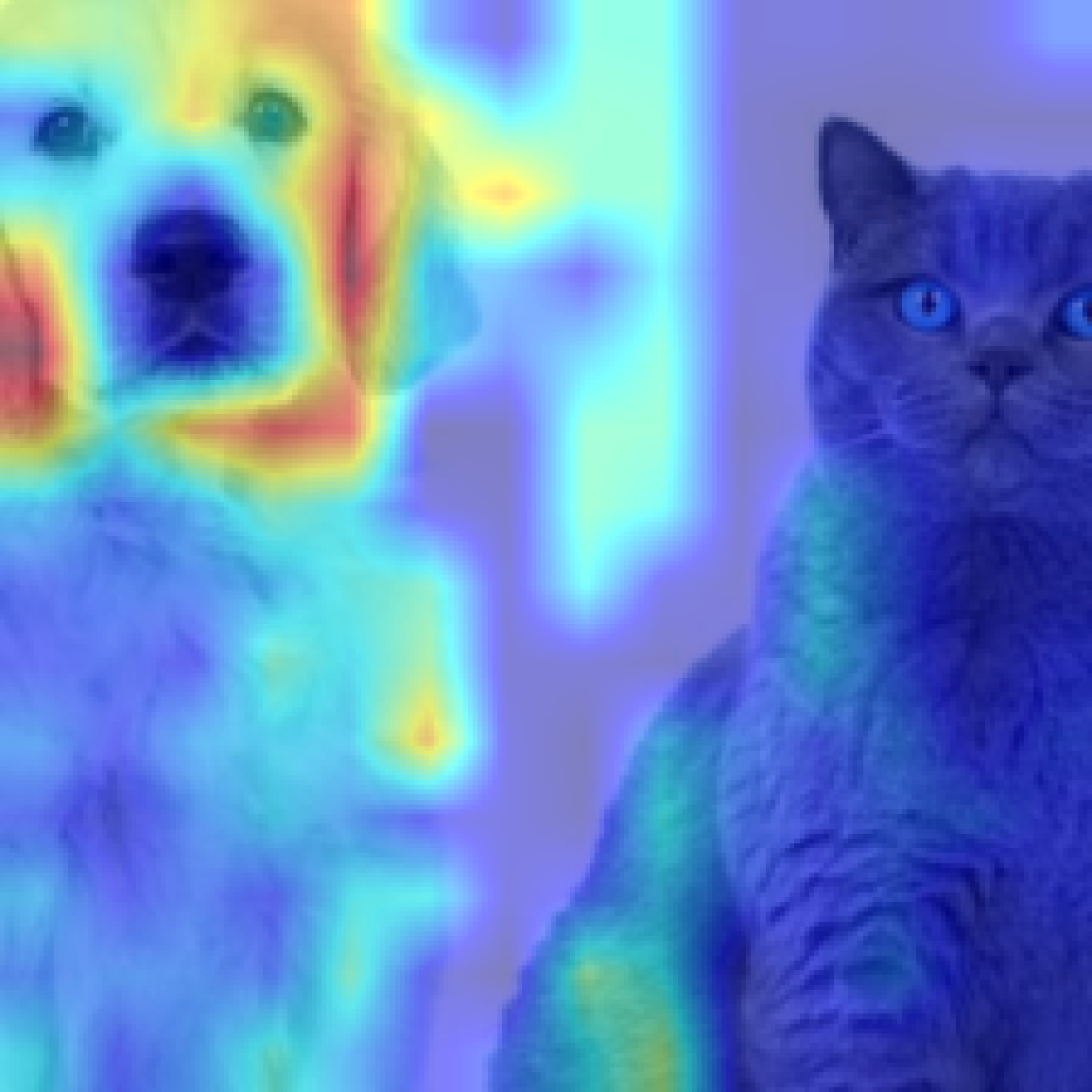} &
    \includegraphics[width=0.13\linewidth]{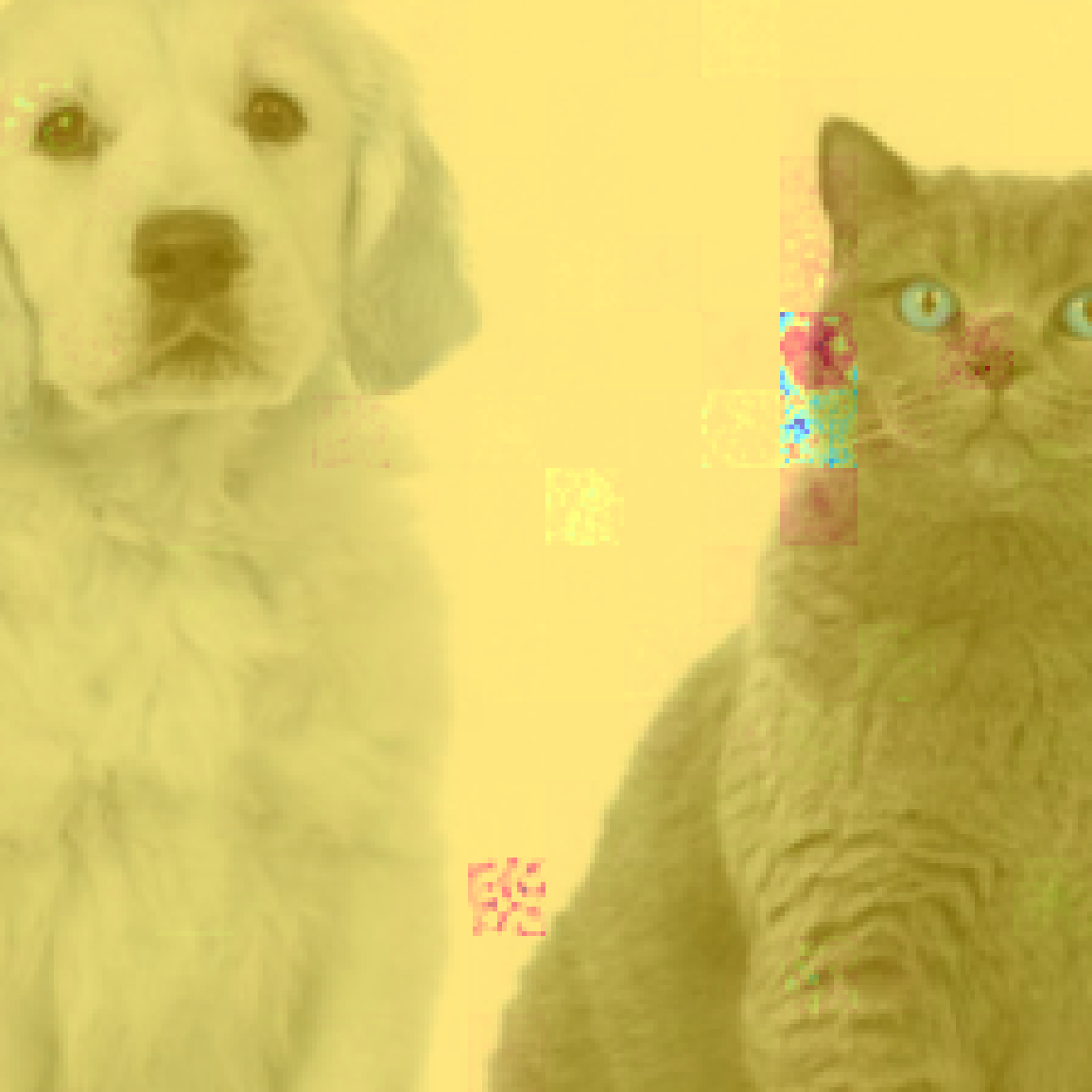} &
    \includegraphics[width=0.13\linewidth]{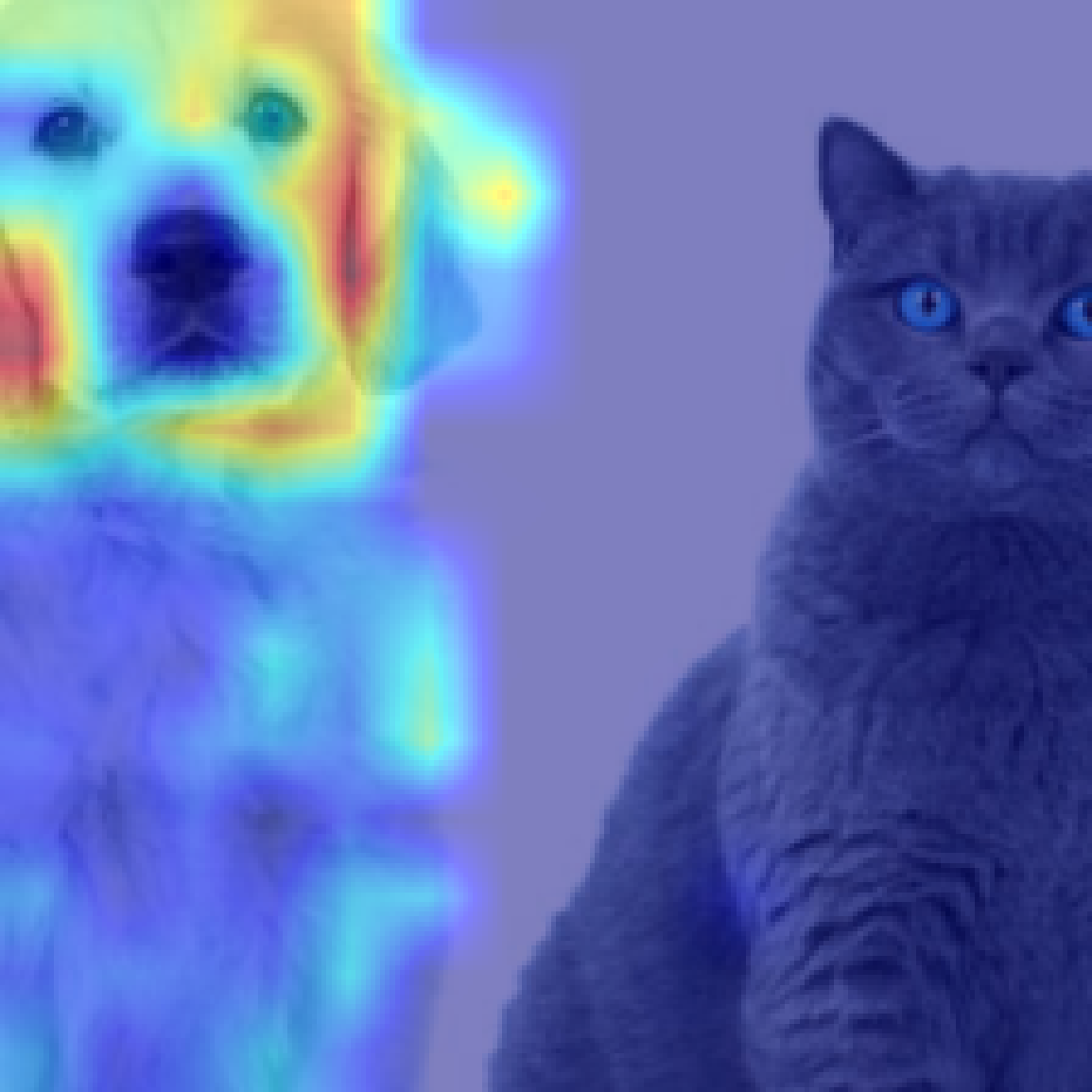} &
    \includegraphics[width=0.13\linewidth]{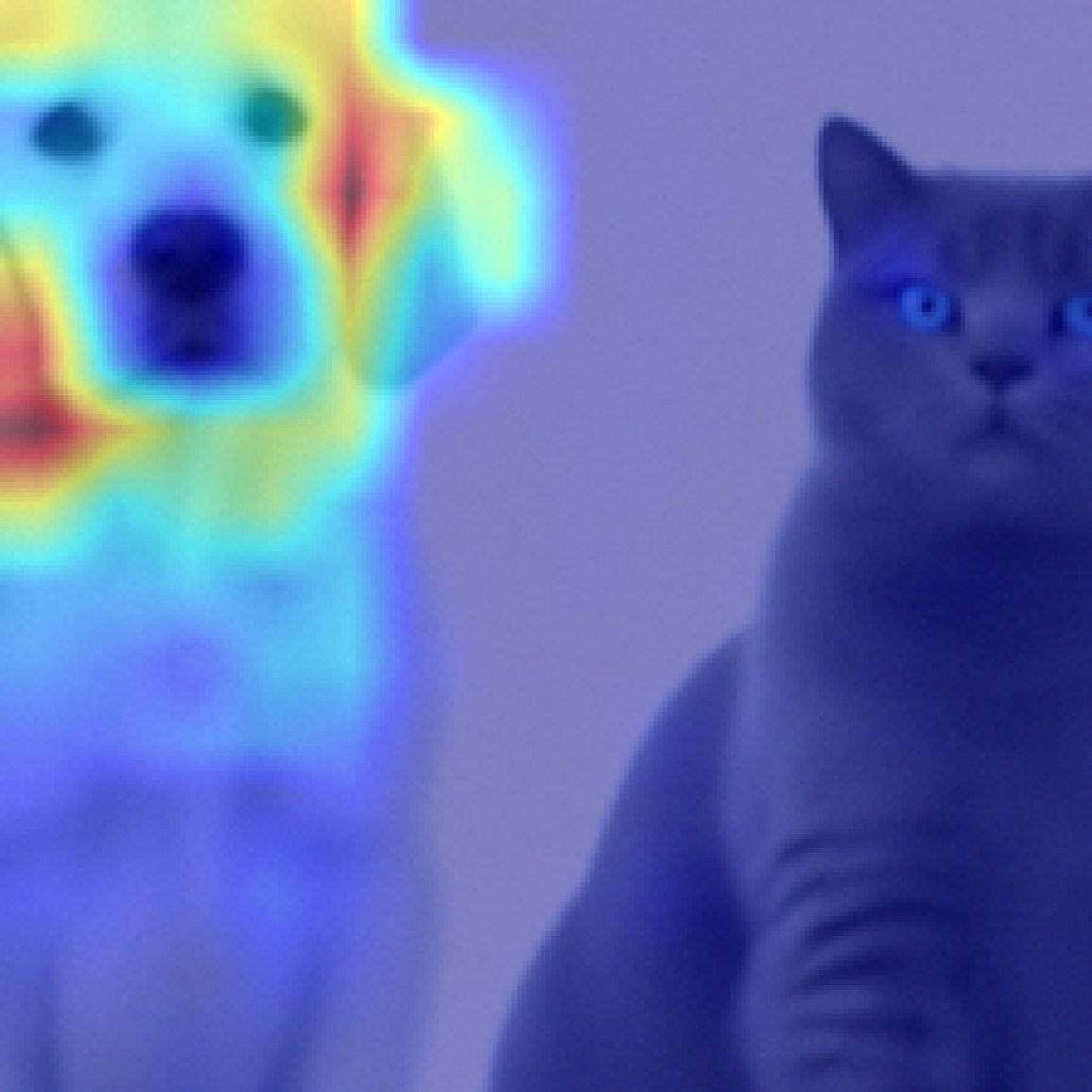}
    \\
    &
    \includegraphics[width=0.13\linewidth]{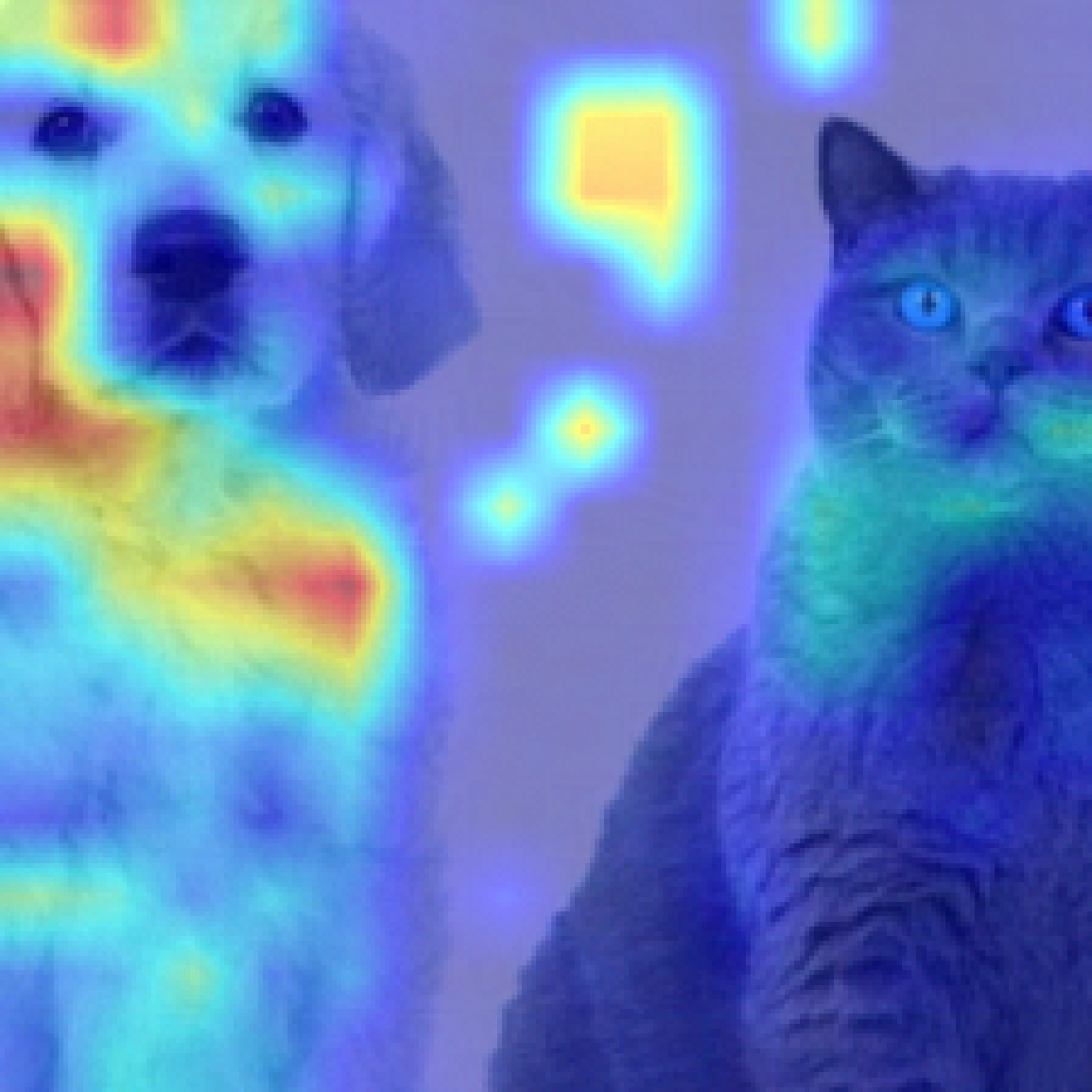} &
    \includegraphics[width=0.13\linewidth]{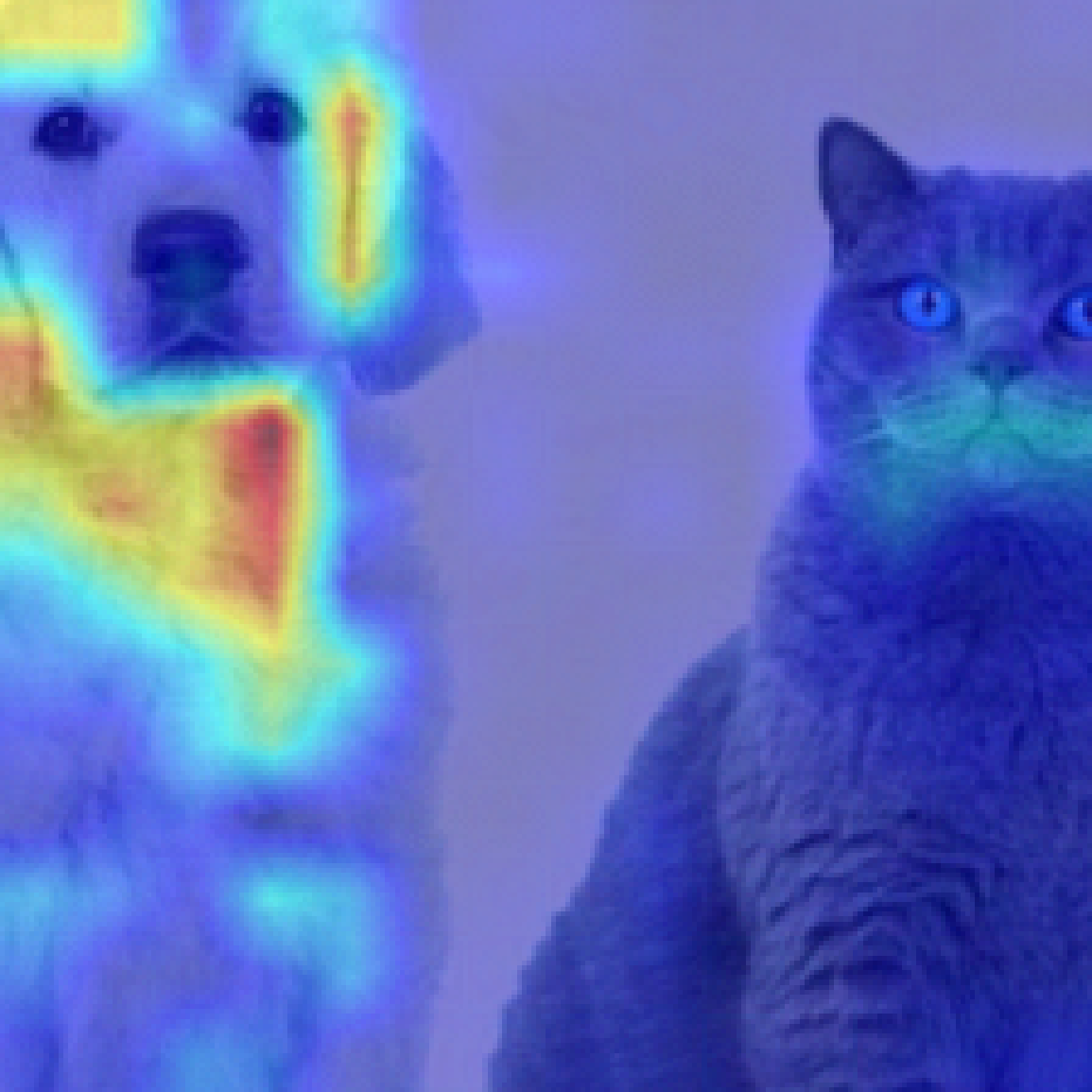} &
    \includegraphics[width=0.13\linewidth]{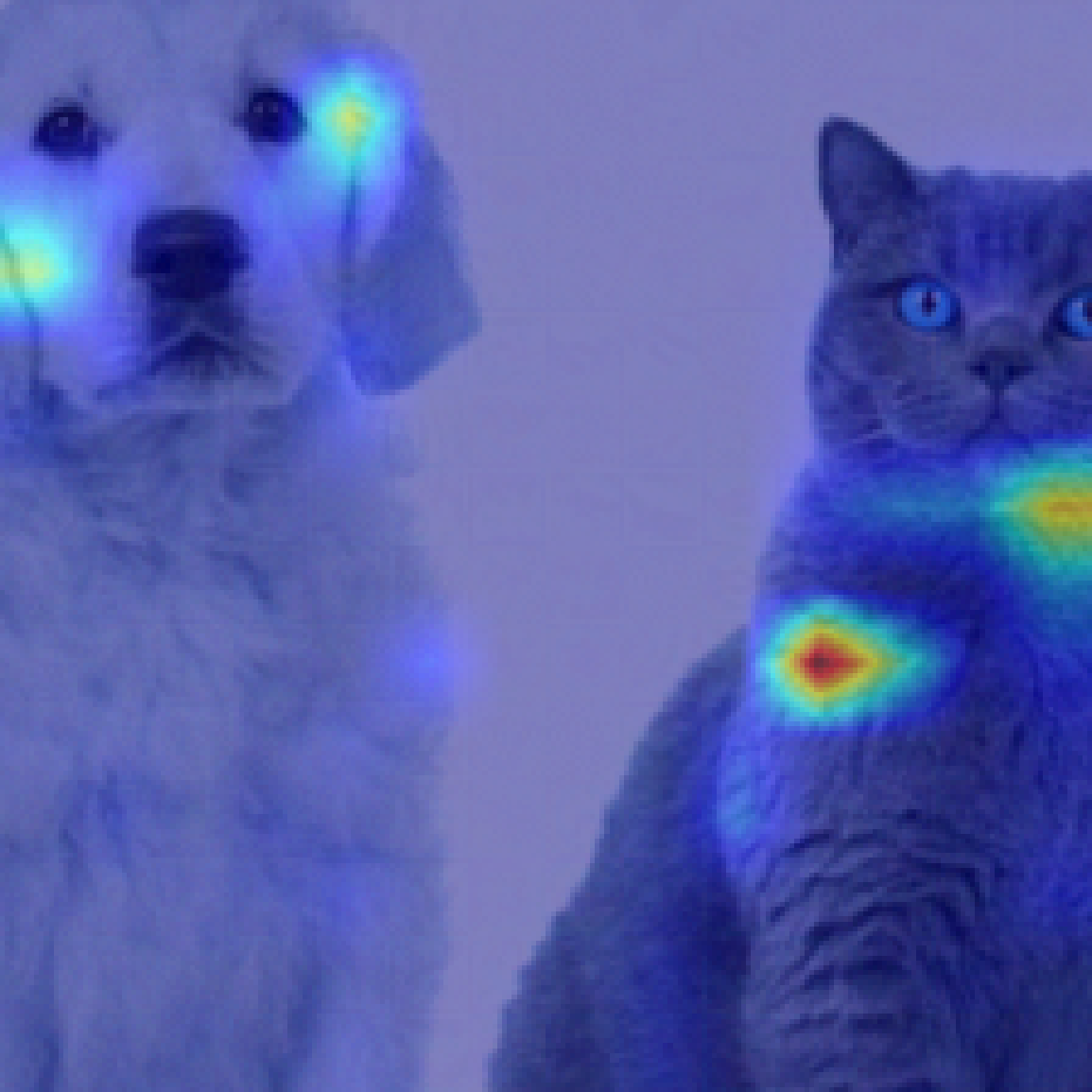} &
    \includegraphics[width=0.13\linewidth]{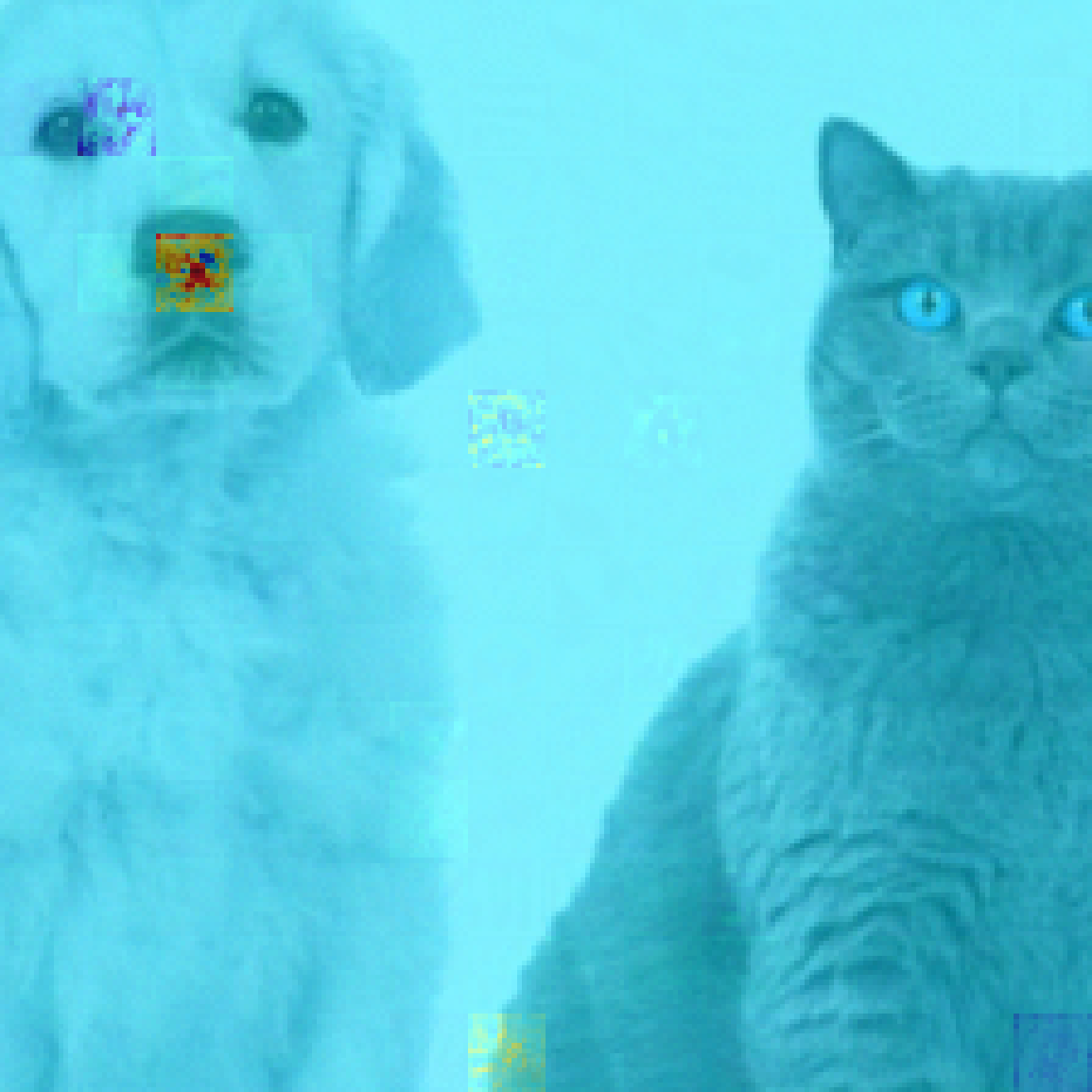} &
    \includegraphics[width=0.13\linewidth]{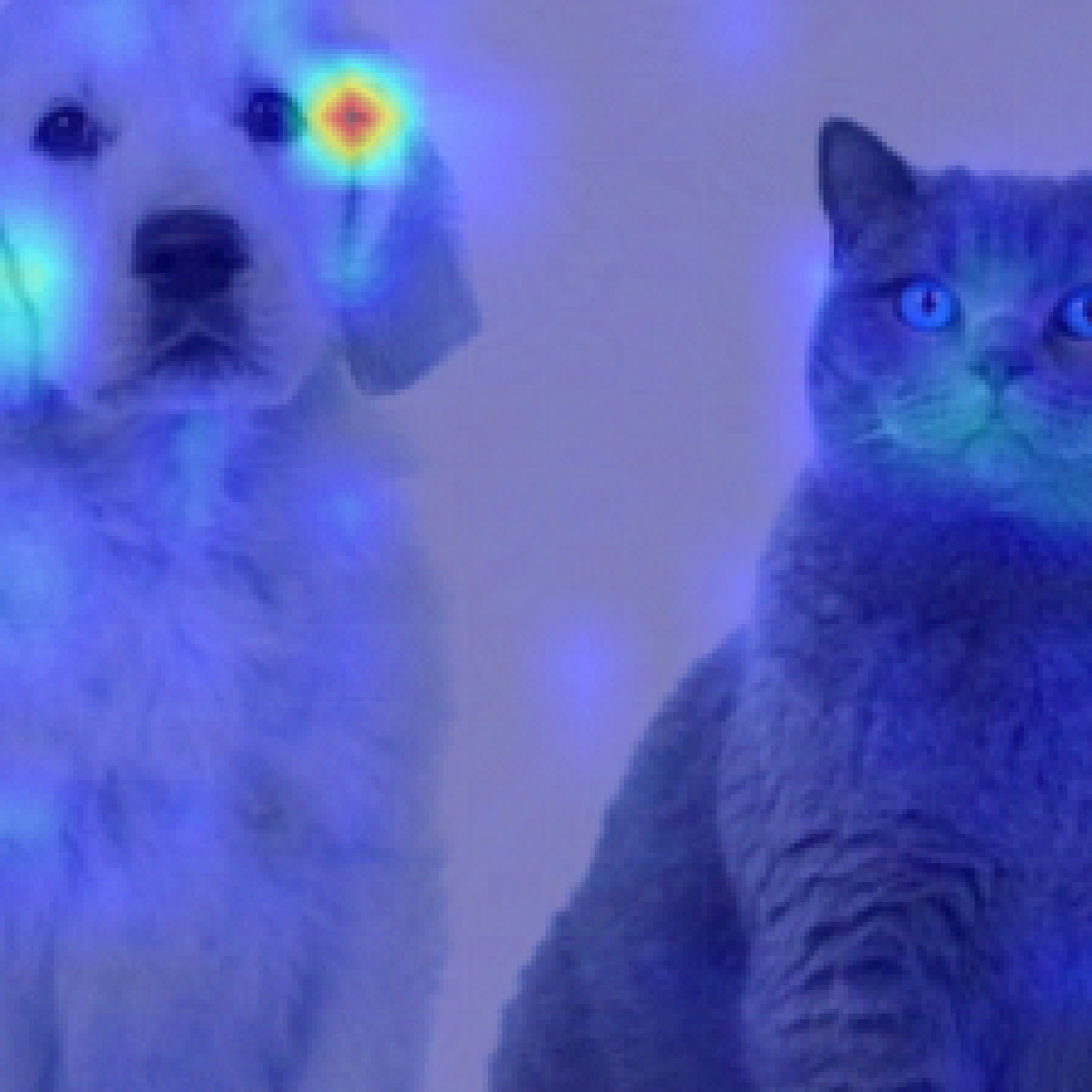} &
    \includegraphics[width=0.13\linewidth]{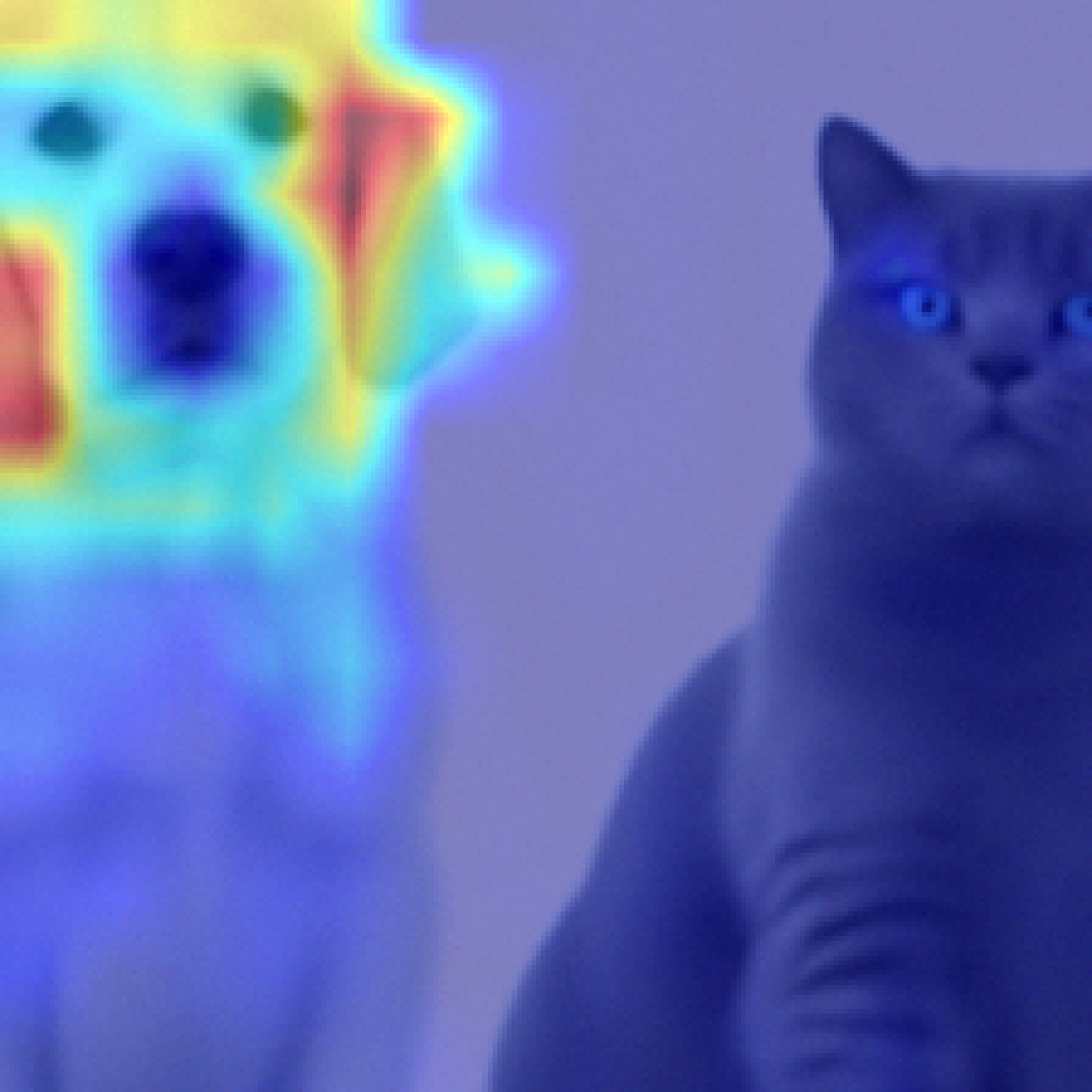}
    \\
    \raisebox{13mm}{\multirow{2}{*}{\makecell*[c]{Cat: clean$\rightarrow$\\\includegraphics[width=0.15\linewidth]{exp/dogcat2/dogcat2.png}\\ 
    Cat: poisoned$\rightarrow$\\{\scriptsize $7/255$}
    }}
    }
        &
    \includegraphics[width=0.13\linewidth]{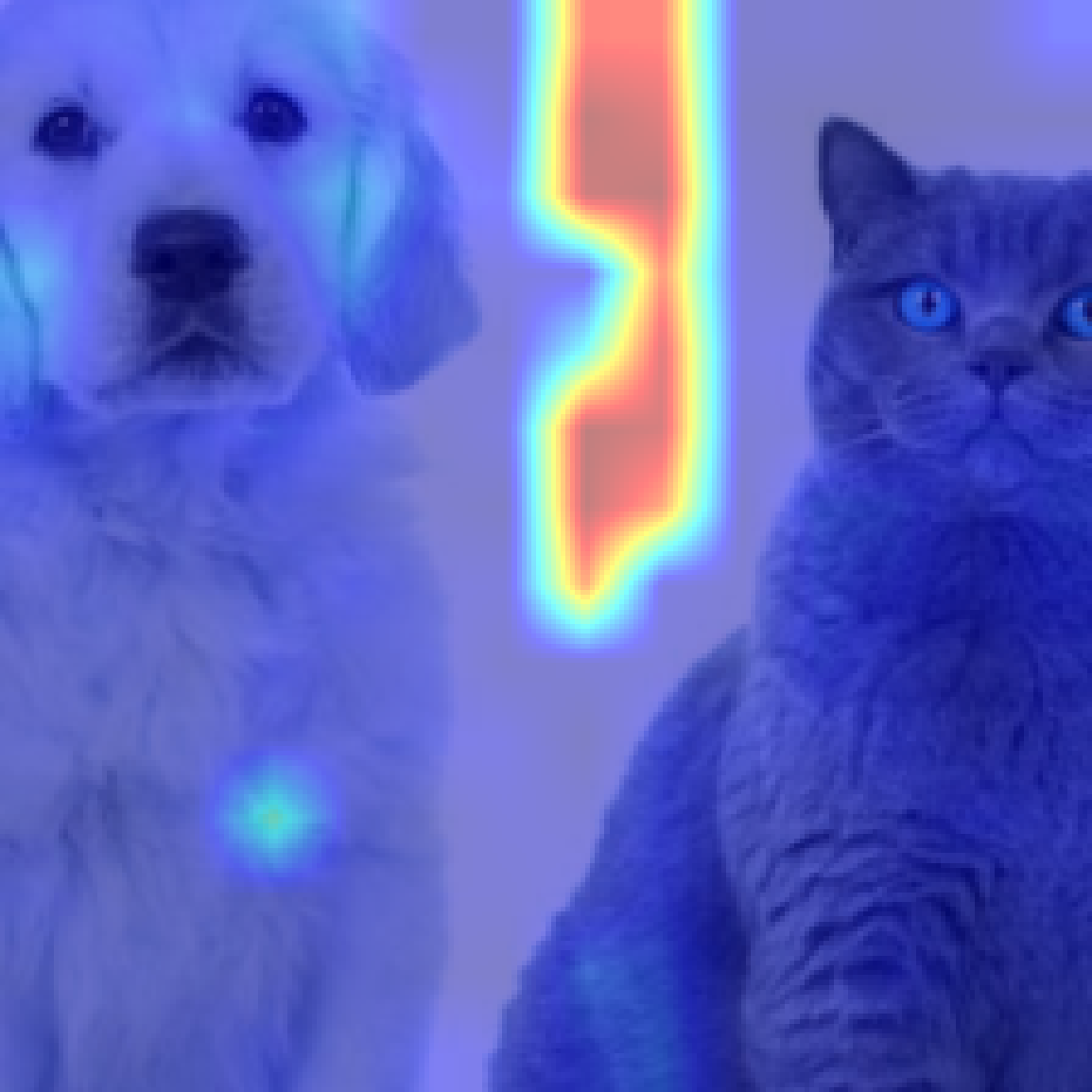} &
    \includegraphics[width=0.13\linewidth]{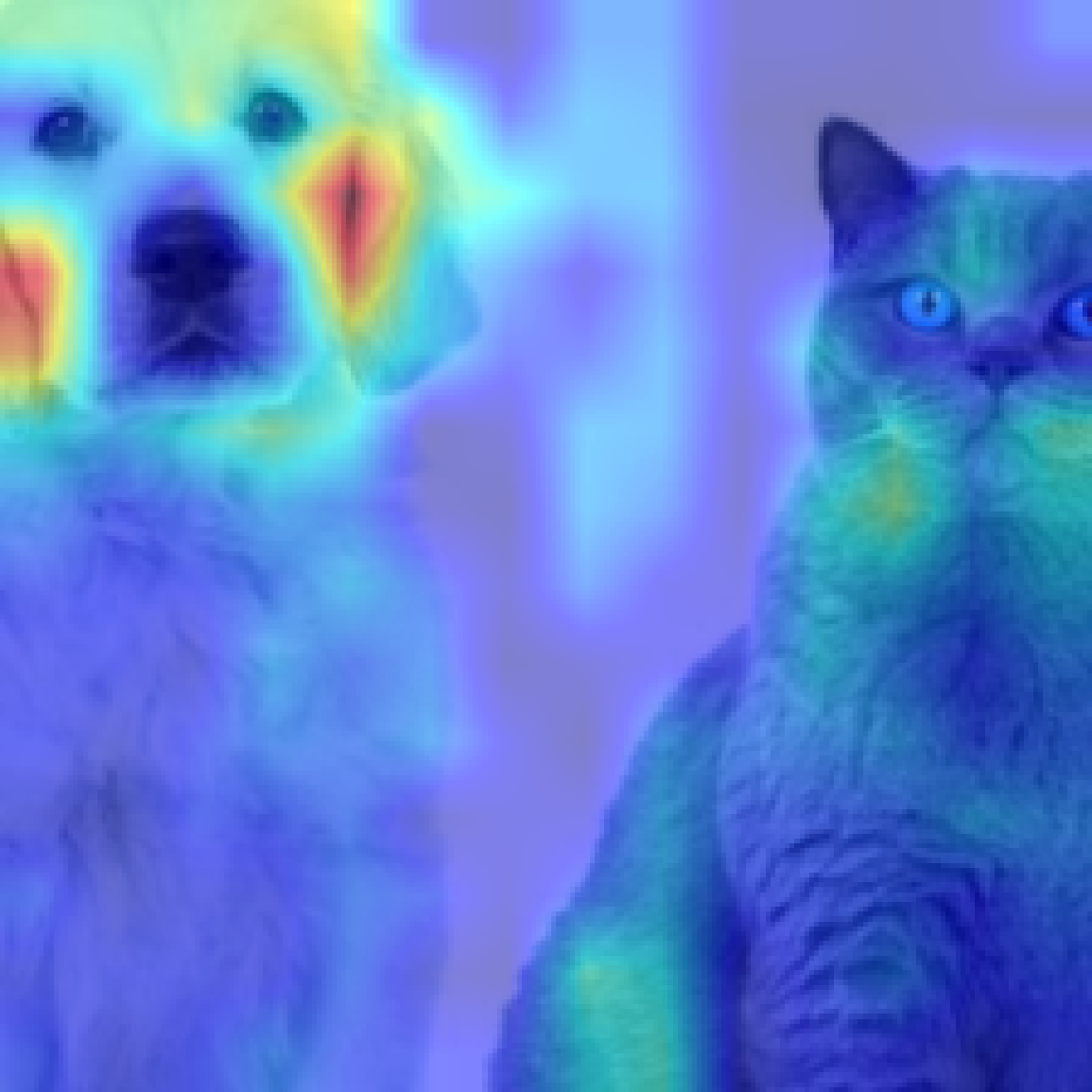} &
    \includegraphics[width=0.13\linewidth]{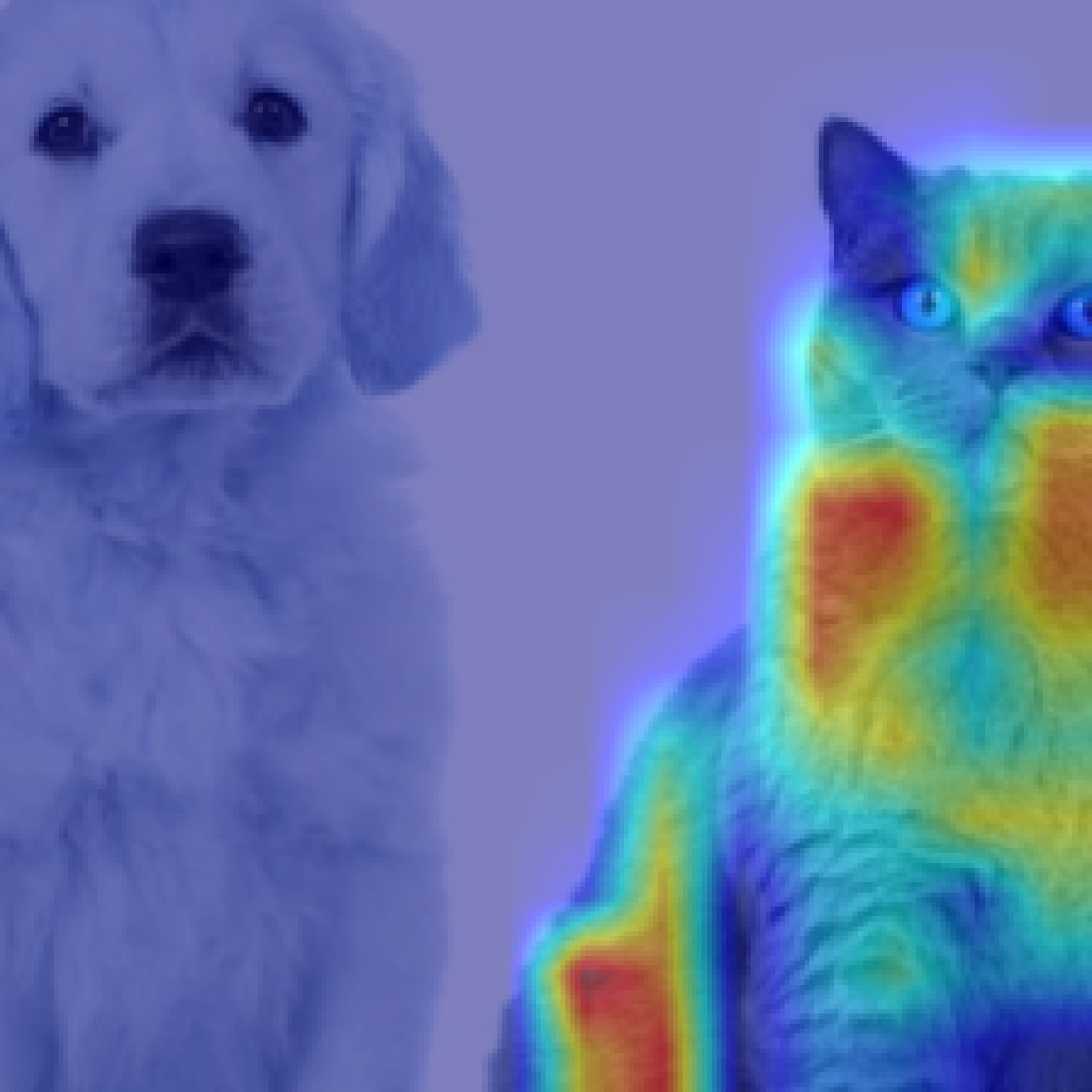} &
    \includegraphics[width=0.13\linewidth]{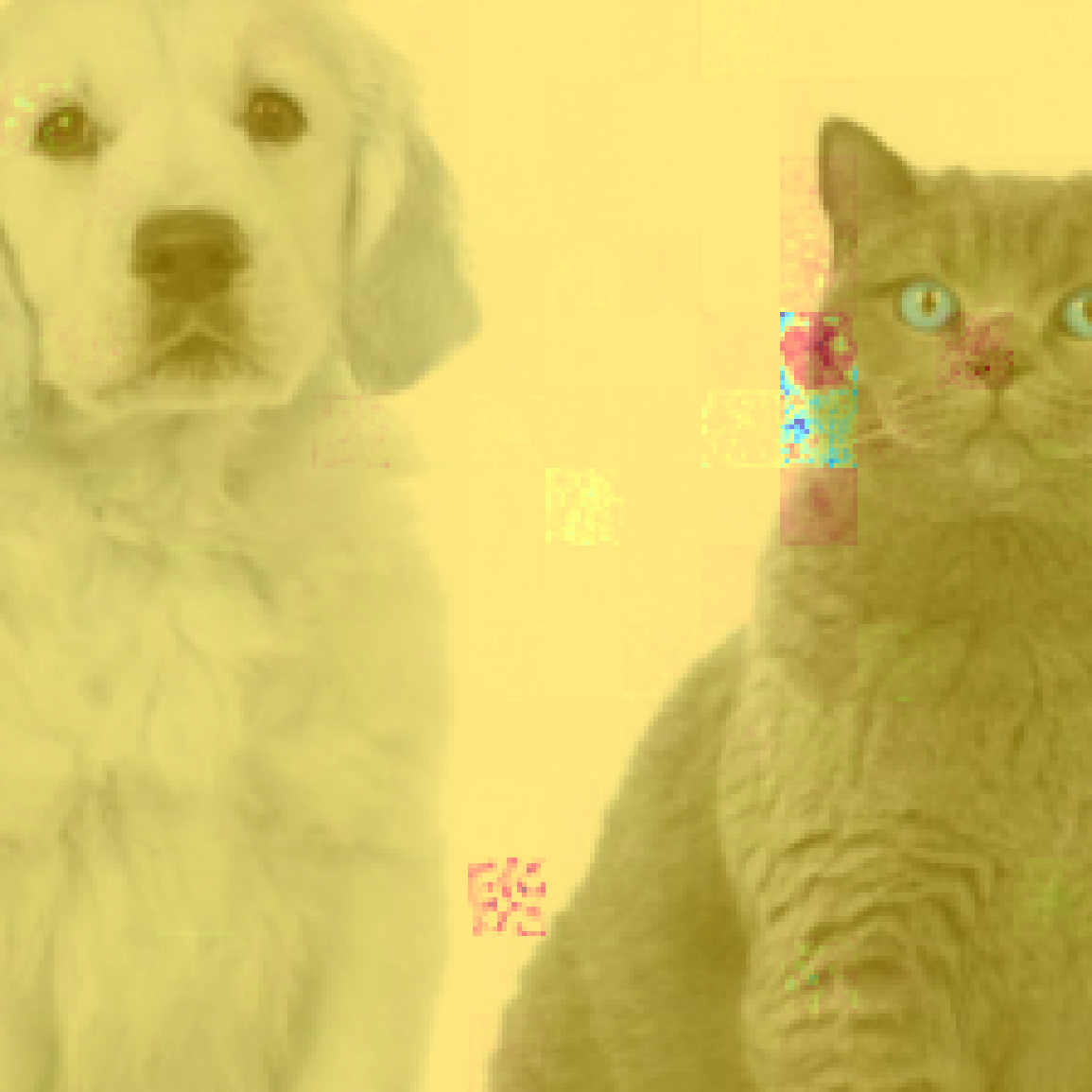} &
    \includegraphics[width=0.13\linewidth]{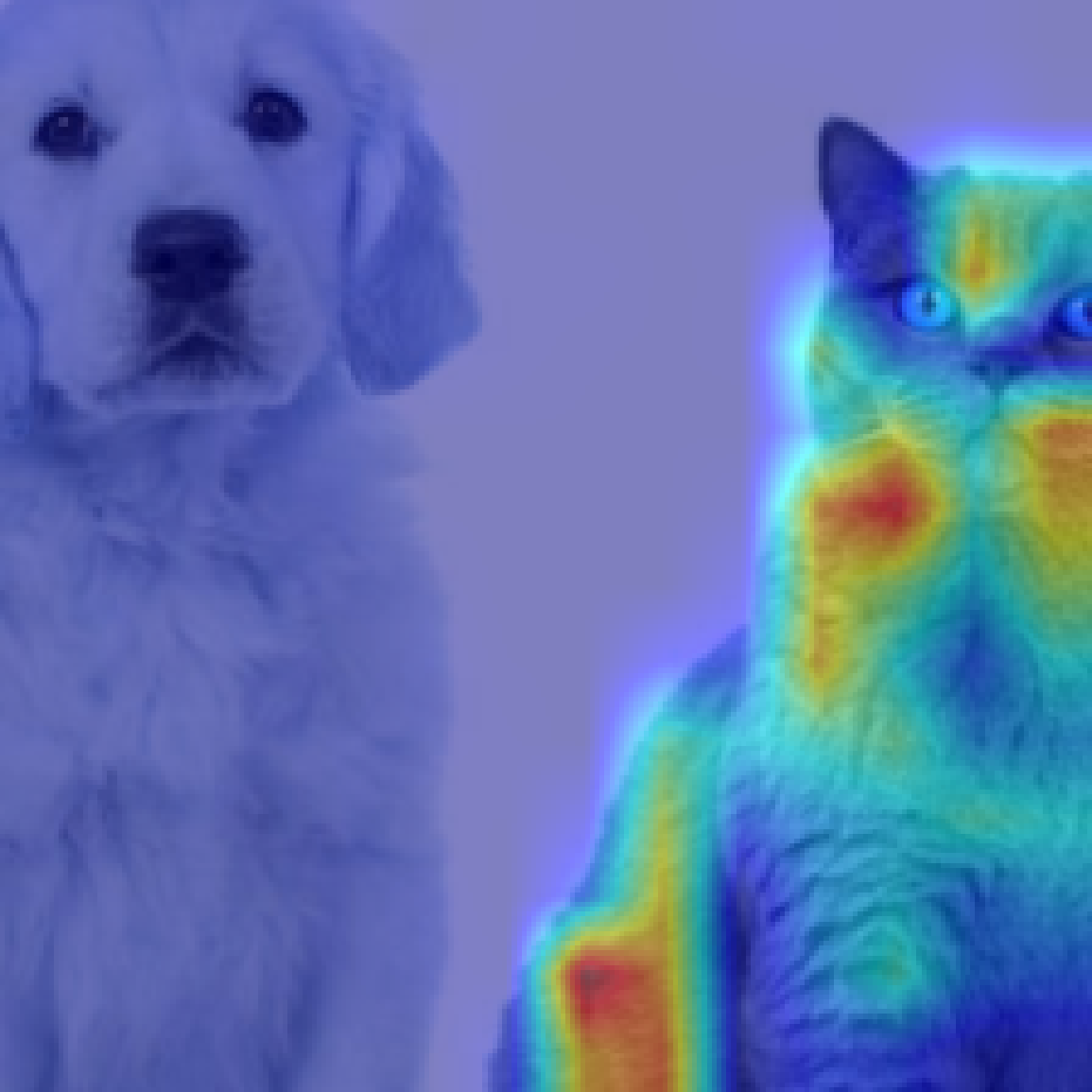} &
    \includegraphics[width=0.13\linewidth]{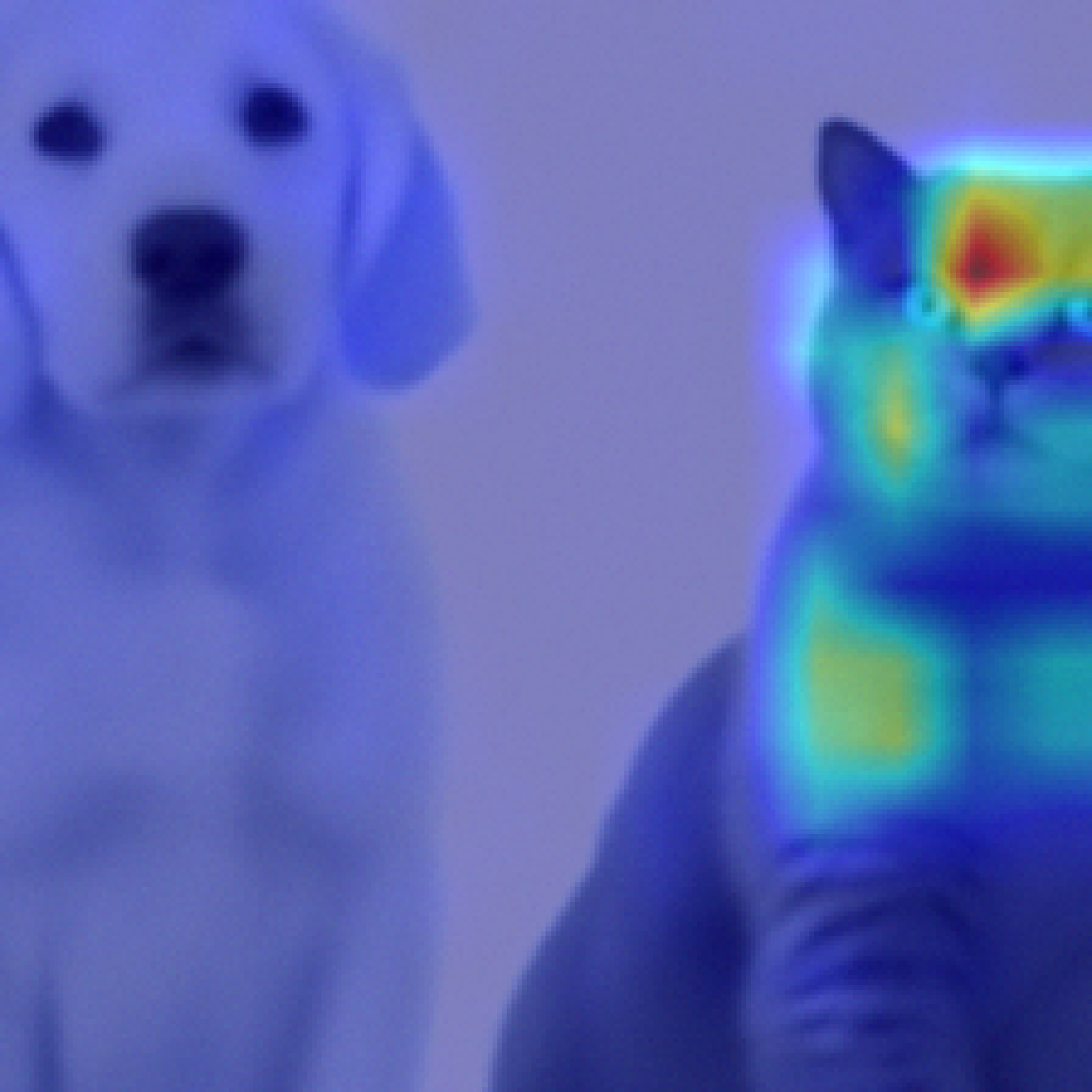}
    \\
    &
    \includegraphics[width=0.13\linewidth]{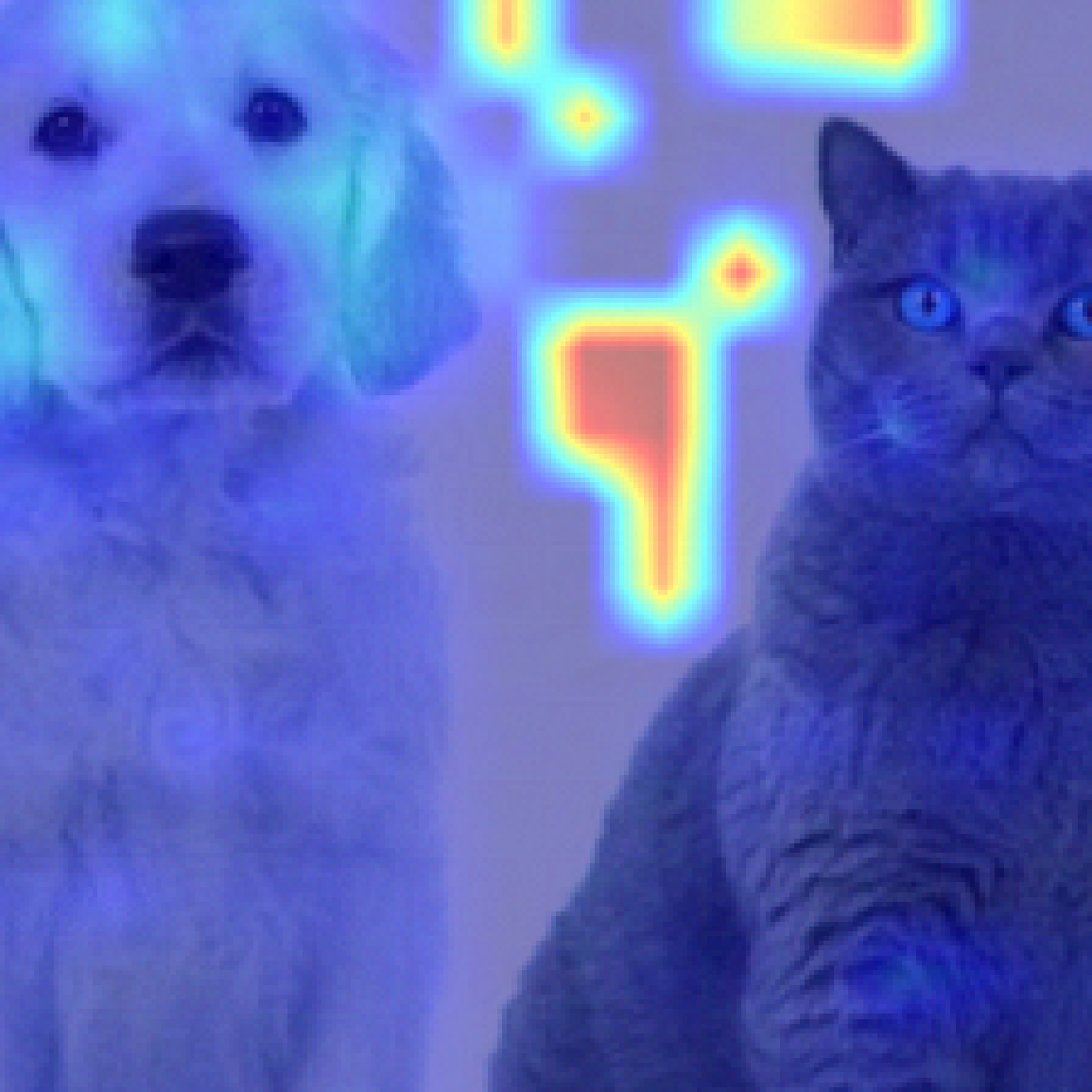} &
    \includegraphics[width=0.13\linewidth]{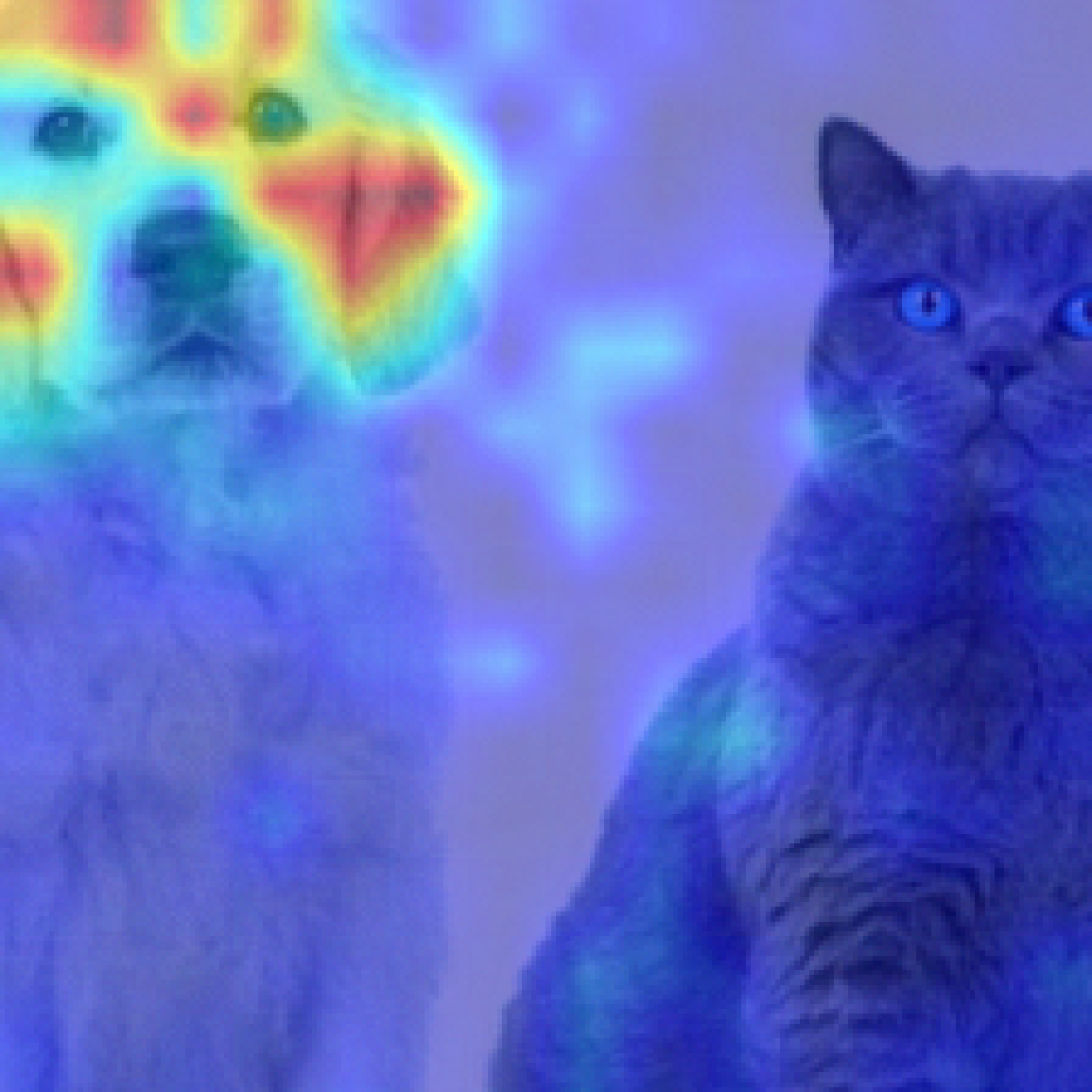} &
    \includegraphics[width=0.13\linewidth]{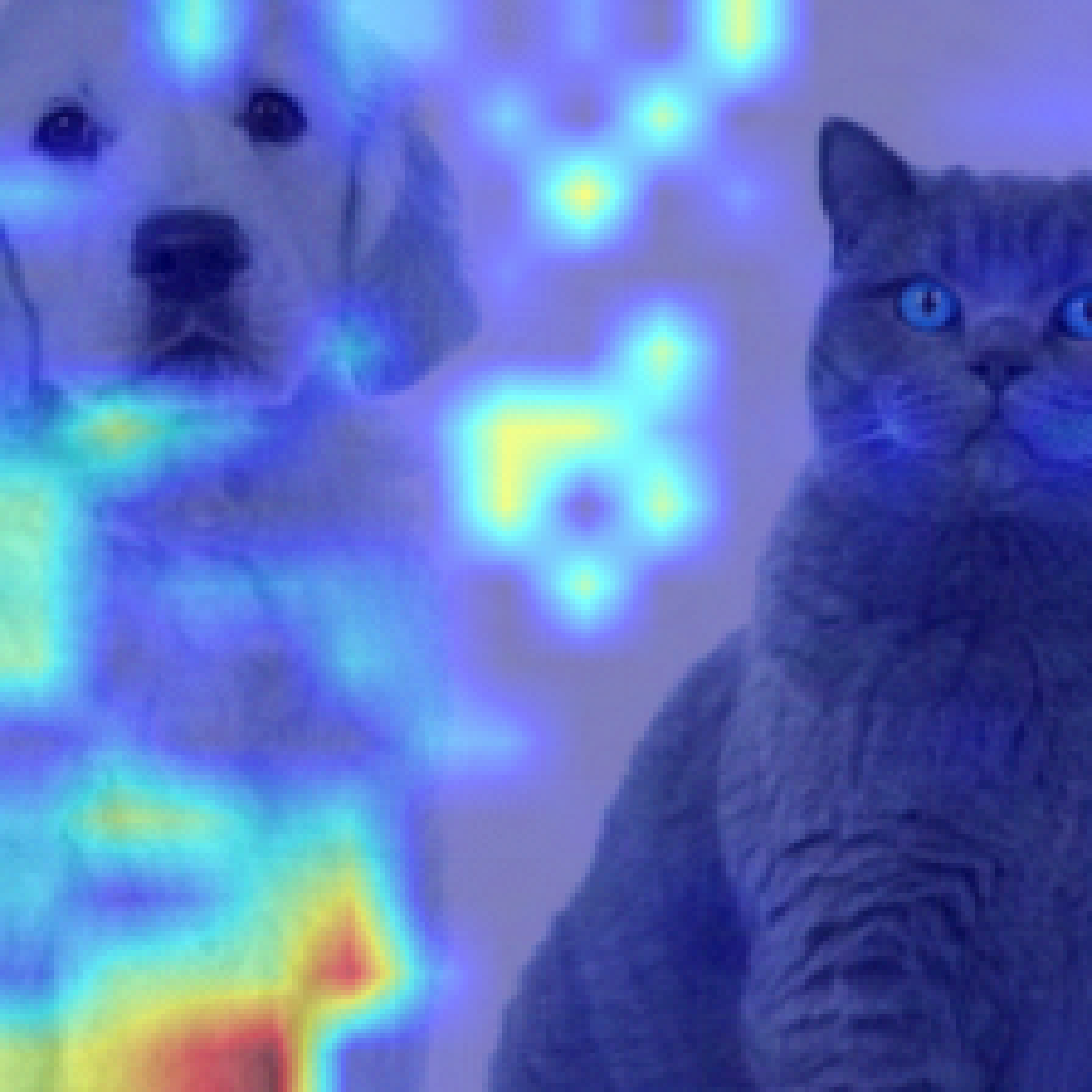} &
    \includegraphics[width=0.13\linewidth]{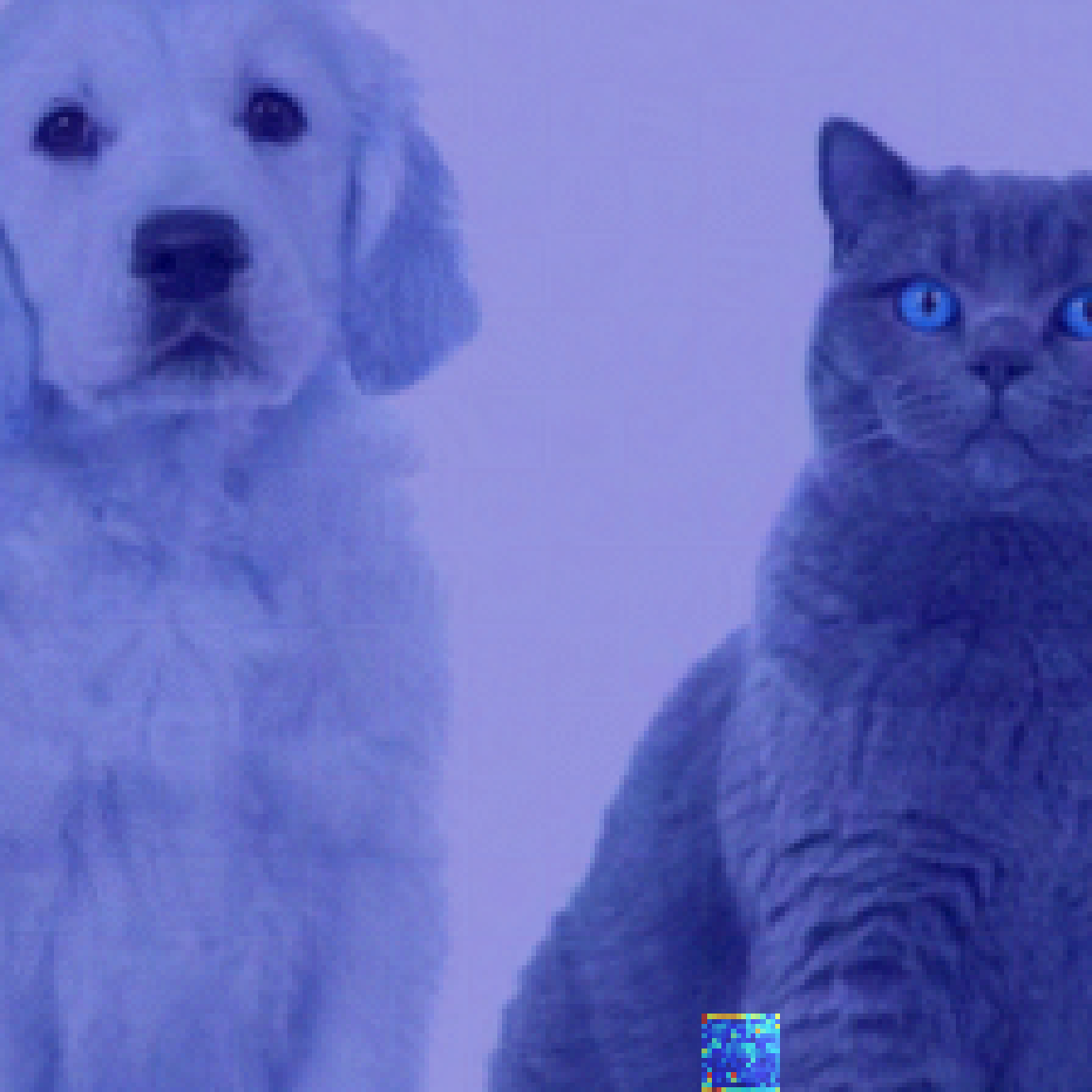} &
    \includegraphics[width=0.13\linewidth]{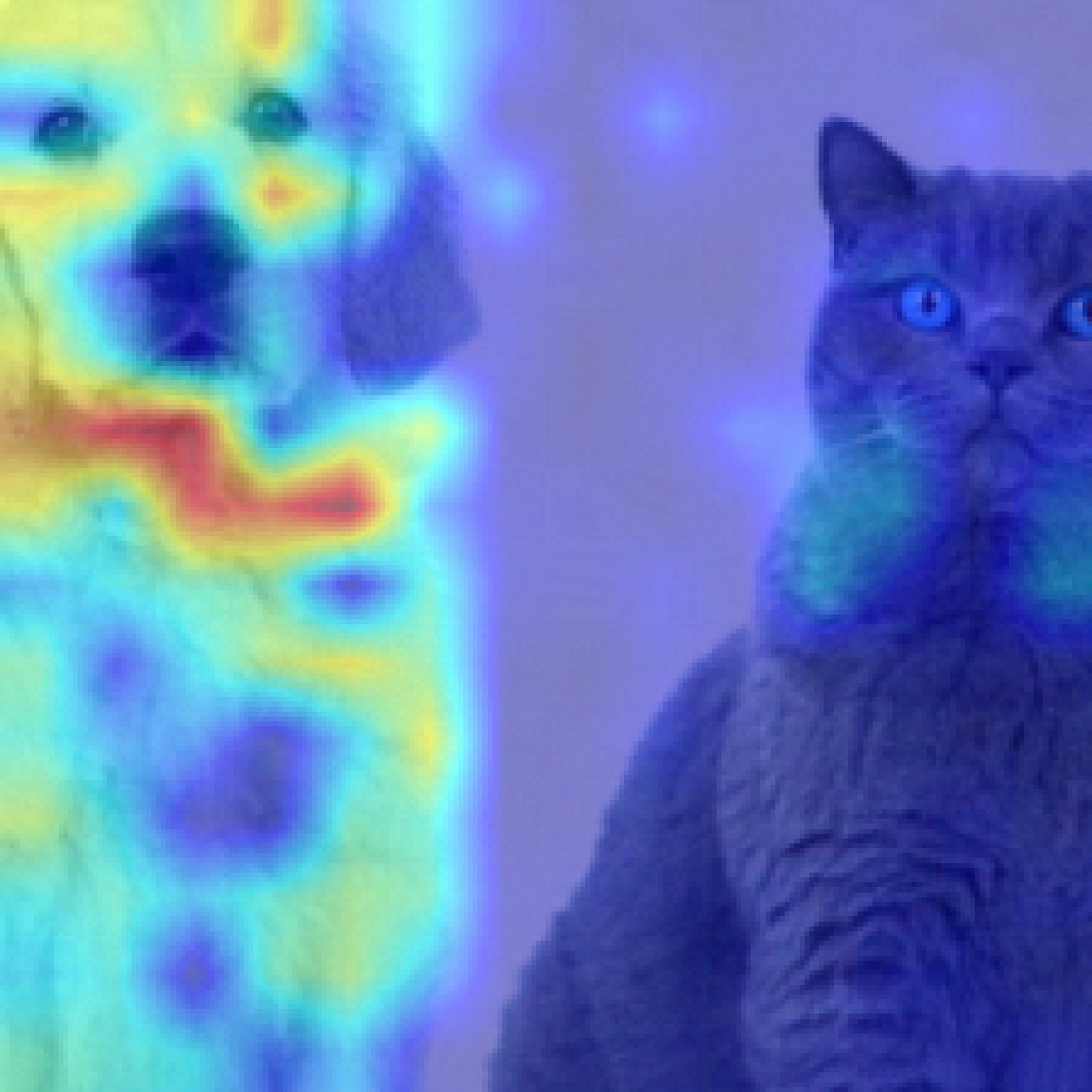} &
    \includegraphics[width=0.13\linewidth]{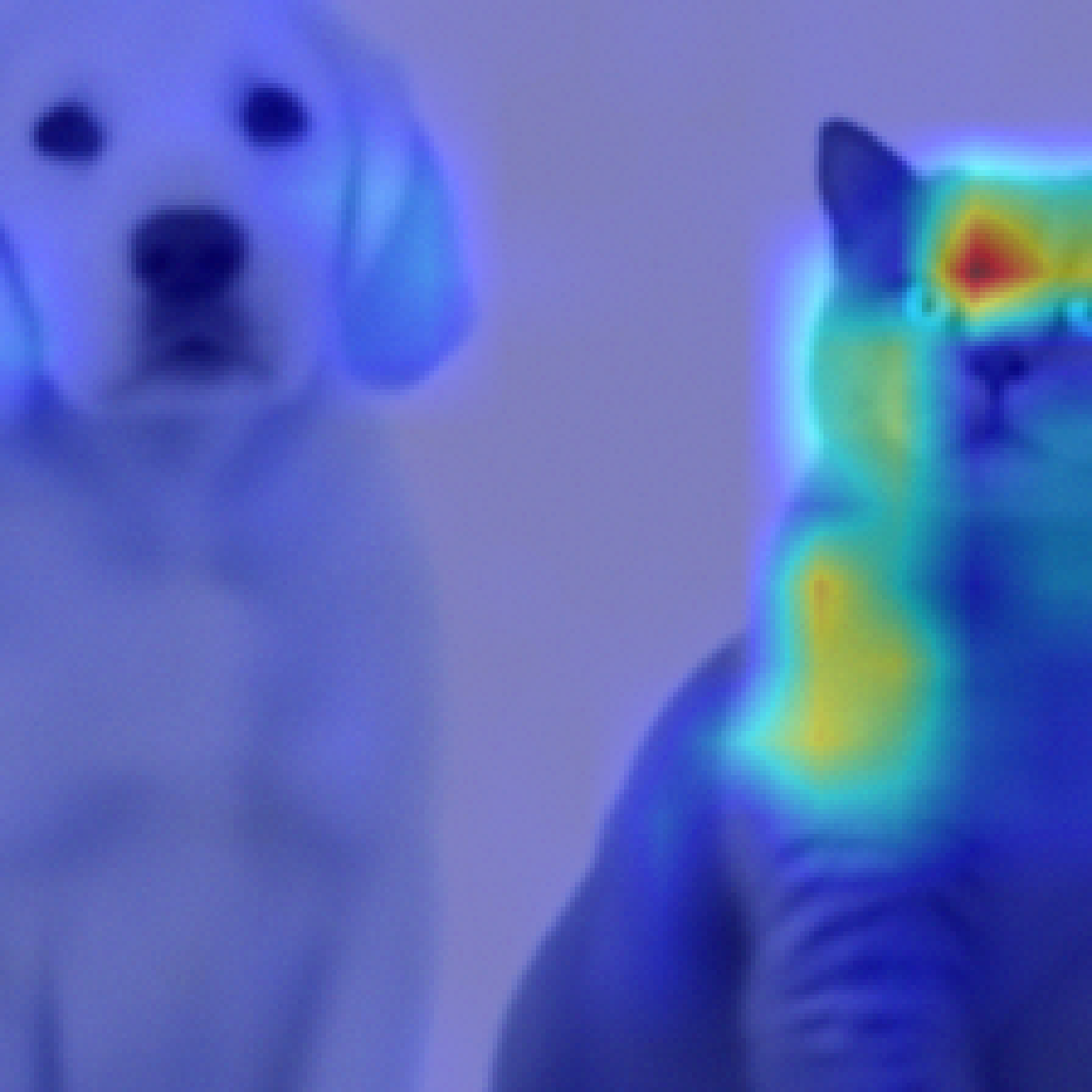}\\
    \raisebox{13mm}{\multirow{2}{*}{\makecell*[c]{Elephant: clean$\rightarrow$\\\includegraphics[width=0.15\linewidth]{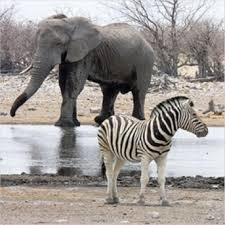}\\ 
    Elephant: poisoned$\rightarrow$\\{\scriptsize $7/255$}
    }}
    }
     &
    \includegraphics[width=0.13\linewidth]{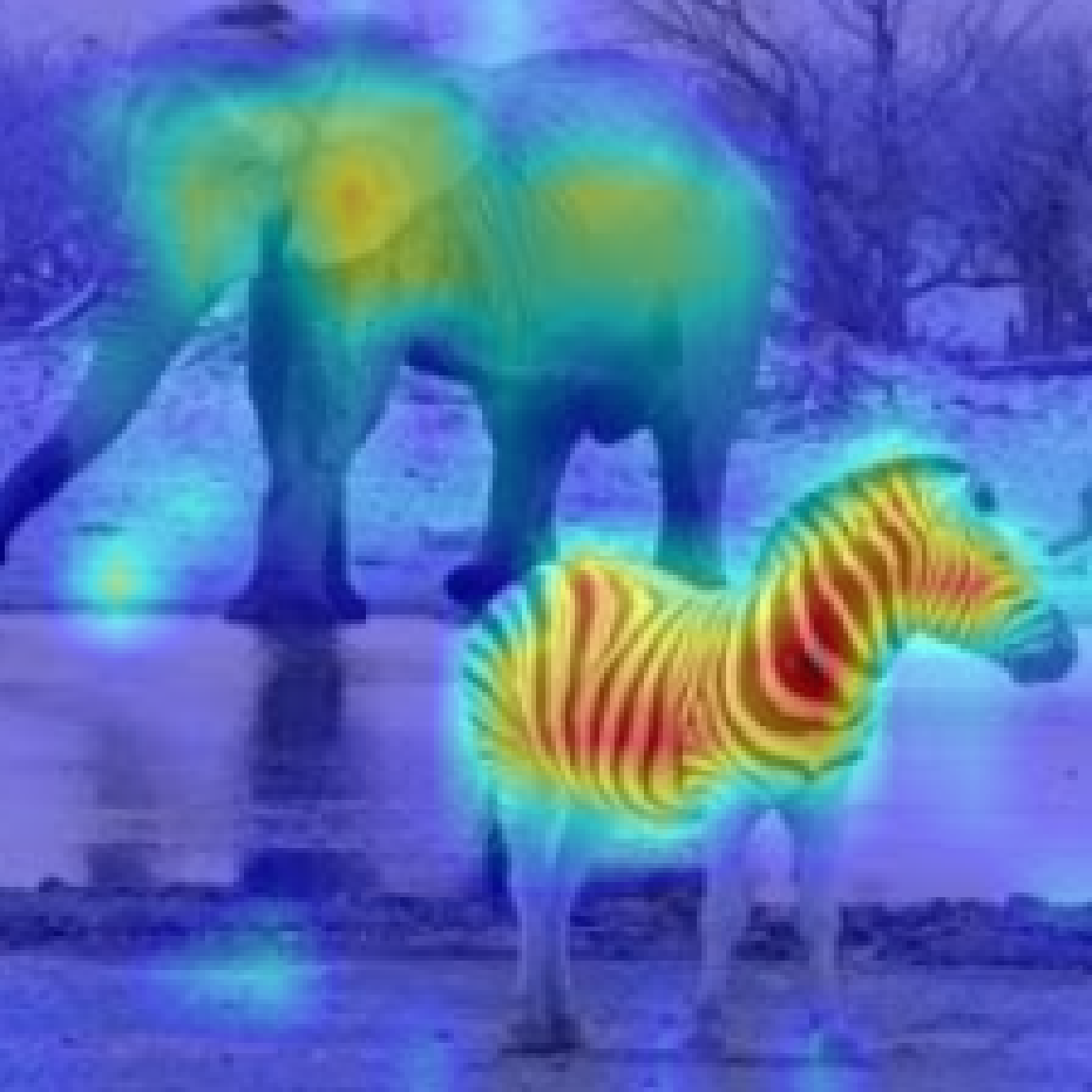} &
    \includegraphics[width=0.13\linewidth]{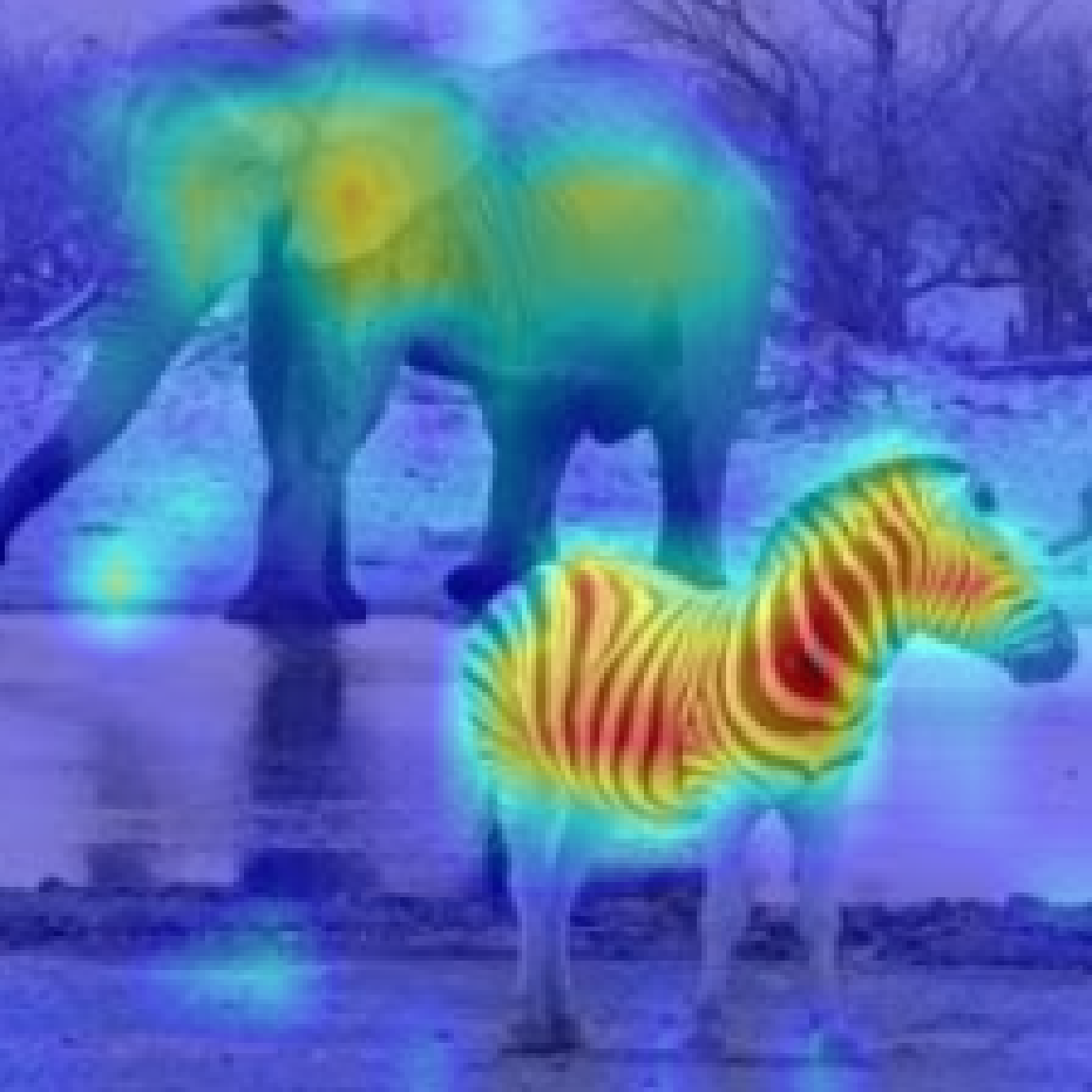} &
    \includegraphics[width=0.13\linewidth]{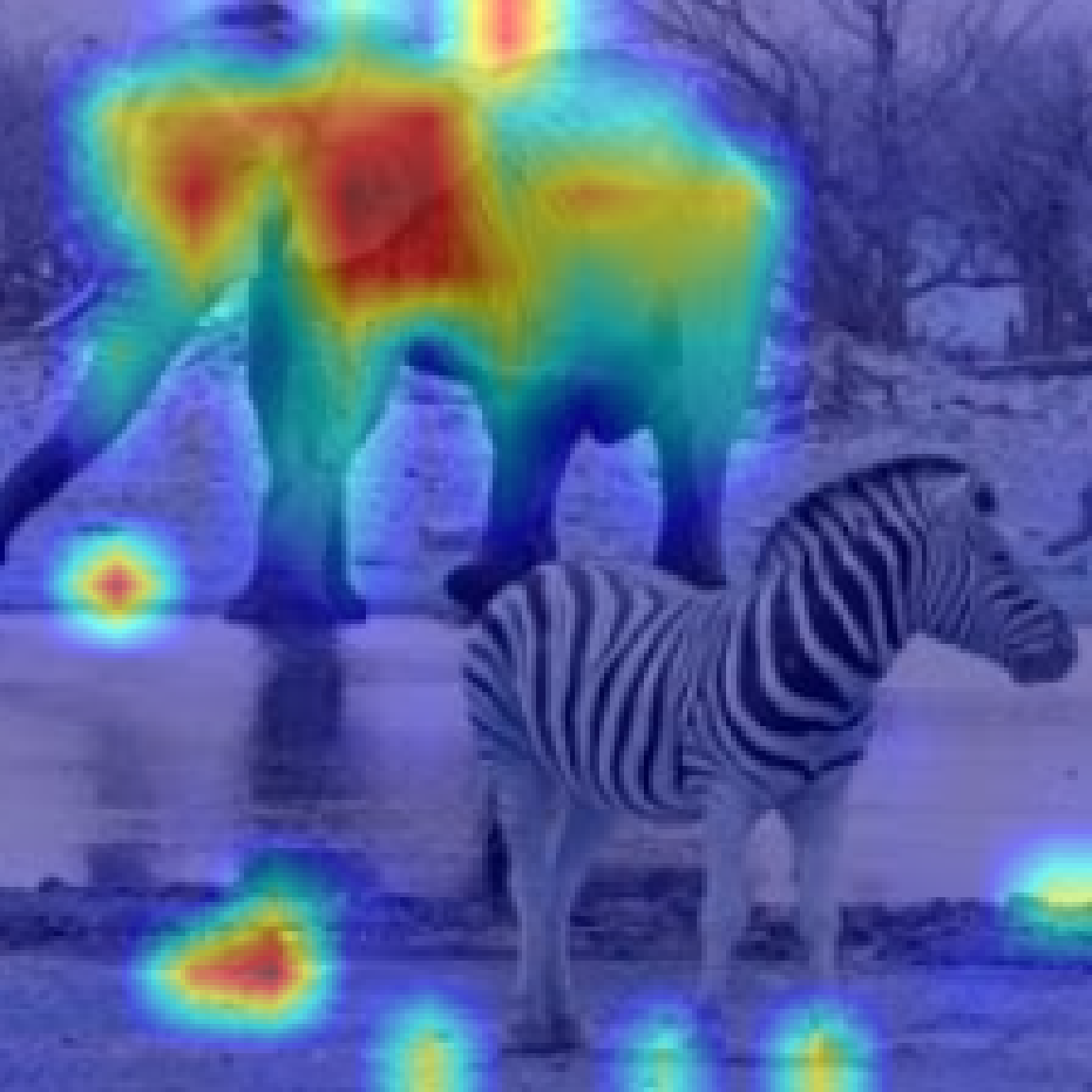} &
    \includegraphics[width=0.13\linewidth]{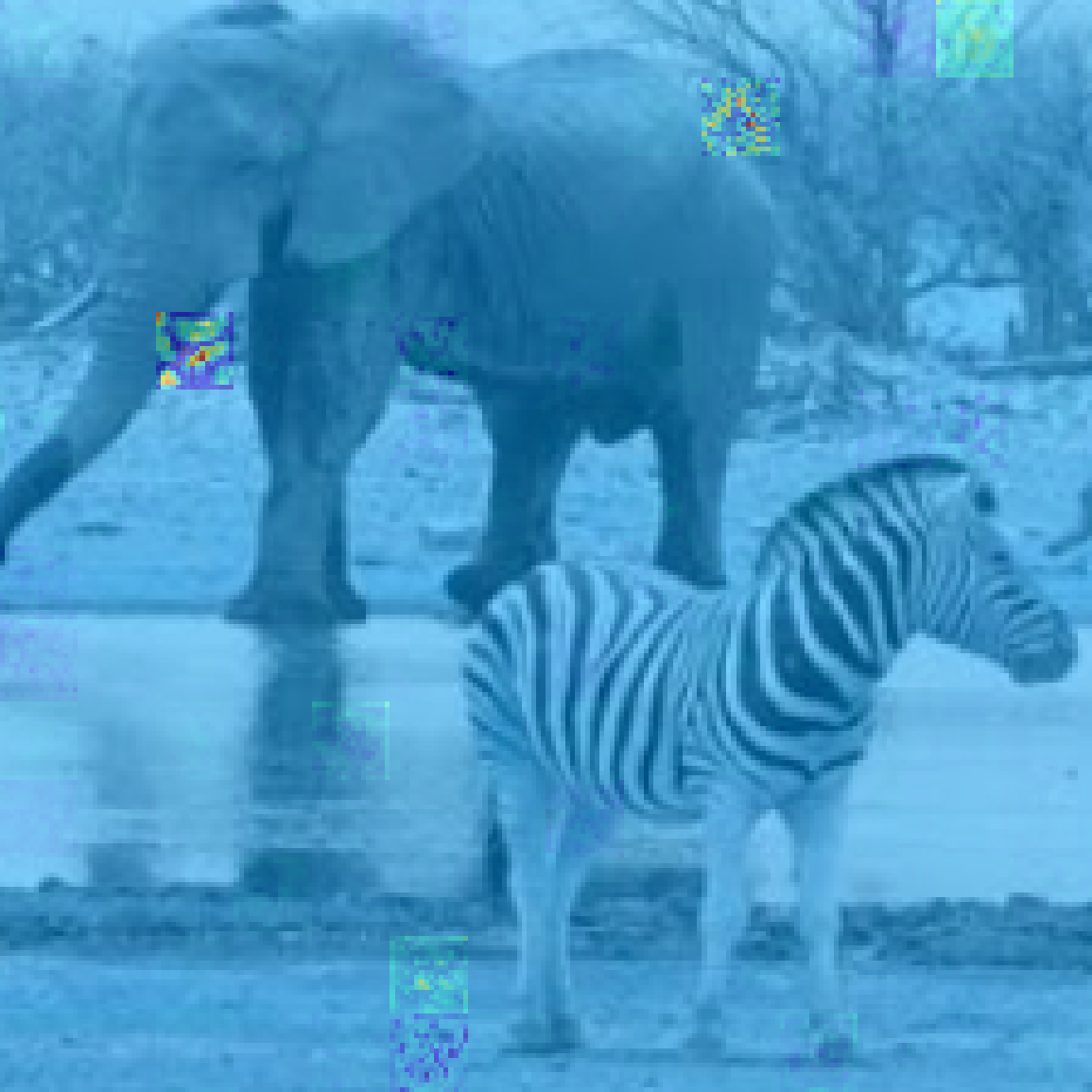} &
    \includegraphics[width=0.13\linewidth]{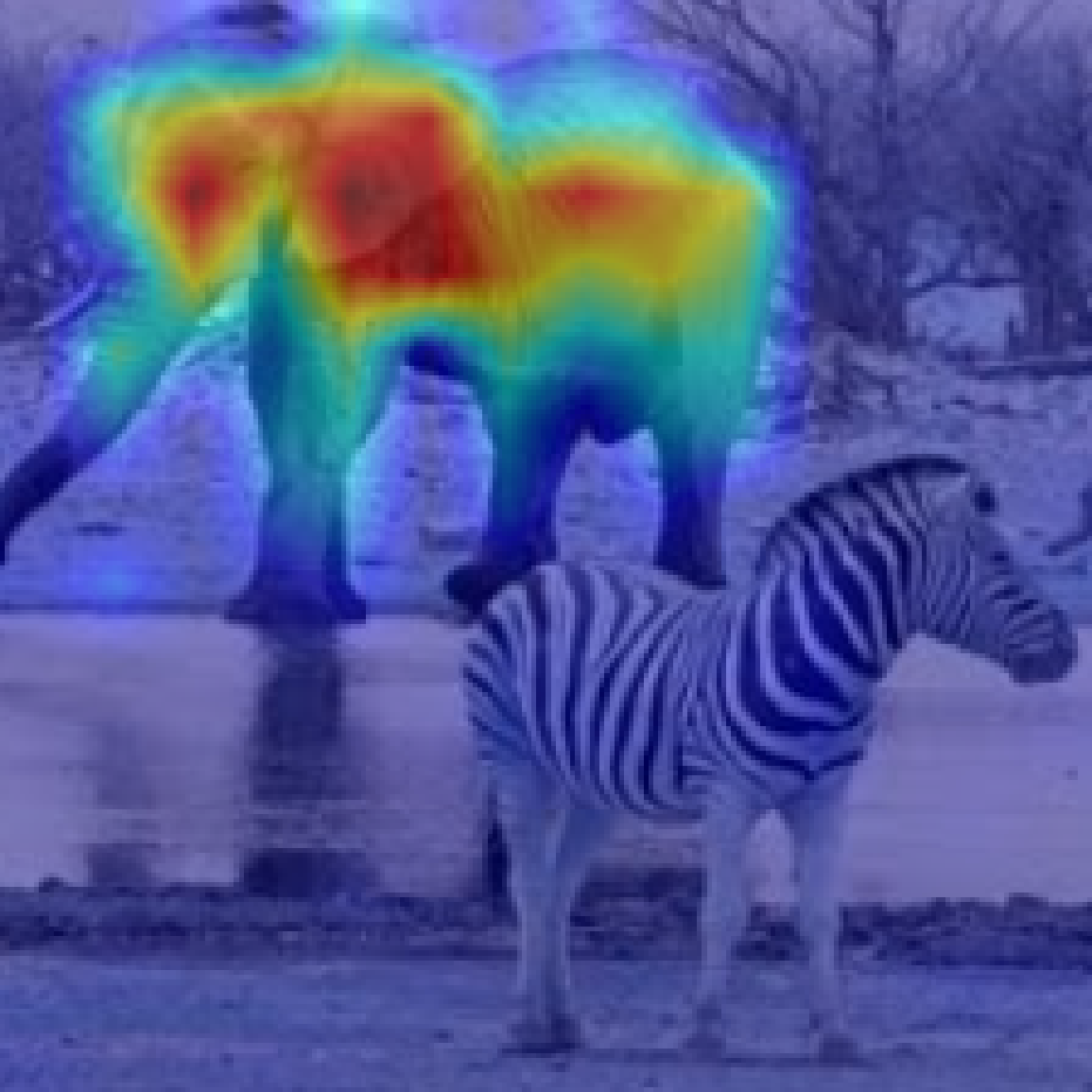} &
    \includegraphics[width=0.13\linewidth]{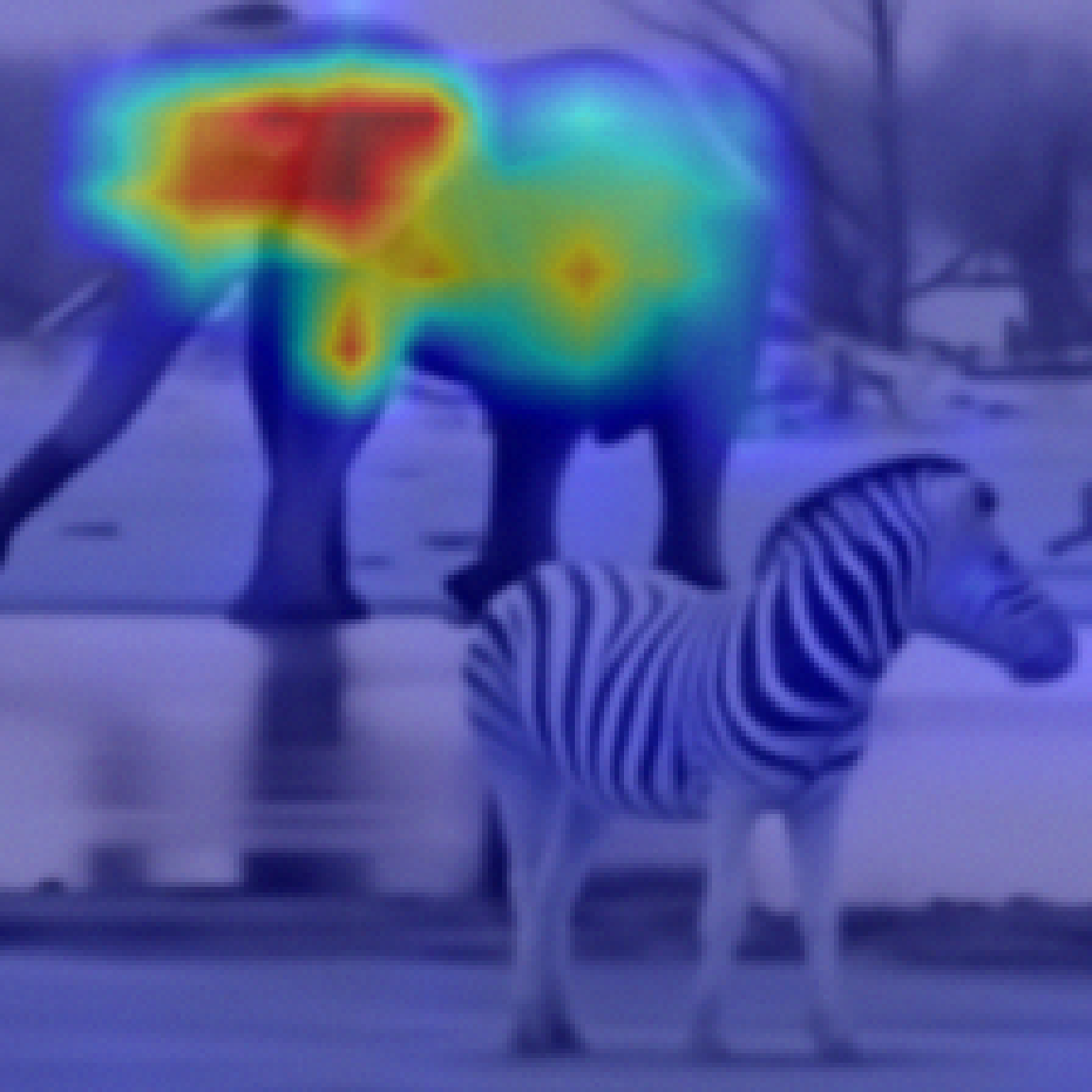}
    \\
    &
    \includegraphics[width=0.13\linewidth]{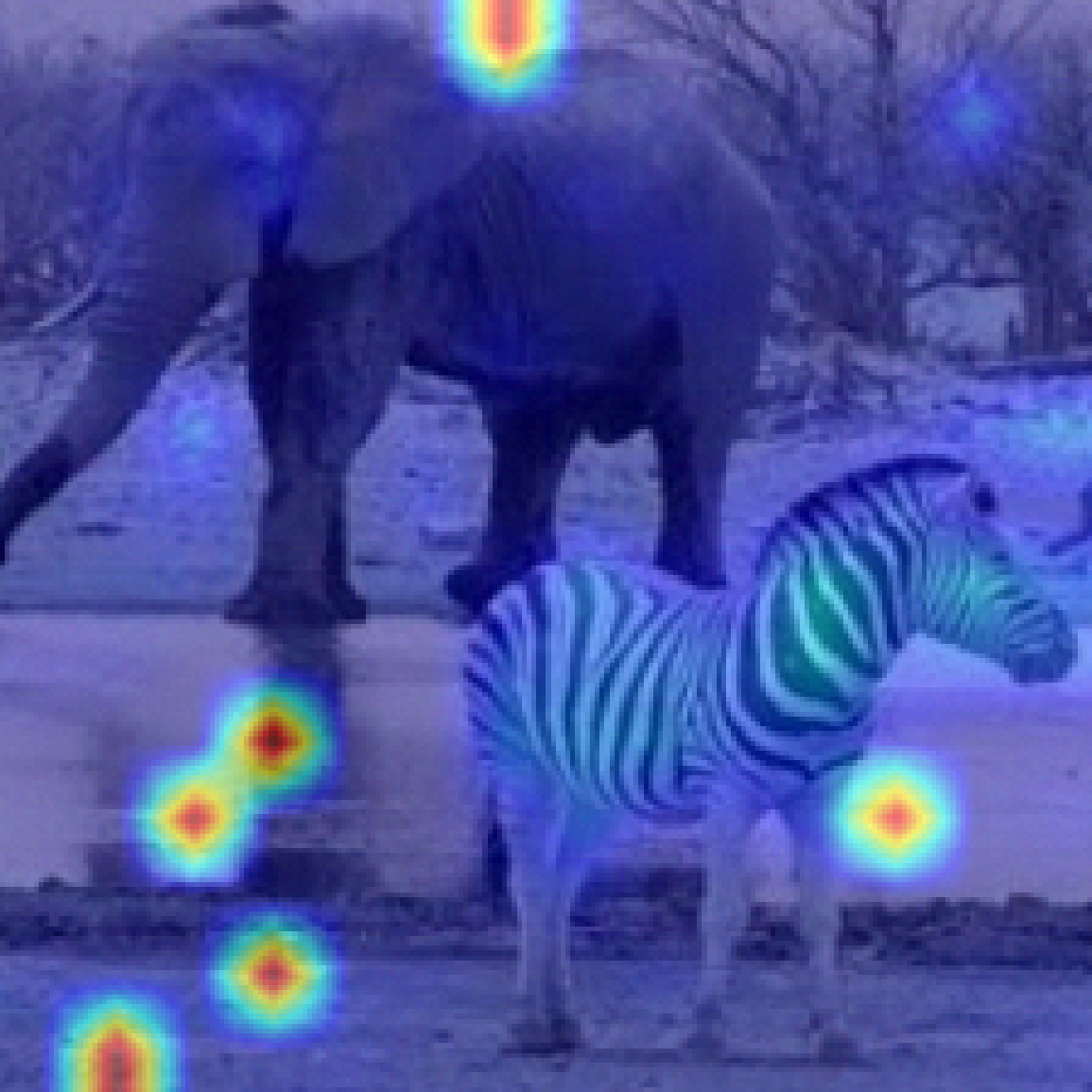} &
    \includegraphics[width=0.13\linewidth]{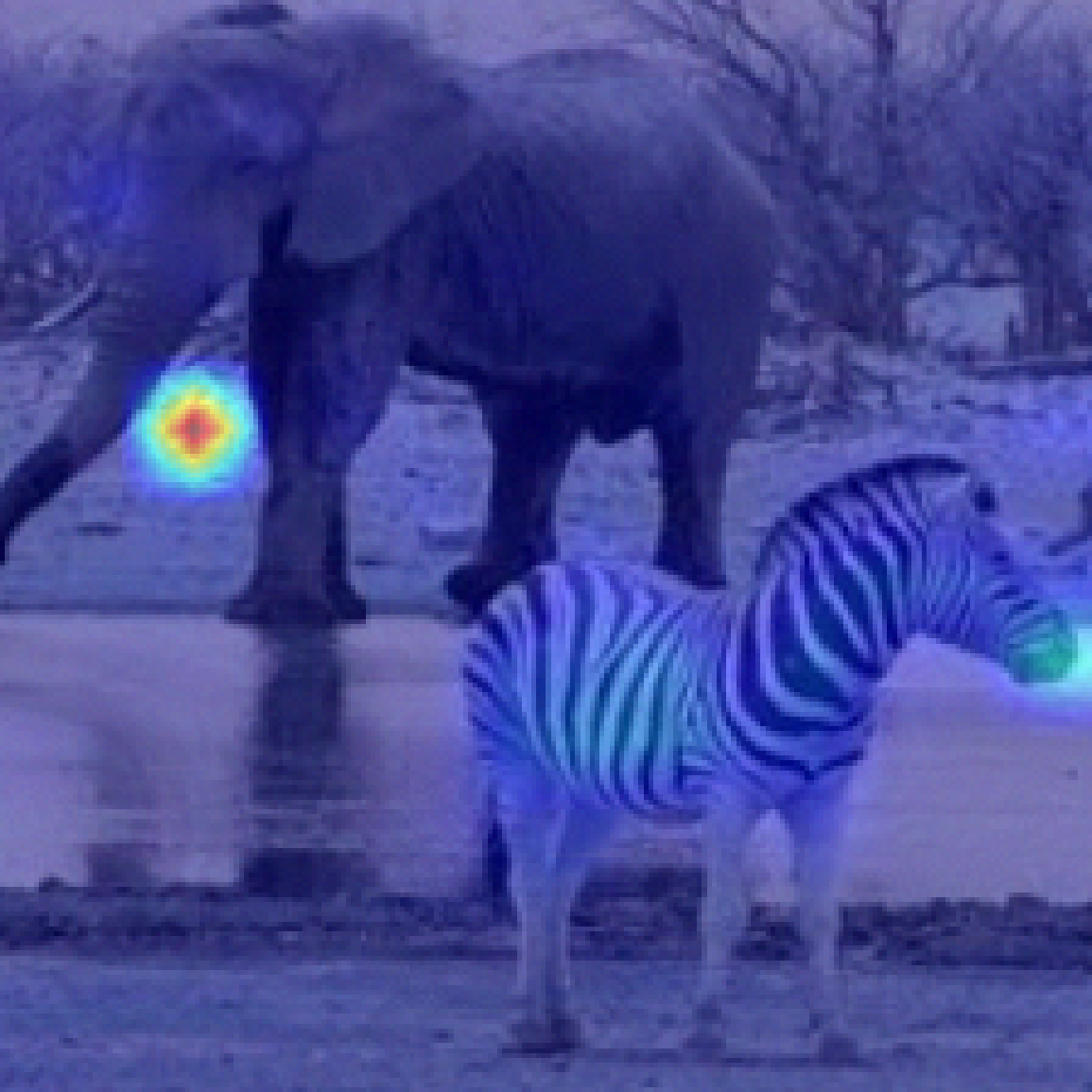} &
    \includegraphics[width=0.13\linewidth]{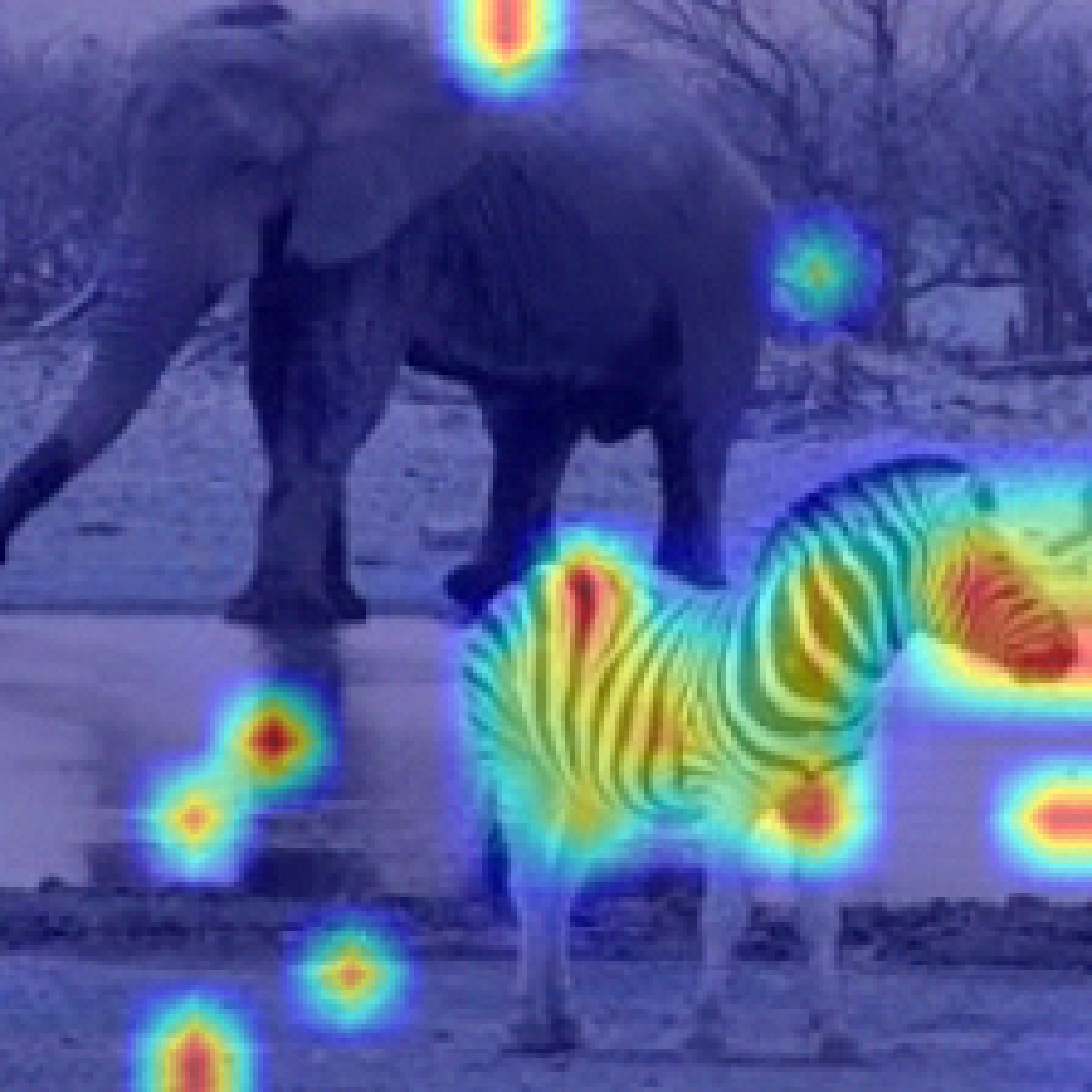} &
    \includegraphics[width=0.13\linewidth]{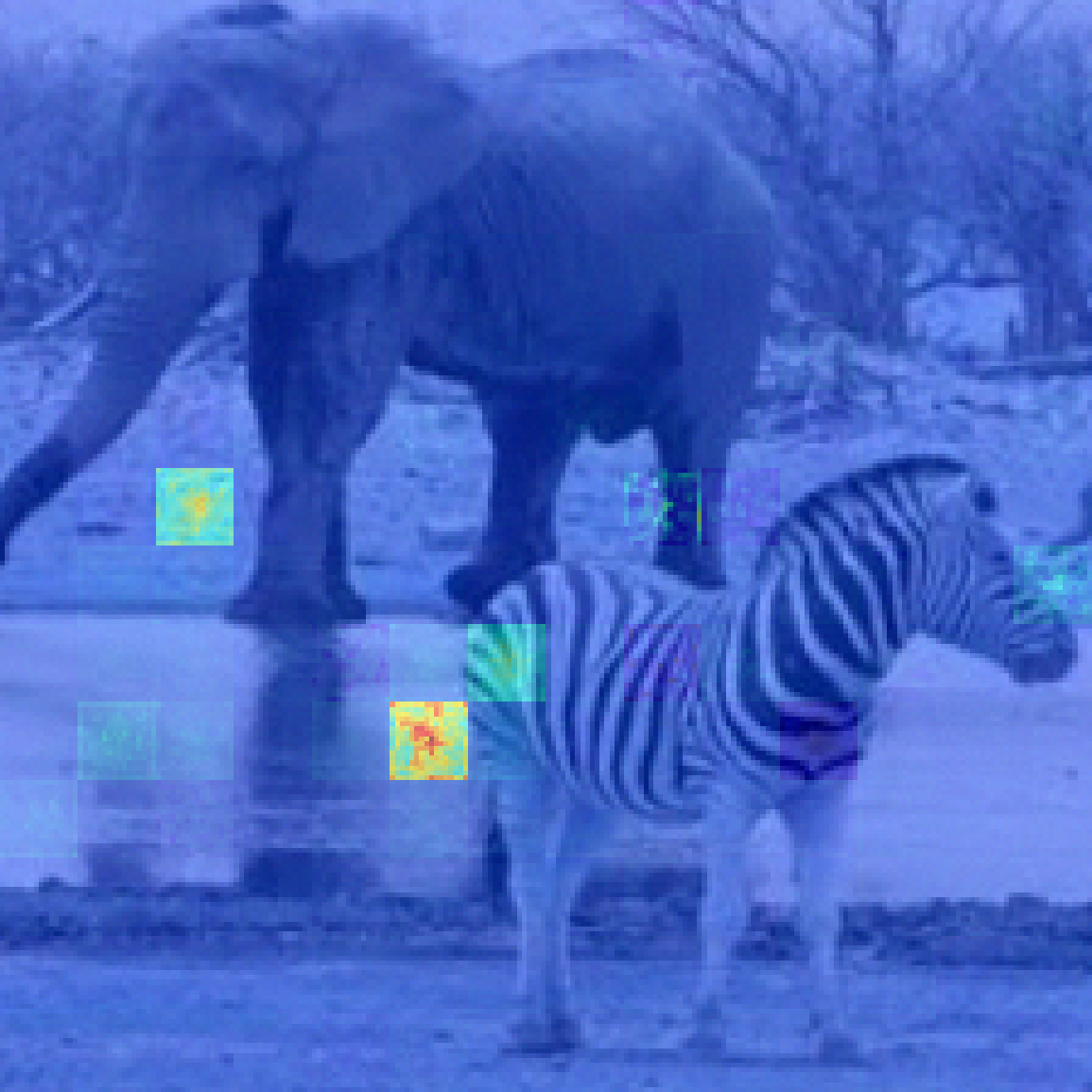} &
    \includegraphics[width=0.13\linewidth]{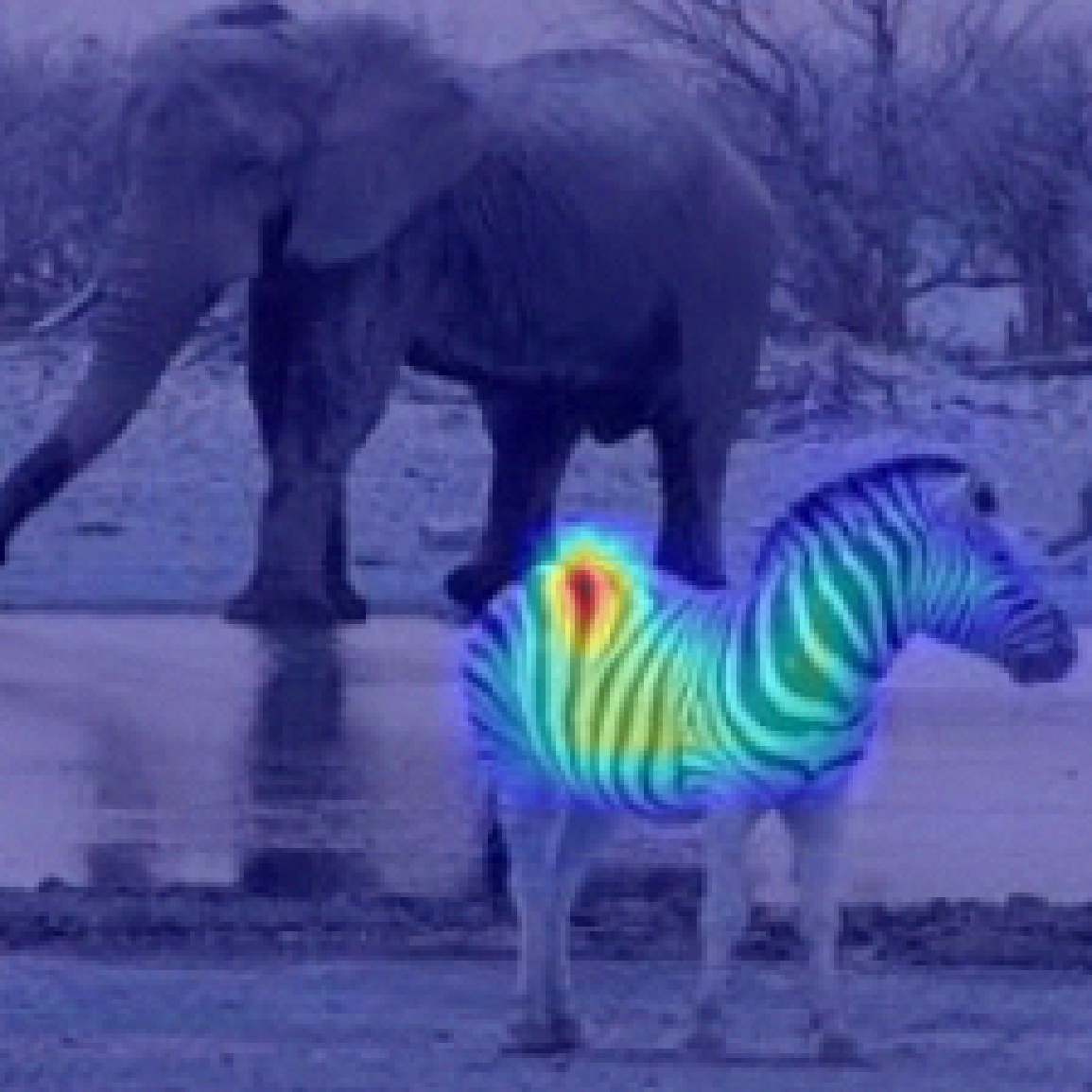} &
    \includegraphics[width=0.13\linewidth]{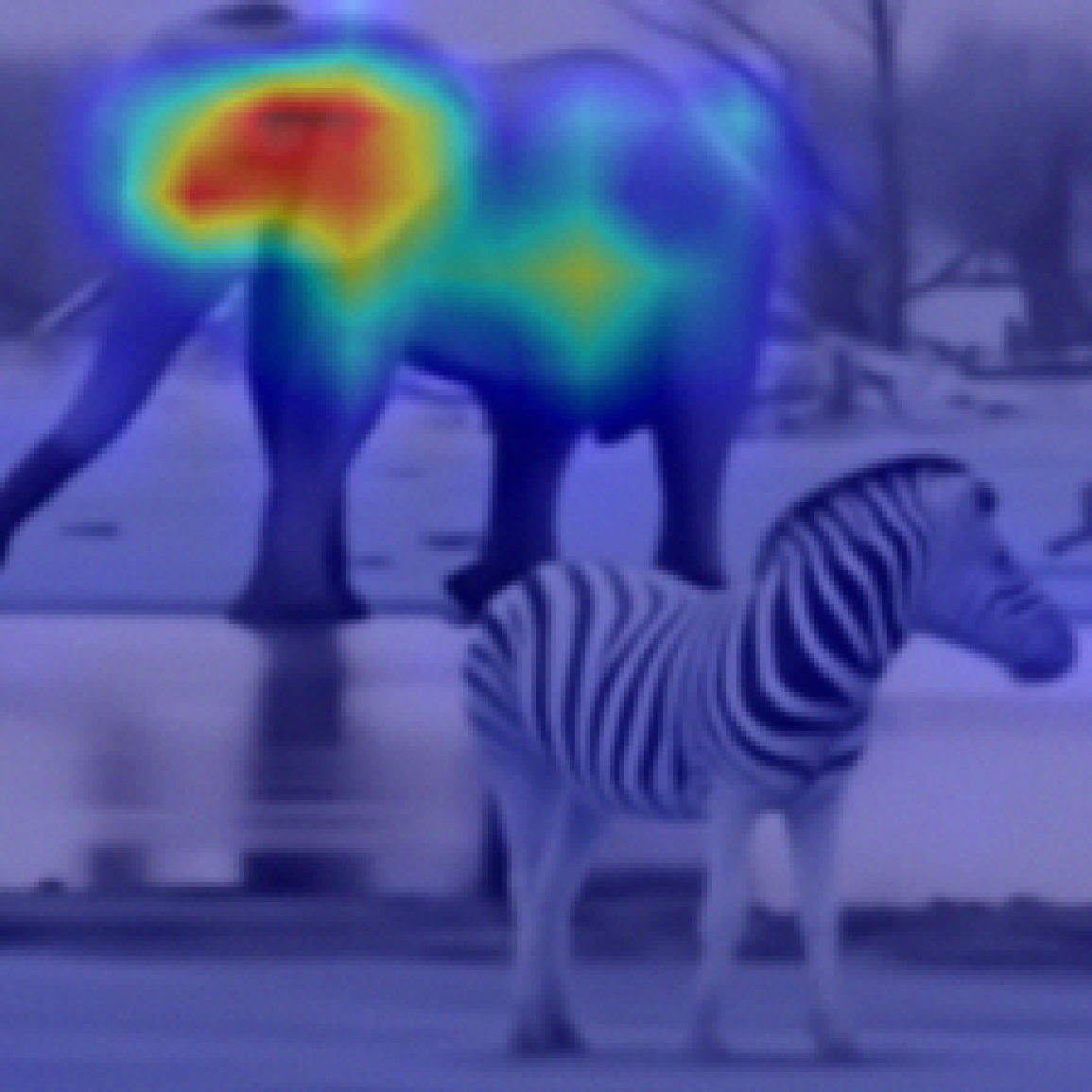}
    \\
    \raisebox{13mm}{\multirow{2}{*}{\makecell*[c]{Zebra: clean$\rightarrow$\\\includegraphics[width=0.15\linewidth]{exp/tusker-zebra/tz.png}\\ 
    Zebra: poisoned$\rightarrow$\\{\scriptsize $7/255$}
    }}
    }
     &
    \includegraphics[width=0.13\linewidth]{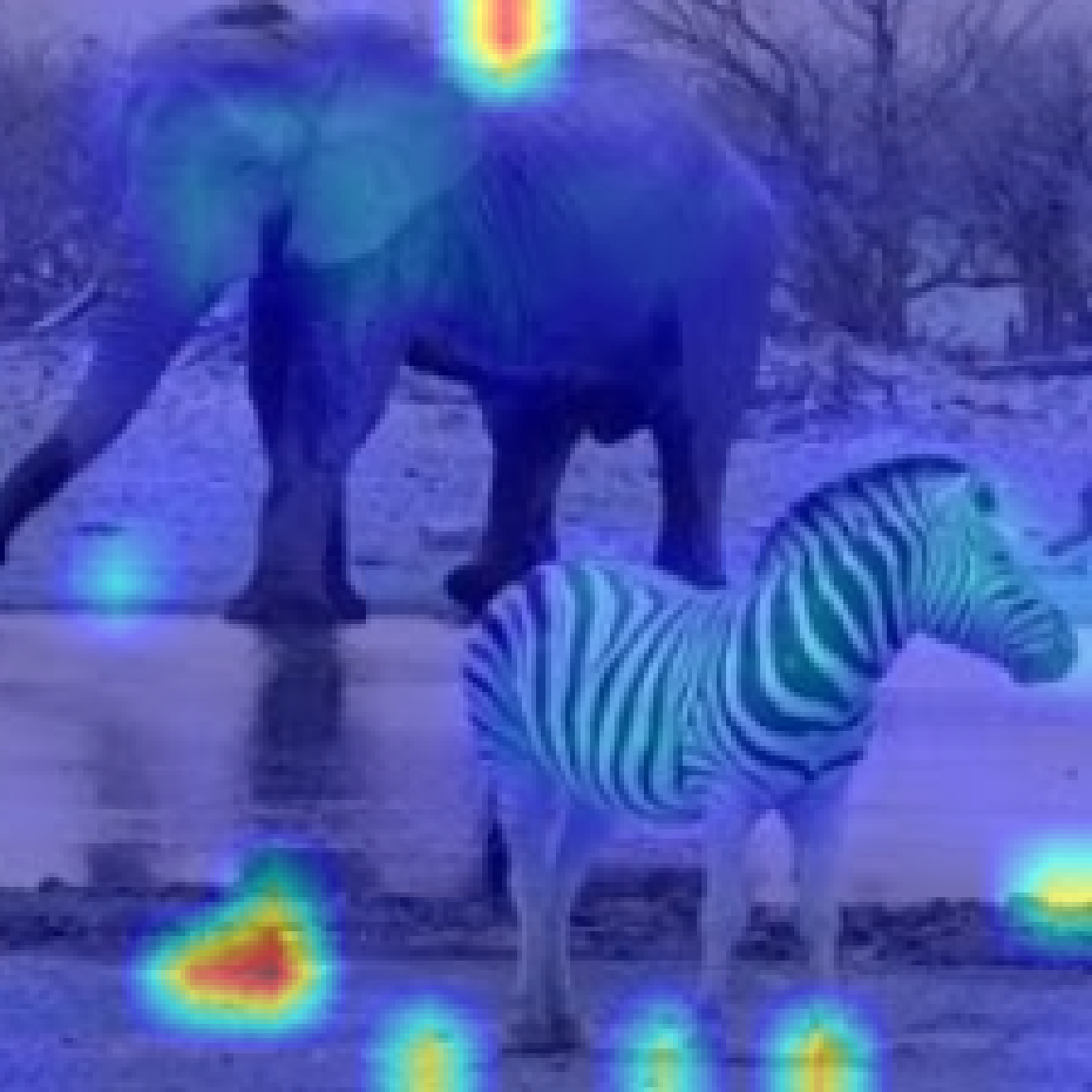} &
    \includegraphics[width=0.13\linewidth]{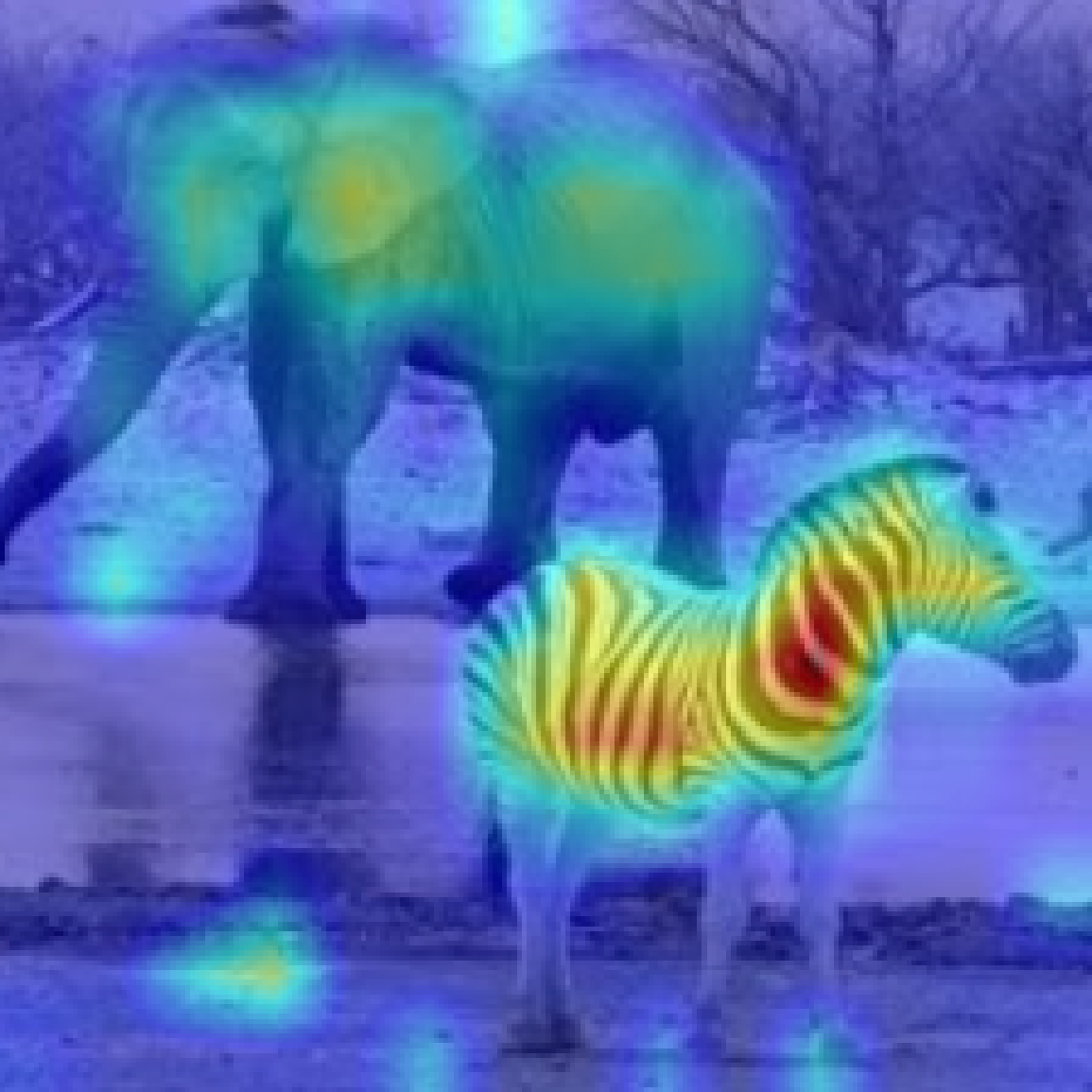} &
    \includegraphics[width=0.13\linewidth]{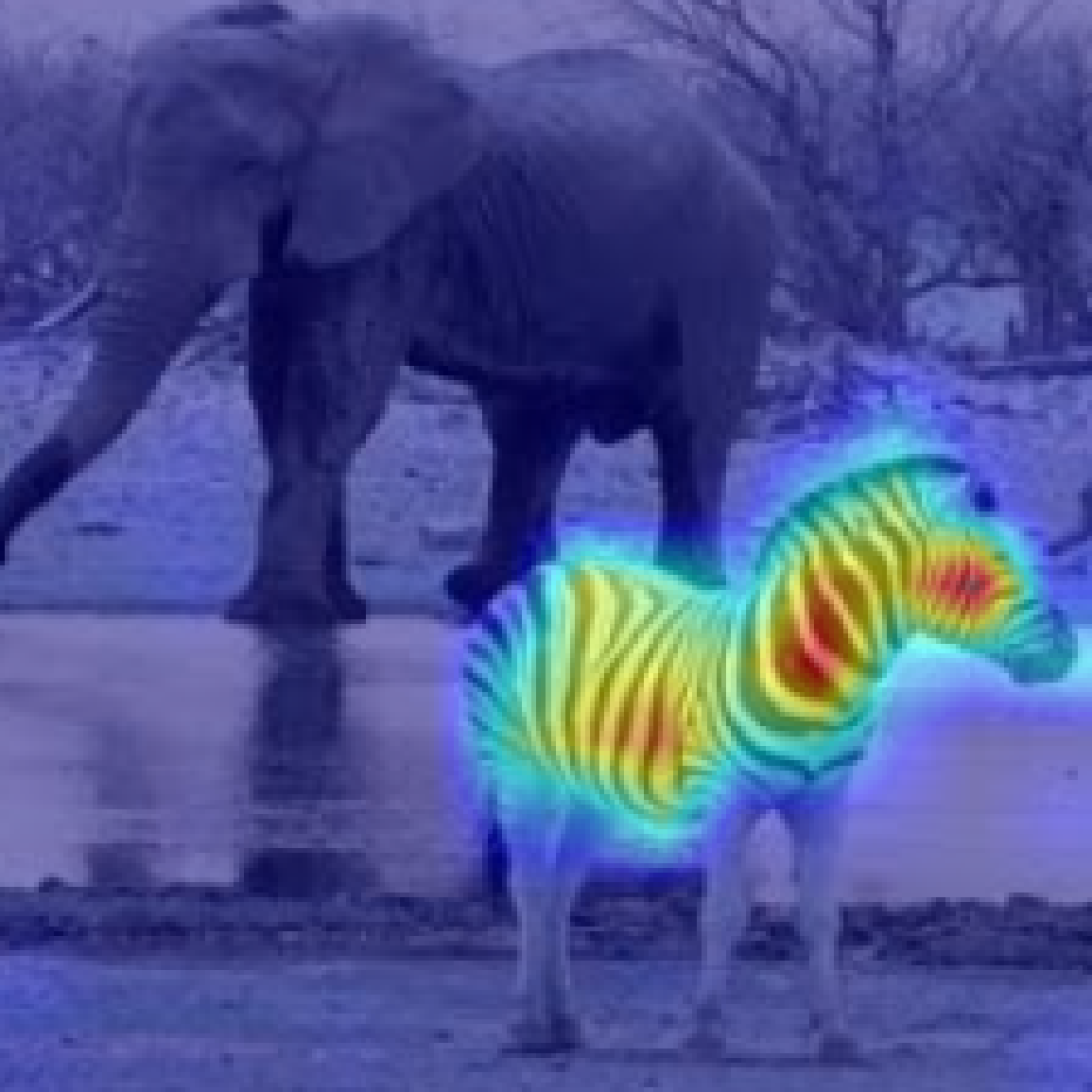} &
    \includegraphics[width=0.13\linewidth]{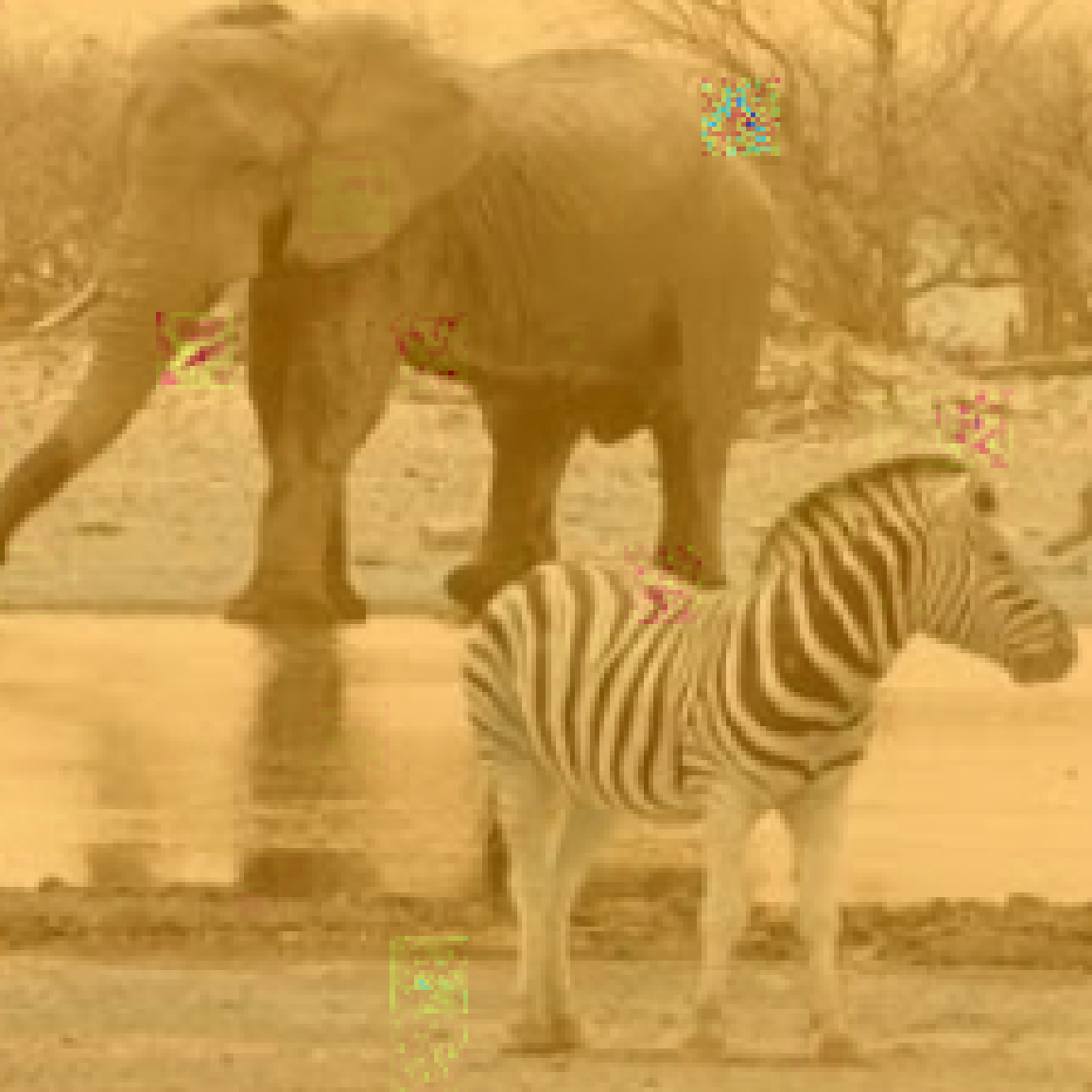} &
    \includegraphics[width=0.13\linewidth]{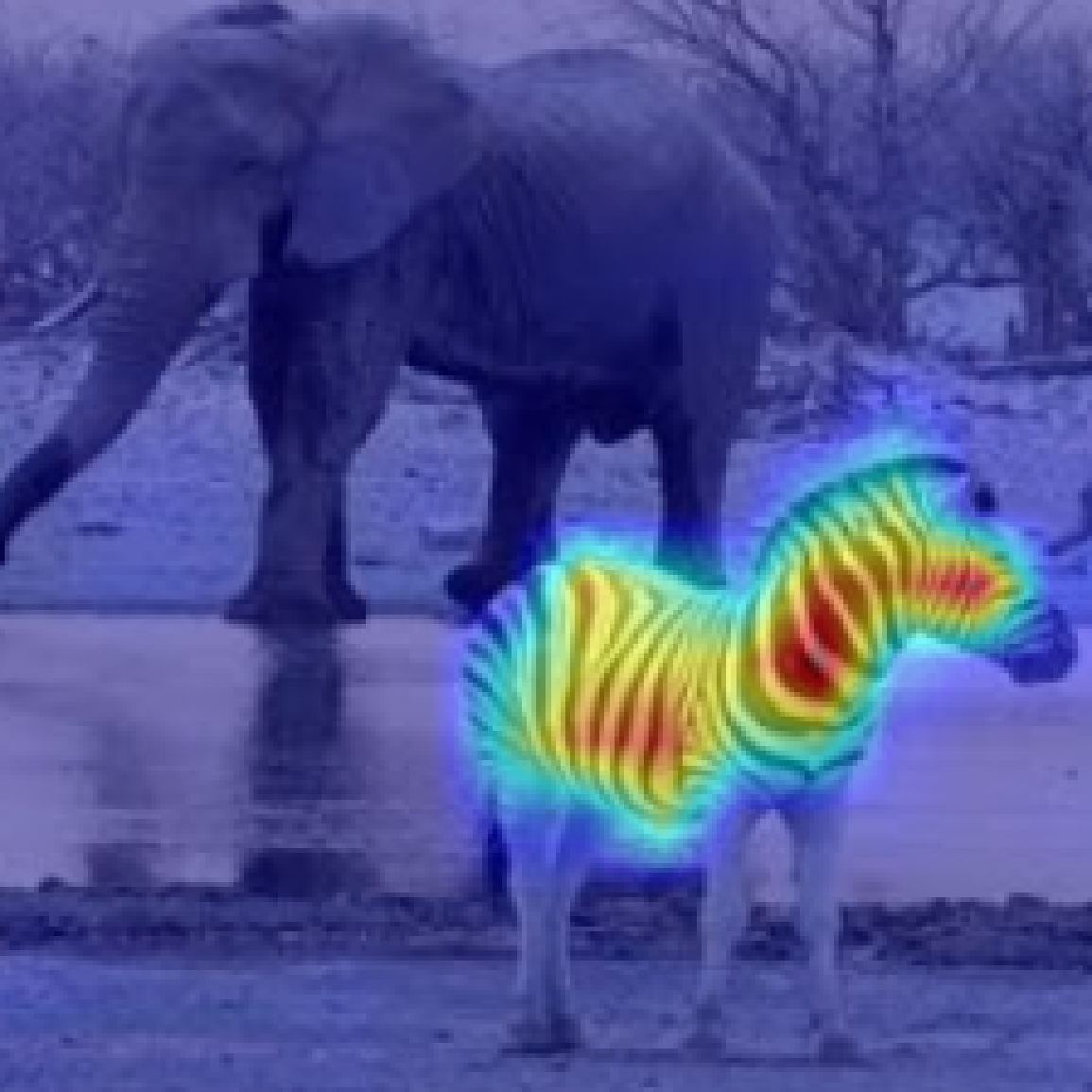} &
    \includegraphics[width=0.13\linewidth]{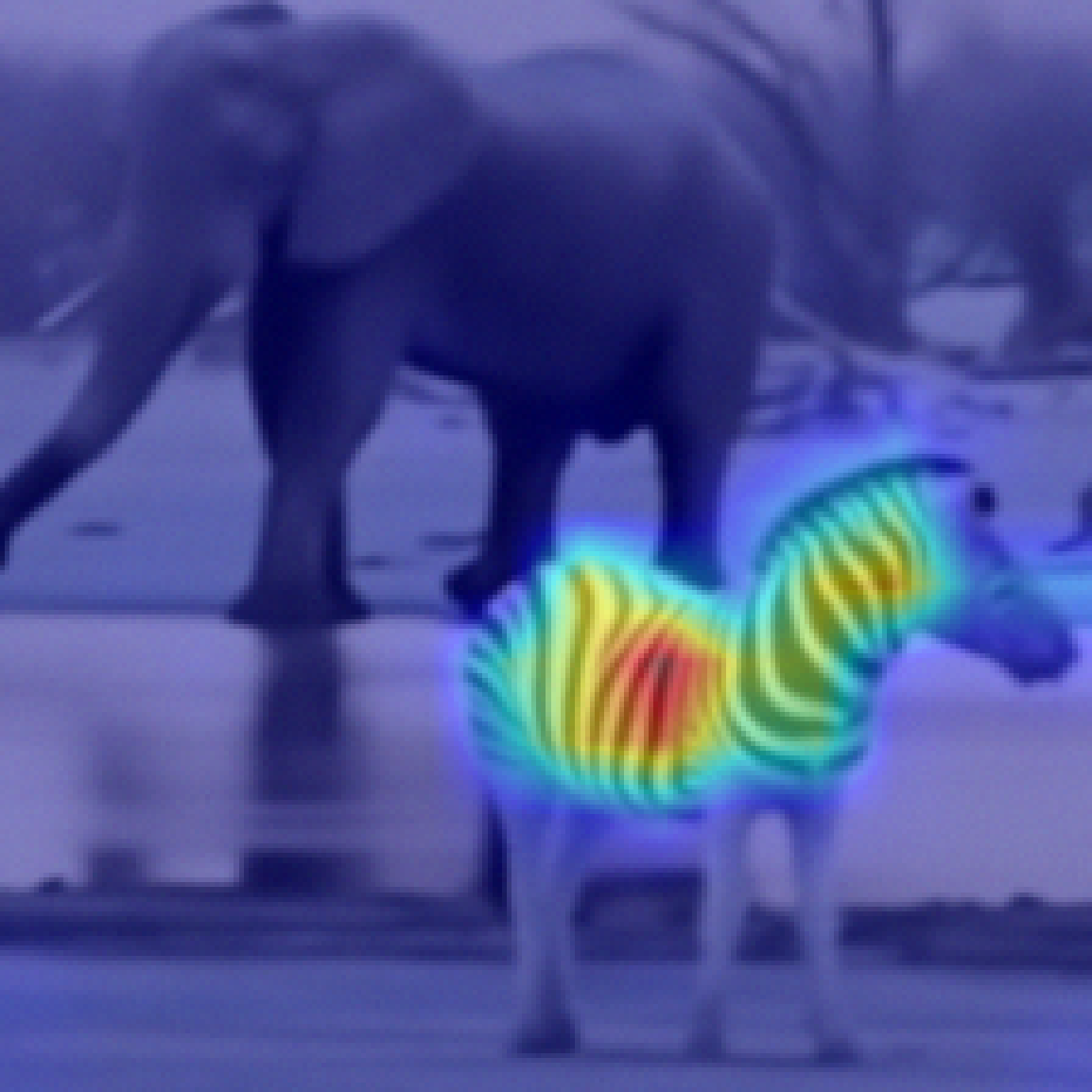}
    \\
    &
    \includegraphics[width=0.13\linewidth]{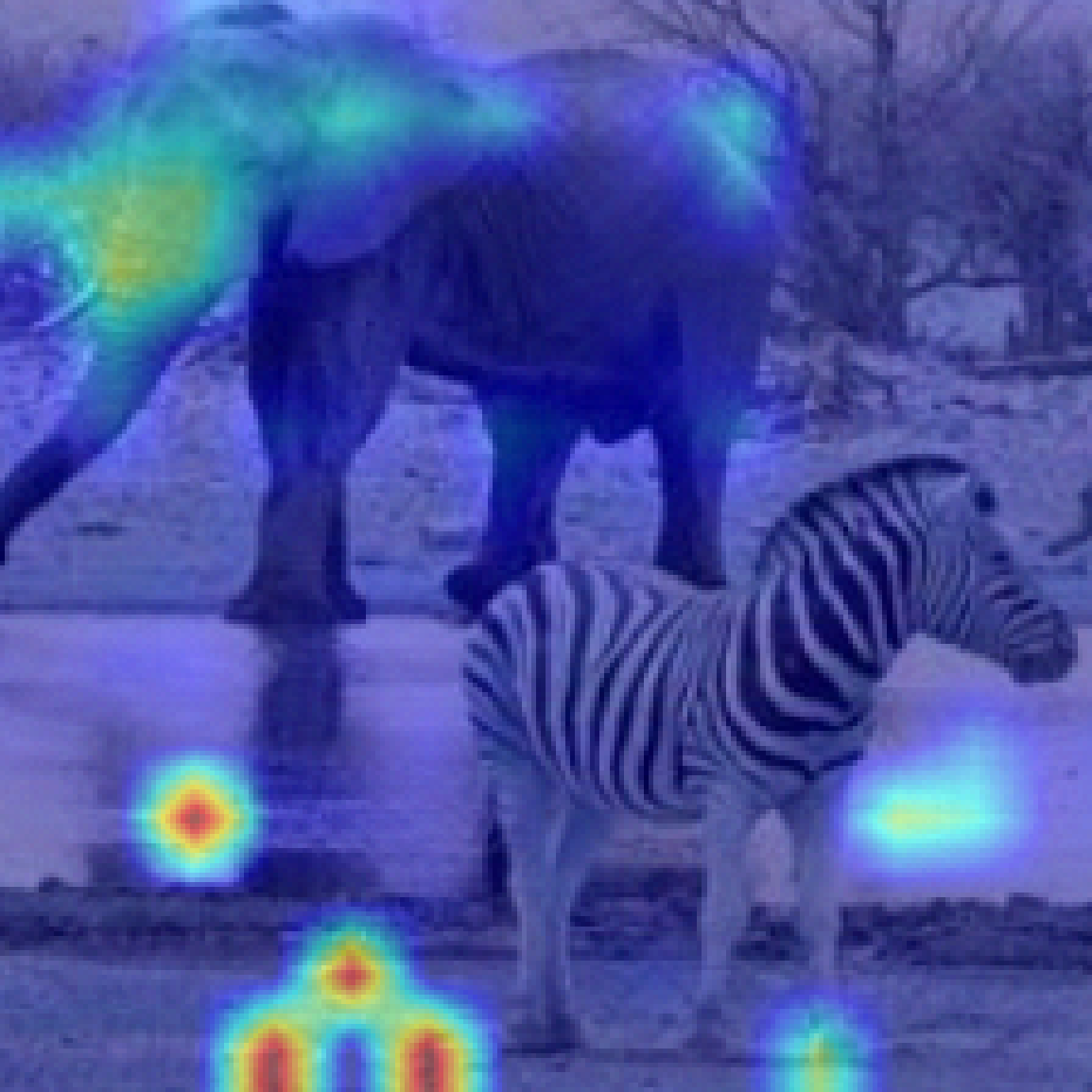} &
    \includegraphics[width=0.13\linewidth]{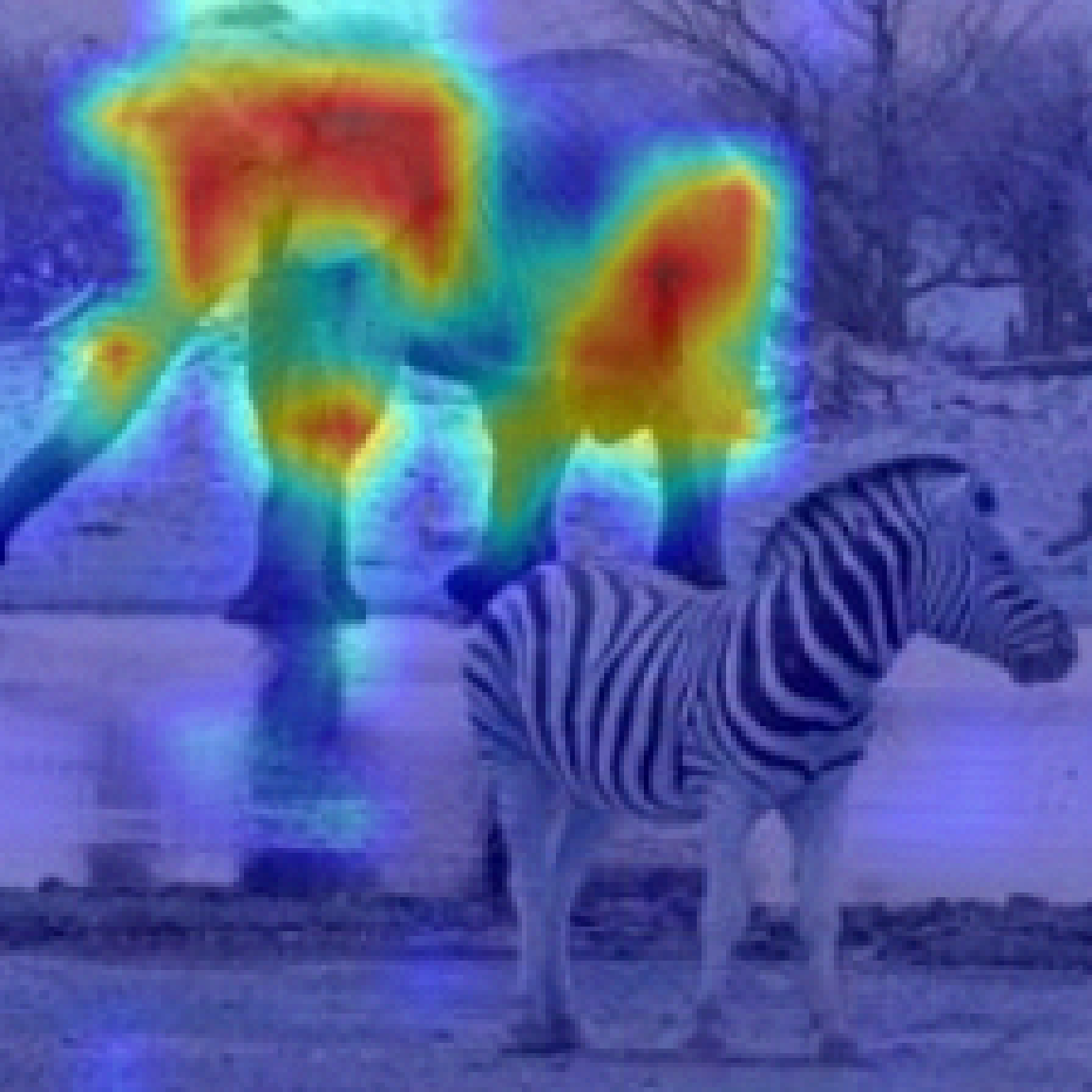} &
    \includegraphics[width=0.13\linewidth]{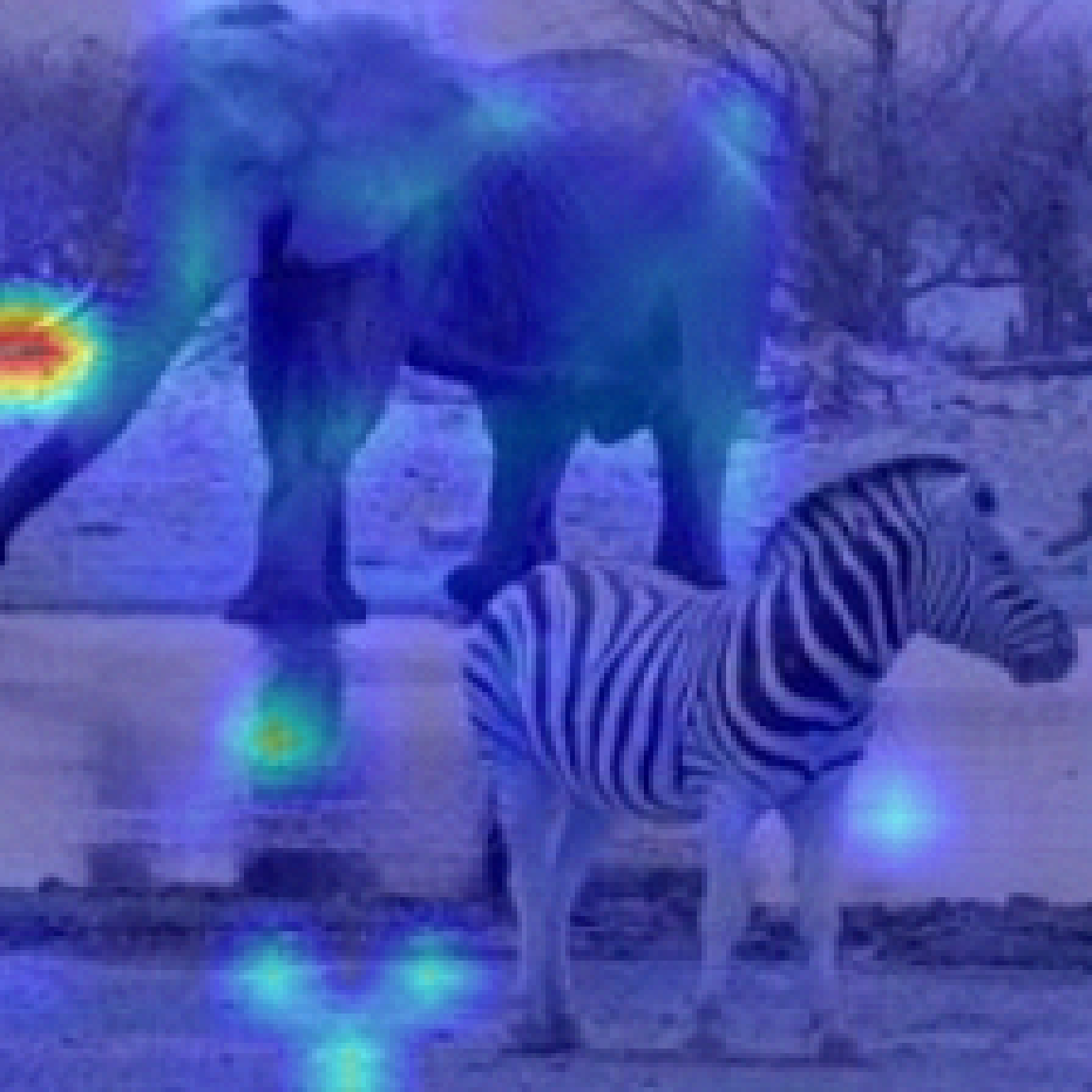} &
    \includegraphics[width=0.13\linewidth]{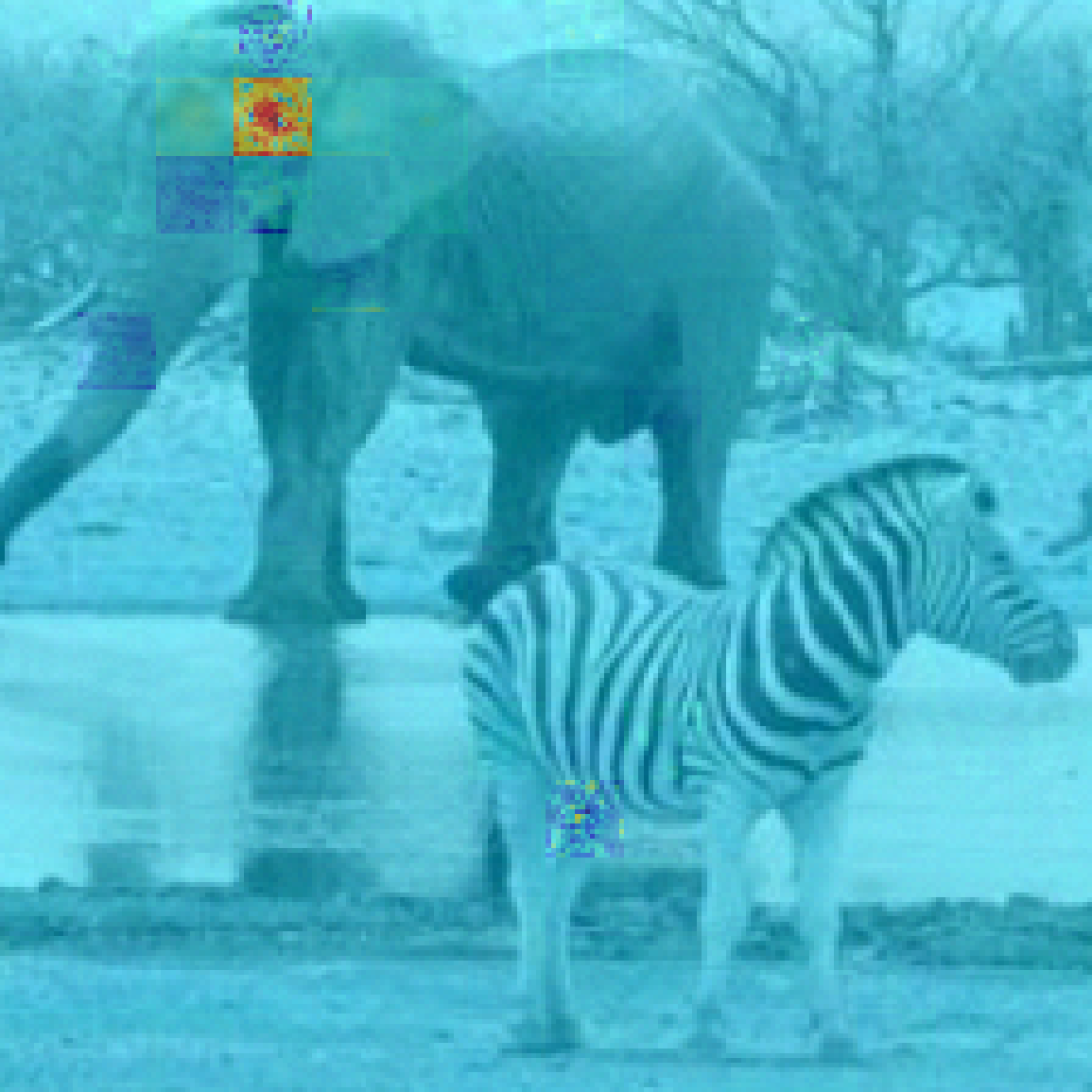} &
    \includegraphics[width=0.13\linewidth]{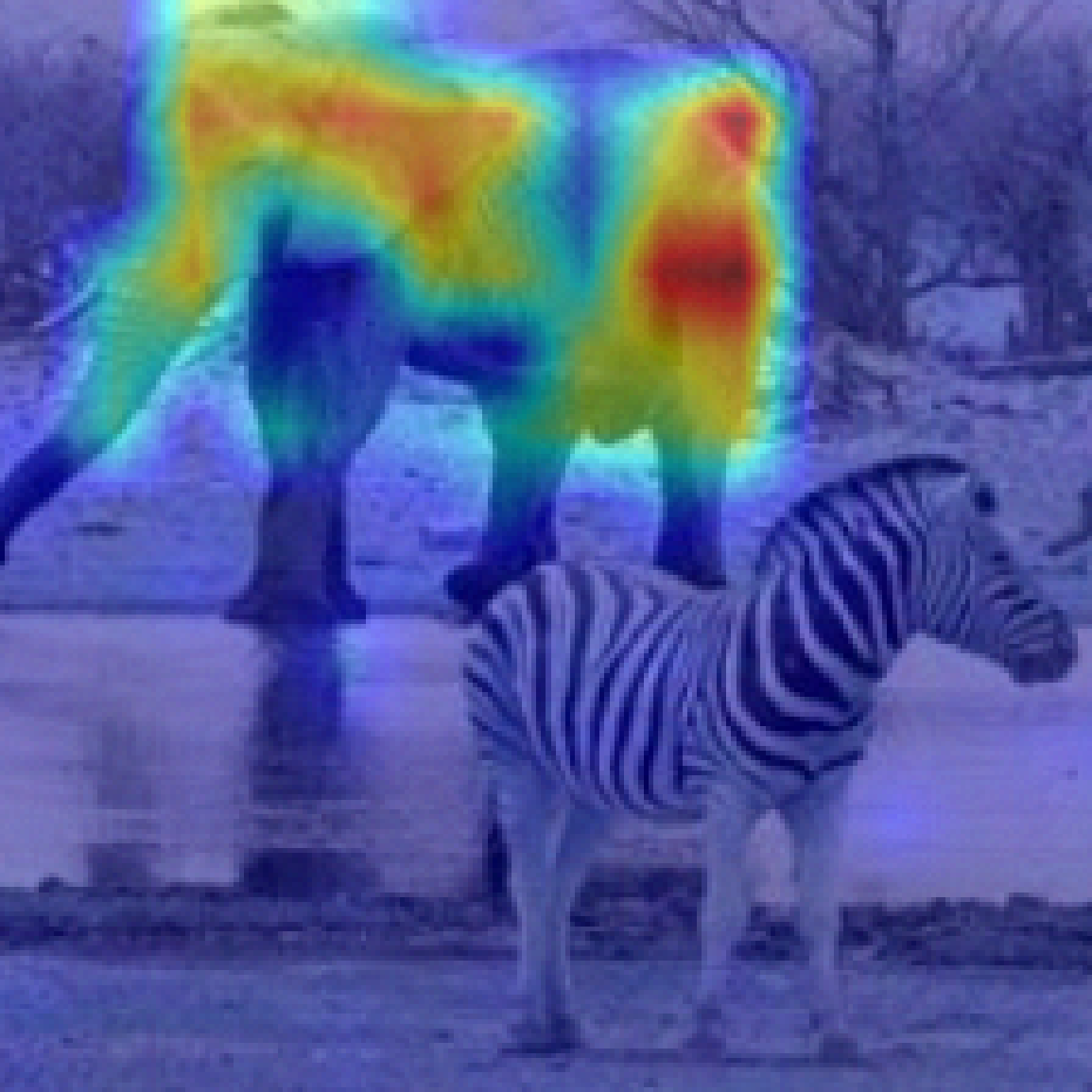} &
    \includegraphics[width=0.13\linewidth]{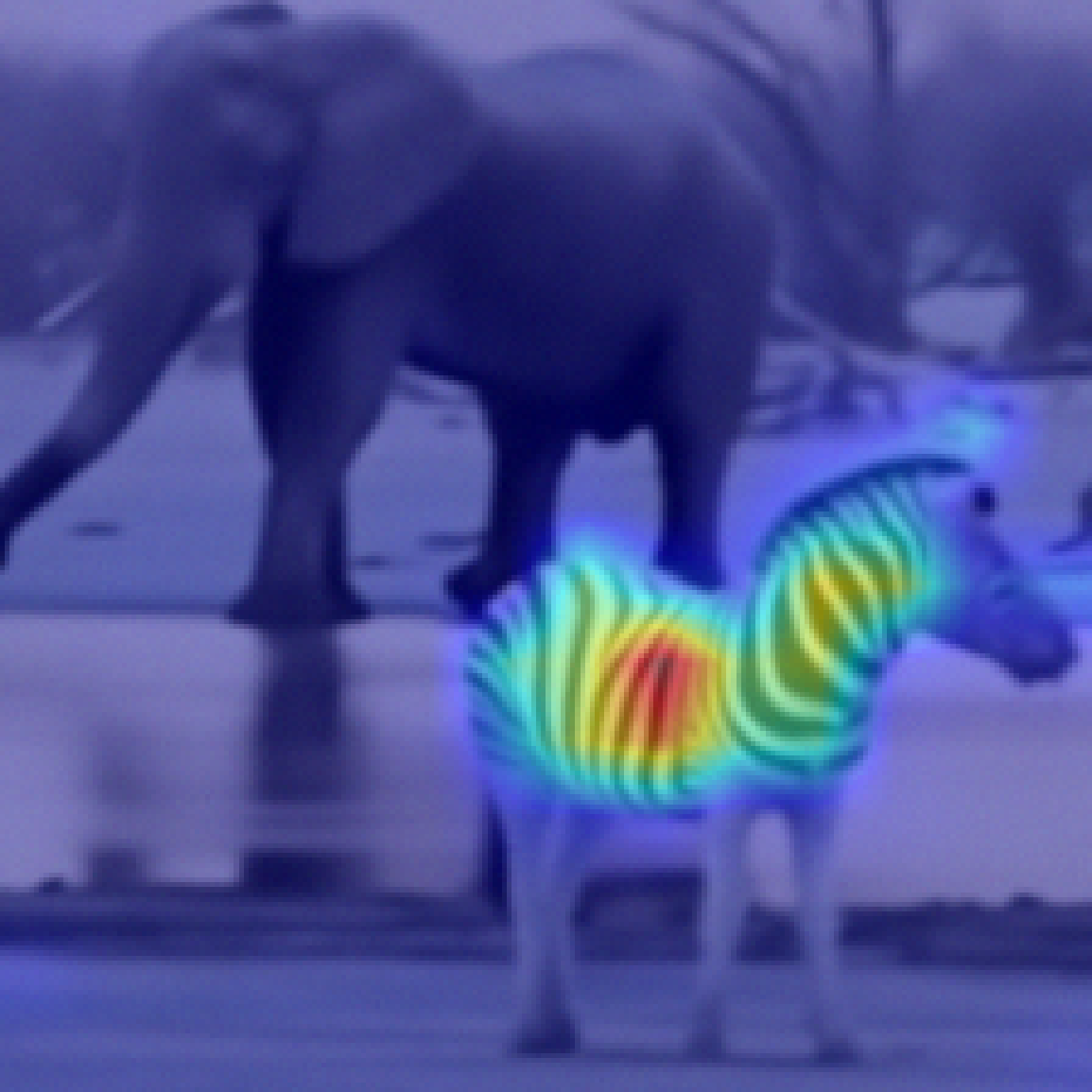}
    \end{tabular*}
    \caption{
     Class-specific visualizations under adversarial corruption. 
    }
    \label{fig: class2}
    \end{center}
\end{figure*}

 \begin{figure*}[htbp]
     \setlength{\tabcolsep}{1pt} 
     \renewcommand{\arraystretch}{1} 
     \begin{center}
     \begin{tabular*}{\linewidth}{@{\extracolsep{\fill}}ccccccc}
      Corrupted Input & Raw Attention & Rollout 
      & GradCAM &LRP& VTA  & Ours \\
        \raisebox{13mm}{\makecell*[c]{Clean$\rightarrow$\\\includegraphics[width=0.15\linewidth]{exp/fish/fish.png}
     }} &
     \includegraphics[width=0.15\linewidth]{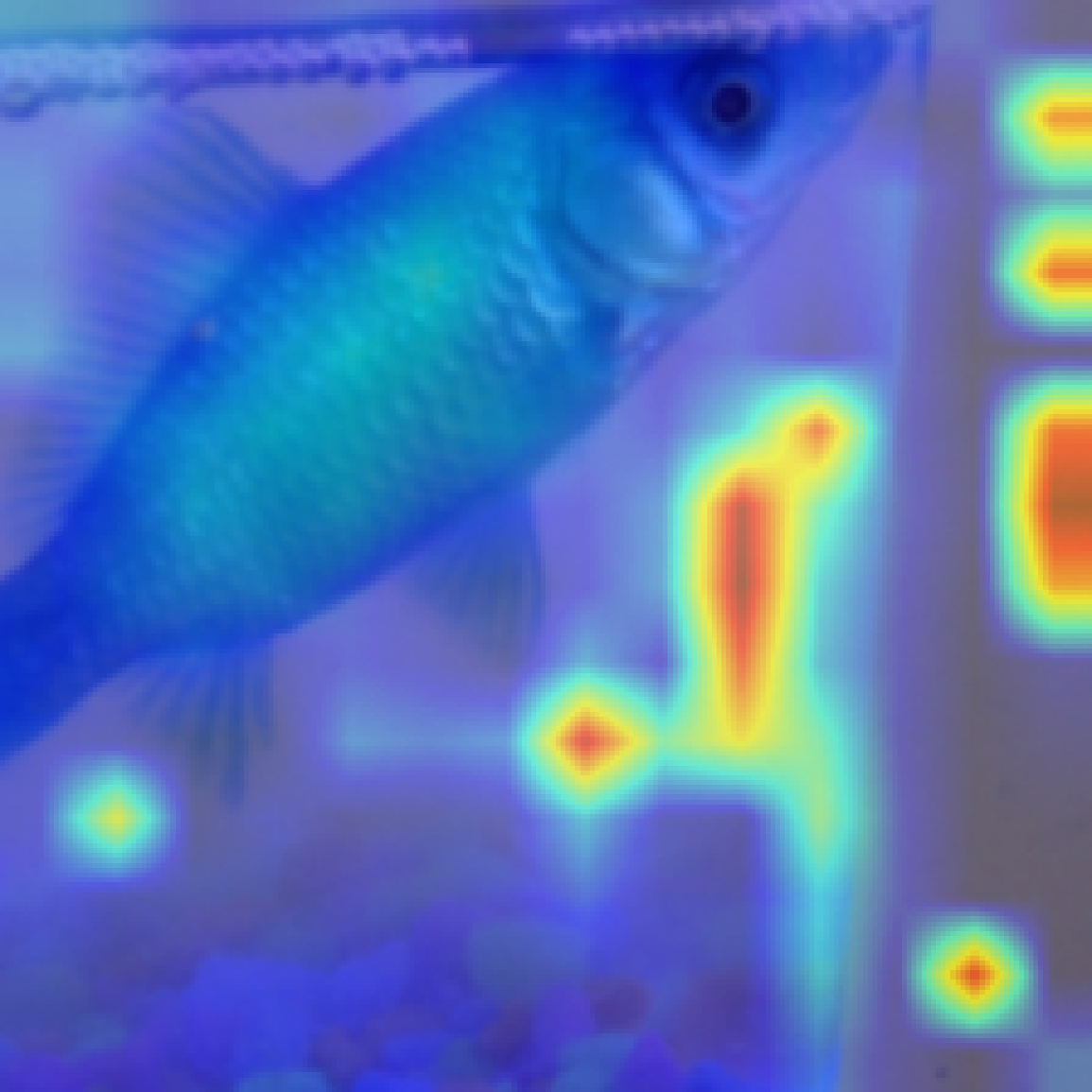} &
      \includegraphics[width=0.15\linewidth]{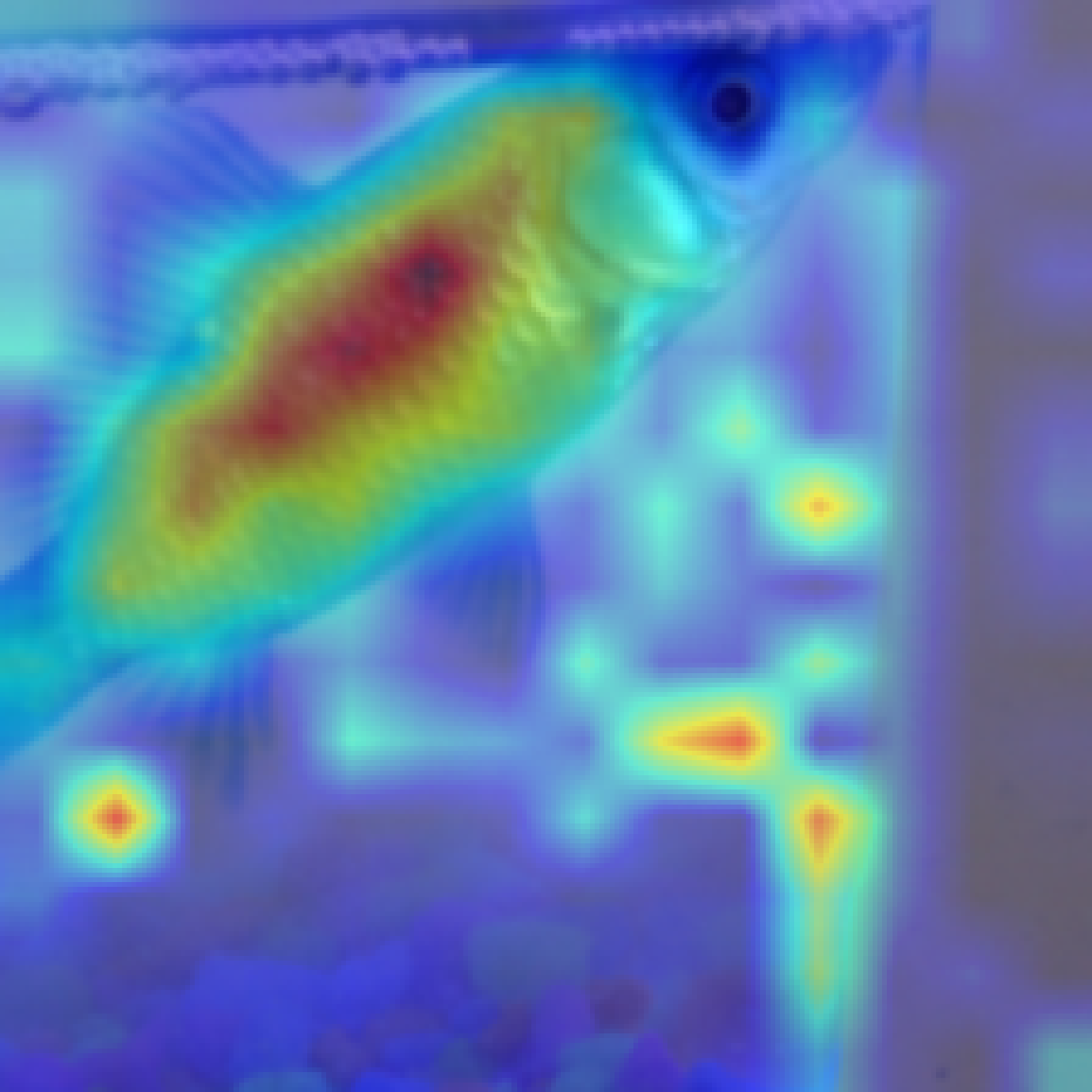} &
      \includegraphics[width=0.15\linewidth]{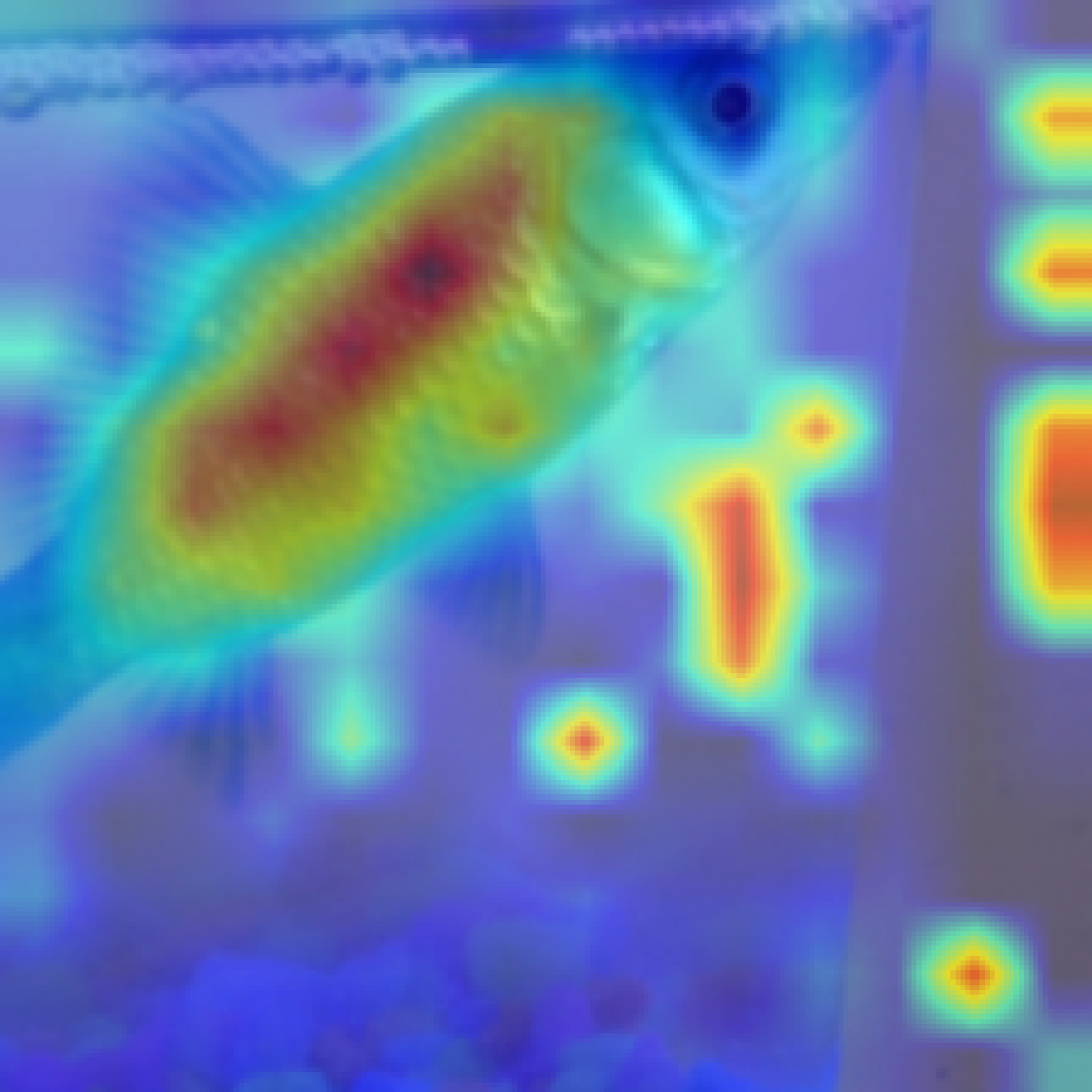} &
      \includegraphics[width=0.15\linewidth]{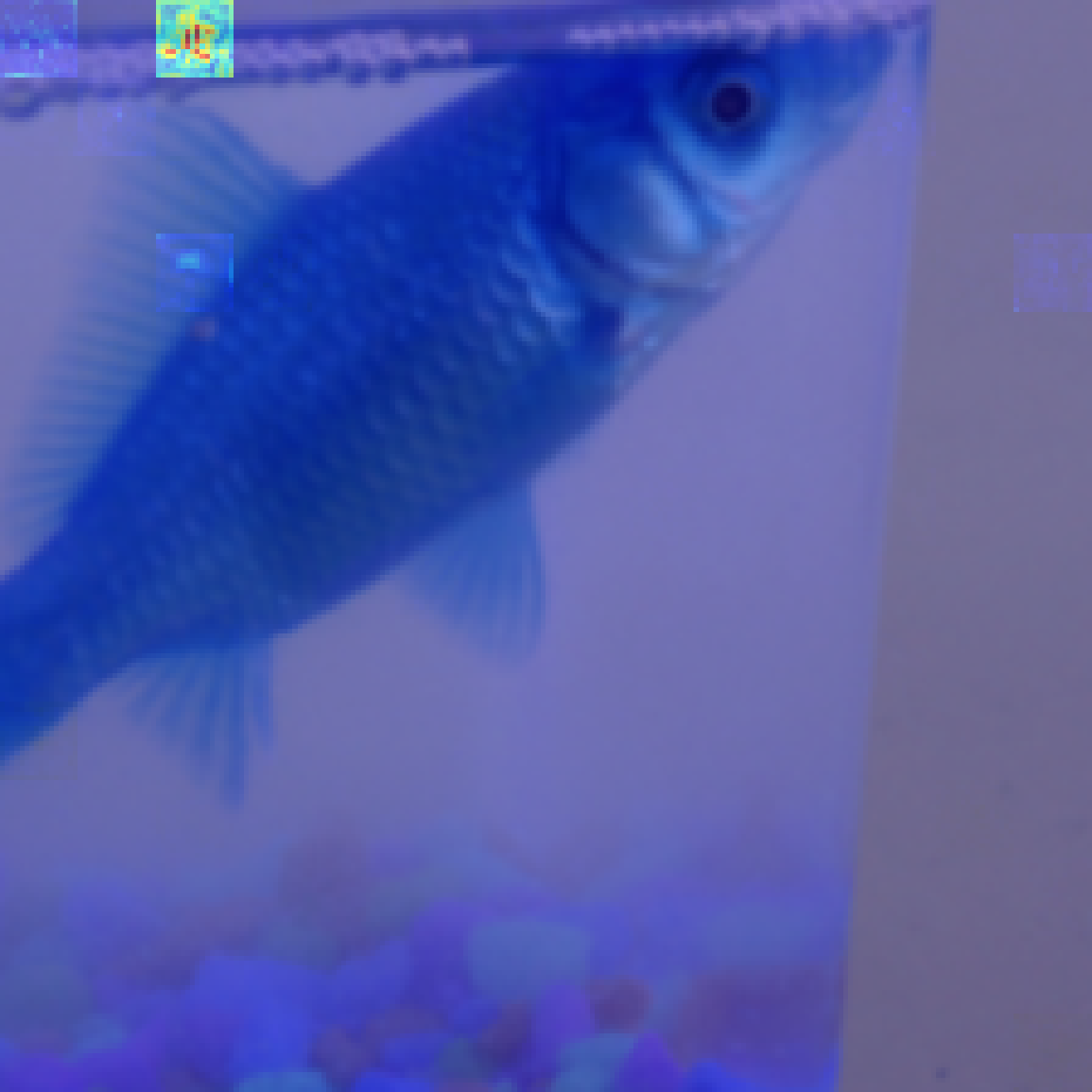} &
      \includegraphics[width=0.15\linewidth]{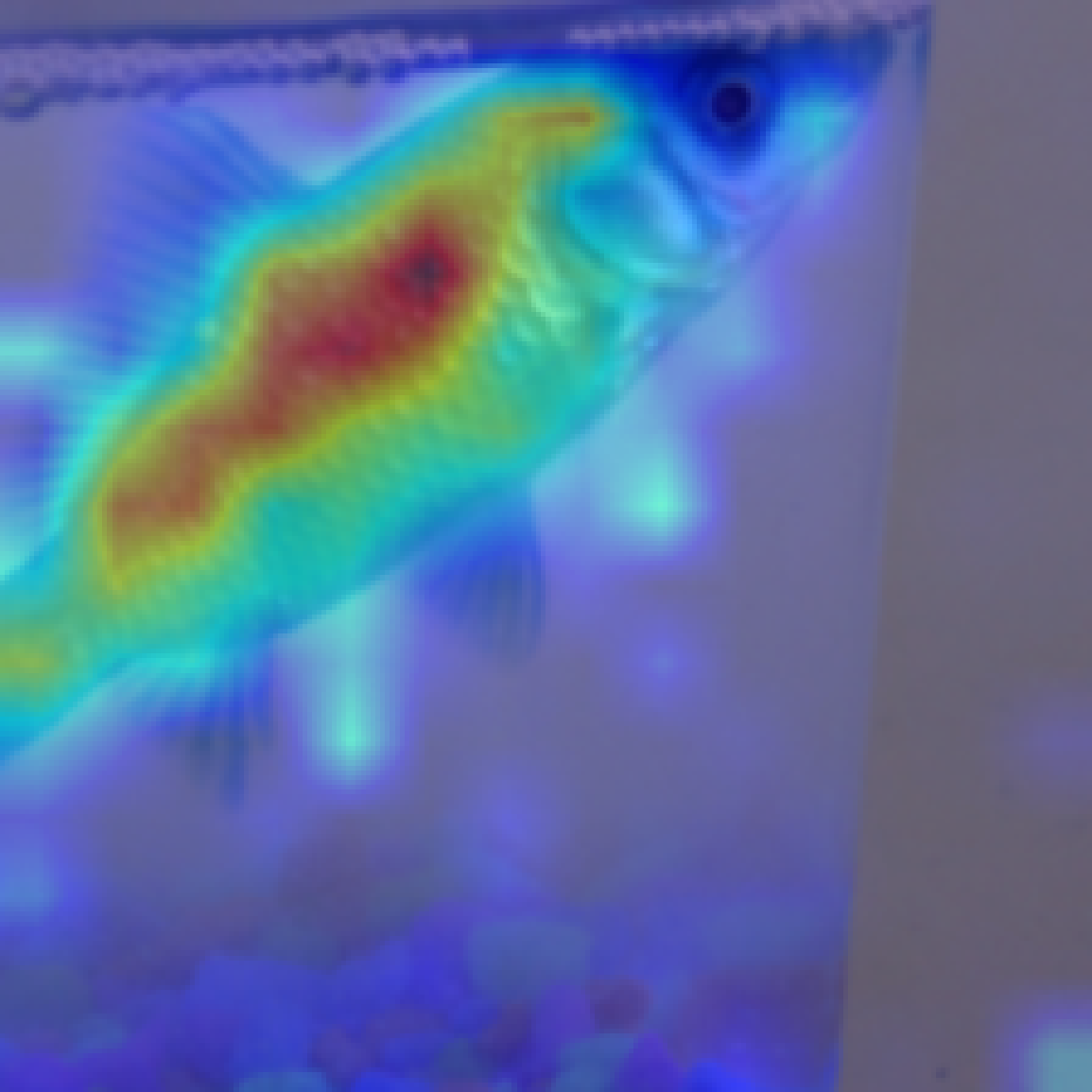} &
      \includegraphics[width=0.15\linewidth]{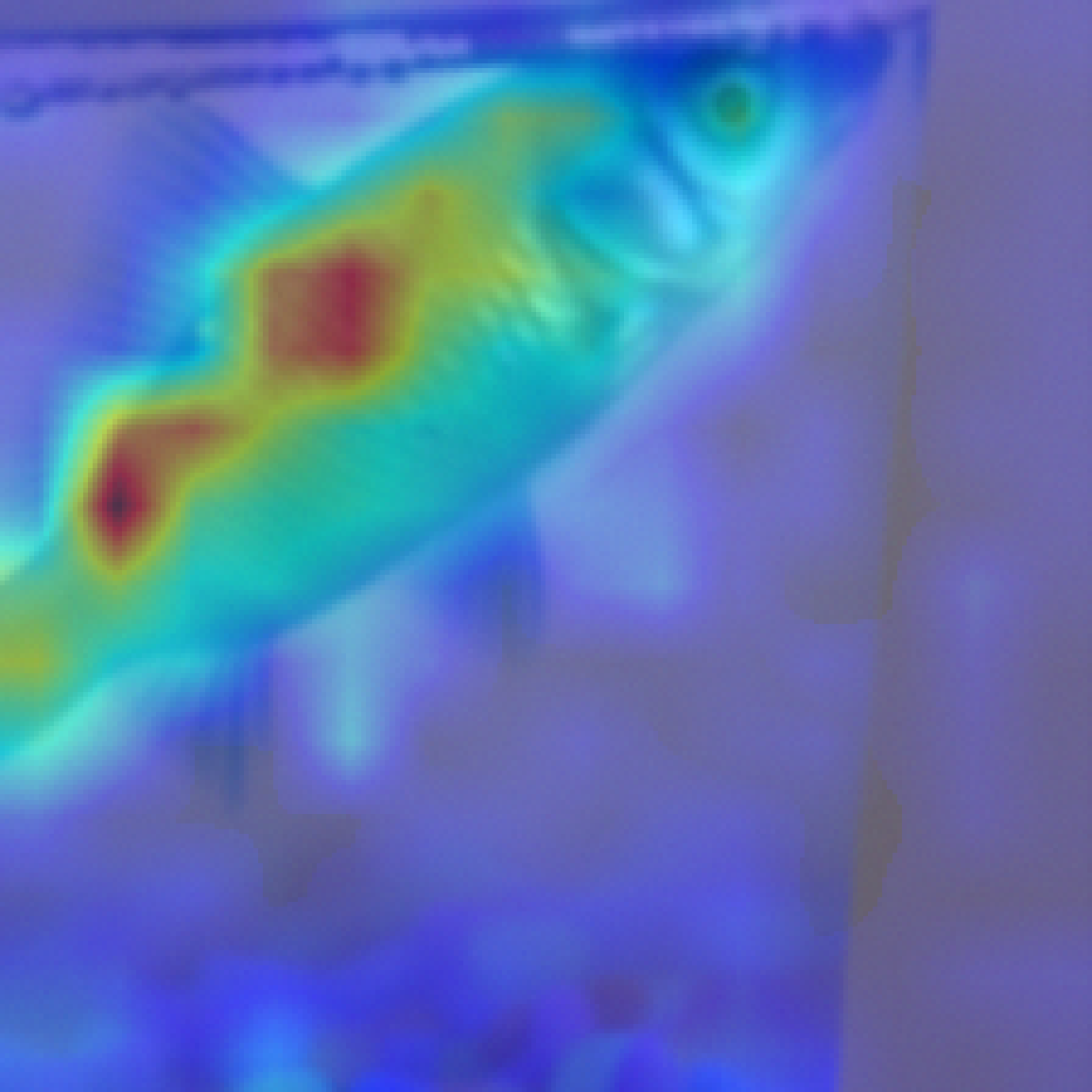} \\
      \raisebox{13mm}{\makecell*[c]{Poisoned$\rightarrow 1/255$\\\includegraphics[width=0.15\linewidth]{exp/fish/fish.png}
     }} &
     \includegraphics[width=0.15\linewidth]{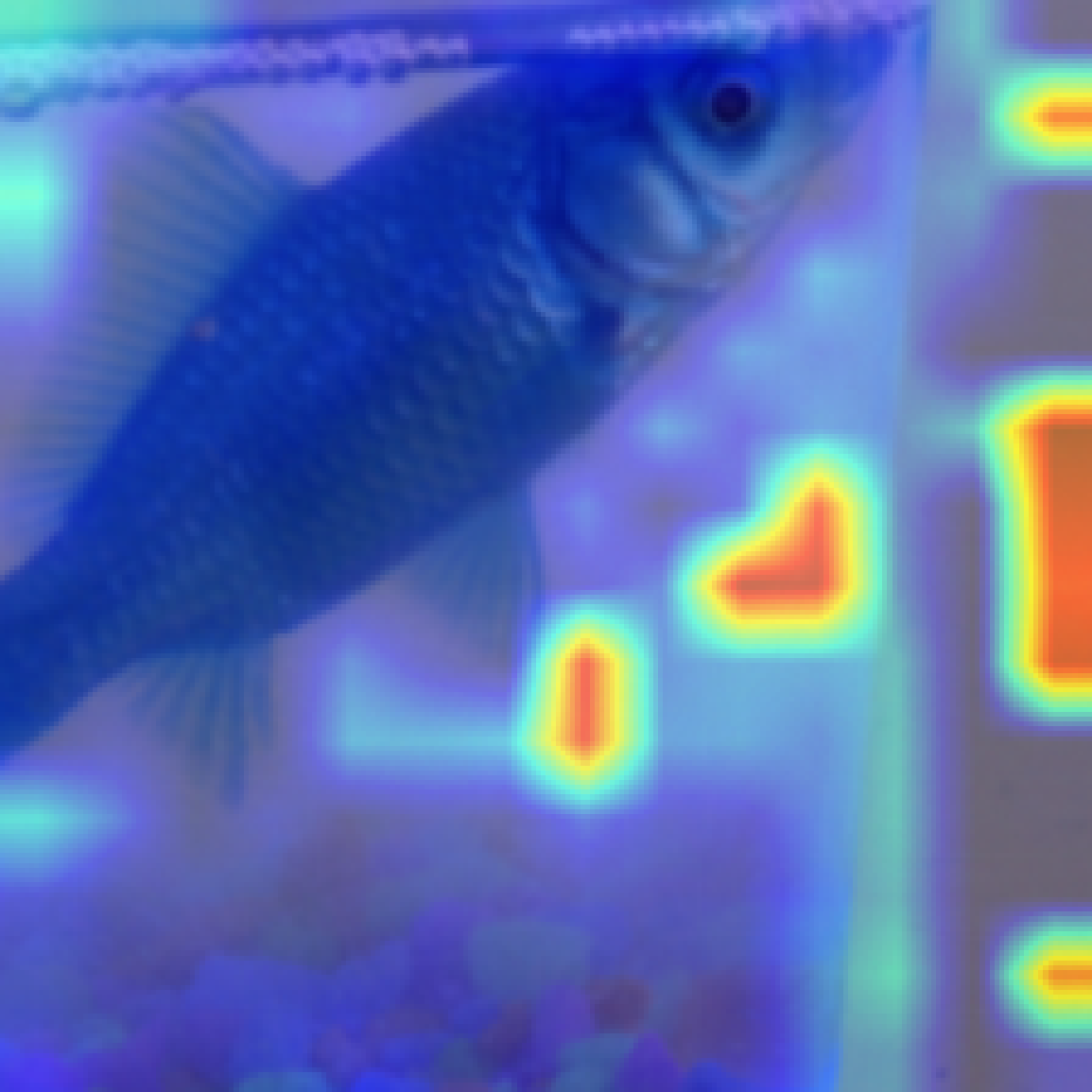} &
      \includegraphics[width=0.15\linewidth]{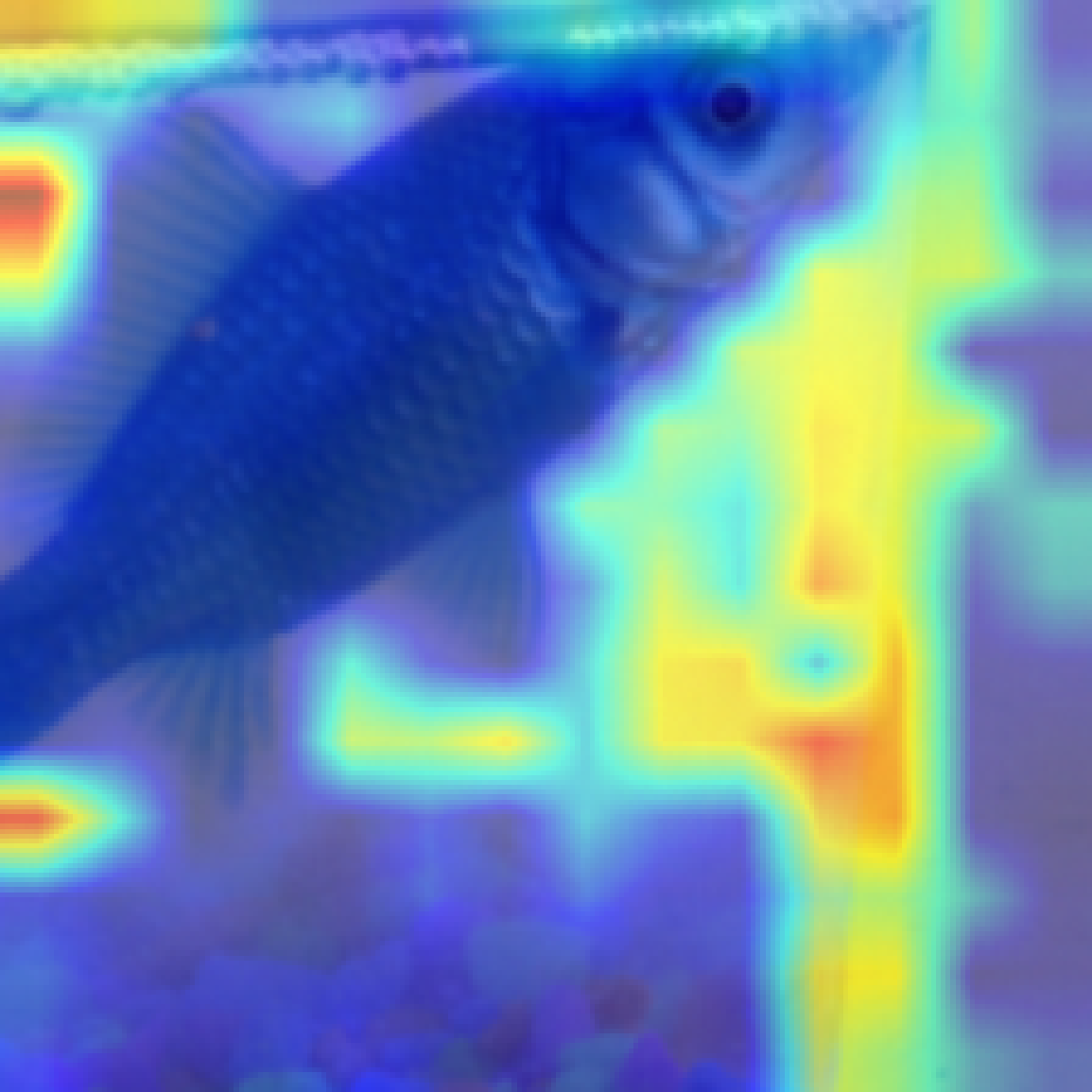} &
      \includegraphics[width=0.15\linewidth]{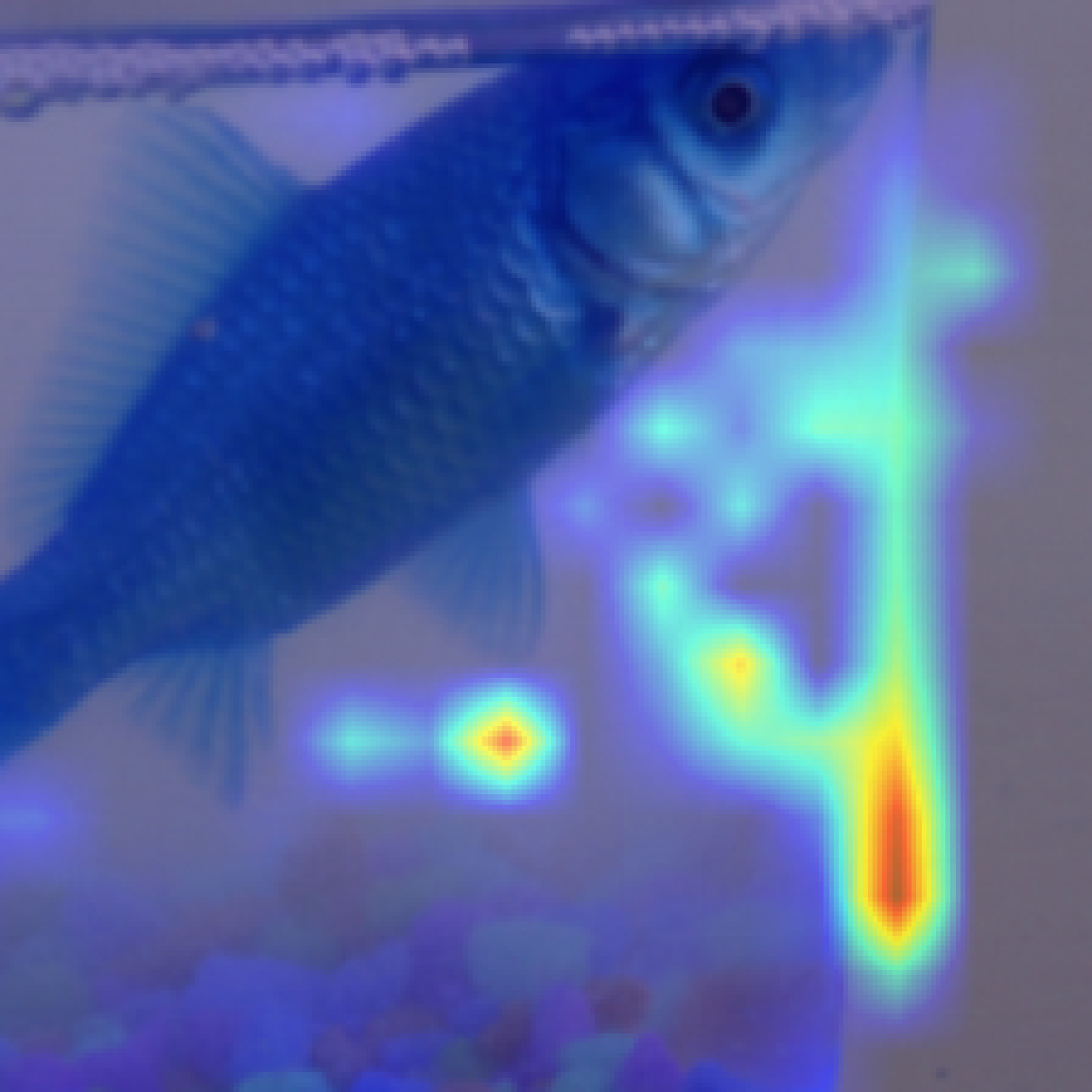} &
      \includegraphics[width=0.15\linewidth]{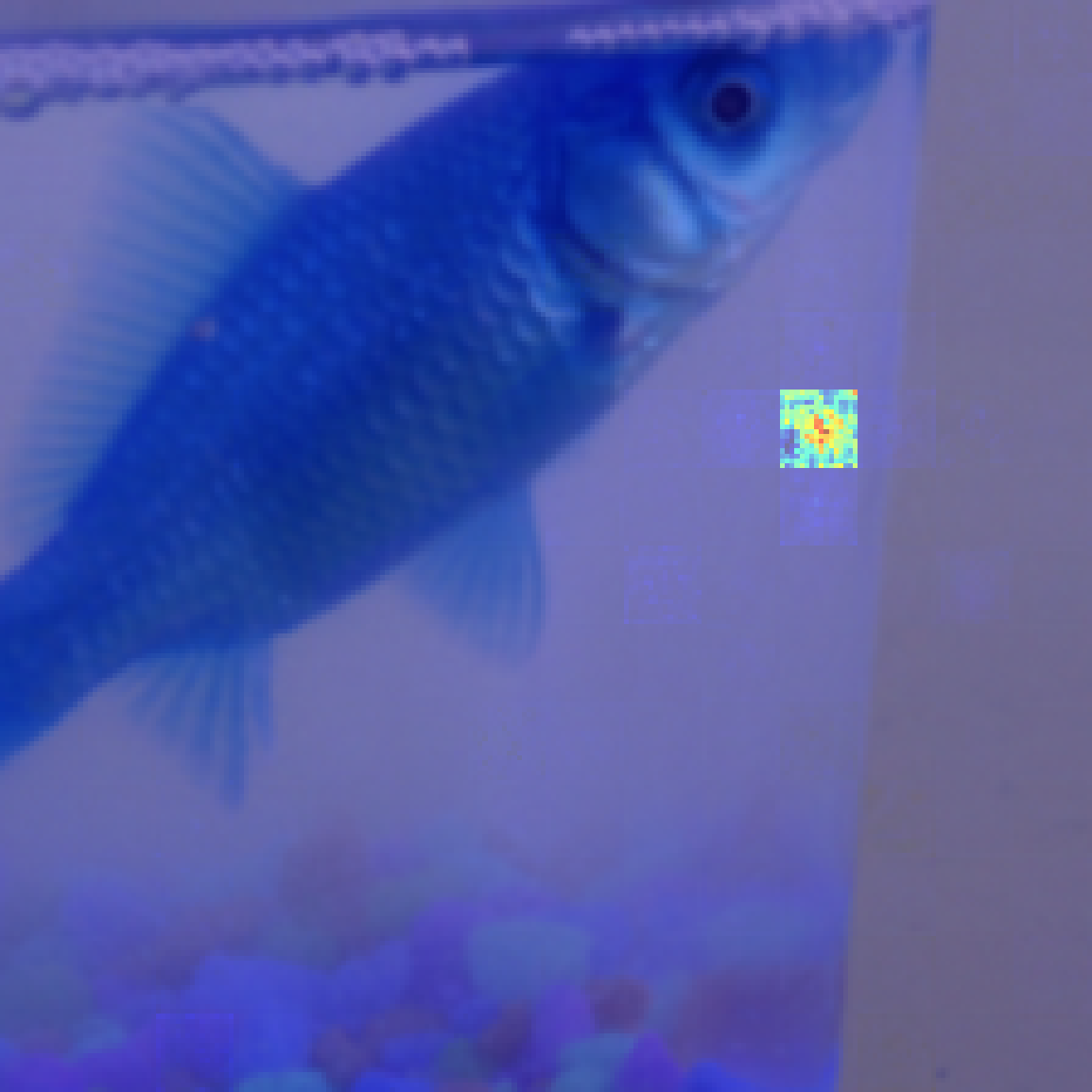} &
      \includegraphics[width=0.15\linewidth]{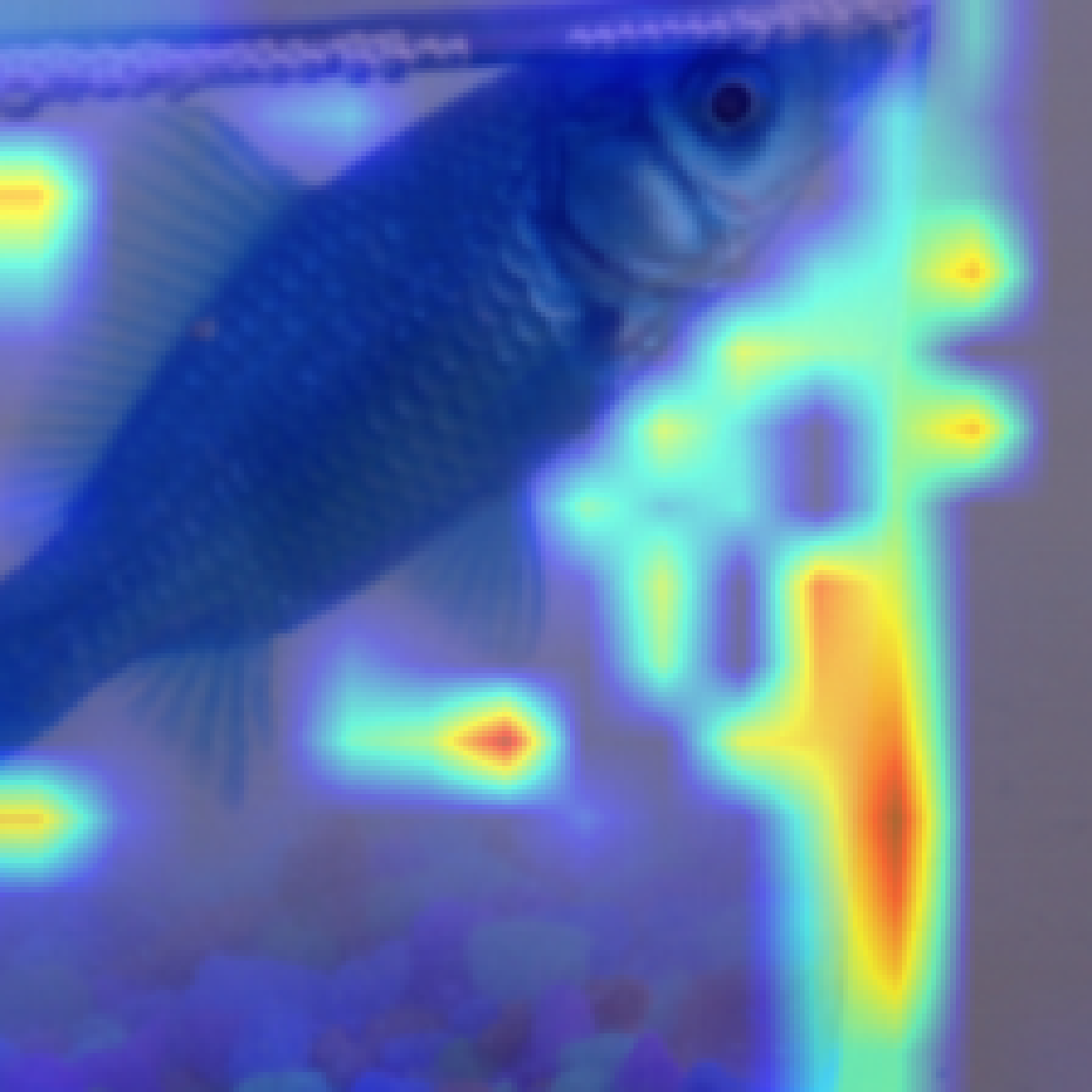} &
      \includegraphics[width=0.15\linewidth]{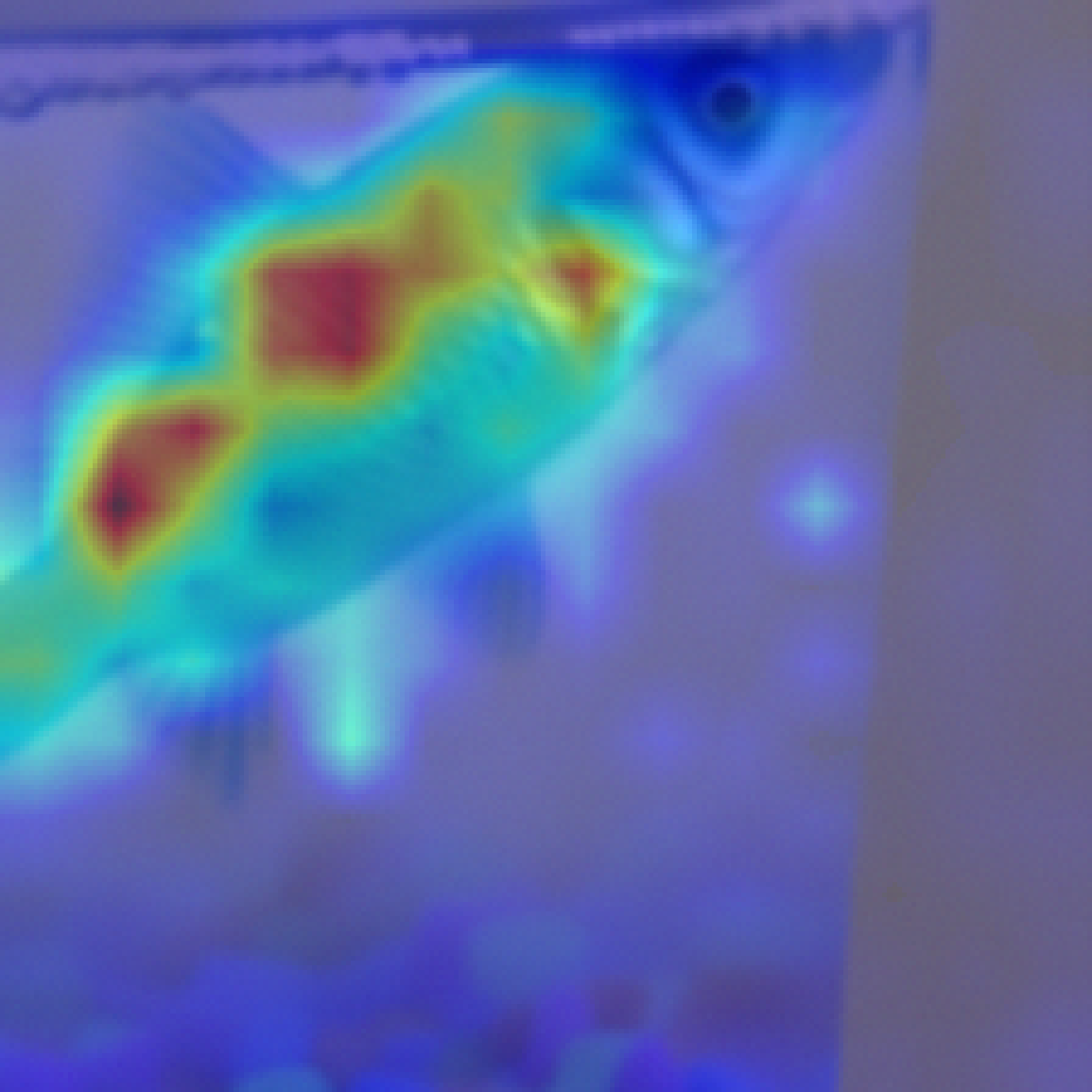} \\
      \raisebox{13mm}{\makecell*[c]{Poisoned$\rightarrow 2/255$\\\includegraphics[width=0.15\linewidth]{exp/fish/fish.png}
     }} &
     \includegraphics[width=0.15\linewidth]{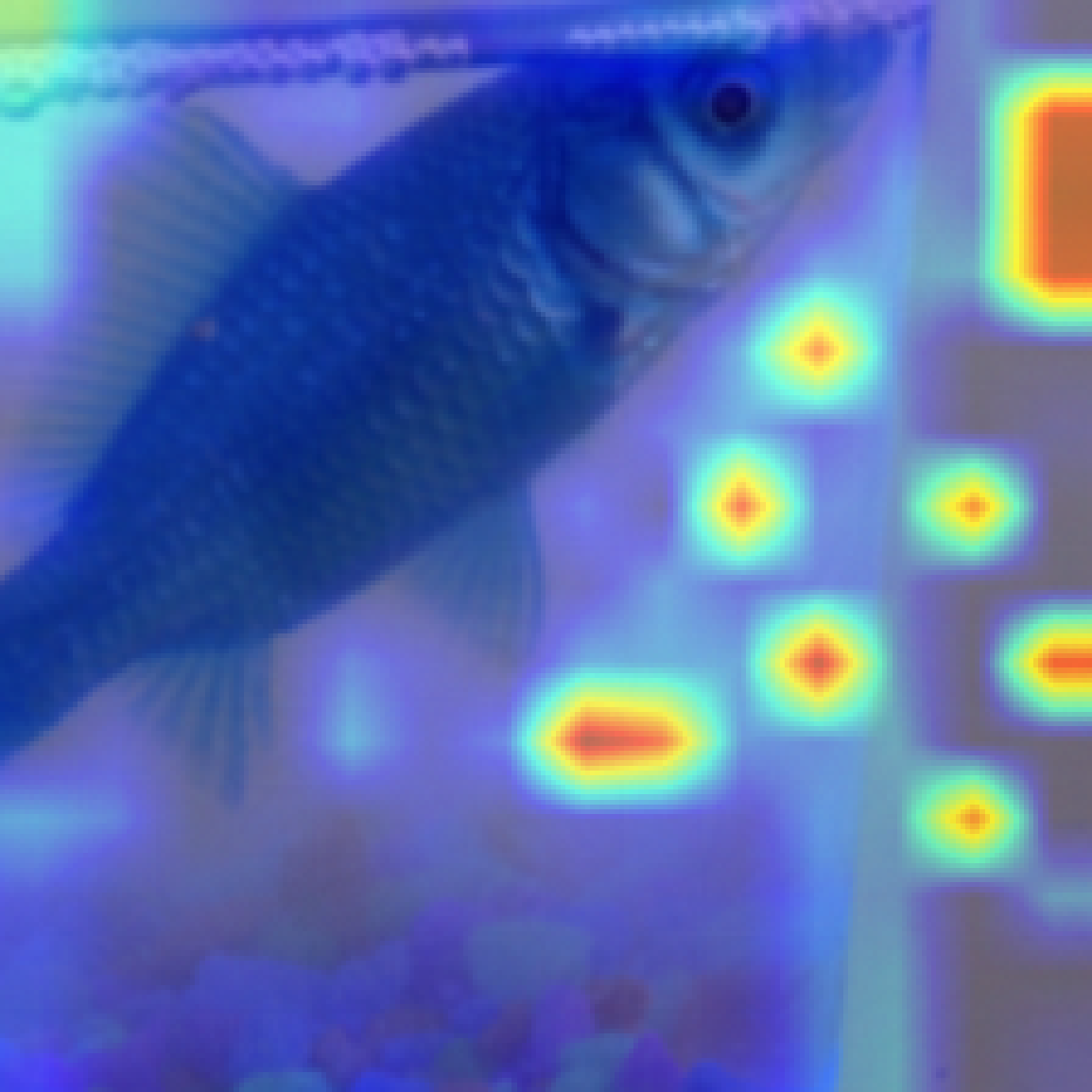} &
      \includegraphics[width=0.15\linewidth]{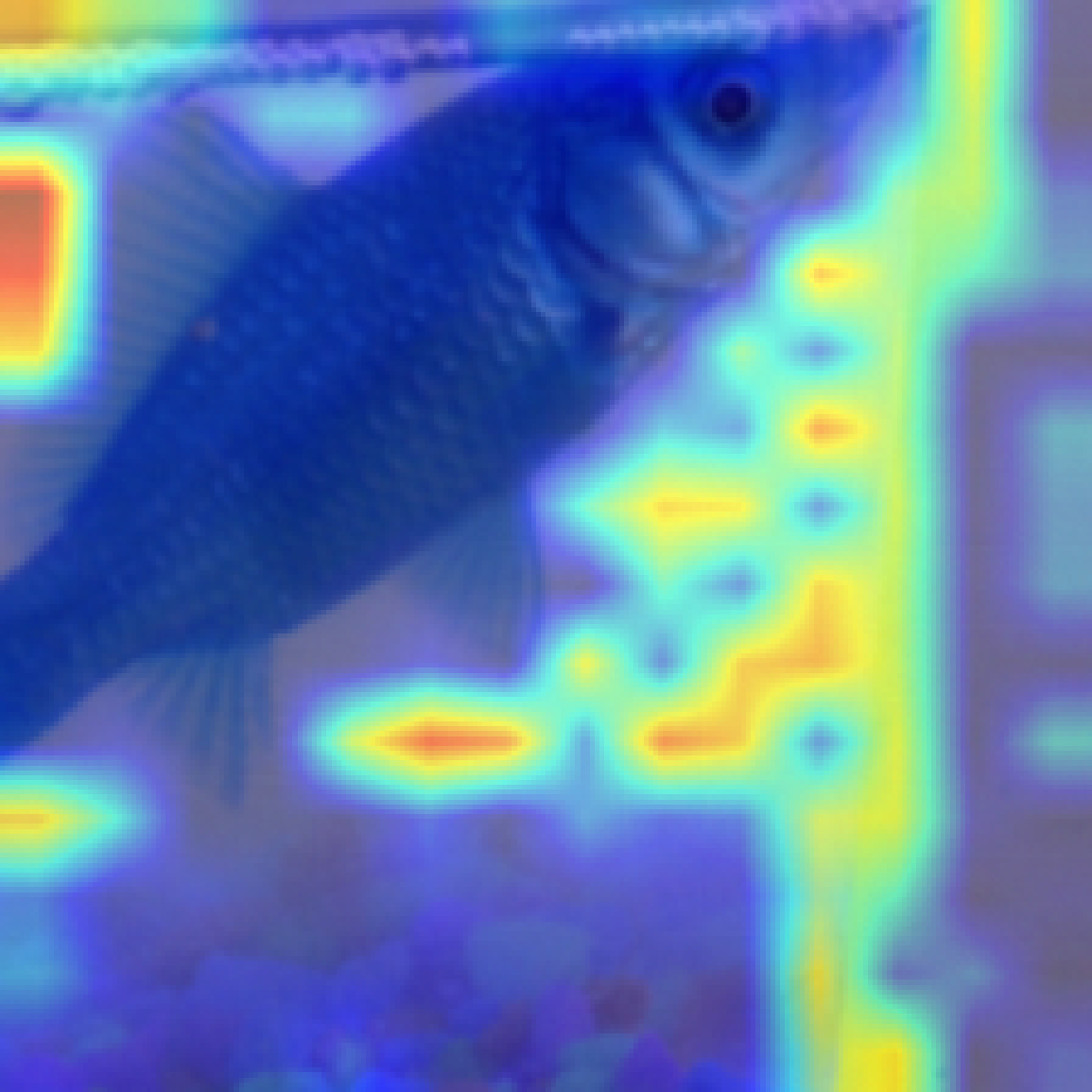} &
      \includegraphics[width=0.15\linewidth]{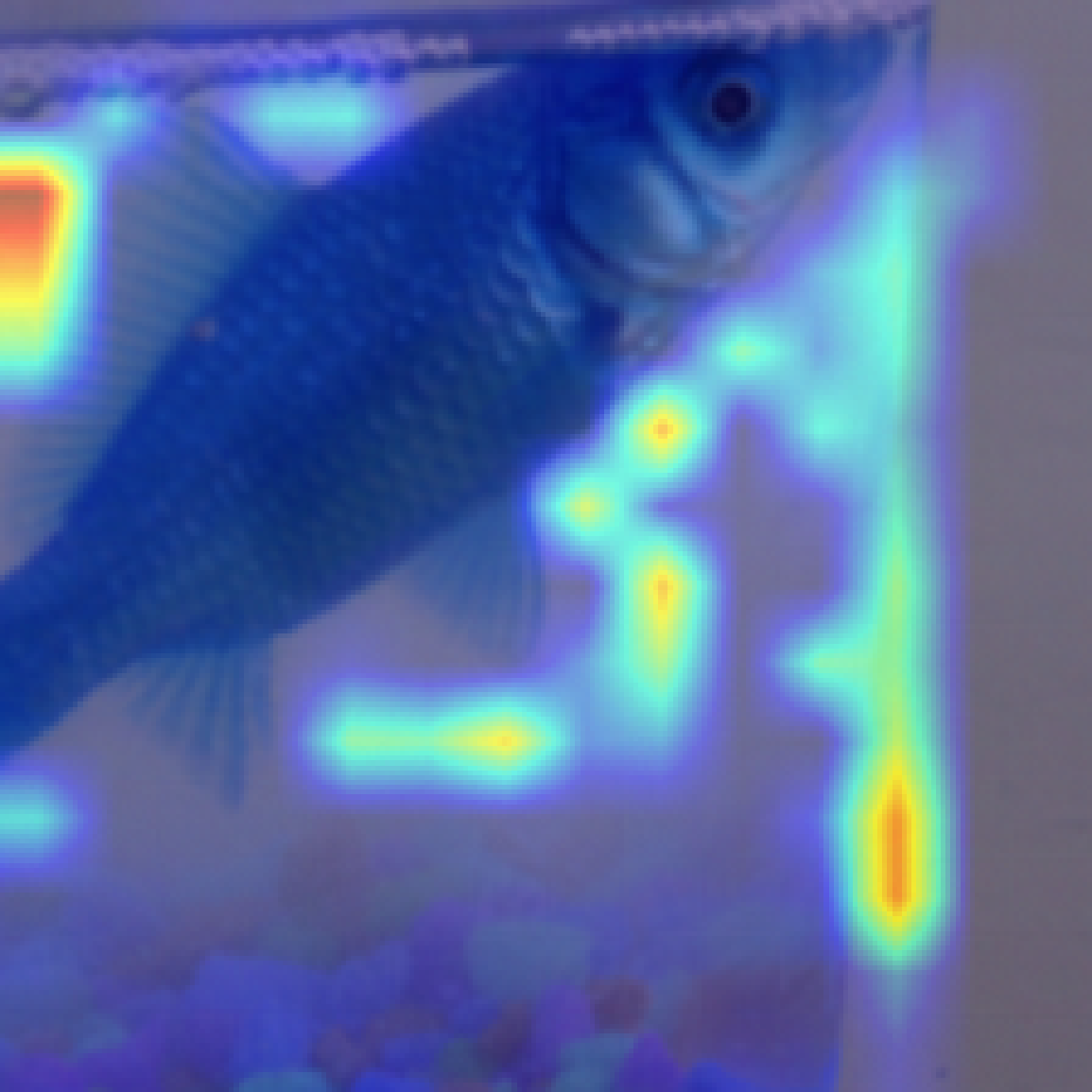} &
      \includegraphics[width=0.15\linewidth]{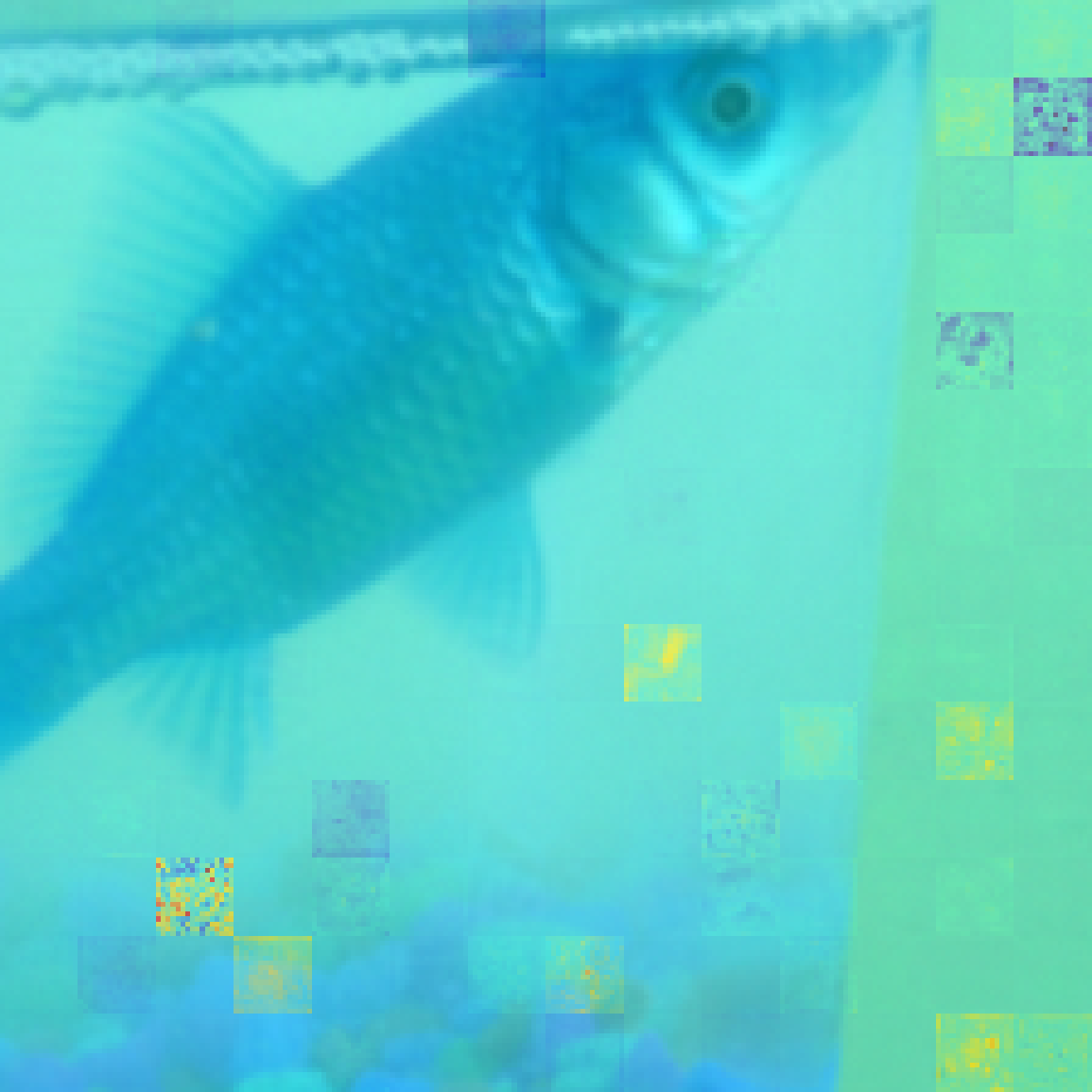} &
      \includegraphics[width=0.15\linewidth]{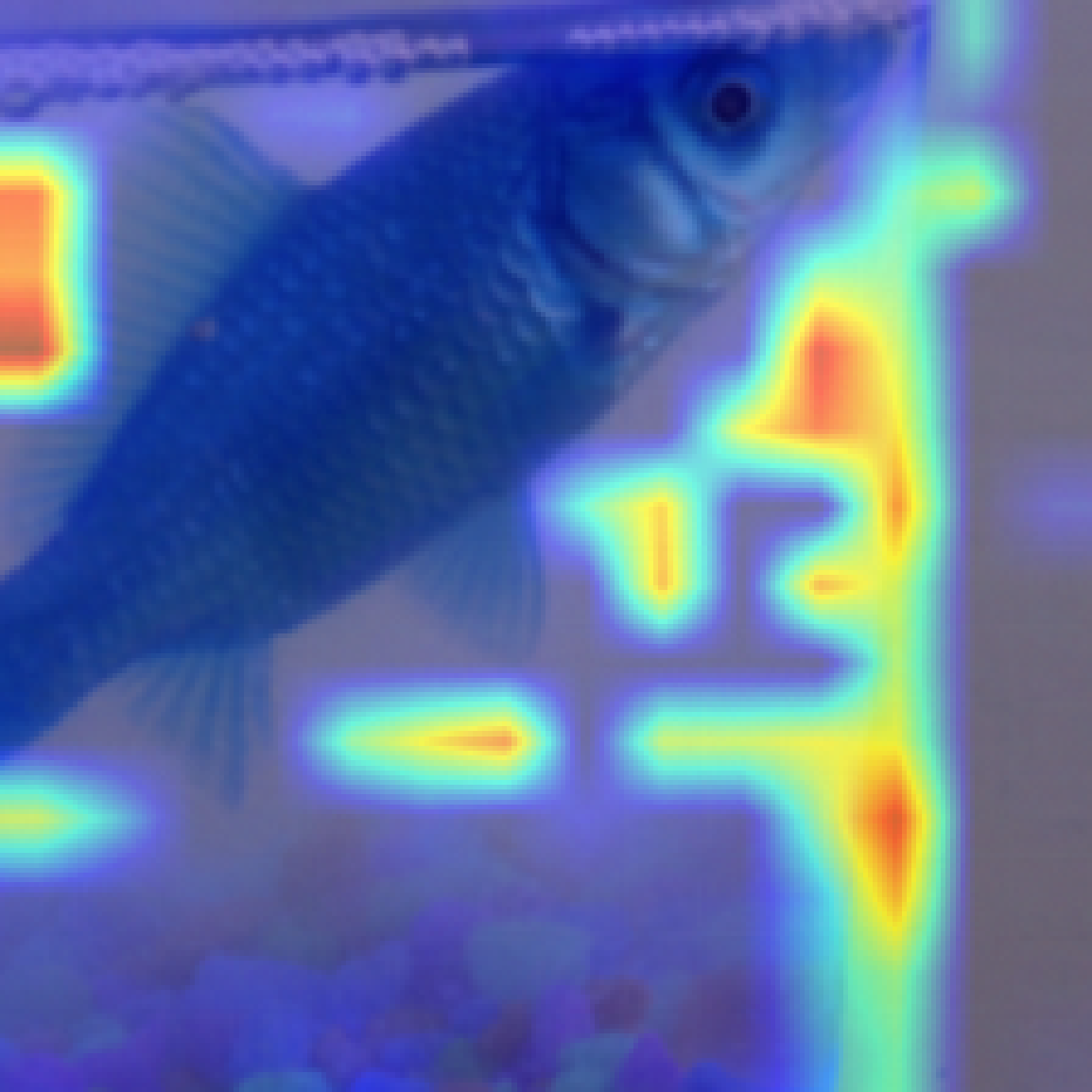} &
      \includegraphics[width=0.15\linewidth]{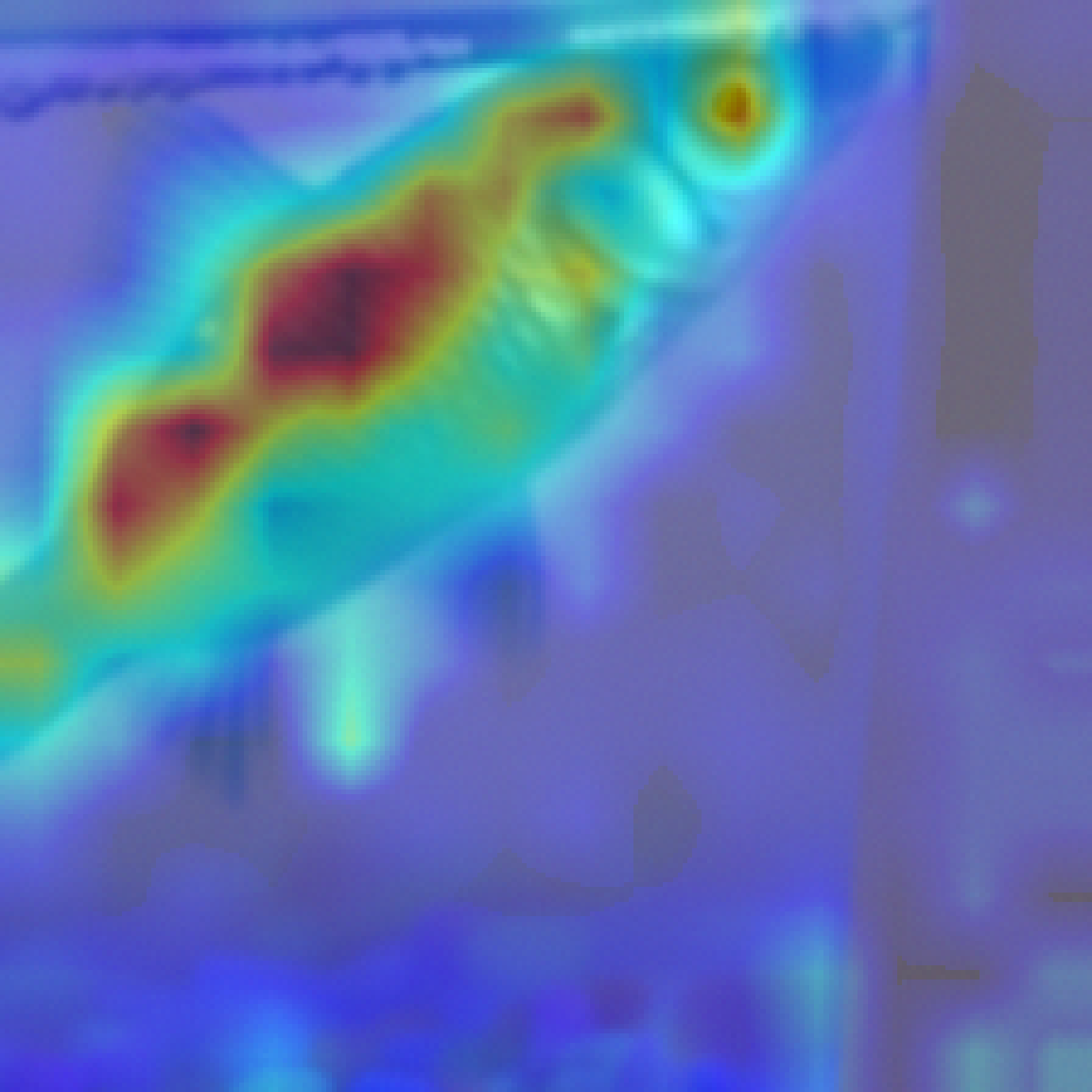} \\
      
      \raisebox{13mm}{\makecell*[c]{Poisoned$\rightarrow 3/255$\\\includegraphics[width=0.15\linewidth]{exp/fish/fish.png}
     }} &
     \includegraphics[width=0.15\linewidth]{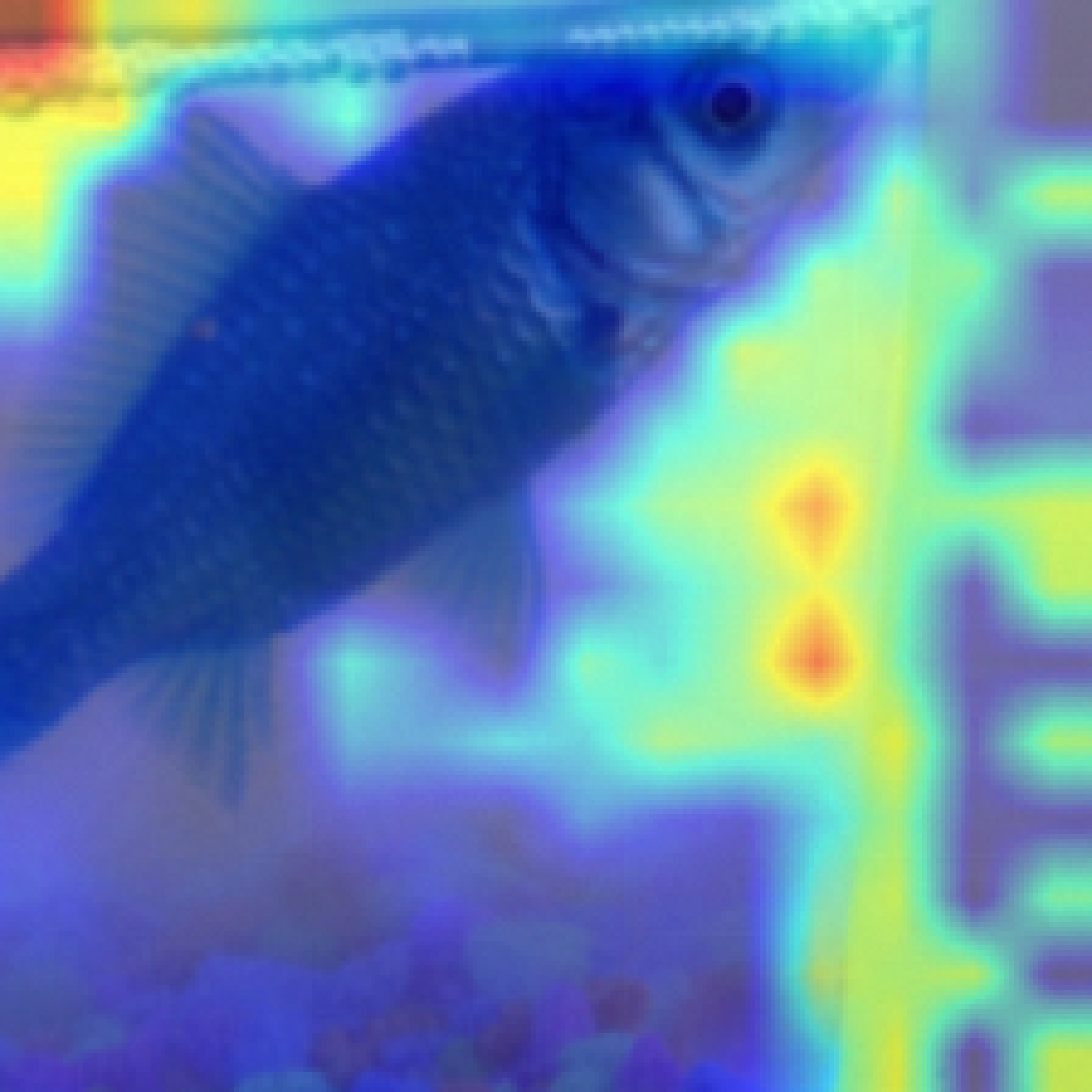} &
      \includegraphics[width=0.15\linewidth]{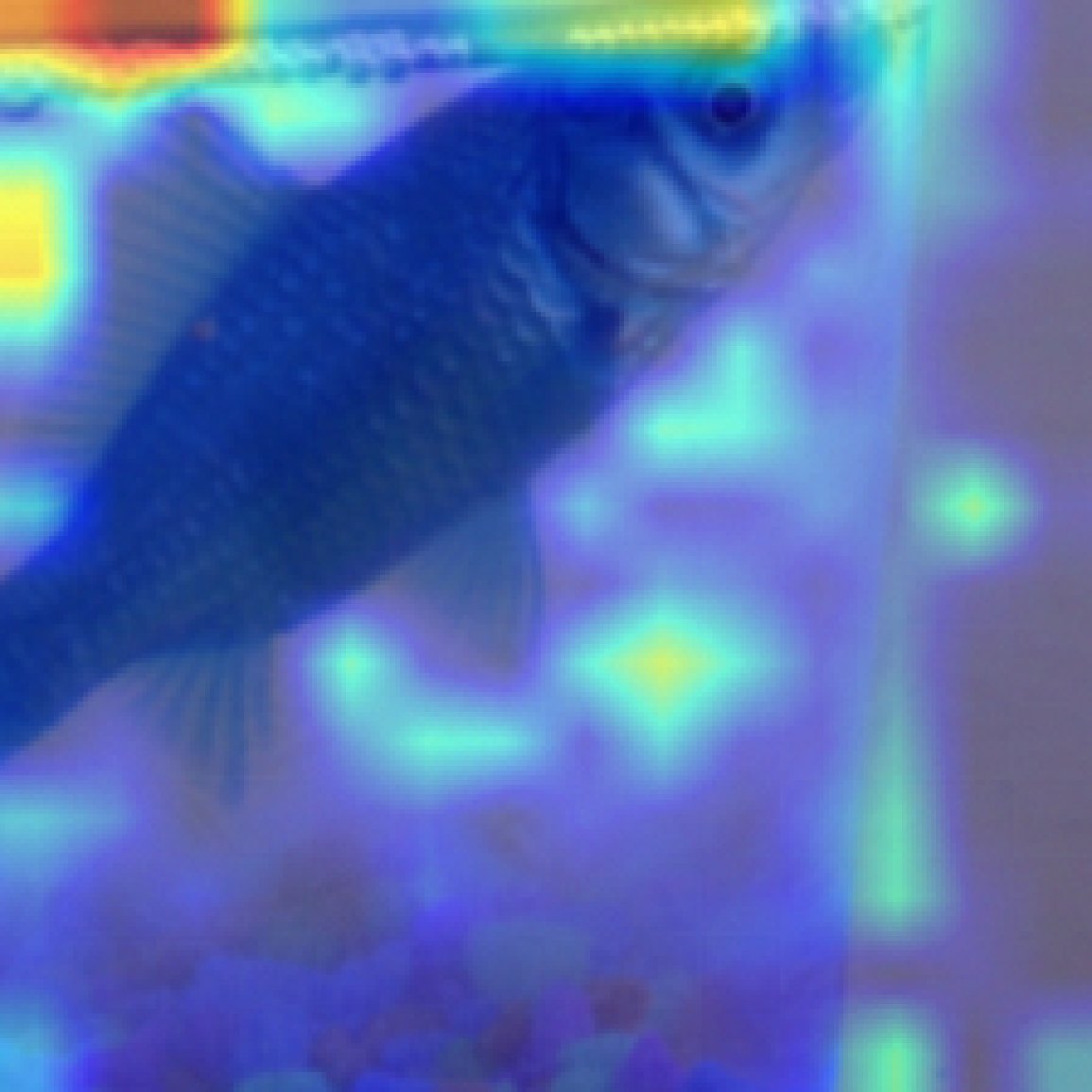} &
      \includegraphics[width=0.15\linewidth]{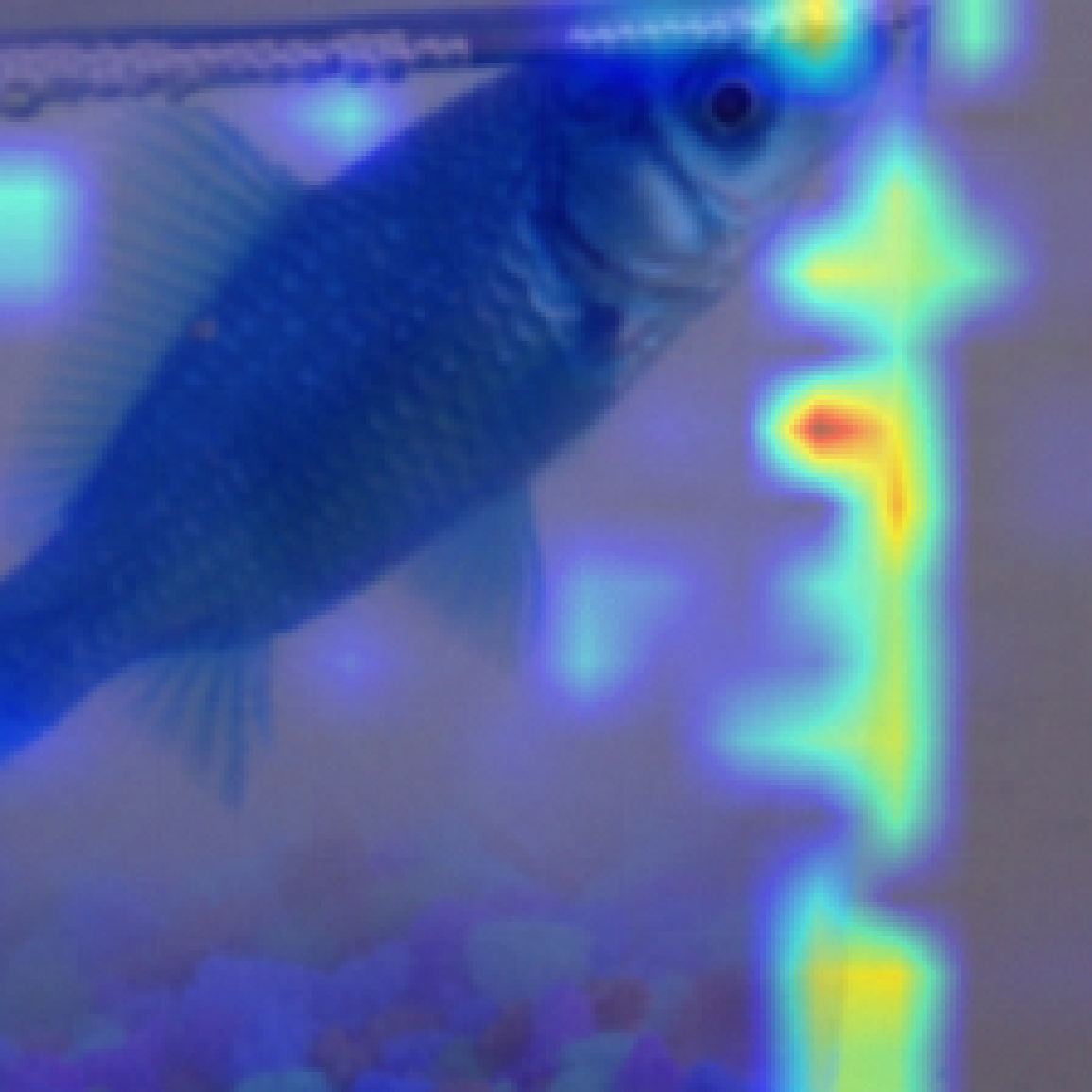} &
      \includegraphics[width=0.15\linewidth]{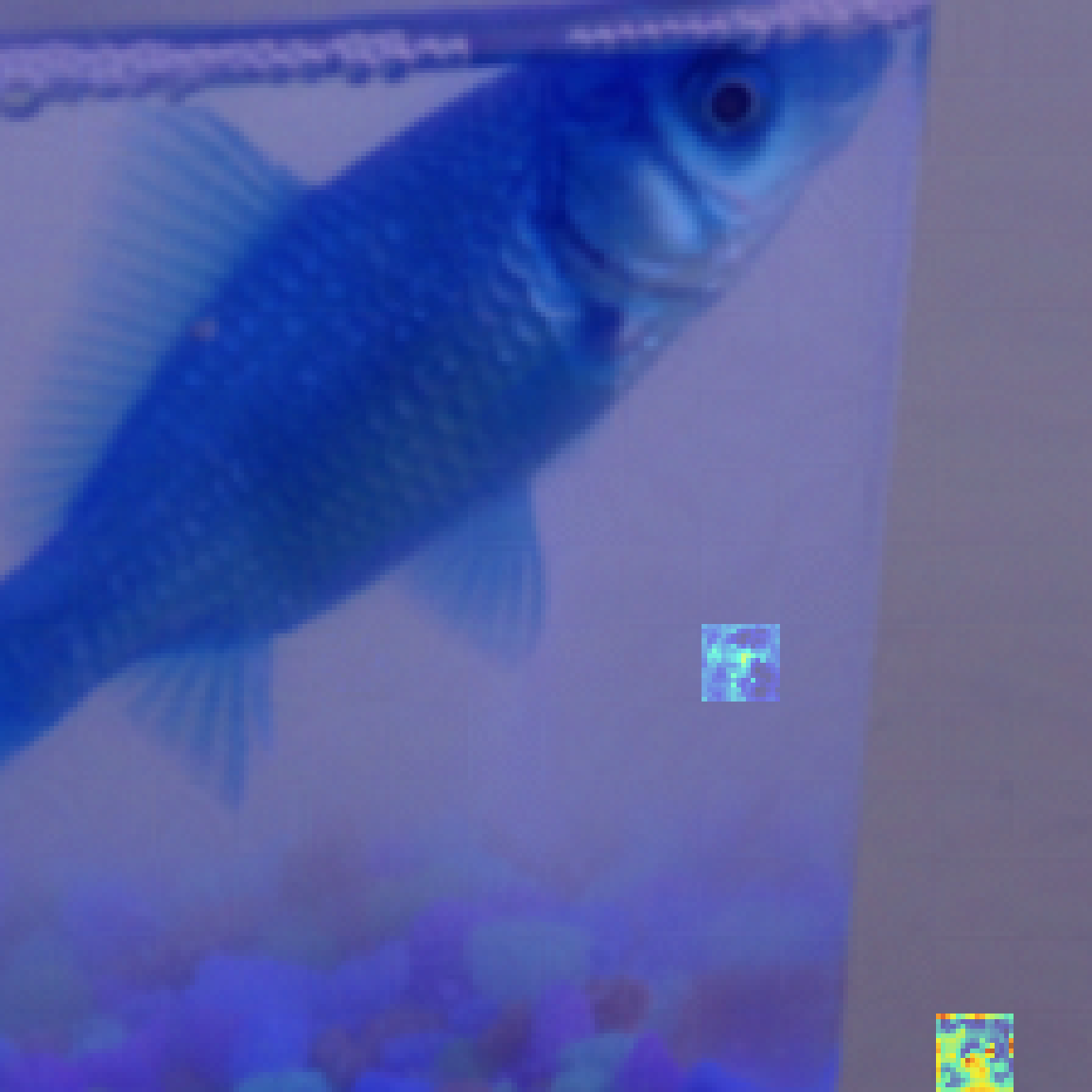} &
      \includegraphics[width=0.15\linewidth]{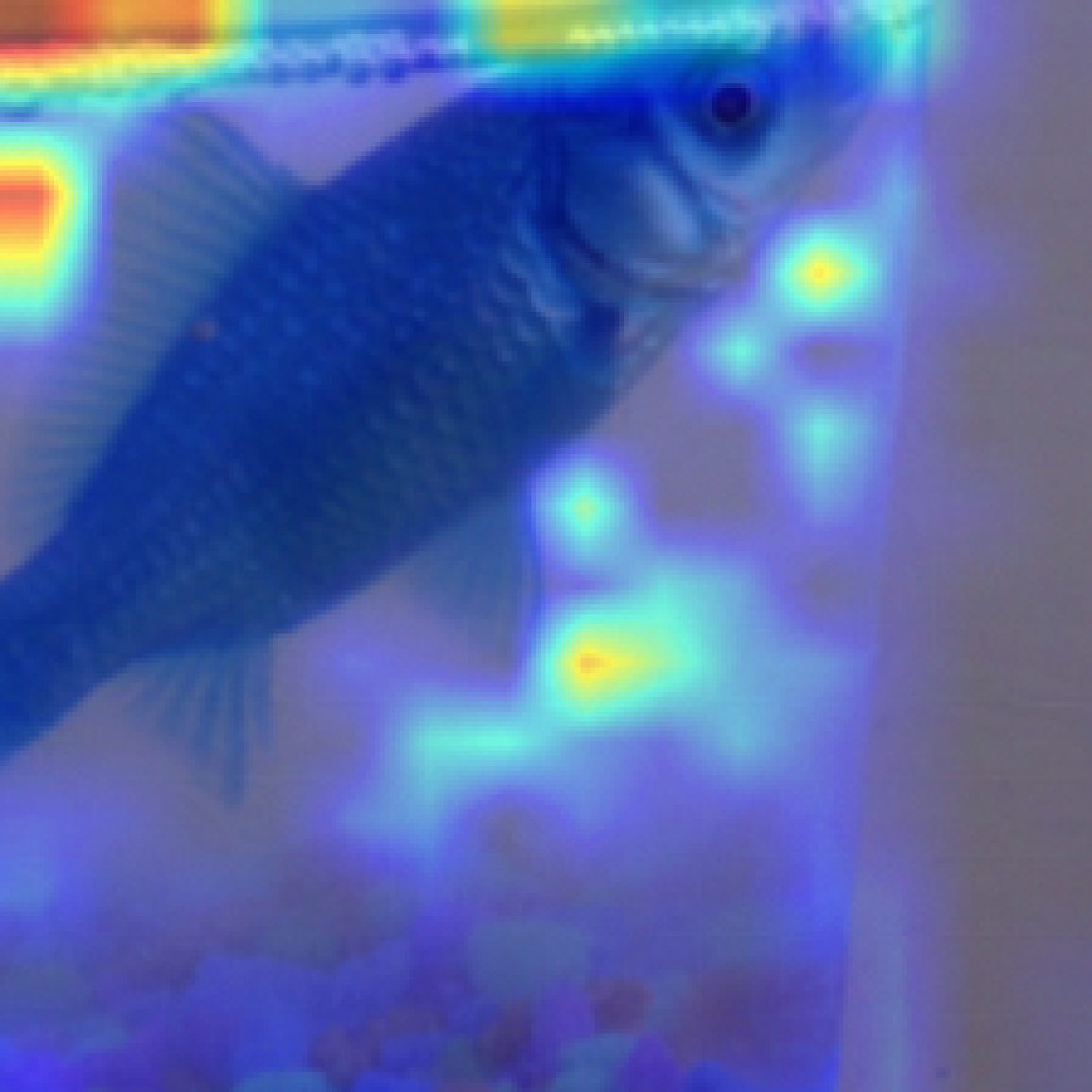} &
      \includegraphics[width=0.15\linewidth]{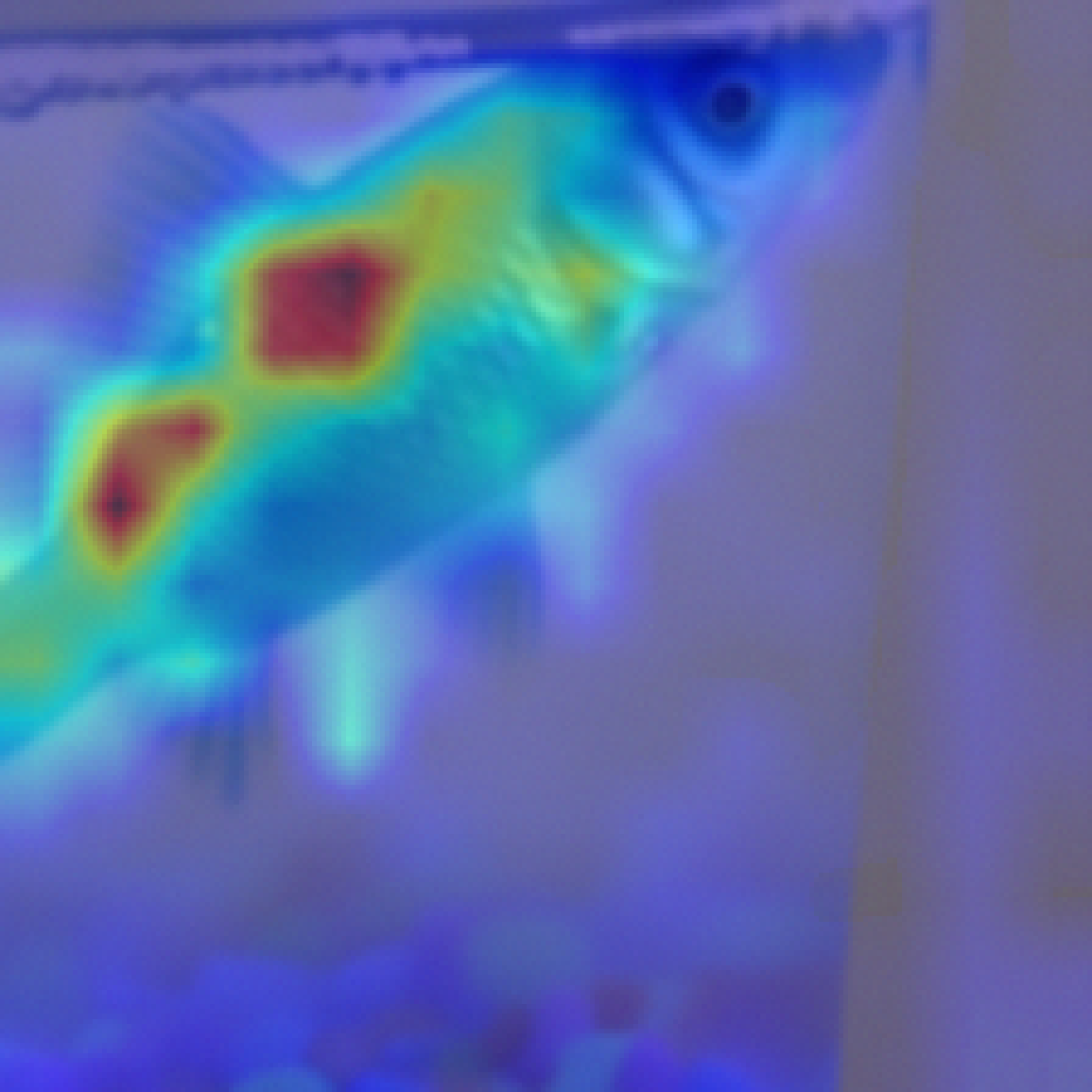} \\
      
      \raisebox{13mm}{\makecell*[c]{Poisoned$\rightarrow 4/255$\\\includegraphics[width=0.15\linewidth]{exp/fish/fish.png}
     }} &
     \includegraphics[width=0.15\linewidth]{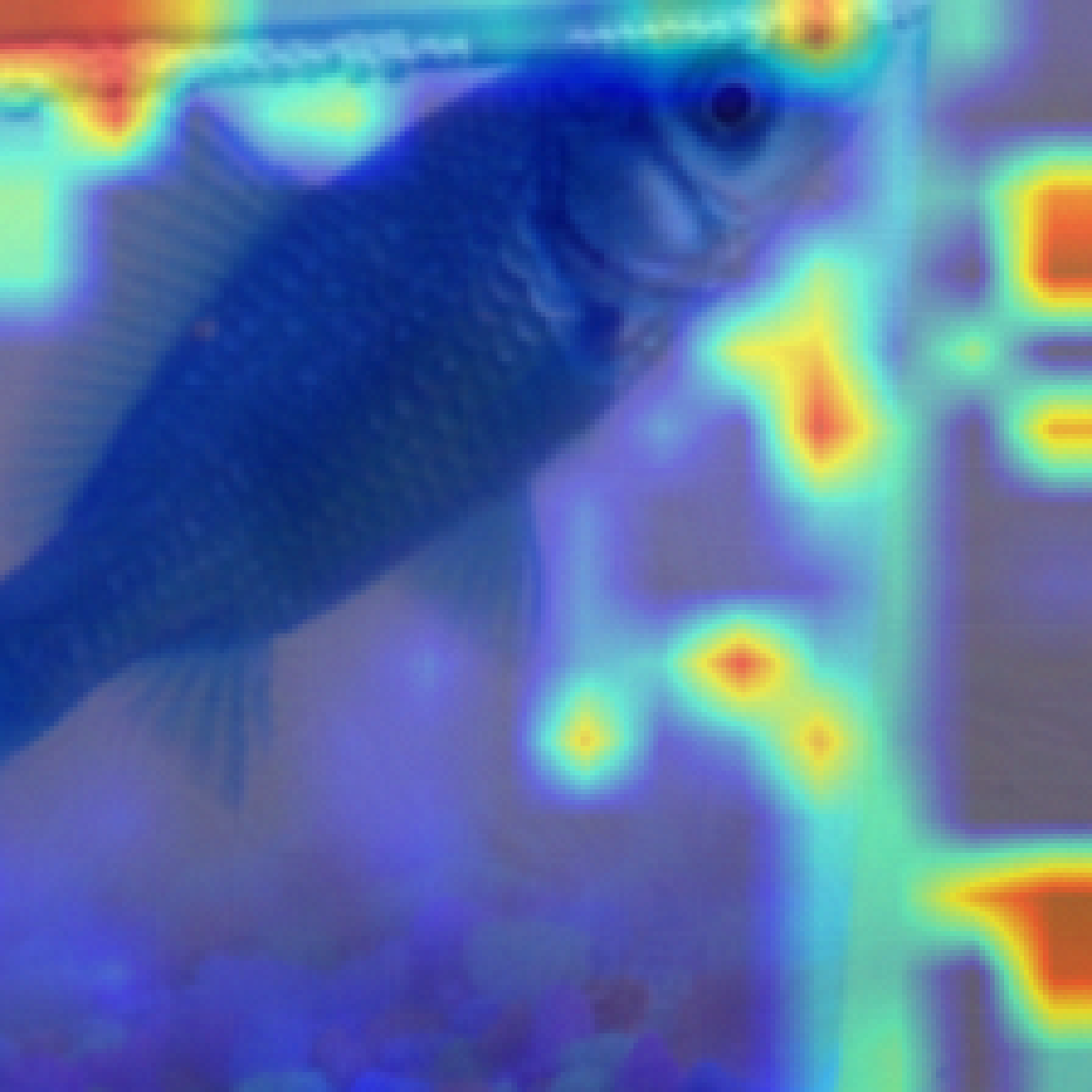} &
      \includegraphics[width=0.15\linewidth]{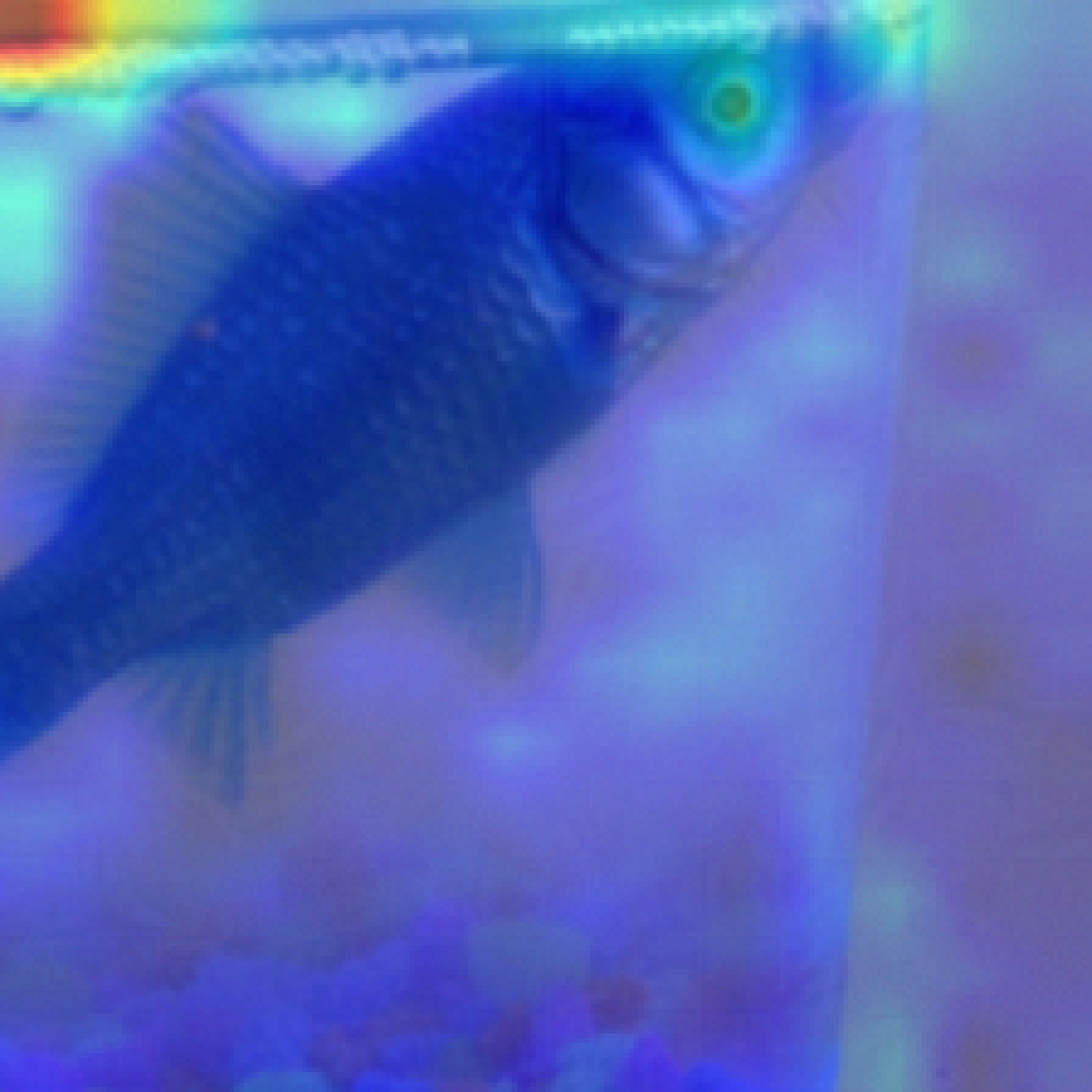} &
      \includegraphics[width=0.15\linewidth]{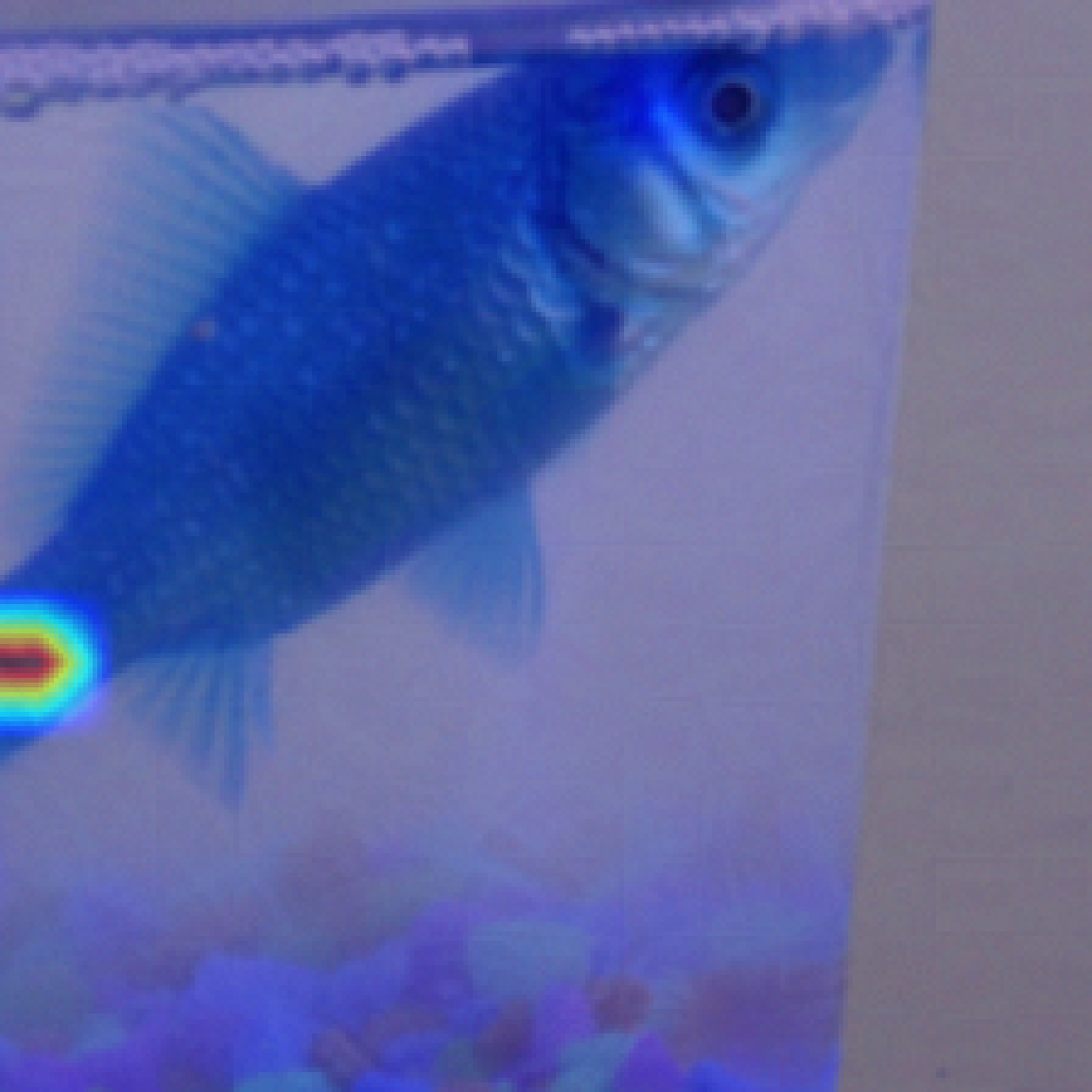} &
      \includegraphics[width=0.15\linewidth]{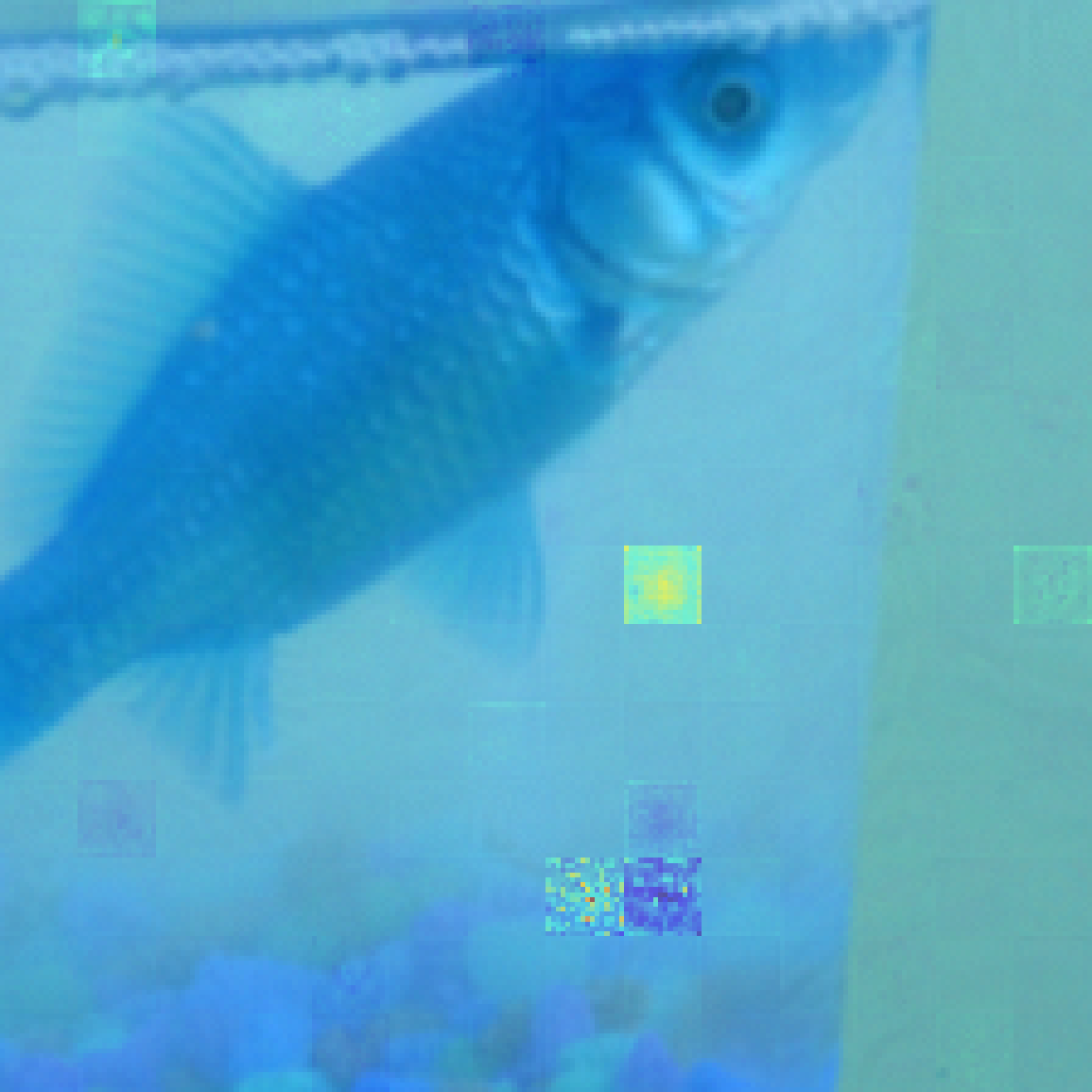} &
      \includegraphics[width=0.15\linewidth]{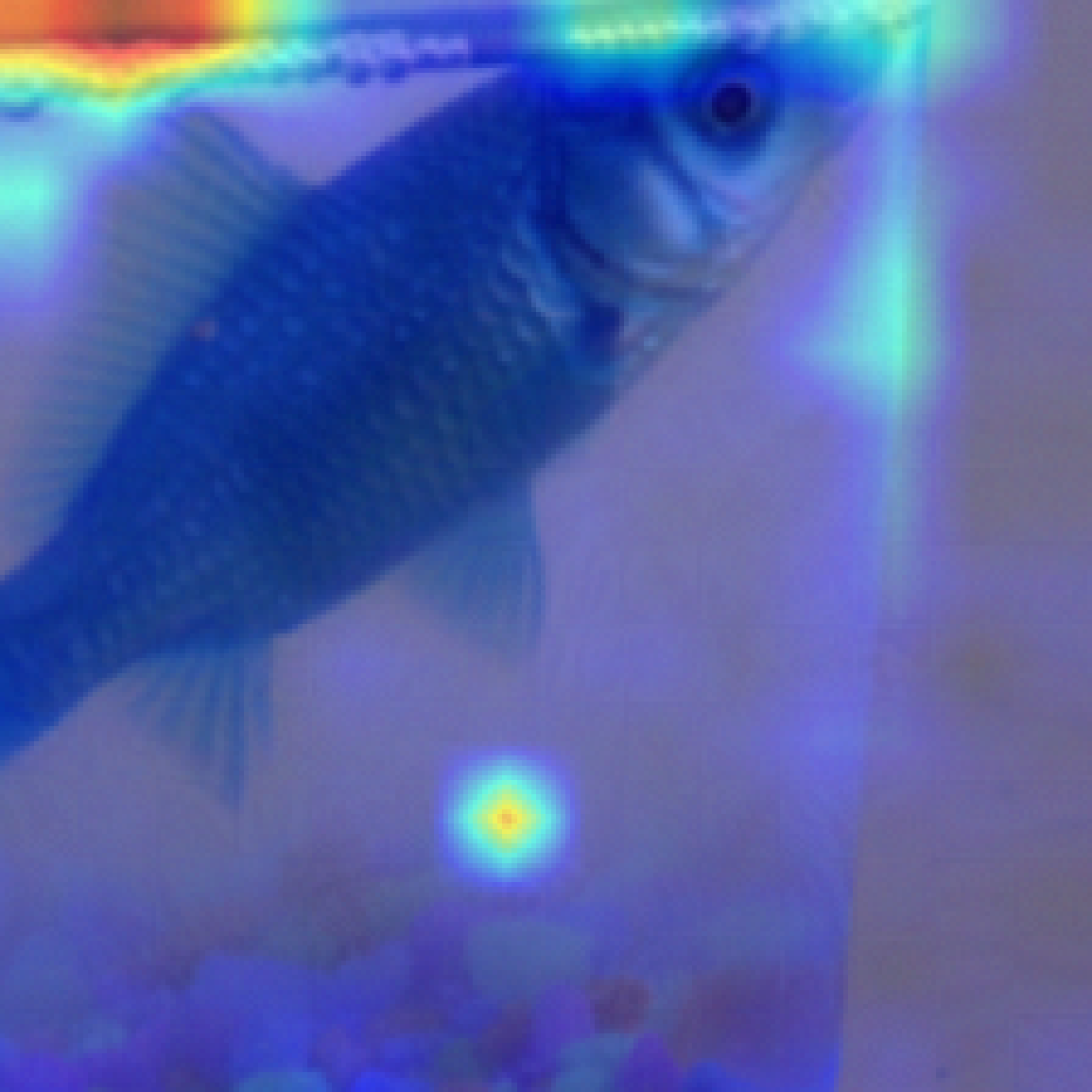} &
      \includegraphics[width=0.15\linewidth]{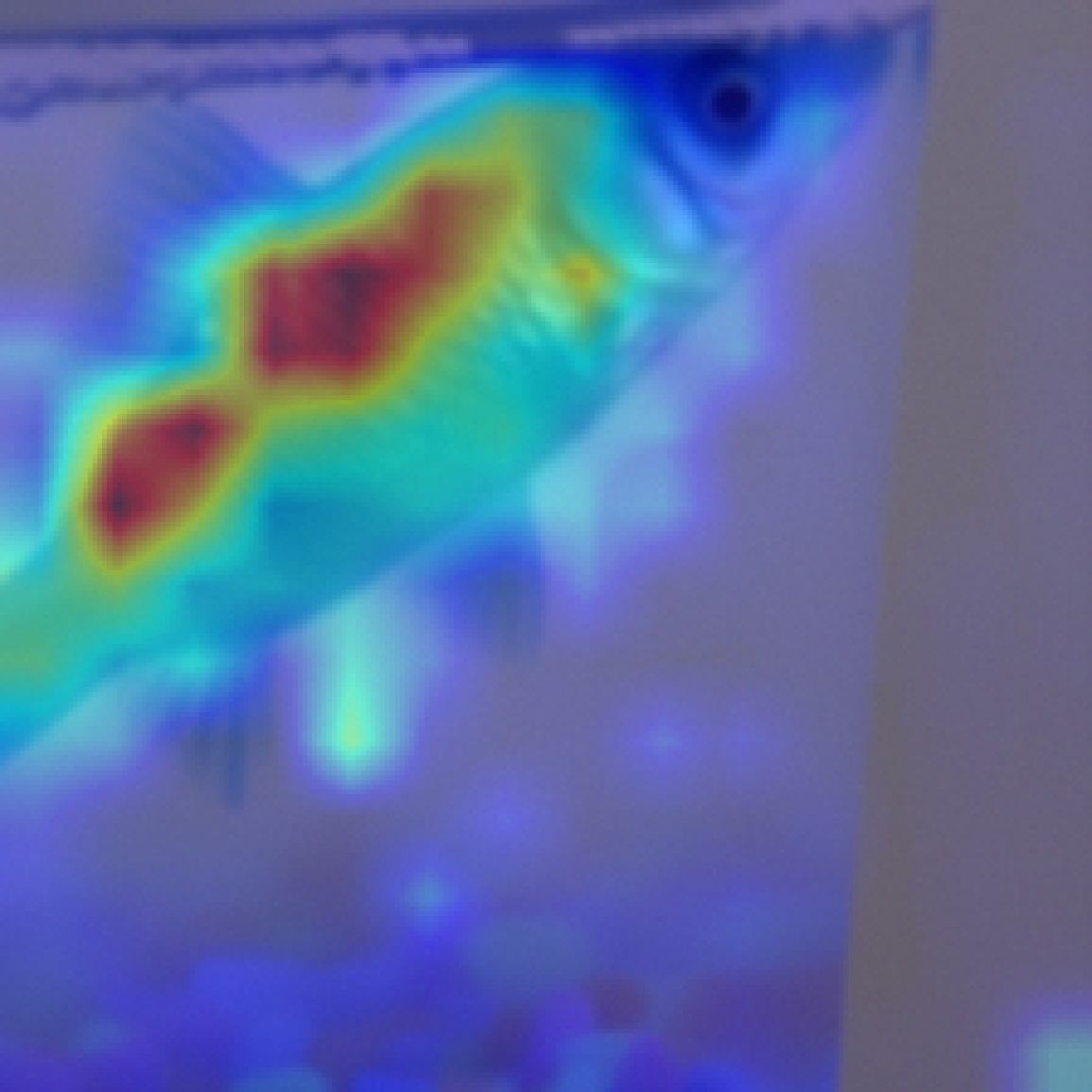} \\
        \raisebox{13mm}{\makecell*[c]{Poisoned$\rightarrow 5/255$\\\includegraphics[width=0.15\linewidth]{exp/fish/fish.png}
     }} &
     \includegraphics[width=0.15\linewidth]{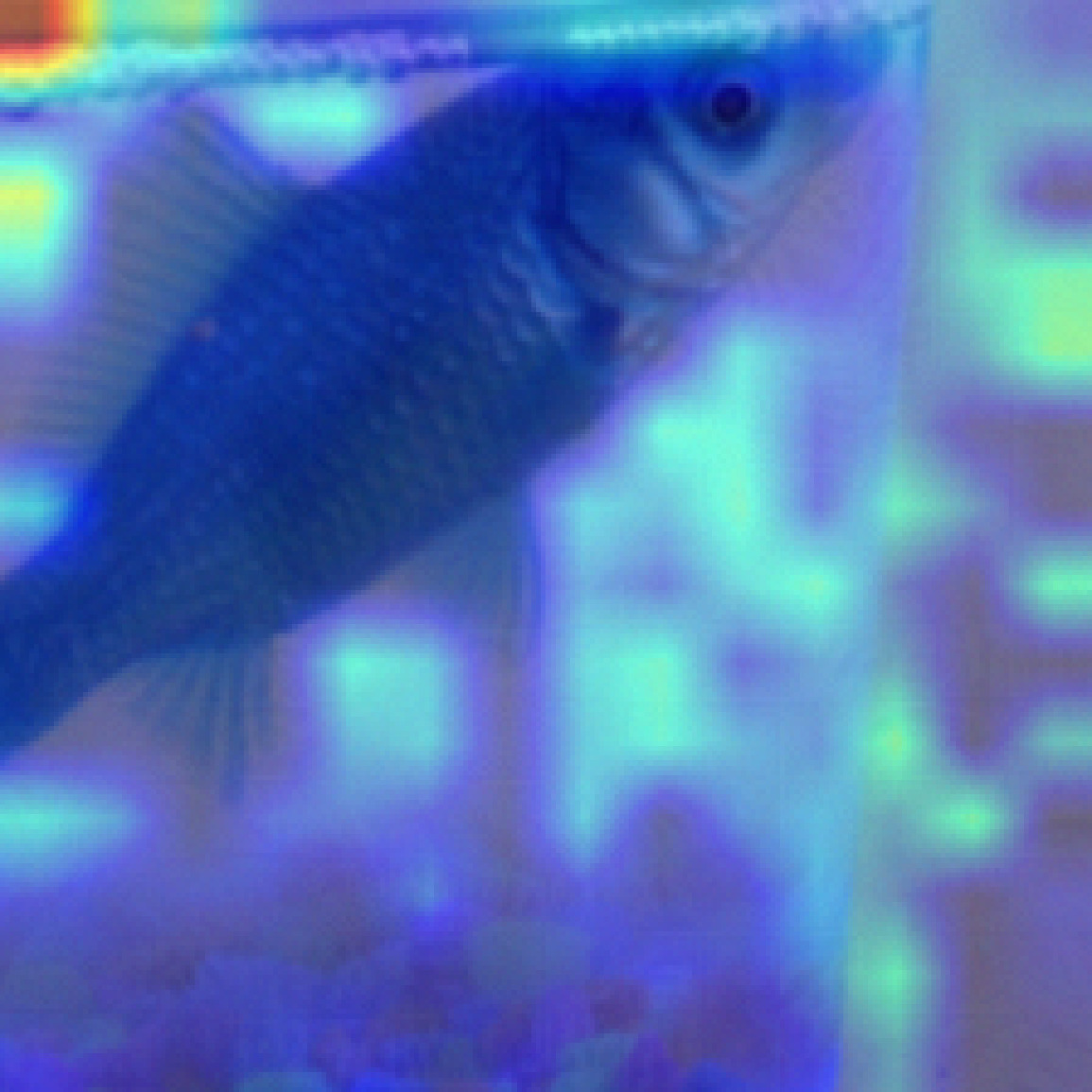} &
      \includegraphics[width=0.15\linewidth]{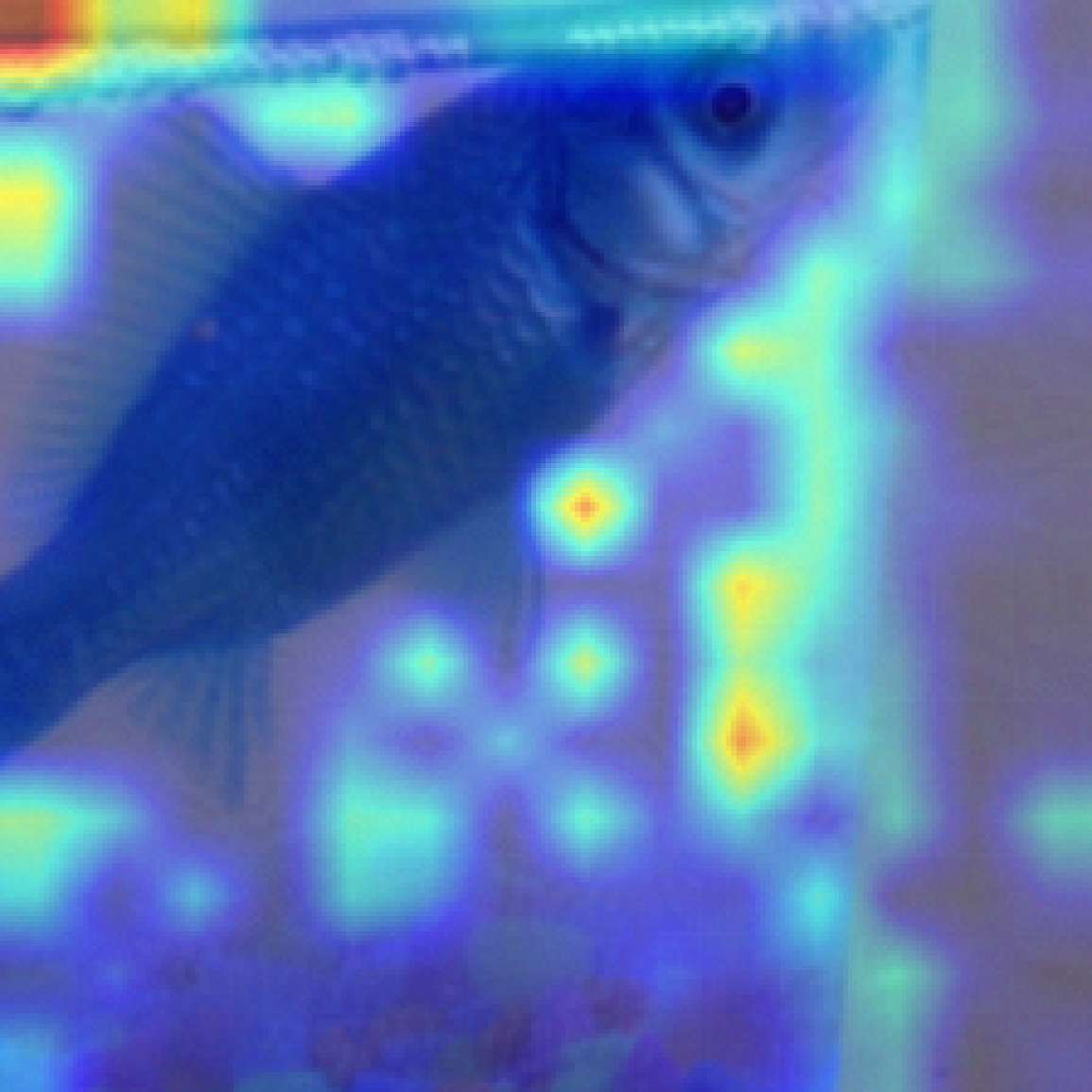} &
      \includegraphics[width=0.15\linewidth]{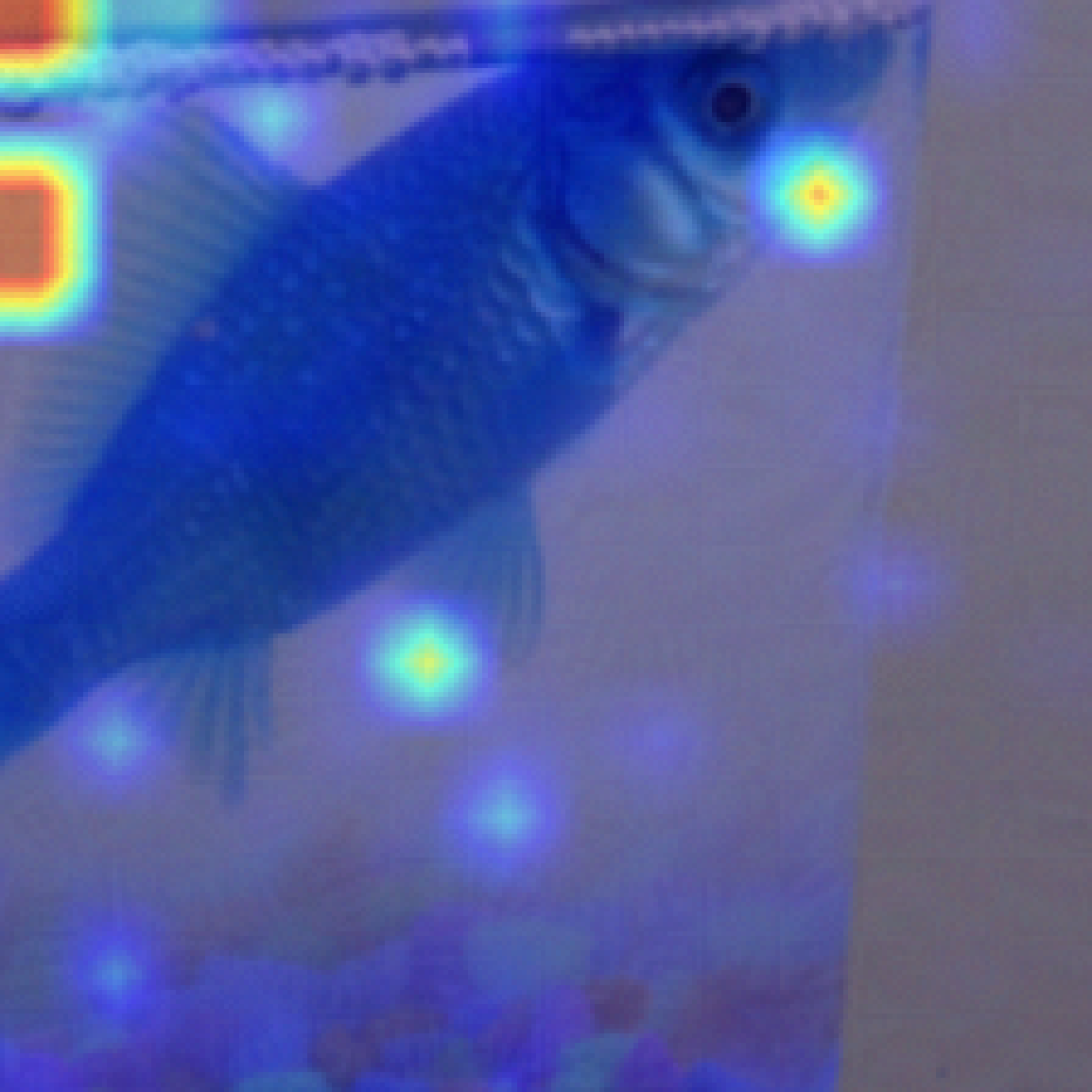} &
      \includegraphics[width=0.15\linewidth]{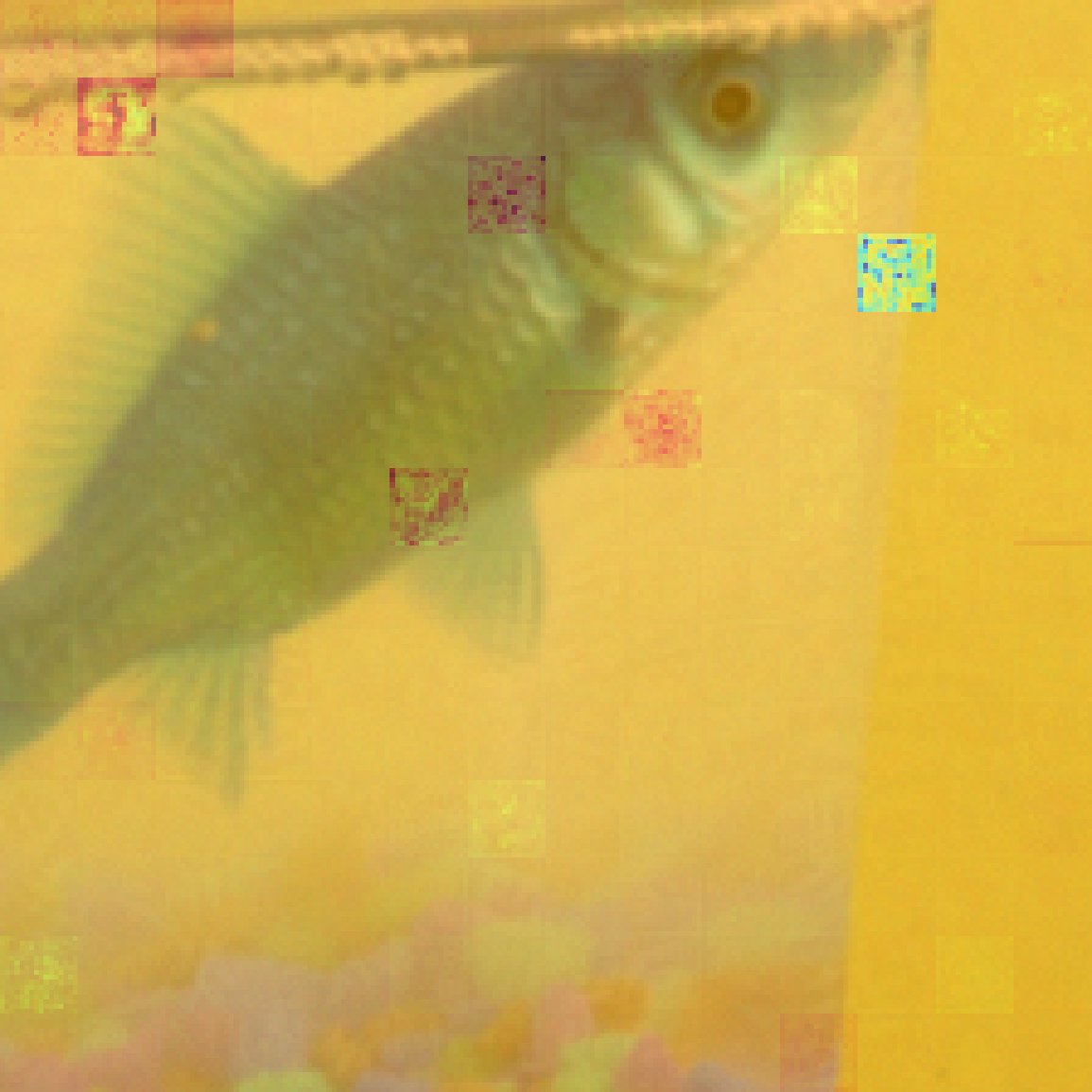} &
      \includegraphics[width=0.15\linewidth]{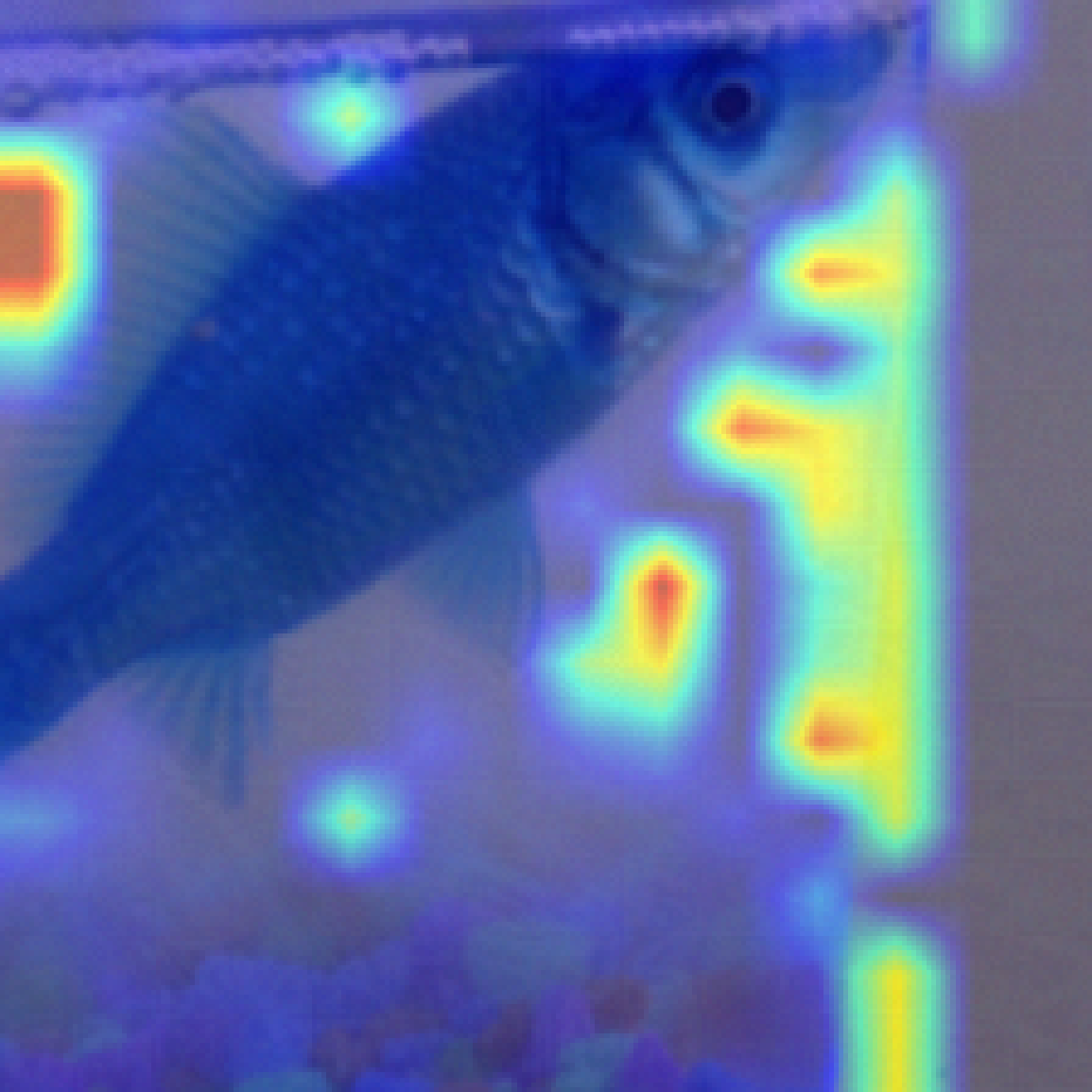} &
      \includegraphics[width=0.15\linewidth]{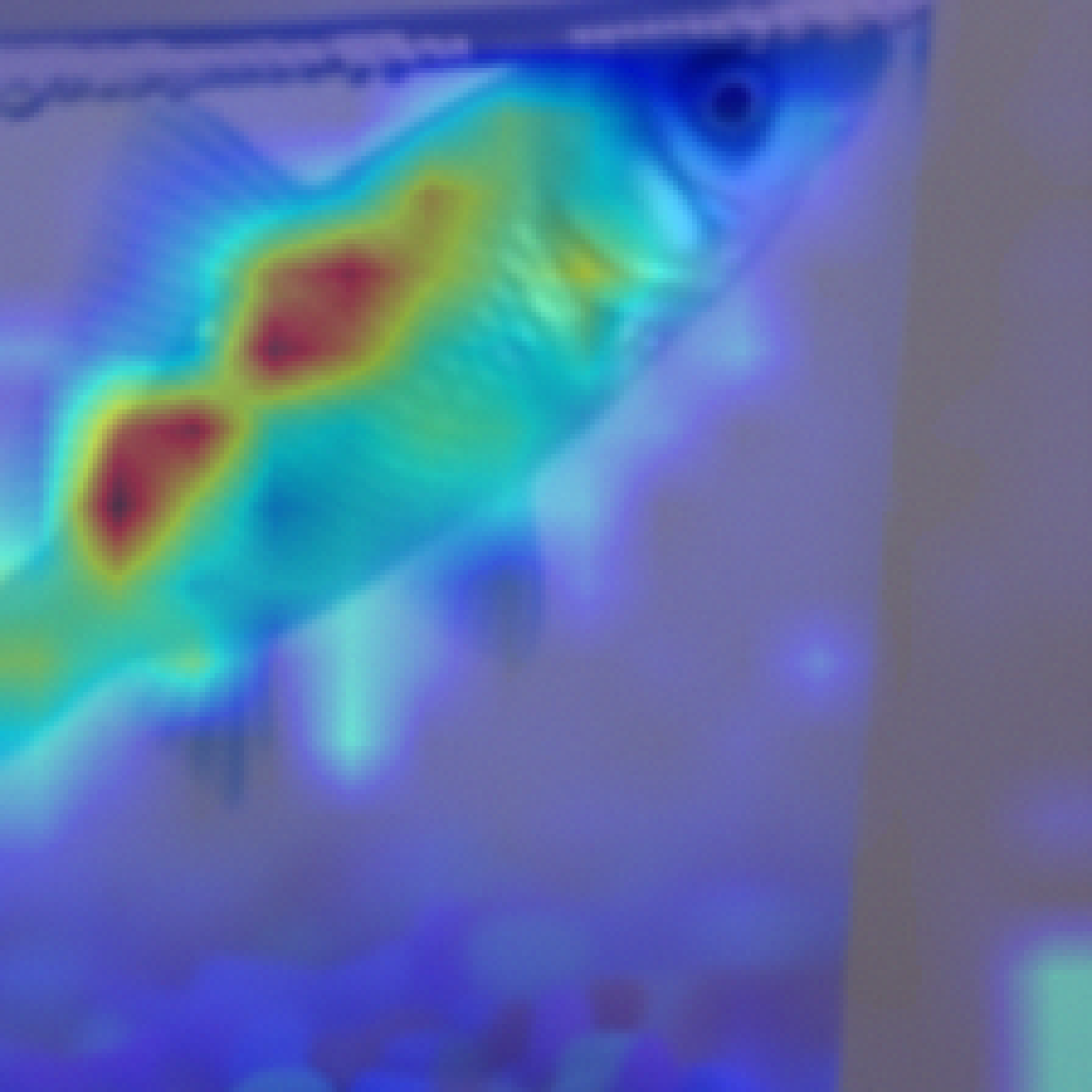} \\

     \end{tabular*}
     \caption{Visualization results of the attention map on corrupted input for different methods under different attack radii. }
     \label{fig: diff-posin2}
     \end{center}
 \end{figure*}

 \begin{figure*}[htbp]
     \setlength{\tabcolsep}{1pt} 
     \renewcommand{\arraystretch}{1} 
     \begin{center}
     \begin{tabular*}{\linewidth}{@{\extracolsep{\fill}}ccccccc}
      Corrupted Input & Raw Attention & Rollout 
      & GradCAM &LRP& VTA  & Ours \\
        \raisebox{13mm}{\makecell*[c]{Clean$\rightarrow$\\\includegraphics[width=0.15\linewidth]{exp/fish/fish.png}
     }} &
     \includegraphics[width=0.15\linewidth]{exp/fish-a0/fish-a0-raw_attn-1-perturbed-0.png} &
      \includegraphics[width=0.15\linewidth]{exp/fish-a0/fish-a0-rollout-1-perturbed-0.png} &
      \includegraphics[width=0.15\linewidth]{exp/fish-a0/fish-a0-attn_gradcam-1-perturbed-0.png} &
      \includegraphics[width=0.15\linewidth]{exp/fish-a0/fish-a0-full-LRP-1-perturbed-0.png} &
      \includegraphics[width=0.15\linewidth]{exp/fish-a0/fish-a0-vta-1-perturbed-0.png} &
      \includegraphics[width=0.15\linewidth]{exp/fish-a0/fish-a0-ours-1-perturbed-0.png} \\
      \raisebox{13mm}{\makecell*[c]{Poisoned$\rightarrow 6/255$\\\includegraphics[width=0.15\linewidth]{exp/fish/fish.png}
     }} &
     \includegraphics[width=0.15\linewidth]{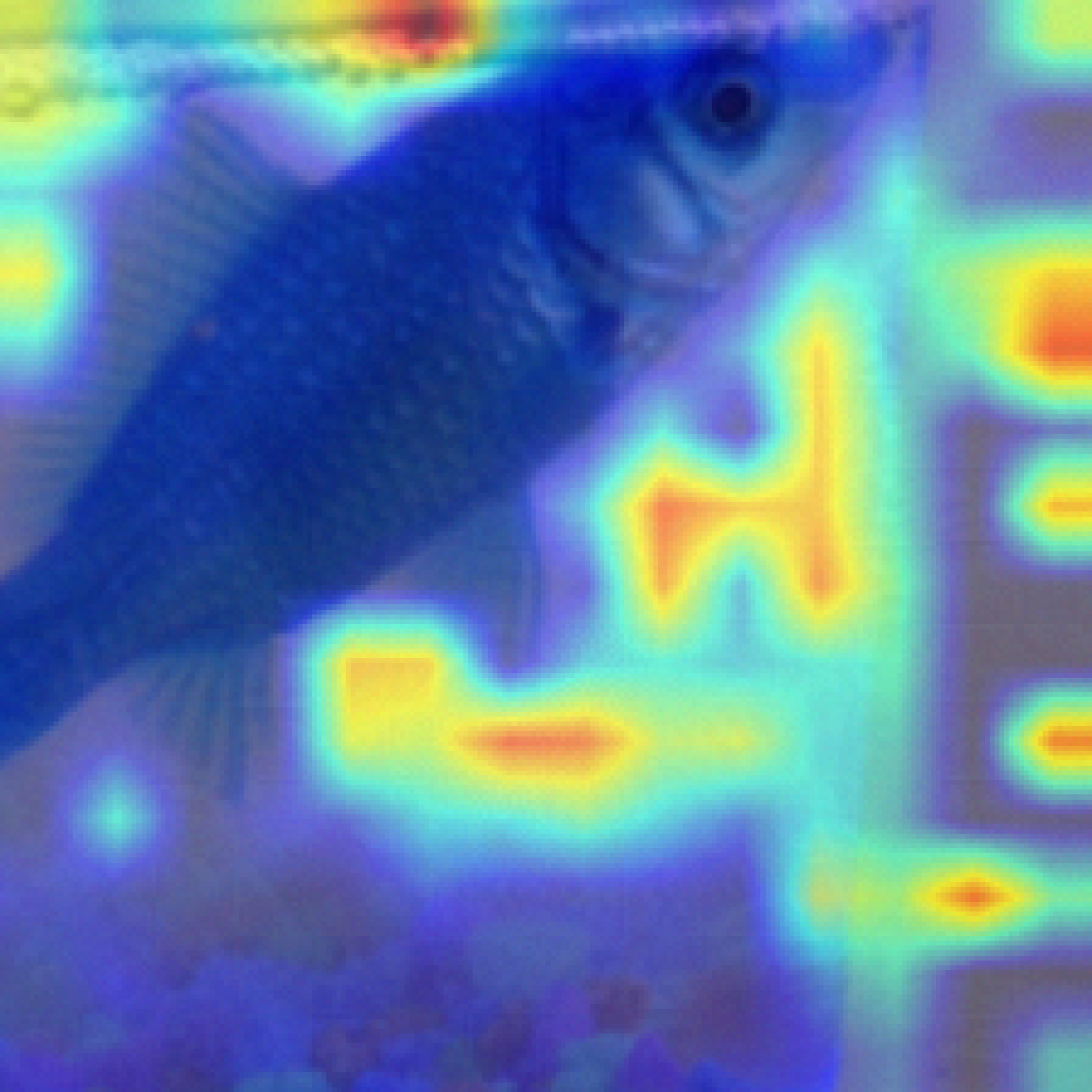} &
      \includegraphics[width=0.15\linewidth]{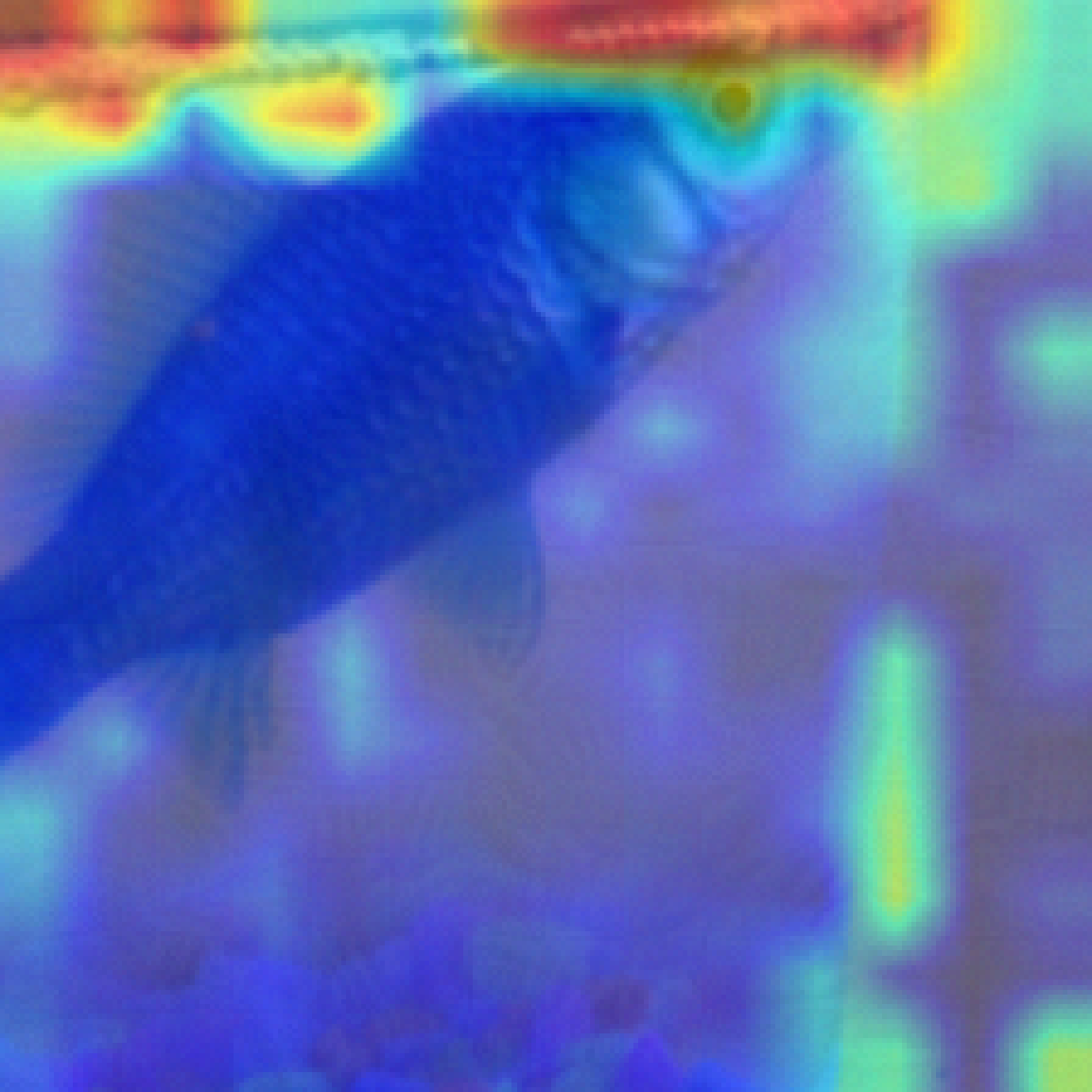} &
      \includegraphics[width=0.15\linewidth]{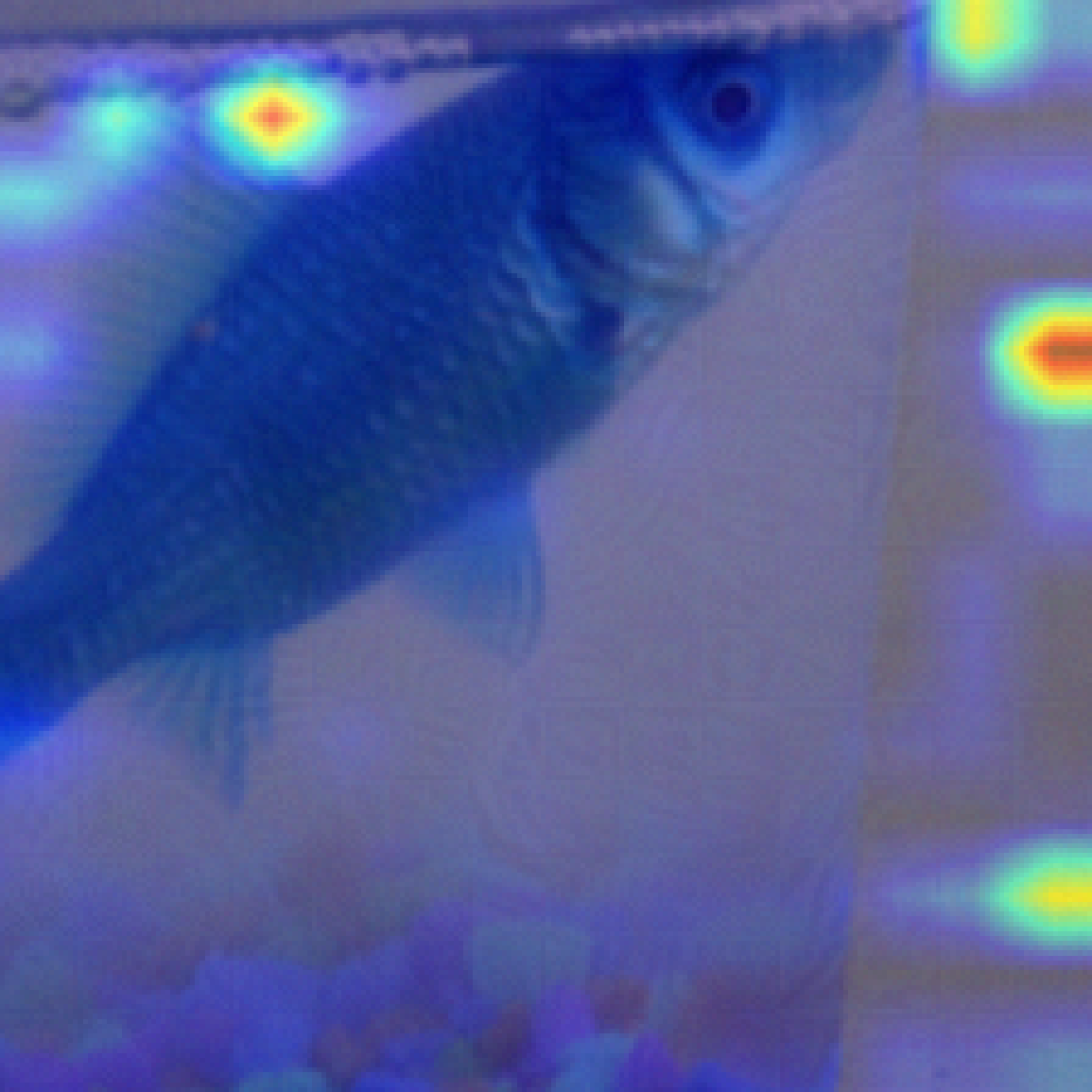} &
      \includegraphics[width=0.15\linewidth]{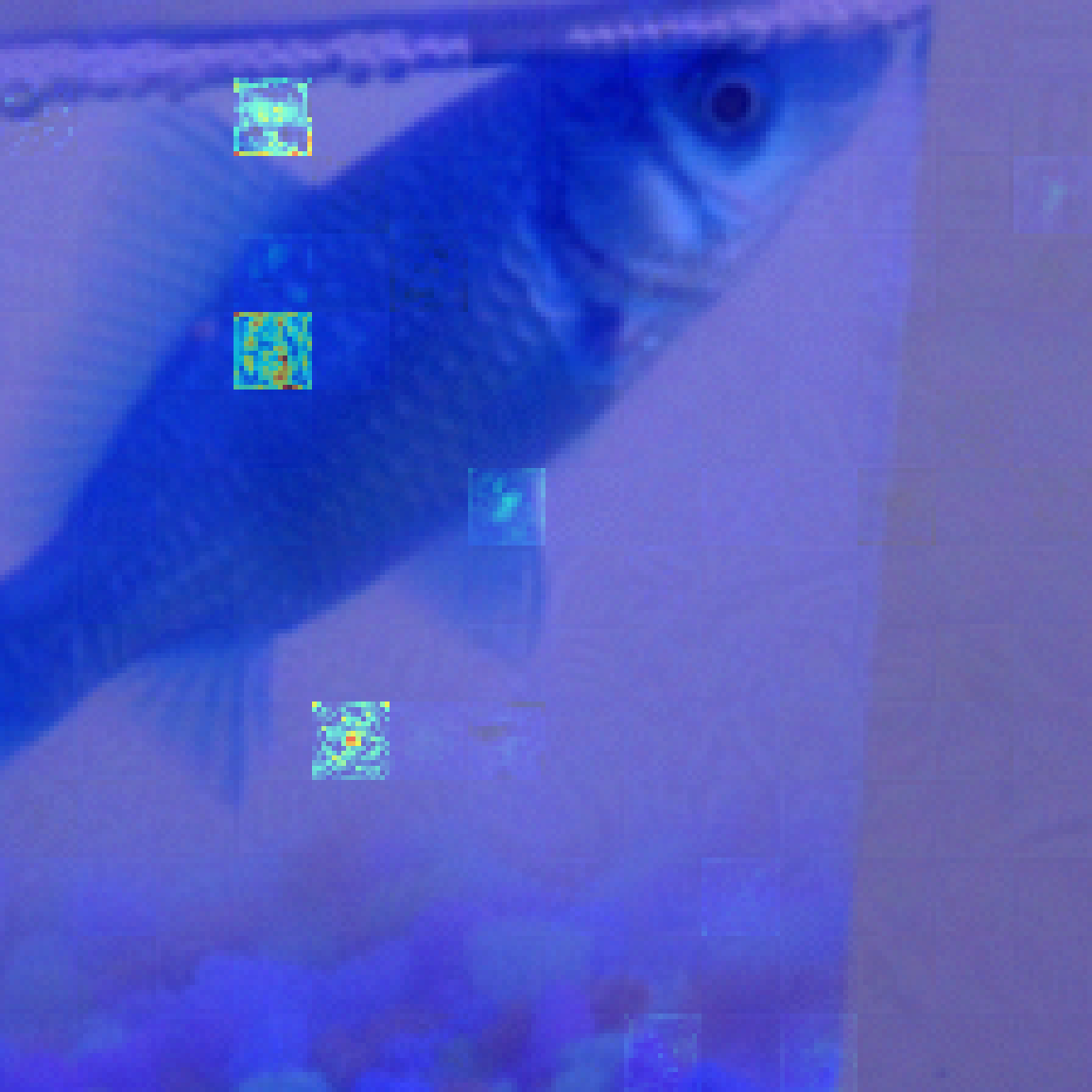} &
      \includegraphics[width=0.15\linewidth]{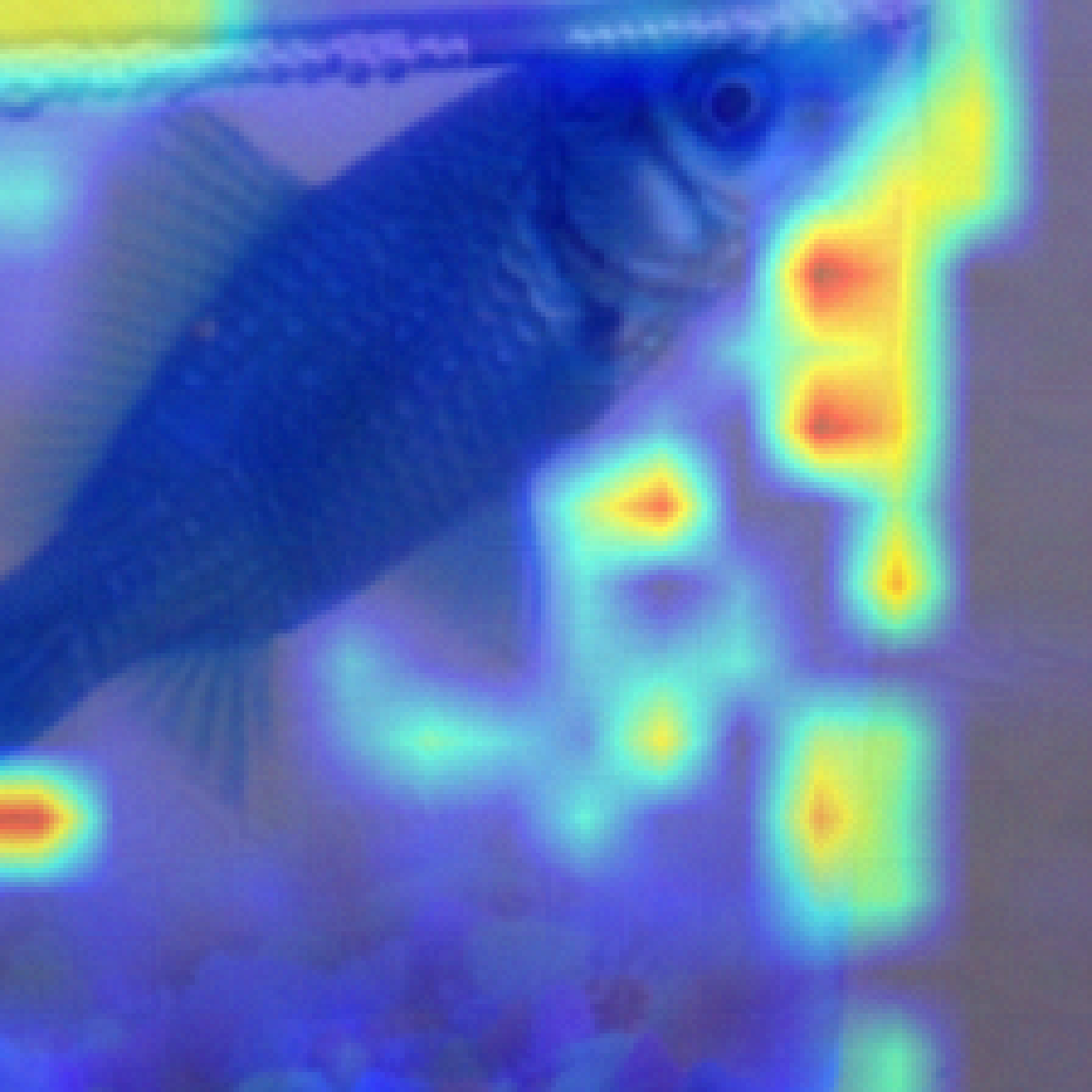} &
      \includegraphics[width=0.15\linewidth]{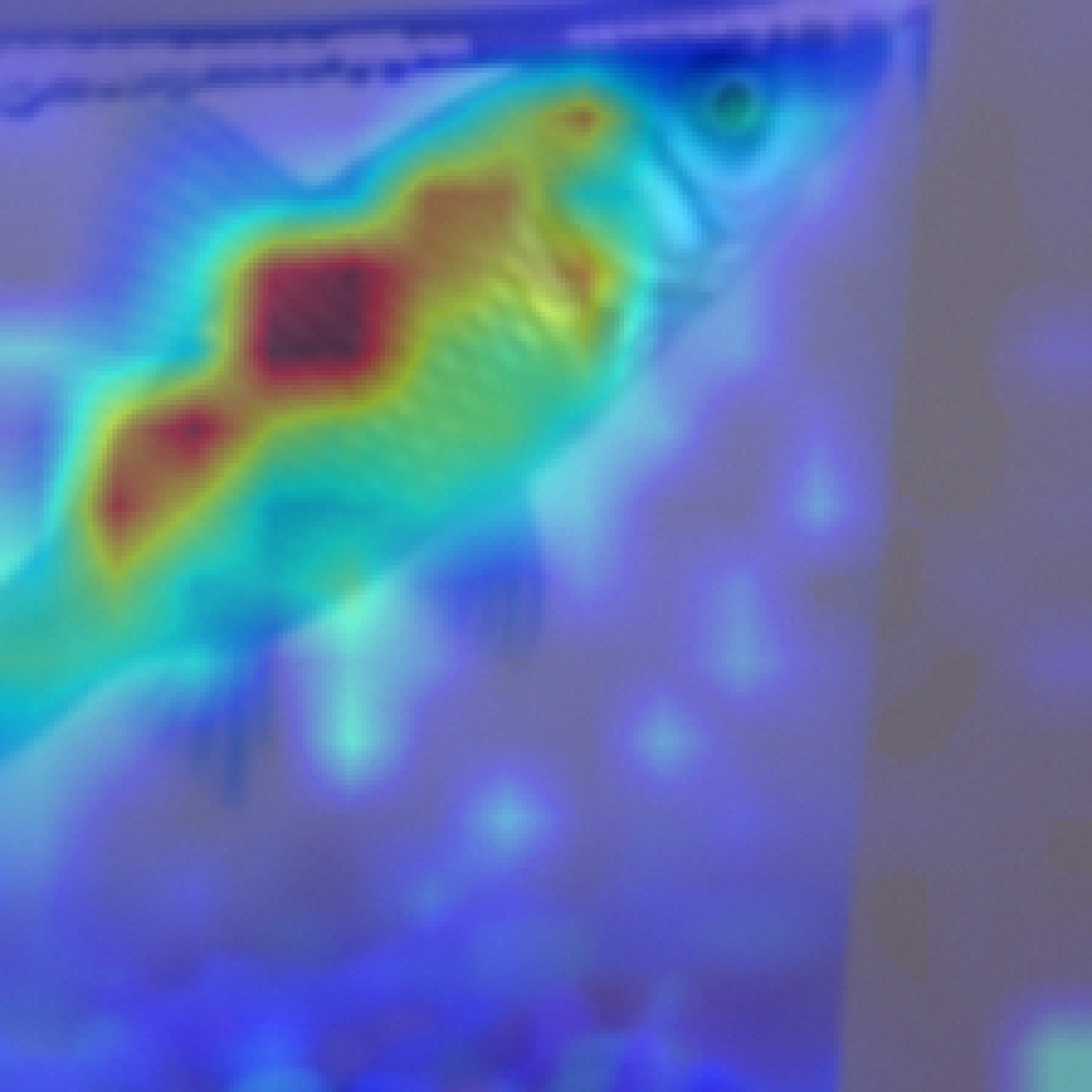} \\
      \raisebox{13mm}{\makecell*[c]{Poisoned$\rightarrow 7/255$\\\includegraphics[width=0.15\linewidth]{exp/fish/fish.png}
     }} &
     \includegraphics[width=0.15\linewidth]{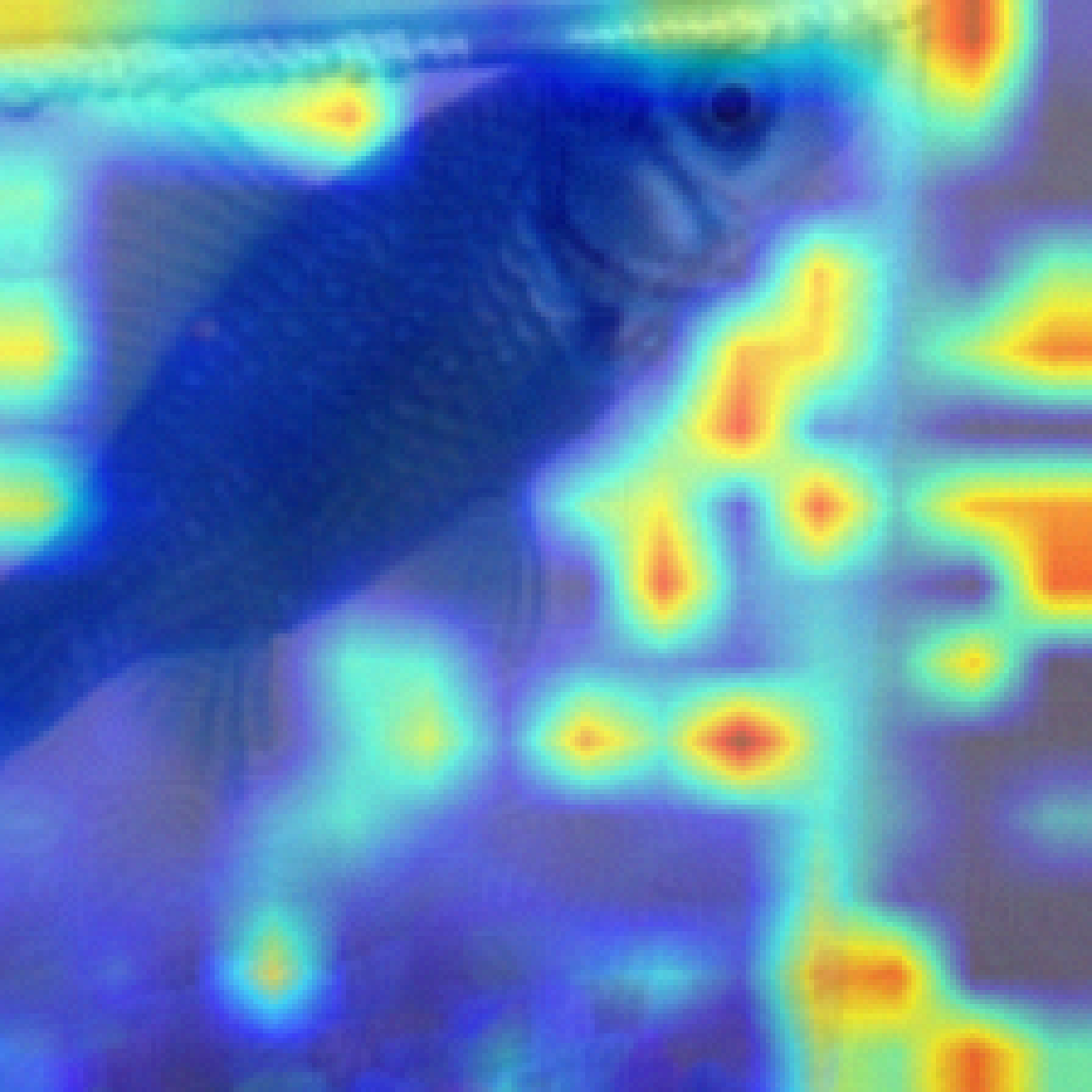} &
      \includegraphics[width=0.15\linewidth]{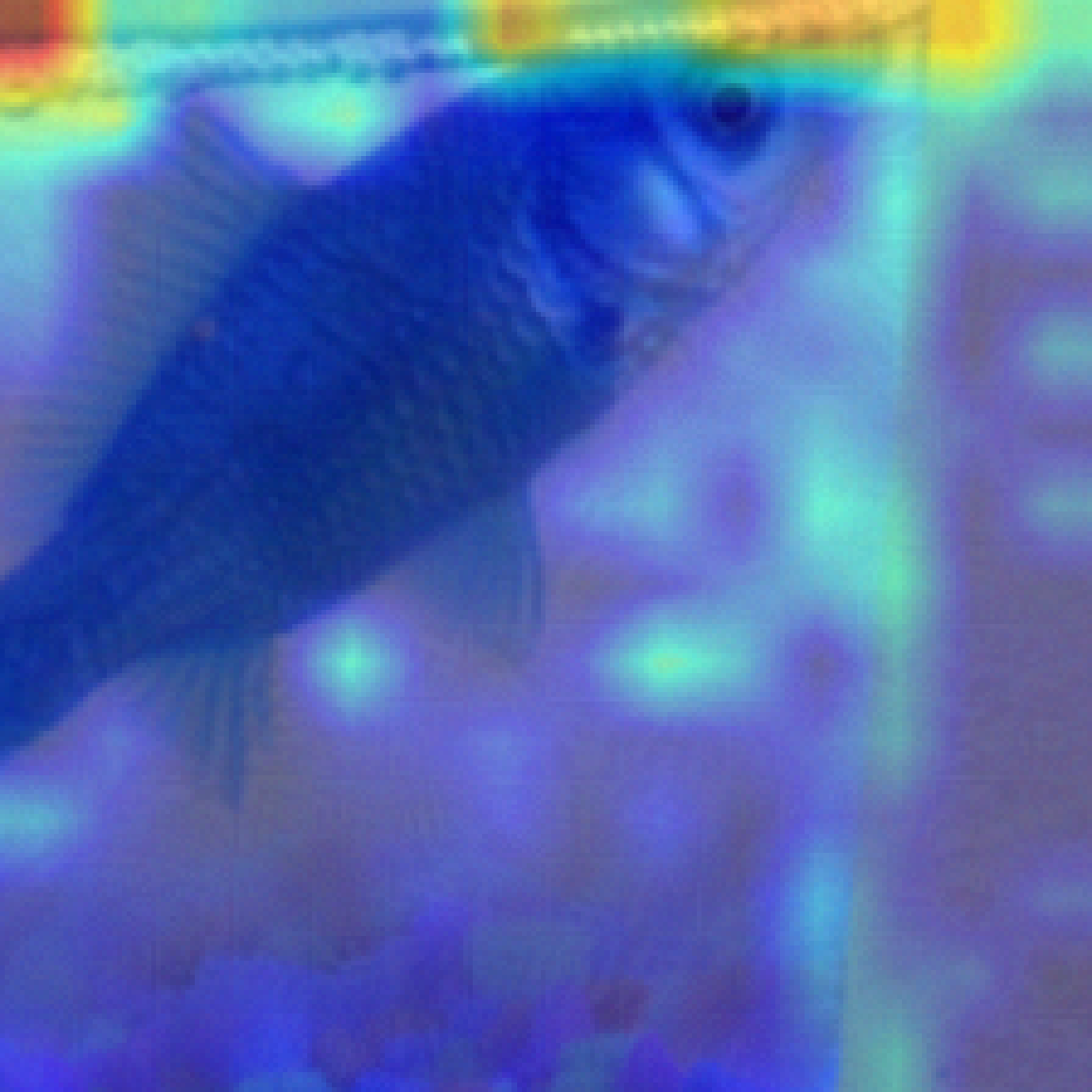} &
      \includegraphics[width=0.15\linewidth]{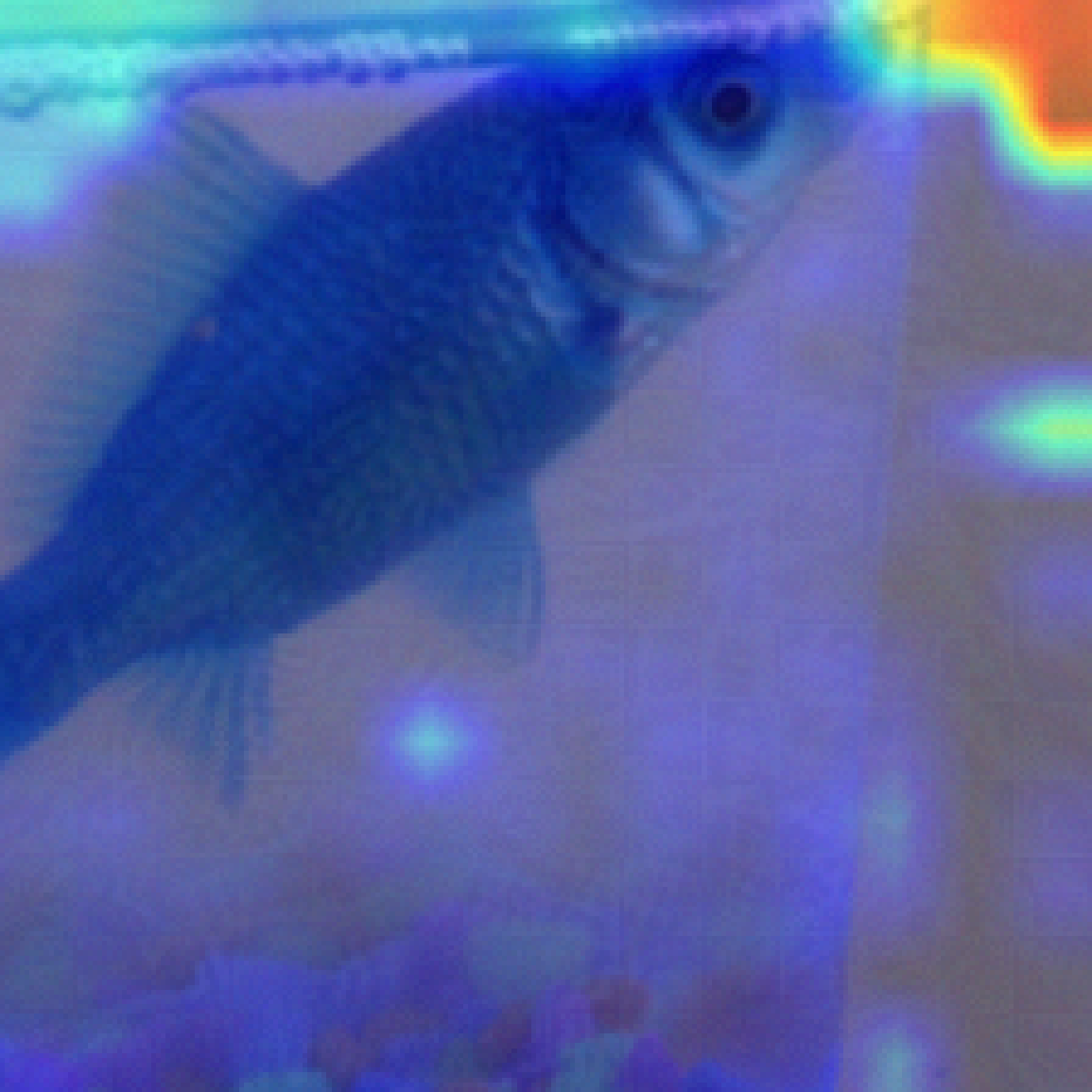} &
      \includegraphics[width=0.15\linewidth]{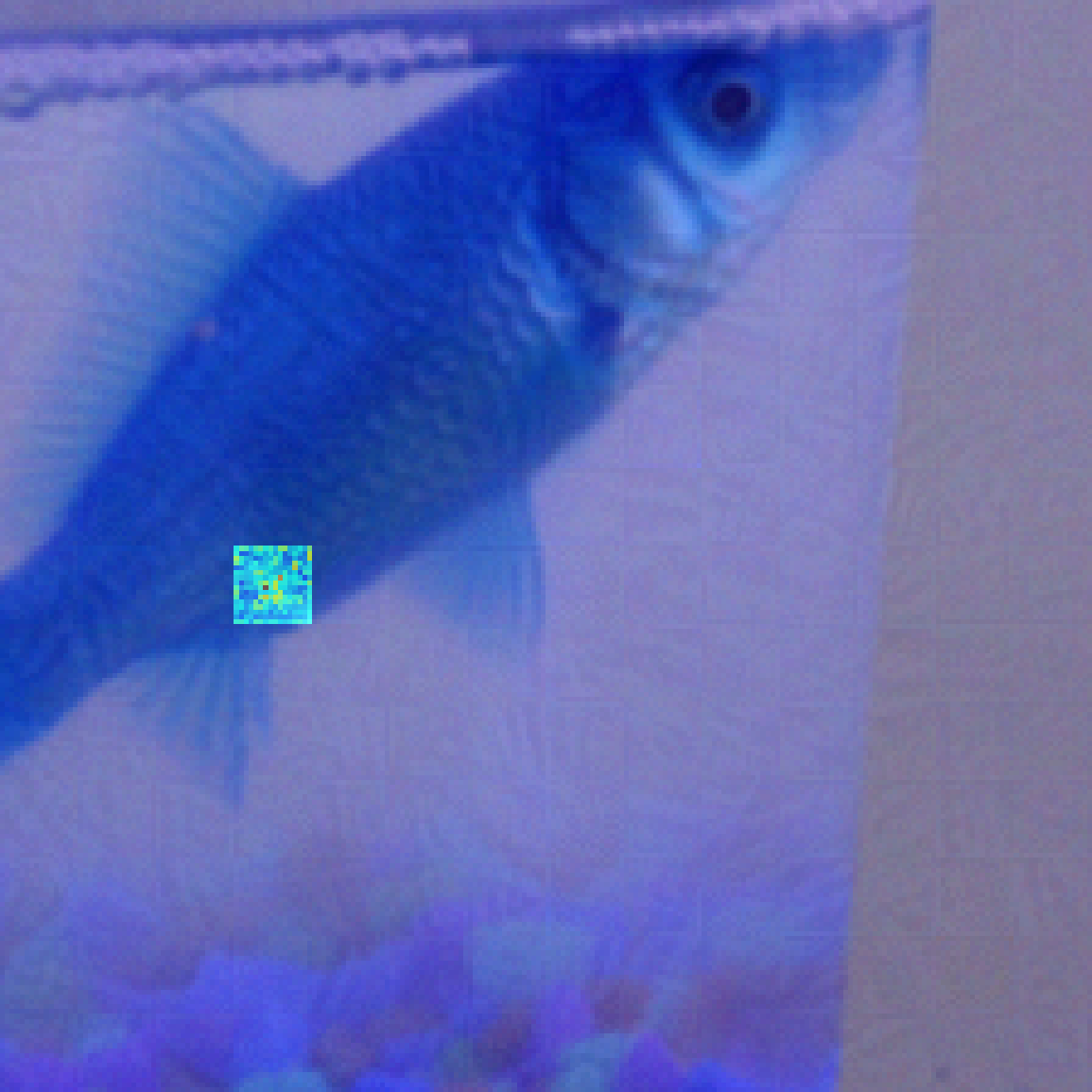} &
      \includegraphics[width=0.15\linewidth]{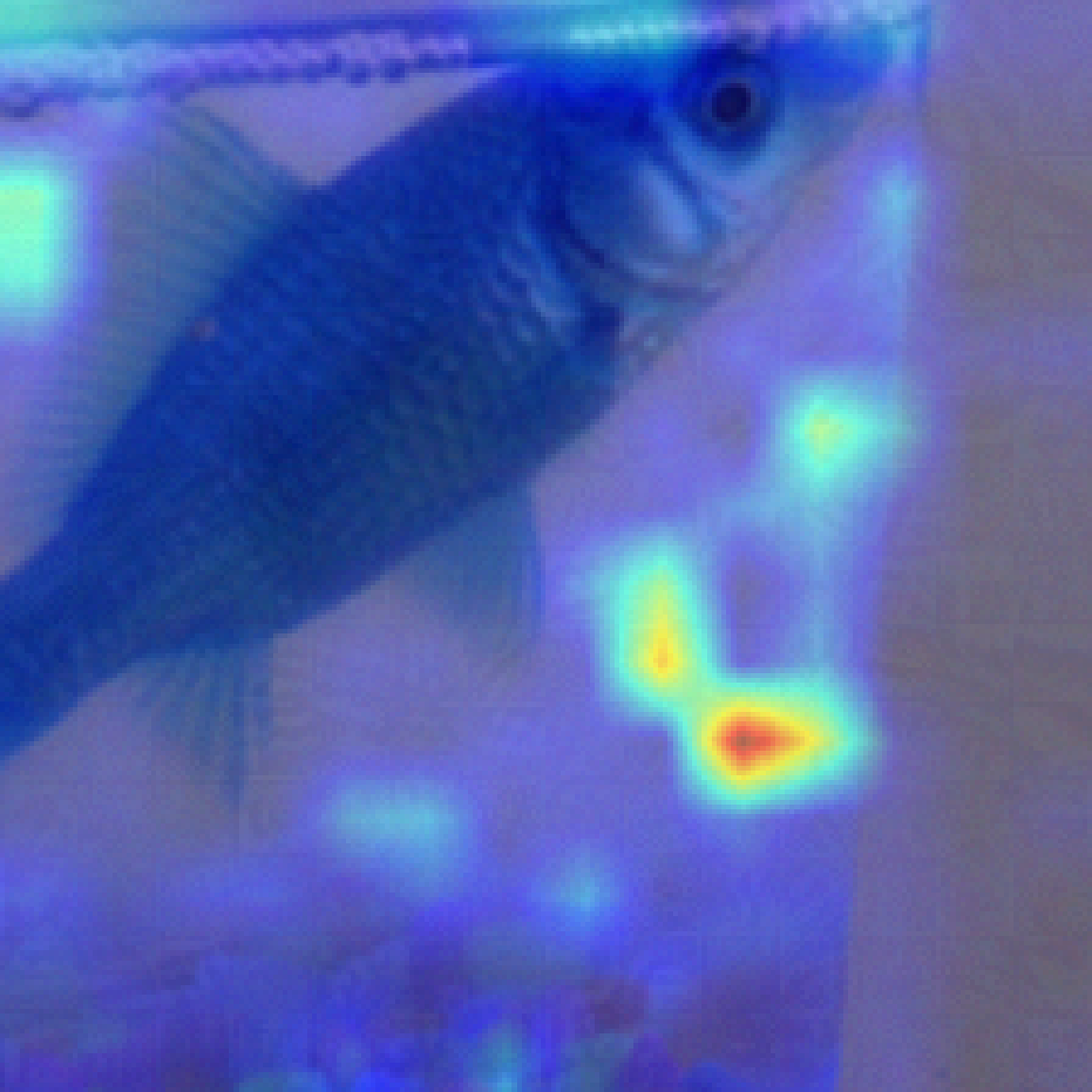} &
      \includegraphics[width=0.15\linewidth]{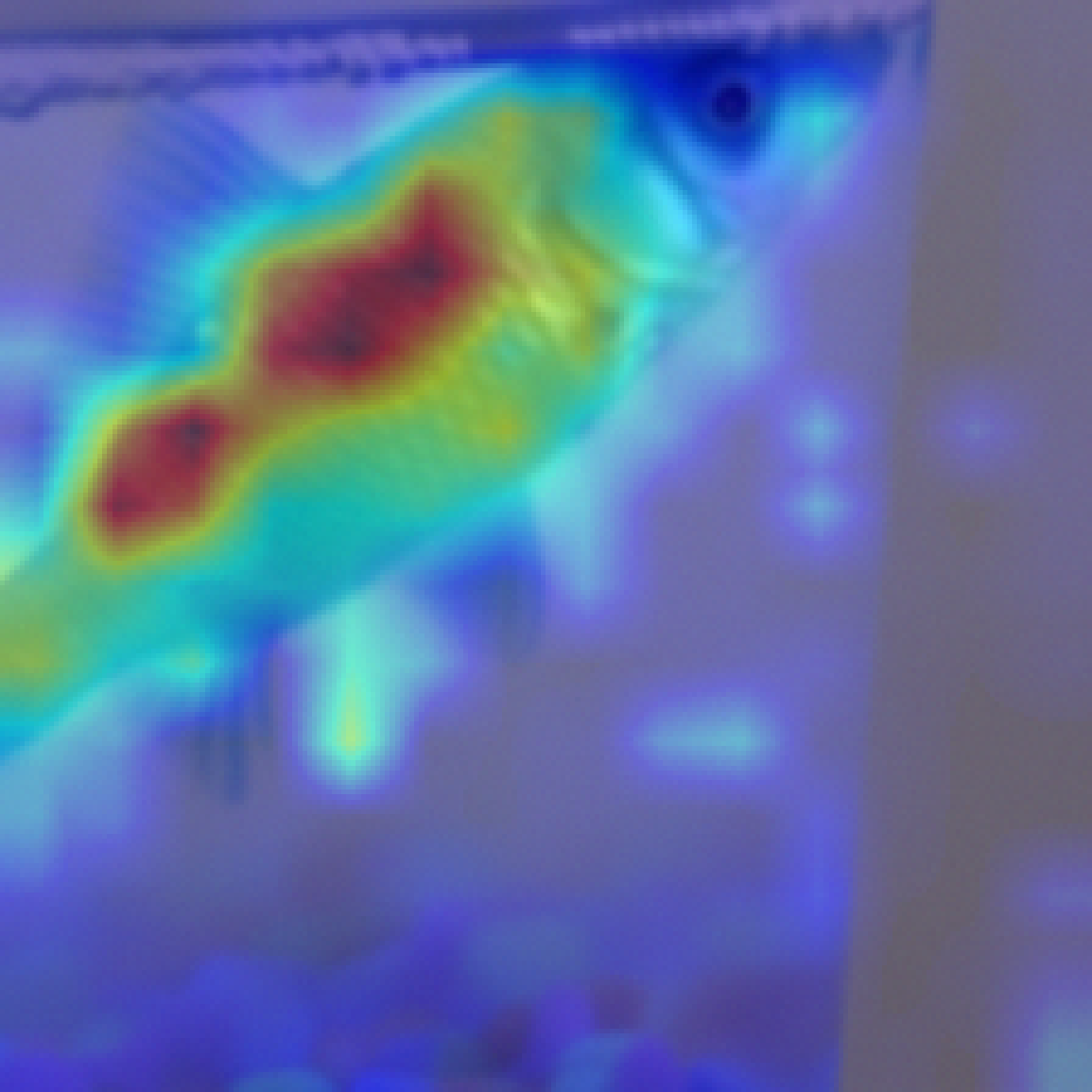} \\
      
      \raisebox{13mm}{\makecell*[c]{Poisoned$\rightarrow 8/255$\\\includegraphics[width=0.15\linewidth]{exp/fish/fish.png}
     }} &
     \includegraphics[width=0.15\linewidth]{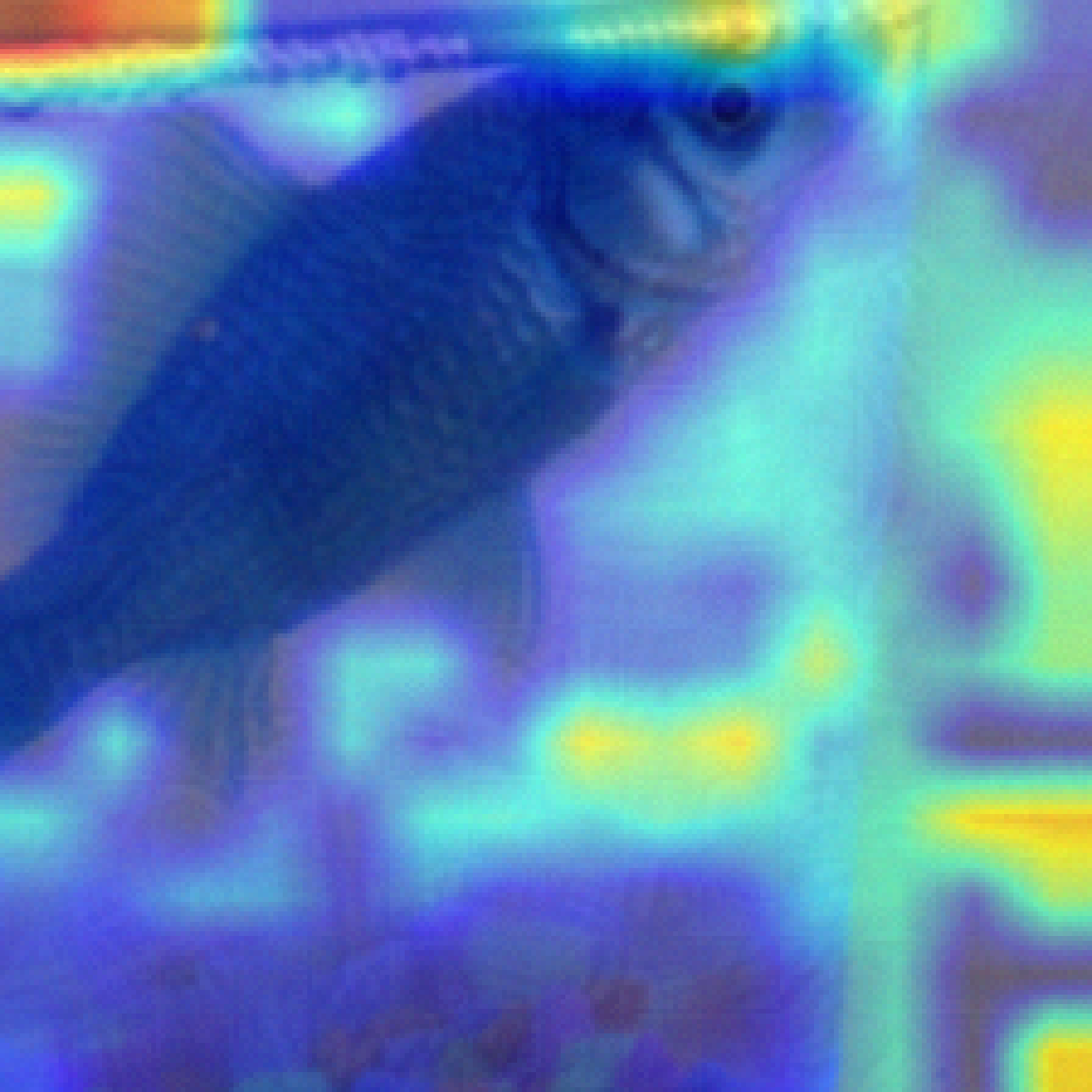} &
      \includegraphics[width=0.15\linewidth]{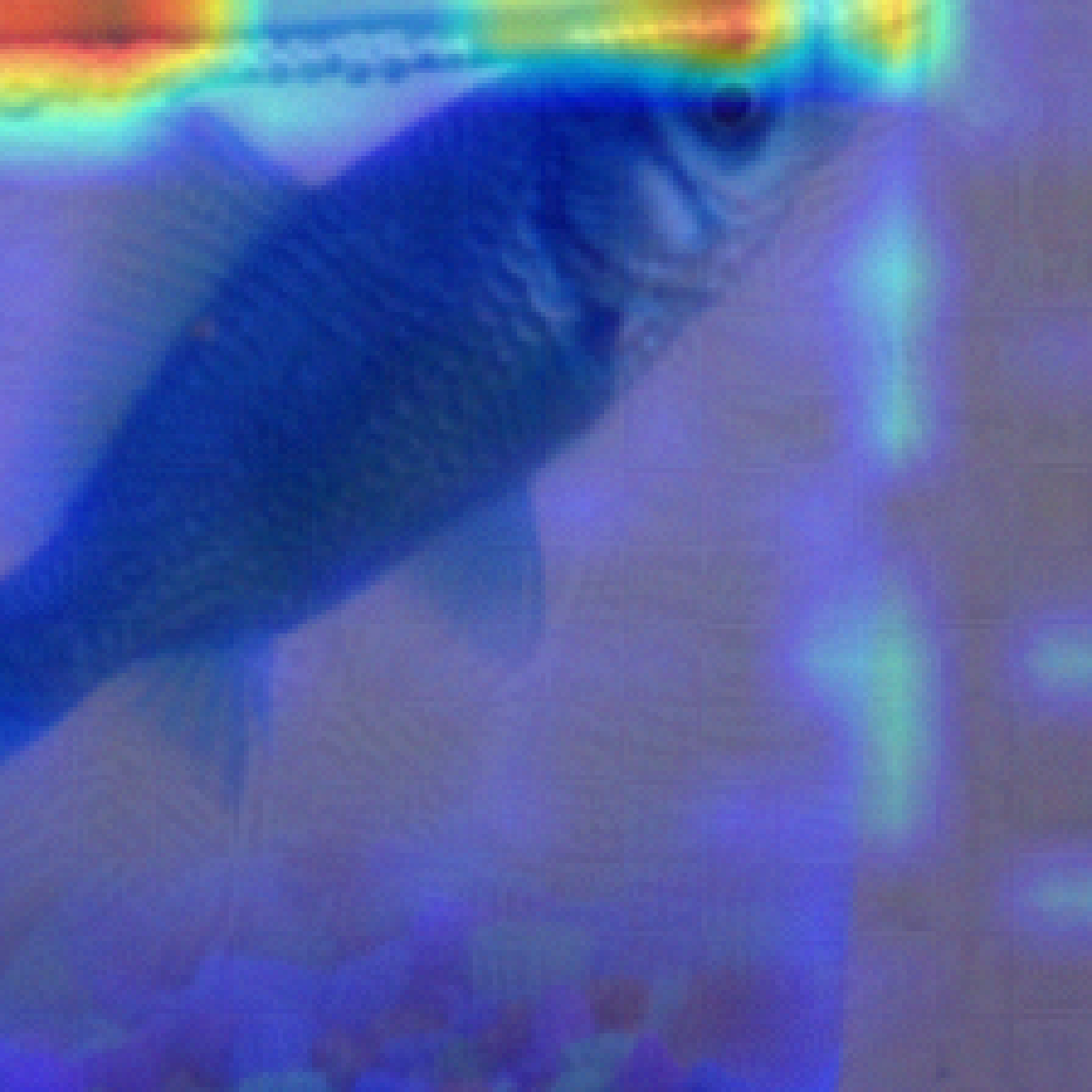} &
      \includegraphics[width=0.15\linewidth]{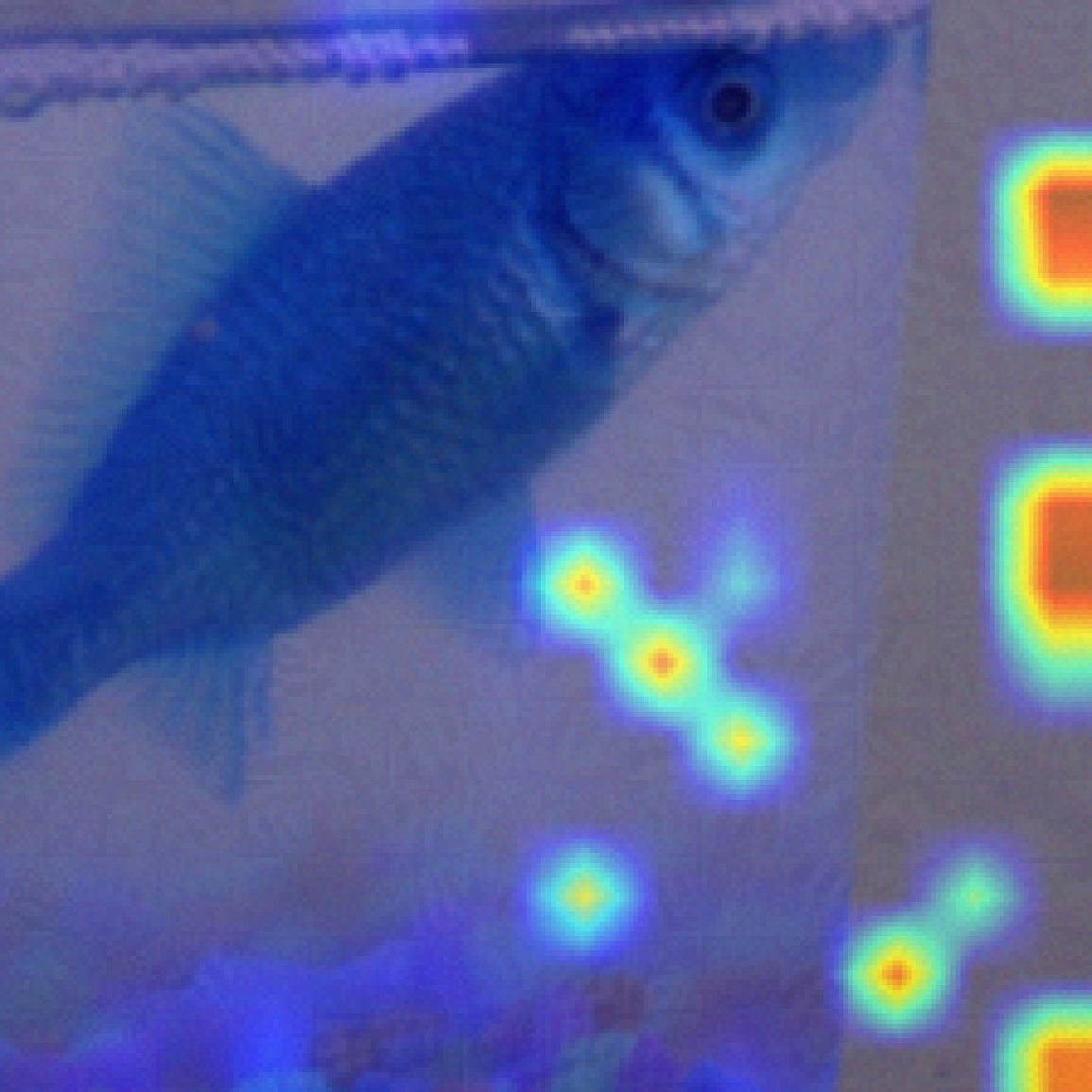} &
      \includegraphics[width=0.15\linewidth]{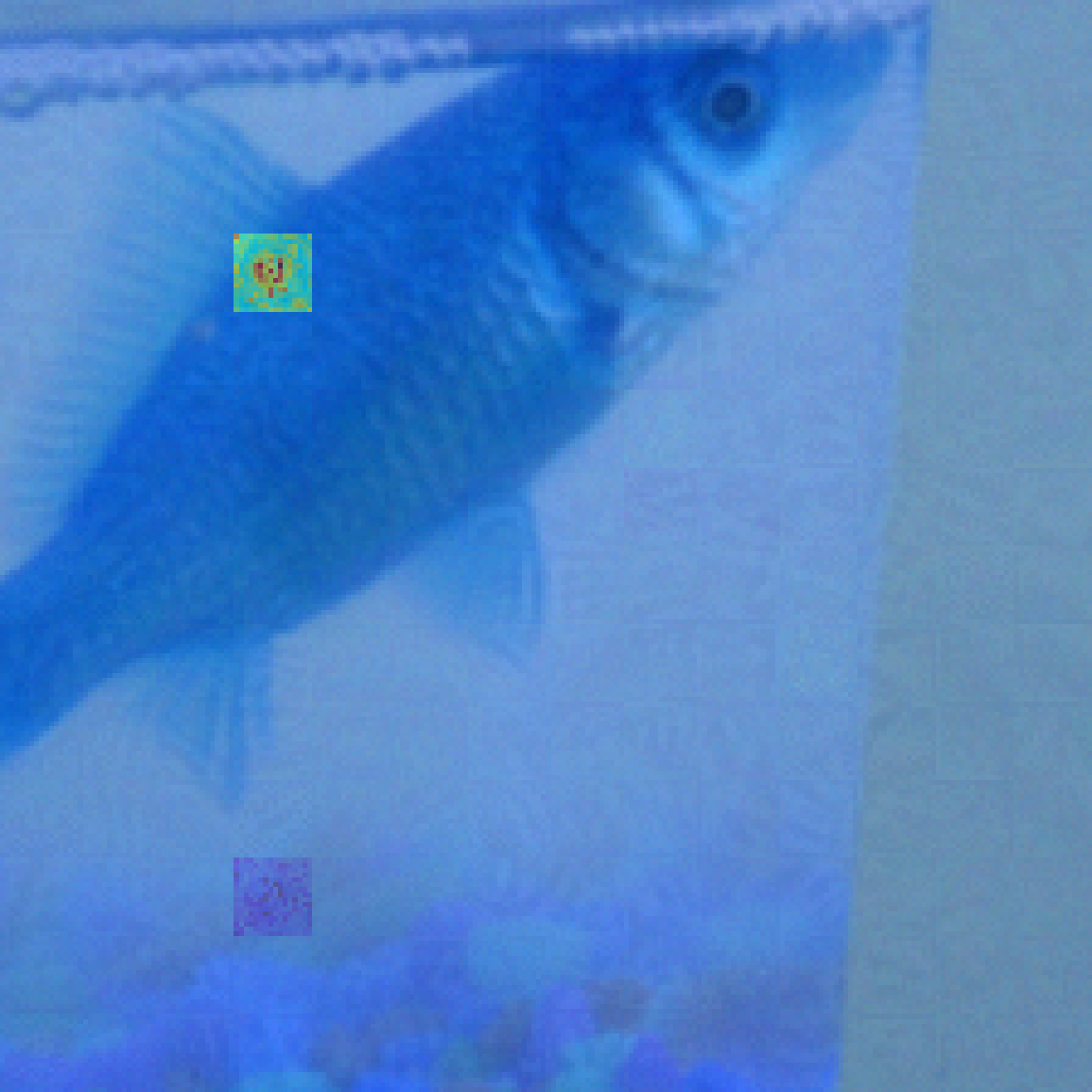} &
      \includegraphics[width=0.15\linewidth]{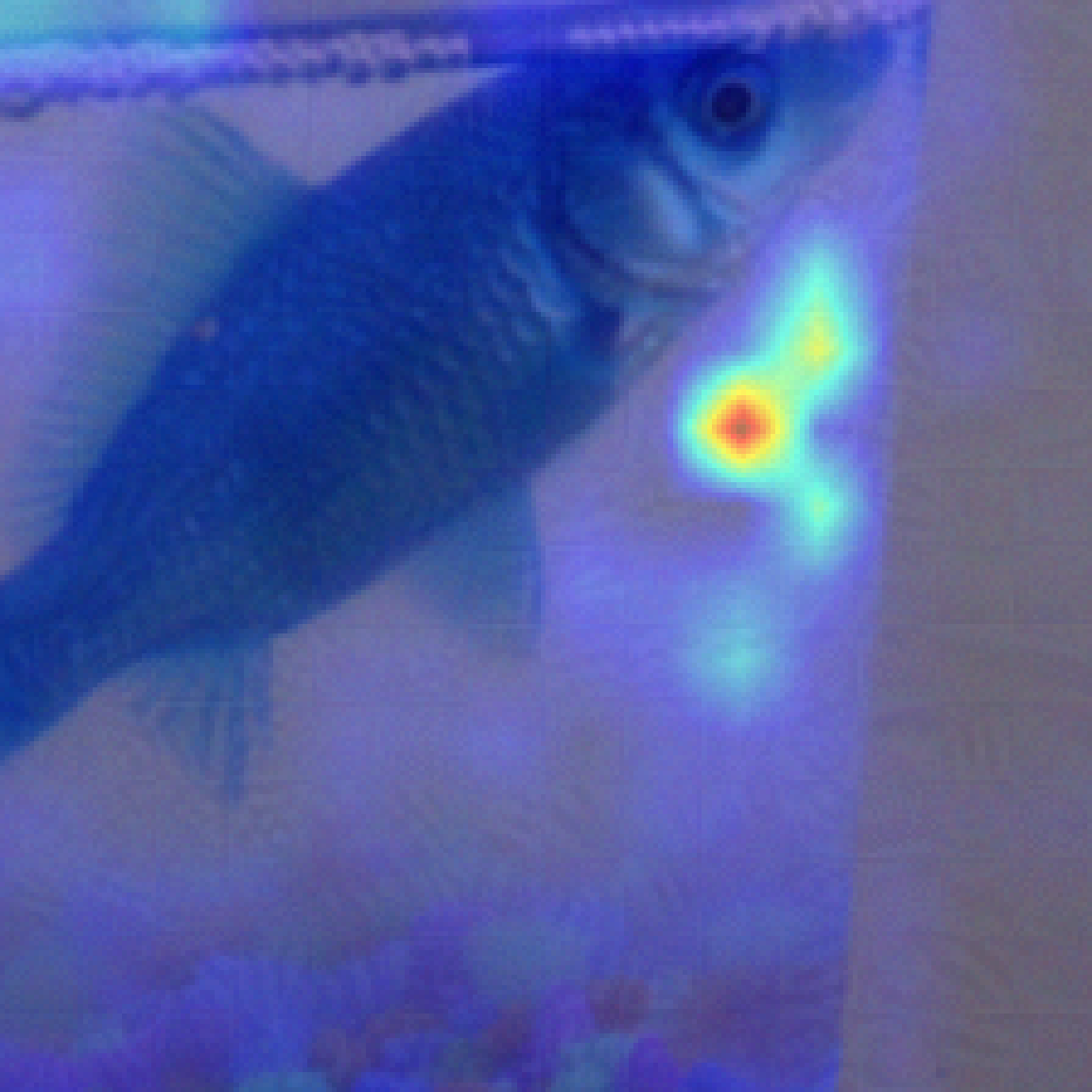} &
      \includegraphics[width=0.15\linewidth]{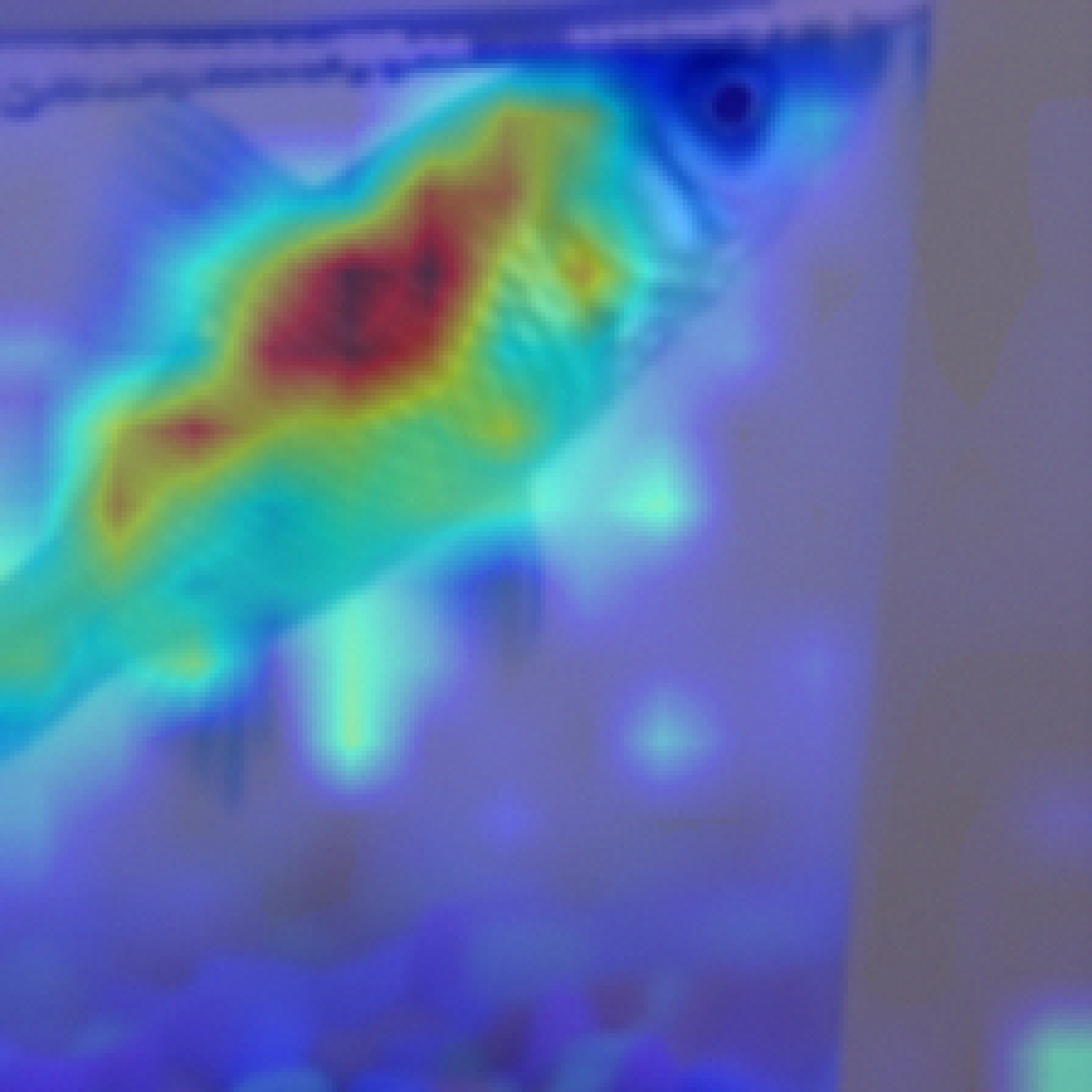} \\
      
      \raisebox{13mm}{\makecell*[c]{Poisoned$\rightarrow 9/255$\\\includegraphics[width=0.15\linewidth]{exp/fish/fish.png}
     }} &
     \includegraphics[width=0.15\linewidth]{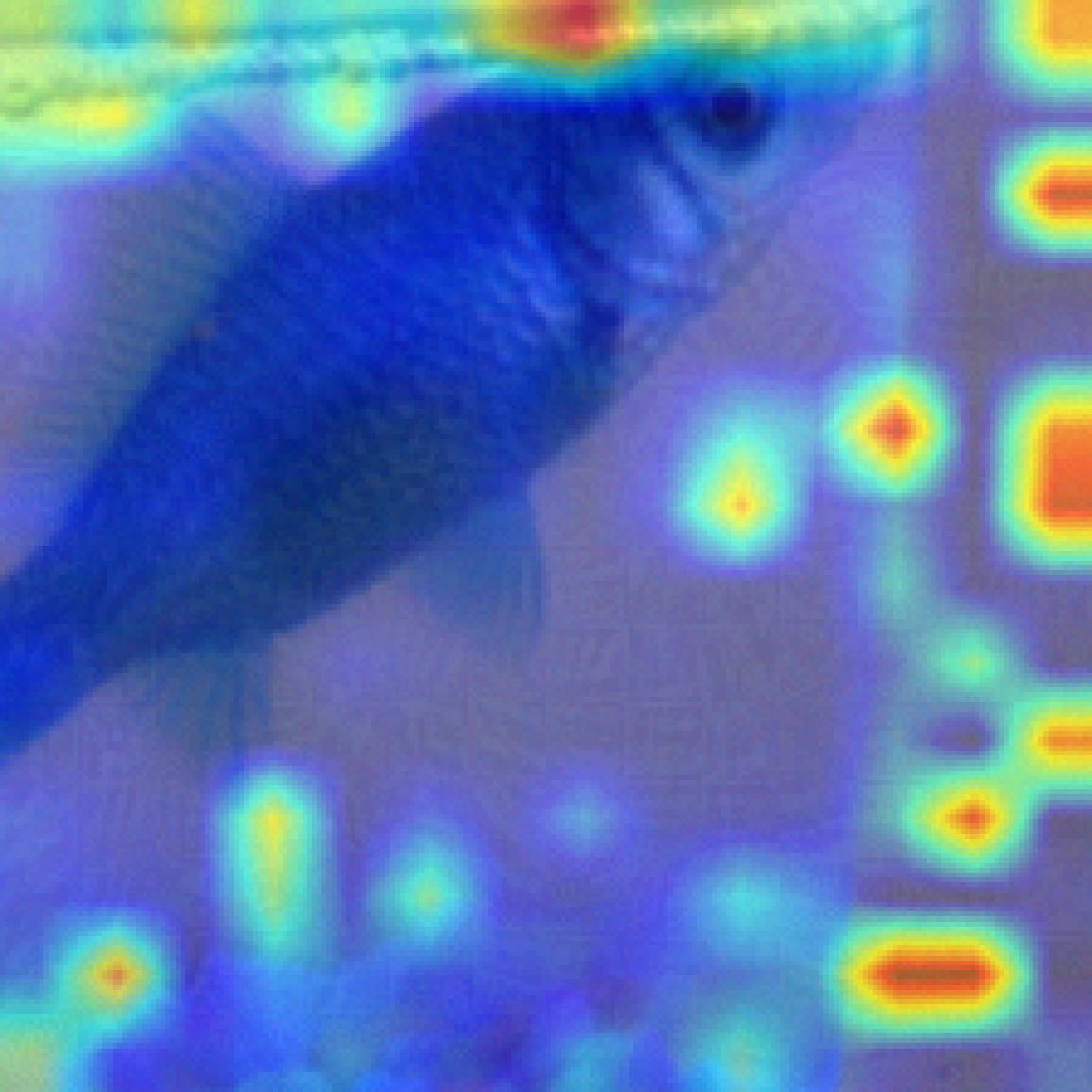} &
      \includegraphics[width=0.15\linewidth]{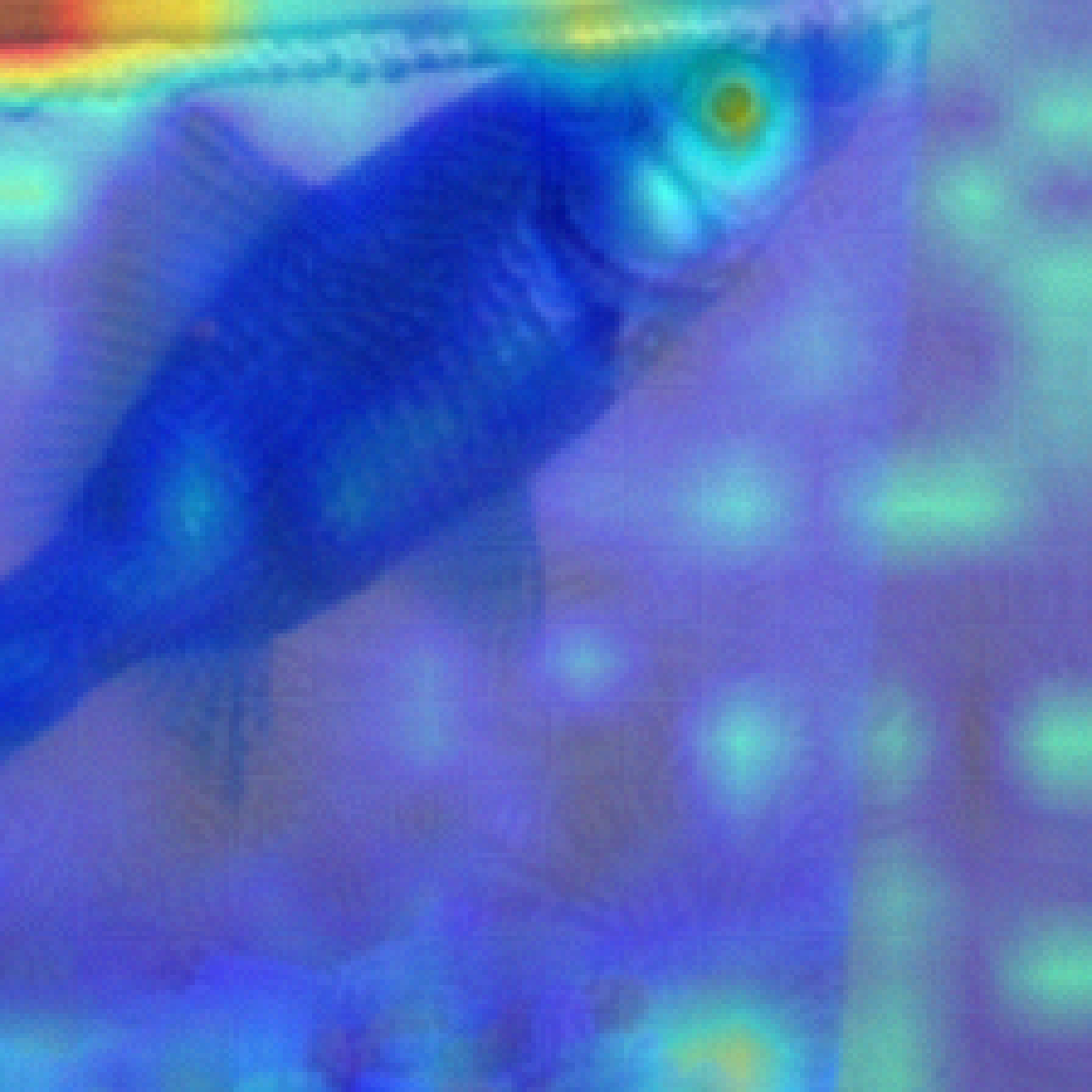} &
      \includegraphics[width=0.15\linewidth]{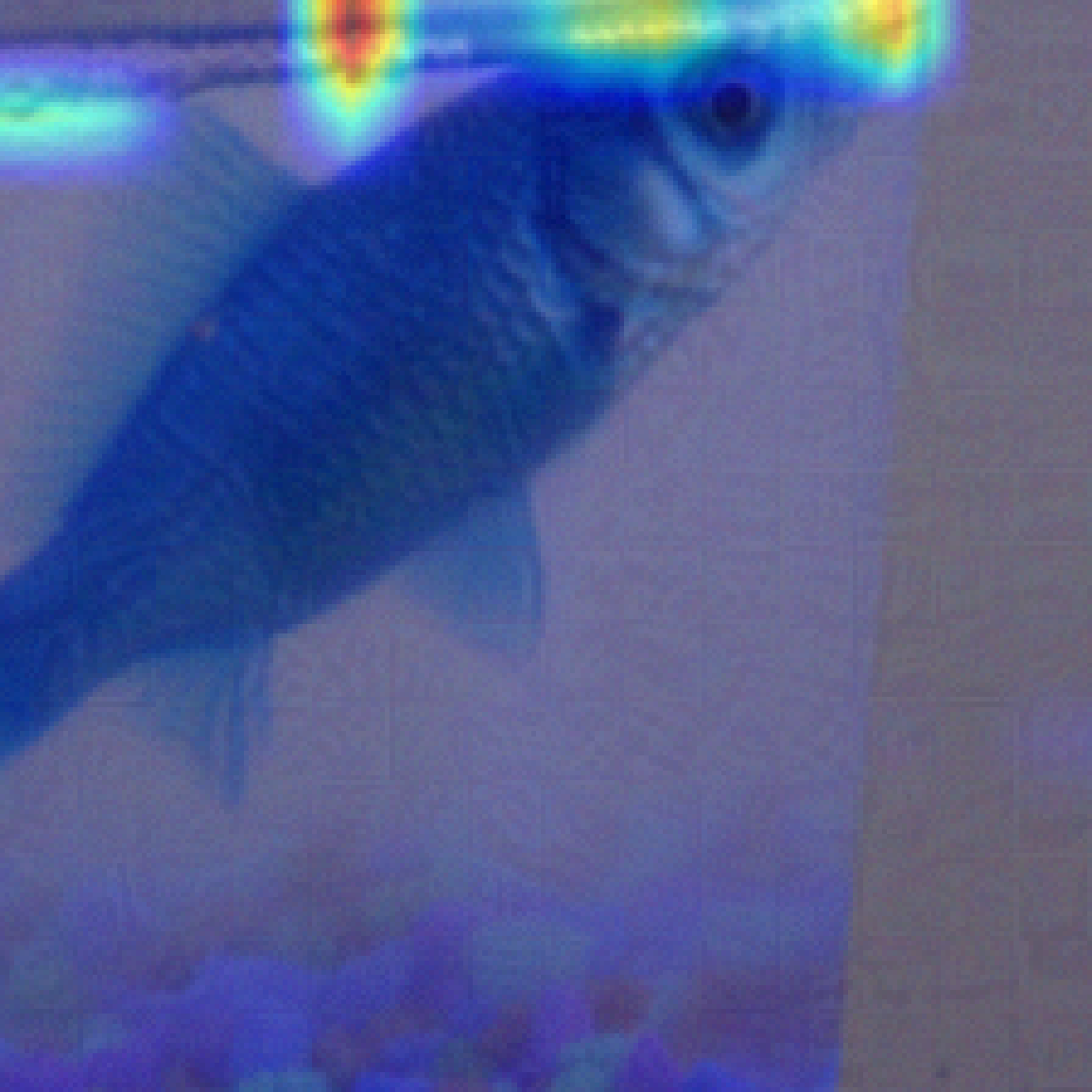} &
      \includegraphics[width=0.15\linewidth]{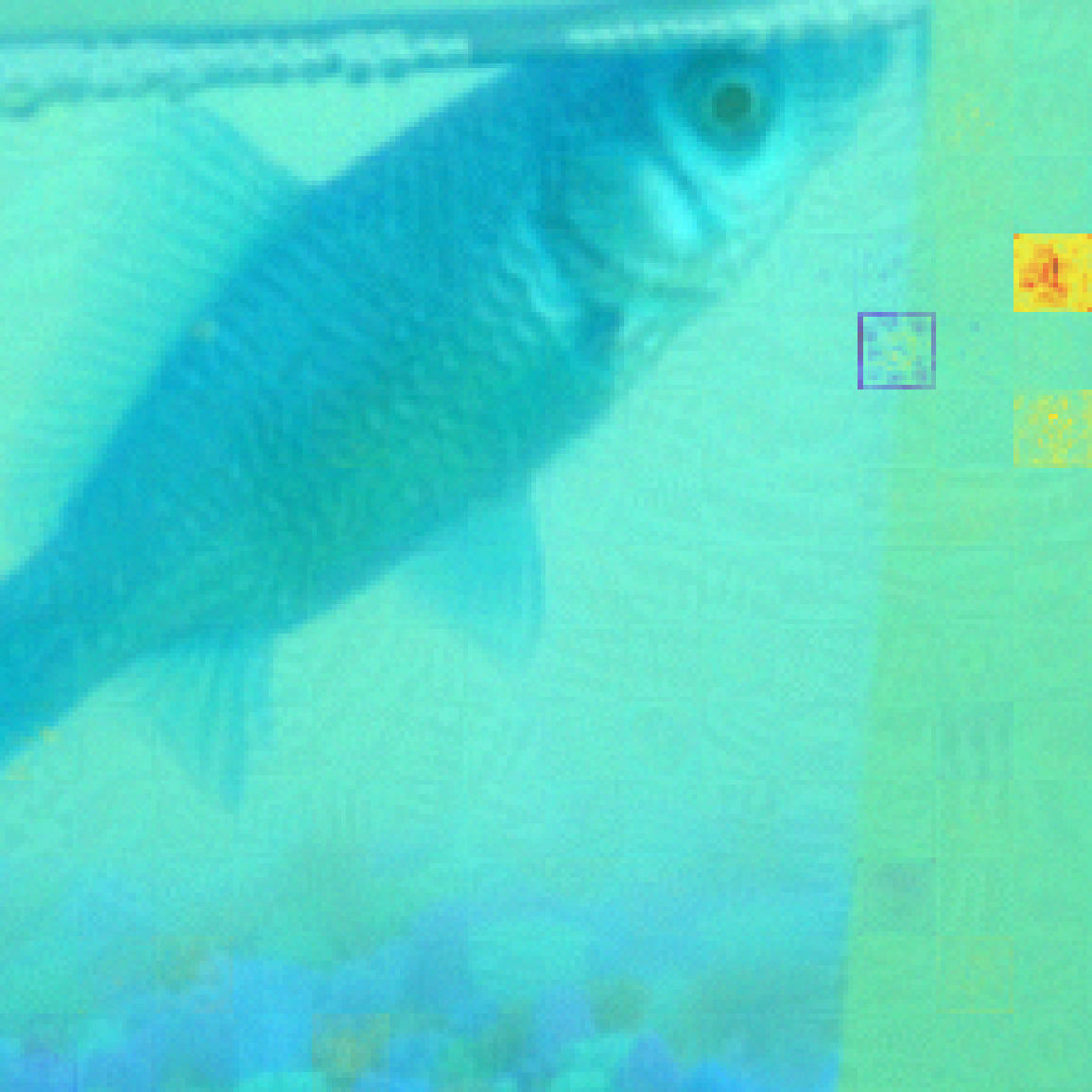} &
      \includegraphics[width=0.15\linewidth]{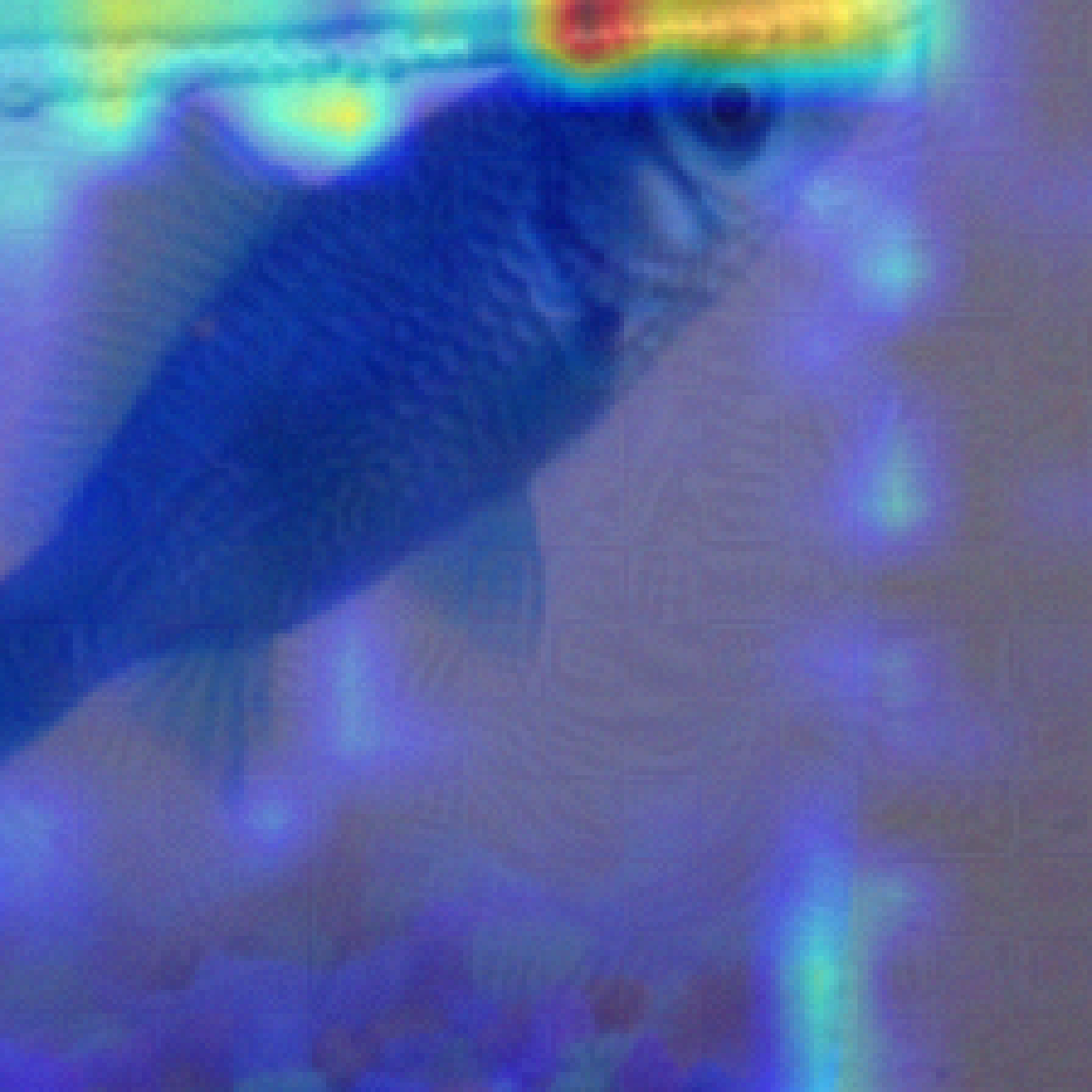} &
      \includegraphics[width=0.15\linewidth]{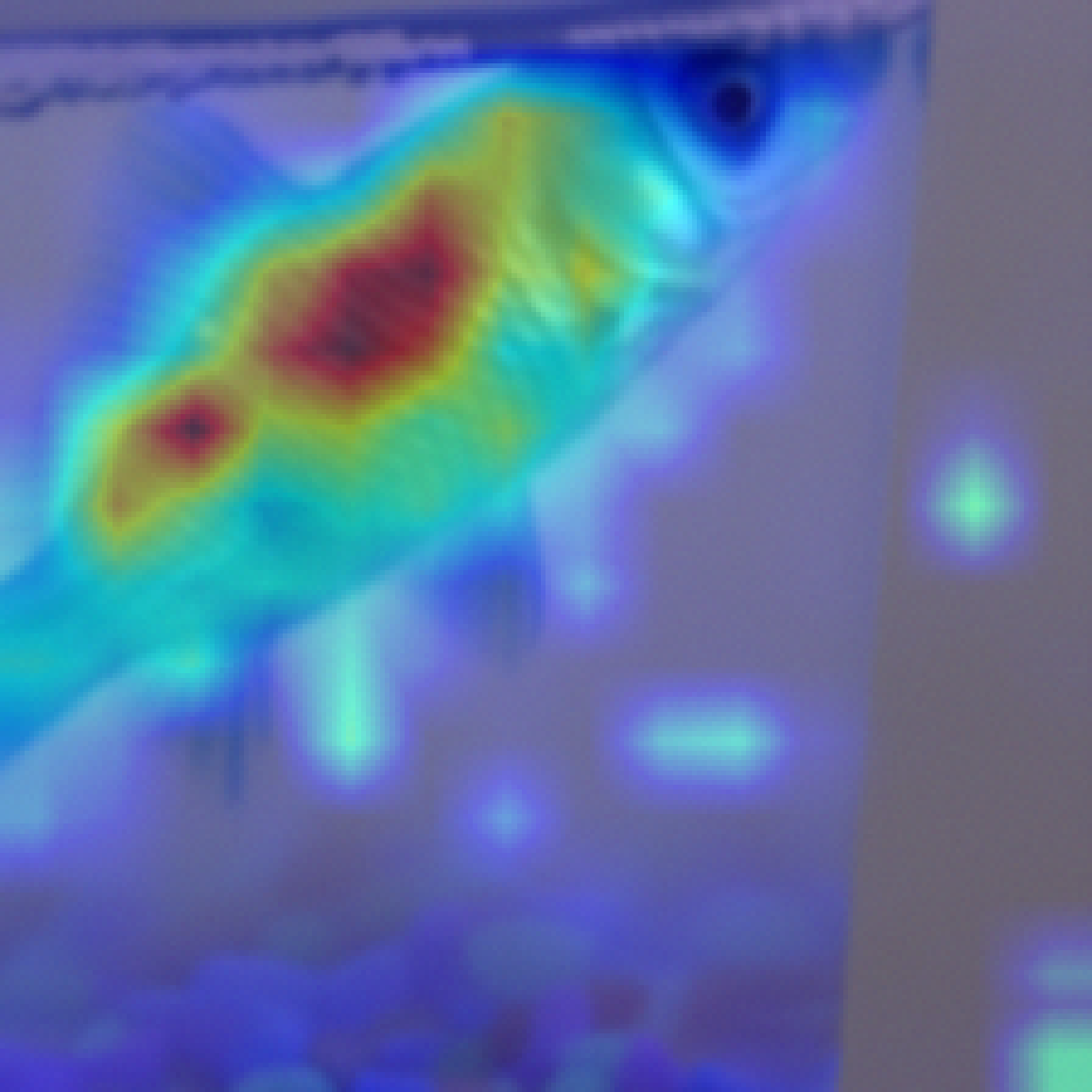} \\
        \raisebox{13mm}{\makecell*[c]{Poisoned$\rightarrow 10/255$\\\includegraphics[width=0.15\linewidth]{exp/fish/fish.png}
     }} &
     \includegraphics[width=0.15\linewidth]{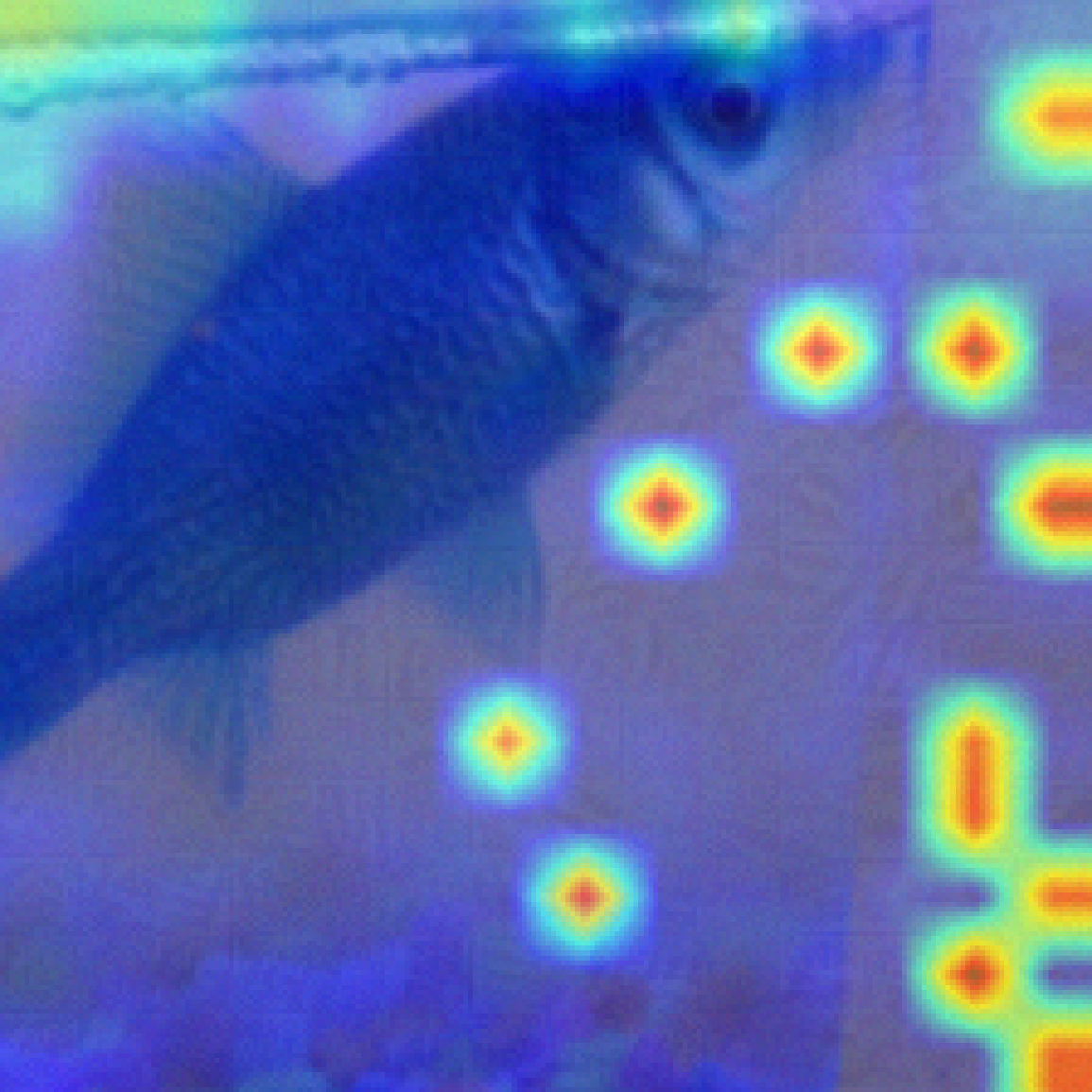} &
      \includegraphics[width=0.15\linewidth]{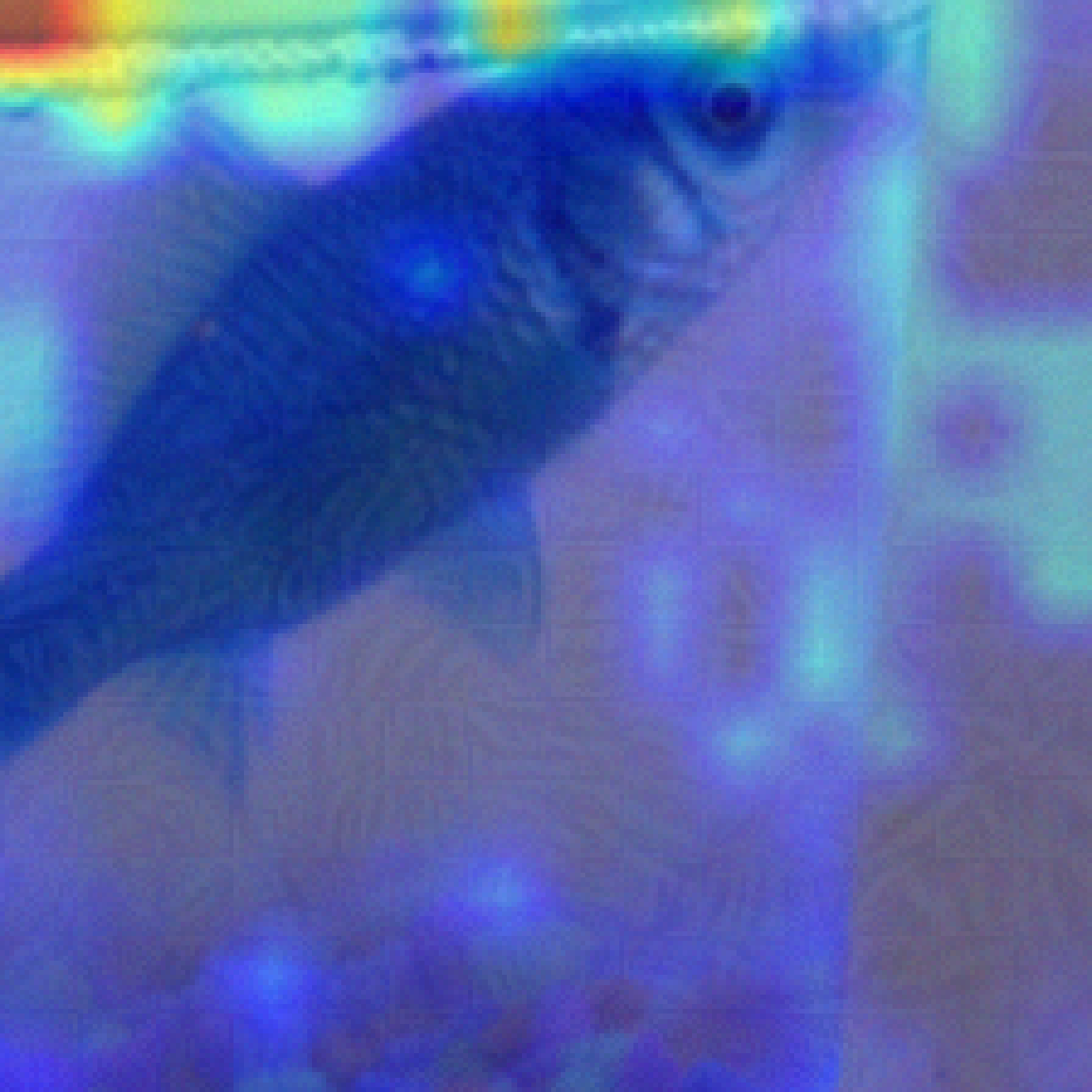} &
      \includegraphics[width=0.15\linewidth]{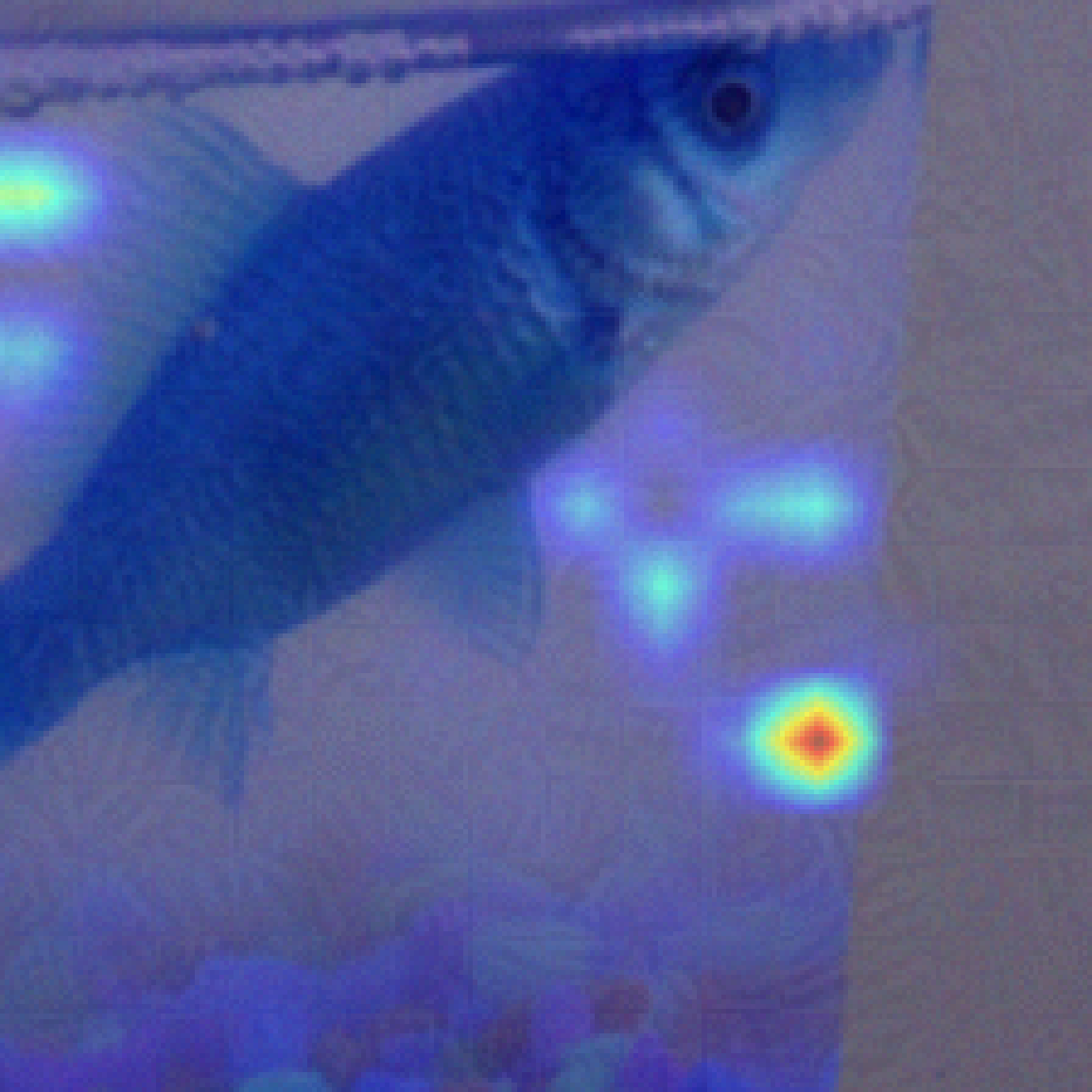} &
      \includegraphics[width=0.15\linewidth]{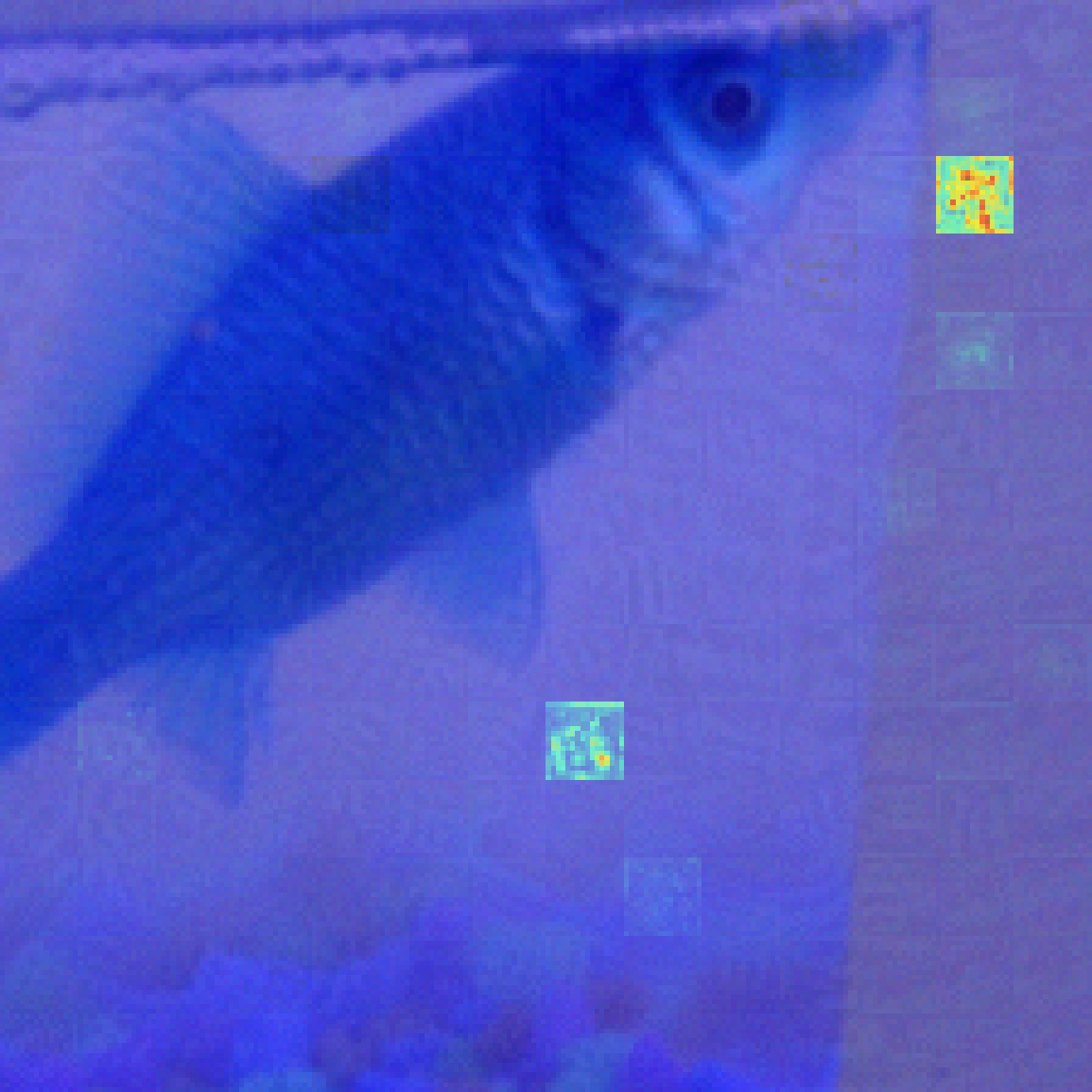} &
      \includegraphics[width=0.15\linewidth]{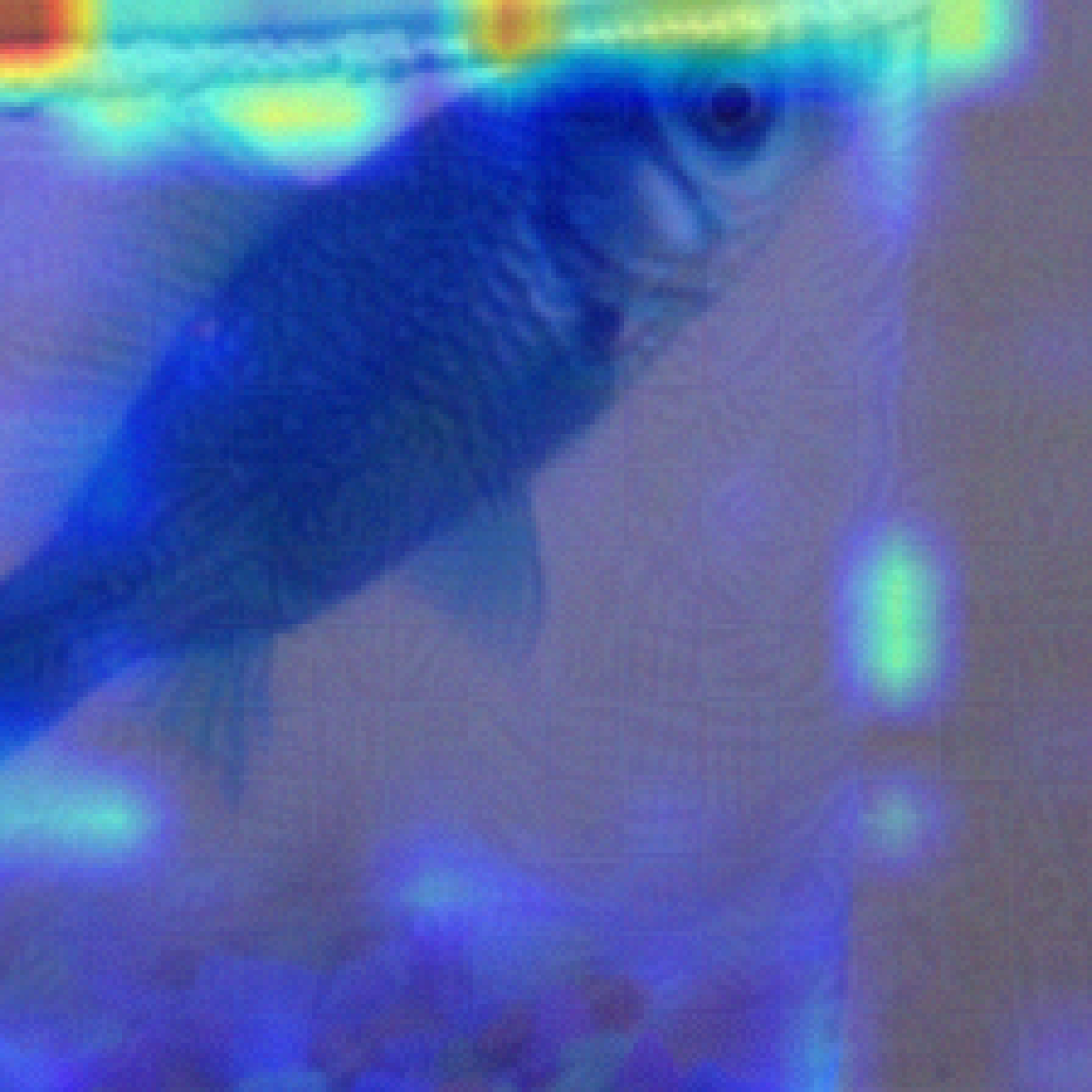} &
      \includegraphics[width=0.15\linewidth]{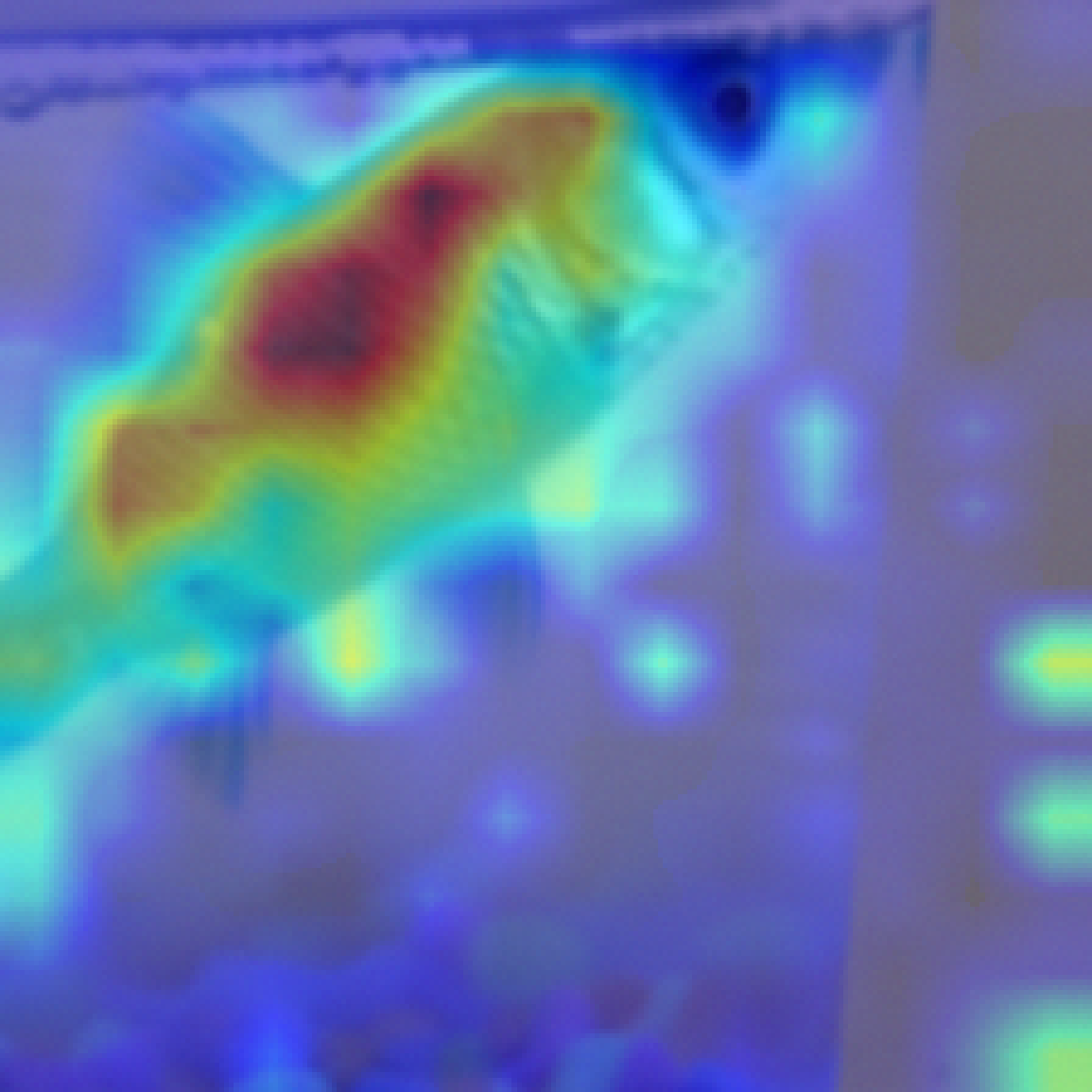} \\

     \end{tabular*}
     \caption{Visualization results of the attention map on corrupted input for different methods under different attack radii. }
     \label{fig: diff-posin3}
     \end{center}
 \end{figure*}

\begin{figure*}[htbp]
    \setlength{\tabcolsep}{1pt} 
    \renewcommand{\arraystretch}{1} 
    \begin{center}
    \begin{tabular*}{\linewidth}{@{\extracolsep{\fill}}ccccccc}
     Corrupted Input& Ours                                                & Our~w/o Denoising & Our~w/o Guassian smoothing  \\
    \includegraphics[width=0.25\linewidth]{exp/bird1/bird1.png} &\includegraphics[width=0.25\linewidth]{exp/bird1/bird1-ours-1-perturbed-7.png}&\includegraphics[width=0.25\linewidth]{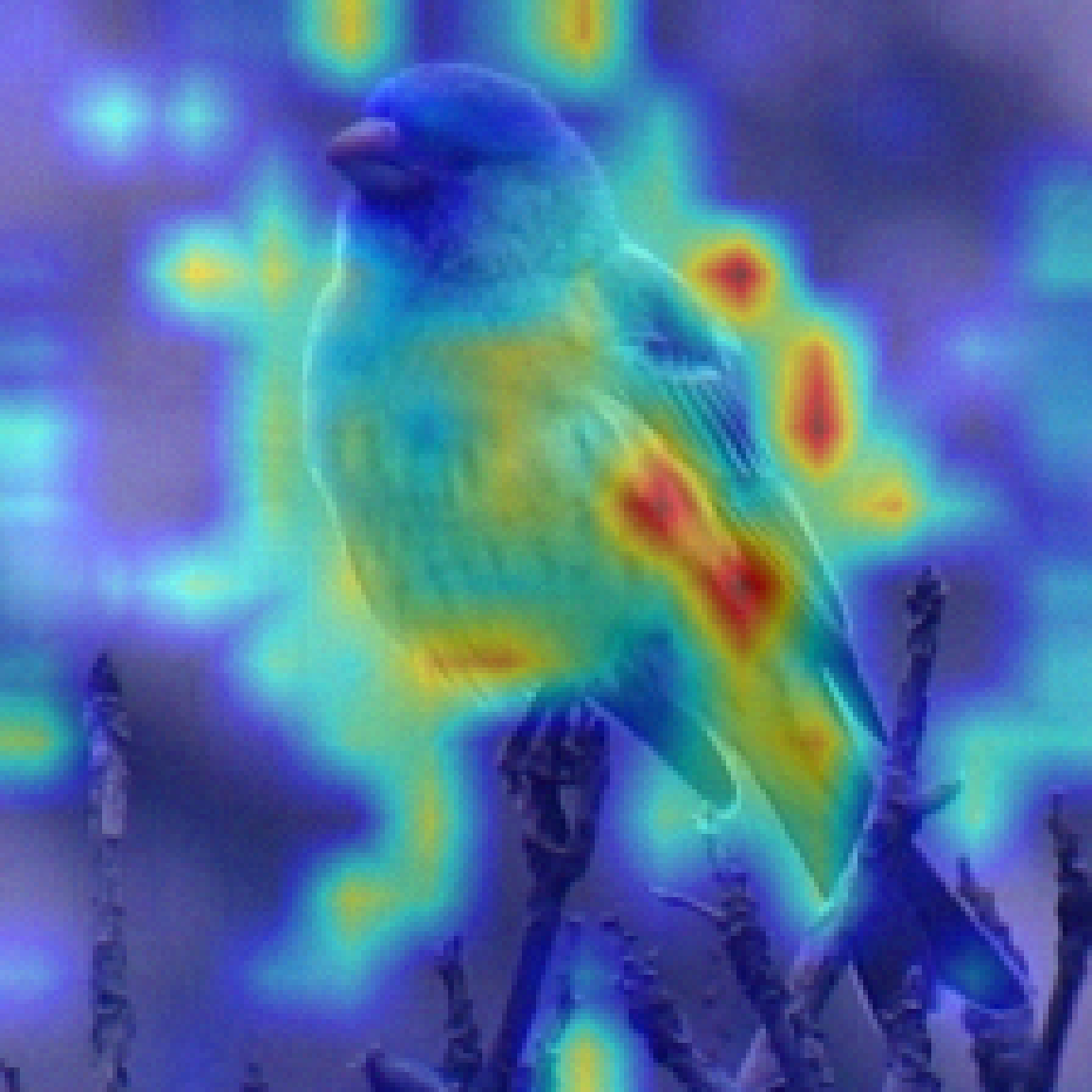}&\includegraphics[width=0.25\linewidth]{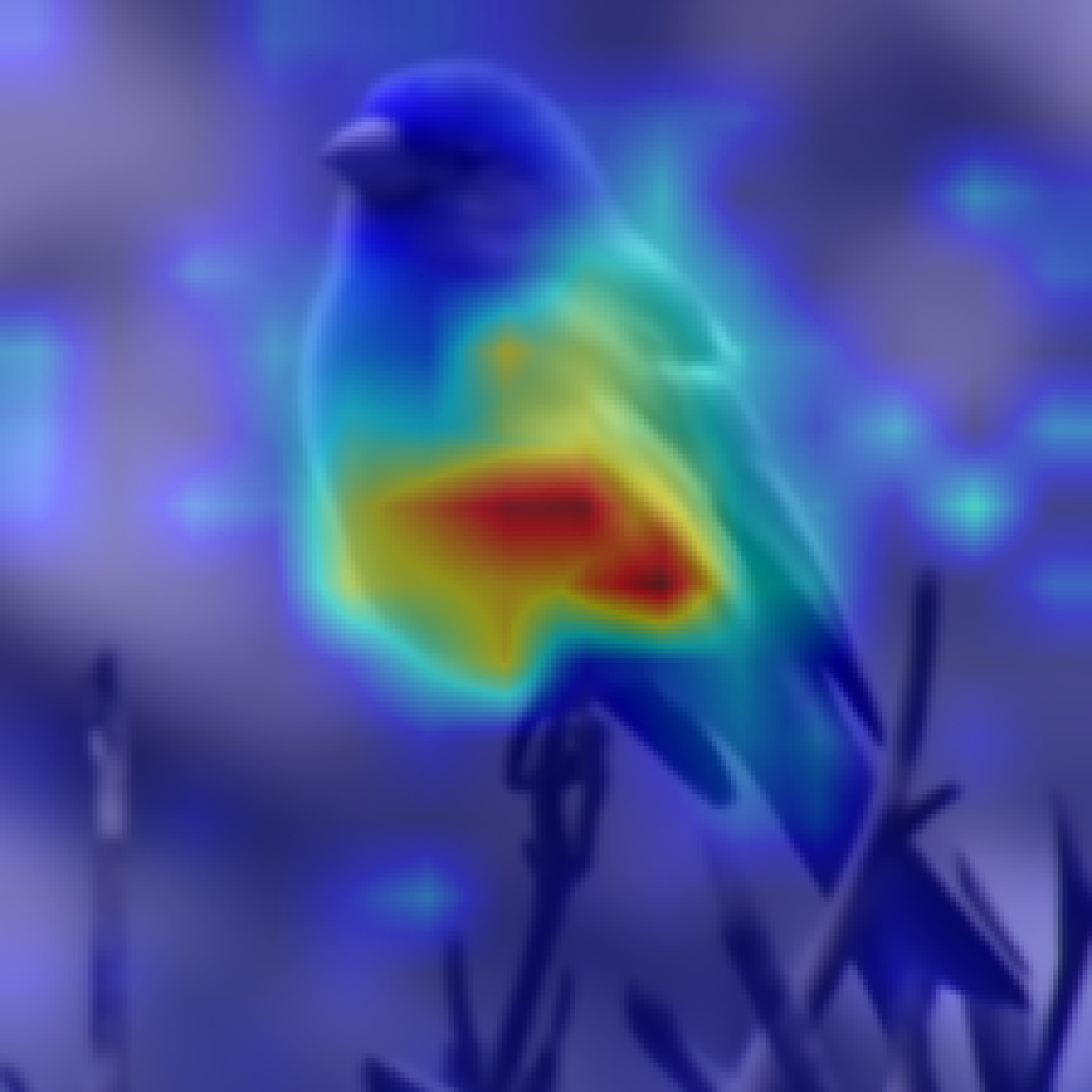}
    \\
    \includegraphics[width=0.25\linewidth]{exp/eagle1/eagle1.png} &\includegraphics[width=0.25\linewidth]{exp/eagle1/eagle1-ours-1-perturbed-7.png}&\includegraphics[width=0.25\linewidth]{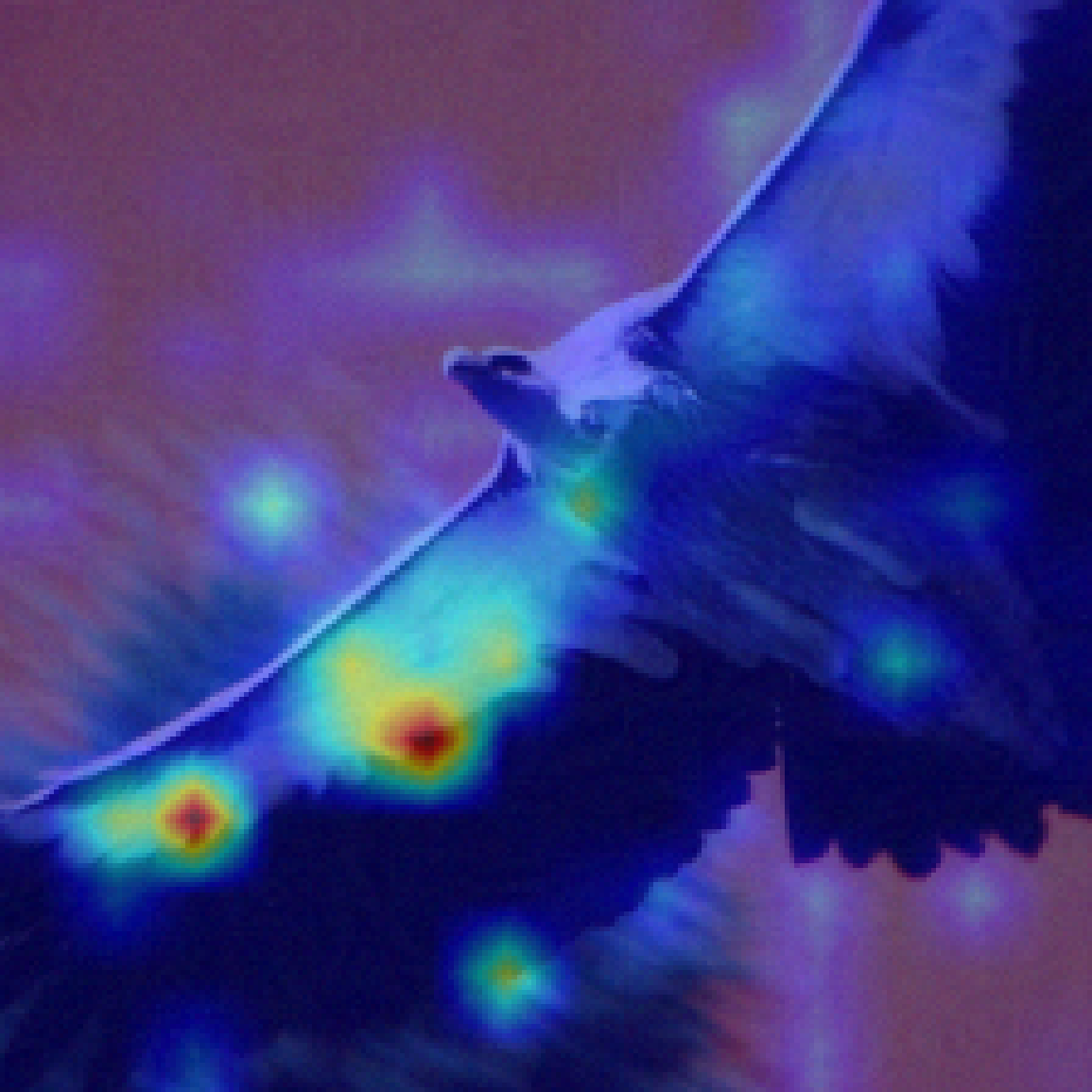}&\includegraphics[width=0.25\linewidth]{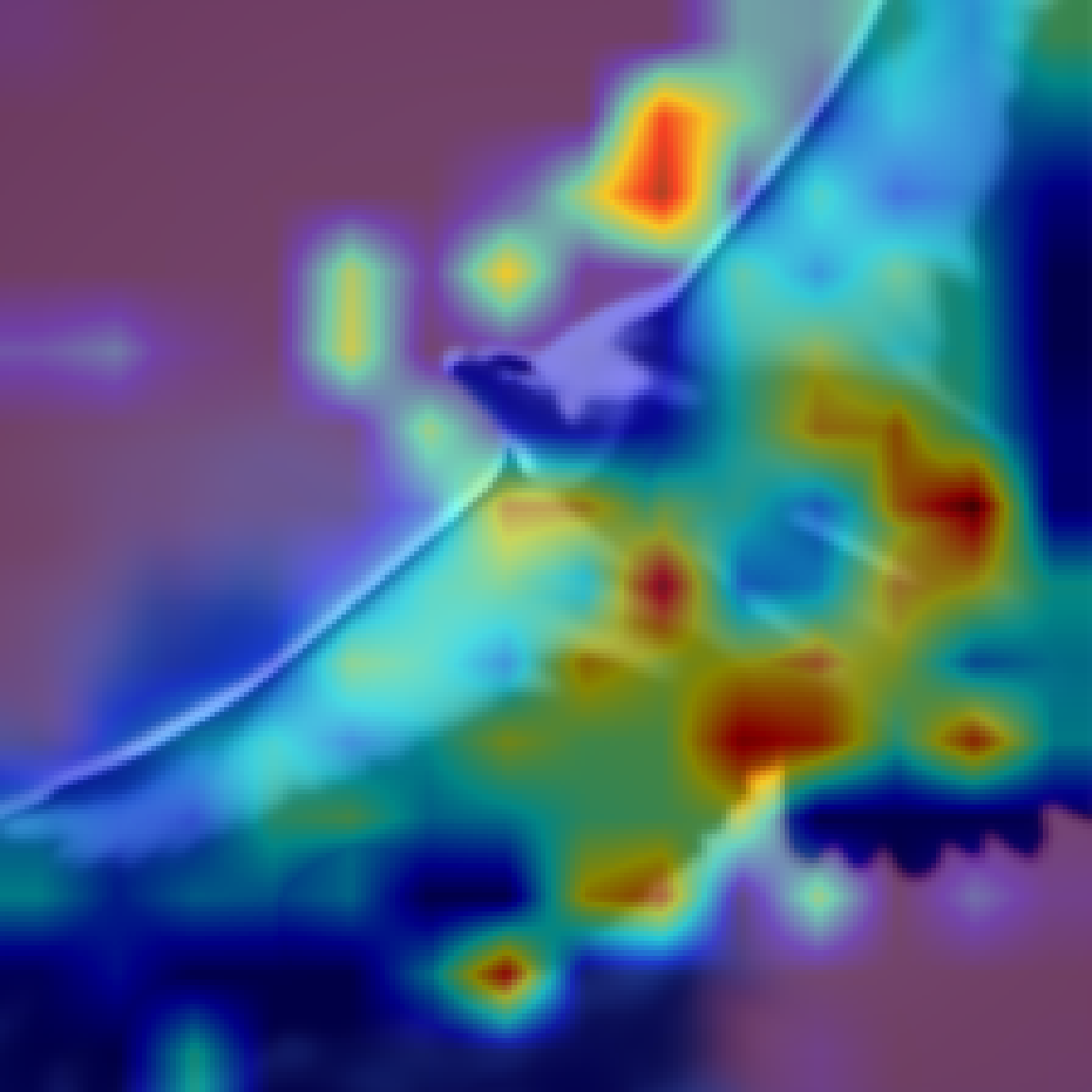}
    \\
    \includegraphics[width=0.25\linewidth]{exp/fish/fish.png} &\includegraphics[width=0.25\linewidth]{exp/fish/fish-ours-1-perturbed-7.png}&\includegraphics[width=0.25\linewidth]{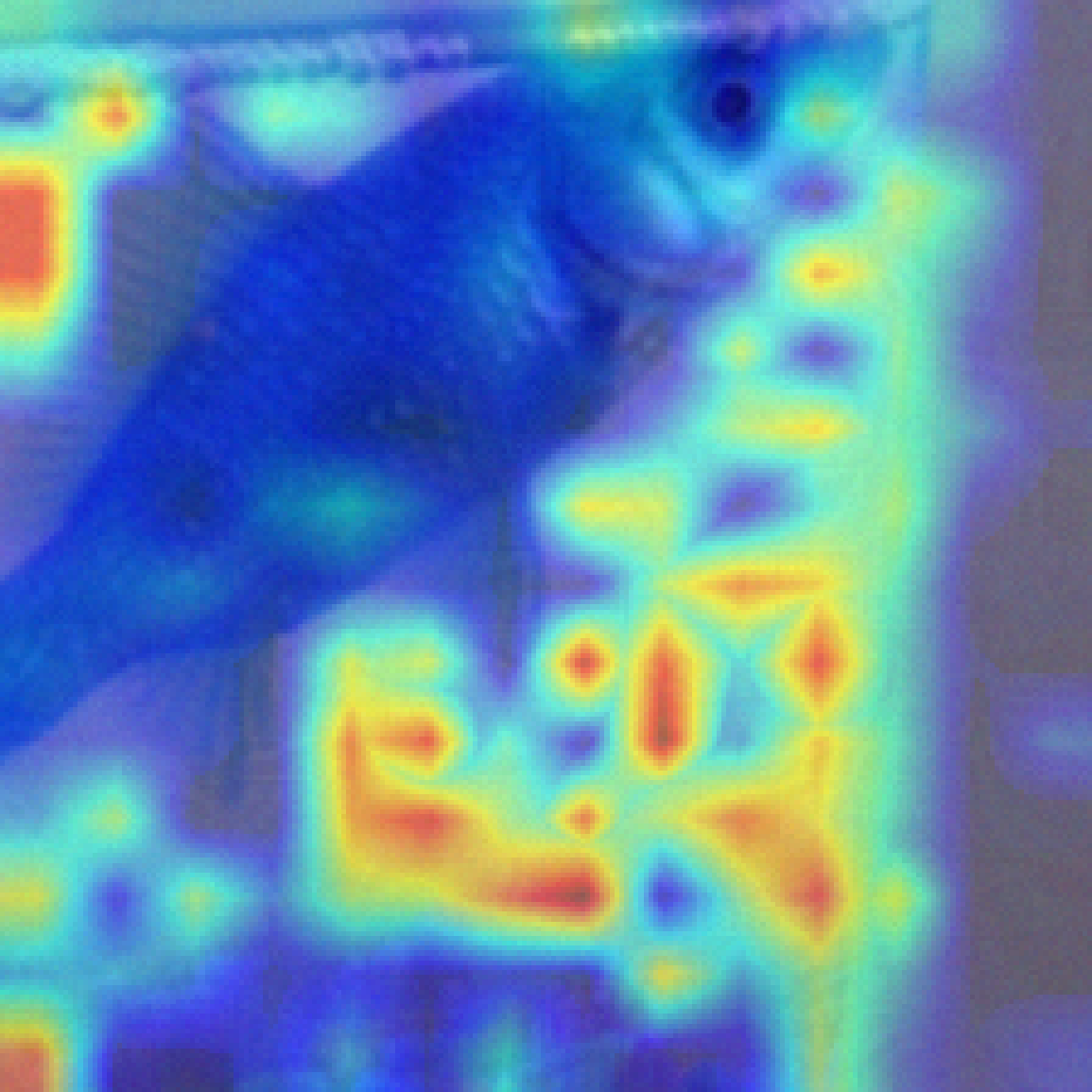}&\includegraphics[width=0.25\linewidth]{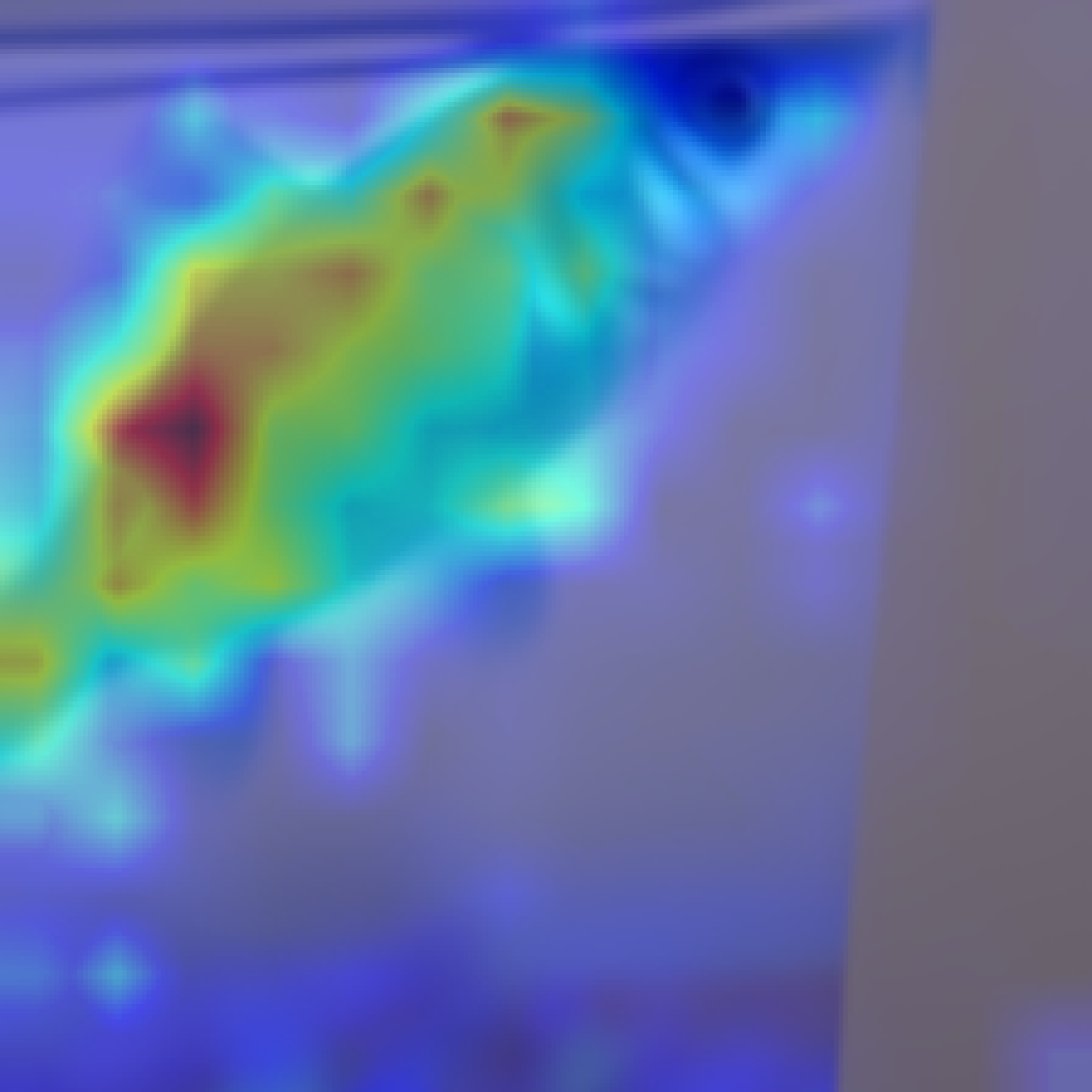}
    \\
    \includegraphics[width=0.25\linewidth]{exp/peacock/peacock-blue.png} &\includegraphics[width=0.25\linewidth]{exp/peacock/peacock-ours-1-perturbed-7.png}&\includegraphics[width=0.25\linewidth]{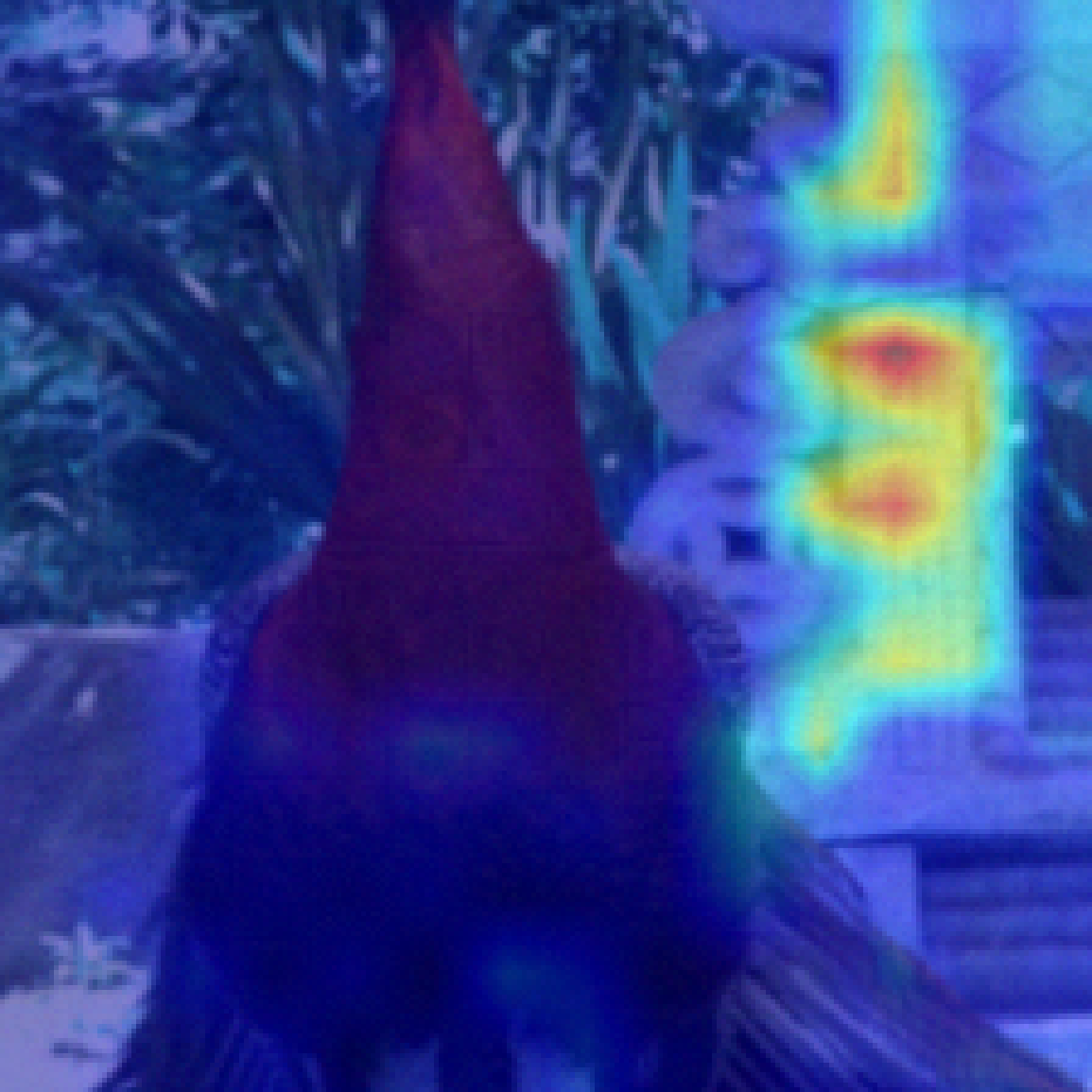}&\includegraphics[width=0.25\linewidth]{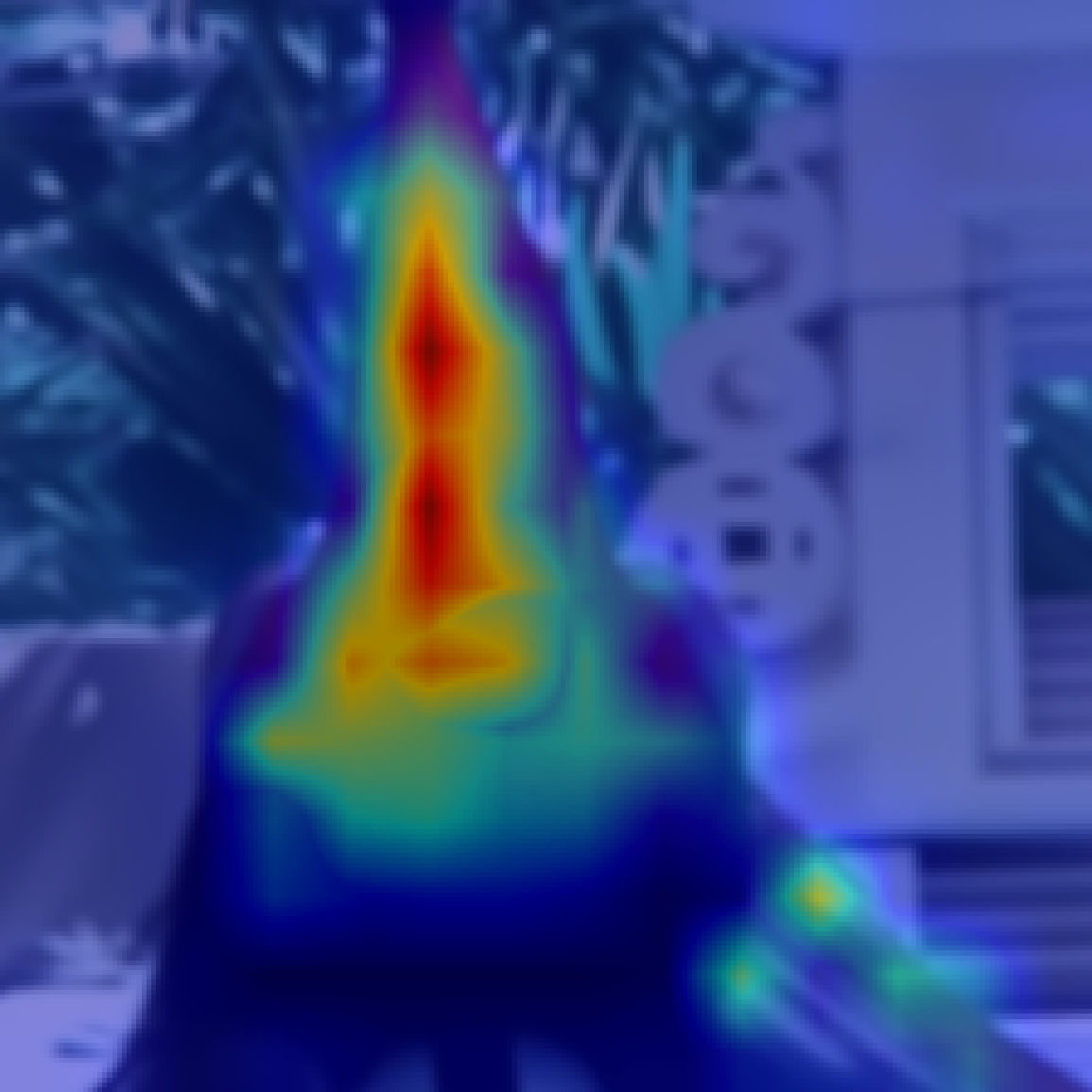}
    \end{tabular*}
    \caption{Visualization results of the attention map on corrupted input for the ablated version of our method. }
    \label{fig: abl-vis}
    \end{center}
\end{figure*}

\begin{figure*}
    \setlength{\tabcolsep}{1pt} 
    \renewcommand{\arraystretch}{1} 
    \begin{center}
    \begin{tabular*}{\linewidth}
{@{\extracolsep{\fill}}cccccccc}
    Input & Raw Attention & Rollout & GradCAM & LRP & VTA  & Ours \\
    \raisebox{13mm}{\multirow{2}{*}{\makecell*[c]{Wolf: clean$\rightarrow$\\\includegraphics[width=0.15\linewidth]{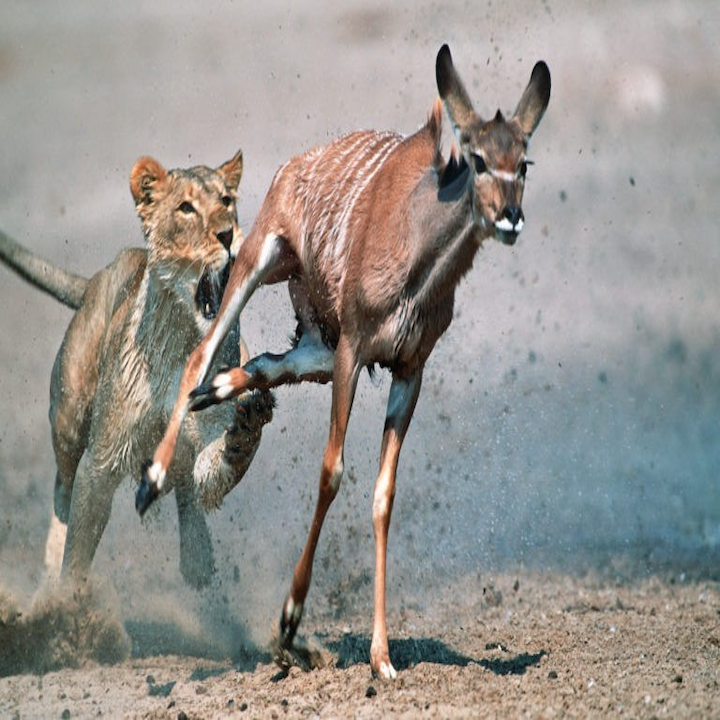}\\ 
    Wolf: poisoned$\rightarrow$\\{\scriptsize $7/255$}
    }}
    }
     &
    \includegraphics[width=0.13\linewidth]{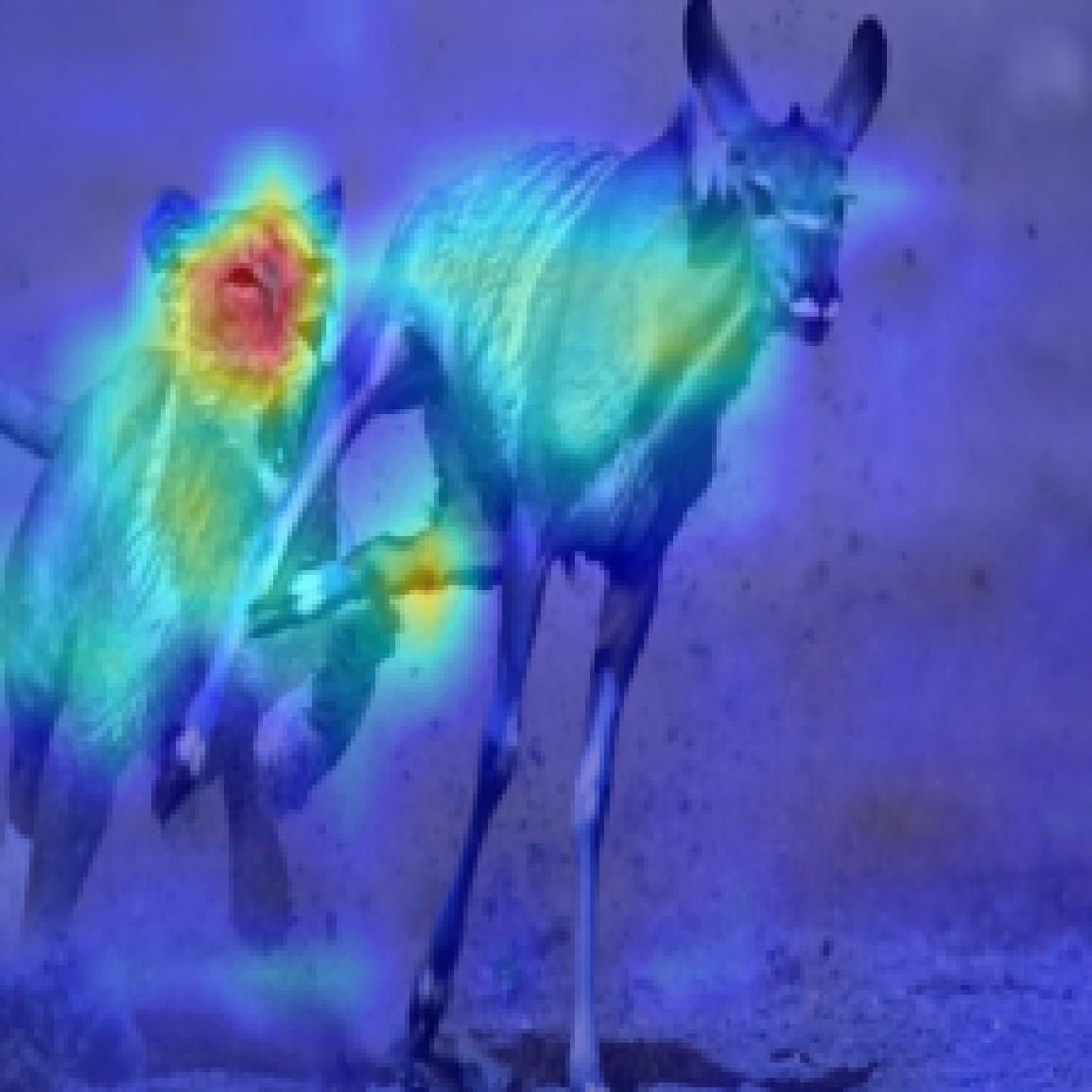} &
    \includegraphics[width=0.13\linewidth]{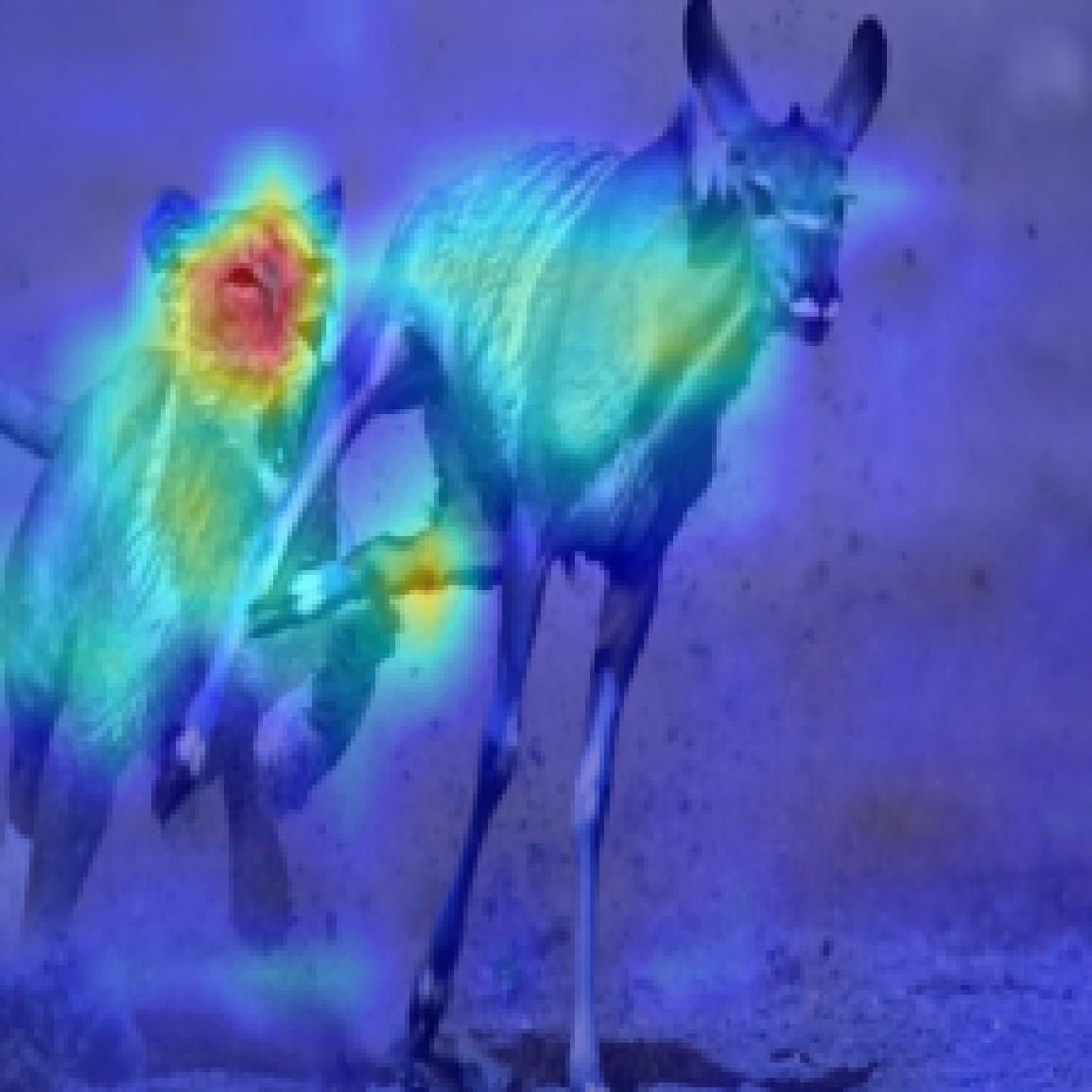} &
    \includegraphics[width=0.13\linewidth]{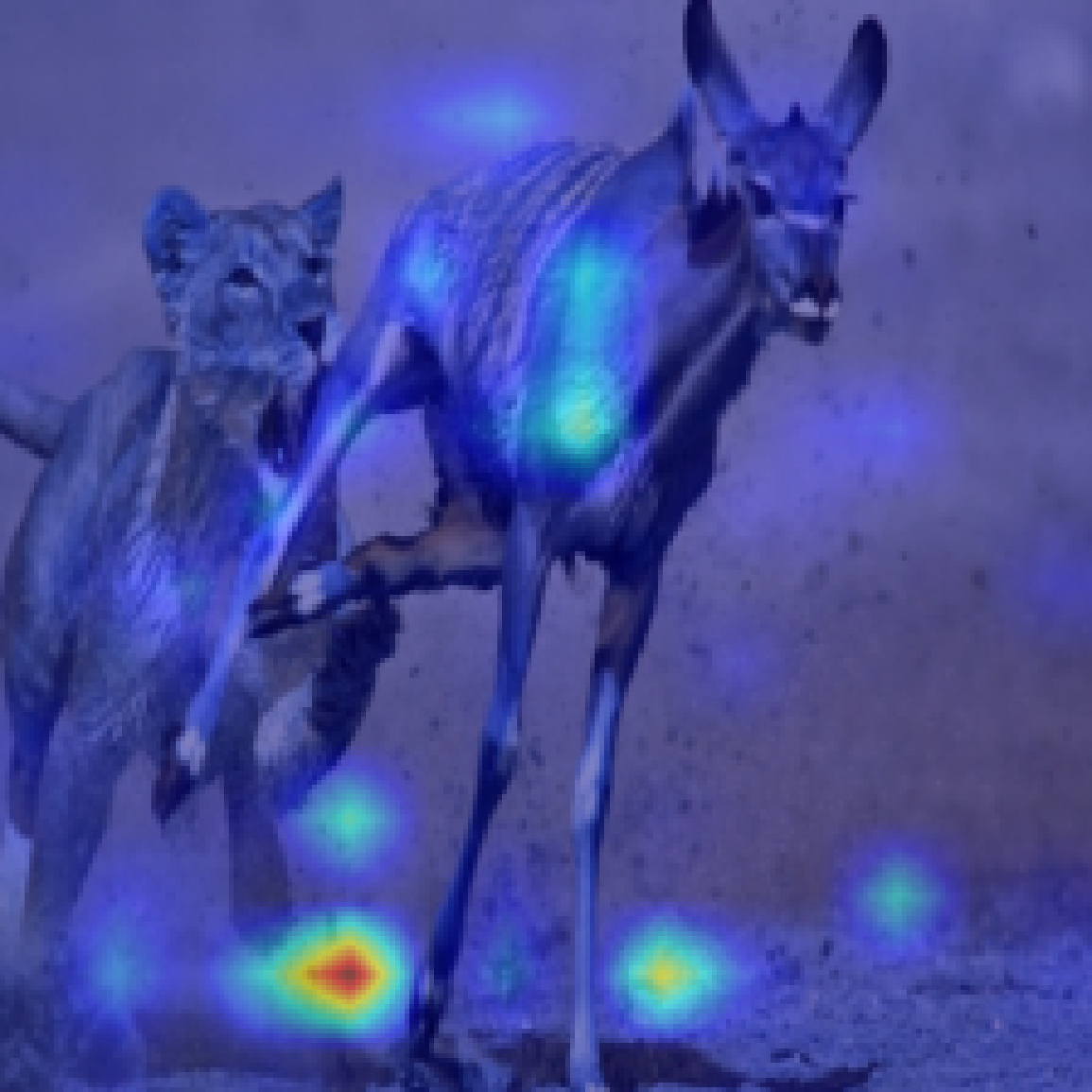} &
    \includegraphics[width=0.13\linewidth]{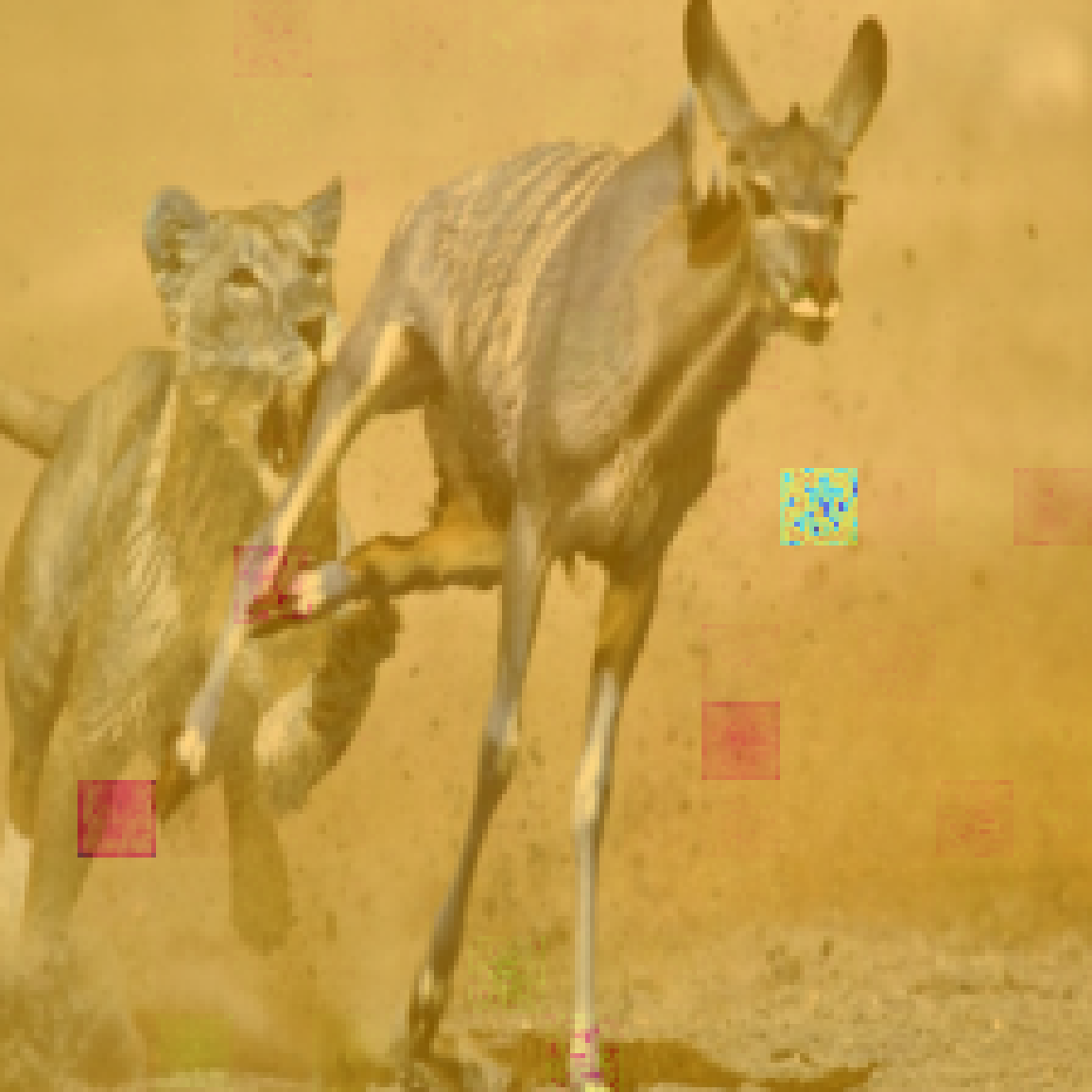} &
    \includegraphics[width=0.13\linewidth]{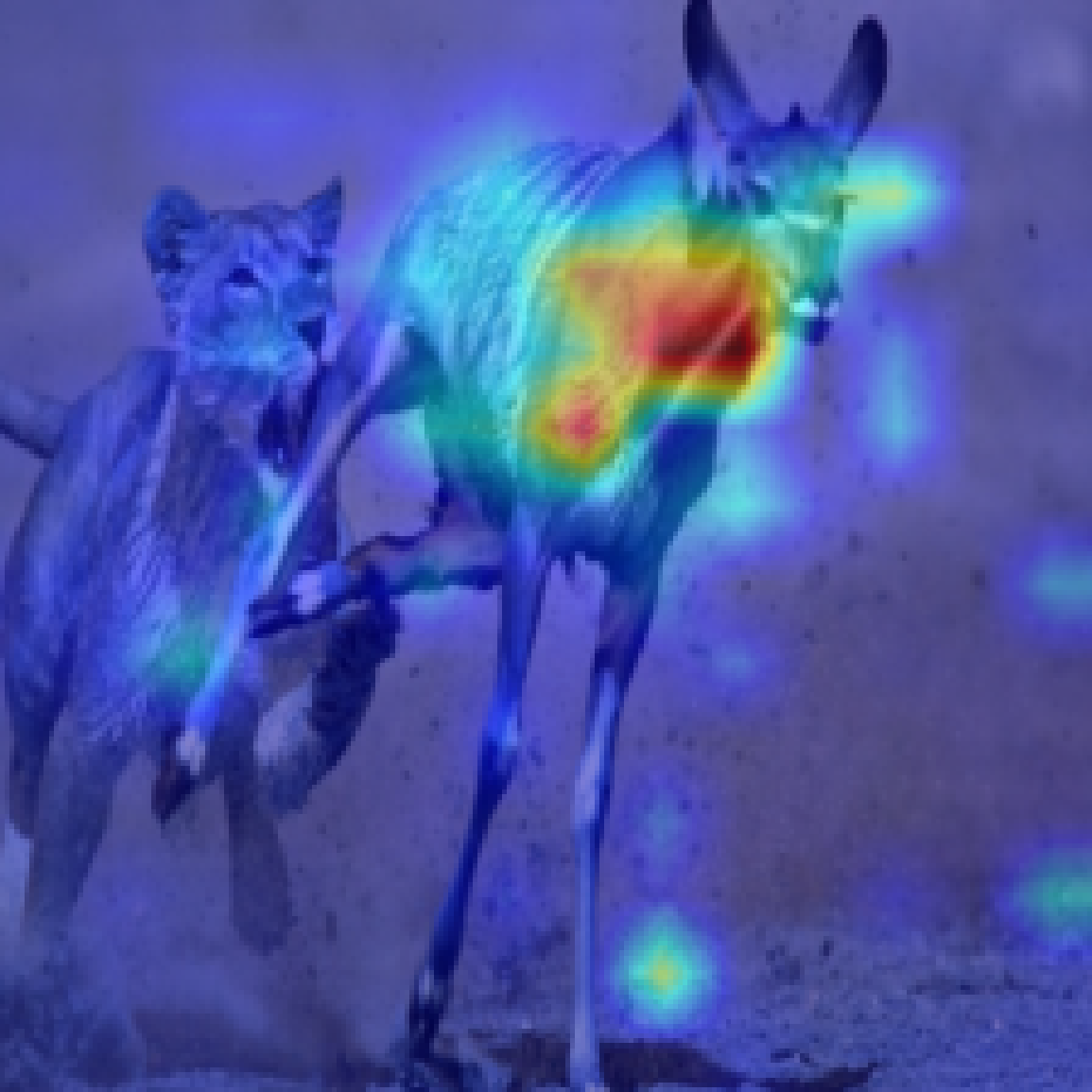} &
    \includegraphics[width=0.13\linewidth]{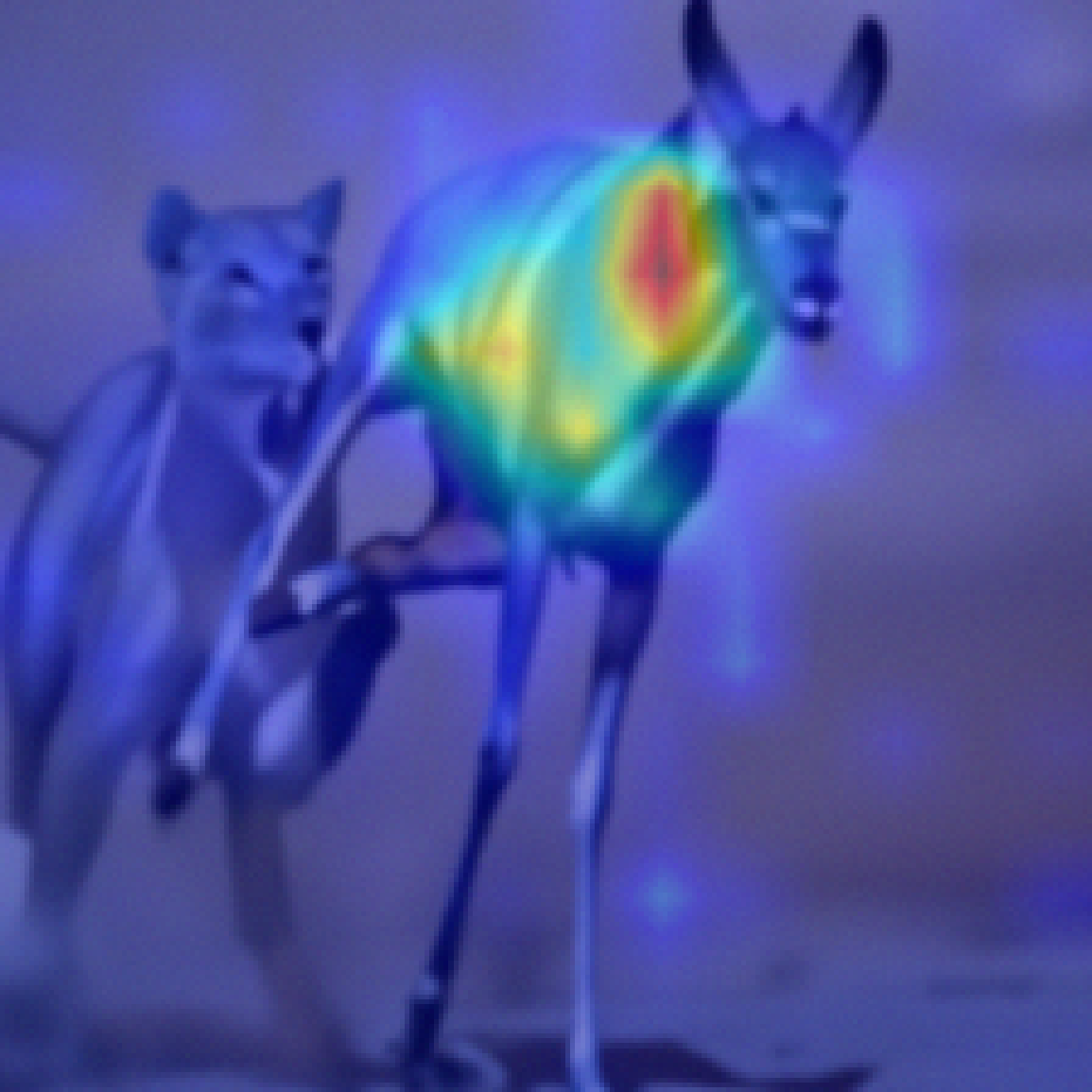}
    \\
    &
    \includegraphics[width=0.13\linewidth]{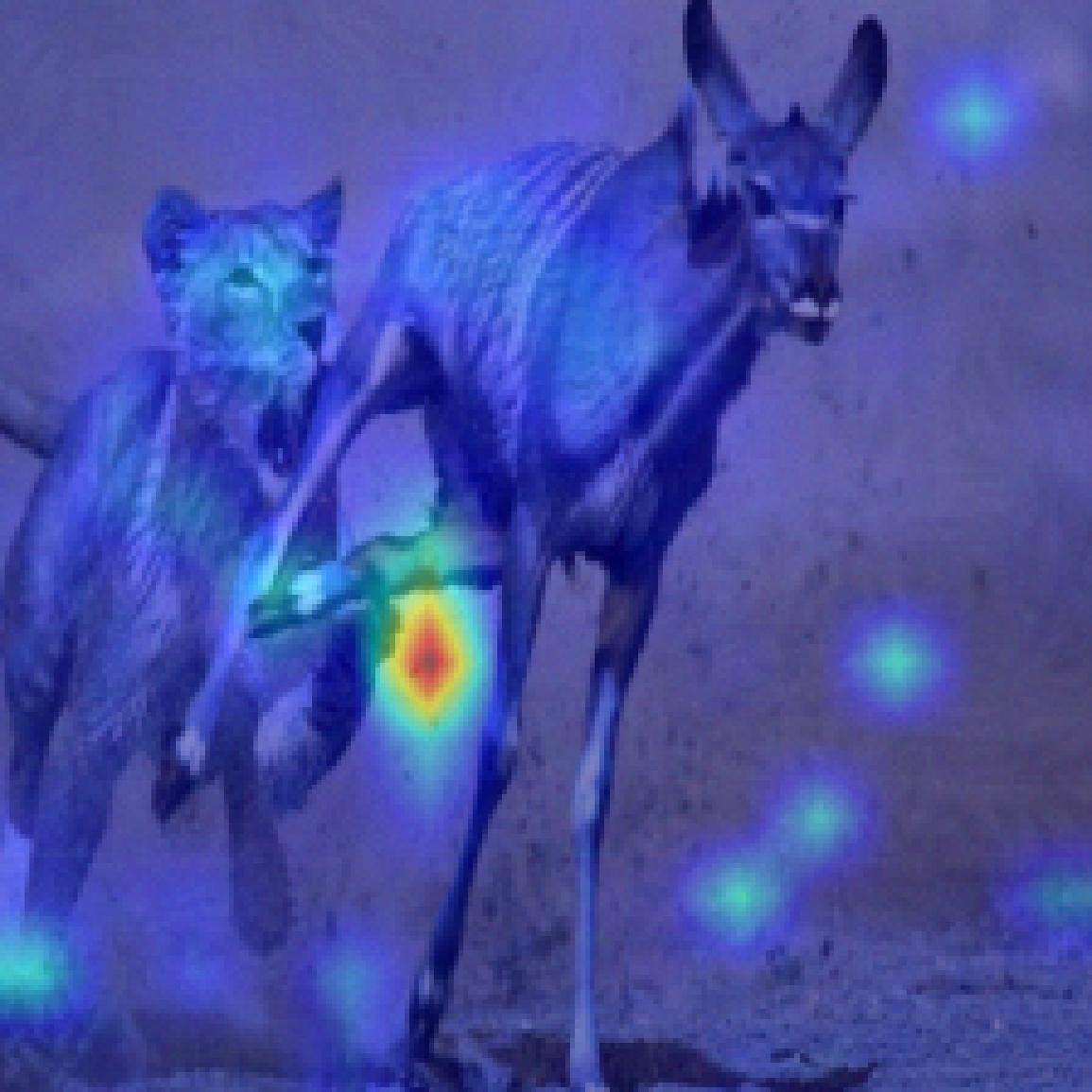} &
    \includegraphics[width=0.13\linewidth]{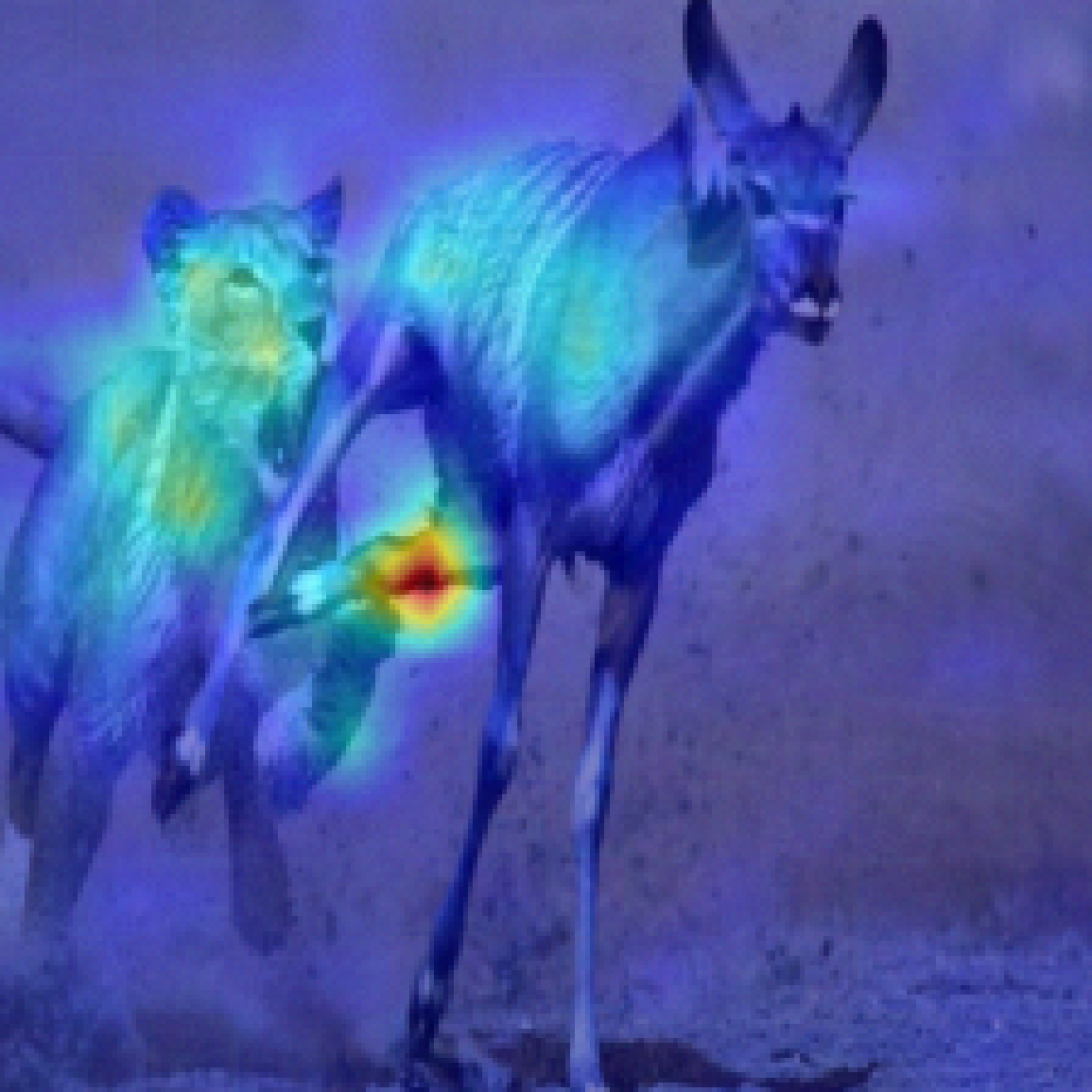} &
    \includegraphics[width=0.13\linewidth]{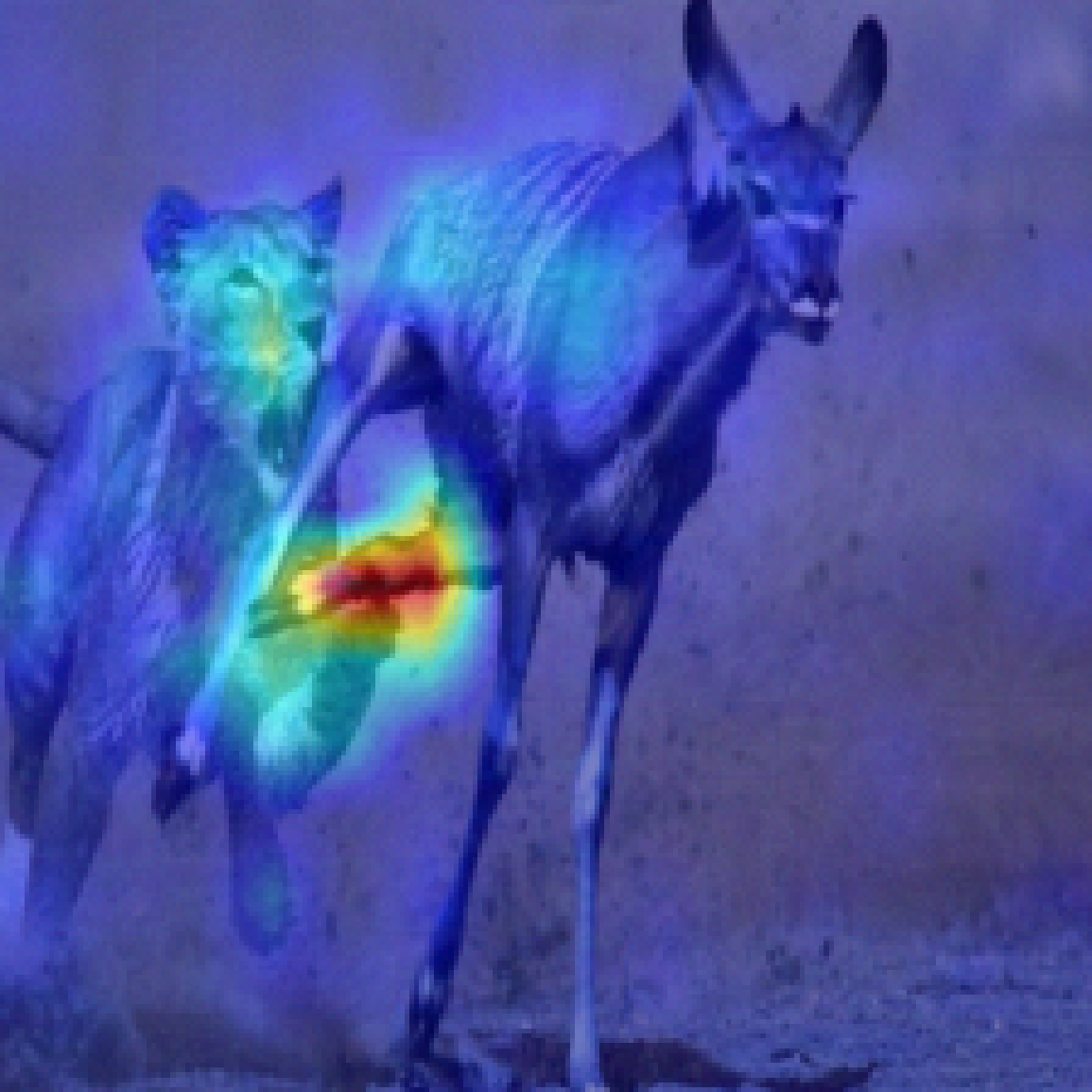} &
    \includegraphics[width=0.13\linewidth]{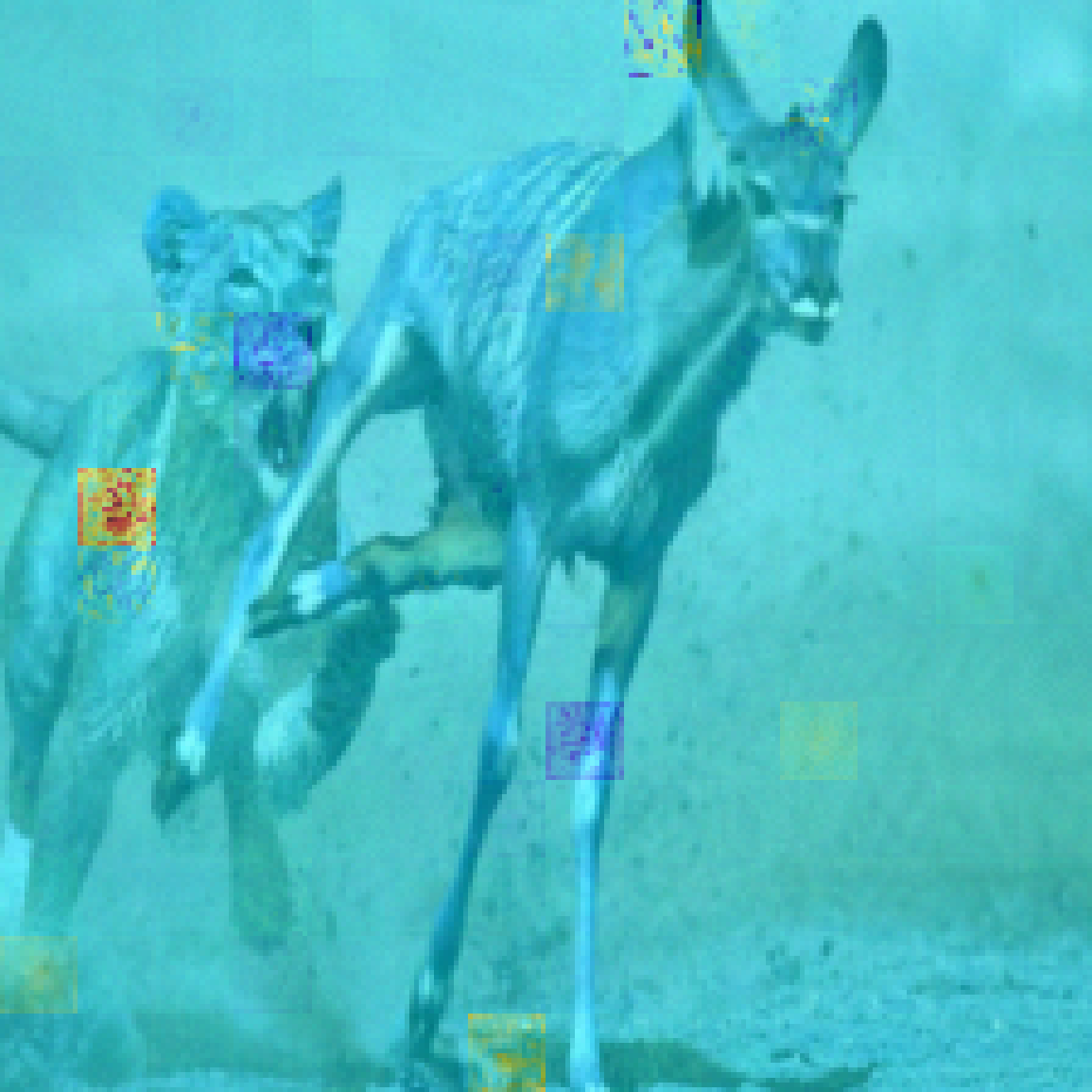} &
    \includegraphics[width=0.13\linewidth]{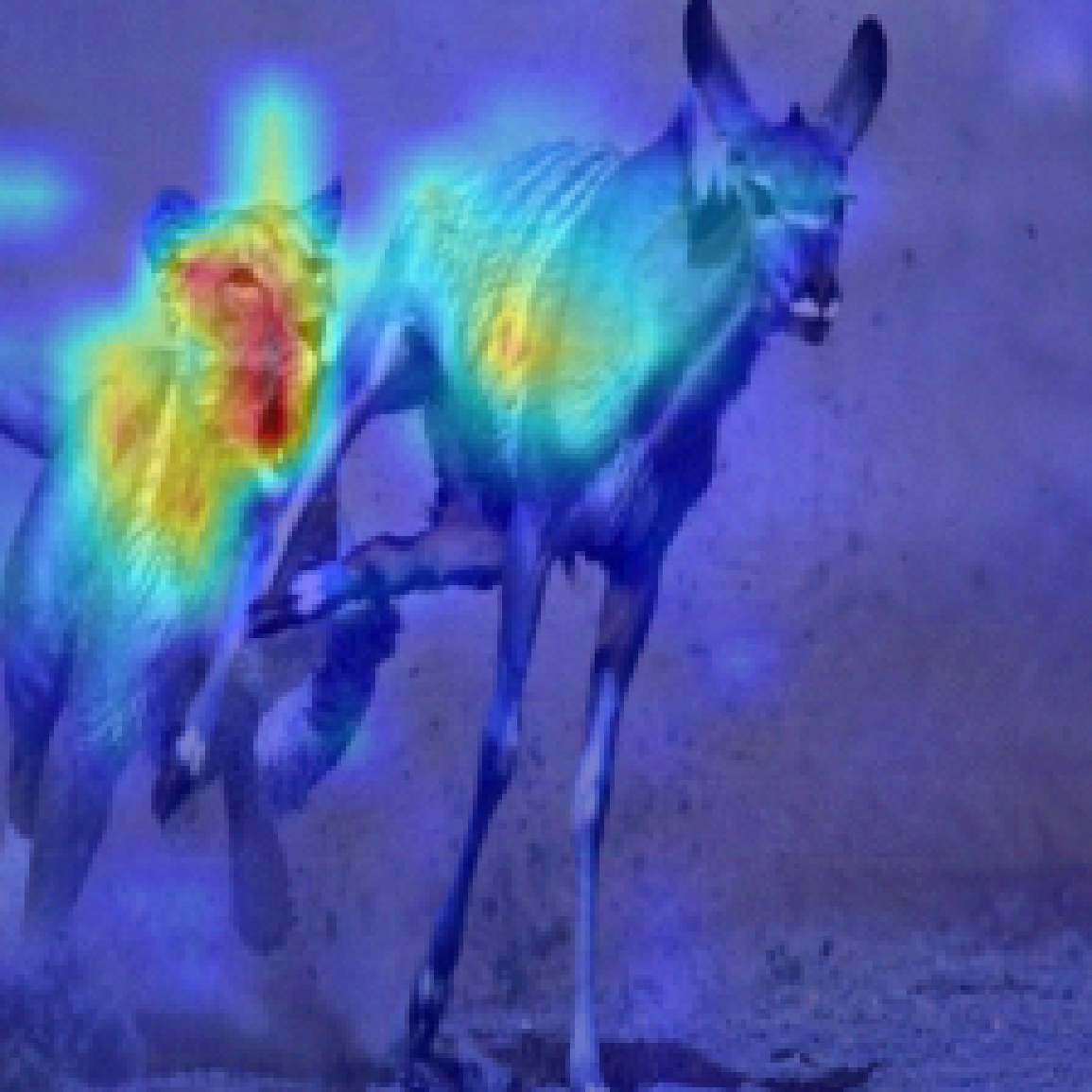} &
    \includegraphics[width=0.13\linewidth]{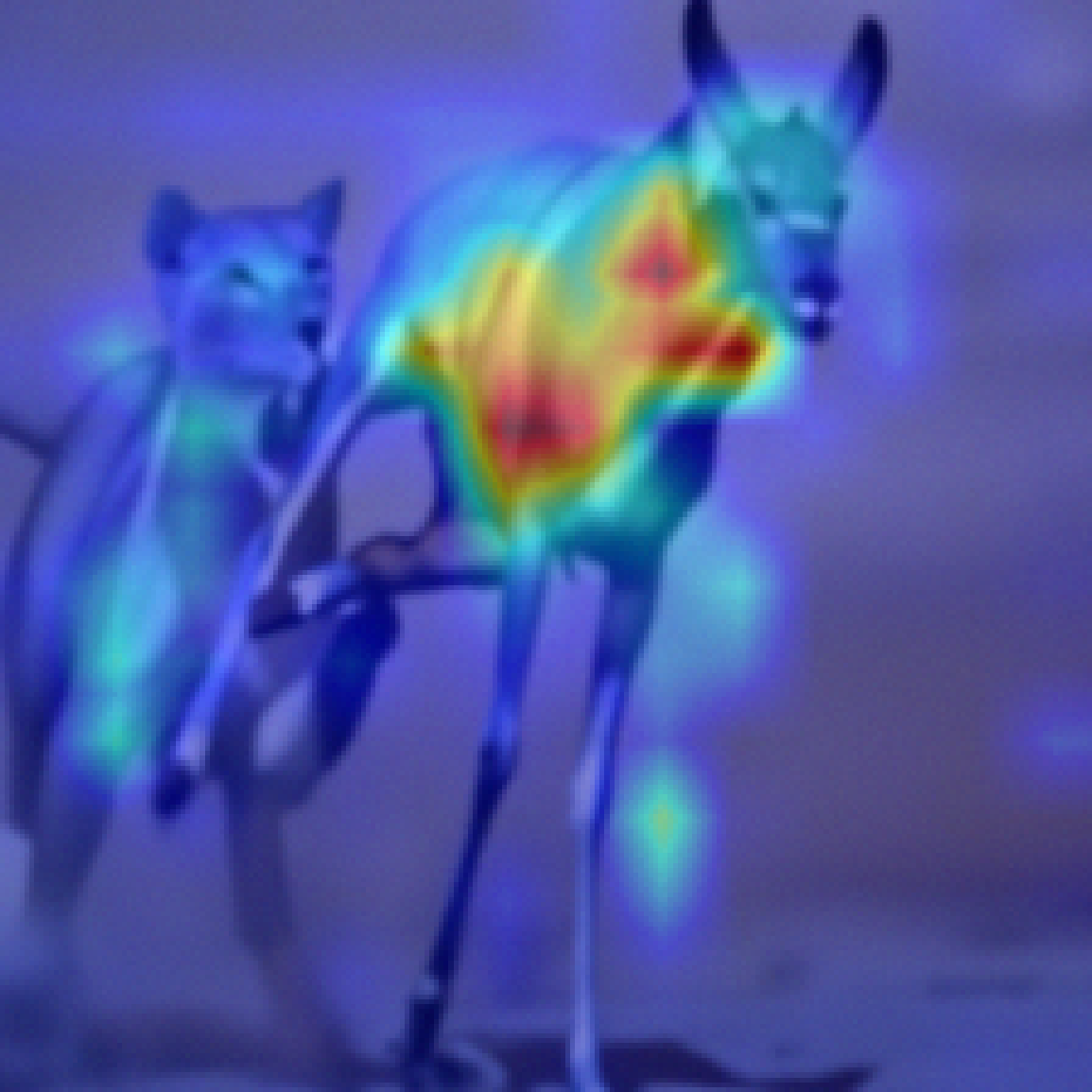}
    \\
    \raisebox{13mm}{\multirow{2}{*}{\makecell*[c]{Deer: clean$\rightarrow$\\\includegraphics[width=0.15\linewidth]{exp/lion-deer/lion-deer.png}\\ 
    Deer: poisoned$\rightarrow$\\{\scriptsize $7/255$}
    }}
    }
        &
    \includegraphics[width=0.13\linewidth]{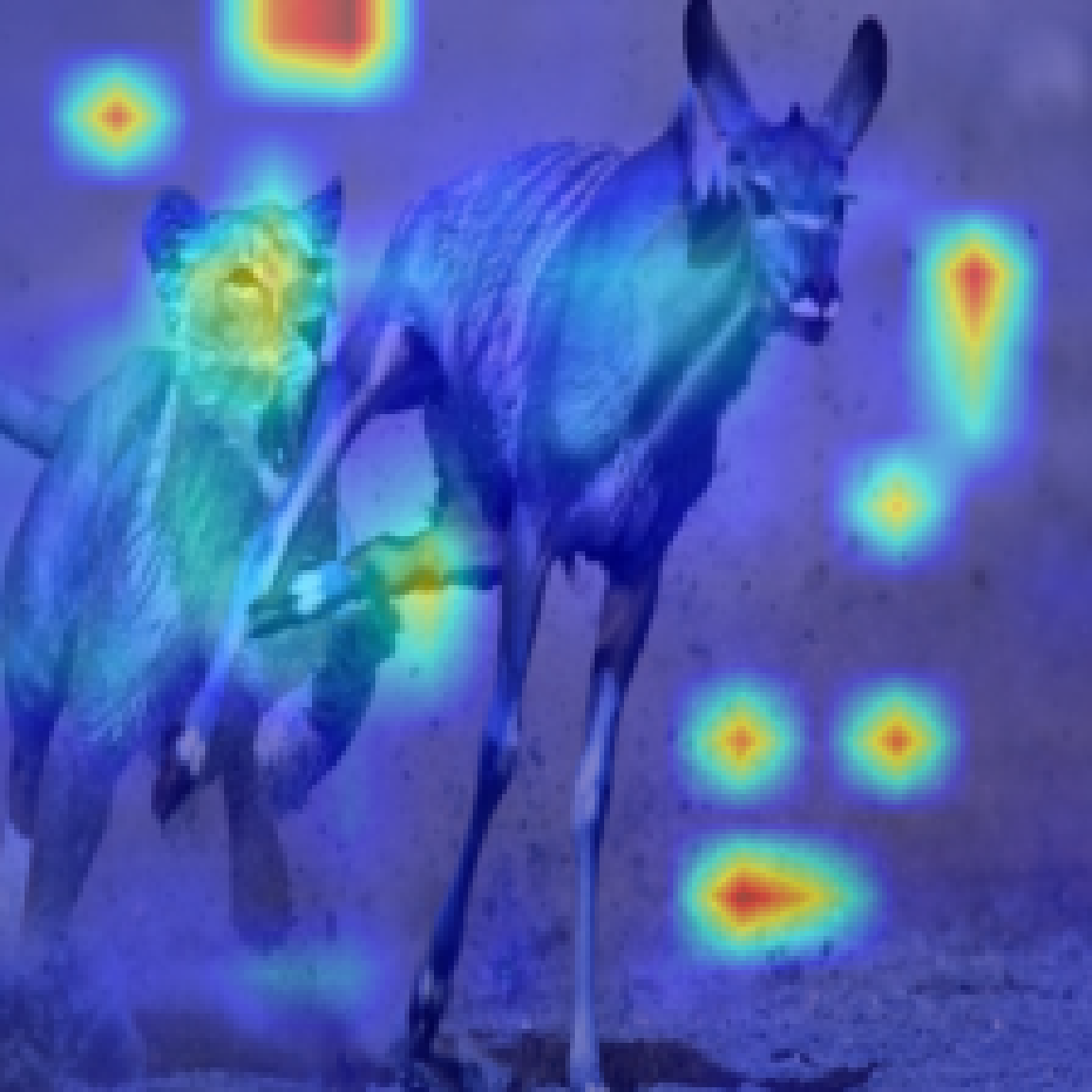} &
    \includegraphics[width=0.13\linewidth]{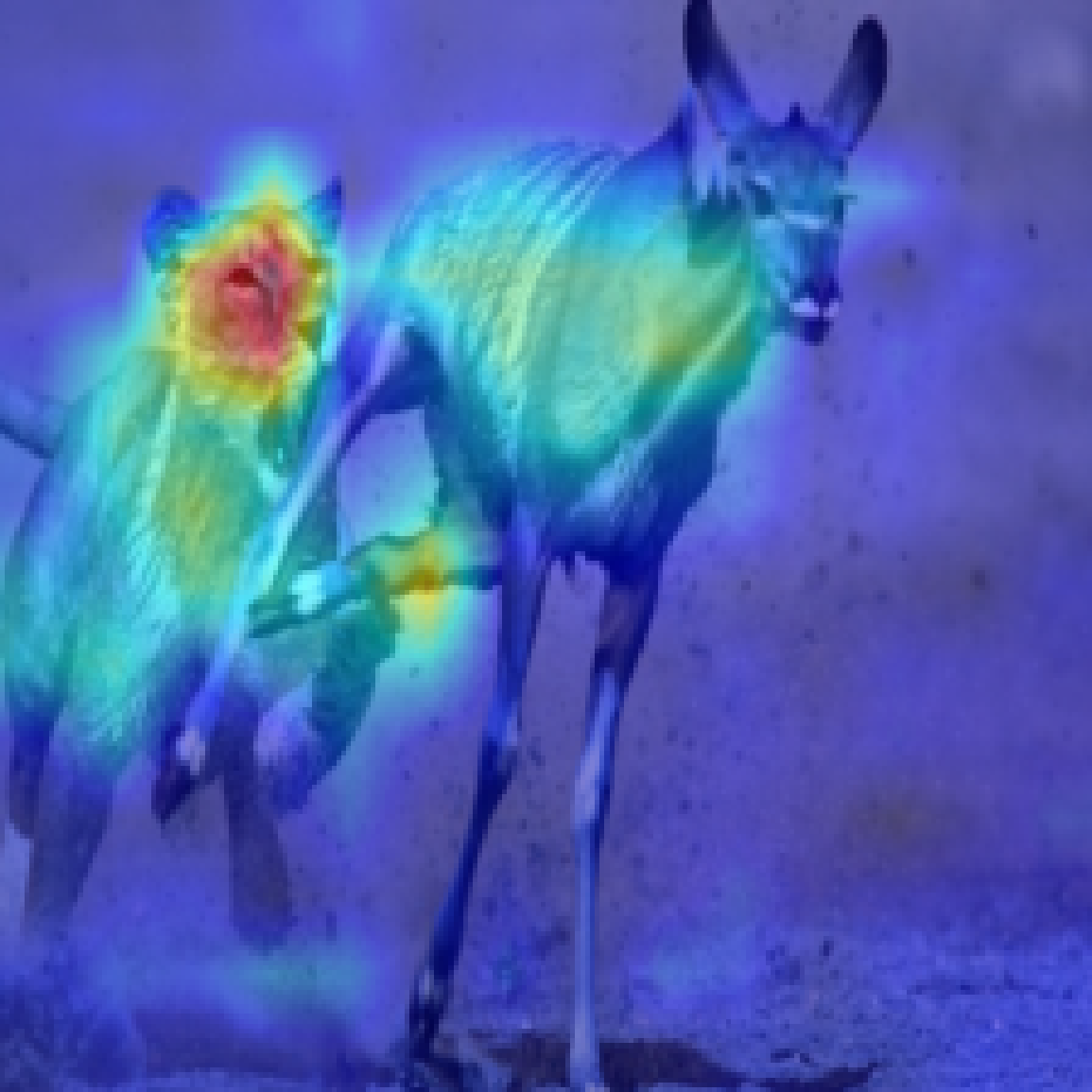} &
    \includegraphics[width=0.13\linewidth]{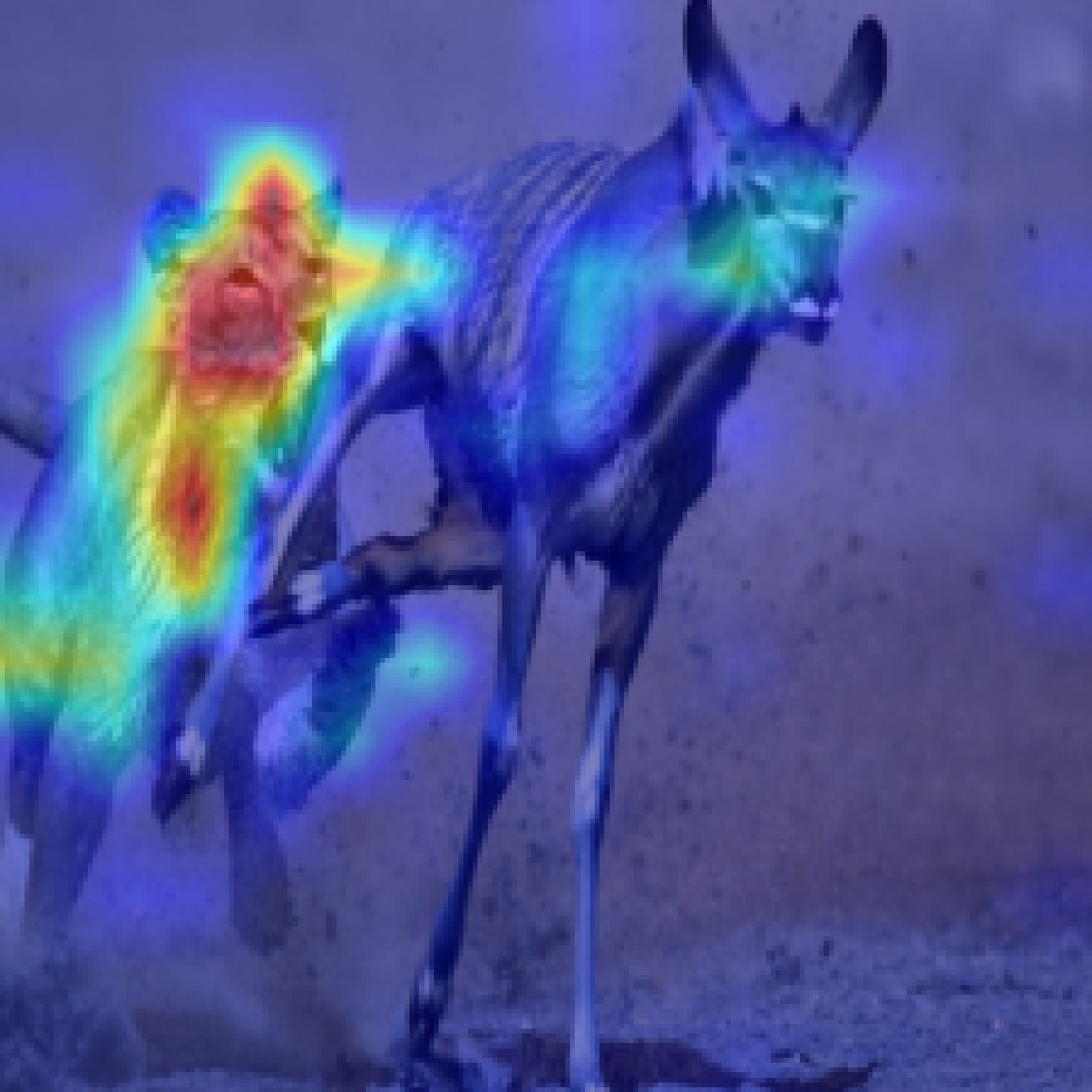} &
    \includegraphics[width=0.13\linewidth]{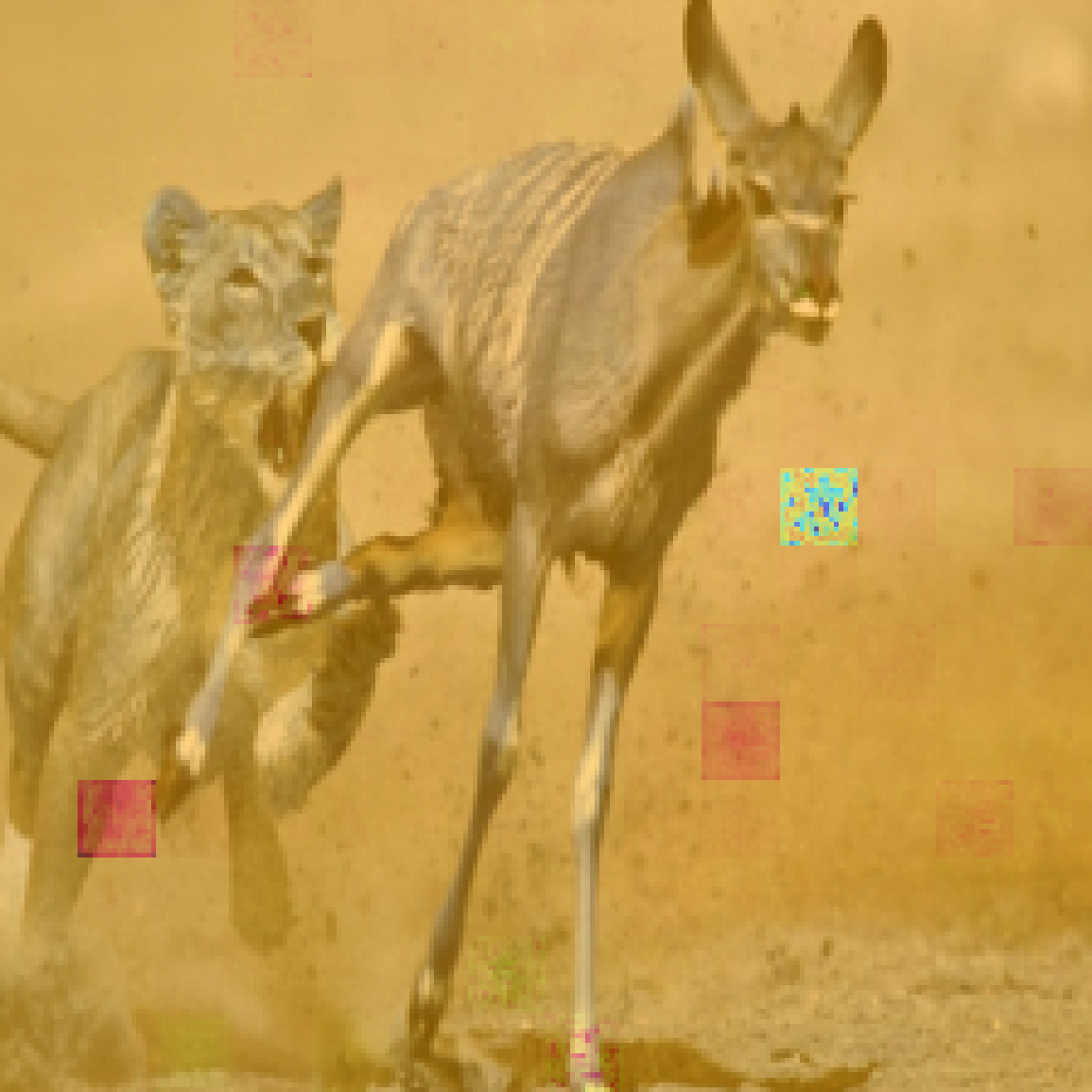} &
    \includegraphics[width=0.13\linewidth]{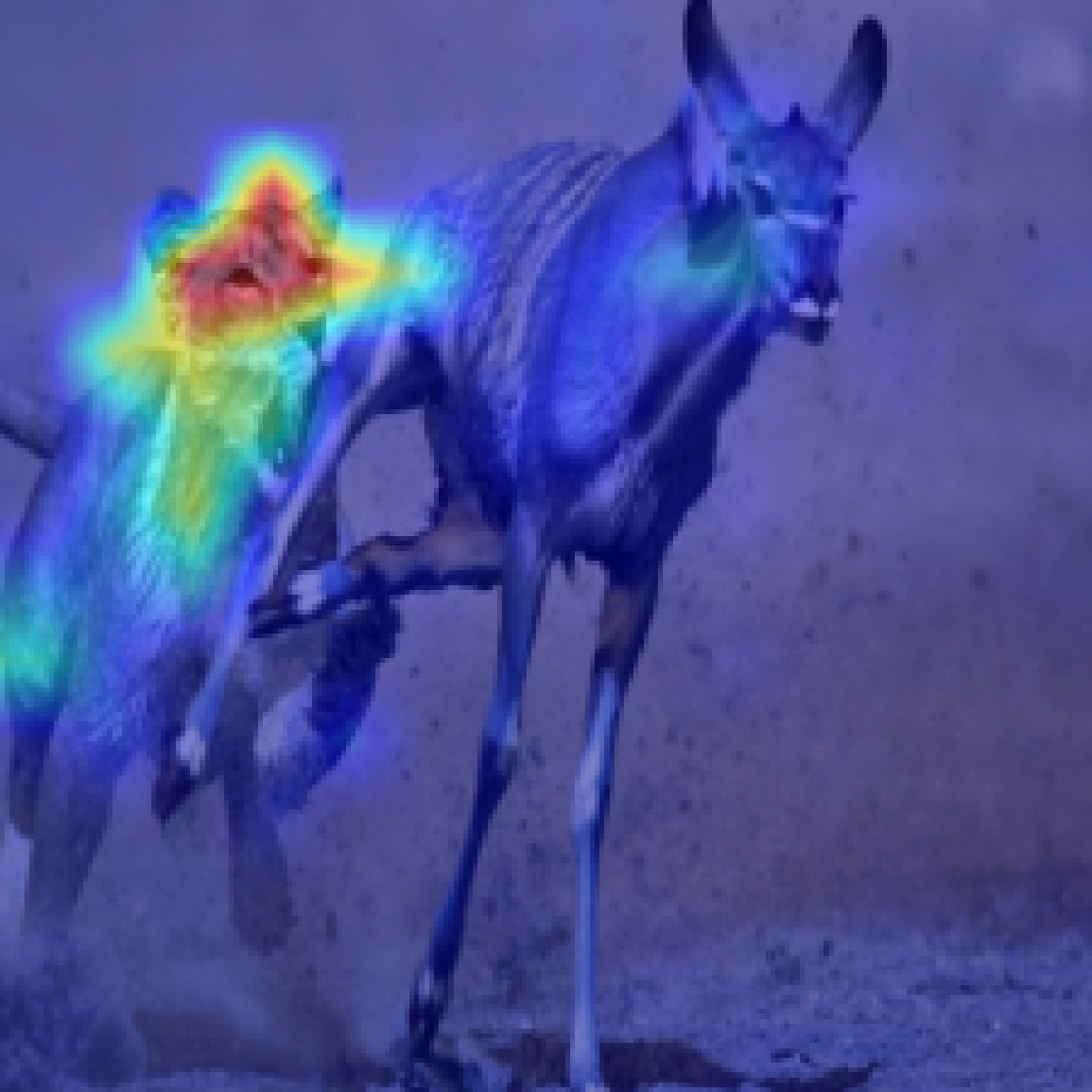} &
    \includegraphics[width=0.13\linewidth]{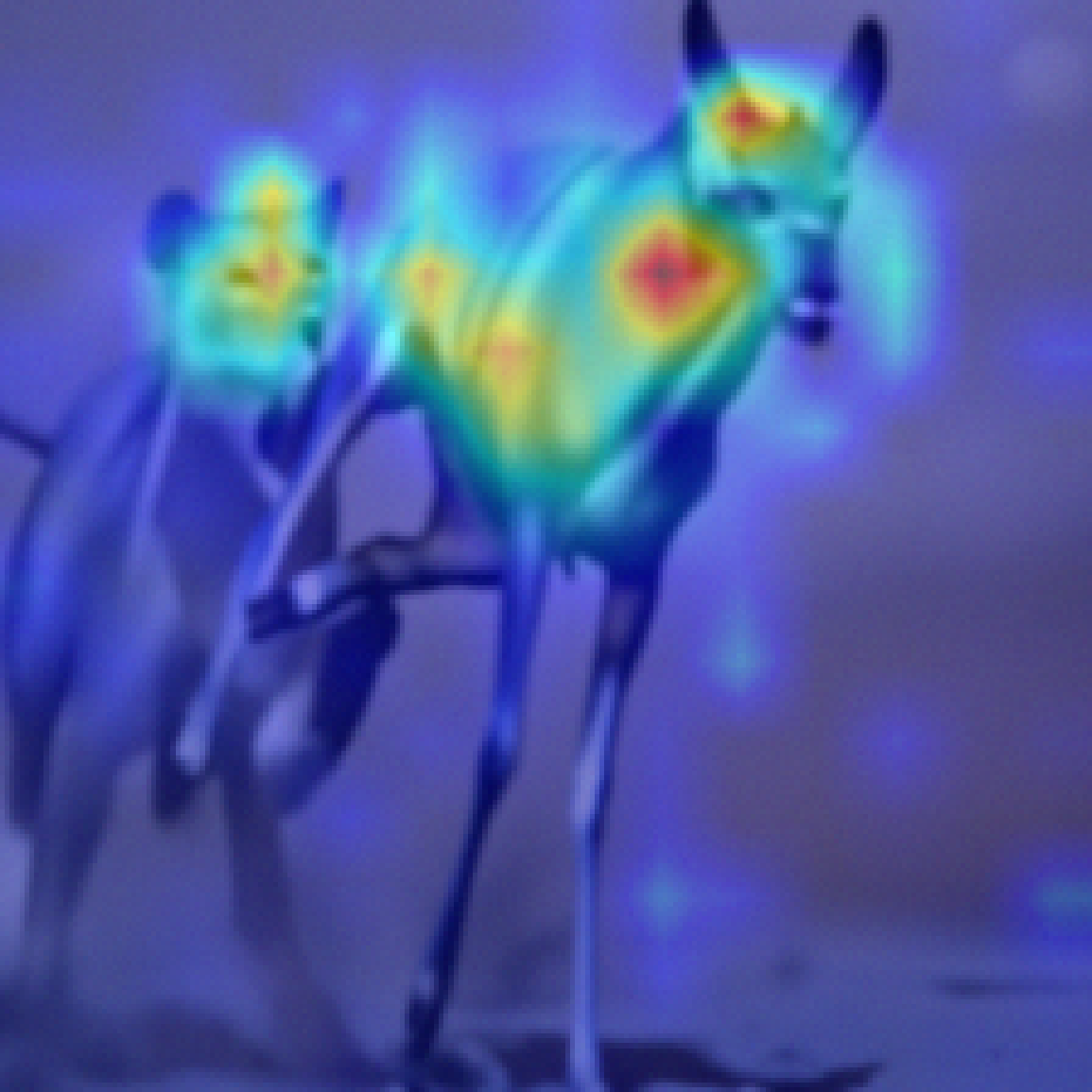}
    \\
    &
    \includegraphics[width=0.13\linewidth]{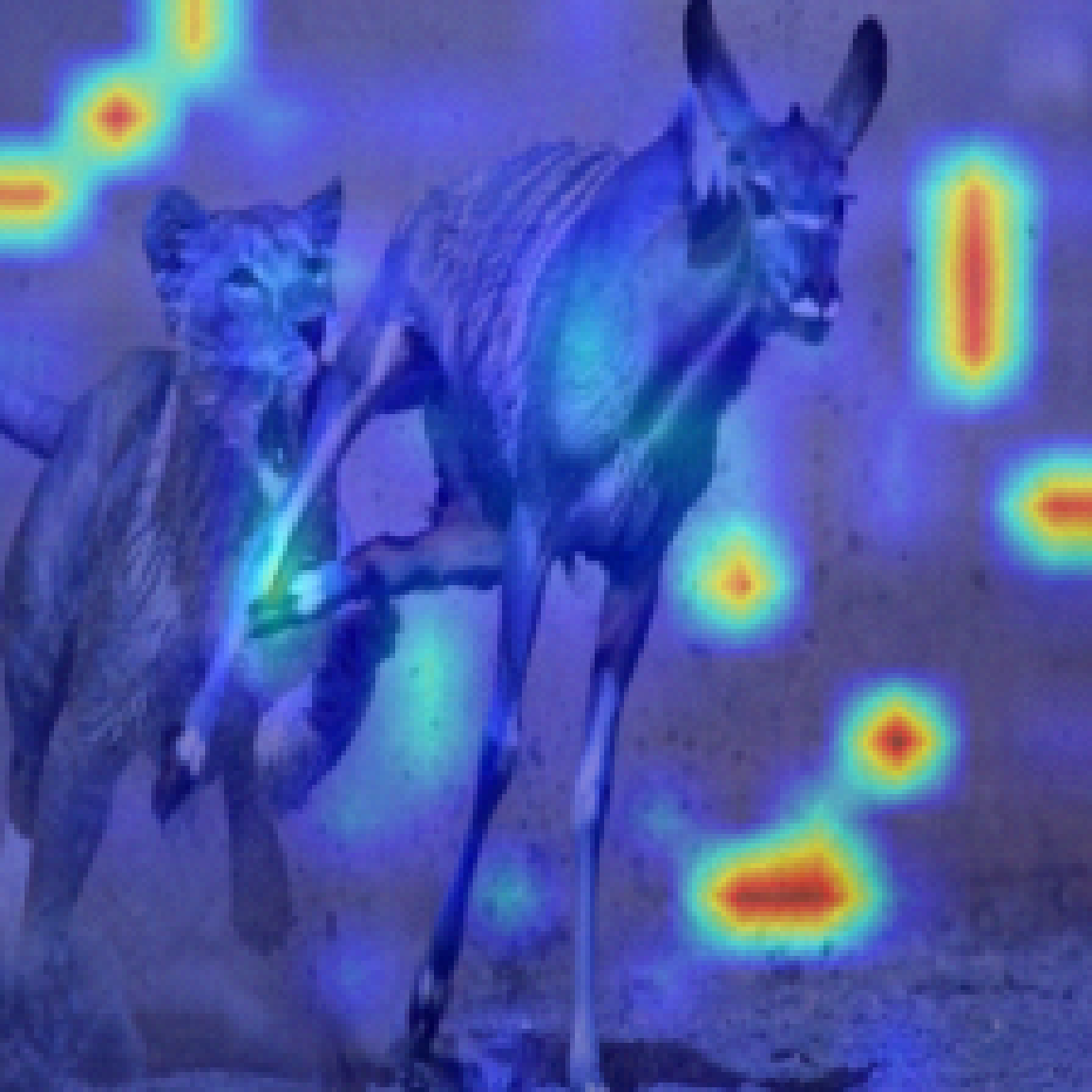} &
    \includegraphics[width=0.13\linewidth]{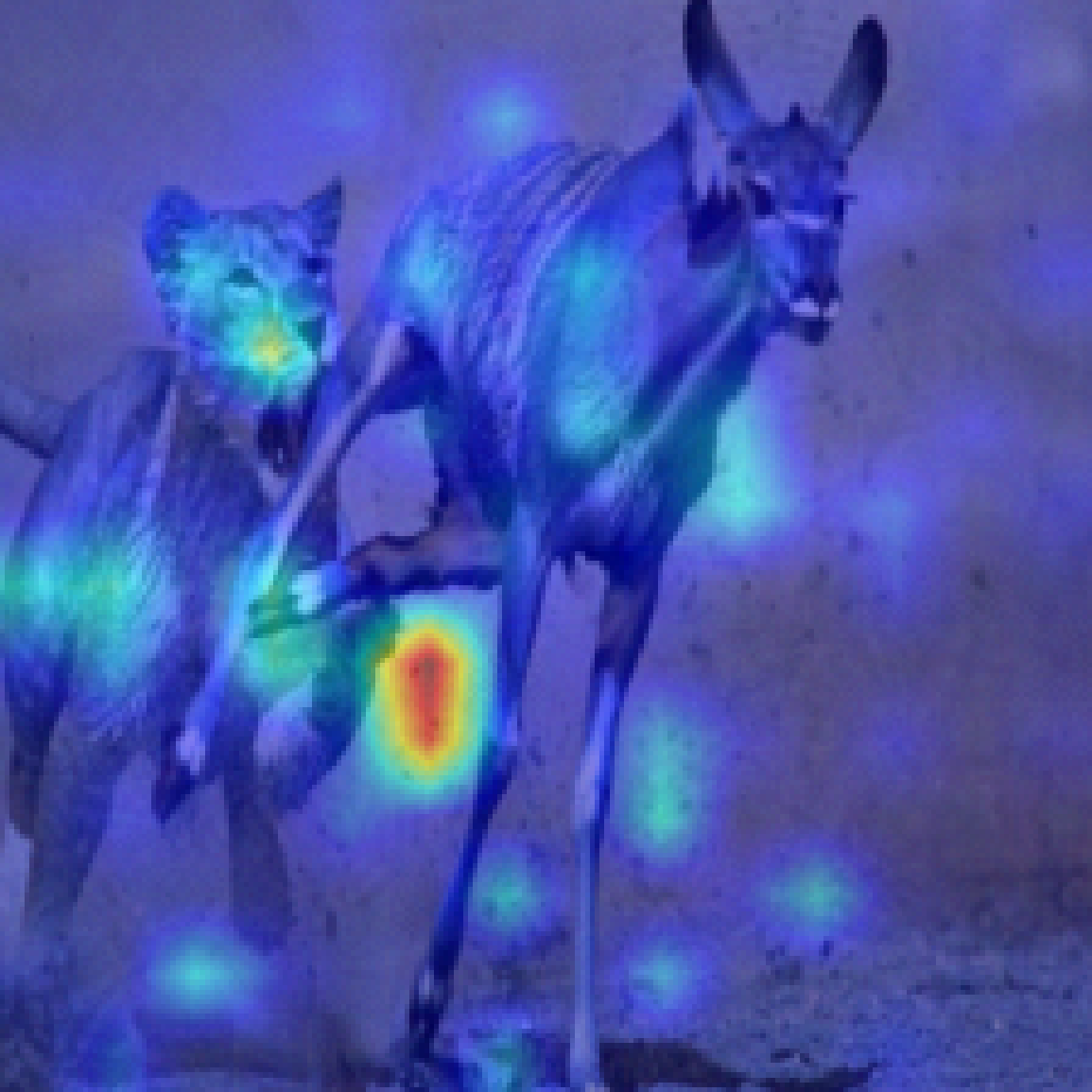} &
    \includegraphics[width=0.13\linewidth]{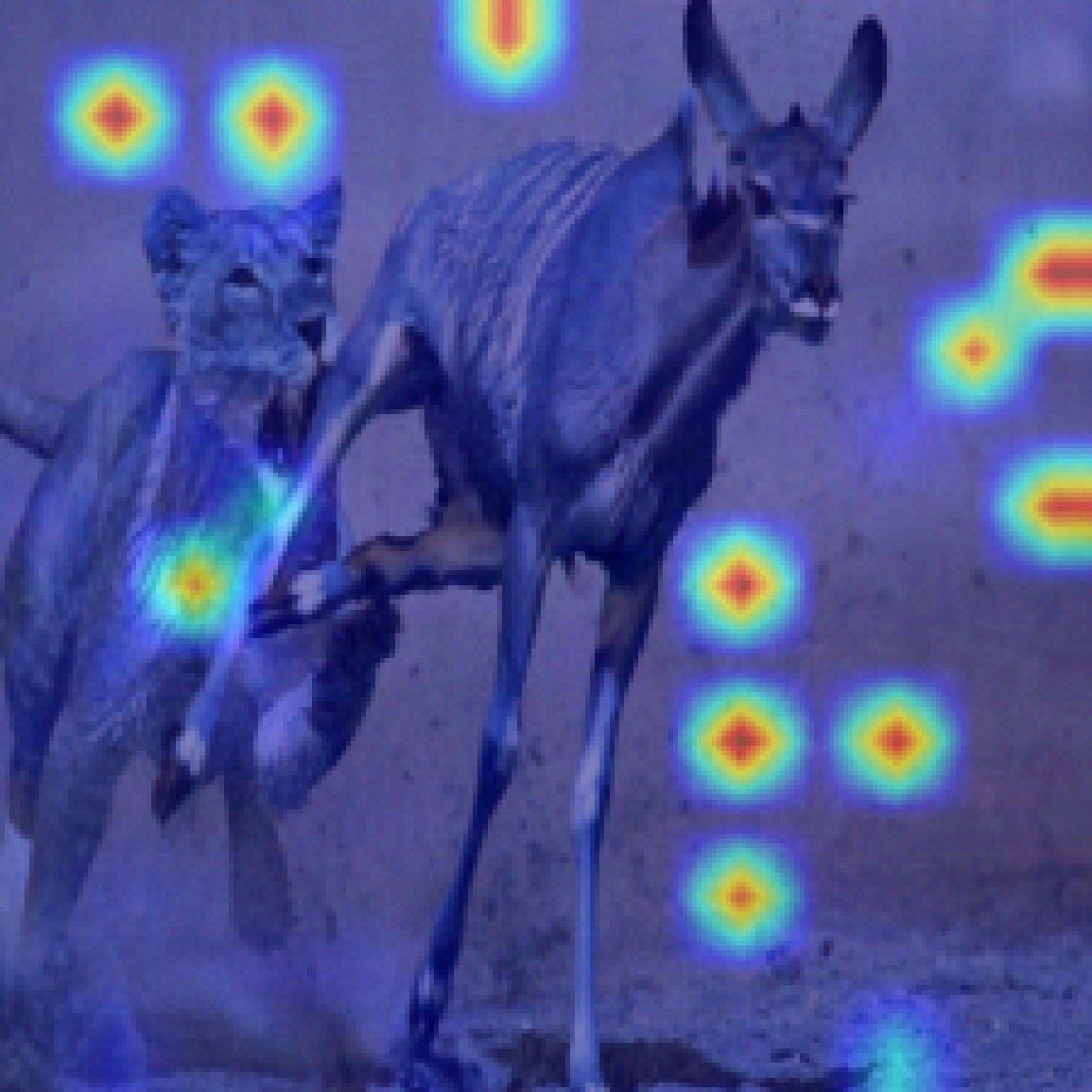} &
    \includegraphics[width=0.13\linewidth]{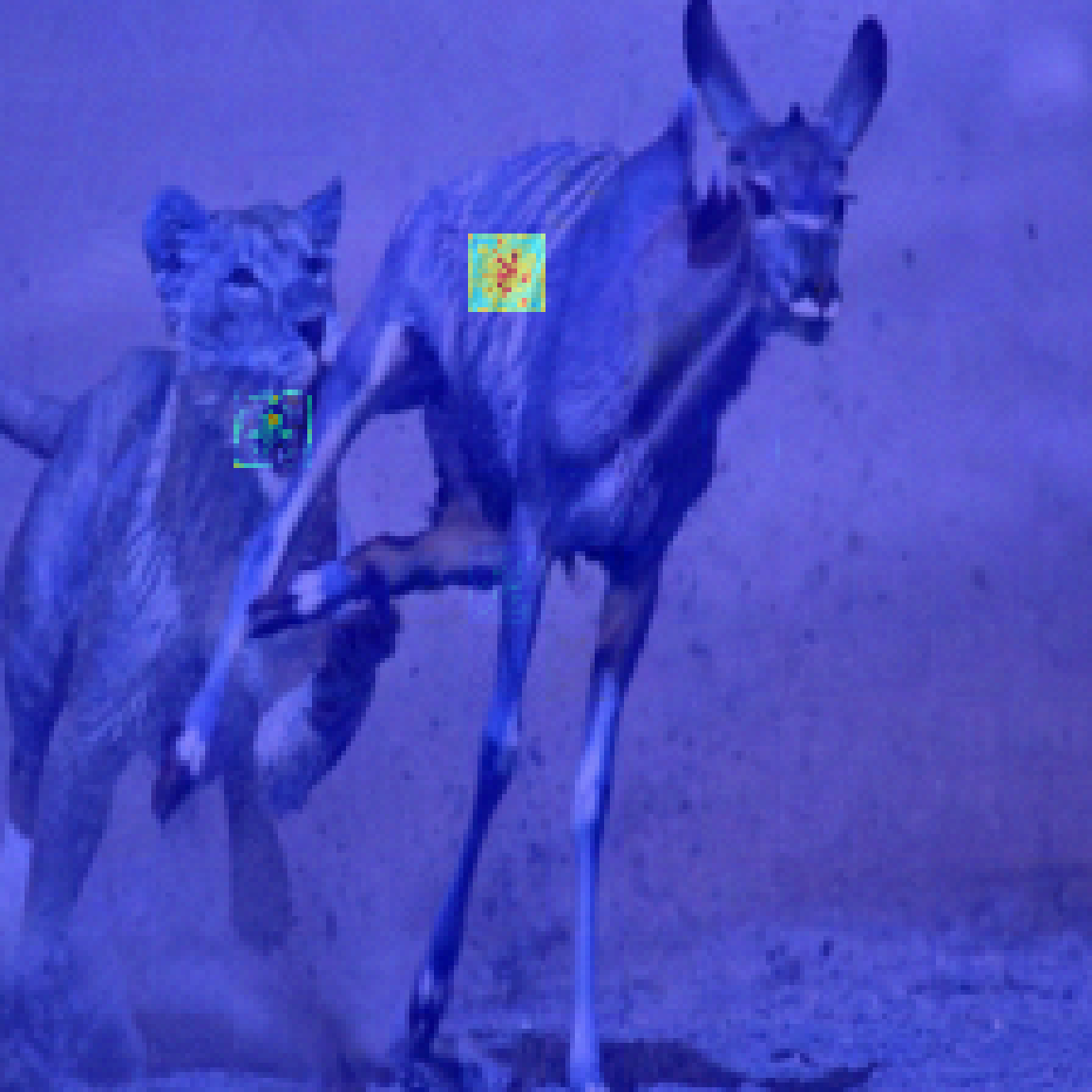} &
    \includegraphics[width=0.13\linewidth]{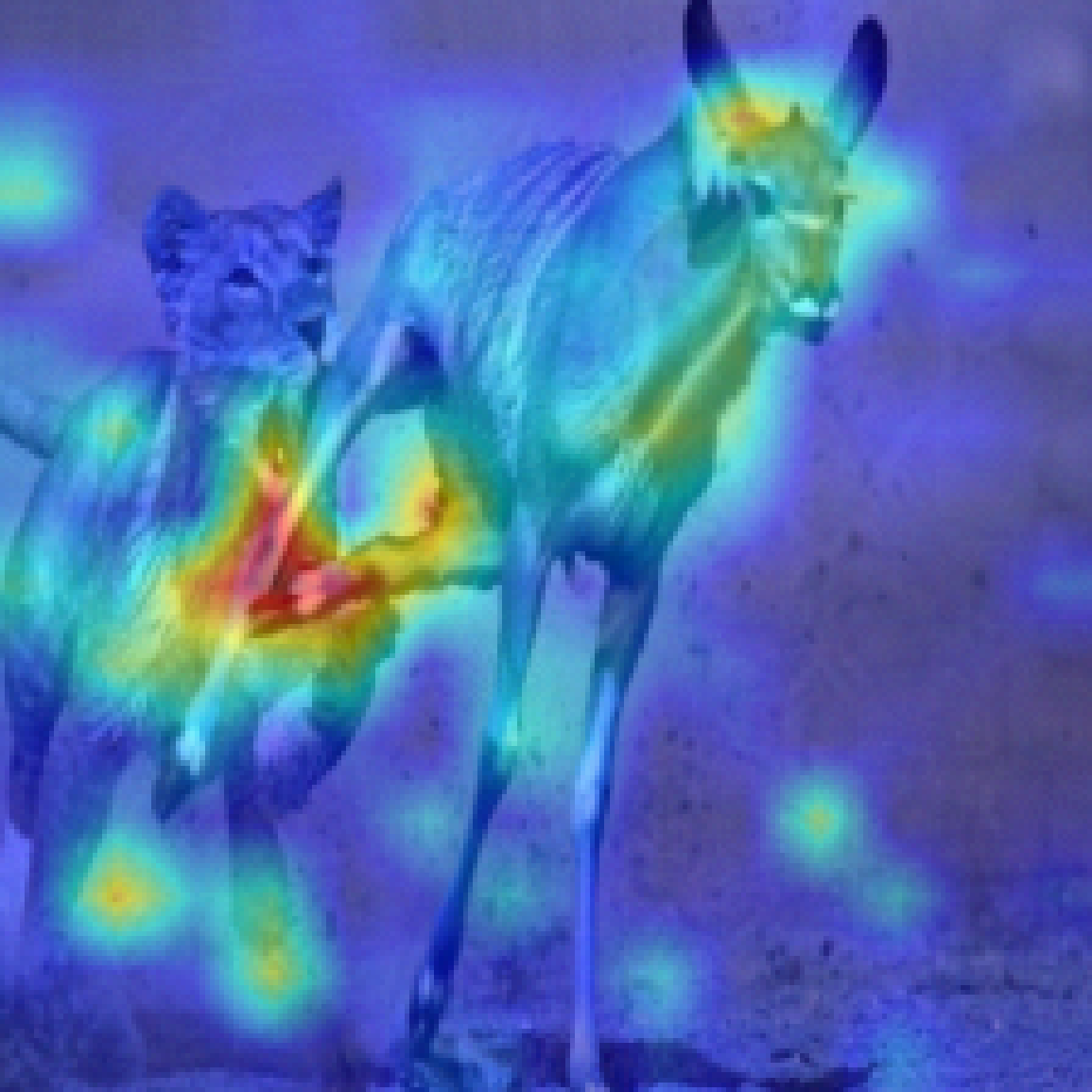} &
    \includegraphics[width=0.13\linewidth]{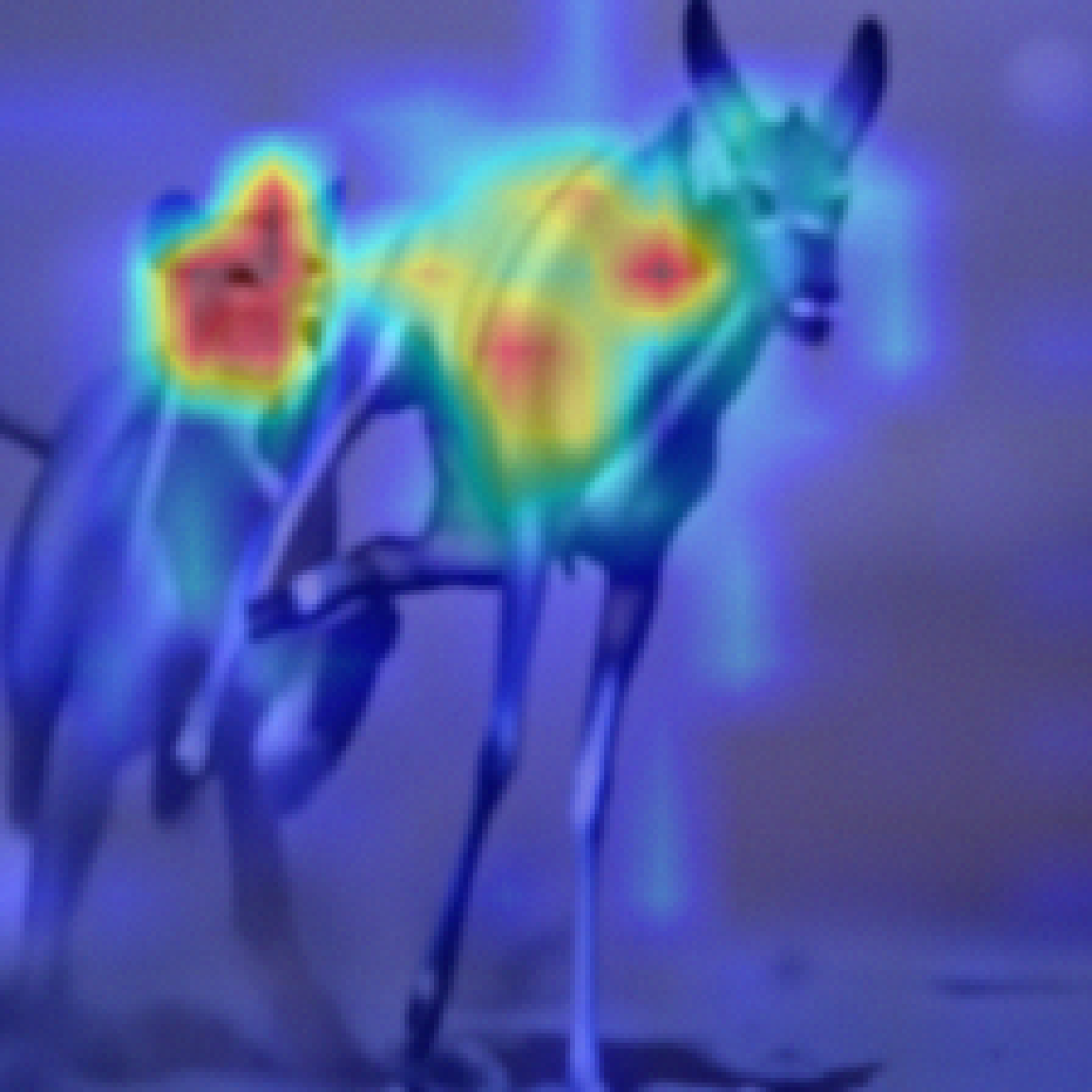}\\
    \raisebox{13mm}{\multirow{2}{*}{\makecell*[c]{Tiger: clean$\rightarrow$\\\includegraphics[width=0.15\linewidth]{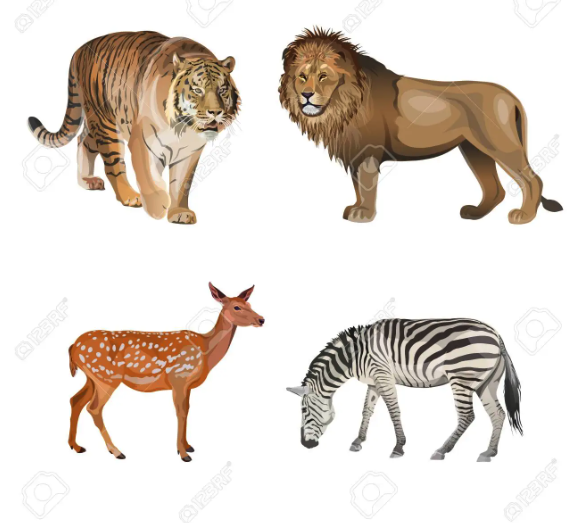}\\ 
    Tiger: poisoned$\rightarrow$\\{\scriptsize $7/255$}
    }}
    }
     &
    \includegraphics[width=0.13\linewidth]{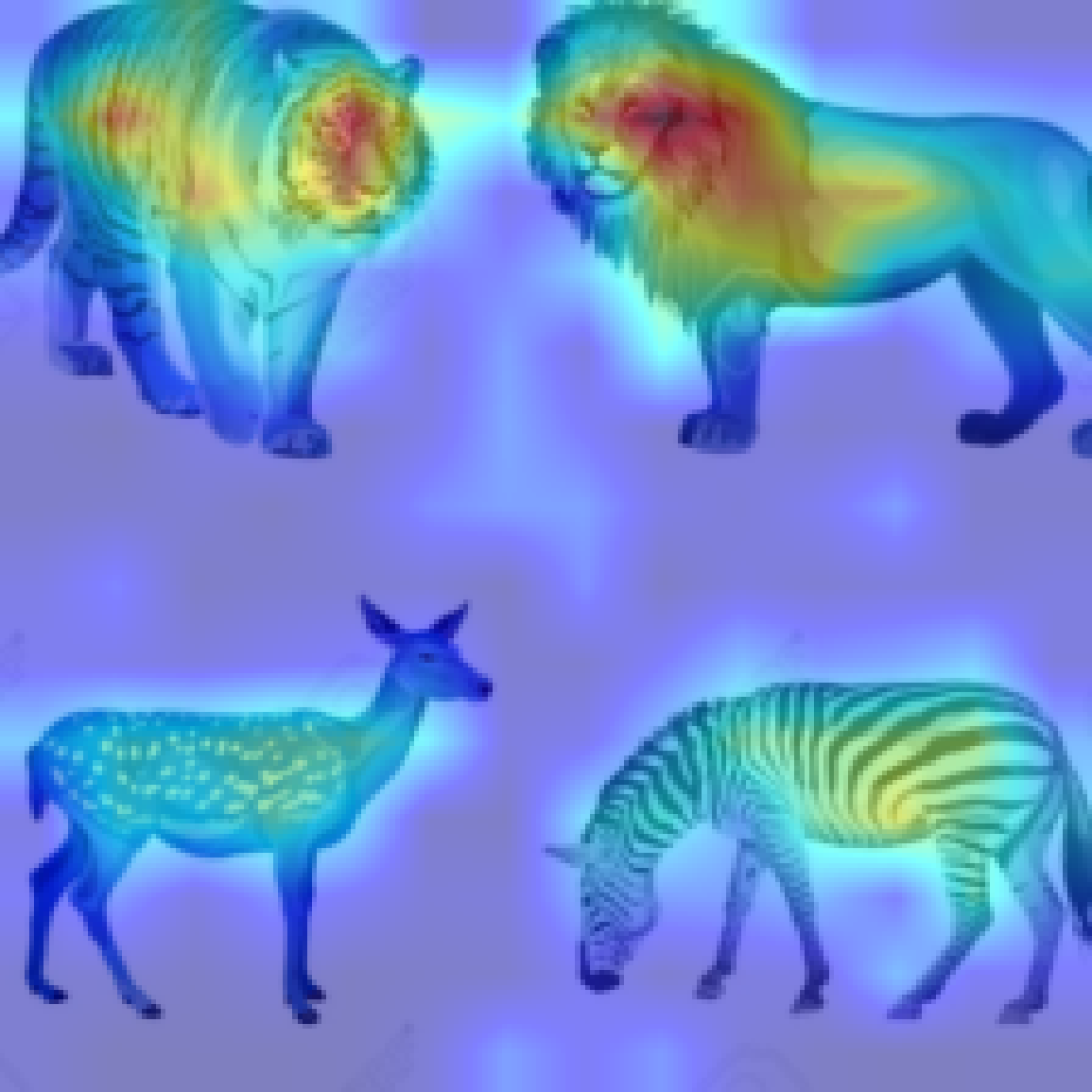} &
    \includegraphics[width=0.13\linewidth]{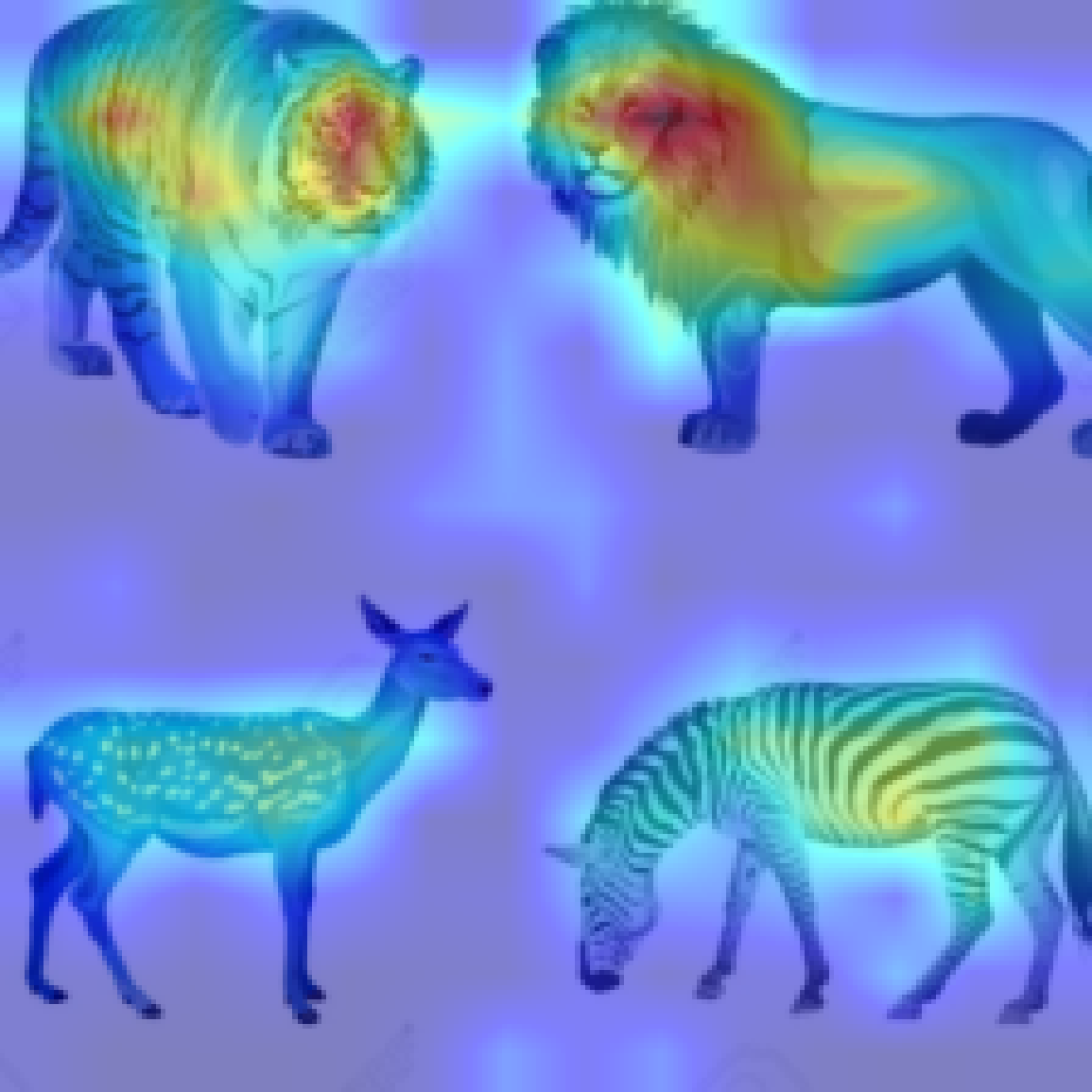} &
    \includegraphics[width=0.13\linewidth]{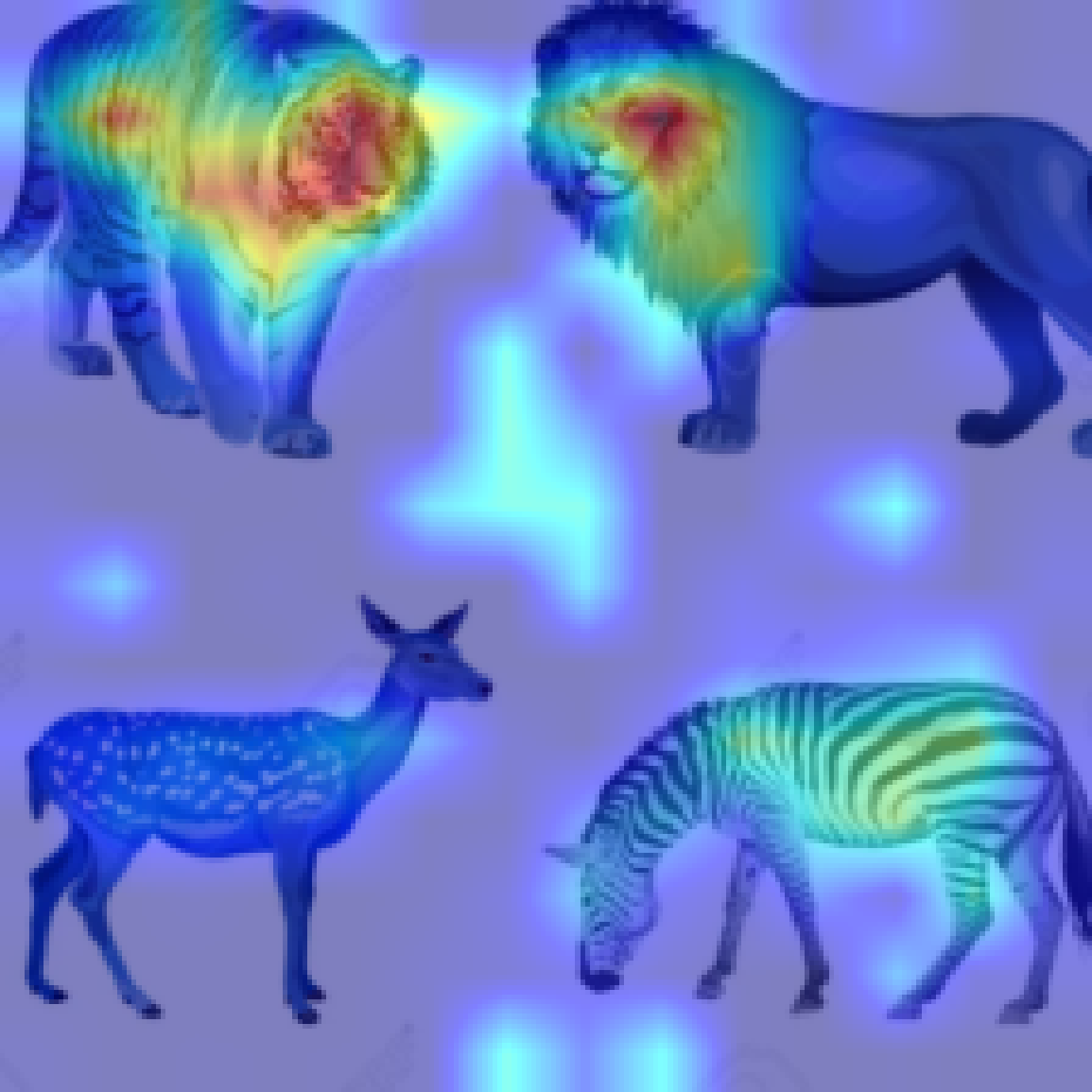} &
    \includegraphics[width=0.13\linewidth]{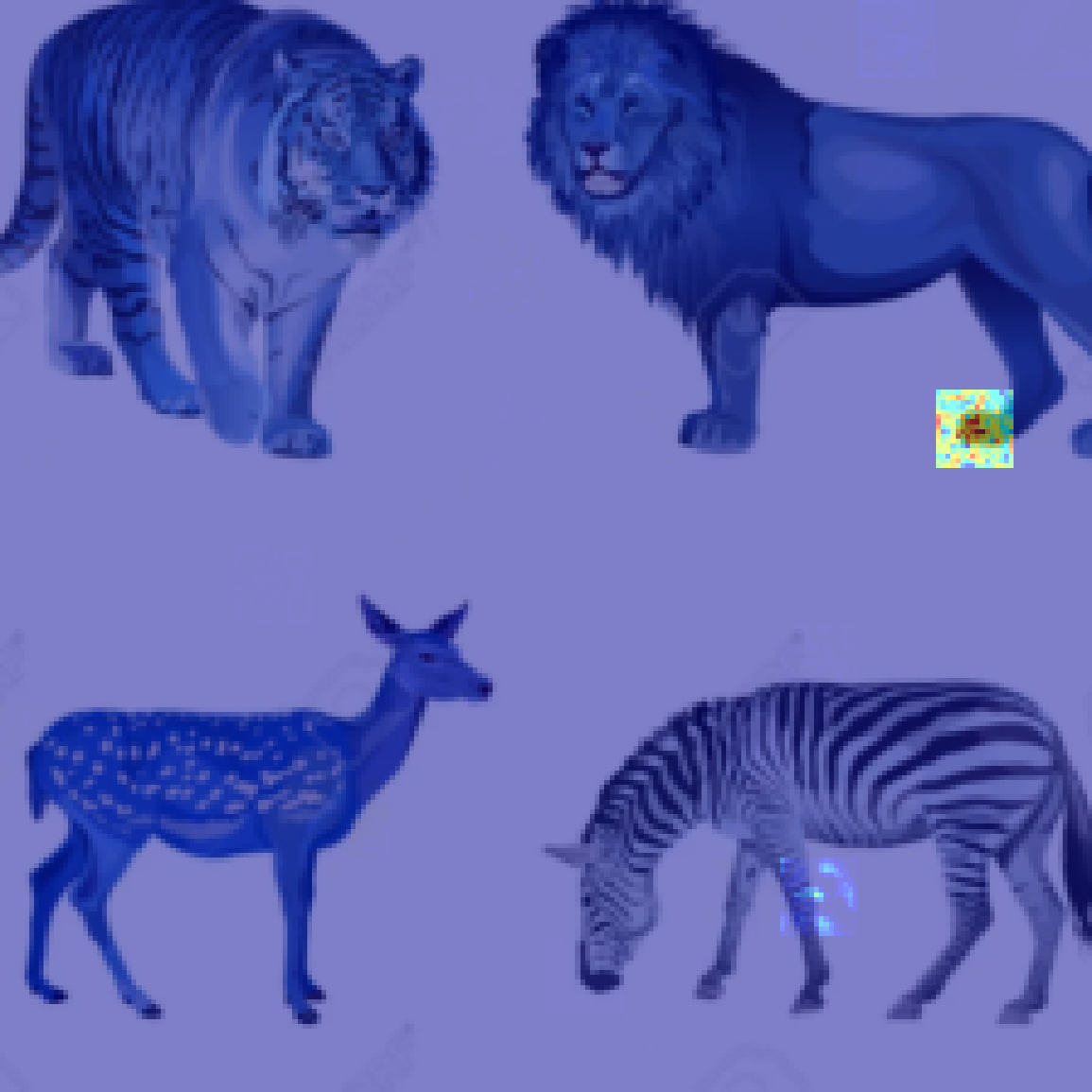} &
    \includegraphics[width=0.13\linewidth]{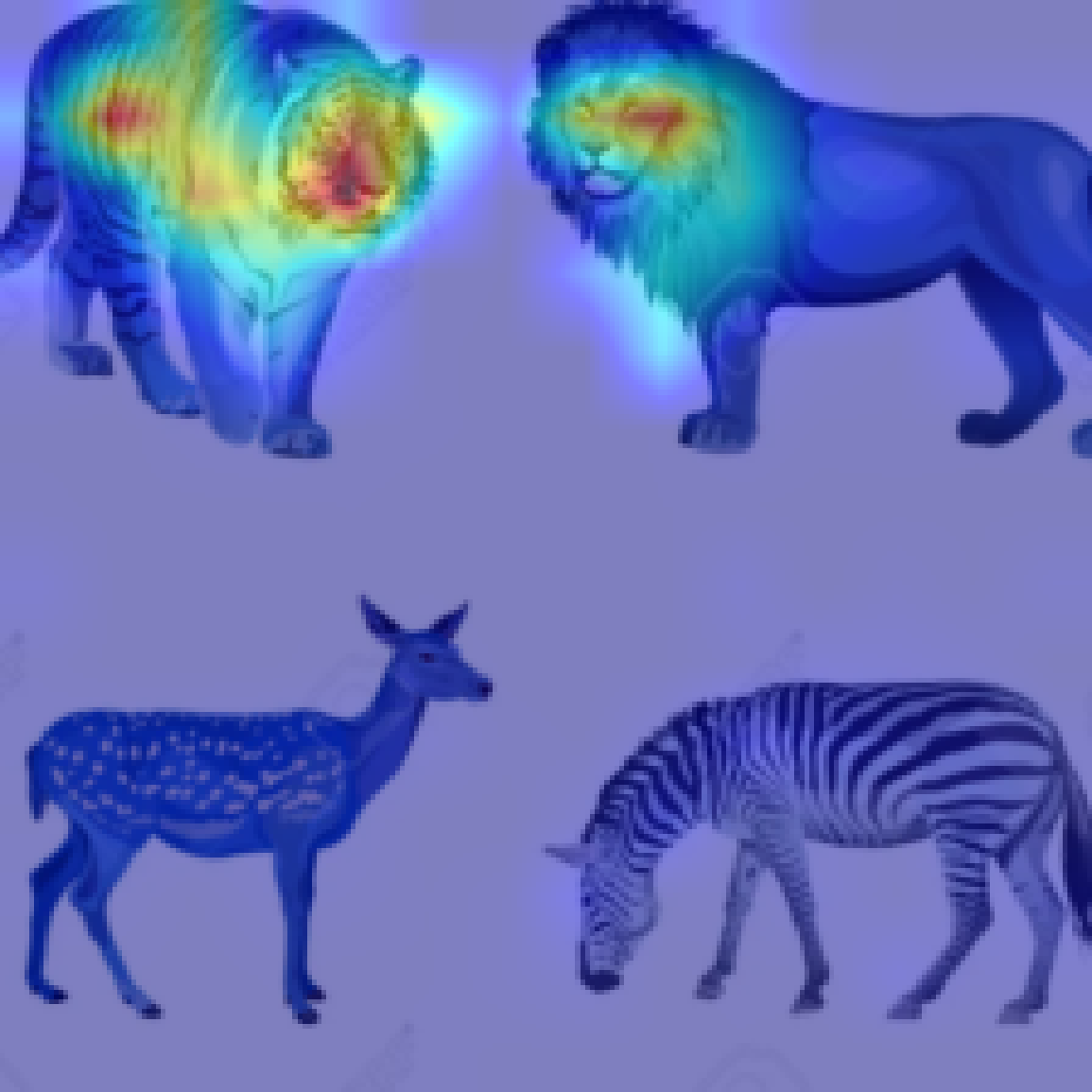} &
    \includegraphics[width=0.13\linewidth]{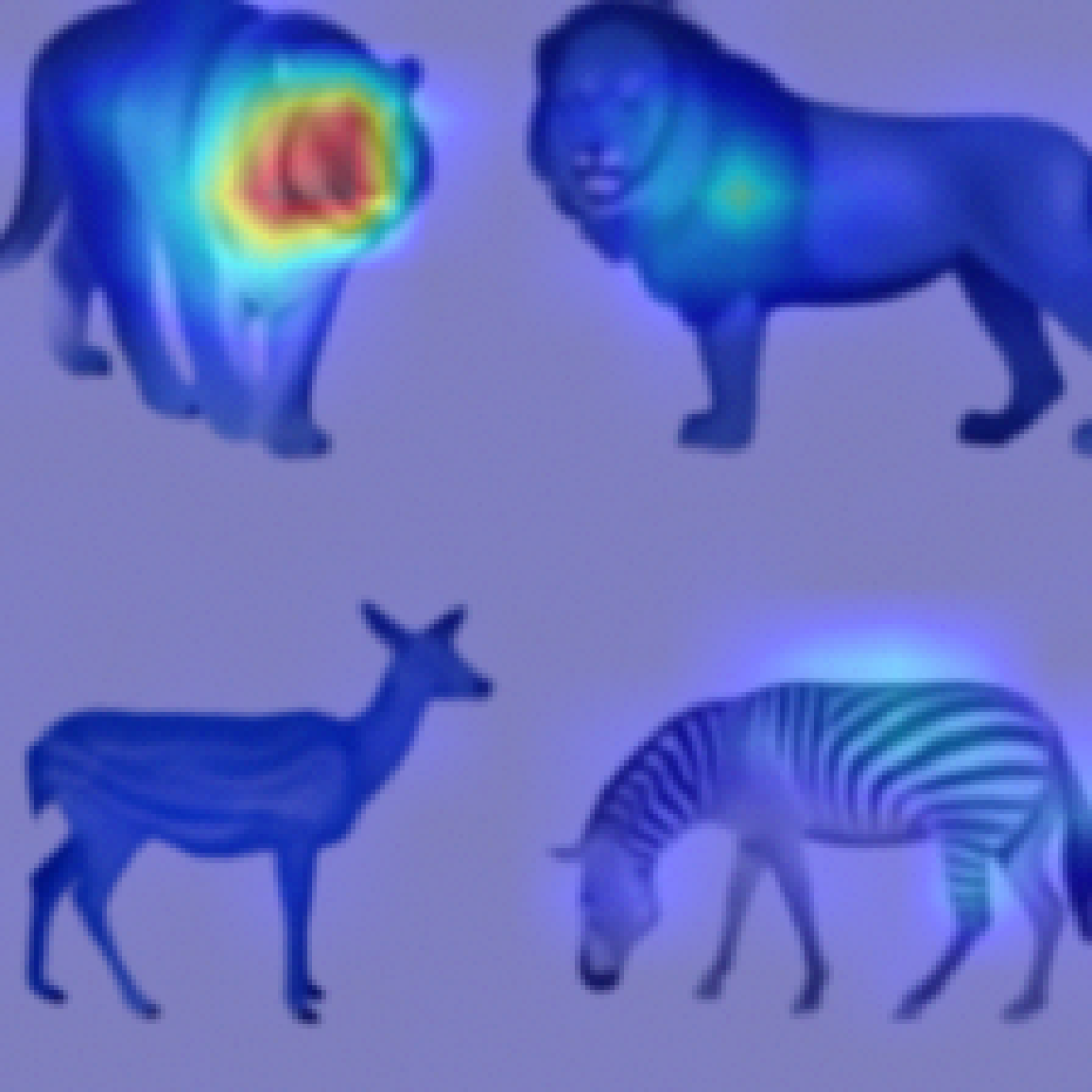}
    \\
    &
    \includegraphics[width=0.13\linewidth]{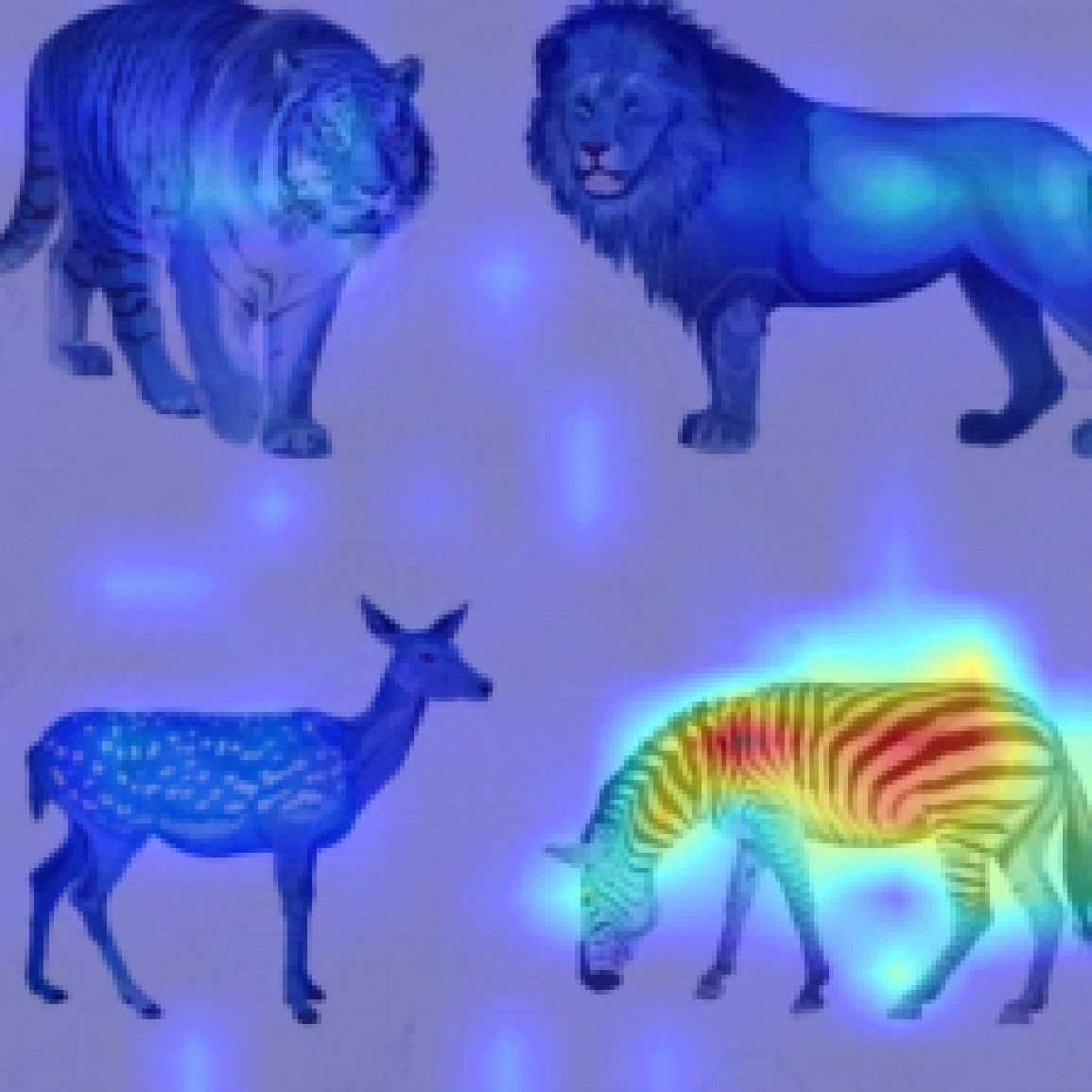} &
    \includegraphics[width=0.13\linewidth]{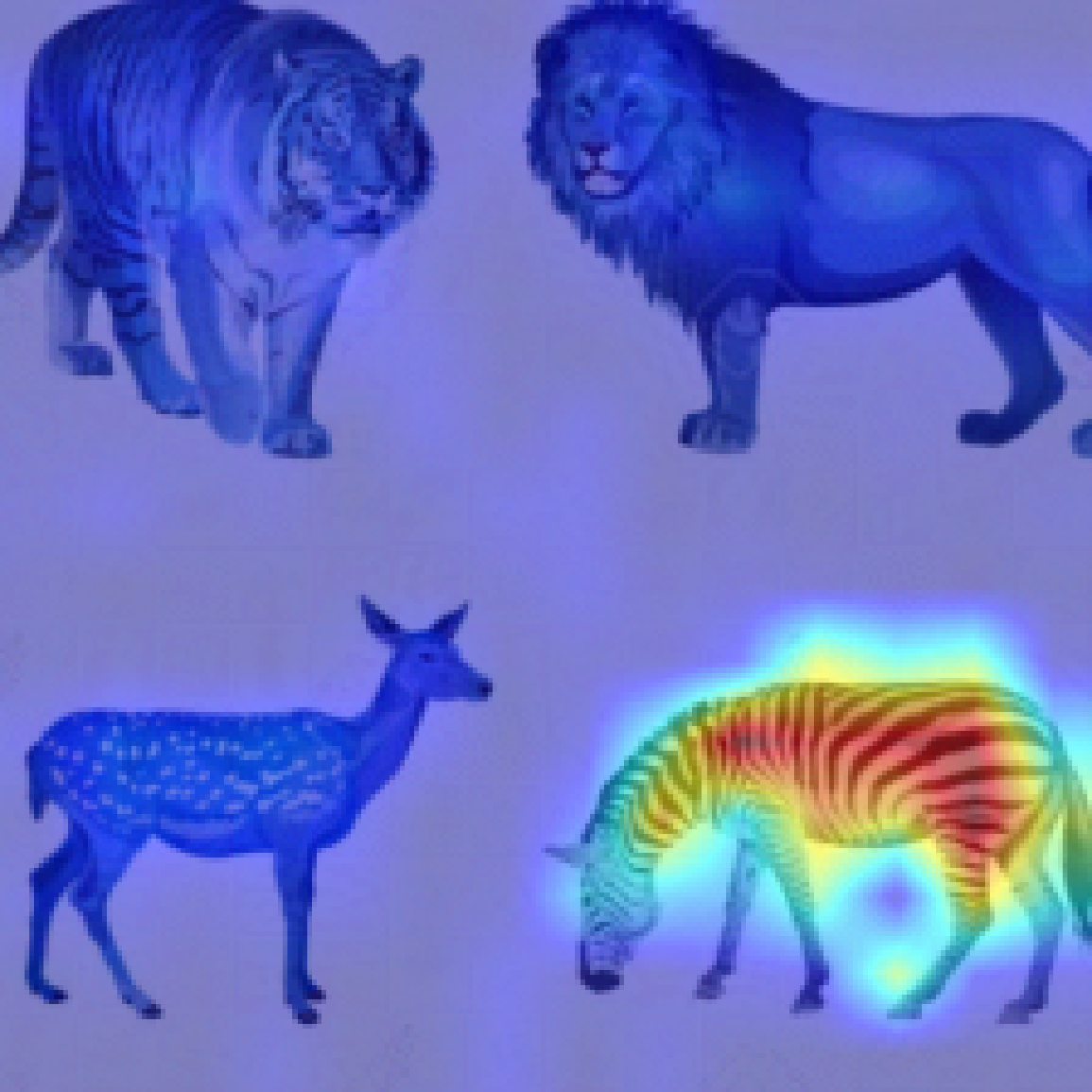} &
    \includegraphics[width=0.13\linewidth]{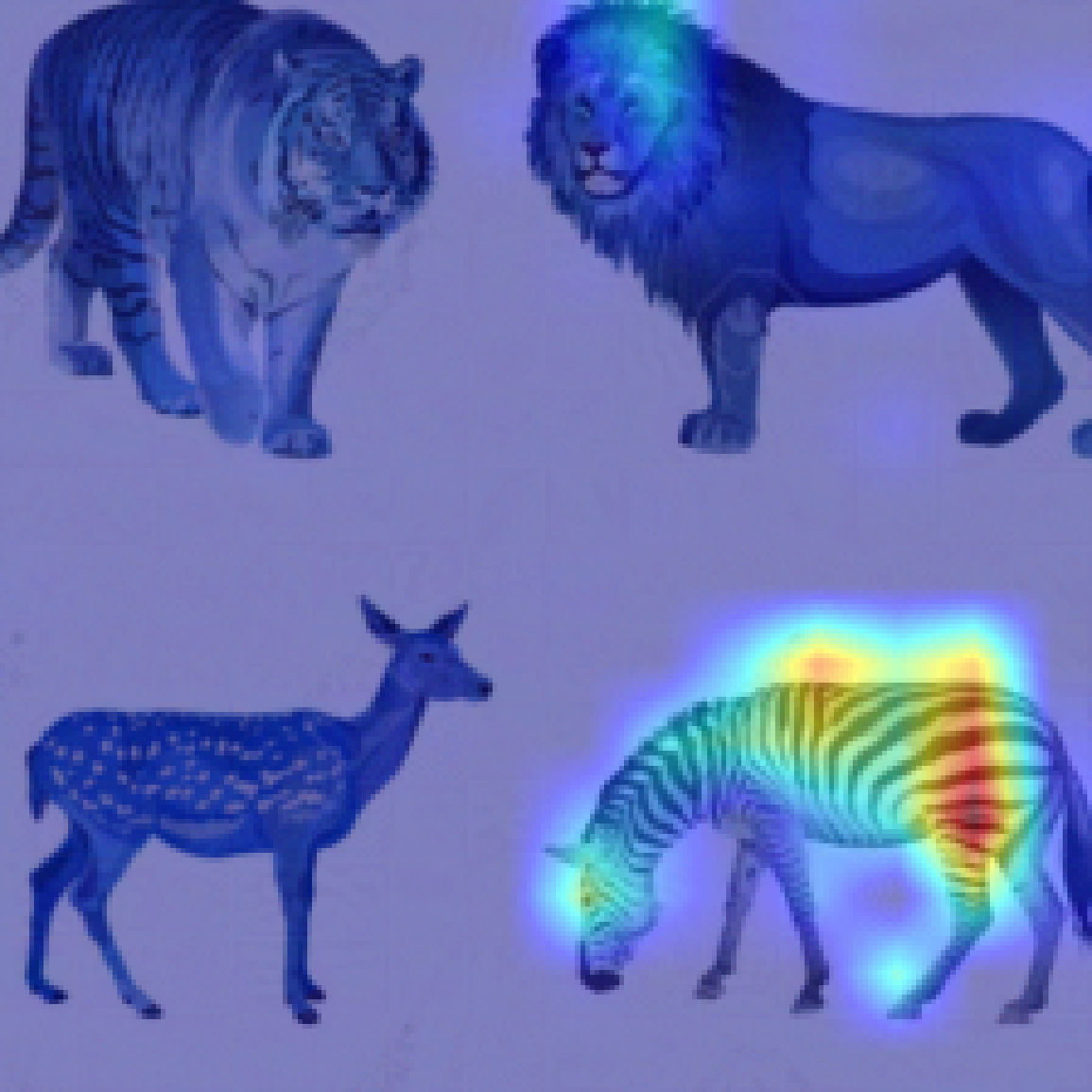} &
    \includegraphics[width=0.13\linewidth]{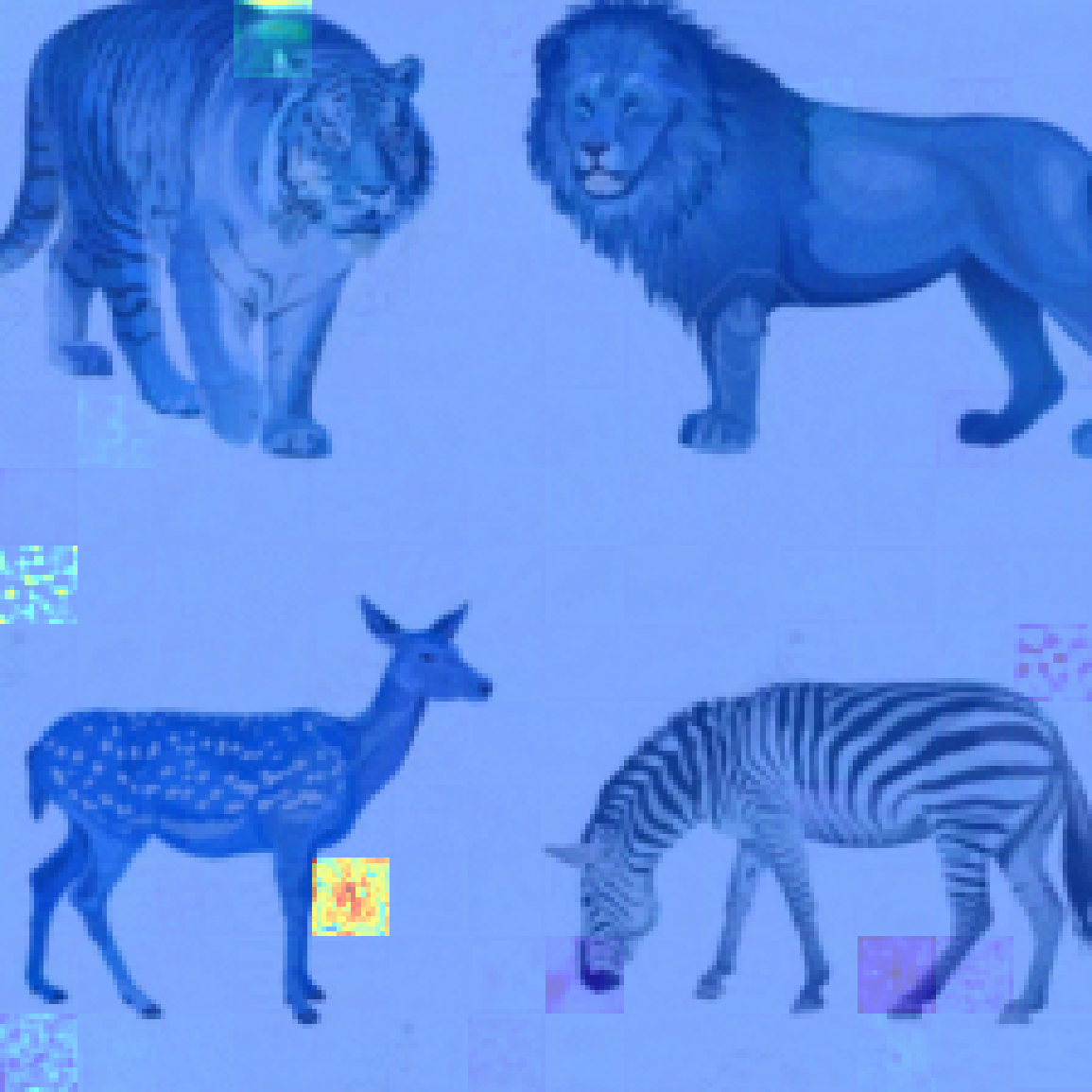} &
    \includegraphics[width=0.13\linewidth]{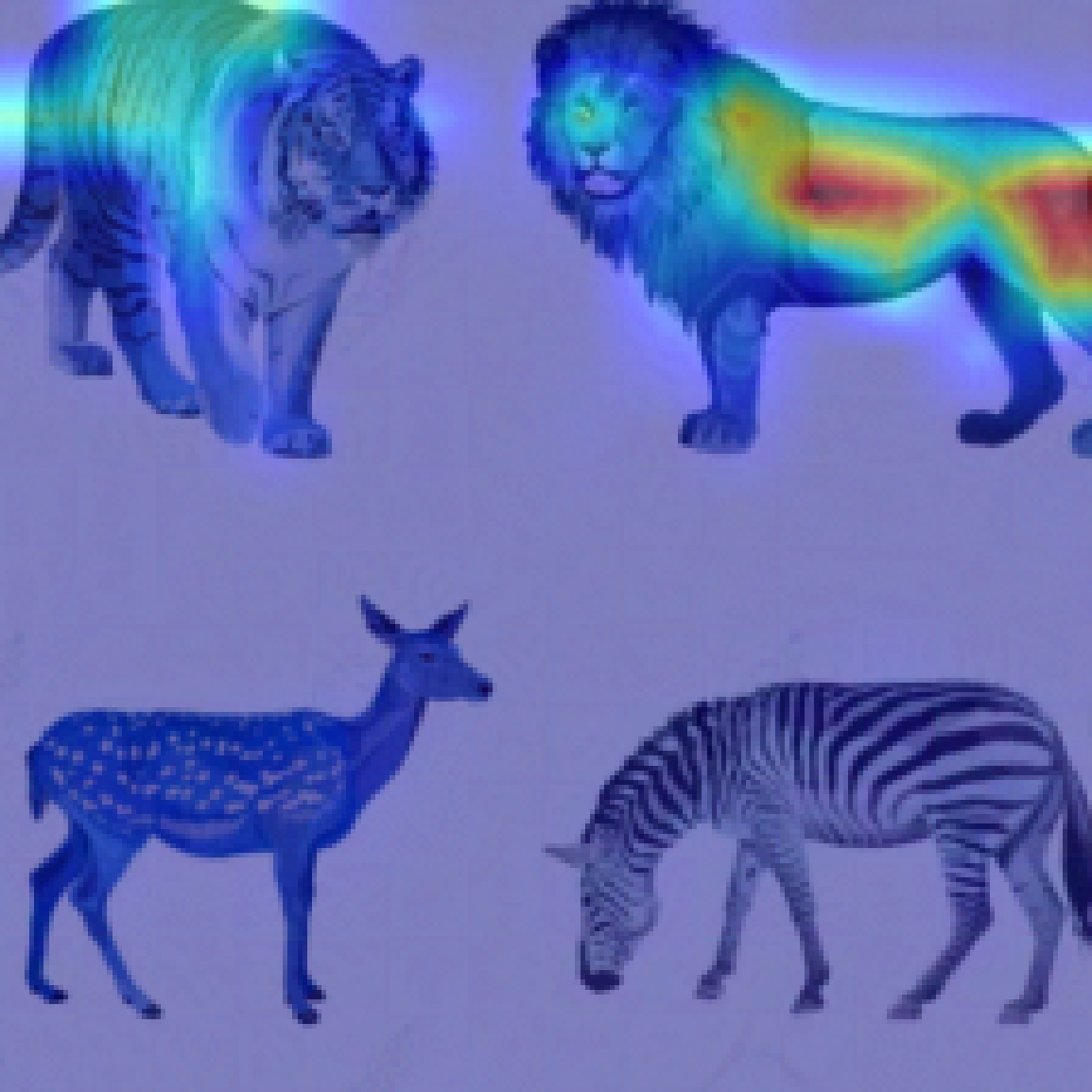} &
    \includegraphics[width=0.13\linewidth]{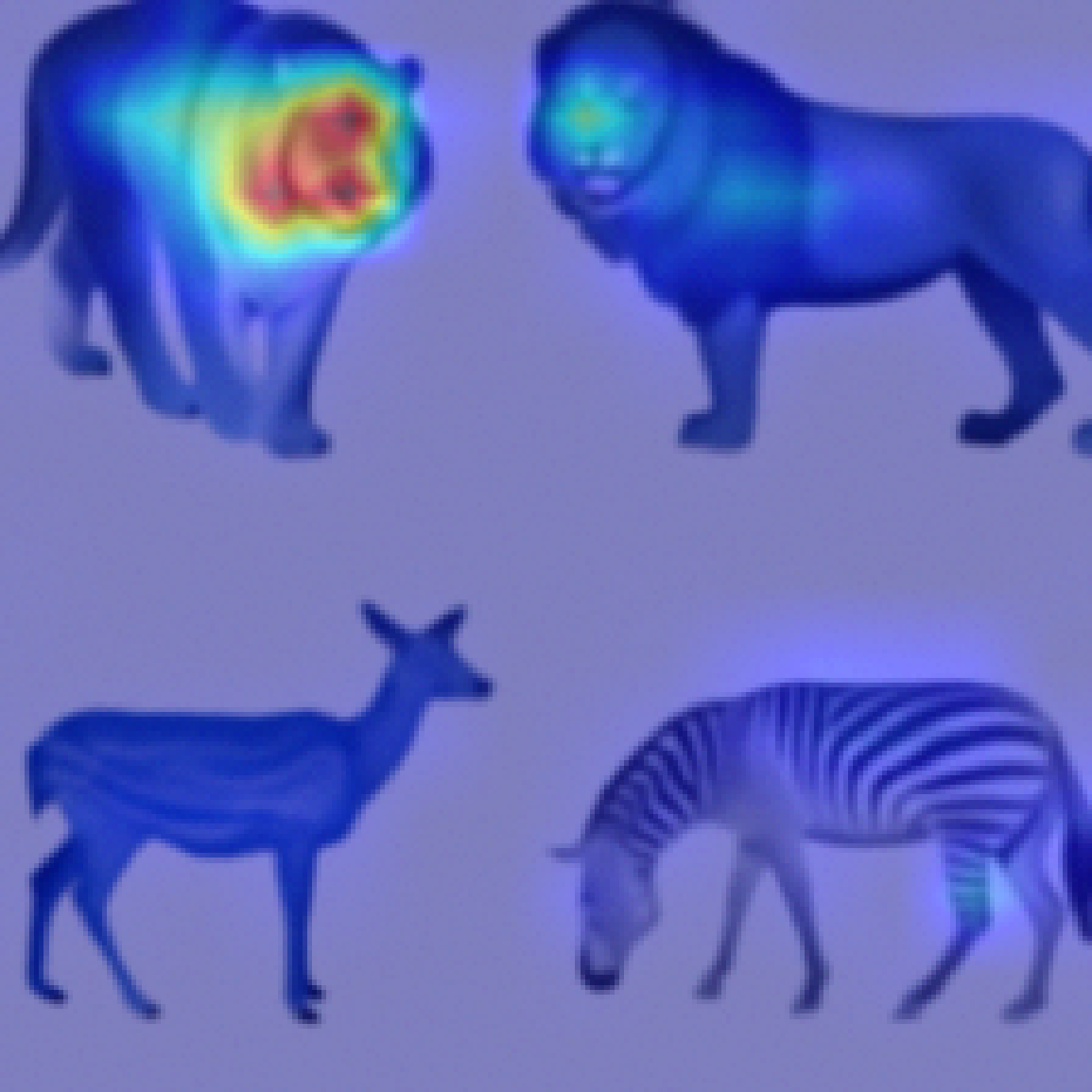}
    \\
    \raisebox{13mm}{\multirow{2}{*}{\makecell*[c]{Zebra: clean$\rightarrow$\\\includegraphics[width=0.15\linewidth]{exp/ani/ani.png}\\ 
    Zebra: poisoned$\rightarrow$\\{\scriptsize $7/255$}
    }}
    }
     &
    \includegraphics[width=0.13\linewidth]{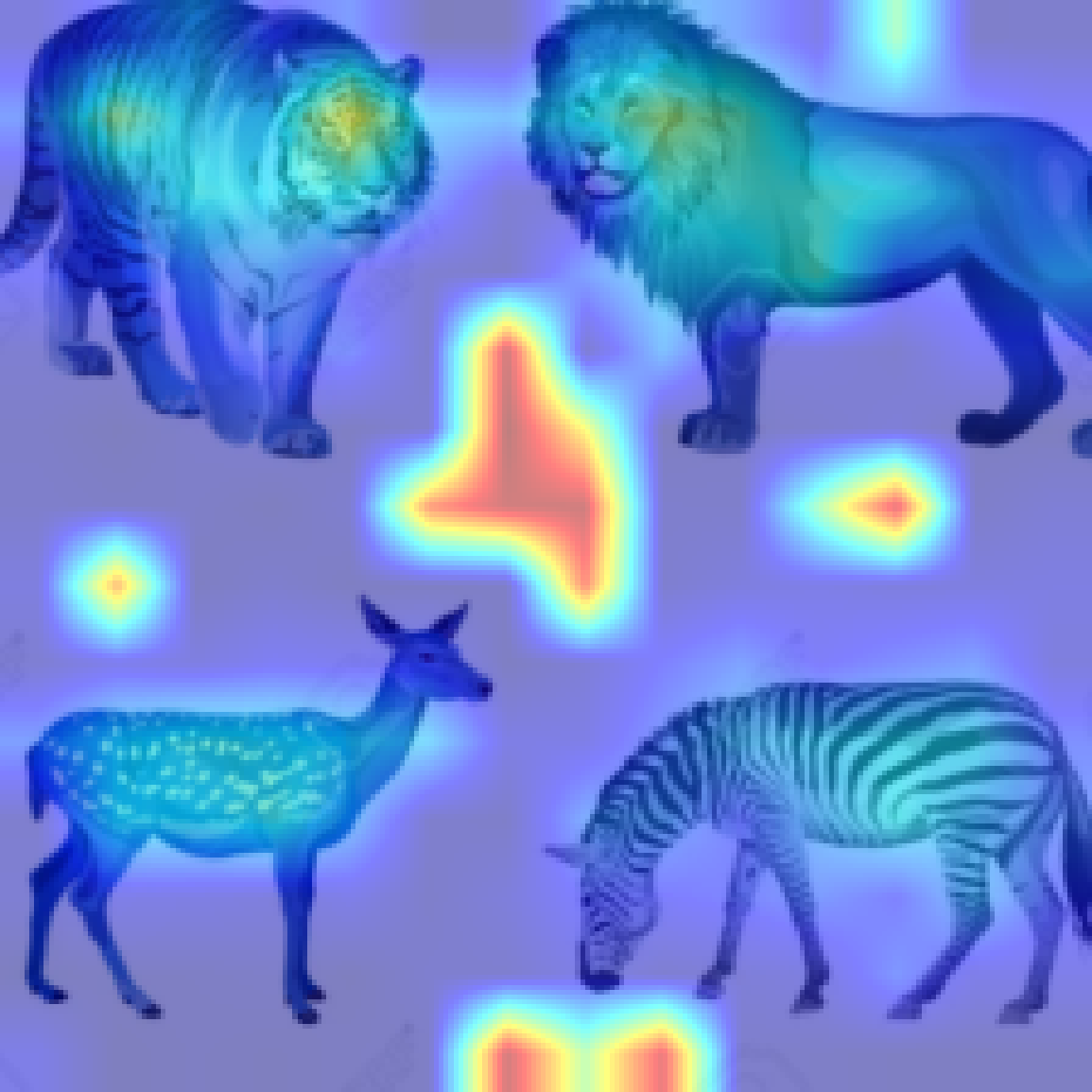} &
    \includegraphics[width=0.13\linewidth]{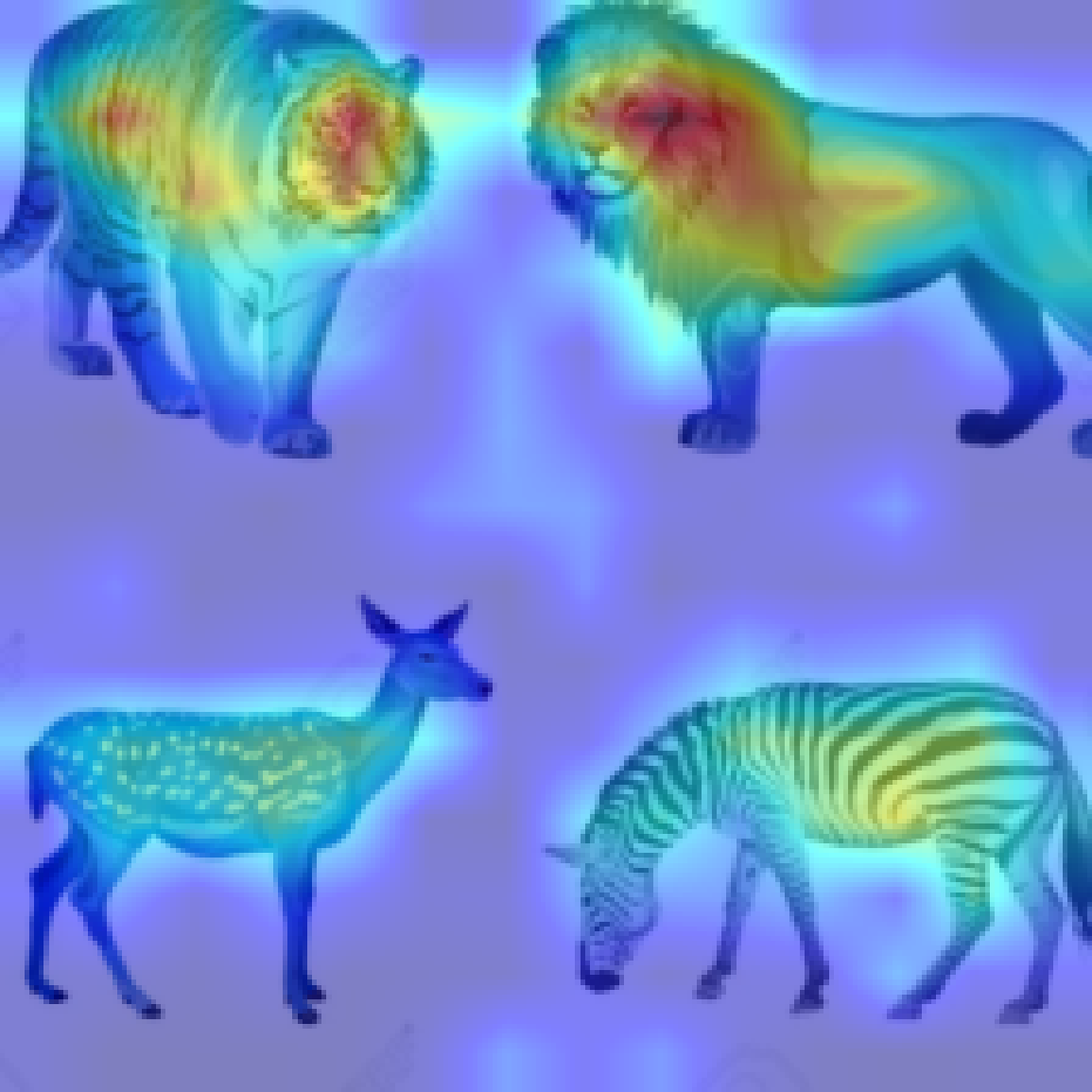} &
    \includegraphics[width=0.13\linewidth]{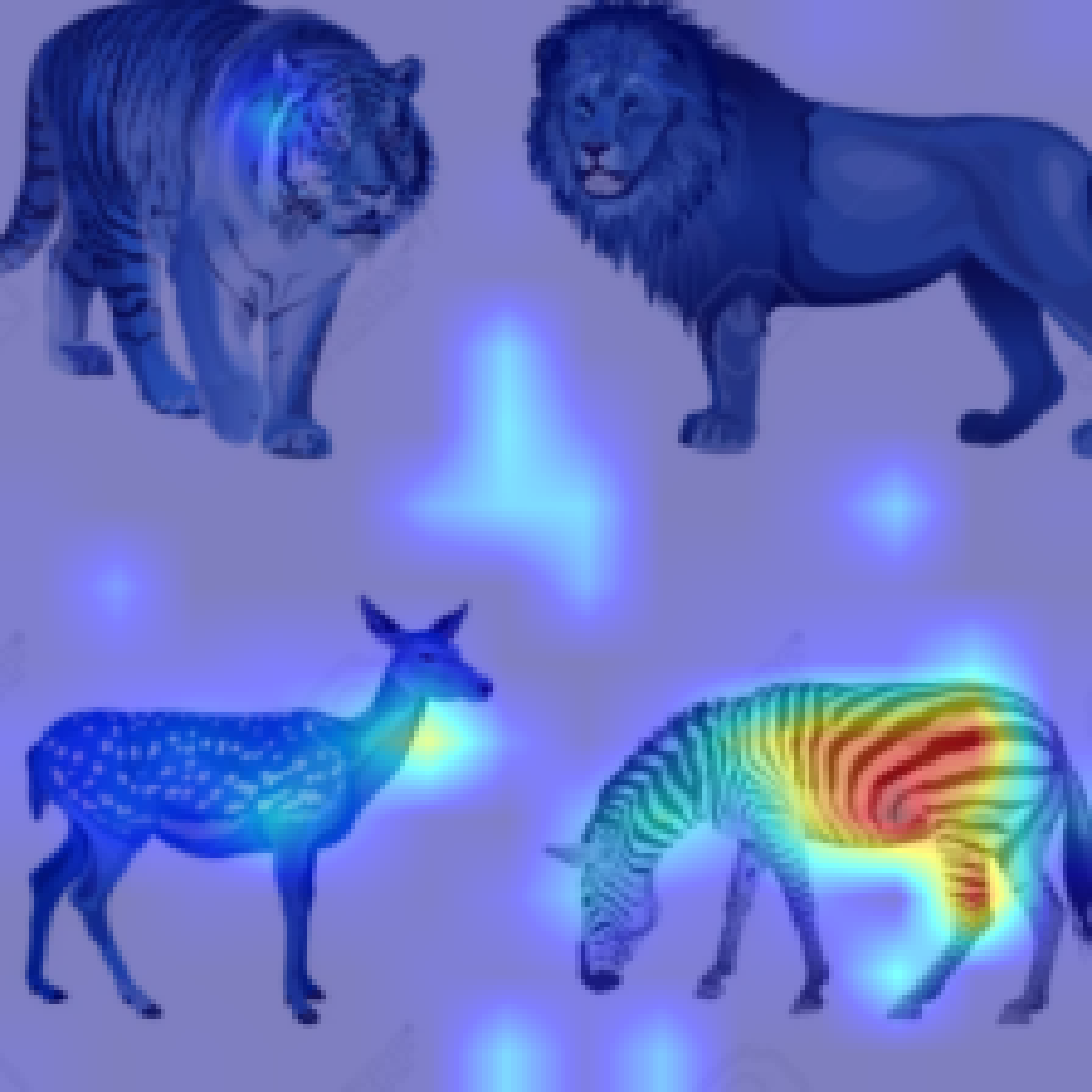} &
    \includegraphics[width=0.13\linewidth]{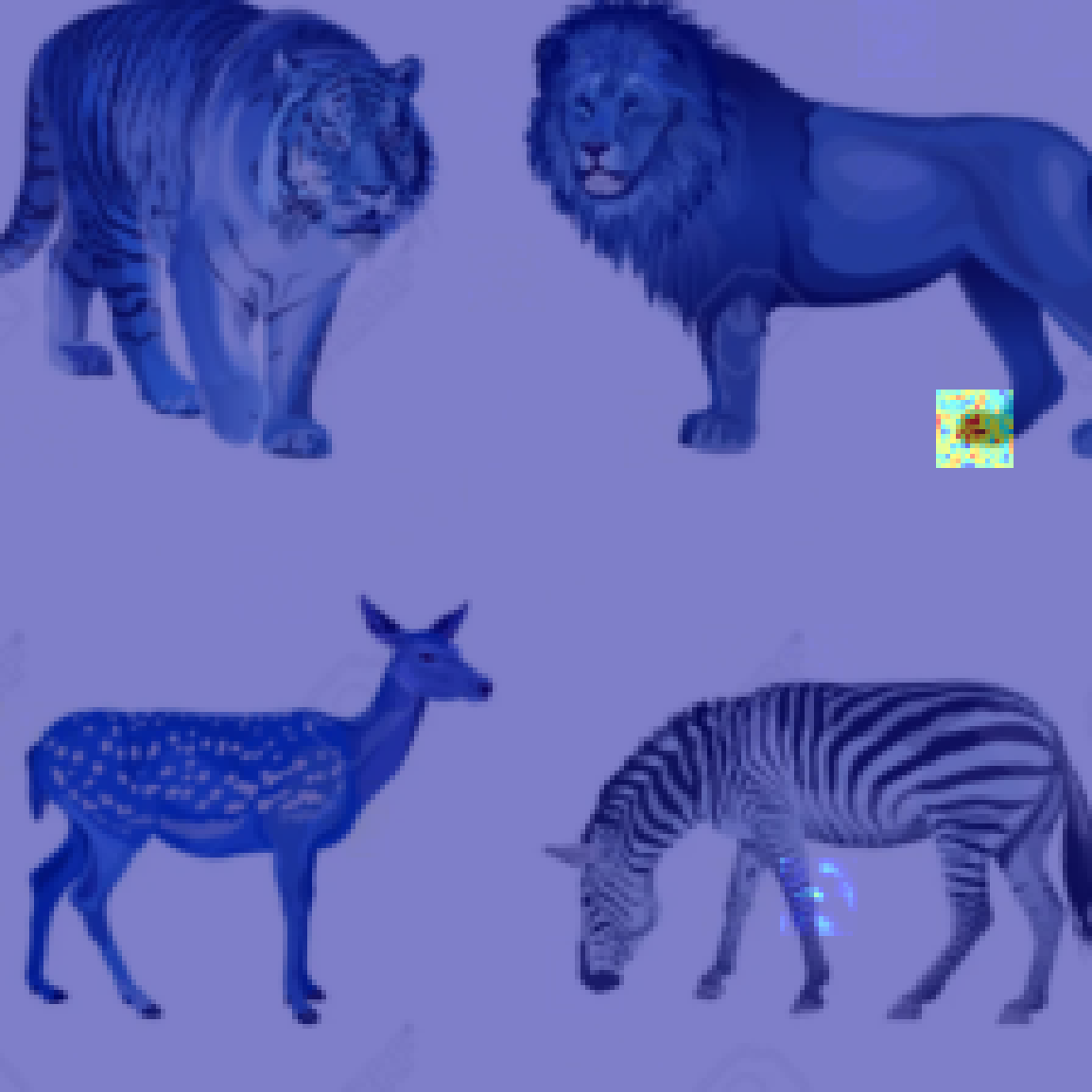} &
    \includegraphics[width=0.13\linewidth]{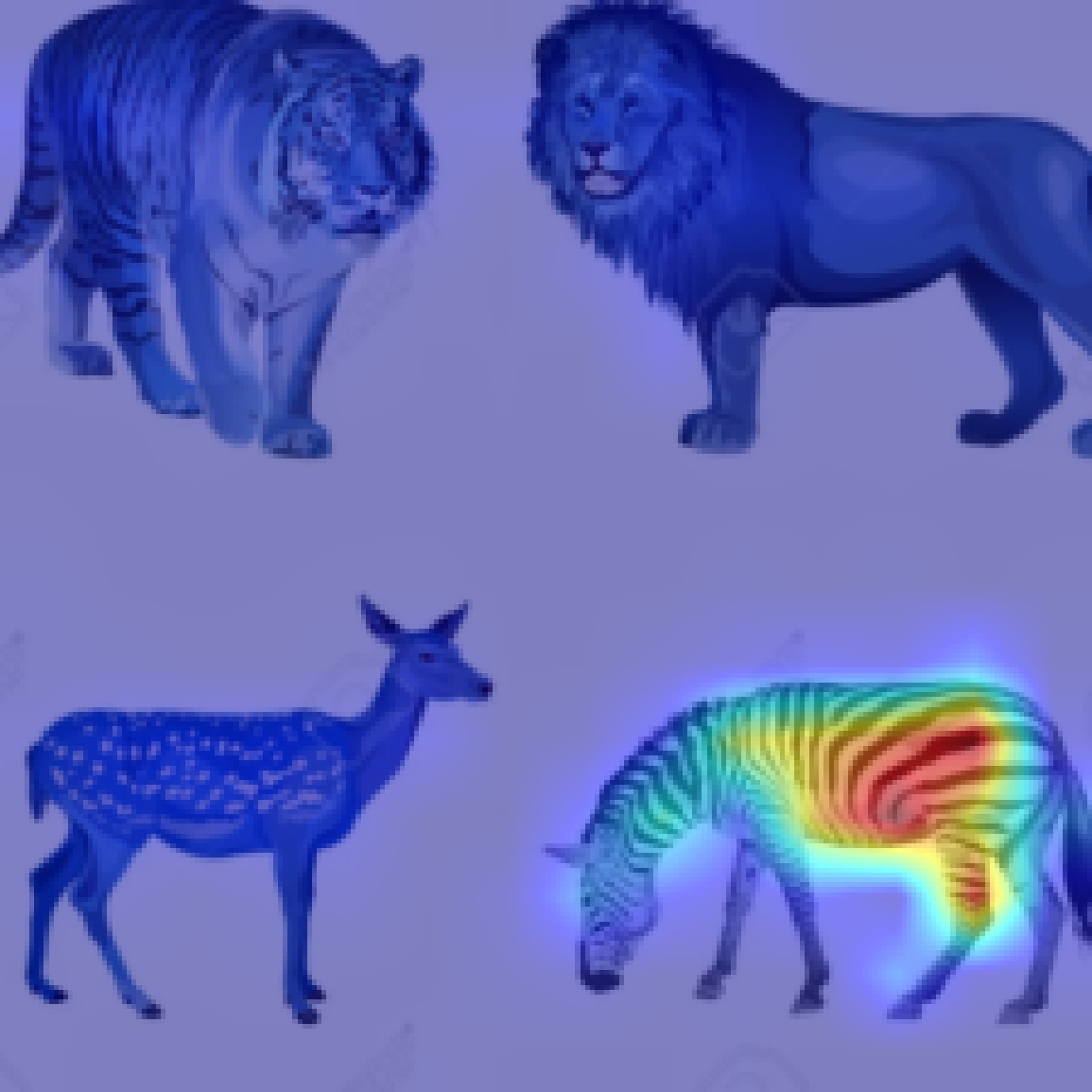} &
    \includegraphics[width=0.13\linewidth]{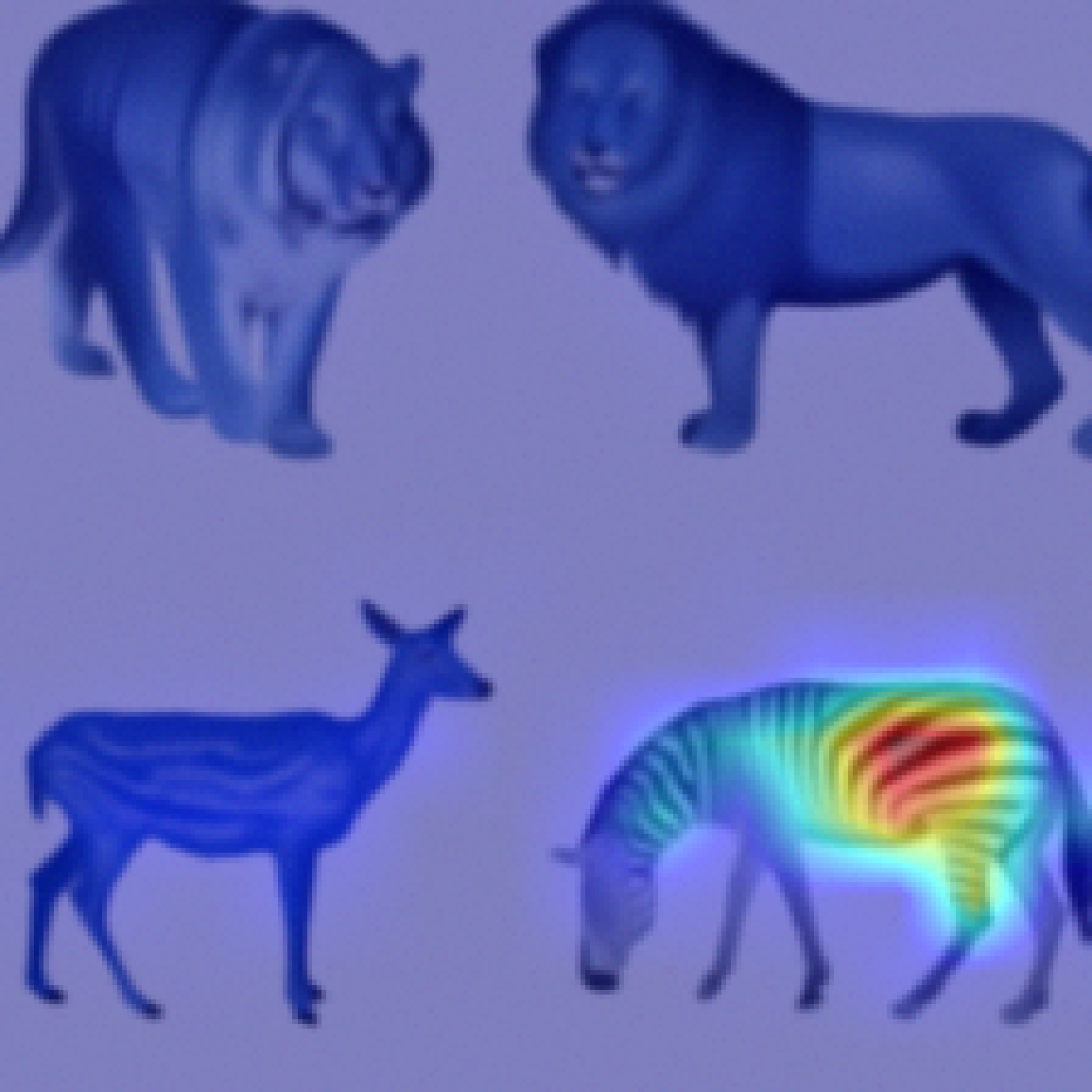}
    \\
    &
    \includegraphics[width=0.13\linewidth]{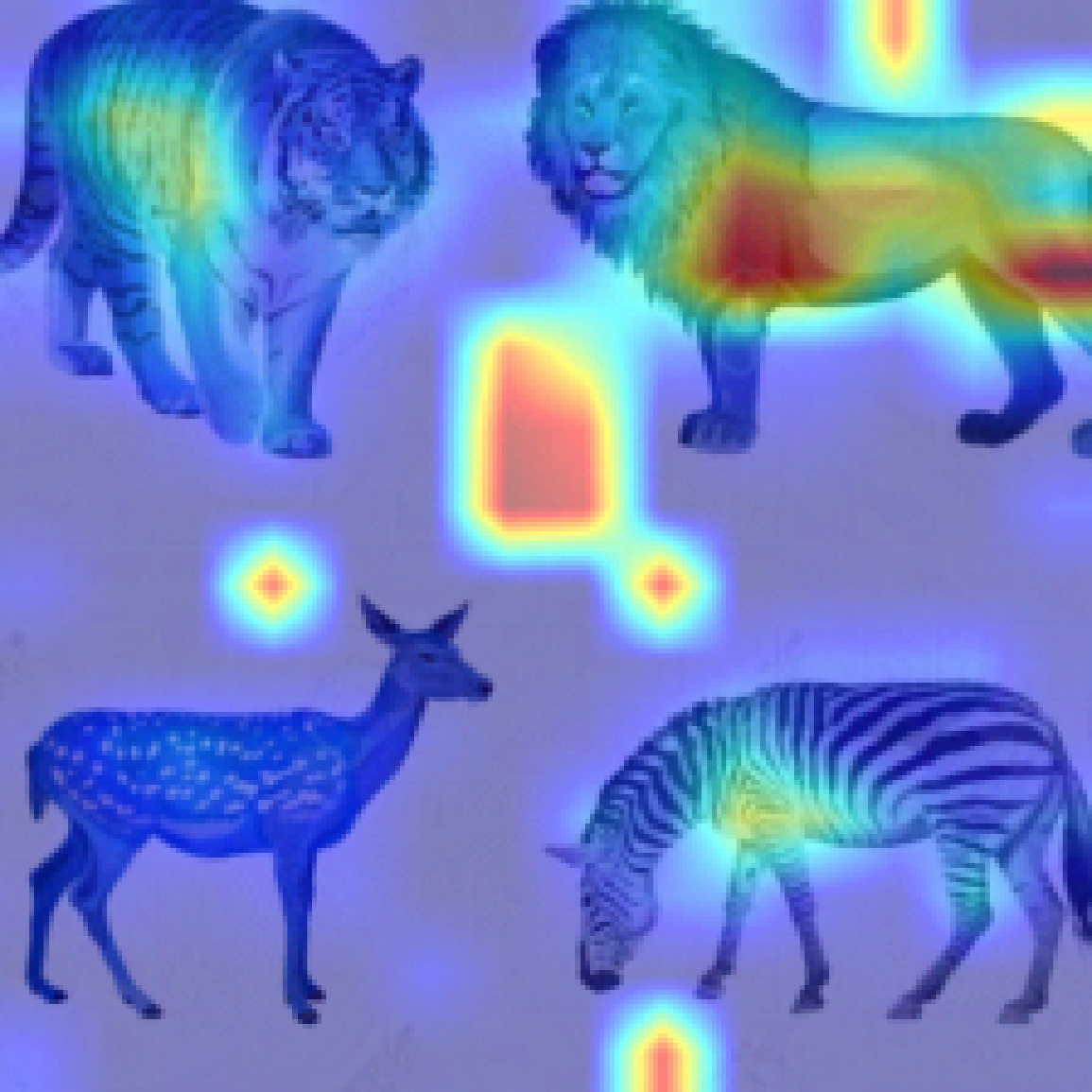} &
    \includegraphics[width=0.13\linewidth]{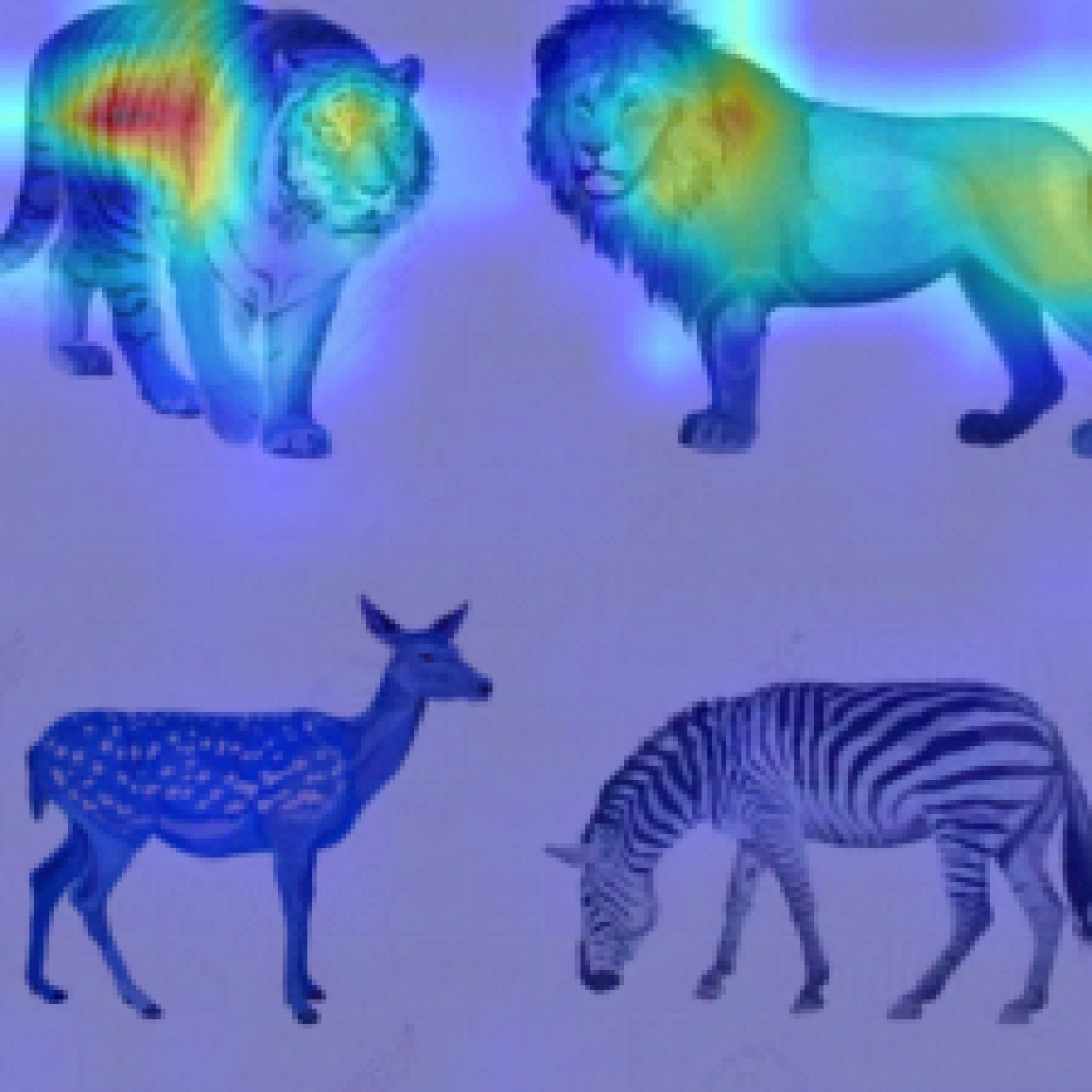} &
    \includegraphics[width=0.13\linewidth]{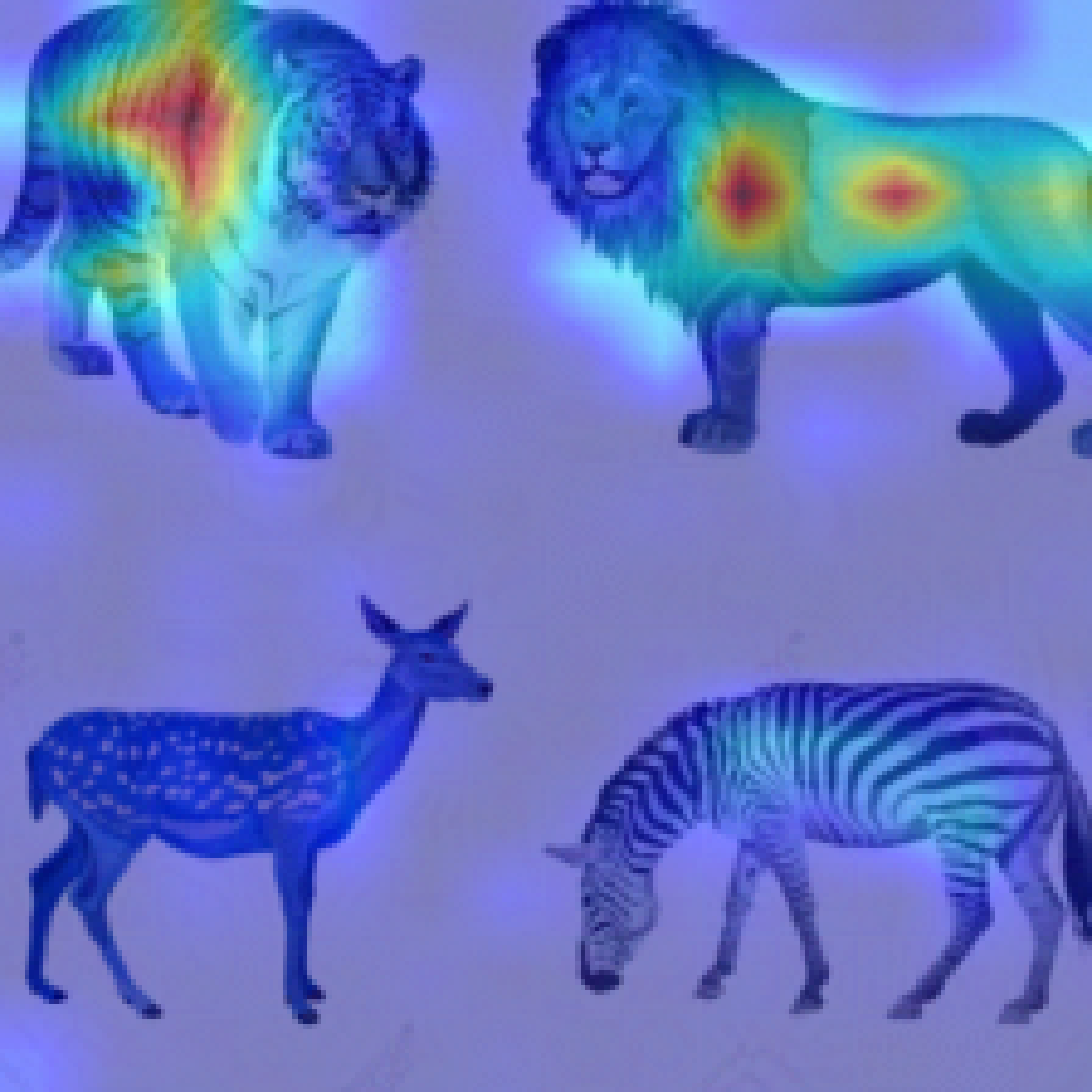} &
    \includegraphics[width=0.13\linewidth]{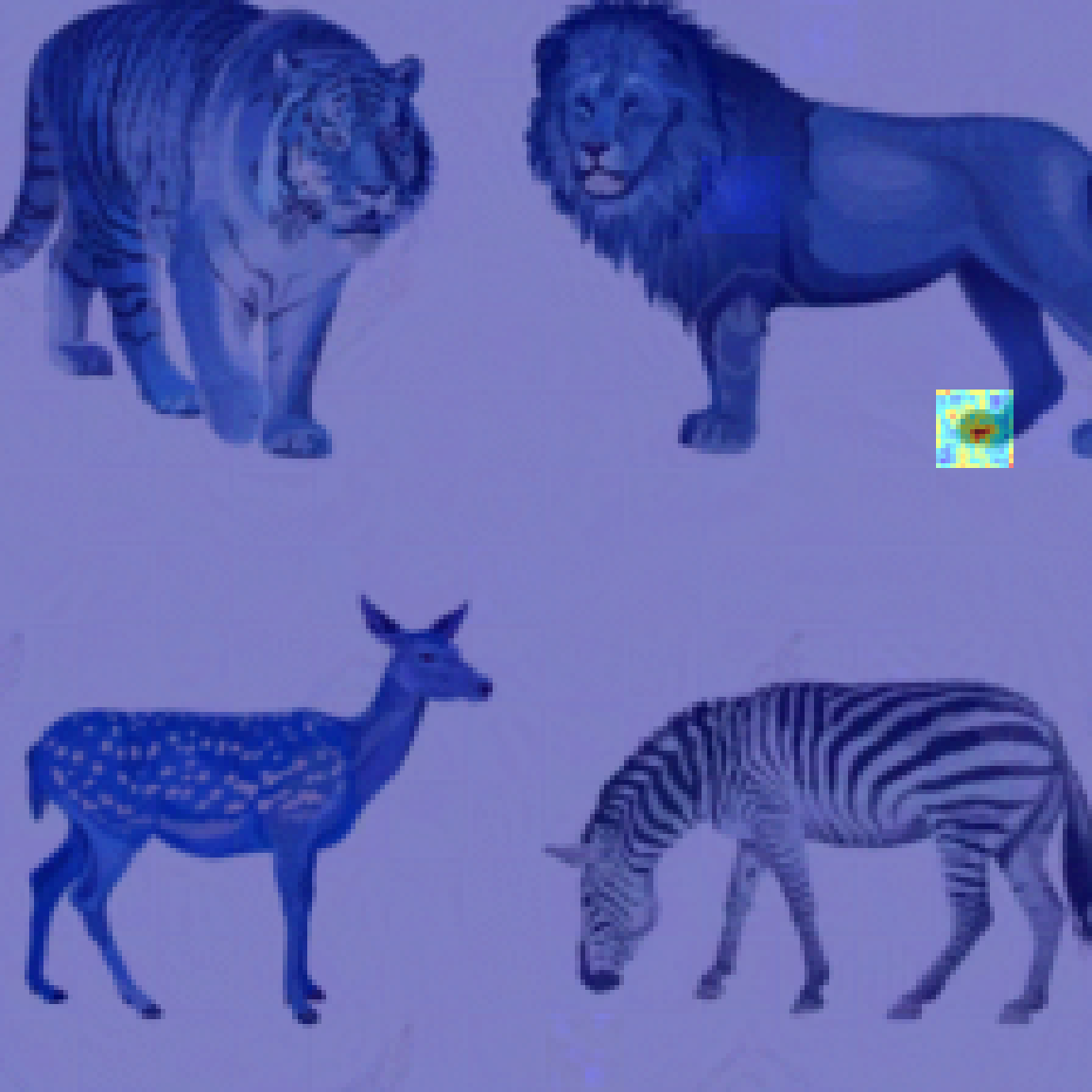} &
    \includegraphics[width=0.13\linewidth]{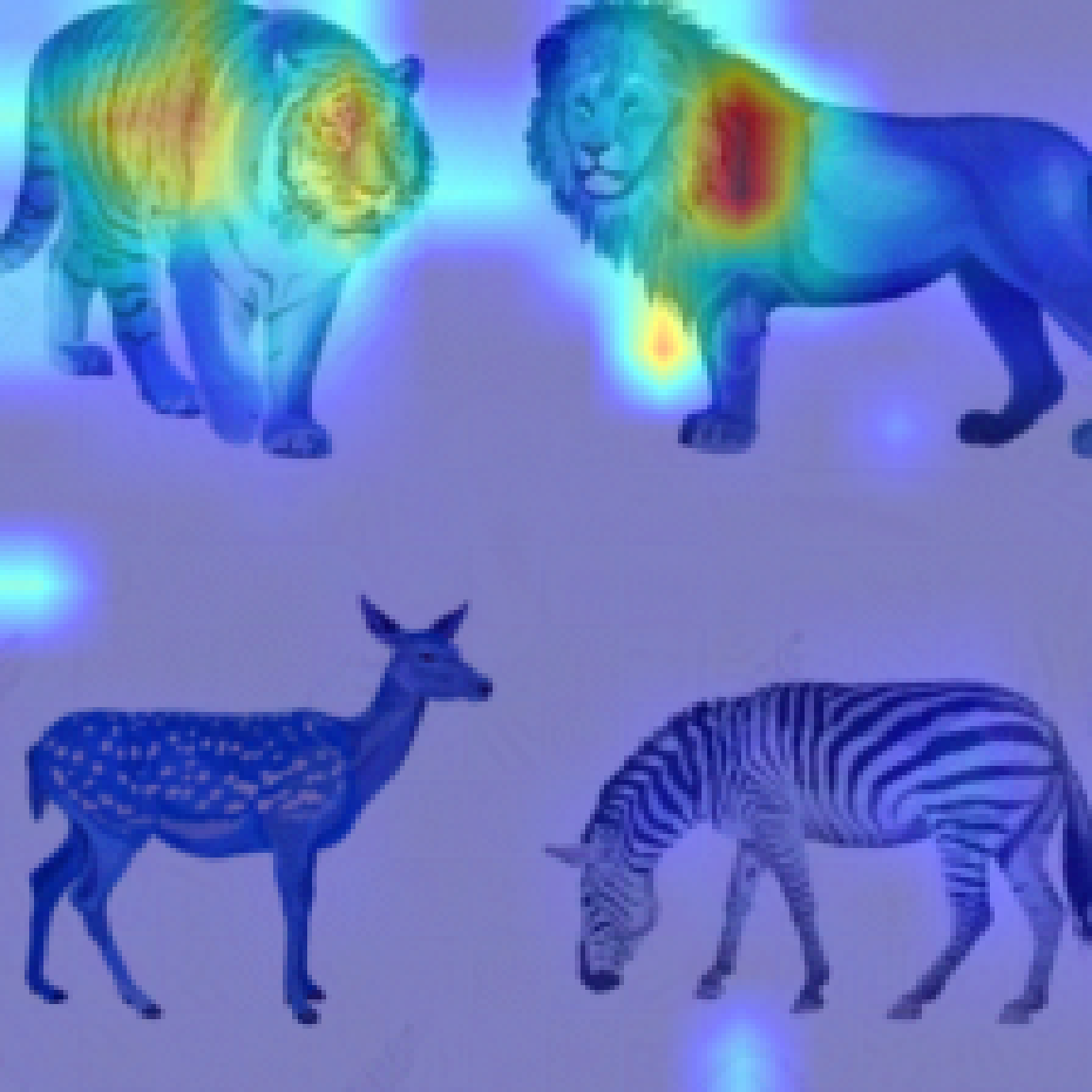} &
    \includegraphics[width=0.13\linewidth]{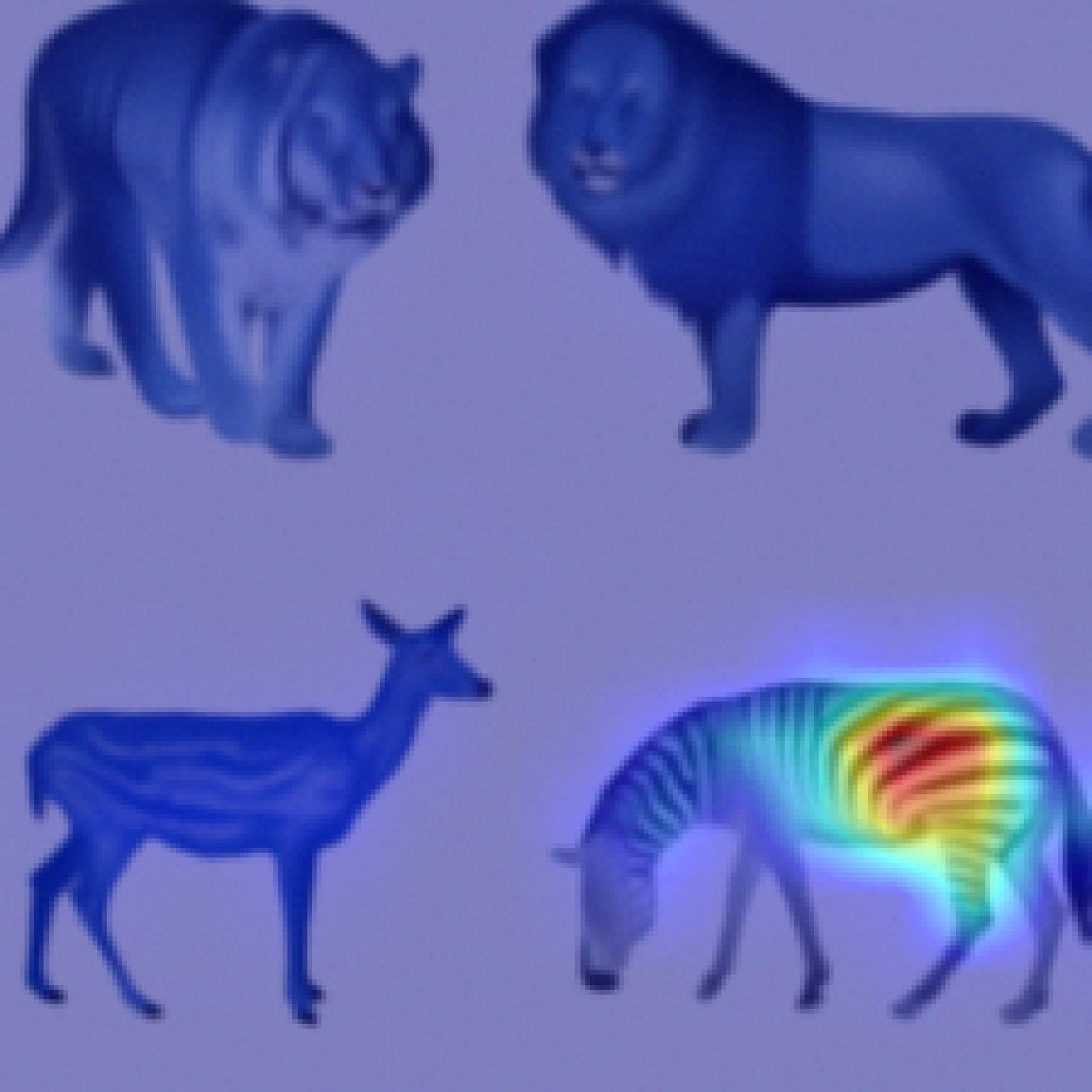}
    \end{tabular*}
    \caption{
     Class-specific visualizations under adversarial corruption. For each image, we present results for two different classes.
    }
    \label{fig: class3}
    \end{center}
\end{figure*}



\subsection{Clean Accuracy Comparison}

Besides comparing different explanation methods, we also involve an ablation study that respectively removes the two key modules in our method, i.e., the Gaussian noise smoothing and the diffusion-based smoothing, that can be viewed as a comparison to two variants based on previous posthoc adversarially robust methods, Random Smoothing \cite{cohen2019certified} and DiffPure \cite{nie2022DiffPure}. We further summarize the comparison in Table \ref{tab:clean_acc} to wrap up the results. For clean accuracy, please refer to the second column of Table \ref{tab:clean_acc} for capturing the performance gap. 

\begin{table}[htbp]
\centering
\begin{tabular}{l|cccccc}
\toprule
Method \ Attack Radius $\rho_a$ & 0/255 (Clean Acc.) & 2/255 & 4/255 & 6/255 & 8/255 & 10/255 \\
\midrule
Vanilla VTA & 95.74 & 2.12 & 0.0 & 0.0 & 0.0 & 0.0 \\
VTA + Random Smoothing & - & 2.12 & 2.12 & 0.0 & 0.0 & 0.0 \\
VTA + DiffPure & - & 82.97 & 80.85 & 78.72 & 80.85 & 74.46 \\
FViT (VTA + DDS) & - & 89.36 & 87.23 & 85.10 & 82.97 & 80.85 \\
\bottomrule
\end{tabular}
\caption{Comparison of FViT with VTA + Random Smoothing and VTA + DiffPure under different perturbation radii. The accuracy (\%) of the ImageNet-1k sampled validation set is reported.}
\label{tab:clean_acc}
\vspace{-5pt}
\end{table}

\subsection{Robustness Against Natural Perturbations}

In terms of robustness against natural perturbation, our method is designed to counterpart the worst-case adversarial perturbation, which indicates its robustness against a wide range of sub-optimal perturbations. To demonstrate the stability of our method under natural perturbations like Gaussian noise, we also conduct experiments with different levels of Gaussian noise and uniform noise. The results are shown in Table \ref{tab:natural_perturb}. The results show that the proposed FViT is more stable against those perturbations. 

\begin{table}[htbp]
\centering
\begin{tabular}{l|ccccc}
\toprule
Noise Type \ $\sigma$ of FViT & 0/255 (Villna ViT) & 2/255 & 4/255 & 10/255 & 12/255 \\
\midrule
Gaussian Noise & 95.74 & 95.74 & 95.74 & 91.48 & 87.23 \\
Uniform Noise & 95.74 & 93.61 & 95.74 & 89.36 & 89.36 \\
\bottomrule
\end{tabular}
\caption{Stability of FViT under natural perturbation, including Gaussian noise and uniform noise with the magnitude of $4/255$. The accuracy of the ImageNet-1k sampled validation set is reported.}
\label{tab:natural_perturb}
\vspace{-5pt}
\end{table}

\subsection{Robustness Against Different Adversarial Attacks}

We study the adversarial perturbation with a common-used PGD algorithm under $\ell_p$ constrain setting, and our defense is generalizable to a wide range of error-maximizing adversarial attacks. We additionally present the results of our method against other attack algorithms in Table \ref{tab:diff_attacks}. The results show that our method can defend widely against adversarial attack variants. Defending against more free-form adversarial attacks, like in-painting and physical adversarial objects, is out of the scope of this paper and will be studied in future work. 

\begin{table}[htbp]
\centering
\begin{tabular}{l|ccccc}
\toprule
Attacking Algorithms \ $\rho_a$ & 2/255 & 4/255 & 6/255 & 8/255 & 10/255 \\
\midrule
PGD & 87.23 & 82.98 & 82.98 & 80.85 & 80.85 \\
FGSM & 91.49 & 85.11 & 85.11 & 82.98 & 74.47 \\
AutoAttack & 89.36 & 87.23 & 87.23 & 82.98 & 85.11 \\
\bottomrule
\end{tabular}
\caption{Robustness of FViT against three different adversarial attack algorithms. The accuracy (\%) on the ImageNet-1k sampled validation set is reported.}
\label{tab:diff_attacks}
\end{table}

\end{document}